\documentclass[10pt,a4paper,twoside,openany]{book}

\pdfoutput=1

\usepackage[T1]{fontenc}
\usepackage[a4paper,inner=28mm,outer=24mm,top=26mm,bottom=28mm,headheight=14pt,headsep=8mm]{geometry}
\usepackage{newpxtext}
\usepackage{microtype}
\usepackage{fancyhdr}
\usepackage[explicit]{titlesec}
\usepackage{makeidx}
\usepackage{longtable}

\usepackage[latin1]{inputenc}
\usepackage{mystyle}
\let\otmltransp\transp
\let\transp\undefined
\usepackage{newpxmath}

\let\transp\otmltransp
\usepackage[scaled=.95]{inconsolata}
\usepackage{wrapfig}
\usepackage{notations_ot} 
\usepackage{tikz}

\definecolor{otblue}{HTML}{1F4E79}
\definecolor{otred}{HTML}{A43D3D}
\definecolor{otgray}{HTML}{5F6670}

\titleformat{\chapter}[display]
	{\normalfont\bfseries\color{otblue}}
	{\filleft\MakeUppercase{\chaptertitlename}\ \thechapter}
	{1ex}
	{\titlerule[0.8pt]\vspace{1.3ex}\huge\raggedright #1}
	[\vspace{1ex}\titlerule]
\titleformat{name=\chapter,numberless}[display]
	{\normalfont\bfseries\color{otblue}}
	{}
	{0pt}
	{\titlerule[0.8pt]\vspace{1.3ex}\huge\raggedright #1}
	[\vspace{1ex}\titlerule]
\titleformat{\section}
	{\normalfont\Large\bfseries\color{otblue}}
	{\thesection}{.75em}{#1}
\titleformat{\subsection}
	{\normalfont\large\bfseries\color{otgray}}
	{\thesubsection}{.75em}{#1}
\titlespacing*{\chapter}{0pt}{-8pt}{28pt}
\titlespacing*{\section}{0pt}{2.8ex plus .8ex minus .2ex}{1.2ex plus .2ex}
\titlespacing*{\subsection}{0pt}{2.2ex plus .6ex minus .2ex}{.8ex plus .2ex}

\fancypagestyle{plain}{%
	\fancyhf{}%
	\fancyhead[LE,RO]{\small\thepage}%
}

\usepackage[bookmarks,bookmarksdepth=2, colorlinks=true, linkcolor=blue,citecolor=red, urlcolor=blue]{hyperref}
\renewcommand{\a}{\VectMode{a}}
\renewcommand{\b}{\VectMode{b}}
\renewcommand{\c}{c}
\renewcommand{\d}{\ins{d}}

\renewcommand{\P}{\VectMode{P}}
\renewcommand{\S}{\VectMode{S}}

\renewcommand{\th}{\theta}

\newcommand{\todo}[1]{\textbf{\textcolor{red}{[ToDo: #1]}}}

\renewcommand{\todo}[1]{}
\newcommand{\removed}[1]{}

\newcommand{\dims}{d}
\makeindex

\title{Optimal Transport\\for Machine Learners}

\author{%
\begin{tabular}{c}
	Gabriel Peyr{\'e} \\ CNRS and ENS, PSL Universit\'e \\
	 \url{gabriel.peyre@ens.fr}
\end{tabular}
}

\date{\today}

\makeatletter
\renewcommand{\maketitle}{%
\begin{titlepage}
\thispagestyle{empty}
\vspace*{.08\textheight}
\noindent{\color{otblue}\rule{\textwidth}{1.1pt}}\par
\vspace{1.4cm}
{\centering\fontsize{30}{36}\selectfont\bfseries\color{otblue}\@title\par}
\vspace{1.2cm}
{\centering\large\@author\par}
\vfill
\begin{center}
{\color{otred}\rule{.38\textwidth}{2pt}}\par
\vspace{.65cm}
{\large\@date\par}
\end{center}
\end{titlepage}
}
\makeatother

\begin{document}

\hypersetup{pageanchor=false}
\pagenumbering{gobble}
\maketitle
\cleardoublepage
\hypersetup{pageanchor=true}

\frontmatter

\chapter*{Abstract}
\addcontentsline{toc}{chapter}{Abstract}
Modern machine learning repeatedly manipulates probability measures: empirical datasets, generated samples, latent distributions, class-conditional laws, particle systems, weights of wide networks and attention patterns. Optimal transport is useful in this setting because it compares such objects by asking how mass should move. It therefore combines a statistically meaningful notion of discrepancy with a geometry of interpolation, dual certificates and variational dynamics. This makes OT a common language for losses, generative modeling, domain adaptation, robust learning, barycenters, gradient flows and mean-field descriptions of learning algorithms.
\index{conditional law}
\index{dual!certificate}
\index{probability measure}
\index{gradient!flow}
\index{generative model}

This book presents the main OT techniques with these machine-learning uses in mind. It starts from finite assignment and the Monge map viewpoint, passes to Kantorovich couplings and dual potentials, and then explains the algorithmic ideas that make transport usable: linear programming, semi-discrete cells, Sinkhorn scaling and low-dimensional projections. The same objects are then reused as a geometry of measures, giving Wasserstein distances, barycenters, gradient flows, dynamic formulations and Gaussian/Bures formulas. The final chapters emphasize the variants most relevant to modern ML: divergences and adversarial losses, entropic and unbalanced relaxations, robust or spectral ground geometries, Gromov and quantum extensions, and transport-based views of generative models, mean-field networks and attention dynamics. The goal is to keep the mathematics explicit while exposing the computational and geometric intuitions needed to turn OT into a working toolbox for machine learners.
\index{Sinkhorn!scaling}
\index{adversarial!loss}
\index{Monge!problem}
\index{dual!potential}
\index{linear programming}
\index{Wasserstein!distance}
\index{gradient!flow}
\index{generative model}

\noindent All material for this book, including the code used to reproduce the figures, is available at \href{http://github.com/gpeyre/ot4ml}{\texttt{gpeyre/ot4ml}}. Most computational figures were produced with the Python Optimal Transport (POT) library~\cite{flamary2021pot}. The author warmly thanks the POT team and contributors for their important and sustained effort in making reliable optimal-transport algorithms available to the community.

\chapter*{Guide to the Literature and Scope}
\addcontentsline{toc}{chapter}{Guide to the Literature and Scope}

Several books already cover optimal transport from complementary viewpoints. The two-volume monograph of Rachev and R\"uschendorf~\cite{rachev1998mass,rachev1998mass2} gives a broad probabilistic treatment of mass transportation and its applications. Villani's books~\cite{Villani03,Villani09} are the standard references for the modern mathematical theory, from Kantorovich duality to curvature, concentration and geometric analysis. Santambrogio's text~\cite{SantambrogioBook} offers a concise applied-mathematics route through the same foundations, with a strong emphasis on PDEs and variational arguments. Ambrosio, Gigli and Savar\'e~\cite{ambrosio2006gradient} develop the metric-space theory of gradient flows that underlies the dynamical part of the subject.
\index{Kantorovich!duality}
\index{gradient!flow}

On the computational side, Peyr\'e and Cuturi~\cite{peyre2019computational} provide the reference account of numerical OT, entropic regularization and applications in data sciences. Galichon's book~\cite{galichon2016optimal} explains the economic and matching-theoretic viewpoint, while the statistical theory of OT is developed in the recent lecture notes of Chewi, Niles-Weed and Rigollet~\cite{weed2025statistical}. Recent surveys complement these books by emphasizing scalable algorithms and machine-learning applications~\cite{khamis2024scalable,montesuma2023recent}, as well as the role of OT in imaging and graphics~\cite{bonneel2023survey}. These references remain the natural places to find exhaustive proofs, historical details and specialized variants.
\index{entropic!regularization}

The aim here is different and more selective. The book keeps the core mathematics explicit, but organizes it around the questions that repeatedly arise in machine learning: how to compare singular empirical measures, how to compute differentiable transport losses, how regularization changes optimization and statistics, how dual potentials become discriminators, and how transport geometry produces flows of particles, neurons and tokens. The intended contribution is therefore not a replacement for the references above, but a compact bridge between rigorous OT and the geometric intuitions needed to use it in modern ML.
\index{neuron}
\index{dual!potential}
\index{empirical!measure}

\tableofcontents

\mainmatter


\chapter{Optimal Matching between Point Clouds}
\index{matching!optimal}
\label{sec-matching}

This opening chapter isolates the simplest form of optimal transport: pairing two finite point clouds. The stakes are algorithmic and geometric at once: one sees the combinatorial nature of transport, the special simplicity of the line, and the first hints that convex relaxation will be necessary in higher dimension. Classical assignment algorithms such as the Hungarian method and auction methods~\cite{Kuhn1955,bertsekas1992auction} provide the computational backdrop, while the geometric examples prepare the Kantorovich relaxation.
\index{Hungarian primal-dual method}
\index{Kantorovich!relaxation}

\section{Monge Problem for Discrete Points}
\index{Monge!problem}
\label{sec-monge-pbm}

This section formulates matching as Monge's deterministic transport problem on two equally weighted clouds. The one-dimensional case is a transparent reference case where the optimal map can be read off by sorting.

\paragraph{Matching Problem}
\index{assignment problem}

Given a cost matrix $(\C_{i,j})_{i \in \range{n}, j \in \range{m}}$ and assuming $n=m$, the optimal assignment problem aims to find a bijection $\sigma$ within the set $\Perm(n)$ of permutations of $n$ elements that solves
\index{cost matrix}
\index{assignment problem}
\eql{\label{eq-optimal-assignment}
	\umin{\sigma \in \Perm(n)} \frac{1}{n}\sum_{i=1}^n \C_{i,\sigma(i)}.
}
One could naively evaluate the cost function above using all permutations in the set $\Perm(n)$. However, this set has size $n!$, which becomes enormous even for small values of $n$.
In general, the optimal $\sigma$ is not unique.

\paragraph{1D Case}

In 1D, for convex cost, the matching defines a monotonic map.

\begin{proposition}[Monotone matching on the line]\label{prop-matching-1d-monotone}
\index{monotone!matching}
Assume that the points $(x_i)_i$ and $(y_j)_j$ are pairwise distinct. If the cost is of the form $\C_{i,j}=h(x_i-y_j)$, where $h: \RR \rightarrow \RR^+$ is strictly convex (for example, $\C_{i,j}=|x_i-y_j|^p$ for $p > 1$), then any optimal $\sigma$ defines a strictly increasing map $x_i \mapsto y_{\sigma(i)}$ (and thus is unique), i.e.,
\eq{
	\forall (i,i'), \quad (x_i-x_{i'})(y_{\sigma(i)}-y_{\sigma(i')}) > 0.
}
\end{proposition}
\begin{proof}
Indeed, if this property is violated, i.e., there exists $(i,i')$ such that $(x_i-x_{i'})(y_{\sigma(i)}-y_{\sigma(i')}) < 0$, then one can define a permutation $\tilde{\sigma}$ by swapping the match, i.e., $\tilde{\sigma}(i)=\sigma(i')$ and $\tilde{\sigma}(i')=\sigma(i)$, yielding a strictly better cost, as proved in the following fact.
Let $h:\RR\to\RR$ be strictly convex and let $x < x'$ and $y < y'$. Then
\[
  h(x-y)+h(x'-y') < h(x-y')+h(x'-y).
\]
We set the gap $d:=y'-y > 0$ and define for every $s\in\RR$
\[
	  D(s):=\frac{h(s)-h(s-d)}{d}
	\quad\text{and}\quad
	\Delta:=h(x-y')+h(x'-y)-h(x-y)-h(x'-y') = d ( D(x'-y)-D(x-y) ).
\]
Because $h$ is strictly convex, $D$ is strictly increasing.
Since $x-y < x'-y$, monotonicity yields $D(x-y) < D(x'-y)$, that is $\Delta > 0$.
\end{proof}

This property extends by continuity to convex (not strictly convex) costs such as $|x-y|$, but in this case, the matching is not necessarily unique.
For convex costs, the algorithm to compute an optimal transport is therefore to sort the points, i.e., find some pair of permutations $\sigma_X, \sigma_Y$ such that
\eq{
	x_{\sigma_X(1)} \leq x_{\sigma_X(2)} \leq \ldots
	\qandq
	y_{\sigma_Y(1)} \leq y_{\sigma_Y(2)} \leq \ldots
}
and then an optimal match is mapping $x_{\sigma_X(k)} \mapsto y_{\sigma_Y(k)}$, i.e., an optimal transport is $\sigma = \sigma_Y \circ \sigma_X^{-1}$. The total computational cost is thus $O(n\log(n))$, using, for instance, the quicksort algorithm.

\begin{alg}[One-dimensional sorting assignment]\label{alg:one-dimensional-sorting}
\textbf{Input:} Equal-weight point clouds $(x_i)_{i=1}^n$, $(y_j)_{j=1}^n$ on $\RR$; convex cost $h(x-y)$.

\textbf{Output:} Optimal permutation $\sigma$.

\textbf{Sort} source and target points:
\(x_{\sigma_X(1)}\leq\cdots\leq x_{\sigma_X(n)}, \qquad y_{\sigma_Y(1)}\leq\cdots\leq y_{\sigma_Y(n)}.\)

\textbf{For} $k=1,\ldots,n$ \textbf{do}:
\begin{algblock}

\textbf{Match} $x_{\sigma_X(k)}$ with $y_{\sigma_Y(k)}$.

\end{algblock}
\algreturnskip
\textbf{Return} \(\sigma=\sigma_Y\circ\sigma_X^{-1}.\)
\end{alg}

If the distance profile is concave instead of convex, the geometry changes. For costs such as $c_p(x,y)=|x-y|^p$ with $0<p<1$, splitting a displacement into several smaller displacements is expensive, so optimal matchings tend to create long exchanges rather than the monotone equal-rank pairing; see Figure~\ref{fig:matching-1d-convex-concave-costs}. This is the strictly concave regime studied by Gangbo and McCann~\cite{gangbo1996geometry}.

The real line still gives special structure. After sorting all red and blue points together, the ordered sequence decomposes into maximal alternating chains, and local matching indicators can certify pairs that must be matched in an optimum. Removing such certified pairs and repeating yields an exact hierarchical algorithm for the unit-mass balanced assignment problem, with worst-case complexity $O(n^2)$ in the framework of Delon, Salomon and Sobolevski~\cite{delon-concave}. Very concave costs also motivate simpler greedy heuristics, studied for instance by Ottolini and Steinerberger~\cite{OttoliniSteinerberger2023GreedyConcave}. The point is that these methods are not generic linear-programming solvers; they use the one-dimensional order and the concavity of the distance profile.
\index{mass!balance}
\index{local!matching indicator}
\index{assignment problem}

The resulting constructive rule is summarized as Algorithm~\ref{alg:concave-line-local-indicators}.

\begin{alg}[Concave line matching by local indicators]\label{alg:concave-line-local-indicators}
\textbf{Input:} Unit-mass source and target points on $\RR$; strictly concave distance cost.

\textbf{Output:} Optimal concave-cost matching $M$.

\textbf{Sort} combined red-blue sequence on the line.

\textbf{Decompose} it into maximal alternating chains.

\textbf{Initialize:} Set $M=\emptyset$.

\textbf{While} unmatched points remain \textbf{do}:
\begin{algblock}

\textbf{For} each active alternating chain \textbf{do}:
\begin{algblock}

\textbf{Compute} local matching indicators~\cite{delon-concave}.

\textbf{If} an indicator certifies a pair $(i,j)$ \textbf{then}:
\begin{algblock}

\textbf{Update} $M\leftarrow M\cup\{(i,j)\}$.

\textbf{Remove} points $i$ and $j$.

\end{algblock}
\end{algblock}
\textbf{Recompute} only chains affected by removals.

\end{algblock}
\algreturnskip
\textbf{Return} $M$.
\end{alg}

\begin{figure}[ht]
\centering
\begin{tabular}{@{}cc@{}}
\includegraphics[width=.46\linewidth]{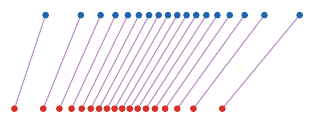} &
\includegraphics[width=.46\linewidth]{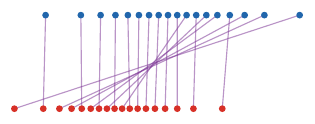} \\[-.1em]
\small convex cost, $p=2$ &
\small concave cost, $p=1/2$ \\[.25em]
\includegraphics[width=.46\linewidth]{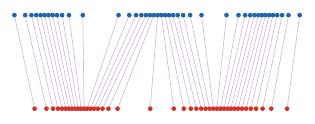} &
\includegraphics[width=.46\linewidth]{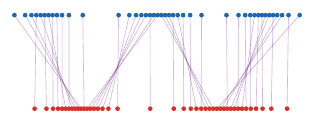} \\[-.1em]
\small two-mode to three-mode, $p=2$ &
\small two-mode to three-mode, $p=1/2$
\end{tabular}
\caption{One-dimensional assignments for ordered source and target clouds with costs $c_p(x,y)=|x-y|^p$. The top row uses single-Gaussian source and target clouds; the bottom row uses a denser two-component source and three-component target. For the convex quadratic cost, equal ranks are matched and the segments do not cross. For the concave cost, the optimum creates long crossing exchanges; the ordered line remains useful, but through the alternating-chain structure of concave transport rather than through monotone rearrangement.}
\index{monotone!rearrangement}
\index{cost!quadratic}
\label{fig:matching-1d-convex-concave-costs}
\end{figure}

\begin{figure}[ht]
\centering
\begin{tabular}{@{}cc@{}}
\small two mixtures & \small one mode to three modes \\[-.15em]
\includegraphics[width=.46\linewidth]{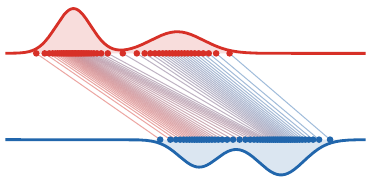} &
\includegraphics[width=.46\linewidth]{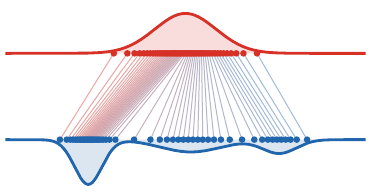}
\end{tabular}
\caption{One-dimensional optimal matching by quantile sorting. The red and blue curves are smooth laws used to generate equal-weight empirical measures; the dots are inverse-CDF samples at common quantile levels. The monotone assignment connects equal ranks, both for two Gaussian mixtures and for the transport from one central Gaussian toward a three-mode target law.}
\index{Gaussian mixture}
\index{empirical!measure}
\label{fig:matching-1d-quantile-assignment}
\end{figure}

Note that if $\phi : \RR \rightarrow \RR$ is an increasing map, one can apply this technique to costs of the form $h(|\phi(x)-\phi(y)|)$ with a change of variable.
A typical application is grayscale histogram equalization. The empirical cumulative distribution of the luminance values of an image is transported to a prescribed target histogram, for instance a concentrated or reference-image histogram. In one dimension, the monotone rearrangement above gives the exact transport map, so the operation is both computationally simple and geometrically faithful: it matches distributions of intensities rather than individual pixels.
\index{histogram}
\index{transport map}
\index{monotone!rearrangement}

\begin{figure}[H]
\centering
\begin{tabular}{@{}cccc@{}}
\small $t=0$ & \small $t=1/3$ & \small $t=2/3$ & \small $t=1$ \\[-.15em]
\includegraphics[width=.21\linewidth]{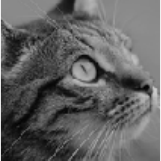} &
\includegraphics[width=.21\linewidth]{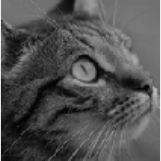} &
\includegraphics[width=.21\linewidth]{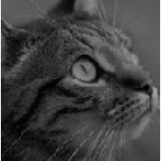} &
\includegraphics[width=.21\linewidth]{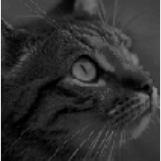} \\[-.1em]
\includegraphics[width=.21\linewidth]{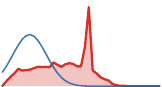} &
\includegraphics[width=.21\linewidth]{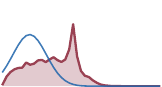} &
\includegraphics[width=.21\linewidth]{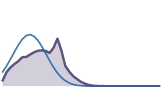} &
\includegraphics[width=.21\linewidth]{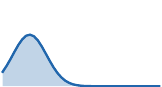}
\end{tabular}
\caption{Histogram equalization as one-dimensional Monge transport on pixel intensities. The map is the monotone rearrangement $T=Q_\beta\circ F_\alpha$, using the cumulative and quantile functions defined later in Definition~\ref{def-cdf-quantile}; here $\beta$ is a truncated Gaussian concentrated near dark intensities. The images are interpolated pointwise by $I_t=(1-t)I+tT(I)$, and all histograms share the same vertical scale to make the displacement of intensity mass comparable across time.}
\index{cumulative!function}
\index{quantile!function}
\index{monotone!rearrangement}
\label{fig:monge-histogram-equalization}
\end{figure}

Note that if $h$ is strictly convex, then all optimal assignments are increasing, and if the points are all distinct, this increasing map is unique. If $h$ is not strictly convex, for instance $c(x,y)=|x-y|$, non-increasing optimal assignments can also exist. This happens, for example, in the book-shifting problem with overlapping uniform distributions, where the mass in the intersection can stay fixed.

\paragraph{Optimal transport on the circle.}\label{rem-circle-ot-cut}
\index{circle!transport}

The sorting rule on the line has a periodic analogue. Identify the circle with $\mathbb S^1=\RR/\ZZ$, let
\[
	d_{\mathbb S^1}(x,y):=\min_{k\in\ZZ}|x-y+k|,
	\qquad
	c_p(x,y):=d_{\mathbb S^1}(x,y)^p,\qquad p>1.
\]
The only extra datum, compared with the line, is where one opens the circle. Once a cut has been chosen, the circle is unfolded into an interval and the one-dimensional monotone assignment can be used. In the discrete case, changing the cut is the same as applying a cyclic shift to one of the two circular orderings.
\index{cyclic!shift}

\begin{proposition}[Discrete circle transport by a cut]\label{prop-circle-ot-cut}
\index{circle!transport}
	Let $x_1,\ldots,x_n$ and $y_1,\ldots,y_n$ be two families of distinct points on $\mathbb S^1$, with equal weights. Let $x_{(1)},\ldots,x_{(n)}$ and $y_{(1)},\ldots,y_{(n)}$ denote any fixed cyclic orderings, with the convention $y_{(k+n)}=y_{(k)}$. For the cost $c_p$, $p>1$, an optimal assignment is one of the cyclic shifts
	\[
		x_{(k)}\longmapsto y_{(k+s)},\qquad k\in\range{n},\quad s\in\{0,\ldots,n-1\},
	\]
	and is found by minimizing
	\[
		\sum_{k=1}^n d_{\mathbb S^1}\!\left(x_{(k)},y_{(k+s)}\right)^p
	\]
	over the $n$ possible shifts. Equivalently, for an optimal shift one can choose a cut $\theta\in\mathbb S^1\setminus(\{x_i\}_i\cup\{y_j\}_j)$ so that, after lifting all points to $(\theta,\theta+1)$ and sorting them, the optimal matching is the equal-rank monotone matching on this interval.
\index{monotone!matching}
\end{proposition}

\begin{proof}
	Call two matched pairs cyclically inverted if the circular order of their source endpoints is opposite to the circular order of their target endpoints. Among optimal assignments, choose one with the smallest number of such inversions. The elementary exchange step is the circular analogue of the line argument in Proposition~\ref{prop-matching-1d-monotone}: if two matched pairs are inverted, then cutting the circle in a gap which does not meet the four endpoints and choosing integer lifts realizes the four geodesic distances involved in the exchange as ordinary distances between two ordered source lifts and two oppositely ordered target lifts. The one-dimensional Monge inequality for the strictly convex function $r\mapsto |r|^p$ then shows that swapping the two targets cannot increase the cost, and decreases it unless the four endpoints are in a degenerate tie configuration.
\index{convex!function}
	
	Thus an optimal assignment can be chosen with no cyclic inversion. A bijection between two finite cyclically ordered sets with no cyclic inversion is a rotation of the order, hence a cyclic shift. This shift specifies how the two cyclic orderings should be opened; after this cut, the rotation becomes an ordinary linear order and the matching is the equal-rank monotone assignment on the unfolded interval. Conversely, each cut gives one such cyclic shift, so minimizing over the finitely many shifts gives an optimal discrete circle assignment. Repeated points or ties are obtained by the same argument after an arbitrarily small perturbation and a limiting passage. This is the discrete form of the fast circle-Monge construction of~\cite{delon-circle}.
\index{circle!unfolded interval}
\index{cyclic!shift}
\end{proof}

\begin{alg}[Circle assignment by cutting]\label{alg:circle-cut-assignment}
\textbf{Input:} Equal-weight points $(x_i)_{i=1}^n$, $(y_j)_{j=1}^n$ on $\mathbb S^1$; cost $d_{\mathbb S^1}^p$.

\textbf{Output:} Optimal cyclic assignment.

\textbf{Let} $x_{(1)},\ldots,x_{(n)}$ and $y_{(1)},\ldots,y_{(n)}$ be the points sorted by increasing angle from a fixed origin.

\textbf{For} $s=0,\ldots,n-1$ \textbf{do}:
\begin{algblock}
\(E_s=\sum_{k=1}^n d_{\mathbb S^1}\!\left(x_{(k)},y_{(k+s)}\right)^p, \qquad y_{(k+n)}=y_{(k)}.\)
\end{algblock}
\textbf{Set} $s^\star=\min\argmin_{0\leq s<n}E_s$.

\textbf{Set} $\theta_{\rm cut}$ in an empty arc separating two consecutive matched pairs for the shift $s^\star$.

\textbf{Replace} every angle by its representative in $[\theta_{\rm cut},\theta_{\rm cut}+2\pi)$.
\textbf{Return} $x_{(k)}\mapsto y_{(k+s^\star)}$.
\end{alg}

\begin{figure}[htbp]
\centering
\begin{tabular}{cc}
\includegraphics[width=.34\linewidth]{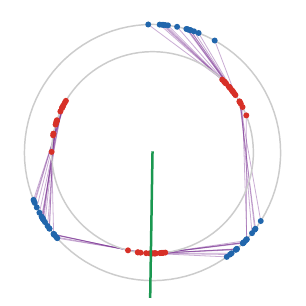} &
\includegraphics[width=.52\linewidth]{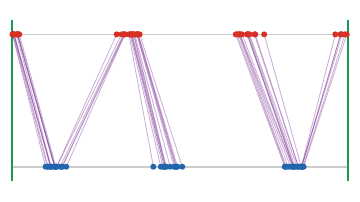} \\[-.1em]
\small circular problem &
\small unfolded interval
\index{circle!unfolded interval}
\end{tabular}
\caption{Optimal transport on the circle by cutting and unfolding. The red and blue atoms live on two copies of the circle; the denser point clouds make the cyclic ordering visible. Purple segments show the optimal matching and the green radius marks the chosen cut. Once the circle is opened at this angle, the same matching appears as a monotone one-dimensional assignment on the interval, with the two green endpoints identified.}
\index{circle!transport}
\label{fig:monge-circle-cut-unfolding}
\end{figure}

\begin{figure}[ht]
\centering
\begin{tabular}{ccc}
\includegraphics[width=.30\linewidth]{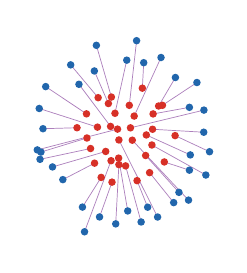} &
\includegraphics[width=.30\linewidth]{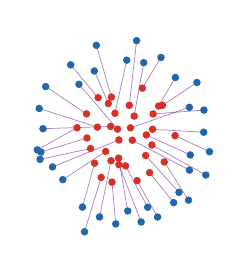} &
\includegraphics[width=.30\linewidth]{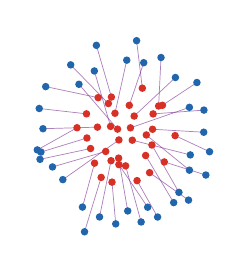} \\[-.1em]
\small $c(x,y)=\norm{x-y}$ &
\small $c(x,y)=\norm{x-y}^2$ &
\small $c(x,y)=\norm{x-y}^6$
\end{tabular}
\caption{Optimal assignments between the same two point clouds for three powers of the Euclidean distance. The source atoms are semi-regular samples in a central disk, while the target atoms are semi-regular samples on a thin annulus; this canonical geometry is reused in later coupling and regularization figures. The feasible set is unchanged, but increasing $p$ penalizes the longest edges more strongly and changes the global organization of the permutation.}
\label{fig:matching-2d-cost-exponent}
\end{figure}

\paragraph{Rational weights.}
\index{rational weights}

The strict assignment model is also tied to equal cardinalities and equal weights. As soon as the target resolution changes or the weights are not uniform, a permutation no longer describes the feasible transports. One instead needs a nonnegative transport matrix with prescribed row and column sums, as illustrated in Figure~\ref{fig:matching-resolution-and-weights}; this is the finite-dimensional Kantorovich relaxation developed in Chapter~\ref{sec-kantorovich}.
\index{matrix!nonnegative transport}
\index{Kantorovich!relaxation}

\begin{figure}[ht]
\centering
\begin{tabular}{ccc}
\includegraphics[width=.30\linewidth]{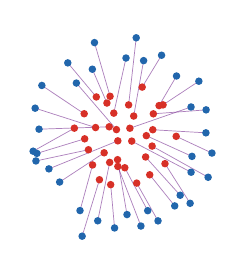} &
\includegraphics[width=.30\linewidth]{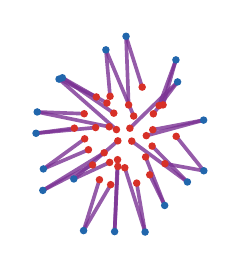} &
\includegraphics[width=.30\linewidth]{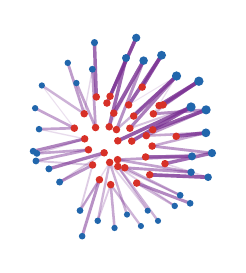} \\[-.1em]
\small $n=m$ &
\small $m<n$ &
\small $n=m$ but $\b_j\neq 1/n$
\end{tabular}
\caption{From assignments to transport plans, using the same disk-to-annulus geometry as Figure~\ref{fig:matching-2d-cost-exponent}. In the balanced equal-weight case, each source atom is matched to one target atom. With a target cloud that has half as many atoms, or with strongly nonuniform target weights, the coupling matrix can merge or split mass; segment thickness and opacity encode its nonzero entries, and blue marker areas encode the prescribed target masses.}
\index{plan!transport}
\label{fig:matching-resolution-and-weights}
\end{figure}

\begin{proposition}[Rational weights as duplicated uniform matching]\label{prop-rational-weights-duplicated-matching}
\index{rational weights}
\index{matching!uniform}
	Let
	\[
		\mu=\sum_{i=1}^n \frac{k_i}{N}\delta_{x_i},
		\qquad
		\nu=\sum_{j=1}^m \frac{\ell_j}{N}\delta_{y_j},
		\qquad
		\sum_i k_i=\sum_j\ell_j=N,
	\]
	with $k_i,\ell_j\in\NN$. The discrete Kantorovich problem between $(\mu,\nu)$ is equivalent to the uniform assignment problem obtained by replacing each $x_i$ by $k_i$ identical copies and each $y_j$ by $\ell_j$ identical copies. More precisely, after multiplying transport masses by $N$, optimal couplings correspond to optimal integer count matrices $(n_{ij})$ with row sums $k_i$ and column sums $\ell_j$, and these count matrices are exactly the collapsed form of assignments between the duplicated clouds.
\index{Kantorovich!problem}
\index{optimal coupling}
\index{assignment problem}
\end{proposition}
\begin{proof}
	Any assignment between the duplicated source and target clouds defines integers $n_{ij}$ counting how many copied particles of type $x_i$ are matched to copied particles of type $y_j$. These counts satisfy $\sum_j n_{ij}=k_i$ and $\sum_i n_{ij}=\ell_j$, and the associated coupling $P_{ij}=n_{ij}/N$ has marginals $k_i/N$ and $\ell_j/N$. The assignment cost is
\index{cost!assignment}
	\[
		\frac1N\sum_{i,j} n_{ij}c(x_i,y_j)
		=
		\sum_{i,j}P_{ij}c(x_i,y_j).
	\]
	Conversely, any nonnegative integer count matrix with those row and column sums can be realized by matching the corresponding duplicated particles. Finally, the transportation constraint matrix is totally unimodular, so the linear problem with integer supplies and demands has an optimal integer count matrix. Thus the optimum of the rational-weight Kantorovich problem is the same as the optimum of the duplicated uniform assignment problem.
\index{assignment problem}
\end{proof}

\begin{figure}[ht]
\centering
\begin{tabular}{ccc}
\includegraphics[width=.30\linewidth]{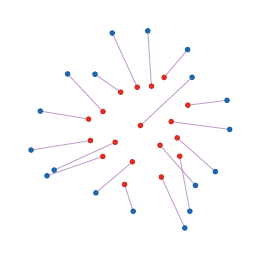} &
\includegraphics[width=.30\linewidth]{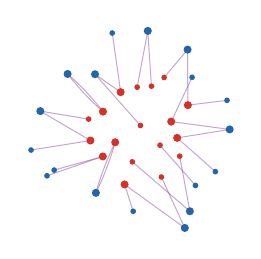} &
\includegraphics[width=.30\linewidth]{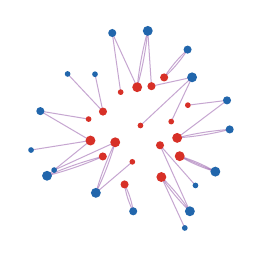} \\[-.1em]
\small $k_i=\ell_j=1$ &
\small $k_i,\ell_j\in\{1,2\}$ &
\small $k_i,\ell_j\in\{1,2,3\}$
\end{tabular}
\caption{Rational weights as duplicated uniform matchings, using the same disk-to-annulus geometry as Figure~\ref{fig:matching-resolution-and-weights} with fewer displayed atoms. The red and blue locations are kept fixed, while disk areas encode the integer multiplicities $k_i$ and $\ell_j$. Solving the assignment problem after duplicating particles produces several collapsed segments attached to high-multiplicity atoms; this is the integer count matrix of Proposition~\ref{prop-rational-weights-duplicated-matching}.}
\index{integer multiplicity}
\index{assignment problem}
\index{rational weights}
\index{matching!uniform}
\label{fig:matching-rational-duplication}
\end{figure}

\paragraph{2D case.}

This efficient strategy to compute the OT in 1-D does not extend to higher dimensions. In 2-D, as already noted by Monge, one has the following property.

\begin{proposition}[Non-crossing optimal matchings]
\index{matching!non-crossing}
	In dimension 2, for $c(x,y) = \norm{x-y}$, if $\sigma$ is an optimal assignment, then segments $[x_i,y_{\sigma(i)}]$ cannot cross.
\end{proposition}

\begin{proof}
	If two segments $[x_i,y_{\sigma(i)}]$ and $[x_j,y_{\sigma(j)}]$ cross at an interior point $z$, then the triangle inequality gives
\index{triangle inequality}
	\[
		\norm{x_i-y_{\sigma(j)}}+\norm{x_j-y_{\sigma(i)}}
		<
		\norm{x_i-y_{\sigma(i)}}+\norm{x_j-y_{\sigma(j)}}.
	\]
	The assignment obtained by swapping $(\sigma(i),\sigma(j))$ therefore has a strictly smaller cost, which contradicts optimality.
\end{proof}

This property alone is, however, not enough to lead to an efficient algorithm. Non-crossing is only a necessary local test, not a compact certificate of optimality. For instance, if $n$ sources and $n$ targets are placed alternately on the boundary of a convex polygon, the number of non-crossing perfect matchings is the Catalan number
\index{matching!perfect}
\index{Catalan number}
\[
	C_n=\frac{1}{n+1}\binom{2n}{n}\sim \frac{4^n}{\sqrt{\pi}n^{3/2}}.
\]

\begin{rem}[Catalan count of alternating non-crossing matchings]
\index{matching!non-crossing}
	The count follows from the standard Catalan recurrence. Fix one red vertex $r$. In a non-crossing perfect matching, if $r$ is matched to a blue vertex $b$, the chord $[r,b]$ splits the polygon into two smaller polygons. Since the boundary colors alternate, each side contains the same number of red and blue vertices. If one side contains $k$ red and $k$ blue vertices, the other contains $n-1-k$ red and $n-1-k$ blue vertices. Non-crossing matchings on the two sides are independent, because no segment can cross the chord $[r,b]$. Thus, denoting by $M_n$ the number of such matchings, one has
\index{matching!perfect}
	\[
		M_0=1,
		\qquad
		M_n=\sum_{k=0}^{n-1} M_k M_{n-1-k}.
	\]
	This recurrence characterizes the Catalan numbers, hence $M_n=C_n$.
\index{Catalan number}
\end{rem}

Thus even after forbidding crossings, an exhaustive search remains exponential. The two-segment swap in the proof above is nevertheless useful: it explains why a crossing matching cannot be optimal, but it does not select among the exponentially many planar matchings that survive this local test.

\section{Matching Algorithms}
\index{matching!algorithm}

This section briefly locates matching within classical combinatorial optimization. Its main point is that efficient algorithms exist, but their cleanest analysis is obtained only after introducing the linear-programming viewpoint.

Efficient algorithms exist to solve the optimal matching problem. The most well-known are the Hungarian method and auction algorithms~\cite{Kuhn1955,bertsekas1981new,bertsekas1992auction}. Auction algorithms use prices on the target points: each source bids for the target with largest reduced profit, the target price is increased, and the process terminates when the $\epsilon$-complementary slackness conditions are satisfied. For integer costs, choosing $\epsilon<1/n$ gives an exact optimum after a finite number of bids~\cite{bertsekas1992auction}. Section~\ref{sec-auction-dual-ascent} revisits this algorithm after Kantorovich duality and explains why it is a dual price method, parallel in spirit to Sinkhorn scaling.
\index{Sinkhorn!scaling}
\index{Hungarian primal-dual method}
\index{optimality!complementary slackness}
\index{Kantorovich!duality}
\index{auction algorithm}
\index{assignment problem}

\paragraph{Hungarian primal-dual method.}
\index{Hungarian primal-dual method}

The Hungarian method is best understood as a certificate-building algorithm for the assignment linear program. It maintains a partial matching $M$ and dual prices $(u_i,v_j)$ satisfying
\[
	u_i+v_j\leq C_{i,j}\qquad\forall i,j.
\]
The equality graph $E(u,v)=\{(i,j):u_i+v_j=C_{i,j}\}$ contains the edges whose reduced cost is zero. The algorithm only augments $M$ along alternating paths made of equality edges. Starting from an unmatched source, it grows an alternating tree with source set $S$ and target set $T$. If the tree reaches an unmatched target, the matching is augmented along the path. If no such edge exists, the dual variables are shifted by the smallest slack
\index{alternating!tree}
\index{graph!equality}
\index{cost!reduced}
\[
	\delta=\min_{i\in S,\ j\notin T}\bigl(C_{i,j}-u_i-v_j\bigr),
	\qquad
	u_i\leftarrow u_i+\delta\ (i\in S),\qquad
	v_j\leftarrow v_j-\delta\ (j\in T).
\]
This update preserves all inequalities $u_i+v_j\leq C_{i,j}$, keeps the current alternating tree tight, and creates at least one new equality edge leaving $S$. Maintaining these slacks incrementally gives the standard $O(n^3)$ implementation for an $n\times n$ assignment problem.
\index{assignment problem}
\index{alternating!tree}
Algorithm~\ref{alg:hungarian-primal-dual} summarizes the primal-dual loop. Figure~\ref{fig:matching-hungarian-progression} displays actual iterates by showing only the evolving partial assignment: unmatched rows are shown as flat rows to keep a fixed matrix format, and matched rows are shown as one-hot rows.
\index{Hungarian primal-dual method}

\begin{alg}[Hungarian primal-dual augmentation]\label{alg:hungarian-primal-dual}
\textbf{Input:} Square cost matrix $C\in\RR^{n\times n}$.

\textbf{Output:} Minimum-cost perfect matching $M$.

\textbf{Initialize:} Set $u_i=\min_j C_{ij}$ and $v_j=0$.

\textbf{Set} $M=\emptyset$.

\textbf{While} $M$ is not perfect \textbf{do}:
\begin{algblock}

\textbf{Build} equality graph:
\(E(u,v)=\{(i,j):u_i+v_j=C_{i,j}\}.\)

\textbf{Set} root $i_0=\min\{i:\ i\text{ is unmatched in }M\}$.

\textbf{Set} reached sets $S=\{i_0\}$ and $T=\emptyset$; clear parent pointers.

\textbf{While} $T$ contains no unmatched target \textbf{do}:
\begin{algblock}

\textbf{If} $N_E(S)\setminus T=\emptyset$ \textbf{then}:
\begin{algblock}

\textbf{Compute} $\delta=\min_{i\in S,\ j\notin T}\bigl(C_{i,j}-u_i-v_j\bigr)$.

\textbf{Update} $u_i\leftarrow u_i+\delta$ for $i\in S$ and $v_j\leftarrow v_j-\delta$ for $j\in T$.

\textbf{Refresh} equality graph $E(u,v)$.

\end{algblock}
\textbf{Set} $J=N_E(S)\setminus T$.

\textbf{For} each $j\in J$ in increasing order \textbf{do}:
\begin{algblock}

\textbf{Add} $j$ to $T$ and set parent row $p(j)=\min\{i\in S:(i,j)\in E(u,v)\}$.

\textbf{If} $j$ is matched to $i'$ in $M$ \textbf{then set} \(S\leftarrow S\cup\{i'\}\) and \(q(i')=j\).

\end{algblock}
\end{algblock}
\textbf{Set} $j_0=\min\{j\in T:\ j\text{ is unmatched in }M\}$.

\textbf{Set} $j=j_0$.

\textbf{While} $j$ is defined \textbf{do}:
\begin{algblock}

\textbf{Set} $i=p(j)$.

\textbf{Set} $M\leftarrow M\cup\{(i,j)\}$.

\textbf{Set} $j_{\rm old}=q(i)$.

\textbf{If} $j_{\rm old}$ is defined \textbf{then set} \(M\leftarrow M\setminus\{(i,j_{\rm old})\}\).
\textbf{Set} $j=j_{\rm old}$.
\end{algblock}
\end{algblock}
\algreturnskip
\textbf{Return} $M$.
\end{alg}

\begin{figure}[ht]
\centering
\setlength{\tabcolsep}{2pt}
\begin{tabular}{ccccc}
\includegraphics[width=.18\linewidth]{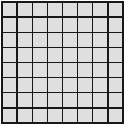} &
\includegraphics[width=.18\linewidth]{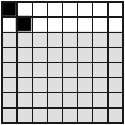} &
\includegraphics[width=.18\linewidth]{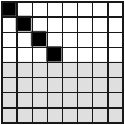} &
\includegraphics[width=.18\linewidth]{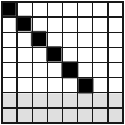} &
\includegraphics[width=.18\linewidth]{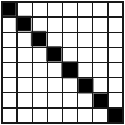} \\[-.1em]
\small initial &
\small 2 aug. &
\small 4 aug. &
\small 6 aug. &
\small final
\end{tabular}
\caption{Matrix view of actual Hungarian primal-dual iterates on a diagonally dominant ordered one-dimensional squared-distance assignment. Each panel records the current partial assignment state: unassigned rows are kept flat, while assigned rows are one-hot. The snapshots are taken at initialization and after two, four, six and eight augmentations; for this pedagogical instance the partial assignments grow along the diagonal, and the final matrix is the identity assignment certified by complementary slackness.}
\index{optimality!complementary slackness}
\label{fig:matching-hungarian-progression}
\end{figure}

\begin{prop}[Correctness and complexity of the Hungarian primal-dual method]\label{prop-hungarian-correct}
\index{Hungarian primal-dual method}
	Assume the Hungarian method terminates with a perfect matching $\sigma$ contained in the equality graph
\index{matching!perfect}
\index{graph!equality}
	\[
		E(u,v)=\enscond{(i,j)}{u_i+v_j=C_{i,j}},
	\]
	where $(u,v)$ is dual feasible, i.e. $u_i+v_j\leq C_{i,j}$ for all $(i,j)$. Then $\sigma$ is an optimal assignment. Moreover, the usual Hungarian updates terminate after finitely many augmentations.
\index{Hungarian primal-dual method}
	With maintained slacks, the method uses $O(n^3)$ arithmetic operations.
\end{prop}

\begin{proof}
	For any permutation $\tau$, dual feasibility gives
	\[
		\sum_i C_{i,\tau(i)} \geq \sum_i (u_i+v_{\tau(i)})
		= \sum_i u_i+\sum_j v_j.
	\]
	This is the weak duality lower bound. If $\sigma$ is contained in the equality graph, then
\index{duality!weak}
\index{graph!equality}
	\[
		\sum_i C_{i,\sigma(i)}
		= \sum_i u_i+\sum_j v_j,
	\]
	so the primal cost of $\sigma$ reaches the dual lower bound and is optimal.

	It remains to justify finite termination and the complexity bound. At each successful augmentation, the matching cardinality increases by one, so there are at most $n$ augmentations. During one augmentation phase, the algorithm grows an alternating tree in the equality graph. If no augmenting path is available, the dual update uses the smallest slack of an edge leaving the current tree. For edges inside the tree, adding $\delta$ to source labels and subtracting $\delta$ from target labels preserves tightness; for edges from $S$ to $T^c$, the definition of $\delta$ preserves feasibility and makes at least one new edge tight; all other inequalities are unchanged or become looser. Thus the reachable sets strictly grow between two failed augmentation attempts, and they can grow at most $n$ times within one phase. If the current slacks $\min_{i\in S}(C_{i,j}-u_i-v_j)$ are updated when a source enters $S$, each tree expansion costs $O(n)$. A phase therefore costs $O(n^2)$, and the $n$ phases give $O(n^3)$ operations. Hence the method reaches a perfect optimal matching.
\index{tightness}
\index{alternating!tree}
\index{graph!augmenting path}
\index{graph!equality}
\end{proof}


\chapter{Monge Problem between Measures}
\index{Monge!problem}
\label{sec-monge}

The goal of this chapter is to pass from finite matching to transport between arbitrary probability laws. The central stakes are to define measures, push-forwards and Monge maps carefully enough that the discrete picture survives, while exposing why deterministic maps can fail to exist. Monge's original formulation~\cite{Monge1781} and modern treatments~\cite{Villani03,Villani09,SantambrogioBook,rachev1998mass} are the conceptual background for this transition.
\index{push-forward}

The presentation of the previous chapter could only handle two sets with the same number of points. To relax this to a more general setting, one needs to consider probability distributions, so that the points are weighted by masses.

\section{Measures}
\label{sec-measures}

Measures are the language that lets point clouds, densities and singular objects be handled uniformly. We only recall the facts needed later: integration, total variation, densities and probabilistic laws.
\index{total variation}

\paragraph{Histograms}
\index{histogram}

The finite-dimensional model for a probability law is a nonnegative vector with unit total mass.

\begin{defn}[Probability simplex]\label{def-probability-simplex}
\index{probability simplex}
\index{histogram}
	The probability simplex of length $n$ is
	\[
		\simplex_n \eqdef \enscond{\a \in \RR_+^n}{ \sum_{i=1}^n \a_i = 1 }.
	\]
	Its elements are also called probability vectors or histograms.
\end{defn}

\paragraph{Discrete measure, empirical measure}
\index{discrete!measure}
\index{empirical!measure}

Probability vectors become measures once their masses are attached to locations.

\begin{defn}[Discrete measure]\label{def-discrete-measure}
\index{discrete!measure}
\index{Dirac mass}
	A discrete measure with weights $\a$ and locations $x_1,\dots,x_n\in\X$ is
	\eql{\label{eq-discr-meas}
			\al = \sum_{i=1}^n \a_i \de_{x_i},
	}
	where $\de_x$ is the Dirac mass at position $x$. It is a probability measure if $\a\in\simplex_n$, and a positive measure if all weights $\a_i$ are nonnegative.
\end{defn}
The Dirac mass should be thought of as a unit of mass infinitely concentrated at one location.
\index{probability measure}
\index{Dirac mass}
An ``empirical'' probability distribution is uniform on a point cloud, i.e. $\al=\frac{1}{n}\sum_i \de_{x_i}$.
In practice, in many applications, it is useful to be able to manipulate both the positions $x_i$ (``Lagrangian'' discretization) and the weights $\a_i$ (``Eulerian'' discretization). Lagrangian modification is usually more powerful (because it leads to adaptive discretization) but it breaks the convexity of most problems.

\paragraph{General measures}

We consider Borel measures $\al \in \Mm(\X)$ on a metric space $(\Xx,d)$, meaning that $\al(A)$ is defined for every Borel set $A$ (obtained from open sets by countable unions, countable intersections and complements). Unless otherwise stated, the measures are finite.
\index{Borel set}
\index{Borel measure}
A Dirac measure $\de_x$ is then defined as $\de_x(A)=1$ if $x \in A$ and $0$ otherwise, and this extends by linearity for discrete measures of the form~\eqref{eq-discr-meas} as
\index{Dirac mass}
\eq{
	\al(A) = \sum_{x_i \in A} \a_i
}
We denote $\Mm_+(\X)$ the subset of all positive measures on $\X$, i.e. $\al(A) \geq 0$ (and $\al(\X)<+\infty$ for the measure to be finite). The set of probability measures is denoted $\Mm_+^1(\X)$, which means that any $\al \in \Mm_+^1(\X)$ is positive, and that $\al(\X)=1$.
\index{probability measure}

\paragraph{Polish metric spaces.}
Many measure-theoretic statements used later require a mild regularity assumption on the underlying space. The point is not to restrict applications, since Euclidean spaces, complete separable manifolds and separable Hilbert spaces are covered, but to exclude pathological measurable spaces where disintegration, tightness or weak convergence can fail to behave properly.
\index{metric space}
\index{weak!convergence}
\index{tightness}

\begin{defn}[Polish metric space]\label{def-polish-metric-space}
\index{Polish space}
	A metric space $(\X,d)$ is Polish if it is complete and separable: every Cauchy sequence converges in $\X$, and $\X$ contains a countable dense subset. More generally, a topological space is called Polish if its topology can be induced by some complete separable metric.
\index{complete metric space}
\index{separable space}
\index{countable dense subset}
\end{defn}
Polish spaces are the natural ambient category for probability measures. Borel probability measures on them are regular, tightness gives compactness criteria, regular conditional probabilities and disintegrations exist under standard assumptions, and Wasserstein spaces remain Polish; see Proposition~\ref{prop-wasserstein-space-polish}.
\index{probability measure}
\index{Borel measure}
\index{conditional probability}
\index{disintegration}

\begin{defn}[Support of a measure]\label{def:support}
\index{support}
	The support $\supp(\al)$ of a Borel measure $\al$ on a metric space $\Xx$ is the smallest closed set of full $\al$-mass. Equivalently, $x\in\supp(\al)$ if every open ball centered at $x$ has positive $\al$-mass.
\index{Borel measure}
\end{defn}

\paragraph{Radon measures}
\index{Radon!measure}

Using Lebesgue integration, a Borel measure can be used to compute the integral of measurable functions (i.e. such that level sets $\enscond{x}{f(x) < t}$ are Borel sets), and we denote this pairing as
\index{Borel set}
\index{measurable function}
\index{Borel measure}
\eq{
	\dotp{f}{\al} \eqdef \int f(x) \d\al(x).
}
Integration of such a measurable $f$ against a discrete measure $\al$ computes a sum
\eq{
	\int_\X f(x) \d\al(x) = \sum_{i=1}^n \a_i f(x_i).
}

This applies in particular to continuous test functions, which are Borel measurable.
Integration against a finite measure on a compact space thus defines a continuous linear form $f \mapsto \int f \d\al$ on the Banach space of continuous functions $(\Cc(\Xx),\norm{\cdot}_\infty)$, indeed $|\int f \d\al| \leq \norm{f}_\infty |\al|(\Xx)$. On compact spaces, the converse is true: every continuous linear form on $\Cc(\Xx)$ is represented by integration against a finite signed Radon measure. This is the Riesz--Markov--Kakutani representation theorem~\cite{rudin1987realcomplex,bogachev2007measure}. It identifies $\Mm(\Xx)$ with the Banach dual of $\Cc(\Xx)$, and this duality pairing $\dotp{f}{\al}$ will be crucial for the convex duality arguments used later.
\index{Banach space}
\index{Banach dual}
\index{linear form}
\index{finite measure}
\index{Riesz-Markov-Kakutani theorem}
\index{dual!pairing}
\index{Radon!measure}

\paragraph{Relative densities.}
\index{density!ratio}

Many formulas below compare measures through densities with respect to a reference measure.

\begin{defn}[Relative density]\label{def-relative-density}
\index{Lebesgue measure}
\index{reference!measure}
\index{density!relative}
	If $\al$ is absolutely continuous with respect to a reference measure $\lambda$, its relative density is the Radon--Nikodym derivative
	\[
		\density{\al}\eqdef\frac{\d\al}{\d\lambda},
		\qquad
		\d\al(x)=\density{\al}(x)\d\lambda(x).
	\]
	Equivalently, for all $h\in\Cc(\Xx)$,
	\[
		\int_{\Xx} h(x) \d\al(x) = \int_{\Xx} h(x) \density{\al}(x) \d\lambda(x).
	\]
\end{defn}
On $\RR^d$ the reference $\lambda$ is often Lebesgue measure $\d x$.

\paragraph{Total variation norm.}
\index{total variation}

The norm inherited from the duality $\Mm(\Xx)=\Cc(\Xx)^*$ is the total variation norm. We use the notation
\index{total variation}
\begin{defn}[Total variation]\label{defn-total-variation}
	For a finite signed Radon measure $\al$ on a compact space $\Xx$,
\index{Radon!measure}
	\[
		\norm{\al}_{\TV}
		\eqdef
		\usup{f \in \Cc(\Xx)} \enscond{\dotp{f}{\al}}{\norm{f}_\infty \leq 1}.
	\]
\end{defn}
This formula is useful because it computes the norm of a measure as the operator norm of the corresponding linear form on $\Cc(\Xx)$.
\index{linear form}
\index{operator norm}

The same norm also has a direct measure-theoretic expression. The absolute value of a signed measure is
\index{signed!measure}
\[
	|\al|(A)
	\eqdef
	\usup{A=\cup_i B_i} \sum_i |\al(B_i)|,
\]
where the supremum is over finite or countable measurable partitions of $A$. Thus, if $\al=\sum_i \a_i \de_{x_i}$ with distinct atoms, $|\al|=\sum_i |\a_i|\de_{x_i}$; if $\d\al(x)=\rho(x)\d\lambda(x)$, then $\d|\al|(x)=|\rho(x)|\d\lambda(x)$.

\begin{prop}[Dual and measure definitions of total variation]\label{prop-tv-dual-measure}
\index{total variation}
	For a finite signed Radon measure $\al$ on a compact space $\Xx$,
\index{Radon!measure}
	\[
		\norm{\al}_{\TV}=|\al|(\Xx).
	\]
	Consequently $(\Mm(\Xx),\norm{\cdot}_{\TV})$ is isometrically the Banach dual of $(\Cc(\Xx),\norm{\cdot}_\infty)$.
\index{Banach dual}
\end{prop}
\begin{proof}
	The inequality $\norm{\al}_{\TV}\leq |\al|(\Xx)$ follows from
	\[
		\abs{\int f\,\d\al}\leq \int |f|\,\d|\al|\leq \norm{f}_\infty |\al|(\Xx).
	\]
	For the reverse inequality, write the Jordan decomposition $\al=\al^+-\al^-$, so that $|\al|=\al^++\al^-$. The measurable sign $s=\frac{\d\al}{\d|\al|}$ takes values in $\{-1,1\}$ outside a $|\al|$-null set and satisfies $\d\al=s\,\d|\al|$. By regularity of Radon measures on compact spaces, $s$ can be approximated in $L^1(|\al|)$ by continuous functions $f_k$ with $\norm{f_k}_\infty\leq1$. Hence $\int f_k\,\d\al\to\int s\,\d\al=|\al|(\Xx)$, which proves the equality. The final statement is the Riesz--Markov--Kakutani representation theorem with this norm.
\index{Riesz-Markov-Kakutani theorem}
\index{Jordan decomposition}
\index{Radon!measure}
\end{proof}

For two absolutely continuous measures $\d\al=\rho_\al\d\lambda$ and $\d\be=\rho_\be\d\lambda$, this gives the concrete formula
\[
	\norm{\al-\be}_{\TV}
	=
	\int_\Xx |\rho_\al(x)-\rho_\be(x)|\,\d\lambda(x).
\]
For two discrete measures written on the same union of supports, $\al=\sum_k a_k\delta_{z_k}$ and $\be=\sum_k b_k\delta_{z_k}$, with missing coefficients set to zero,
\[
	\norm{\al-\be}_{\TV}
	=
	\sum_k |a_k-b_k|.
\]
\paragraph{Probabilistic interpretation.}
\index{probabilistic coupling}

Radon probability measures represent the laws of random variables. A random variable with values in $\X$ is a measurable map $X:\Omega\to\X$ from an abstract probability space $(\Omega,\PP)$. Its law is the Radon probability measure $\al$ defined by
\index{random variable}
\index{probability measure}
\[
	\al(A)=\PP(\enscond{\omega\in\Omega}{X(\omega)\in A})
	\qquad\text{for Borel sets } A\subset \X .
\index{Borel set}
\]
Integrals with respect to this law are expectations:
\[
	\int_\X f(x)\,\d\al(x)=\mathbb{E}[f(X)].
\]

\section{Push Forward}
\index{push-forward}
\label{sec-push-forward}

Push-forwards encode how maps move mass. This short section is the bridge between deterministic maps and linear operations on measures.
\index{push-forward}

For some continuous map $\T : \X \rightarrow \Y$, we define the pushforward operator $\T_\sharp : \Mm(\X) \rightarrow \Mm(\Y)$.
For a Dirac mass, one has $\T_\sharp \de_{x} = \de_{\T(x)}$, and this formula is extended to arbitrary measures by linearity. In some sense, moving from $\T$ to $\T_\sharp$ is a way to linearize any map at the price of moving from a (possibly) finite-dimensional space $\Xx$ to the infinite-dimensional space $\Mm(\Xx)$, and this idea is central to many convex relaxation methods, most notably Lasserre's relaxation.
\index{Dirac mass}
For discrete measures~\eqref{eq-discr-meas}, the pushforward operation consists simply in moving the positions of all the points in the support of the measure
\index{support}
\eq{
	\T_{\sharp} \al \eqdef \sum_i \a_i \de_{\T(x_i)}.
}
For more general measures, for instance for those with a density, the notion of push-forward plays a fundamental role in describing spatial modifications of probability measures. The formal definition reads as follows.
\index{probability measure}
\index{push-forward}

\begin{defn}[Push-forward]\label{defn-pushfwd}
For $\T : \X \rightarrow \Y$, the push forward measure $\be = \T_\sharp \al \in \Mm(\Y)$ of some $\al \in \Mm(\X)$ satisfies
\eql{\label{eq-push-fwd}
	\foralls h \in \Cc(\Y), \quad \int_\Y h(y) \d \be(y) = \int_\X h(\T(x)) \d\al(x).
}
Equivalently, for any measurable set $B \subset \Y$, one has
\eql{\label{eq-equiv-pushfwd}
	\be(B) = \al( \enscond{x \in \X}{\T(x) \in B} ).
}
Note that $\T_\sharp$ preserves positivity and total mass, so that if $\al \in \Mm_+^1(\X)$ then $\T_\sharp \al \in \Mm_+^1(\Y)$.
\end{defn}

\begin{rem}[Pullback and push-forward]\label{rem-pullback-pushforward}\label{rem-push-forward-pull-back}
\index{pullback}
\index{push-forward}
	If $\T:\X\to\Y$ is continuous, the pullback by $\T$ is the linear operator
	\[
		\T^\sharp:\Cc(\Y)\to\Cc(\X),
		\qquad
		\T^\sharp g=g\circ\T .
	\]
	The definition of the push-forward is exactly the dual relation between this pullback on functions and the action of $\T_\sharp$ on measures:
\index{pullback}
\index{push-forward}
	\[
		\int_\X \T^\sharp g(x)\d\mu(x)
		=
		\int_\Y g(y)\d(\T_\sharp\mu)(y).
	\]
	In pairing notation,
	\[
		\left\langle \T^\sharp g,\mu\right\rangle_{\Cc(\X),\Mm(\X)}
		=
		\left\langle g,\T_\sharp\mu\right\rangle_{\Cc(\Y),\Mm(\Y)}.
	\]
	Thus push-forward is the adjoint operation to pullback, with the direction reversed. The two arrows should not be confused: $\T_\sharp$ transports mass from $\X$ to $\Y$, whereas $\T^\sharp$ transports test functions from $\Y$ back to $\X$.
\index{push-forward}
\end{rem}

\begin{prop}[Push-forward formula for densities]\label{prop-push-forward-densities}
\index{push-forward}
\index{density!formula}
	Let $\alpha$ and $\beta$ have densities $\density{\al}$ and $\density{\be}$ on open subsets of $\RR^\dims$. Assume that $\T$ is a $C^1$ diffeomorphism and that $\beta=\T_\sharp\alpha$. Then, for all $x$,
	\eql{\label{eq-pfwd-density}
			\density{\al}(x) = |\det(\T'(x))| \density{\be}(\T(x)).
	}
	Equivalently, for $y=\T(x)$,
	\[
		\rho_\beta(y)=\rho_\alpha(\T^{-1}y)\,
		\frac{1}{|\det(\T'(\T^{-1}y))|}.
	\]
\end{prop}

\begin{proof}
Explicitly doing the change of variable $y=\T(x)$, so that $\d y = |\det(\T'(x))| \d x$ in formula~\eqref{eq-push-fwd} for measures with densities $(\density{\al},\density{\be})$ on $\RR^\dims$, one has for all $h \in \Cc(\Yy)$
\begin{align*}
	\int_\Yy h(y)\rho_\be(y) \d y &= \int_\Yy h(y) \d \be(y) = \int_\Xx h(\T(x)) \d \al(x) = \int_\Xx h(\T(x)) \rho_\al(x) \d x \\
		&= \int_\Yy h(y) \rho_\al(\T^{-1}y) \frac{\d y}{|\det(\T'(\T^{-1} y))|},
\end{align*}
which shows that
\eq{
	\rho_\be(y) = \rho_\al(\T^{-1}y) \frac{1}{|\det(\T'(\T^{-1} y))|}.
}
Since $\T$ is a diffeomorphism, one obtains equivalently
\eq{
		\density{\al}(x) = |\det(\T'(x))| \density{\be}(\T(x))
}
where $\T'(x) \in \RR^{\dims \times \dims}$ is the Jacobian matrix of $\T$ (the matrix formed by taking the gradient of each coordinate of $\T$).
This implies, denoting $y=\T(x)$
\eq{
	|\det(\T'(x))| = \frac{ \density{\al}(x) }{ \density{\be}(y) }.
}
\end{proof}

\begin{rem}[Probabilistic interpretation]
\index{probabilistic coupling}
The law, i.e. probability distribution, of a random variable $X$ is the push-forward of $\PP$ by $X$, namely $\al=X_\sharp\PP$.
\index{random variable}
\index{push-forward}
Applying another push-forward $\be = \T_\sharp\al$ for $\T : \X \rightarrow \Y$, following~\eqref{eq-push-fwd}, is equivalent to defining another random variable $Y=\T(X)$, namely $\omega \in \Om \mapsto \T(X(\omega)) \in \Y$. Thus $\be$ is the law of $Y$.
Drawing a random sample $y$ from $Y$ is thus simply achieved by computing $y=\T(x)$ where $x$ is drawn from $X$.
\end{rem}

\section{Monge's Formulation}
\index{Monge!problem}
\label{sec-monge-formulation}
\label{sec-continuous-monge}

Monge's problem asks for a deterministic map transporting one law onto another while minimizing a prescribed cost. This is the original formulation introduced by Monge in his memoir on the ``d\'eblai et remblai'' problem~\cite{Monge1781}. It is geometrically direct, because every source point is assigned one destination, but it is also analytically fragile: the feasible set is non-convex, it can be empty, and a map cannot split mass. These limitations are precisely what motivate Kantorovich's relaxation in the next section.
\index{Kantorovich!relaxation}
\index{Monge!problem}

\paragraph{Monge problem.}

Given $\al\in\Mm_+^1(\Xx)$, $\be\in\Mm_+^1(\Yy)$ and a cost $c:\Xx\times\Yy\to\RR_+$, the Monge problem is
\index{Monge!problem}
\eql{\label{eq-monge-continuous}
	\Monge_c(\al,\be)
	\eqdef
	\inf_{\T:\Xx\to\Yy}
	\enscond{\int_\Xx c(x,\T(x))\d\al(x)}{\T_\sharp\al=\be}.
}
The constraint $\T_\sharp\al=\be$ means that $\T$ pushes the mass of $\al$ onto $\be$ in the sense of Definition~\ref{defn-pushfwd}.

\begin{prop}[Empirical Monge maps and matchings]\label{prop-empirical-monge-matching}
\index{Monge!problem}
\index{empirical!Monge map}
	Assume that the source atoms $x_1,\ldots,x_n$ are distinct and that
	\[
		\al=\frac1n\sum_{i=1}^n\de_{x_i},
		\qquad
		\be=\frac1n\sum_{j=1}^n\de_{y_j}.
	\]
	If $\T_\sharp\al=\be$, then for each distinct target value $z$ in the support of $\be$, exactly $n\be(\{z\})$ source atoms are mapped to $z$. In particular, if the $y_j$ are distinct, then there exists a permutation $\sigma\in\Perm(n)$ such that $\T(x_i)=y_{\sigma(i)}$ for all $i$. Conversely, every such assignment of source atoms to target atoms with the correct masses defines a feasible Monge map on the support of $\al$, and in the distinct-target case
	\[
		\int_\Xx c(x,\T(x))\d\al(x)=\frac1n\sum_{i=1}^n c(x_i,y_{\sigma(i)}).
	\]
	If source locations are repeated, they should first be merged into atoms with larger masses; such atoms cannot be split by a Monge map.
\end{prop}
\begin{proof}
	Since $\T_\sharp\al=\be$, all images $\T(x_i)$ must belong to the support of $\be$; otherwise the push-forward would give positive mass to a point outside that support. For any target atom $z$,
\index{push-forward}
	\[
		\be(\{z\})=\al(\T^{-1}(\{z\}))=\frac1n\#\enscond{i}{\T(x_i)=z}.
	\]
	This proves the counting statement. If all target atoms have mass $1/n$, each target receives exactly one source atom, which is a permutation. The converse and the cost identity follow by direct substitution.
\end{proof}

\begin{prop}[Existence of transport maps from atomless sources]\label{prop-existence-transport-map-atomless}
\index{transport map}
\index{measure!atomless}
	Let $\al$ and $\be$ be Borel probability measures on Polish spaces, and assume that $\al$ is atomless. Then there exists a measurable map $\T$ such that $\T_\sharp\al=\be$.
\index{probability measure}
\index{Polish space}
\end{prop}
\begin{proof}
	A standard measure-isomorphism theorem identifies the atomless probability space generated by $\al$ with Lebesgue measure on $[0,1]$, modulo null sets~\cite{bogachev2007measure}. It is therefore enough to construct a map from $[0,1]$ to the target law $\be$. Choose a Borel isomorphism $i$ from the support of $\be$ onto a Borel subset of $[0,1]$, set $\nu=i_\sharp\be$, and use the generalized inverse of the cumulative distribution function of $\nu$. This map sends Lebesgue measure on $[0,1]$ to $\nu$ and takes values in $i(\supp\be)$ almost surely. Composing with $i^{-1}$ and then with the source isomorphism gives a measurable transport map from $\al$ to $\be$.
\index{Lebesgue measure}
\index{support}
\index{transport map}
\index{cumulative!distribution function}
\index{generalized!inverse}
\index{measure!isomorphism}
\end{proof}

\begin{rem}[Feasibility versus optimality]
	An optimal map solving~\eqref{eq-monge-continuous} might fail to exist for two distinct reasons. First, the constraint set can be empty, for instance if $\al=\de_x$ and $\be$ is not a single Dirac mass. Proposition~\ref{prop-existence-transport-map-atomless} shows that non-atomicity of the source removes this feasibility obstruction. Second, even when feasible maps exist, the infimum may fail to be attained because the class of maps is not closed under weak limits.
\index{weak!limit}
\index{Dirac mass}
\end{rem}

\begin{example}[A splitting obstruction]\label{ex-splitting-obstruction}
\index{splitting obstruction}
	A classical example, discussed for instance in~\cite{SantambrogioBook}, takes $\al$ uniform on a vertical segment and $\be$ equal to the average of the uniform measures on two parallel vertical segments placed symmetrically to the left and to the right. For the quadratic cost, the relaxed Kantorovich optimizer splits each source point into its two symmetric targets. A deterministic Monge map cannot split one point into two destinations, so minimizing sequences must oscillate between the two sides and the Monge problem has no optimizer.
\index{cost!quadratic}
\index{Monge!problem}
\end{example}

\begin{example}[Semi-discrete Monge maps]
\index{semi-discrete!Monge map}
	The Monge formulation is not symmetric in $\al$ and $\be$. It makes sense, for instance, when $\al$ has a density with respect to Lebesgue measure and $\be$ is discrete. On $\Xx=\Yy=\RR^d$, let $\be=\sum_j b_j\delta_{y_j}$ be supported on $\{y_1,\ldots,y_m\}$. A map $\T$ such that $\T_\sharp\al=\be$ defines a segmentation of the space into cells
\index{Lebesgue measure}
	\[
		C_j\eqdef \T^{-1}(y_j),
		\qquad
		\al(C_j)=b_j.
	\]
	This is the semi-discrete setting. Chapter~\ref{sec-semidiscr-w1} explains how the cells become Laguerre cells for prescribed masses and ordinary Voronoi cells when the masses are free. If one exchanges the roles of $\al$ and $\be$ so that $\al$ is discrete, then no valid $\T$ exists in general: it is not possible to push forward a discrete measure to a measure with density.
\index{Laguerre cell}
\index{semi-discrete!OT}
\index{push-forward}
\index{Voronoi cell}
\end{example}

Figure~\ref{fig:monge-color-transfer-rgb} shows a finite-dimensional instance of this deterministic viewpoint. The source and target measures are empirical color clouds in RGB space, and the map transports colors while leaving pixel positions fixed. Grayscale equalization is one-dimensional, but transferring a full color palette requires transporting empirical measures in a three-dimensional color space, for instance RGB or Lab. Early color-transfer methods used affine statistics or iterated one-dimensional projections~\cite{reinhard2001color,pitie2005n}; replacing these projections by a genuine three-dimensional OT map gives a more intrinsic way to match palettes while preserving their geometry~\cite{rabin-ssvm-11}.
\index{empirical!measure}
\index{one-dimensional!projection}

\begin{figure}[H]
\centering
\begin{tabular}{@{}ccccc@{}}
\small source & \small $t=1/3$ & \small $t=2/3$ & \small transported & \small target \\[-.15em]
\includegraphics[width=.17\linewidth]{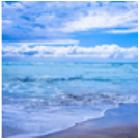} &
\includegraphics[width=.17\linewidth]{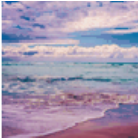} &
\includegraphics[width=.17\linewidth]{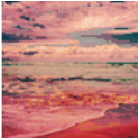} &
\includegraphics[width=.17\linewidth]{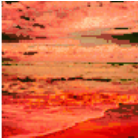} &
\includegraphics[width=.17\linewidth]{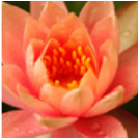} \\[-.1em]
\includegraphics[width=.17\linewidth]{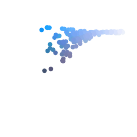} &
\includegraphics[width=.17\linewidth]{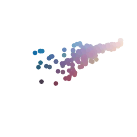} &
\includegraphics[width=.17\linewidth]{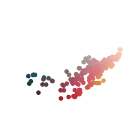} &
\includegraphics[width=.17\linewidth]{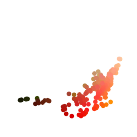} &
\includegraphics[width=.17\linewidth]{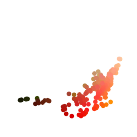}
\end{tabular}
\caption{Color transfer as a Monge map in RGB space, from a beach photograph to a flower photograph. The top row applies the palette map to the source image; the bottom row shows the empirical color clouds in the RGB cube. Only colors are transported here, not pixel locations.}
\label{fig:monge-color-transfer-rgb}
\end{figure}

\paragraph{Monge distance.}
\index{Monge!distance}

When $\Xx=\Yy$ and $c(x,y)=d(x,y)^p$ for a metric $d$, set
\[
	\Ee_\al(\T)
	\eqdef
	\int_\Xx d(x,\T(x))^p\d\al(x).
\]
The Monge value defines the directed quantity
\eql{\label{eq-monge-distance}
\index{Monge!distance}
	\tilde\Wass_p(\al,\be)^p
	\eqdef
	\inf_{\T:\Xx\to\Xx}
	\enscond{\Ee_\al(\T)}{\T_\sharp\al=\be}.
}
If the constraint set is empty, then $\tilde\Wass_p(\al,\be)=+\infty$.

\begin{prop}[Directed Monge distance]\label{prop-directed-monge-distance}
\index{Monge!distance}
	Assume that $\Xx=\Yy$ is a metric space. The quantity $\tilde\Wass_p$ is nonnegative, vanishes only on the diagonal, and satisfies the triangle inequality. It is therefore a directed extended distance: it need not be symmetric and may take the value $+\infty$.
\index{triangle inequality}
\end{prop}
\begin{proof}
	Nonnegativity is immediate. If $\tilde\Wass_p(\al,\be)=0$, choose feasible maps $\T_k$ with $\int d(x,\T_k(x))^p\d\al(x)\to0$. For every bounded $1$-Lipschitz function $h$,
\index{Lipschitz!function}
	\[
		\left|\int h\d\be-\int h\d\al\right|
		=
		\left|\int h(\T_k(x))-h(x)\d\al(x)\right|
		\leq
		\left(\int d(x,\T_k(x))^p\d\al(x)\right)^{1/p}
		\longrightarrow0.
	\]
	Since bounded Lipschitz functions separate probability measures on metric spaces, $\al=\be$.
\index{probability measure}
\index{Lipschitz!function}

	We prove the triangle inequality. If $\tilde\Wass_p(\al,\be)=+\infty$ while both $\tilde\Wass_p(\al,\ga)$ and $\tilde\Wass_p(\ga,\be)$ were finite, there would be maps $S_\sharp\al=\ga$ and $T_\sharp\ga=\be$, hence $(T\circ S)_\sharp\al=\be$, a contradiction. Thus the inequality is automatic whenever the left-hand side is infinite. In the finite case, fix $\epsilon>0$ and choose $\epsilon$-minimizers $S_\sharp\al=\ga$ and $T_\sharp\ga=\be$ such that
\index{triangle inequality}
	\[
		\Ee_\al(S)^{1/p}\leq\tilde\Wass_p(\al,\ga)+\epsilon,
		\qquad
		\Ee_\ga(T)^{1/p}\leq\tilde\Wass_p(\ga,\be)+\epsilon.
	\]
	The composed map is feasible from $\al$ to $\be$, and Minkowski's inequality gives
\index{Minkowski inequality}
	\[
	\begin{aligned}
		\tilde\Wass_p(\al,\be)
		&\leq
		\left(\int d(x,T(S(x)))^p\d\al(x)\right)^{1/p} \\
		&\leq
		\left(\int d(x,S(x))^p\d\al(x)\right)^{1/p}
		+
		\left(\int d(S(x),T(S(x)))^p\d\al(x)\right)^{1/p} \\
		&=
		\Ee_\al(S)^{1/p}+\Ee_\ga(T)^{1/p}
		\leq
		\tilde\Wass_p(\al,\ga)+\tilde\Wass_p(\ga,\be)+2\epsilon.
	\end{aligned}
	\]
	Letting $\epsilon\to0$ gives the result.
\end{proof}

The directed value $\tilde\Wass_p$ is useful conceptually, but it is too rigid to be the main distance between measures: it can be infinite and asymmetric. The Kantorovich formulation remedies both issues by replacing maps with couplings.

\section{Existence and Uniqueness of the Monge Map}
\index{Monge!problem}
\label{sec-monge-existence-uniqueness}

This section records the main regimes where Monge's deterministic formulation becomes well posed. Brenier's theorem is the central result: for the squared Euclidean cost, absolute continuity of the source restores existence, uniqueness and convex-potential structure.
\index{Brenier!theorem}
\index{absolute continuity}
\index{convex!potential}

\paragraph{Brenier's theorem.}

Brenier's theorem~\cite{MR923203,Brenier91} ensures that, in $\RR^d$ for the quadratic cost, absolute continuity of the source is enough for Monge's problem to have a unique solution. It also gives the decisive structural description of this solution: the optimal map is not an arbitrary map, but the gradient of a convex potential.
\index{Brenier!theorem}
\index{absolute continuity}
\index{convex!potential}
\index{cost!quadratic}
\index{Monge!problem}

\begin{thm}[Brenier]\label{thm-brenier}
	Let $\al,\be\in\Mm_+^1(\RR^d)$ have finite second moments, and assume that $\al$ is absolutely continuous with respect to Lebesgue measure. For the quadratic cost $c(x,y)=\norm{x-y}^2$, there exists a convex function $\phi:\RR^d\to\RR\cup\{+\infty\}$ such that
\index{Lebesgue measure}
\index{convex!function}
	\[
		\T=\nabla\phi,
		\qquad
		\T_\sharp\al=\be,
	\]
	and $\T$ is the unique optimal Monge map $\al$-almost everywhere. The optimal Kantorovich plan is $(\Id,\T)_\sharp\al$.
\index{Monge!problem}
\end{thm}
\begin{proof}
	The proof uses the Kantorovich relaxation and duality developed later in Chapter~\ref{sec-dual}. Kantorovich duality for the quadratic cost gives optimal potentials $(f,g)$ with equality $f(x)+g(y)=\norm{x-y}^2$ on the support of any optimal plan. After subtracting the harmless quadratic terms, this equality is equivalent to the Fenchel equality $\phi(x)+\phi^*(y)=\dotp{x}{y}$ for a convex function $\phi$. Hence the support of every optimal plan lies in the graph of the subdifferential $\partial\phi$. Since $\al$ has a density and convex functions are differentiable Lebesgue-almost everywhere, $\partial\phi(x)$ is a singleton for $\al$-almost every $x$. The plan is therefore concentrated on the graph of $\T=\nabla\phi$, which proves that the relaxed optimizer is induced by a Monge map. Any two optimal plans are concentrated on the same single-valued graph $\al$-almost everywhere, which gives uniqueness of the map. This is the standard route behind Brenier's polar factorization theorem; related existence and uniqueness results for more general strictly convex costs are developed for instance in~\cite{gangbo1996geometry,caffarelli2002constructing,Villani09}.
\index{support}
\index{subdifferential}
\index{optimal plan}
\index{Kantorovich!relaxation}
\index{Kantorovich!duality}
\index{polar factorization}
\index{cost!quadratic}
\end{proof}

\begin{defn}[Measures not charging hypersurfaces]\label{defn-not-charging-hypersurfaces}
\index{hypersurface}
	A Borel measure $\al$ on $\RR^d$ does not charge hypersurfaces if $\al(S)=0$ for every countably $(d-1)$-rectifiable set $S$, i.e. every set that can be covered, up to an $\mathcal H^{d-1}$-null set, by countably many $C^1$ hypersurfaces.
\index{rectifiable set}
\index{Hausdorff measure}
\index{Borel measure}
\end{defn}

\begin{rem}[A sharp source hypothesis]
\index{Brenier!source hypothesis}
	The absolute-continuity assumption in Brenier's theorem can be weakened: for the quadratic cost, it is enough that $\al$ does not charge hypersurfaces~\cite{gangbo1996geometry,Villani09,SantambrogioBook}. The reason is that the set where a convex potential has a genuinely multi-valued subdifferential is contained in a countable union of lower-dimensional pieces. This condition is close to sharp. If the source gives positive mass to a segment or a hypersurface, the subdifferential may be multi-valued on a set of positive $\al$-mass, and the optimal relaxed plan may need to split mass, as in Example~\ref{ex-splitting-obstruction}.
\index{multi-valued subdifferential}
\index{lower-dimensional set}
\index{hypersurface}
\index{subdifferential}
\index{Brenier!theorem}
\index{splitting obstruction}
\index{absolute continuity}
\index{convex!potential}
\index{cost!quadratic}
\end{rem}

\begin{rem}[Beyond the quadratic Euclidean cost]
\index{cost!quadratic}
	Brenier's theorem is the cleanest statement because the squared Euclidean cost turns optimal maps into gradients of convex functions. For costs $c(x,y)=\norm{x-y}^p$ with $p>1$, or more generally costs $c(x,y)=h(x-y)$ with $h$ smooth and strictly convex, the same strategy gives a unique optimal map under absolute continuity of the source, but the map is written in terms of a $c$-convex potential. On a Riemannian manifold, the local analogue for the cost $d_M(x,y)^2/2$ uses the exponential map
\index{convex!function}
\index{Brenier!theorem}
\index{absolute continuity}
\index{convex!potential}
	\[
		\T(x)=\exp_x(-\nabla\phi(x)),
	\]
	where $\phi$ is $c$-convex. The main additional issues are the cut locus and regularity of geodesics, which is why the Euclidean statement is usually presented first. These extensions are part of McCann's displacement convexity theory and the general theory of optimal maps on manifolds~\cite{mccann1997convexity,Villani09}.
\index{displacement!convexity}
\end{rem}

Brenier's theorem should be understood through the analogy between convex gradients and increasing functions. The gradient of a convex function is a monotone field:
\index{convex!function}
\index{Brenier!theorem}
\[
	\dotp{\nabla\phi(x)-\nabla\phi(x')}{x-x'}\geq0.
\]

\begin{rem}[Monotone fields need not be gradients]
\index{monotone!field}
\index{gradient!field}
	In dimensions larger than one, not all monotone fields are gradients of convex functions. Consider in $\RR^2$ the rotation matrix
\index{convex!function}
	\[
		R_\theta=\begin{pmatrix}
		\cos\theta & -\sin\theta\\
		\sin\theta & \cos\theta
		\end{pmatrix}.
	\]
	The linear map $x\mapsto R_\theta x$ is monotone as soon as $|\theta|\leq\pi/2$, because
	\[
		\dotp{R_\theta x-R_\theta x'}{x-x'}
		=
		\dotp{R_\theta(x-x')}{x-x'}
		=
		\cos(\theta)\norm{x-x'}^2\geq0.
	\]
	However, for $\theta\neq0$, $R_\theta$ is not symmetric and therefore cannot be the gradient of a scalar potential. Indeed, if a linear field $Ax$ equals $\nabla\phi(x)$, then its Jacobian $A$ must be symmetric; equivalently, a quadratic potential $\phi(x)=\dotp{Bx}{x}/2$ has gradient $((B+B^\top)/2)x$. Thus monotonicity is weaker than Brenier optimality in dimension $d\geq2$.
\index{quadratic!potential}
\end{rem}

\paragraph{Polar factorization.}
\index{polar factorization}

Brenier's theorem does more than solve one transport problem: it provides a canonical way to extract the ``monotone part'' of an arbitrary map. Suppose one starts from a square-integrable deformation $u:\Omega\to\RR^d$, for instance a velocity snapshot, an image deformation or a generative map. The law $\nu=u_\sharp\lambda$ records where the mass ends up, but it forgets how the points of $\Omega$ were labelled before being sent there. Brenier's polar factorization~\cite{MR923203,Brenier91} separates these two effects. The map first applies a measure-preserving rearrangement $s$ of the source, which changes labels but not mass, and then applies the unique convex-gradient map $\nabla\phi$ sending the uniform source to the output law. Thus the Brenier factor is the canonical optimal-transport representative among all maps with the same image distribution. This is useful because it isolates the part of a map that carries genuine Wasserstein displacement from the volume-preserving noise of a parametrization. It is also a bridge to fluid mechanics, where measure-preserving maps describe incompressible relabellings, and to data analysis, where one often wants to compare maps modulo such relabellings.
\index{measure-preserving!map}
\index{Brenier!theorem}
\index{Brenier!factor}

\begin{prop}[Polar factorization]\label{prop-polar-factorization}
\index{polar factorization}
	Let $\Omega\subset\RR^d$ be endowed with normalized Lebesgue measure $\lambda$, and let $u\in L^2(\Omega;\RR^d)$. Assume that the law $\nu=u_\sharp\lambda$ has finite second moment. Then there exist a measure-preserving map $s:\Omega\to\Omega$ and a convex function $\phi$ such that
\index{Lebesgue measure}
\index{convex!function}
	\[
		u=\nabla\phi\circ s
		\qquad \lambda\text{-a.e.}
	\]
	The Brenier factor $\nabla\phi$ is unique $\lambda$-almost everywhere.
\index{Brenier!factor}
\end{prop}
\begin{proof}
	By Brenier's theorem there is a unique gradient of a convex function $\T=\nabla\phi$ transporting $\lambda$ to $\nu$. The maps $u$ and $\T$ have the same image law. The rearrangement theorem for non-atomic probability spaces gives a measure-preserving map $s$ such that $u=\T\circ s$; equivalently, $s$ chooses, with the correct conditional probabilities, preimages of $u(x)$ under $\T$. Uniqueness of the Brenier factor follows from Theorem~\ref{thm-brenier}.
\index{non-atomic probability space}
\index{conditional probability}
\index{measure-preserving!map}
\index{Brenier!theorem}
\end{proof}

The name is meant to echo the usual polar decomposition of matrices. This analogy becomes literal for linear maps under the Gaussian reference measure. If $X\sim\Gaussian(0,\Id)$ and $u(x)=Ax$, then $u_\sharp\Gaussian(0,\Id)=\Gaussian(0,AA^\top)$. The Brenier map from $\Gaussian(0,\Id)$ to this Gaussian is $x\mapsto Sx$, where $S=(AA^\top)^{1/2}$ is symmetric positive semidefinite. Hence the decomposition $u=\nabla\phi\circ s$ reads
\index{matrix!polar decomposition}
\index{reference!measure}
\index{Brenier!map}
\[
	A = S O,
	\qquad
	O=S^\dagger A,
\]
with $O$ orthogonal when $A$ is invertible, or a partial isometry in the singular case. The factor $Sx$ is the convex-gradient transport part, while $Ox$ preserves the standard Gaussian law. This finite-dimensional picture is often the fastest way to remember the general theorem: polar factorization is matrix polar decomposition with matrices replaced by maps and orthogonal transformations replaced by measure-preserving rearrangements.
\index{matrix!polar decomposition}
\index{polar factorization}

\paragraph{Displacement interpolation.}
\index{displacement!interpolation}
\label{sec-monge-interpolation}

An optimal map does not only match two endpoint measures; it also tells how to draw a path between them. The construction is Lagrangian and geometric: each particle keeps its identity and travels at constant speed from its initial position to its image under the transport map.
\index{transport map}

\begin{defn}[Monge and McCann displacement interpolation]\label{def-monge-mccann-interpolation}
\index{Monge!interpolation}
\index{displacement!interpolation}
	If $\T_\sharp\al=\be$, the Monge interpolation generated by $\T$ is the curve
	\[
		\T_t(x)\eqdef(1-t)x+t\T(x),
		\qquad
		\al_t\eqdef(\T_t)_\sharp\al,
		\qquad t\in[0,1].
	\]
	For the quadratic cost, when $\T$ is the optimal Brenier map, this curve is called McCann's displacement interpolation.
\end{defn}
McCann's displacement convexity theory~\cite{mccann1997convexity} clarifies the geometric meaning of interpolating along Brenier maps; this is the language of Wasserstein geodesics used later for barycenters in Section~\ref{sec-barycenters} and for gradient flows in Chapter~\ref{sec-wasserstein-gradient-flows}. If no Monge map exists, or if the optimal object splits mass, the same straight-line idea is applied to each active pair of a coupling, as explained in the Kantorovich interpolation of Section~\ref{sec-kantorovich-plan-interpolation}. The following proposition makes the constant-speed property precise for the directed Monge value.
\index{Brenier!map}
\index{Monge!problem}
\index{Wasserstein!geodesic}
\index{gradient!flow}
\index{Wasserstein!gradient flow}
\index{displacement!interpolation}
\index{Kantorovich!interpolation}
\index{displacement!convexity}
\index{Monge!interpolation}
\index{cost!quadratic}

\begin{prop}[Directed Monge displacement geodesics]\label{prop-monge-displacement-geodesic}
\index{Monge!geodesic}
	Let $p\geq1$, let $\al,\be\in\Mm_+^1(\RR^d)$ have finite $p$-th moments, and let $\T$ be an optimal map for $\tilde\Wass_p(\al,\be)$ in~\eqref{eq-monge-distance}. Set $\al_t=(\T_t)_\sharp\al$ with $\T_t=(1-t)\Id+t\T$. Assume that, for every $t<1$, $\T_t$ is one-to-one on a full $\al$-measure Borel set, so that $\T_t^{-1}$ is defined $\al_t$-almost everywhere. Then, for $0\leq s\leq t\leq1$,
\index{Borel set}
\index{Monge!distance}
	\[
		\tilde\Wass_p(\al_s,\al_t)=(t-s)\tilde\Wass_p(\al,\be).
	\]
	Thus $t\mapsto\al_t$ is an oriented constant-speed geodesic for the directed Monge distance. In particular, for $p=2$, this applies to the Brenier map under the hypotheses of Theorem~\ref{thm-brenier}.
\index{constant-speed geodesic}
\index{Monge!distance}
\index{Brenier!map}
\end{prop}
\begin{proof}
	The case $s=t$ is trivial, so assume $s<t$. Since $s<1$, the inverse $\T_s^{-1}$ is defined $\al_s$-almost everywhere along the transported particles. Define
	\[
		S_{s,t}\eqdef\T_t\circ\T_s^{-1}
		\qquad \al_s\text{-a.e.}
	\]
	Then $(S_{s,t})_\sharp\al_s=\al_t$, and, using the optimality of $\T$,
	\[
	\begin{aligned}
		\tilde\Wass_p(\al_s,\al_t)^p
		&\leq
		\int \norm{S_{s,t}(z)-z}^p\d\al_s(z) \\
		&=
		\int \norm{\T_t(x)-\T_s(x)}^p\d\al(x)
		=
		(t-s)^p\tilde\Wass_p(\al,\be)^p.
	\end{aligned}
	\]
	The same particle construction gives $\tilde\Wass_p(\al,\al_s)\leq s\tilde\Wass_p(\al,\be)$. If $t<1$, the map $\T\circ\T_t^{-1}$ sends $\al_t$ to $\be$ and gives $\tilde\Wass_p(\al_t,\be)\leq(1-t)\tilde\Wass_p(\al,\be)$; if $t=1$, this latter distance is zero. Using the triangle inequality from Proposition~\ref{prop-directed-monge-distance},
\index{triangle inequality}
\index{Monge!distance}
	\[
		\tilde\Wass_p(\al,\be)
		\leq
		\tilde\Wass_p(\al,\al_s)+\tilde\Wass_p(\al_s,\al_t)+\tilde\Wass_p(\al_t,\be)
		\leq
		s\tilde\Wass_p(\al,\be)+\tilde\Wass_p(\al_s,\al_t)+(1-t)\tilde\Wass_p(\al,\be).
	\]
	Hence $\tilde\Wass_p(\al_s,\al_t)\geq(t-s)\tilde\Wass_p(\al,\be)$, which proves equality. For a Brenier map $\T=\nabla\phi$, the map $\T_t$ is the gradient of
\index{Brenier!map}
	\[
		x\mapsto (1-t)\frac{\norm{x}^2}{2}+t\phi(x),
	\]
	which is $(1-t)$-strongly convex for every $t<1$. On the full-measure set where $\phi$ is differentiable, this gives
	\[
		\dotp{\T_t(x)-\T_t(y)}{x-y}\geq(1-t)\norm{x-y}^2,
	\]
	so $\T_t$ is injective there. This gives the last claim.
\end{proof}

\begin{figure}[htbp]
\centering
\begin{tabular}{@{}ccccc@{}}
\includegraphics[width=.17\linewidth]{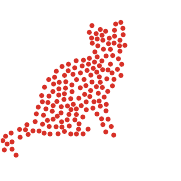} &
\index{McCann interpolation}
\includegraphics[width=.17\linewidth]{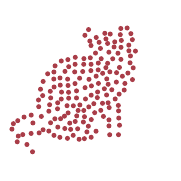} &
\includegraphics[width=.17\linewidth]{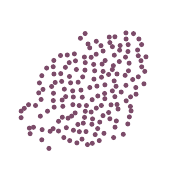} &
\includegraphics[width=.17\linewidth]{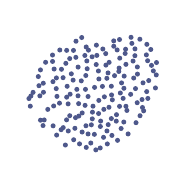} &
\includegraphics[width=.17\linewidth]{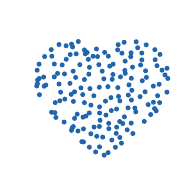} \\[-.1em]
\index{McCann interpolation}
\includegraphics[width=.17\linewidth]{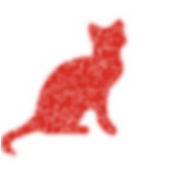} &
\includegraphics[width=.17\linewidth]{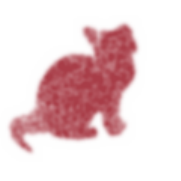} &
\includegraphics[width=.17\linewidth]{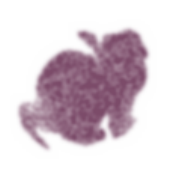} &
\includegraphics[width=.17\linewidth]{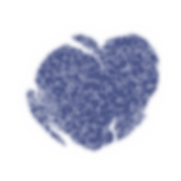} &
\index{McCann interpolation}
\includegraphics[width=.17\linewidth]{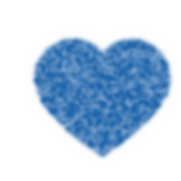} \\[-.1em]
\small $t=0$ & \small $t=1/4$ & \small $t=1/2$ & \small $t=3/4$ & \small $t=1$
\end{tabular}
\caption{McCann displacement interpolation between a cat silhouette and a heart silhouette. The first row displays a small farthest-point subset of transported particles along $T_t(x)=(1-t)x+tT(x)$. The second row renders kernel-smoothed densities from a denser transported cloud as color images: white means zero density, while high density saturates in the red-to-blue interpolation color of the corresponding time.}
\index{displacement!interpolation}
\label{fig:monge-shape-mccann-interpolation}
\end{figure}

Caffarelli's regularity theory, discussed next, explains when the convex potential is actually smooth enough to define a classical deformation.
\index{regularity theory}
\index{Caffarelli regularity}
\index{convex!potential}

\paragraph{Regularity and the Monge--Amp\`ere equation.}
\index{Monge-Ampere equation}

The previous results identify the optimal map; regularity theory asks when this map is a classical smooth deformation rather than only an almost-everywhere gradient. For quadratic costs this question becomes the regularity theory of the Monge--Amp\`ere equation.
\index{regularity theory}
\index{Monge-Ampere equation}
\index{cost!quadratic}

\begin{prop}[Caffarelli regularity]\label{prop-caffarelli-regularity}
\index{Caffarelli regularity}
	Let $\Omega,\Lambda\subset\RR^d$ be bounded uniformly convex domains with $C^2$ boundaries. Let $\al=\rho(x)\d x$ be supported on $\Omega$ and $\be=\eta(y)\d y$ be supported on $\Lambda$, with $0<m\leq\rho,\eta\leq M<+\infty$. If $\rho,\eta\in C^\alpha$ for some $\alpha\in(0,1)$, then the Brenier potential $\phi$ transporting $\al$ to $\be$ is strictly convex and satisfies $\phi\in C^{2,\alpha}_{\mathrm{loc}}(\Omega)$; in particular $\nabla\phi$ is locally H\"older. Under the corresponding boundary compatibility and smoothness assumptions, the regularity holds up to the boundary.
\end{prop}
\begin{proof}
	The potential solves the Monge--Amp\`ere equation
\index{Monge-Ampere equation}
	\[
		\det(\nabla^2\phi(x))
		=
		\frac{\rho(x)}{\eta(\nabla\phi(x))}
	\]
	in the Alexandrov sense, with second boundary condition $\nabla\phi(\Omega)=\Lambda$. The density bounds and convexity of the domains give strict convexity and localization of sections. Caffarelli's interior theory then yields $C^{2,\alpha}_{\mathrm{loc}}$ estimates for $\phi$; the boundary statement follows from the boundary regularity theory under uniform convexity and compatibility assumptions~\cite{caffarelli2003monge,Villani09}.
\index{second boundary condition}
\index{localization}
\index{regularity theory}
\index{Alexandrov solution}
\index{strict!convexity}
\end{proof}

\begin{rem}[Regularity, weak maps, and splitting]
\index{weak!map}
\index{splitting obstruction}
	Caffarelli's theorem should be read as a warning as well as a theorem. Brenier's theorem gives existence and uniqueness under mild assumptions, but smoothness requires density bounds, smoothness and convex geometry that are rarely satisfied by empirical, manifold-supported or neural generative distributions. In such applications, the exact OT map is often only weakly defined, possibly unstable, and better represented by a coupling, an entropic approximation or a learned parametric surrogate.
\index{exact OT}
\index{Brenier!theorem}

	Even without smoothness, the convex potential is locally Lipschitz on the interior of its domain, so $\nabla\phi$ is defined Lebesgue-almost everywhere. If the source measure does not satisfy the non-splitting hypotheses of Brenier's theorem, the correct relaxed object is instead an optimal Kantorovich plan concentrated on the graph of the set-valued map $\partial\phi$. At points where $\partial\phi(x)$ contains several target locations, the plan may split the mass starting from $x$. Thus the subdifferential still describes the geometry of optimality, but the transport object is a coupling rather than a single-valued map.
\index{subdifferential}
\index{convex!potential}
\end{rem}

For smooth densities, the change-of-variables formula~\eqref{eq-pfwd-density} gives the Monge--Amp\`ere equation
\index{change-of-variables}
\index{change-of-variables formula}
\index{Monge-Ampere equation}
\eql{\label{eq-monge-ampere}
	\det(\nabla^2\phi(x))\density{\be}(\nabla\phi(x))=\density{\al}(x).
}
With suitable boundary conditions, this characterizes the Brenier potential up to an additive constant among convex solutions. The convexity constraint forces $\det(\nabla^2\phi(x))\geq0$ and is necessary for this fully nonlinear elliptic equation to be well posed. Numerical schemes for this PDE connect OT with fully nonlinear elliptic solvers; see for instance~\cite{Benamou:2014jw,froese2011convergent}.

The following proposition records the infinitesimal form of this nonlinear equation. It is useful both conceptually and numerically: close to a smooth reference density, Monge--Amp\`ere transport reduces at first order to a weighted Poisson equation.
\index{weighted!Poisson equation}

\begin{prop}[Linearization of the Monge--Amp\`ere equation]\label{prop-linearized-monge-ampere}
\index{Monge-Ampere equation}
\index{linear!linearization}
	Let $\rho_\epsilon=\rho_0+\epsilon r+o(\epsilon)$ be a smooth perturbation of a positive reference density $\rho_0$ on a smooth bounded domain, with $\int r=0$. If $\T_\epsilon(x)=x+\epsilon\nabla u(x)+o(\epsilon)$ transports $\rho_0$ to $\rho_\epsilon$, then, to first order,
	\[
		-\nabla\cdot(\rho_0\nabla u)=r.
	\]
	In particular, when $\rho_0$ is constant, the linearized equation is $-\Delta u=r/\rho_0$.
\end{prop}
\begin{proof}
	The change-of-variables equation for $\T_\epsilon$ is
\index{change-of-variables}
	\[
		\rho_0(x)
		=
		\rho_\epsilon(\T_\epsilon(x))\det(\nabla \T_\epsilon(x)).
	\]
	Expanding $\rho_\epsilon(x+\epsilon\nabla u)=\rho_0(x)+\epsilon r(x)+\epsilon\dotp{\nabla\rho_0}{\nabla u}+o(\epsilon)$ and
	\[
		\det(I+\epsilon\nabla^2u)=1+\epsilon\Delta u+o(\epsilon)
	\]
	gives
	\[
		\rho_0
		=
		\rho_0+
		\epsilon\bigl(r+\dotp{\nabla\rho_0}{\nabla u}+\rho_0\Delta u\bigr)+o(\epsilon)
		=
		\rho_0+
		\epsilon\bigl(r+\nabla\cdot(\rho_0\nabla u)\bigr)+o(\epsilon).
	\]
	The first-order term must vanish.
\end{proof}

\section{One-Dimensional Transport and Quantiles}
\index{one-dimensional!transport}
\index{quantile}
\label{sec-1d-transport-quantiles}

In one dimension, optimal transport is completely explicit. The cumulative distribution function orders the mass, and the optimal coupling is obtained by matching equal quantile levels. This special case is both a computational tool and the template for several linearized constructions used later.
\index{optimal coupling}
\index{cumulative!distribution function}

\begin{defn}[Cumulative and quantile functions]\label{def-cdf-quantile}
\index{cumulative!function}
\index{quantile!function}
	For $\al\in\Mm_+^1(\RR)$, its cumulative distribution function is
	\eql{\label{eq-cumul-defn}
		\cumul{\al}(x)\eqdef\al((-\infty,x]).
	}
	Its generalized inverse, or quantile function, is
\index{generalized!inverse}
\index{quantile!function}
	\eql{\label{eq-OT-map-1d}
		\cumul{\al}^{-1}(r)
		\eqdef
		\inf\enscond{x\in\RR}{\cumul{\al}(x)\geq r},
		\qquad r\in(0,1).
	}
\end{defn}

\begin{prop}[Quantile push-forward]\label{prop-quantile-pushforward}
\index{push-forward}
\index{quantile!push-forward}
	One has $(\cumul{\al}^{-1})_\sharp\mathrm{Leb}_{[0,1]}=\al$. If $\al$ has no atoms, then $(\cumul{\al})_\sharp\al=\mathrm{Leb}_{[0,1]}$.
\end{prop}
\begin{proof}
	For simplicity, assume first that $\al$ has a strictly positive density, so that $\cumul{\al}$ is strictly increasing and continuous. Denote $\ga\eqdef(\cumul{\al}^{-1})_\sharp\mathrm{Leb}_{[0,1]}$. It is enough to prove that $\cumul{\ga}=\cumul{\al}$. For every $x$,
	\[
	\begin{aligned}
		\cumul{\ga}(x)
		&=
		\int_0^1 \mathbf 1_{(-\infty,x]}(\cumul{\al}^{-1}(z))\d z
		=
		\int_0^1 \mathbf 1_{[0,\cumul{\al}(x)]}(z)\d z
		=
		\cumul{\al}(x),
	\end{aligned}
	\]
	where we used $\cumul{\al}^{-1}(z)\leq x$ if and only if $z\leq\cumul{\al}(x)$. General measures follow from the same argument with generalized inverses and right-continuity of the cumulative distribution function. If $\al$ has no atoms, the probability integral transform gives $(\cumul{\al})_\sharp\al=\mathrm{Leb}_{[0,1]}$.
\index{cumulative!distribution function}
\index{generalized!inverse}
\end{proof}

If $\al$ has no atoms, the map
\eql{\label{eq-1d-monge-map}
	\T=\cumul{\be}^{-1}\circ\cumul{\al}
}
satisfies $\T_\sharp\al=\be$.

For the cost $c(x,y)=|x-y|^2$, this map is nondecreasing, hence the derivative of a convex function in dimension one. Brenier's theorem therefore identifies it as the optimal Monge map whenever $\al$ is atomless. With generalized inverses, the same quantile construction gives the optimal Kantorovich coupling for arbitrary measures, and it is also optimal for costs of the form $h(|x-y|)$ with $h$ convex and nondecreasing.
\index{Monge!problem}
\index{convex!function}
\index{Brenier!theorem}
\index{generalized!inverse}

\begin{prop}[Monotone rearrangement on the line]\label{prop-1d-quantile-map}
\index{monotone!rearrangement}
\index{quantile!map}
	Let $\al,\be\in\Mm_+^1(\RR)$ have finite $p$-th moments, with $p\geq1$. The coupling
	\[
		\pi^\star=(\cumul{\al}^{-1},\cumul{\be}^{-1})_\sharp\mathrm{Leb}_{[0,1]}
	\]
	minimizes $\int |x-y|^p\d\pi(x,y)$ among couplings. If $\al$ has no atoms, it is induced by the monotone Monge map~\eqref{eq-1d-monge-map}.
\index{Monge!problem}
\end{prop}
\begin{proof}
		The displayed measure is a coupling by Proposition~\ref{prop-quantile-pushforward}. Its support is monotone: for $s<t$, both quantile functions satisfy $\cumul{\al}^{-1}(s)\leq\cumul{\al}^{-1}(t)$ and $\cumul{\be}^{-1}(s)\leq\cumul{\be}^{-1}(t)$. If a coupling had two crossed pairs $x<x'$ and $y>y'$ with positive mass, exchanging the targets decreases the cost for strictly convex powers and does not increase it for $p=1$, by the two-point inequality used in Proposition~\ref{prop-matching-1d-monotone}. Eliminating crossings yields a monotone optimal coupling, and the monotone coupling with prescribed marginals is exactly the quantile coupling. If $\al$ has no atoms, $(\cumul{\al})_\sharp\al=\mathrm{Leb}_{[0,1]}$, so the coupling is generated by~\eqref{eq-1d-monge-map}.
\index{optimal coupling}
\index{push-forward}
\index{quantile!function}
\end{proof}

\begin{rem}[Composition is one-dimensional]\label{rem-1d-composition-optimal}
	In dimension one, optimal maps compose. Assume for simplicity that the intermediate laws have no atoms, so that the monotone rearrangements
\index{monotone!rearrangement}
	\[
		\T_{\al\to\be}=\cumul{\be}^{-1}\circ\cumul{\al},
		\qquad
		\T_{\be\to\ga}=\cumul{\ga}^{-1}\circ\cumul{\be}
	\]
	are well defined $\al$- and $\be$-almost everywhere. Then
	\[
		\T_{\be\to\ga}\circ\T_{\al\to\be}
		=
		\T_{\al\to\ga}
		\qquad \al\text{-a.e.}
	\]
	Indeed, each map is nondecreasing and sends quantile level $r$ to the same quantile level of the target law. This semigroup property is special to the ordered line.
\end{rem}

The obstruction in higher dimension is already visible for the most elementary Gaussian maps.

\begin{example}[Linear obstruction to composing Brenier maps]
	In higher dimension, Brenier maps for the quadratic cost are gradients of convex functions, and such maps do not generally remain gradients after composition. The simplest obstruction is linear. If $\T_1(x)=A_1x$ and $\T_2(x)=A_2x$ with $A_1,A_2$ symmetric positive definite, then $\T_2\circ\T_1$ has matrix $A_2A_1$. It is a gradient field only when this product is symmetric, equivalently $A_1A_2=A_2A_1$. Gaussian transport gives a concrete instance: between nondegenerate Gaussian laws, the Brenier map is affine with symmetric positive definite linear part. Compositions of Gaussian optimal maps are therefore optimal only in special commuting situations, for instance when all covariance matrices are simultaneously diagonalizable. Otherwise the composition contains a rotational or shearing component and is not the Brenier map between the initial and final Gaussians.
\index{rotational component}
\index{shearing component}
\index{convex!function}
\index{covariance!matrix}
\index{Gaussian!optimal map}
\index{cost!quadratic}
\index{Brenier!map}
\end{example}

For discrete measures, one cannot directly apply the map formula when the source has atoms, but if the measures are uniform on the same number of Dirac masses, then it is exactly the sorting formula of Proposition~\ref{prop-matching-1d-monotone}.
\index{Dirac mass}

\begin{alg}[Quantile rearrangement and one-dimensional geodesic]\label{alg:quantile-rearrangement-geodesic}
\index{quantile!map}
\index{one-dimensional!transport}
\textbf{Input:} One-dimensional probability measures $\alpha,\beta$; time $t\in[0,1]$.

\textbf{Output:} Quantile coupling, Monge map when defined, and geodesic point $\alpha_t$.

\textbf{Compute} $\cumul{\al}$, $\cumul{\be}$ and generalized inverses.

\textbf{Couple} equal quantile levels:
\(X=\cumul{\al}^{-1}(r), \qquad Y=\cumul{\be}^{-1}(r), \qquad r\in(0,1).\)

\textbf{If} $\alpha$ has no atoms \textbf{then}:
\begin{algblock}

\textbf{Set}
\(T(x)=\cumul{\be}^{-1}(\cumul{\al}(x)).\)
\end{algblock}
\textbf{Interpolate} quantiles:
\(Q_t(r)=(1-t)\cumul{\al}^{-1}(r)+t\cumul{\be}^{-1}(r), \qquad \al_t=(Q_t)_\sharp\mathrm{Leb}_{[0,1]}.\)
\textbf{Return} $(X,Y)$, $T$ if defined, and $\alpha_t$.
\end{alg}

\begin{prop}[One-dimensional Wasserstein formulas]\label{prop-wass-quantile-1d}
\index{Wasserstein!formula}
	Let $\al,\be\in\Mm_+^1(\RR)$ have finite $p$-th moments. For every $p\geq1$,
	\eql{\label{eq-wass-cumul}
		\Wass_p(\al,\be)^p
		=
		\int_0^1\abs{\cumul{\al}^{-1}(r)-\cumul{\be}^{-1}(r)}^p\d r
		=
		\norm{\cumul{\al}^{-1}-\cumul{\be}^{-1}}_{L^p([0,1])}^p.
	}
	For $p=1$, this is equivalently
	\begin{align}\label{eq-w1-1d}
		\Wass_1(\al,\be)
		&=
		\norm{\cumul{\al}-\cumul{\be}}_{L^1(\RR)}
		=
		\int_\RR \abs{\cumul{\al}(x)-\cumul{\be}(x)}\d x \\
		&=
		\int_\RR \abs{\int_{-\infty}^x\d(\al-\be)}\d x.
	\end{align}
\end{prop}
\begin{proof}
	The first formula follows from Proposition~\ref{prop-1d-quantile-map}: the optimal coupling is obtained by taking the same quantile level $r$ for both measures. For $p=1$, use the layer-cake identity. If $q_\al$ and $q_\be$ are the two quantile functions, then
\index{optimal coupling}
\index{quantile!function}
\index{quantile!map}
	\[
		\int_0^1 |q_\al(r)-q_\be(r)|\d r
		=
		\int_\RR
		\lambda\bigl(\{r:q_\al(r)\leq x<q_\be(r)\}\cup\{r:q_\be(r)\leq x<q_\al(r)\}\bigr)\d x.
	\]
	The measure of the displayed set is exactly $|\cumul{\al}(x)-\cumul{\be}(x)|$ for almost every $x$.
\end{proof}

Formula~\eqref{eq-wass-cumul} means that the map $\al\mapsto\cumul{\al}^{-1}$ embeds one-dimensional Wasserstein geometry isometrically into a linear $L^p$ space. For $p=2$, the Wasserstein distance on probability measures over the real line is therefore Hilbertian. This makes one-dimensional OT much simpler than higher-dimensional OT, where the Wasserstein geometry is not globally Hilbertian.
\index{probability measure}
\index{Wasserstein!distance}

\begin{figure}[htbp]
\centering
\setlength{\tabcolsep}{2pt}
\begin{tabular}{@{}cccc@{}}
\includegraphics[width=.24\linewidth]{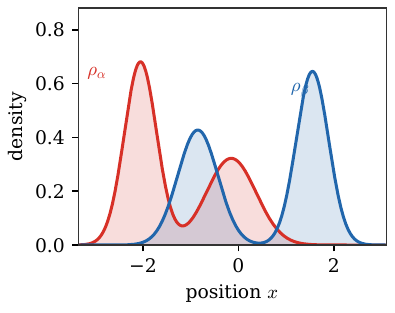} &
\includegraphics[width=.24\linewidth]{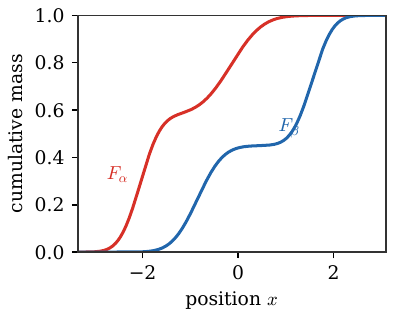} &
\includegraphics[width=.24\linewidth]{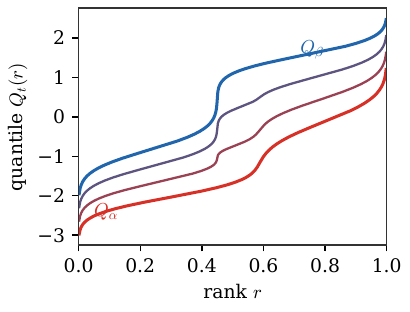} &
\includegraphics[width=.24\linewidth]{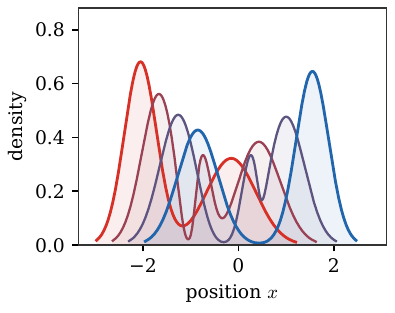} \\[-.1em]
\small densities &
\small cumulative functions &
\small quantiles &
\small displacement geodesic
\end{tabular}
\caption{One-dimensional transport through quantiles. The same two smooth laws are shown as densities, cumulative functions and quantile functions. The last panel displays the displacement interpolation obtained by the linear quantile path $Q_t=(1-t)Q_\alpha+tQ_\beta$, which is the explicit one-dimensional $\Wass_2$ geodesic.}
\index{one-dimensional!transport}
\index{displacement!interpolation}
\index{quantile!function}
\label{fig:monge-1d-quantile-geodesic}
\end{figure}

The last panel of Figure~\ref{fig:monge-1d-quantile-geodesic} is the one-dimensional specialization of the displacement interpolation of Section~\ref{sec-monge-interpolation}. In quantile coordinates, the interpolating measure is characterized by
\index{displacement!interpolation}
\index{Monge!interpolation}
\[
	\cumul{\al_t}^{-1}(r)
	=
	(1-t)\cumul{\al}^{-1}(r)+t\cumul{\be}^{-1}(r),
	\qquad r\in(0,1).
\]
The histogram-equalization figure in Section~\ref{sec-monge-pbm} is the same construction applied to pixel intensities.
\index{histogram}

For $p=1$, formula~\eqref{eq-w1-1d} shows that $\Wass_1$ is a norm on signed measures with zero total mass once they are identified with their cumulative primitives. Other classical one-dimensional distances are obtained by replacing the $L^1$ norm of cumulative functions with $L^2$ or $L^\infty$ norms; under suitable tightness assumptions, such norms also metrize convergence in law and lead for instance to Cram\'er--von Mises and Kolmogorov--Smirnov-type distances.
\index{convergence!in law}
\index{Cramer-von Mises distance}
\index{Kolmogorov-Smirnov distance}
\index{tightness}
\index{cumulative!function}
\index{signed!measure}

\paragraph{Triangular rearrangements.}
\index{triangular!rearrangement}

There is another canonical way to build transport maps in several dimensions: transport one coordinate at a time by conditional one-dimensional quantiles. This construction goes back to Knothe and Rosenblatt~\cite{Knothe1957,Rosenblatt1952}. It is not usually cost-optimal, but it gives a deterministic rearrangement under weak assumptions and clarifies how multivariate transport can be reduced recursively to scalar monotone maps.
\index{transport map}

\begin{prop}[Knothe--Rosenblatt triangular rearrangement]\label{prop-knothe-rosenblatt}
\index{Knothe-Rosenblatt rearrangement}
	Let $\al,\be\in\Mm_+^1(\RR^d)$. Assume, for simplicity, that the first marginal of $\al$ and the one-dimensional conditional laws of $\al$ used below are atomless, and that regular conditional distributions are fixed. There is a triangular map
\index{regular conditional distribution}
\index{conditional law}
\index{triangular!map}
	\[
		\T(x_1,\ldots,x_d)
		=
		(\T_1(x_1),\T_2(x_1,x_2),\ldots,\T_d(x_1,\ldots,x_d))
	\]
	such that $\T_\sharp\al=\be$ and, for each $k$, the function $x_k\mapsto\T_k(x_1,\ldots,x_k)$ is nondecreasing for $\al$-almost every value of $(x_1,\ldots,x_{k-1})$.
\end{prop}
\begin{proof}
	The construction is recursive. For $k=1$, let $\T_1$ be the monotone rearrangement between the first marginal of $\al$ and the first marginal of $\be$. Suppose that $\T_1,\ldots,\T_{k-1}$ have already been constructed. Write $x_{<k}=(x_1,\ldots,x_{k-1})$ and $\T_{<k}=(\T_1,\ldots,\T_{k-1})$. By induction, $(\T_{<k})_\sharp\al_{<k}=\be_{<k}$, where $\al_{<k}$ and $\be_{<k}$ are the first $(k-1)$-coordinate marginals. Let $\al^k_{x_{<k}}$ and $\be^k_{y_{<k}}$ be regular conditional laws of the $k$-th coordinate given the previous coordinates. Define $\T_k(x_{<k},\cdot)$ as the one-dimensional monotone rearrangement from $\al^k_{x_{<k}}$ to $\be^k_{\T_{<k}(x_{<k})}$.
\index{monotone!rearrangement}

	The map $\T_k(x_{<k},\cdot)$ is nondecreasing by the one-dimensional rearrangement theorem. The chain rule for disintegrations then shows that, after step $k$, the first $k$ coordinates of $\T_\sharp\al$ have the same law as the first $k$ coordinates of $\be$. At $k=d$ this gives $\T_\sharp\al=\be$. Target atoms are handled by generalized quantiles. If a source conditional law has atoms that must be split, the same recursive construction defines a triangular Markov kernel rather than a deterministic map.
\index{disintegration}
\index{Markov kernel}
\index{generalized quantile}
\index{conditional law}
\end{proof}

The recursive construction in the proof is an algorithm: it repeatedly applies the one-dimensional quantile rearrangement to conditional distributions.

\begin{alg}[Knothe--Rosenblatt triangular rearrangement]\label{alg:triangular-rearrangement}
\textbf{Input:} Probability measures $\alpha,\beta$ on $\RR^d$ with conditional laws.

\textbf{Output:} Knothe--Rosenblatt triangular map $\T$.

\textbf{Compute} first-coordinate rearrangement:
\(\T_1=(F_{\be_1})^{-1}\circ F_{\al_1}.\)

\textbf{For} $k=2,\ldots,d$ \textbf{do}:
\begin{algblock}

\textbf{Set} $x_{<k}=(x_1,\ldots,x_{k-1})$.

\textbf{Compute} conditional laws $\al^k_{x_{<k}}$ and $\be^k_{\T_{<k}(x_{<k})}$.

\textbf{Set}
\(\T_k(x_{<k},x_k) = \bigl(F_{\be^k_{\T_{<k}(x_{<k})}}\bigr)^{-1} \circ F_{\al^k_{x_{<k}}}(x_k).\)
\end{algblock}
\algreturnskip
\textbf{Return} \(\T(x)=(\T_1(x_1),\T_2(x_1,x_2),\ldots,\T_d(x_1,\ldots,x_d)).\)
\end{alg}

Figure~\ref{fig:monge-triangular-rearrangement} shows the two-dimensional mechanism on image histograms. The first stage transports only the horizontal marginal, so the middle pivot has the same $x$-marginal as the target but keeps the source vertical conditionals. The second stage then transports each vertical conditional law inside the corresponding column.
\index{histogram}
\index{conditional law}
\index{triangular!rearrangement}

\begin{figure}[H]
\centering
\setlength{\tabcolsep}{1pt}
\begin{tabular}{@{}ccccccc@{}}
\includegraphics[width=.135\linewidth]{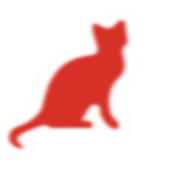} &
\index{triangular!rearrangement}
\includegraphics[width=.135\linewidth]{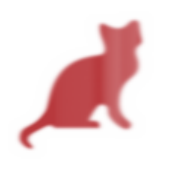} &
\includegraphics[width=.135\linewidth]{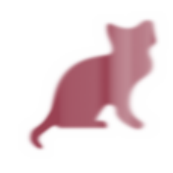} &
\includegraphics[width=.135\linewidth]{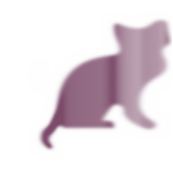} &
\includegraphics[width=.135\linewidth]{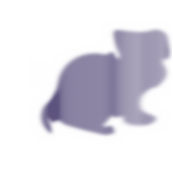} &
\index{triangular!rearrangement}
\includegraphics[width=.135\linewidth]{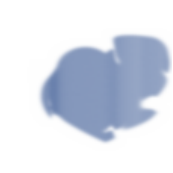} &
\includegraphics[width=.135\linewidth]{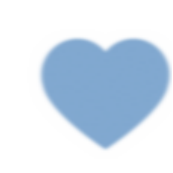} \\[-.1em]
\small source &
\small $x$-move &
\small $x$-move &
\small pivot &
\small $y$-move &
\small $y$-move &
\small target
\end{tabular}
\caption{Triangular rearrangement between the same cat and heart densities as in Figure~\ref{fig:monge-shape-mccann-interpolation}. The panels are computed directly on image histograms. The first three transitions move mass horizontally by the monotone rearrangement between the $x$-marginals; the pivot has the target horizontal marginal. The last three transitions keep each column fixed and move mass vertically by one-dimensional monotone rearrangements between conditional laws.}
\index{triangular!rearrangement}
\index{monotone!rearrangement}
\index{McCann interpolation}
\label{fig:monge-triangular-rearrangement}
\end{figure}

This construction transports successively along coordinate axes and is therefore often called axis-wise transport. It depends on the chosen ordering of coordinates and is not generally optimal for the quadratic cost. It is nevertheless a useful limiting object: Brenier maps for increasingly anisotropic quadratic costs converge to triangular rearrangements under suitable assumptions~\cite{carlier2010knothe}.
\index{triangular!rearrangement}
\index{cost!quadratic}
\index{Brenier!map}

\begin{prop}[Anisotropic Brenier maps converge to Knothe--Rosenblatt]\label{prop-knothe-limit-anisotropic-brenier}
\index{Knothe-Rosenblatt rearrangement}
	Let $\alpha,\beta$ be compactly supported probability measures on $\RR^d$ with densities bounded above and below on their rectangular supports, and assume that the conditional laws entering Proposition~\ref{prop-knothe-rosenblatt} are atomless. For $\epsilon>0$, set
\index{conditional law}
\index{probability measure}
	\[
		c_\epsilon(x,y)
		\eqdef
		\sum_{k=1}^d \epsilon^{k-1}\abs{x_k-y_k}^2 .
	\]
	Let $T_\epsilon$ be the Monge map from $\alpha$ to $\beta$ for the cost $c_\epsilon$, and let $T_{\mathrm{KR}}$ be the triangular Knothe--Rosenblatt rearrangement with the coordinate order used above. Then
\index{Monge!problem}
\index{Knothe-Rosenblatt rearrangement}
	\[
		T_\epsilon \longrightarrow T_{\mathrm{KR}}
		\qquad\text{in } L^2(\alpha;\RR^d)
		\qquad\text{as } \epsilon\to0 .
	\]
\end{prop}

\begin{proof}
	We give the standard lexicographic argument, which is the variational core of~\cite{carlier2010knothe}. Let $\pi_\epsilon=(\Id,T_\epsilon)_\sharp\alpha$. Since the supports are compact, a subsequence converges weakly to some coupling $\pi$ between $\alpha$ and $\beta$. The optimality of $\pi_\epsilon$ for
	\[
		F_\epsilon(\gamma)
		=
		\sum_{k=1}^d \epsilon^{k-1}
		\int\abs{x_k-y_k}^2\d\gamma(x,y)
	\]
	first implies, by letting $\epsilon\to0$, that $\pi$ minimizes the one-dimensional quadratic cost in the first coordinate among all couplings. Since the first marginal of $\alpha$ is atomless, the one-dimensional minimizer is the monotone rearrangement, so $y_1=T_1(x_1)$ under $\pi$.
\index{monotone!rearrangement}
\index{cost!quadratic}

	Now restrict attention to couplings satisfying this first-coordinate constraint. Subtract the common minimal value of the first-coordinate cost, divide the optimality inequality by $\epsilon$, and let $\epsilon\to0$. The limit coupling must minimize the second-coordinate quadratic cost among all couplings that already realize the first monotone rearrangement. Disintegrating with respect to $(x_1,y_1)$ reduces this constrained problem to the one-dimensional monotone rearrangement between the conditional laws of $x_2$ and $y_2$. Hence $y_2=T_2(x_1,x_2)$ under $\pi$.
\index{conditional law}

	Repeating the same subtraction-and-rescaling argument gives, for every $k$, the conditional monotone rearrangement $y_k=T_k(x_1,\ldots,x_k)$. Thus every weak limit of $(\pi_\epsilon)_\epsilon$ is concentrated on the graph of $T_{\mathrm{KR}}$. This graph coupling is unique, so the whole family $\pi_\epsilon$ converges weakly to $(\Id,T_{\mathrm{KR}})_\sharp\alpha$.
\index{weak!limit}
\index{monotone!rearrangement}

	Finally, let $X\sim\alpha$. The graph couplings are the laws of $(X,T_\epsilon(X))$, and they converge weakly to the law of $(X,T_{\mathrm{KR}}(X))$. To turn this into convergence of maps, fix $\zeta>0$. By Lusin's theorem, there is a compact set $K$ with $\alpha(K)>1-\zeta$ on which $T_{\mathrm{KR}}$ is continuous. On $K$, the set
	\[
		\enscond{(x,y)}{x\in K,\ \norm{y-T_{\mathrm{KR}}(x)}\geq\delta}
	\]
	is closed and has zero mass under the limiting graph coupling. Portmanteau's theorem gives
\index{zero mass}
	\[
		\limsup_{\epsilon\to0}
		\alpha\!\left(\enscond{x}{x\in K,\ \norm{T_\epsilon(x)-T_{\mathrm{KR}}(x)}\geq\delta}\right)
		=0 .
	\]
	Adding the complement of $K$ and letting $\zeta\to0$ proves convergence in $\alpha$-probability. Since all maps take values in a common compact set, this convergence is uniformly integrable and hence holds in $L^2(\alpha)$.
\end{proof}

\section{Gaussian Measures and the Bures Metric}
\index{Gaussian!measure}
\index{Bures!metric}
\label{sec-gaussian-bures}

Gaussian measures form the most important finite-dimensional family preserved by quadratic optimal transport. The mean moves linearly, while the covariance follows the Bures--Wasserstein geometry of positive semidefinite matrices.
\index{positive!semidefinite matrix}
\index{Bures-Wasserstein geometry}

\paragraph{One-dimensional Gaussians.}
\index{Gaussian!measure}

Let $\al=\Gaussian(m_\al,\sigma_\al^2)$ and $\be=\Gaussian(m_\be,\sigma_\be^2)$ be two nondegenerate Gaussians on $\RR$. Then
\[
	\T(x)=\frac{\sigma_\be}{\sigma_\al}(x-m_\al)+m_\be
\]
satisfies $\T_\sharp\al=\be$. It is the derivative of the convex function
\index{convex!function}
\[
	\phi(x)=\frac{\sigma_\be}{2\sigma_\al}(x-m_\al)^2+m_\be x,
\]
so Brenier's theorem shows that it is the optimal quadratic transport. The associated distance is
\index{Brenier!theorem}
\[
	\Wass_2(\al,\be)^2
	=
	\int_\RR\left(\frac{\sigma_\be}{\sigma_\al}(x-m_\al)+m_\be-x\right)^2\d\al(x)
	=
	(m_\al-m_\be)^2+(\sigma_\al-\sigma_\be)^2.
\]
The formula extends by continuity to Dirac masses, although the affine Monge map itself only pushes a Dirac source to another Dirac. Thus the OT geometry of one-dimensional Gaussians is the Euclidean geometry of the half-plane $(m,\sigma)\in\RR\times\RR_+$. This contrasts with KL geometry, where singular Gaussians are infinitely distant.
\index{Monge!problem}
\index{Dirac mass}
\index{Gaussian!measure}

\begin{figure}[htbp]
\centering
\setlength{\tabcolsep}{2pt}
\begin{tabular}{@{}ccc@{}}
\includegraphics[width=.30\linewidth]{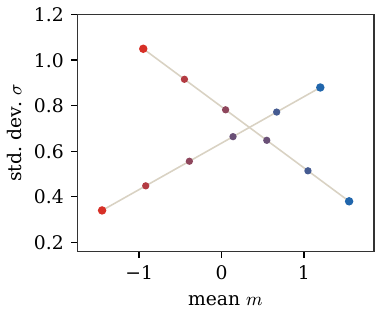} &
\includegraphics[width=.32\linewidth]{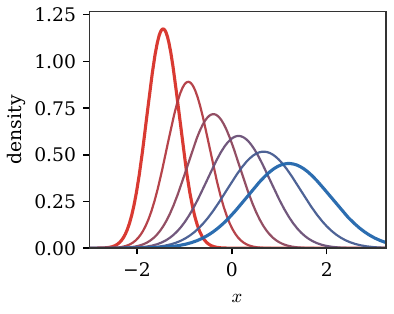} &
\includegraphics[width=.32\linewidth]{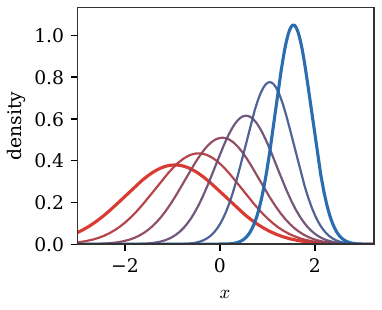} \\[-.1em]
\small $(m,\sigma)$ half-plane &
\small first geodesic &
\small second geodesic
\end{tabular}
\caption{One-dimensional Gaussian $\Wass_2$ geodesics. In the coordinates $(m,\sigma)$, the geodesics are Euclidean segments in the upper half-plane. The two density panels show the corresponding Gaussian densities along the two segments, with colors interpolating from the red endpoint to the blue endpoint.}
\label{fig:monge-gaussian-w2-geodesic-1d}
\end{figure}

\paragraph{Multivariate Gaussians.}
\index{Gaussian!measure}

If
\eql{\label{eq-gauss-pf}
	\al=\Gaussian(\mean_\al,\cov_\al),
	\qquad
	\be=\Gaussian(\mean_\be,\cov_\be),
	\qquad
	\T(x)=\mean_\be+A(x-\mean_\al),
}
then $\T$ is the gradient of the convex quadratic potential $\phi(x)=\dotp{\mean_\be}{x}+\dotp{A(x-\mean_\al)}{x-\mean_\al}/2$ if and only if $A$ is symmetric positive semidefinite.
\index{quadratic!potential}

\begin{prop}[Affine push-forward of Gaussians]\label{prop-gaussian-affine-push-forward}
\index{push-forward}
\index{Gaussian!push-forward}
	One has $\T_\sharp\al=\be$ if and only if
	\eql{\label{eq-gauss-covariance-pushforward}
		A\cov_\al A^\top=\cov_\be.
	}
\end{prop}
\begin{proof}
	An affine function maps a Gaussian to a Gaussian, so the law of $\T(X)$ is determined by its mean and covariance. If $X\sim\al$ and $Y=\T(X)$, then
	\[
		\EE(Y)=\mean_\be+A\EE(X-\mean_\al)=\mean_\be,
	\]
	and
	\[
		\EE((Y-\mean_\be)(Y-\mean_\be)^\top)
		=
		A\EE((X-\mean_\al)(X-\mean_\al)^\top)A^\top
		=
		A\cov_\al A^\top.
	\]
	Thus $A\cov_\al A^\top=\cov_\be$ is necessary and sufficient for $Y$ to have the same mean and covariance as $\be$.
\end{proof}

The covariance equation is quadratic in $A$. Under the symmetry constraint imposed by Brenier's theorem, it becomes $A\cov_\al A=\cov_\be$. The covariance part of the resulting cost is a matrix metric.
\index{Brenier!theorem}

\begin{defn}[Bures metric]\label{def-bures-metric}
\index{Bures!metric}
\index{positive!semidefinite matrix}
	For positive semidefinite covariance matrices $\Sigma$ and $\Lambda$, the Bures metric is
	\eql{\label{eq-bure-defn}
		\Bb(\Sigma,\Lambda)^2
		\eqdef
		\tr\pa{\Sigma+\Lambda-2(\Sigma^{1/2}\Lambda\Sigma^{1/2})^{1/2}}.
	}
\end{defn}
The next proposition solves the covariance equation and shows that this metric is exactly the covariance contribution to Gaussian $\Wass_2$.

\begin{prop}[Gaussian $\Wass_2$ formula and Bures covariance term]\label{prop-gaussian-w2-bures}
\index{covariance!term}
\index{Bures!covariance}
	Assume that $\cov_\al$ and $\cov_\be$ are positive definite. The unique symmetric positive-definite solution of
	\[
		A\cov_\al A=\cov_\be
	\]
	is
	\eql{\label{eq-bures-map}
		A=
		\cov_\al^{-1/2}
		\pa{\cov_\al^{1/2}\cov_\be\cov_\al^{1/2}}^{1/2}
		\cov_\al^{-1/2}.
	}
	The affine map $\T(x)=\mean_\be+A(x-\mean_\al)$ is the optimal quadratic-cost transport from $\Gaussian(\mean_\al,\cov_\al)$ to $\Gaussian(\mean_\be,\cov_\be)$, and
\index{affine map}
\index{cost!quadratic}
	\eql{\label{eq-dist-gauss}
		\Wass_2(\al,\be)^2
		=
		\norm{\mean_\al-\mean_\be}^2+\Bb(\cov_\al,\cov_\be)^2,
	}
	where $\Bb$ is the Bures metric of Definition~\ref{def-bures-metric}.
\end{prop}
\begin{proof}
	Multiplying $A\cov_\al A=\cov_\be$ on the left and right by $\cov_\al^{1/2}$ gives
	\[
		(\cov_\al^{1/2}A\cov_\al^{1/2})^2
		=
		\cov_\al^{1/2}\cov_\be\cov_\al^{1/2}.
	\]
	Since $A$ is symmetric positive, $\cov_\al^{1/2}A\cov_\al^{1/2}$ is symmetric positive and is therefore the unique positive square root of the right-hand side. Conversely, the matrix in~\eqref{eq-bures-map} is symmetric positive and satisfies the covariance equation.

	By Proposition~\ref{prop-gaussian-affine-push-forward}, this affine map pushes $\al$ to $\be$. It is the gradient of a convex quadratic potential, so Brenier's theorem implies optimality. If $X\sim\al$, then
\index{affine map}
\index{Brenier!theorem}
\index{quadratic!potential}
\index{push-forward}
	\begin{align*}
		\EE\norm{X-\T(X)}^2
		&=
		\norm{\mean_\al-\mean_\be}^2
		+
		\EE\norm{(I-A)(X-\mean_\al)}^2 \\
		&=
		\norm{\mean_\al-\mean_\be}^2
		+
		\tr((I-A)\cov_\al(I-A)^\top) \\
		&=
		\norm{\mean_\al-\mean_\be}^2
		+
		\tr(\cov_\al)+\tr(A\cov_\al A)-2\tr(A\cov_\al) \\
		&=
		\norm{\mean_\al-\mean_\be}^2
		+
		\tr(\cov_\al)+\tr(\cov_\be)
		-2\tr((\cov_\al^{1/2}\cov_\be\cov_\al^{1/2})^{1/2}),
	\end{align*}
	which is the desired expression. The formula for singular covariance matrices follows by adding $\eta I$ and letting $\eta\downarrow0$.
\index{covariance!matrix}
\end{proof}

The covariance term $\Bb$ is the Bures--Wasserstein metric on positive semidefinite matrices~\cite{bures1969extension,gelbrich1990formula,bhatia2018bures}. It separates the Euclidean displacement of the mean from the intrinsic transport geometry of covariance ellipsoids.
\index{Bures-Wasserstein geometry}
\index{positive!semidefinite matrix}
\index{covariance!term}

\begin{figure}[htbp]
\centering
\setlength{\tabcolsep}{2pt}
\begin{tabular}{@{}cc@{}}
\includegraphics[width=.42\linewidth]{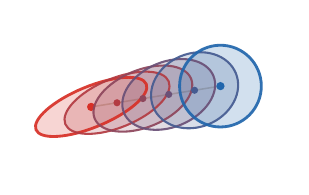} &
\includegraphics[width=.42\linewidth]{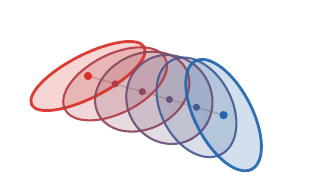} \\[-.1em]
\small anisotropic to isotropic &
\small rotated anisotropies
\end{tabular}
\caption{Two-dimensional Gaussian $\Wass_2$ geodesics. Means move linearly, while covariance ellipses follow the Bures--Wasserstein interpolation. The left panel contracts an anisotropic Gaussian toward an isotropic one; the right panel interpolates between two strongly oriented anisotropic covariances.}
\label{fig:monge-gaussian-w2-geodesic-2d}
\end{figure}

\begin{prop}[Metric and convexity properties of the Bures term]\label{prop-bures-metric-convex}
\index{Bures!metric}
	The function $\Bb$ is a distance on positive semidefinite covariance matrices. Moreover, $\Bb^2$ is jointly convex: for $t\in[0,1]$,
\index{covariance!matrix}
	\[
		\Bb^2((1-t)\Sigma_0+t\Sigma_1,(1-t)\Lambda_0+t\Lambda_1)
		\leq
		(1-t)\Bb^2(\Sigma_0,\Lambda_0)+t\Bb^2(\Sigma_1,\Lambda_1).
	\]
\end{prop}
\begin{proof}
	The key identity is the Procrustes representation
	\[
		\Bb^2(\Sigma,\Lambda)
		=
		\min_{Q^\top Q=I}\norm{\Sigma^{1/2}-\Lambda^{1/2}Q}_F^2.
	\]
	Indeed, expanding the square gives $\tr\Sigma+\tr\Lambda-2\max_{Q^\top Q=I}\tr(\Sigma^{1/2}Q^\top\Lambda^{1/2})$, and the orthogonal Procrustes formula identifies the maximum with $\tr((\Sigma^{1/2}\Lambda\Sigma^{1/2})^{1/2})$. Symmetry, positivity and separation follow immediately from this representation. The triangle inequality follows by choosing two almost optimal orthogonal matrices $Q_1,Q_2$ and writing
\index{triangle inequality}
	\[
		\norm{\Sigma^{1/2}-\Gamma^{1/2}Q_2Q_1}_F
		\leq
		\norm{\Sigma^{1/2}-\Lambda^{1/2}Q_1}_F
		+
		\norm{\Lambda^{1/2}-\Gamma^{1/2}Q_2}_F.
	\]
	Letting the two choices become optimal proves the metric property.

	For convexity, use the equivalent factor formulation
	\[
		\Bb^2(\Sigma,\Lambda)
		=
		\min_{UU^\top=\Sigma,\thinspace VV^\top=\Lambda}\norm{U-V}_F^2.
	\]
	Choose nearly optimal factors $(U_0,V_0)$ and $(U_1,V_1)$ for the two pairs, and define block factors
	\[
		U_t=[\sqrt{1-t}\,U_0,\sqrt t\,U_1],
		\qquad
		V_t=[\sqrt{1-t}\,V_0,\sqrt t\,V_1].
	\]
	Then $U_tU_t^\top=(1-t)\Sigma_0+t\Sigma_1$ and $V_tV_t^\top=(1-t)\Lambda_0+t\Lambda_1$, while
	\[
		\norm{U_t-V_t}_F^2
		=
		(1-t)\norm{U_0-V_0}_F^2+t\norm{U_1-V_1}_F^2.
	\]
	Taking the infimum over the initial factors proves joint convexity.
\end{proof}

\begin{rem}[Diagonal covariances and Hellinger geometry]
\index{covariance!diagonal}
\index{Hellinger!geometry}
	If $\cov_\al=\diag(r_i)_i$ and $\cov_\be=\diag(s_i)_i$ are diagonal, the Bures metric reduces to the Euclidean distance between square roots,
\index{Bures!metric}
	\[
		\Bb(\cov_\al,\cov_\be)=\norm{\sqrt r-\sqrt s}_2.
	\]
	This is the finite-dimensional Hellinger geometry on the nonnegative covariance coordinates: variances are compared after the amplitude change of variables $r_i\mapsto\sqrt{r_i}$.
\index{change-of-variables}
\index{Hellinger!geometry}
\end{rem}


\chapter{Kantorovich Relaxation}
\index{Kantorovich!relaxation}
\label{sec-kantorovich}

Kantorovich's relaxation is the decisive move that turns transport into convex optimization. This chapter explains how deterministic maps are replaced by couplings, why this fixes infeasibility and symmetry, and how it produces the Wasserstein distances. Historically, this linear-programming viewpoint grew from Kantorovich's economic planning work~\cite{Kantorovich42} and is now the standard foundation of OT~\cite{Villani03,Villani09,rachev1998mass2}.
\index{Wasserstein!distance}

\section{Discrete Relaxation}
\index{Kantorovich!relaxation}
\label{sec-discrete-relaxation}

The discrete relaxation is the cleanest place to see mass splitting. It replaces permutations by a transportation polytope and reveals the linear-programming structure that algorithms exploit.
\index{transportation!polytope}

Monge's discrete matching problem is problematic because it cannot be applied when $n \neq m$. A faithful formulation must keep track of masses $(\a_i,\b_j)$. The continuous Monge problem~\eqref{eq-monge-continuous}, based on push-forward maps, has the same obstruction in another form: there may be no map $T$ such that $T_\sharp\al=\be$. This happens, for instance, when a single Dirac mass should be sent to several Dirac masses.
\index{assignment problem}
\index{Monge!problem}
\index{push-forward}
\index{Dirac mass}

This lack of mass splitting also makes the Monge formulation asymmetric in $\al$ and $\be$: one can map two Dirac masses to a single one, but not the other way around. Even when a feasible map exists, the resulting optimization problem is non-convex and therefore difficult to solve numerically.

Kantorovich's key idea~\cite{Kantorovich42} is to relax the deterministic nature of transportation. Instead of requiring each source point $x_i$ to be sent to exactly one target, the mass at $x_i$ may be dispatched across several locations. This moves from deterministic transport maps to probabilistic, or fuzzy, transport plans. The relaxation is encoded, in place of a permutation $\sigma$ or a map $T$, by a coupling matrix $\P\in\RR_+^{n\times m}$, where $\P_{i,j}$ describes the amount of mass flowing from $x_i$ to $y_j$ in the formalism of discrete measures
\index{plan!transport}
\index{transport map}
\[
	\al=\sum_i \a_i\de_{x_i},
	\qquad
	\be=\sum_j \b_j\de_{y_j}.
\]
\begin{defn}[Discrete couplings and mass conservation]\label{def-discrete-couplings}
\index{discrete!coupling}
\index{mass!conservation}
	Admissible couplings are only constrained to satisfy conservation of mass:
	\eql{\label{eq-discr-couplings}
		\CouplingsD(\a,\b)
		\eqdef
		\enscond{\P\in\RR_+^{n\times m}}{\P\ones_m=\a \qandq \transp{\P}\ones_n=\b}.
	}
	Equivalently,
	\[
		\P\ones_m=\left(\sum_j \P_{i,j}\right)_i\in\RR^n,
		\qquad
		\transp{\P}\ones_n=\left(\sum_i \P_{i,j}\right)_j\in\RR^m.
	\]
\end{defn}
The first useful consequence of this relaxation is feasibility: there is always at least one admissible plan, obtained by making source and target independent.

\begin{defn}[Discrete product coupling]\label{def-discrete-product-coupling}
\index{product!coupling}
	Given weights $\a\in\simplex_n$ and $\b\in\simplex_m$, the discrete product, or trivial, coupling is
	\[
		(\a\otimes\b)_{i,j}\eqdef \a_i\b_j.
	\]
	It belongs to $\CouplingsD(\a,\b)$ and corresponds to choosing the source and target labels independently.
\end{defn}
This feasible set is the bounded intersection of an affine space with the nonnegative orthant, hence a convex polytope. In one dimension, an ordered coupling can be read as a matrix: rows index source bins, columns index target bins, and the marginal constraints appear as prescribed row and column sums.
\index{marginal!constraint}

\begin{prop}[Discrete product optimality is degenerate]\label{prop-discrete-product-coupling-degenerate}
\index{product!coupling}
	Assume that the zero-mass rows and columns have been removed, so that $\a_i>0$ and $\b_j>0$, and let $\C$ be a finite cost matrix. The product plan $\a\otimes\b$ minimizes $\P\mapsto\dotp{\C}{\P}$ over $\CouplingsD(\a,\b)$ if and only if every coupling $\P\in\CouplingsD(\a,\b)$ minimizes it.
\index{cost matrix}
\end{prop}
\begin{proof}
	The reverse implication is immediate. Assume conversely that $\a\otimes\b$ is optimal and let $\Q\in\CouplingsD(\a,\b)$ be arbitrary. Since all entries of $\a\otimes\b$ are positive, there exists $t>0$ small enough that
	\[
		\R\eqdef (1+t)(\a\otimes\b)-t\Q
	\]
	has nonnegative entries. Its row and column sums are
	\[
		\R\ones_m=(1+t)\a-t\a=\a,
		\qquad
		\R^\top\ones_n=(1+t)\b-t\b=\b,
	\]
	so $\R\in\CouplingsD(\a,\b)$. Moreover,
	\[
		\a\otimes\b=\frac{1}{1+t}\R+\frac{t}{1+t}\Q.
	\]
	By optimality of $\a\otimes\b$, both $\dotp{\C}{\R}$ and $\dotp{\C}{\Q}$ are at least $\dotp{\C}{\a\otimes\b}$. Taking the scalar product of the convex-combination identity with $\C$ gives
	\[
		\dotp{\C}{\a\otimes\b}
		=
		\frac{1}{1+t}\dotp{\C}{\R}
		+\frac{t}{1+t}\dotp{\C}{\Q}.
	\]
	A convex average of two numbers not smaller than $\dotp{\C}{\a\otimes\b}$ can equal $\dotp{\C}{\a\otimes\b}$ only if both numbers are equal to it. Hence $\dotp{\C}{\Q}=\dotp{\C}{\a\otimes\b}$, and $\Q$ is optimal. Since $\Q$ was arbitrary, all couplings are optimal.
\end{proof}

Thus the product plan is mainly a feasibility witness. Except in the degenerate situation where the linear cost is constant on the whole transportation polytope, it is not expected to solve optimal transport. It is also maximally diffuse: when all masses are positive, it has $nm$ positive entries, whereas Proposition~\ref{prop-sparse-optimal-plans} shows below that sparse optimal plans with at most $n+m-1$ positive entries always exist.
\index{optimal plan}
\index{plan!sparse}
\index{transportation!polytope}

\begin{figure}[H]
\centering
\begin{tabular}{@{}ccc@{}}
\includegraphics[width=.28\linewidth]{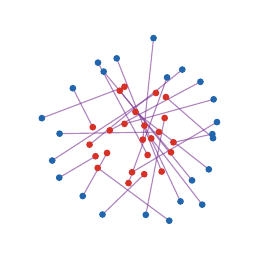} &
\includegraphics[width=.28\linewidth]{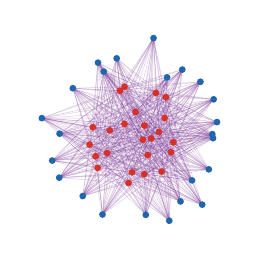} &
\includegraphics[width=.28\linewidth]{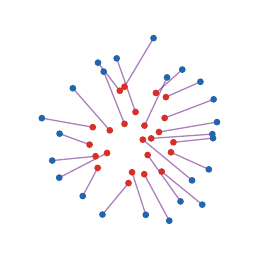} \\[-.1em]
\small deterministic graph & \small product plan $\a\otimes\b$ & \small optimal plan
\end{tabular}
\caption{Discrete couplings represented as straight transport segments on the canonical point clouds used in the matching section. The deterministic graph is a feasible Monge-type plan, the product plan spreads every source mass over all targets, and the optimal Kantorovich plan minimizes the quadratic transport cost. Line width and opacity encode the transported mass.}
\label{fig:kantorovich-coupling-polylines}
\end{figure}

\begin{figure}[H]
\centering
\begin{tabular}{@{}cccc@{}}
\small product, 20 bins & \small OT, 20 bins & \small product, 200 bins & \small OT, 200 bins \\[-.15em]
\includegraphics[width=.22\linewidth]{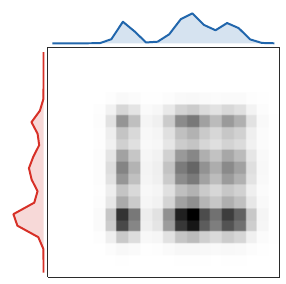} &
\includegraphics[width=.22\linewidth]{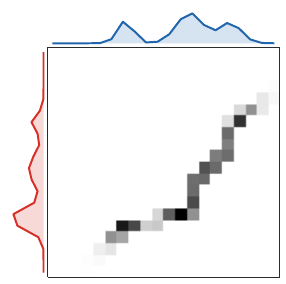} &
\includegraphics[width=.22\linewidth]{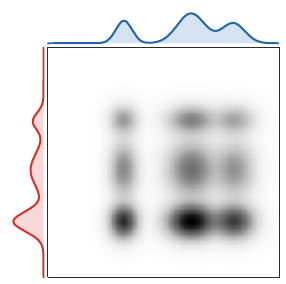} &
\includegraphics[width=.22\linewidth]{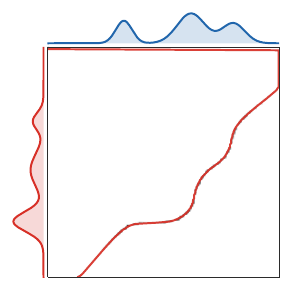}
\end{tabular}
\caption{Coupling matrices with their prescribed marginals. The central grayscale image displays $\P_{ij}$; the red curve on the left is the source marginal $\a$, and the blue curve on top is the target marginal $\b$. The independent product plan is diffuse, whereas the one-dimensional optimal plan concentrates near the monotone quantile correspondence. Only the dense 200-bin optimal panel overlays the barycentric projection $i\mapsto\sum_j \P_{ij}j/\a_i$, because that curve is meaningful visually only at sufficient resolution.}
\index{barycentric!projection}
\index{correspondence}
\label{fig:kantorovich-coupling-matrix-marginals}
\end{figure}

Whereas the Monge formulation is intrinsically asymmetric, Kantorovich's relaxed formulation is symmetric at the level of feasible sets: a coupling $\P$ belongs to $\CouplingsD(\a,\b)$ if and only if $\transp{\P}$ belongs to $\CouplingsD(\b,\a)$.
\index{Kantorovich!relaxation}

Kantorovich, aiming for economic planning, made a strong simplifying assumption: the cost of transportation should be linear in the amount of transported mass. Under this assumption, denoting $\C_{i,j}$ the cost of moving a unit amount of mass from $x_i$ to $y_j$, the discrete Kantorovich problem reads
\index{Kantorovich!problem}
\eql{\label{eq-kanto-discr}
	\MKD_{\C}(\a,\b)
	\eqdef
	\umin{\P\in\CouplingsD(\a,\b)}
		\dotp{\C}{\P}
	\eqdef
	\umin{\P\in\CouplingsD(\a,\b)}
		\sum_{i,j}\C_{i,j}\P_{i,j}.
}
This is a linear program, and its solutions need not be unique.

\begin{figure}[H]
\centering
\begin{tabular}{@{}cccc@{}}
\includegraphics[width=.18\linewidth]{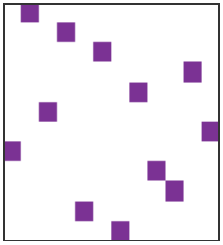} &
\index{matrix!permutation}
\includegraphics[width=.23\linewidth]{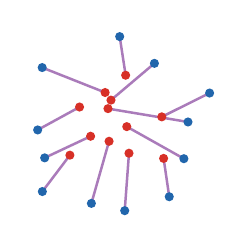} &
\includegraphics[width=.16\linewidth]{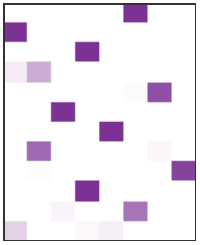} &
\includegraphics[width=.23\linewidth]{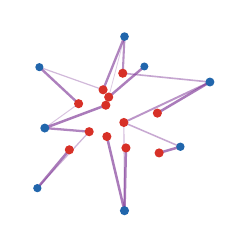} \\[-.1em]
\small permutation matrix $\P_\sigma$ &
\index{matrix!permutation}
\small assignment segments &
\small splitting matrix $\P$ &
\small splitting segments
\end{tabular}
\caption{From permutation matrices to splitting couplings. When the two empirical measures have the same number of atoms and uniform weights, an optimal plan can be a permutation matrix. Once target masses are nonuniform, the admissible set $\CouplingsD(\a,\b)$ also contains plans in which one source sends mass to several targets and several sources merge into the same target.}
\index{matrix!permutation}
\index{empirical!measure}
\label{fig:kantorovich-permutation-versus-splitting}
\end{figure}

\begin{prop}[Sparse optimal plans]\label{prop-sparse-optimal-plans}
\index{optimal plan}
\index{plan!sparse}
	Assume $\a_i>0$, $\b_j>0$ and $\sum_i \a_i=\sum_j \b_j=1$. The linear program~\eqref{eq-kanto-discr} admits an optimal coupling with at most $n+m-1$ nonzero entries.
\index{optimal coupling}
\end{prop}
\begin{proof}
	The transportation polytope is compact, so a linear objective attains its minimum at an extreme point. The row and column marginal equations have rank $n+m-1$: the only linear redundancy is that both sets of constraints impose the same total mass. The support-graph argument below is the combinatorial form of this basic-feasible-variable count.
\index{linear objective}
\index{transportation!polytope}
\index{extreme point}
\index{graph!support}

	Let $\P$ be an extreme point and let $E=\{(i,j):\P_{ij}>0\}$ be its support graph on the bipartite vertex set $\{1,\ldots,n\}\cup\{1,\ldots,m\}$. If this graph contains a cycle, orient the cycle and put alternating signs $+1,-1$ on its edges, obtaining a nonzero matrix $H$ supported on $E$ with zero row and column sums. For sufficiently small $t>0$, both $\P+tH$ and $\P-tH$ are nonnegative couplings, and $\P$ is their midpoint, contradicting extremality. Thus the support graph is a forest. Since a forest on $n+m$ vertices has at most $n+m-1$ edges, the claim follows.
\end{proof}

\begin{prop}[North-west corner feasible plan]\label{prop-northwest-corner}
\index{north-west corner!plan}
	Let $\a\in\RR_+^n$ and $\b\in\RR_+^m$ have the same positive total mass. The following greedy sweep constructs a coupling $\P\in\CouplingsD(\a,\b)$ with at most $n+m-1$ positive entries and an acyclic positive support. Starting from $(i,j)=(1,1)$ with residual masses $r_i=\a_i$ and $s_j=\b_j$, skip zero residuals, set
\index{greedy sweep}
	\[
		\P_{ij}=\min(r_i,s_j),
	\]
	subtract this value from both residuals, and advance every index whose residual has become zero. Repeat until all residual masses are exhausted.
\end{prop}
\begin{proof}
	All assignments are nonnegative. At each step, the mass placed in entry $(i,j)$ is subtracted from exactly one current row residual and one current column residual, so no row or column can receive more mass than prescribed. Conversely, an index is advanced only when its residual has been fully filled. When the algorithm stops, the total assigned mass is $\sum_i \a_i=\sum_j \b_j$, hence all row and column sums are exactly $\a$ and $\b$.

	Each positive assignment exhausts at least one current row or one current column. Before the final assignment, at most $n-1$ row advances and $m-1$ column advances can occur without terminating the construction. Hence the number of positive entries is at most $(n-1)+(m-1)+1=n+m-1$. For acyclicity, view the positive support as a bipartite graph. Once a row or column index is advanced, it never appears again, so each new positive edge either starts a new component or attaches at least one new vertex to the component currently being swept. No edge is ever added between two old vertices of the same component, so no cycle can be created.
\index{graph!bipartite}
\end{proof}

\begin{alg}[North-west corner coupling]\label{alg:north-west-corner}
\textbf{Input:} Source weights $\a\in\simplex_n$ and target weights $\b\in\simplex_m$.

\textbf{Output:} Sparse feasible coupling $\P\in\CouplingsD(\a,\b)$.

\textbf{Initialize:} Set $\P=0$, $r=\a$, $s=\b$, and $(i,j)=(1,1)$.

\textbf{While} $i\leq n$ and $j\leq m$ \textbf{do}:
\begin{algblock}
\(\eta=\min(r_i,s_j), \qquad \P_{ij}\leftarrow \eta.\)

\textbf{Update residuals:}
\(r_i\leftarrow r_i-\eta, \qquad s_j\leftarrow s_j-\eta.\)

\textbf{If} $r_i=0$ \textbf{then}:
\begin{algblock}

\textbf{Set} $i\leftarrow i+1$.

\end{algblock}
\textbf{If} $s_j=0$ \textbf{then}:
\begin{algblock}

\textbf{Set} $j\leftarrow j+1$.

\end{algblock}
\end{algblock}
\algreturnskip
\textbf{Return} $\P$.
\end{alg}

The north-west corner rule, summarized in Algorithm~\ref{alg:north-west-corner}, does not use the cost matrix and is therefore not meant to solve~\eqref{eq-kanto-discr}. Its role is algorithmic: an acyclic support corresponds to linearly independent marginal constraints. When the support has fewer than $n+m-1$ positive entries, transportation simplex implementations complete it with zero-mass basic variables to obtain a degenerate basic feasible solution. This gives a cheap initialization for the pivoting methods discussed in Section~\ref{sec-kantorovich-lp-algorithms}.
\index{cost matrix}
\index{north-west corner!rule}
\index{linear programming!basic feasible solution}
\index{transportation!simplex}
\index{marginal!constraint}

\paragraph{One-dimensional cases.}
\index{one-dimensional!transport}

In one dimension, the transportation polytope has a canonical monotone optimizer. This is the weighted version of the sorting rule from Section~\ref{sec-monge-pbm}.
\index{transportation!polytope}

\begin{prop}[One-dimensional weighted sweep]\label{prop-1d-weighted-sweep}
\index{weighted!sweep}
	Let $x_1\leq\cdots\leq x_n$ and $y_1\leq\cdots\leq y_m$ be points on the line, and let $c(x,y)=h(x-y)$ with $h$ convex. The north-west corner plan between the sorted weighted atoms is optimal for~\eqref{eq-kanto-discr}. Consequently, for unsorted one-dimensional inputs, an optimal plan is obtained in $O(n\log n+m\log m)$ time by sorting and then sweeping the masses once from left to right.
\index{optimal plan}
\index{north-west corner!plan}
\end{prop}
\begin{proof}
	The north-west plan is monotone: if $i<i'$ and $j>j'$, it cannot put positive mass on both $(i,j)$ and $(i',j')$, because the sweep exhausts rows and columns in increasing order. Conversely, any feasible plan with a crossing pair of positive entries can be improved by moving a small mass $\eta$ from $(i,j)$ and $(i',j')$ to $(i,j')$ and $(i',j)$. The two marginals are unchanged, and convexity of $h$ gives
	\[
		h(x_i-y_j)+h(x_{i'}-y_{j'})
		\geq
		h(x_i-y_{j'})+h(x_{i'}-y_j)
	\]
	for $i<i'$ and $j'<j$, with strict inequality for strictly convex $h$ and distinct points. Repeating this uncrossing procedure until no crossing remains yields a monotone optimal plan. There is only one monotone feasible plan with the prescribed sorted marginals, namely the sweep plan: it pairs the leftmost remaining source mass with the leftmost remaining target mass at every step. Sorting costs $O(n\log n+m\log m)$ and the sweep uses at most $n+m-1$ assignments.
\end{proof}

\begin{alg}[Weighted one-dimensional sweep]\label{alg:weighted-one-dimensional-sweep}
\textbf{Input:} One-dimensional atoms $(x_i,\a_i)$ and $(y_j,\b_j)$; convex cost $h(x-y)$.

\textbf{Output:} Monotone optimal coupling $\P$.

\textbf{Sort} atoms:
\(x_1\leq\cdots\leq x_n, \qquad y_1\leq\cdots\leq y_m.\)

\textbf{Set} $\P$ to the output of Algorithm~\ref{alg:north-west-corner} applied to the sorted weights $(\a_i)_i$ and $(\b_j)_j$.
\textbf{Return} $\P$.
\end{alg}

\paragraph{Permutation matrices as couplings.}
\index{matrix!permutation}

We now restrict attention to the special case $n=m$ and uniform weights $\a=\b=\ones_n/n$. In this case a matching can be encoded as a matrix with exactly one active entry per row and per column.

\begin{defn}[Permutation matrices]\label{def-permutation-matrices}
\index{matrix!permutation}
\index{permutation!matrix}
	For a permutation $\sigma\in\Perm(n)$, its permutation matrix $P_\sigma$ is
	\[
		(P_\sigma)_{i,j}=\choice{1 \qifq j=\sigma(i),\\ 0 \quad\text{otherwise}.}
	\]
	The set of all permutation matrices is
	\[
		\mathcal P_n^{\mathrm{perm}}
		\eqdef
		\enscond{P_\sigma}{\sigma\in\Perm(n)} .
	\]
\end{defn}
The corresponding probability coupling is $P_\sigma/n$. If the matching cost matrix is $\C$, then
\index{cost matrix}
\[
	\dotp{\C}{P_\sigma/n}=\frac1n\sum_{i=1}^n C_{i,\sigma(i)}.
\]
Thus the assignment problem is the minimization of a linear function over the discrete, non-convex set of permutation matrices.
\index{matrix!permutation}
\index{assignment problem}

The convex relaxation replaces this finite set by all bistochastic matrices.

\begin{defn}[Birkhoff polytope]\label{def-birkhoff-polytope}
\index{bistochastic matrix}
\index{Birkhoff polytope}
	The Birkhoff polytope is the convex set of bistochastic matrices
	\[
		\mathcal B_n
		\eqdef
		\enscond{P\in\RR_+^{n\times n}}{P\ones_n=\ones_n \qandq \transp{P}\ones_n=\ones_n}.
	\]
\end{defn}
Then $\CouplingsD(\ones_n/n,\ones_n/n)=\mathcal B_n/n$, and permutation couplings are included in this convex relaxation. More precisely,
\[
	\mathcal P_n^{\mathrm{perm}}
	=
	\mathcal B_n\cap\{0,1\}^{n\times n}
	\subset
	\mathcal B_n,
\]
so before using the structure of $\mathcal B_n$ one first obtains only the relaxed inequality
\[
	\min_{P\in\mathcal B_n}\dotp{\C}{P}
	\leq
	\min_{P\in\mathcal P_n^{\mathrm{perm}}}\dotp{\C}{P}.
\]
The next elementary facts lead to the Birkhoff--von Neumann theorem~\cite{birkhoff,vonNeumann1953assignment}, which explains why this relaxation is tight for uniform matching.
\index{Birkhoff-von Neumann theorem}
\index{matching!uniform}

\begin{defn}[Extreme points]\label{def-extreme-points}
\index{extreme point}
	For a compact convex set $\mathcal C$ in a finite-dimensional vector space,
	\[
		\operatorname{Extr}(\mathcal C)
		\eqdef
		\enscond{x\in\mathcal C}{x=(y+z)/2,\ y,z\in\mathcal C \Rightarrow y=z=x}.
	\]
\end{defn}

\begin{prop}[Existence of extreme points]\label{prop-extreme-point-existence}
\index{extreme point}
	If $\mathcal C$ is a non-empty compact convex subset of a finite-dimensional vector space, then $\operatorname{Extr}(\mathcal C)$ is non-empty.
\end{prop}
\begin{proof}
	Among all non-empty faces of $\mathcal C$, choose one of minimal affine dimension. If this face contained two distinct points, maximizing a linear functional that is not constant on the face would produce a non-empty proper exposed subface, contradicting minimality. Hence the minimal face is a singleton, and its point is extreme.
\end{proof}

\begin{example}[Unbounded convex sets may have no extreme point]
	Compactness cannot be dropped from Proposition~\ref{prop-extreme-point-existence}. For instance, the closed convex set $\enscond{(x,y)\in\RR_+^2}{xy\geq1}$ is unbounded and has no extreme point.
\index{extreme point}
\end{example}

\begin{prop}[Linear programs have extreme minimizers]\label{prop-linear-program-extreme-minimizer}
\index{extreme minimizer}
\index{linear programming}
\index{extreme point}
	Let $\mathcal C$ be non-empty and compact. For every linear form $\ell$,
\index{linear form}
	\[
		\operatorname{Extr}(\mathcal C)\cap\operatorname*{argmin}_{x\in\mathcal C}\ell(x)
		\neq\emptyset.
	\]
\end{prop}
\begin{proof}
	The set $S=\operatorname*{argmin}_{x\in\mathcal C}\ell(x)$ is non-empty, compact and convex. By Proposition~\ref{prop-extreme-point-existence}, it has an extreme point $x$. If $x=(y+z)/2$ with $y,z\in\mathcal C$, then by linearity and optimality of $x$, both $y$ and $z$ also minimize $\ell$ on $\mathcal C$, hence $y,z\in S$. Since $x$ is extreme in $S$, $y=z=x$. Thus $x$ is extreme in $\mathcal C$.
\index{extreme point}
\end{proof}

\begin{thm}[Birkhoff--von Neumann]\label{thm-birkhoff-von-neumann}
\index{Birkhoff-von Neumann theorem}
	The extreme points of $\mathcal B_n$ are exactly the permutation matrices.
\index{extreme point}
\end{thm}
\begin{proof}
	We first prove that permutation matrices are extreme. Let $P_\sigma\in\mathcal P_n^{\mathrm{perm}}$ and assume that
	\[
		P_\sigma=\frac{Q+R}{2}
		\qquad\text{with}\qquad
		Q,R\in\mathcal B_n .
	\]
	Every bistochastic matrix has entries in $[0,1]$. Since the only extreme points of $[0,1]$ are $0$ and $1$, each entry of $P_\sigma$ fixes the corresponding entries of $Q$ and $R$: if $(P_\sigma)_{ij}=0$, then $Q_{ij}=R_{ij}=0$, while if $(P_\sigma)_{ij}=1$, then $Q_{ij}=R_{ij}=1$. Hence $Q=R=P_\sigma$, so $P_\sigma$ is extreme.
	\index{matrix!permutation}

	We now prove the converse by contrapositive. Pick $P\in\mathcal B_n\setminus\mathcal P_n^{\mathrm{perm}}$. Since an integral bistochastic matrix is necessarily a permutation matrix, $P$ has at least one fractional entry. We shall split
	\[
		P=\frac{Q+R}{2}
	\]
	with $Q,R\in\mathcal B_n$ and $Q\neq R$, proving that $P$ is not extreme.

	Associate with $P$ the bipartite graph whose left vertices are the rows, whose right vertices are the columns, and whose edges are the fractional entries
	\[
		0<P_{ij}<1 .
	\]
	An entry equal to $1$ uses the whole mass of its row and column, so it is isolated in the positive support and does not appear in this fractional graph. If a left vertex $i$ is incident to a fractional edge $(i,j_1)$, then it must be incident to at least one other fractional edge. Indeed, the row sum is one; after the contribution $P_{i,j_1}\in(0,1)$, a positive amount $1-P_{i,j_1}$ remains in the same row, and it cannot be carried by an entry equal to $1$. The same argument applies to right vertices, using the column sums. Thus every non-isolated vertex of the fractional graph has degree at least two.

	Starting from any fractional edge, one may therefore walk through adjacent fractional edges without immediately backtracking and without getting stuck. Since the graph is finite, some vertex is eventually visited twice; the portion of the walk between the two visits contains a cycle. Choose a shortest such cycle and write it in alternating form
	\[
		(i_1,j_1,i_2,j_2,\ldots,i_p,j_p),
		\qquad i_{p+1}=i_1 ,
	\]
	where both $(i_s,j_s)$ and $(i_{s+1},j_s)$ are fractional for every $s$. The minimality of the cycle implies that the vertices $i_s$ are all distinct and that the vertices $j_s$ are all distinct. In particular,
	\[
		0<P_{i_s,j_s}<1
		\qandq
		0<P_{i_{s+1},j_s}<1 .
	\]
	Define
	\[
		\epsilon
		\eqdef
		\min_{1\leq s\leq p}
		\bigl\{
		P_{i_s,j_s},\,
		P_{i_{s+1},j_s},\,
		1-P_{i_s,j_s},\,
		1-P_{i_{s+1},j_s}
		\bigr\}.
	\]
	All these numbers are positive, so $\epsilon>0$. Split the cycle edges into the two alternating families
	\[
		A\eqdef \{(i_s,j_s)\}_{s=1}^p,
		\qquad
		B\eqdef \{(i_{s+1},j_s)\}_{s=1}^p .
	\]
	We now perform the standard alternating-cycle perturbation:
	\[
		Q_{ij}\eqdef
		\begin{cases}
			P_{ij}, & (i,j)\notin A\cup B,\\
			P_{ij}+\epsilon/2, & (i,j)\in A,\\
			P_{ij}-\epsilon/2, & (i,j)\in B,
		\end{cases}
		\qquad
		R_{ij}\eqdef
		\begin{cases}
			P_{ij}, & (i,j)\notin A\cup B,\\
			P_{ij}-\epsilon/2, & (i,j)\in A,\\
			P_{ij}+\epsilon/2, & (i,j)\in B .
		\end{cases}
	\]
	By the definition of $\epsilon$, all modified entries stay in $[0,1]$, so $Q$ and $R$ are nonnegative. Each row vertex $i_s$ of the cycle is incident to exactly one edge of $A$ and one edge of $B$; the $+\epsilon/2$ and $-\epsilon/2$ perturbations therefore cancel in that row. The same cancellation holds in each column vertex $j_s$, and all other rows and columns are unchanged. Consequently
	\[
		Q\ones_n=R\ones_n=\ones_n,
		\qquad
		\transp{Q}\ones_n=\transp{R}\ones_n=\ones_n,
	\]
	so $Q,R\in\mathcal B_n$. Finally, $Q\neq R$ because $\epsilon>0$ and the cycle is non-empty, while by construction $P=(Q+R)/2$. Thus $P$ is not extreme. Consequently every extreme point of $\mathcal B_n$ is integral, and every integral bistochastic matrix is a permutation matrix.
	\index{bistochastic matrix}
	\index{graph!bipartite}
\index{extreme point}
\end{proof}

The same combinatorial idea gives the constructive decomposition used to express a bistochastic matrix as a convex combination of permutations.

\begin{alg}[Birkhoff--von Neumann decomposition]\label{alg:birkhoff-von-neumann-decomposition}
\textbf{Input:} Bistochastic matrix $P\in\mathcal B_n$.

\textbf{Output:} Decomposition $P=\sum_r\lambda_rP_{\sigma_r}$.

\textbf{Initialize:} Set $R=P$ and $\mathcal L=\emptyset$.

\textbf{While} $R\neq0$ \textbf{do}:
\begin{algblock}

\textbf{Build} bipartite graph $G_R=\{(i,j):R_{ij}>0\}$.

\textbf{Set} $\sigma$ to the lexicographically first perfect matching of $G_R$.

\textbf{Set}
\(\lambda=\min_i R_{i,\sigma(i)}.\)

\textbf{Append} $(\lambda,\sigma)$ to $\mathcal L$.

\textbf{Update}
\(R\leftarrow R-\lambda P_\sigma .\)
\end{algblock}
\algreturnskip
\textbf{Return} \(P=\sum_{(\lambda_r,\sigma_r)\in\mathcal L}\lambda_rP_{\sigma_r}, \qquad \sum_r\lambda_r=1.\)
\end{alg}

\begin{cor}[Kantorovich for matching]\label{cor-kantorovich-matching}
\index{Kantorovich!relaxation}
	If $m=n$ and $\a=\b=\ones_n/n$, then the discrete Kantorovich problem~\eqref{eq-kanto-discr} admits an optimal solution of the form $P_\sigma/n$. The associated permutation $\sigma$ solves the assignment problem of Section~\ref{sec-monge-pbm}.
\index{Kantorovich!problem}
\index{assignment problem}
\end{cor}
\begin{proof}
	The feasible set is $\mathcal B_n/n$. By Proposition~\ref{prop-linear-program-extreme-minimizer}, the linear objective has an optimal extreme point. Since scaling preserves extreme points and Theorem~\ref{thm-birkhoff-von-neumann} identifies the extreme points of $\mathcal B_n$, this optimizer is $P_\sigma/n$ for some permutation $\sigma$. Its cost is exactly $n^{-1}\sum_i C_{i,\sigma(i)}$, so $\sigma$ is an optimal assignment.
\index{linear objective}
\index{Birkhoff-von Neumann theorem}
\index{extreme point}
\end{proof}

Equivalently, for uniform empirical measures, one can always choose a permutation matrix among the minimizers of the relaxed Kantorovich problem: the relaxation is tight for assignment problems.
\index{Kantorovich!problem}
\index{empirical!measure}
\index{assignment problem}
\index{matrix!permutation}

\begin{rem}[General discrete case]
	For general input measures, one does not have equivalence between Monge and Kantorovich problems, since the Monge constraint can be empty. In finite dimension, however, the support of an optimal coupling still enjoys strong sparsity: one can choose an optimal basic feasible plan whose bipartite support is cycle-free, hence with at most $n+m-1$ nonzero entries. Figure~\ref{fig:kantorovich-permutation-versus-splitting} illustrates the difference between the tight uniform matching case and the genuinely splitting nonuniform case.
\index{support}
\index{Kantorovich!problem}
\index{optimal coupling}
\index{matching!uniform}
\end{rem}

\section{Linear-Programming Algorithms}
\index{linear programming}
\label{sec-kantorovich-lp-algorithms}

The discrete Kantorovich problem is a linear program with much more structure than a generic dense LP. Its variables are the arcs of a complete bipartite network, its equality constraints are flow-conservation constraints, and its extreme points are sparse tree-like couplings. This is why classical transportation algorithms remain important even though the formulation is finite-dimensional and convex.
\index{Kantorovich!problem}
\index{equality constraint}
\index{flow!conservation}
\index{extreme point}

\paragraph{Transportation simplex and network simplex.}
\index{transportation!simplex}
\index{simplex network}

The transportation simplex goes back to Dantzig's formulation of the transportation problem~\cite{Dantzig51}. It works on basic feasible couplings, whose positive support is completed into a spanning tree of the bipartite supply-demand graph. Starting from a basis, for instance one produced by the north-west corner rule, reduced costs identify whether an unused arc can decrease the objective. Adding such an arc creates a unique cycle in the tree; one then pushes as much mass as possible around this cycle and removes the exhausted arc. This is exactly the simplex method specialized to the transportation polytope, with pivots that can be implemented using graph operations instead of a generic basis inverse.
\index{spanning tree}
\index{feasible coupling}
\index{north-west corner!rule}
\index{transportation!polytope}
\index{transportation!simplex}
\index{cost!reduced}

The network simplex is the corresponding pivoting method for general minimum-cost-flow problems~\cite{bertsekas1988dual}. It keeps node potentials, reduced costs and a spanning-tree basis, and it has become one of the most effective exact solvers for medium-scale discrete OT. Like the ordinary simplex method, its worst-case number of pivots can be exponential for adversarial instances, but the per-pivot operations exploit sparsity and are very efficient in practice. For theoretical polynomial guarantees, one can instead use strongly polynomial minimum-cost-flow algorithms, such as Orlin's algorithm~\cite{Orlin1997}. In a dense balanced transportation problem with $n$ sources and $n$ targets, the graph has $O(n)$ vertices and $O(n^2)$ arcs, so these general bounds are polynomial but still much heavier than the nearly matrix-vector structure that Sinkhorn will exploit.
\index{flow!minimum-cost}
\index{simplex network}

\paragraph{Interior-point methods.}
\index{interior-point method}

Generic interior-point methods approach the same LP through a smooth central path. For the transport polytope, the logarithmic-barrier version is
\index{barrier!logarithmic}
\index{transportation!polytope}
\index{central!path}
\eql{\label{eq-transport-log-barrier}
	\P_\epsilon
	\eqdef
	\uargmin{\substack{\P\ones_m=\a,\ \P^\top\ones_n=\b\\ \P_{ij}>0}}
		\dotp{\C}{\P}
		-
		\epsilon\sum_{i,j}\log \P_{ij},
}
where $\epsilon>0$ is decreased along the algorithm. The barrier is singular at the boundary, so each iterate stays strictly inside the transportation polytope; as $\epsilon\downarrow0$, the central path approaches the set of LP minimizers. Each Newton step solves a linear system involving the marginal constraints and the current diagonal Hessian $\operatorname{diag}(\epsilon/\P_{ij}^2)$, which gives robust polynomial complexity but can be expensive for dense couplings~\cite{nesterov1994interior}.
\index{Newton step}
\index{transportation!polytope}
\index{marginal!constraint}
\index{central!path}

\begin{figure}[H]
\centering
\setlength{\tabcolsep}{2pt}
\begin{tabular}{@{}cccc@{}}
\includegraphics[width=.21\linewidth]{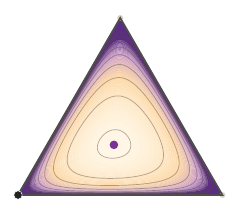} &
\includegraphics[width=.21\linewidth]{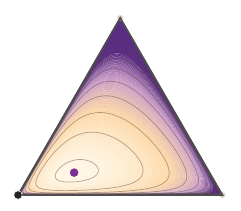} &
\includegraphics[width=.21\linewidth]{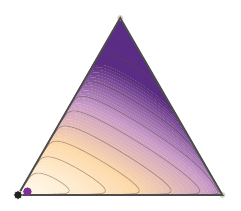} &
\includegraphics[width=.21\linewidth]{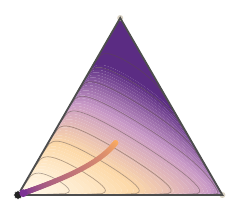} \\[-.1em]
\small large $\epsilon$ &
\small medium $\epsilon$ &
\small small $\epsilon$ &
\small central path
\index{central!path}
\end{tabular}
\caption{Logarithmic-barrier central path for a two-dimensional equilateral triangular slice of a linear program. The feasible triangle is constrained by positive slacks, and the displayed objective is $\ell^\top z-\epsilon\sum_i\log(b_i-\dotp{a_i}{z})$. Large $\epsilon$ selects a central interior point; decreasing $\epsilon$ moves the minimizer toward the optimal vertex while never touching the boundary. This should be contrasted with entropic OT, where the entropy temperature $\epsilon$ is usually fixed and defines the regularized problem itself.}
\index{barrier!logarithmic}
\label{fig:kantorovich-log-barrier-lp-geometry}
\end{figure}

The comparison with Sinkhorn is therefore subtle. Both methods keep iterates positive, but they use positivity in different ways. Interior-point algorithms solve the original LP by decreasing the barrier parameter $\epsilon$ and following a central path. Sinkhorn fixes an entropic temperature $\epsilon$ and solves a different, KL-regularized OT problem by alternating diagonal scalings. The parameter $\epsilon$ may be decreased in continuation strategies, but for a fixed run it is part of the objective rather than only an algorithmic barrier.
\index{barrier!parameter}
\index{central!path}

\section{Relaxation for Arbitrary Measures}
\index{Kantorovich!relaxation}
\label{sec-kantorovich-continuous}

This section lifts the finite-dimensional coupling matrix to a joint probability measure. The payoff is that existence, duality and metric properties can be stated for arbitrary laws, including discrete, singular and continuous distributions.
\index{probability measure}

\paragraph{Continuous couplings.}
\index{continuous!coupling}

The first step is to formalize what it means for a joint law to have prescribed marginals.
\index{joint!law}

\begin{defn}[Marginals of a joint measure]\label{def-joint-marginals}
\index{joint!measure}
	Let $\pi\in\Mm_+^1(\Xx\times\Yy)$ and let $P_\Xx(x,y)=x$, $P_\Yy(x,y)=y$ be the coordinate projections. The marginals of $\pi$ are
	\[
		\pi_1\eqdef(P_\Xx)_\sharp\pi\in\Mm_+^1(\Xx),
		\qquad
		\pi_2\eqdef(P_\Yy)_\sharp\pi\in\Mm_+^1(\Yy).
	\]
	Equivalently, for all bounded continuous test functions $f$ on $\Xx$ and $g$ on $\Yy$,
	\[
		\int_{\Xx\times\Yy} f(x)\d\pi(x,y)=\int_\Xx f\d\pi_1,
		\qquad
		\int_{\Xx\times\Yy} g(y)\d\pi(x,y)=\int_\Yy g\d\pi_2.
	\]
\end{defn}

A useful mnemonic for the marginal constraint $\pi_1=\al$ and $\pi_2=\be$ is the formal row-and-column notation
\index{marginal!constraint}
\[
	\int_\Yy \d\pi(x,y)=\d\al(x),
	\qquad
	\int_\Xx \d\pi(x,y)=\d\be(y),
\]
which is made rigorous by Definition~\ref{def-joint-marginals}, or equivalently by the identities $\pi(A\times\Yy)=\al(A)$ and $\pi(\Xx\times B)=\be(B)$ for measurable sets $A\subset\Xx$ and $B\subset\Yy$.

\begin{defn}[Couplings]\label{def-continuous-couplings}
\index{continuous!coupling}
	Given $\al\in\Mm_+^1(\Xx)$ and $\be\in\Mm_+^1(\Yy)$, the set of couplings between $\al$ and $\be$ is
	\eql{\label{eq-coupling-generic}
		\Couplings(\al,\be)
		\eqdef
		\enscond{\pi\in\Mm_+^1(\Xx\times\Yy)}{\pi_1=\al \qandq \pi_2=\be}.
	}
	This is the continuous analogue of the transportation polytope~\eqref{eq-discr-couplings}.
\index{transportation!polytope}
\end{defn}

\begin{rem}[Probabilistic interpretation of couplings]
	If $X\sim\al$ and $Y\sim\be$, then $\pi\in\Couplings(\al,\be)$ means that $\pi$ is the law of a pair $(X,Y)$ whose coordinates have laws $\al$ and $\be$. The coupling encodes the dependence between $X$ and $Y$. The tensor product $\al\otimes\be$ corresponds to independence, whereas a graph coupling $(\Id,T)_\sharp\al$ corresponds to the deterministic relation $Y=T(X)$.

	In the discrete case, when $\al=\sum_i \a_i\de_{x_i}$ and $\be=\sum_j \b_j\de_{y_j}$, the constraint $\pi_1=\al$ and $\pi_2=\be$ forces every coupling to have the form $\pi=\sum_{i,j}\P_{ij}\de_{(x_i,y_j)}$ with $\P\in\CouplingsD(\a,\b)$. The discrete formulation is therefore a special case of the continuous one, not merely an approximation.
\end{rem}

Unlike the Monge constraint, the coupling constraint is never empty. The continuous feasibility witness is the tensor product coupling, the measure-theoretic version of the discrete product plan above.

\begin{defn}[Tensor product and trivial coupling]\label{def-tensor-product-coupling}
\index{product!coupling}
\index{tensor product!coupling}
	Given $\al\in\Mm_+^1(\Xx)$ and $\be\in\Mm_+^1(\Yy)$, the tensor product coupling $\al\otimes\be$ is the probability measure on $\Xx\times\Yy$ defined by
	\[
		\int_{\Xx\times\Yy} h(x,y)\d(\al\otimes\be)(x,y)
		=
		\int_\Xx\left(\int_\Yy h(x,y)\d\be(y)\right)\d\al(x)
	\]
	for every bounded measurable $h$. It is also called the trivial coupling because it makes the two coordinates independent.
\end{defn}
Indeed, for every $f\in\Cc_b(\Xx)$,
\[
	\int_{\Xx\times\Yy} f(x)\d(\al\otimes\be)(x,y)
	=
	\left(\int_\Xx f(x)\d\al(x)\right)\left(\int_\Yy\d\be(y)\right)
	=
	\int f\d\al,
\]
and similarly for the second marginal, so $\al\otimes\be\in\Couplings(\al,\be)$.

\begin{prop}[Product optimality is degenerate]\label{prop-product-coupling-degenerate}
\index{product!coupling}
	Assume that $\Xx$ and $\Yy$ are compact metric spaces and that $c\in\Cc(\Xx\times\Yy)$. The following statements are equivalent:
	\[
		\al\otimes\be\in\uargmin{\pi\in\Couplings(\al,\be)}\int c\,\d\pi,
		\qquad
		\text{every coupling in }\Couplings(\al,\be)\text{ is optimal}.
	\]
	They are also equivalent to the additive decomposition of the cost on the product support,
	\[
		c(x,y)=u(x)+v(y).
	\]
\end{prop}
\begin{proof}
	If every coupling is optimal, then $\al\otimes\be$ is optimal. Conversely, assume that $\al\otimes\be$ is optimal. We first show that, for every $x_0,x_1\in\supp(\al)$ and $y_0,y_1\in\supp(\be)$,
	\[
		c(x_0,y_0)+c(x_1,y_1)
		=
		c(x_0,y_1)+c(x_1,y_0).
	\]
	Indeed, if this equality failed, after exchanging $y_0$ and $y_1$ if necessary one would have a strict inequality
	\[
		c(x_0,y_0)+c(x_1,y_1)
		>
		c(x_0,y_1)+c(x_1,y_0).
	\]
	By continuity, the strict inequality persists with a uniform margin on small neighborhoods $U_0,U_1$ of $x_0,x_1$ and $V_0,V_1$ of $y_0,y_1$, chosen disjoint within each pair. Since the four points lie in the supports, these neighborhoods have positive marginal mass. Denote by $\al_i$ and $\be_i$ the normalized restrictions of $\al$ to $U_i$ and of $\be$ to $V_i$, and choose
	\[
		0<\lambda\leq
		\min\{\al(U_0)\be(V_0),\al(U_1)\be(V_1)\}.
	\]
	The exchanged measure
	\[
		\tilde\pi
		=
		\al\otimes\be
		-\lambda\,\al_0\otimes\be_0
		-\lambda\,\al_1\otimes\be_1
		+\lambda\,\al_0\otimes\be_1
		+\lambda\,\al_1\otimes\be_0
	\]
	is nonnegative and has the same two marginals as $\al\otimes\be$. The uniform strict inequality on the neighborhoods implies that $\int c\,\d\tilde\pi<\int c\,\d(\al\otimes\be)$, contradicting optimality.

	Fixing any $x_\star\in\supp(\al)$ and $y_\star\in\supp(\be)$, the equality of cross differences gives, for all $(x,y)\in\supp(\al)\times\supp(\be)$,
	\[
		c(x,y)=c(x,y_\star)+c(x_\star,y)-c(x_\star,y_\star).
	\]
	Thus $c=u+v$ on the product support. Every coupling is concentrated on this product support, so for any $\pi\in\Couplings(\al,\be)$,
	\[
		\int c\,\d\pi
		=
		\int u\,\d\al+\int v\,\d\be,
	\]
	which depends only on the marginals. Hence all couplings are optimal.
\end{proof}

The tensor product is therefore a trivial feasible coupling, not a typical optimizer. Product optimality means that the cost cannot distinguish between dependences once the marginals are fixed. The continuity assumption is important: if $\al=\be$ is the uniform law on $[0,1]$ and $c(x,y)=\ones_{\{x=y\}}$, then $\al\otimes\be$ has zero cost and is optimal, whereas the identity coupling has cost one. Thus, for arbitrary merely measurable costs, changing the cost on an $\al\otimes\be$-negligible set may affect singular couplings without changing the product cost.
\index{feasible coupling}

If there exists a map $T:\Xx\to\Yy$ such that $T_\sharp\al=\be$, then the Monge map induces the graph coupling $\pi=(\Id,T)_\sharp\al\in\Couplings(\al,\be)$, characterized by
\index{Monge!problem}
\[
	\int_{\Xx\times\Yy} h(x,y)\d\pi(x,y)
	=
	\int_\Xx h(x,T(x))\d\al(x).
\]
Applying this identity to $h(x,y)=f(x)$ or $h(x,y)=g(y)$ gives respectively $\pi_1=\al$ and $\pi_2=\be$. Thus graph couplings are precisely the Kantorovich representation of deterministic Monge maps.
A last important class consists of semi-discrete problems, where $\al$ has a density and $\be$ is discrete. In this case couplings are singular measures supported on a union of graphs or cells inside $\Xx\times\Yy$.
\index{semi-discrete!OT}

\paragraph{Continuous Kantorovich problem.}
\index{Kantorovich!problem}

The discrete Kantorovich problem~\eqref{eq-kanto-discr} becomes, for arbitrary measures, the minimization of the average cost over all couplings,
\eql{\label{eq-mk-generic}
	\MK_c(\al,\be)
	\eqdef
	\inf_{\pi\in\Couplings(\al,\be)}
	\int_{\Xx\times\Yy} c(x,y)\d\pi(x,y).
}
This is an infinite-dimensional linear program over a space of measures.
\index{linear programming!finite-dimensional}

\begin{rem}[Probabilistic interpretation of Kantorovich's problem]
	The same problem can be written as
	\[
		\MK_c(\al,\be)
		=
		\inf_{X\sim\al,\,Y\sim\be}\EE(c(X,Y)).
	\]
	The minimization is not over the marginal laws, which are fixed, but over all possible dependences between the two random variables. OT therefore chooses the cheapest joint law among all couplings.
\index{random variable}
\index{joint!law}
\end{rem}

\begin{prop}[Existence on compact spaces]\label{prop-kantorovich-existence-compact}
\index{compact space}
	Assume that $\Xx$ and $\Yy$ are compact metric spaces and that $c\in\Cc(\Xx\times\Yy)$. Then the Kantorovich problem~\eqref{eq-mk-generic} admits at least one minimizer.
\index{Kantorovich!problem}
\end{prop}
\begin{proof}
	The constraint set is non-empty because it contains the product coupling $\al\otimes\be$. It is closed for weak convergence of measures because the marginal constraints are preserved under weak convergence. Since $\Xx\times\Yy$ is compact, the set of probability measures on it is compact for the weak topology, and therefore $\Couplings(\al,\be)$ is compact. Finally, the functional $\pi\mapsto\int c\d\pi$ is weakly continuous because $c$ is continuous and bounded. The minimum is thus attained.
\index{probability measure}
\index{marginal!constraint}
\index{product!coupling}
\index{weak!convergence}
\index{topology!weak}
\end{proof}

On non-compact domains, one needs coercivity and moment conditions. For the Wasserstein cost $c(x,y)=d(x,y)^p$ on a Polish metric space, the natural domain is
\[
	\mathcal P_p(\Xx)
	\eqdef
	\enscond{\mu\in\Mm_+^1(\Xx)}{\int d(x,x_0)^p\d\mu(x)<+\infty},
\]
for one, and hence every, reference point $x_0$. If $\al,\be\in\mathcal P_p(\Xx)$, then the product coupling has finite $p$-cost up to the triangle inequality, so the Kantorovich value is finite. Existence of minimizers holds under standard lower-semicontinuity assumptions on $c$, using tightness of finite-moment sublevel sets~\cite{Villani09,SantambrogioBook}.
\index{triangle inequality}
\index{tightness}
\index{product!coupling}

\paragraph{Monge--Kantorovich equivalence.}
\index{Monge-Kantorovich equivalence}

The proof of Brenier's theorem~\ref{thm-brenier} relies on Kantorovich relaxation and duality. It proves that, under its hypotheses, the relaxation is tight: it has the same cost as the Monge problem and its optimal coupling is induced by a map.
\index{optimal coupling}
\index{Brenier!theorem}
\index{Kantorovich!relaxation}
\index{Monge!problem}

\begin{cor}[Monge--Kantorovich equivalence under Brenier]\label{cor-monge-kantorovich-brenier}
\index{Monge-Kantorovich equivalence}
	Assume that $\al$ is absolutely continuous with respect to Lebesgue measure and that $c(x,y)=\norm{x-y}^2$. If $T$ is the Brenier map solving Monge's problem, then $\pi=(\Id,T)_\sharp\al$ is the unique optimal coupling solving the Kantorovich problem. In particular, Monge and Kantorovich costs are the same.
\index{Lebesgue measure}
\index{Kantorovich!problem}
\index{optimal coupling}
\index{Brenier!map}
\end{cor}
\begin{proof}
	The proof of Brenier's theorem shows that the support of any optimal Kantorovich plan is contained in the subdifferential $\partial\phi$ of a convex function $\phi$. When $\al$ has a density, $\phi$ is differentiable $\al$-almost everywhere, so $\partial\phi(x)=\{\nabla\phi(x)\}$ for $\al$-almost every $x$. Thus every optimal coupling is concentrated on the graph of $T=\nabla\phi$ and must equal $(\Id,T)_\sharp\al$. The graph coupling is feasible and optimal, and the two formulations have the same value.
\index{support}
\index{subdifferential}
\index{convex!function}
\index{Brenier!theorem}
\end{proof}

The density assumption is exactly what prevents the relaxed plan from using several destinations at a nonsmooth point.

\begin{rem}[Nonsmooth potentials and splitting]
	If $\al$ does not have a density, then $\phi$ may be non-smooth on a set charged by $\al$, and non-smooth points can lead to mass splitting. For instance, moving $\delta_0$ to $(\delta_{-1}+\delta_{+1})/2$ can be represented by a plan concentrated on the set-valued subdifferential of $\phi(x)=|x|$, but not by a deterministic map. This is the continuous counterpart of the gap between the uniform matching case of Corollary~\ref{cor-kantorovich-matching} and the general splitting case.
\index{matching!uniform}
\end{rem}

\begin{rem}[Probabilistic form of tightness]
\index{tightness}
	If $(X,Y)$ has the optimal Kantorovich law under the assumptions of Corollary~\ref{cor-monge-kantorovich-brenier}, then $Y=T(X)$ almost surely with $X\sim\al$ and $T(X)\sim\be$. This is analogous to the Birkhoff--von Neumann result in the fully discrete uniform case: in both settings, the convex relaxation admits an optimizer satisfying the original deterministic constraint. The hypotheses are quite different, however: Birkhoff--von Neumann is finite-dimensional and need not give uniqueness, whereas Brenier's theorem uses absolute continuity of the source and gives uniqueness of the optimal map almost everywhere.
\index{Birkhoff-von Neumann theorem}
\index{Brenier!theorem}
\index{absolute continuity}
\end{rem}

\section{\texorpdfstring{$c$}{c}-Cyclical Monotonicity}
\index{cyclic!monotonicity}
\label{sec-cyclical-monotonicity}

Cyclical monotonicity is the local geometric fingerprint of optimality. It converts a global minimization problem into finite exchange inequalities and is the bridge from Kantorovich plans to convex potentials.
\index{cyclic!monotonicity}
\index{convex!potential}

Optimal transport plans behave well when one looks at any finite sub-collection of points in their support: the restriction is still an optimal matching for those points alone. For finitely supported marginals this is immediate, and it leads to the notion of $c$-cyclical monotonicity.
\index{plan!transport}
\index{cyclic!monotonicity}
\index{c-cyclical monotonicity}

\paragraph{Support and $c$-cyclical monotonicity.}
\index{cyclic!monotonicity}
\index{c-cyclical monotonicity}

To formalize this, one needs a precise notion of support, i.e. the closed set that carries the mass of the coupling.

\begin{defn}[Support]\label{def-support}
	For a Radon measure $\pi$ on $\Xx\times\Yy$,
\index{Radon!measure}
	\[
		\supp(\pi)
		\eqdef
		\enscond{(x,y)}{\pi(U\times V)>0\text{ for every open }U\ni x,\ V\ni y}.
	\]
\end{defn}

\begin{defn}[$c$-cyclical monotonicity]\label{def:ccm}
\index{cyclic!monotonicity}
\index{c-cyclical monotonicity}
	A set $\Gamma\subset\Xx\times\Yy$ is $c$-cyclically monotone if, for every $k\geq2$, every finite family $(x_i,y_i)_{i=1}^k\subset\Gamma$ and every permutation $\sigma$ of $\{1,\ldots,k\}$,
	\[
		\sum_{i=1}^k c(x_i,y_i)
		\leq
		\sum_{i=1}^k c(x_i,y_{\sigma(i)}).
	\]
\end{defn}

Any permutation is a product of cycles, so it suffices to verify the inequality for cyclic permutations,
\[
	\sum_{i=1}^k c(x_i,y_i)
	\leq
	\sum_{i=1}^k c(x_i,y_{i+1}),
	\qquad y_{k+1}=y_1.
\]

\paragraph{Optimal matching to optimal transport.}
\index{matching!optimal}

Let the marginals be uniform on $n$ points, $\al=\frac1n\sum_{i=1}^n\delta_{x_i}$ and $\be=\frac1n\sum_{i=1}^n\delta_{y_i}$. By Corollary~\ref{cor-kantorovich-matching}, there exists an optimal plan induced by a permutation. Its support $\Gamma=\{(x_i,y_{\sigma(i)})\}_i$ is $c$-cyclically monotone: otherwise exchanging the finitely many targets along a violating cycle would lower the matching cost. The following theorem, in the spirit of Rockafellar's cyclic-monotonicity theorem~\cite{rockafellar2015convex}, says that the same finite-exchange condition holds for any optimal coupling, not only for finite uniform matchings.
\index{optimal coupling}
\index{optimal plan}
\index{cyclic!monotonicity}
\index{matching!uniform}

\begin{thm}[Optimal plans are $c$-cyclically monotone]\label{thm:opt_ccm}
\index{optimal plan}
	Assume $c$ is continuous. For any optimal plan $\pi$ solving the Kantorovich problem~\eqref{eq-mk-generic}, $\supp(\pi)$ is $c$-cyclically monotone.
\index{Kantorovich!problem}
\end{thm}
\begin{proof}
	We prove the contrapositive. Suppose that $\supp(\pi)$ is not $c$-cyclically monotone. Then there exist points $(x_i,y_i)_{i=1}^k$ in the support and a permutation $\sigma$ such that
	\[
		\sum_i c(x_i,y_i)>
		\sum_i c(x_i,y_{\sigma(i)}).
	\]
	By continuity of $c$, after shrinking neighborhoods $U_i\ni x_i$ and $V_i\ni y_i$, the same strict inequality holds uniformly for every choice of points in these neighborhoods:
	\[
		\sum_i c(u_i,v_i)
		>
		\sum_i c(u_i,\tilde v_{\sigma(i)})
		\qquad
		(u_i\in U_i,\ v_i\in V_i,\ \tilde v_{\sigma(i)}\in V_{\sigma(i)}).
	\]
	Choose the sets so that $\pi(U_i\times V_i)>0$. Because there are only finitely many rectangles, one can choose $\lambda>0$ small enough that the scaled restrictions
	\[
		\pi_i=\lambda\frac{\pi|_{U_i\times V_i}}{\pi(U_i\times V_i)}
	\]
	have common mass $\lambda$ and satisfy $\sum_i\pi_i\leq\pi$. Let $\alpha_i=(P_\Xx)_\sharp\pi_i$ and $\beta_i=(P_\Yy)_\sharp\pi_i$. Define
	\[
		\tilde\pi
		=
		\pi-\sum_i\pi_i
		+\sum_i \frac{\alpha_i\otimes\beta_{\sigma(i)}}{\lambda}.
	\]
	The removed and reinserted first marginals are both $\sum_i\alpha_i$, and the removed and reinserted second marginals are both $\sum_i\beta_i$ because $\sigma$ is a permutation. Hence $\tilde\pi\in\Couplings(\al,\be)$. Integrating the uniform strict inequality against the product probability $\otimes_i(\pi_i/\lambda)$ shows that the reinserted crossed terms have strictly smaller cost than the removed diagonal terms. This contradicts the optimality of $\pi$.
\end{proof}

\paragraph{Monotonicity.}

Assume the optimal plan is induced by a measurable map $T:\Xx\to\Yy$, i.e. $\pi=(\Id,T)_\sharp\al$. For any $k$ points $x_1,\ldots,x_k$ in the domain, cyclical monotonicity reads
\index{optimal plan}
\index{cyclic!monotonicity}
\[
	\sum_{i=1}^k c(x_i,T(x_i))
	\leq
	\sum_{i=1}^k c(x_i,T(x_{i+1})),
	\qquad x_{k+1}=x_1.
\]
For $c(x,y)=\frac12\norm{x-y}^2$, taking $k=2$ gives, for any $x,y$,
\[
	\dotp{T(x)-T(y)}{x-y}\geq0,
\]
so $T$ is a monotone vector field. Brenier's theorem adds that, when $\al$ is absolutely continuous, $T=\nabla\phi$ for a convex potential $\phi$. The converse fails in dimension $d\geq2$: as shown by the rotation example in the Monge section, a small rotation is monotone yet not a gradient.
\index{Brenier!theorem}
\index{convex!potential}

\paragraph{One dimension.}

In one space dimension, with cost $c(x,y)=|x-y|^p$ for any $p\geq1$, the two-point inequality becomes
\[
	|x-T(x)|^p+|y-T(y)|^p
	\leq
	|x-T(y)|^p+|y-T(x)|^p,
\]
which is equivalent to $T(x)\leq T(y)$ whenever $x<y$. Thus $T$ must be nondecreasing, recovering the classical monotone rearrangement.
\index{monotone!rearrangement}

\section{Metric Properties: Wasserstein Distances}
\index{Wasserstein!distance}

The final part of the section proves that OT costs are genuine distances when the ground cost comes from a metric. It also compares Wasserstein convergence with total variation and explains why OT is weak enough to move Dirac masses continuously.
\index{ground cost}
\index{Dirac mass}
\index{total variation}

\paragraph{OT defines a distance.}
\index{Wasserstein!distance}

An important feature of OT is that it defines a distance between histograms and probability measures as soon as the cost matrix satisfies certain suitable properties. Indeed, OT can be understood as a canonical way to lift a ground distance between points to a distance between histograms or measures.
\index{histogram}
\index{cost matrix}
\index{probability measure}
The proof of this result relies on a ``gluing lemma'', which we first prove in the discrete case.
\index{gluing lemma}

\begin{lem}[Discrete gluing lemma]\label{lem-gluing-discr}
	Given $(\a,\b,\VectMode{c}) \in \simplex_n \times \simplex_p \times \simplex_m$,
	let $\P \in \CouplingsD(\a,\b)$ and $\Q \in \CouplingsD(\b,\VectMode{c})$. Then there exists a 3-D tensor coupling $\S \in \RR_+^{n \times p \times m}$
	such that the 2-D marginals satisfy
	\eq{
		\sum_{k} \S_{i,j,k} = \P_{i,j}
		\qandq
		\sum_{i} \S_{i,j,k} = \Q_{j,k}.
	}
	Consequently the marginal between the first and third variables,
	\[
		\R_{i,k}\eqdef\sum_j \S_{i,j,k},
	\]
	belongs to $\CouplingsD(\a,\VectMode{c})$. For the canonical construction below, this glued coupling is the twisted matrix product
	\[
		\R=\P\diag(1/\b)\Q,
		\qquad
		\R_{i,k}=\sum_{j:\b_j>0}\frac{\P_{i,j}\Q_{j,k}}{\b_j}.
	\]
	In the matrix notation, $1/\b_j$ is understood as $0$ when $\b_j=0$.
	Figure~\ref{fig:kantorovich-discrete-gluing-lemma} displays this construction in matrix form.
\index{gluing lemma}
\end{lem}
\begin{proof}
	One verifies that
	\eql{\label{eq-glued-discr}
		\S_{i,j,k} = \choice{
			\frac{\P_{i,j} \Q_{j,k}}{\b_j}  \qifq \b_j \neq 0 \\
			0 \text{ otherwise}
		}
	}
	is acceptable. Indeed, if $\b_j \neq 0$
	\eq{
		\sum_{k} \S_{i,j,k} = \sum_{k} \frac{\P_{i,j} \Q_{j,k}}{\b_j}
		= \frac{\P_{i,j}}{\b_j} ( \Q \ones_m )_j = \frac{\P_{i,j}}{\b_j}  \b_j.
	}
	If $\b_j = 0$, then necessarily $\P_{i,j}=0$ and $\sum_{k} \S_{i,j,k} = 0 = \P_{i,j}$.
	The same computation gives the other prescribed marginal:
	\[
		\sum_i \S_{i,j,k}
		=
		\choice{
			\frac{\Q_{j,k}}{\b_j}\sum_i\P_{i,j}=\Q_{j,k} \qifq \b_j>0,\\
			0=\Q_{j,k} \qifq \b_j=0.
		}
	\]
	Summing over $j$ then gives the displayed formula for $\R$. Its row and column sums are
	\[
		\sum_k\R_{i,k}=\sum_j\P_{i,j}=\a_i,
		\qquad
		\sum_i\R_{i,k}=\sum_j\Q_{j,k}=\VectMode{c}_k,
	\]
	so $\R\in\CouplingsD(\a,\VectMode{c})$.
\end{proof}

\begin{figure}[H]
\centering
\begin{tabular}{@{}cccc@{}}
\small $P\in\CouplingsD(\a,\b)$ & \small $Q\in\CouplingsD(\b,\VectMode{c})$ & \small glued $R$ & \small direct OT \\[-.15em]
\includegraphics[width=.22\linewidth]{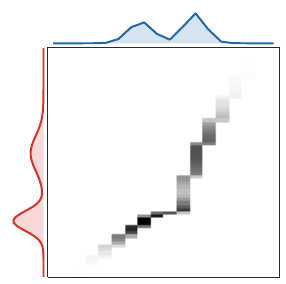} &
\index{gluing lemma}
\includegraphics[width=.22\linewidth]{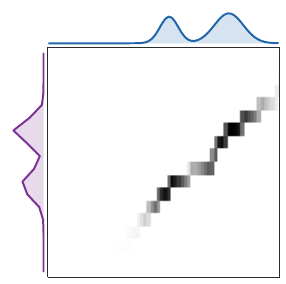} &
\includegraphics[width=.22\linewidth]{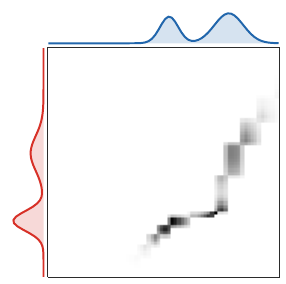} &
\includegraphics[width=.22\linewidth]{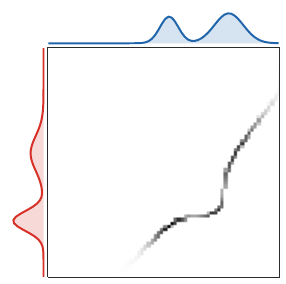}
\end{tabular}
\caption{Discrete gluing lemma in matrix form. The first two panels are optimal one-dimensional couplings through an intermediate marginal $\b$. The third panel shows the induced marginal $R=P\diag(1/\b)Q$ between $\a$ and $\VectMode{c}$; it is feasible and is the coupling used in the triangle-inequality proof. Because the intermediate marginal is represented on a coarser grid, the glued coupling is more mediated than the direct optimal coupling shown on the right. The thin box frames only the coupling matrix in each panel, while the attached marginal strips remain outside it.}
\index{gluing lemma}
\label{fig:kantorovich-discrete-gluing-lemma}
\end{figure}


When the cost matrix is the $p$th power of a distance matrix, the discrete Kantorovich value becomes a metric on histograms.

\begin{defn}[Discrete Wasserstein distance]\label{def-discrete-wasserstein-distance}
\index{Wasserstein!distance}
	Let $\distD\in\RR_+^{n\times n}$ be a distance matrix on $\range{n}$ and let $p\geq1$. The discrete $p$-Wasserstein distance between histograms $\a,\b\in\simplex_n$ is
	\eql{\label{eq-wass-p-disc}
		\WassD_p(\a,\b) \eqdef \MKD_{\distD^p}(\a,\b)^{1/p}.
	}
	It depends on the chosen ground distance $\distD$.
\end{defn}

\begin{prop}[Metric property of discrete Wasserstein distance]\label{prop-metric-histo}
\index{Wasserstein!distance}
	For every distance matrix $\distD$ on $\range{n}$, Definition~\ref{def-discrete-wasserstein-distance} defines a distance on $\simplex_n$: $\WassD_p$ is symmetric, positive, $\WassD_p(\a,\b)=0$ if and only if $\a = \b$, and it satisfies the triangle inequality
\index{triangle inequality}
\index{Wasserstein!distance}
\eq{
	\foralls \a,\b,\VectMode{c} \in \simplex_n, \quad \WassD_p(\a,\VectMode{c}) \leq \WassD_p(\a,\b) + \WassD_p(\b,\VectMode{c}).
}
\end{prop}

\begin{proof}
For symmetry, since $\distD^p$ is symmetric, we use the fact that if $\P \in \CouplingsD(\a,\b)$ is optimal for $\WassD_p(\a,\b)$, then $\P^\top \in \CouplingsD(\b,\a)$ is optimal for $\WassD_p(\b,\a)$. For definiteness, since $\C = \distD^p$ has a null diagonal, $\WassD_p(\a,\b)=0$ is achieved by the diagonal coupling $\P^\star=\diag(\a)=\diag(\b)$ when $\a=\b$; by positivity of all off-diagonal elements of $\distD^p$, $\WassD_p(\a,\b)>0$ whenever $\a\ne \b$ because any admissible coupling then has a nonzero element outside the diagonal.

To prove the triangle inequality in this discrete setting, we consider $\a,\b,\VectMode{c} \in\simplex_n$, and let $\P$ and $\Q$ be two optimal solutions of the transport problems between $\a$ and $\b$, and $\b$ and $\VectMode{c}$ respectively.
\index{triangle inequality}
We use the gluing Lemma~\ref{lem-gluing-discr} which defines $\S \in \RR_+^{n^3}$ with marginals $\sum_{k}\S_{\cdot,\cdot,k}=\P$ and $\sum_i \S_{i,\cdot,\cdot}=\Q$. We define $\R=\sum_{j}\S_{\cdot,j,\cdot}$, which is an element of $\CouplingsD(\a,\VectMode{c})$.
\index{gluing lemma}
\[
\begin{tikzpicture}
  \node (a) at (0,0) {$a$};
  \node (b) at (2,0) {$b$};
  \node (c) at (4,0) {$c$};

  \draw[->] (a) -- (b) node[midway, above] {$P$};
  \draw[->] (b) -- (c) node[midway, above] {$Q$};
  \draw[->, bend right=45] (a) to node[midway, below] {$R$} (c);
\end{tikzpicture}
\]
Note that if one assumes $\b>0$ then $\R = \P \diag(1/\b) \Q$.

The triangle inequality follows from
\index{triangle inequality}
$$\begin{aligned}
\WassD_p(\a,\VectMode{c})&=\Big(\min_{\tilde\R\in \CouplingsD(\a,\VectMode{c})}\dotp{\tilde\R}{\distD^p}\Big)^{1/p} \leq \dotp{\R}{\distD^p}^{1/p}\\
&= \Big(\sum_{i,k}  \distD^p_{ik}\sum_{j} \S_{i,j,k} \Big)^{1/p}
 \leq \Big(\sum_{i,j,k} \left(\distD_{ij}+\distD_{j,k}\Big)^p  \S_{i,j,k} \right)^{1/p} \\
& \leq \Big(\sum_{i,j,k} \distD^p_{ij} \S_{i,j,k} \Big)^{1/p} + \Big(\sum_{i,j,k}\distD^p_{j,k} \S_{i,j,k} \Big)^{1/p} \\
&= \Big(\sum_{i,j} \distD^p_{i,j}  \sum_k \S_{i,j,k} \Big)^{1/p} + \Big(\sum_{j,k} \distD^p_{j,k}  \sum_i \S_{i,j,k} \Big)^{1/p}\\
&= \Big(\sum_{i,j} \distD^p_{i,j}\P_{i,j}\Big)^{1/p} + \Big(\sum_{j,k} \distD^p_{j,k} \Q_{j,k}\Big)^{1/p}
= \WassD_p(\a,\b) +\WassD_p(\b,\VectMode{c}).
\end{aligned}
$$
The first inequality follows from the feasibility of $\R$, the second is the usual triangle inequality for elements in $\distD$, and the third comes from Minkowski's inequality.
\index{Minkowski inequality}
\end{proof}

\paragraph{Continuous gluing.}
\index{gluing lemma}

Proposition~\ref{prop-metric-histo} generalizes from histogram to arbitrary measures that need not be discrete. For this, one needs the following general gluing lemma.
\index{histogram}
\index{gluing lemma}

\begin{lem}[Gluing lemma]\label{lem-gluing-general}
	Let $(\al,\be,\ga) \in \Mm_+^1(\Xx) \times \Mm_+^1(\Yy) \times \Mm_+^1(\Zz)$
	where $(\Xx,\Yy,\Zz)$ are Polish spaces in the sense of Definition~\ref{def-polish-metric-space}.
\index{Polish space}
	Given $\pi \in \Couplings(\al,\be)$ and $\xi \in \Couplings(\be,\ga)$, then there exists a tensor coupling measure
\index{coupling measure}
	$\sigma \in \Mm_+(\Xx \times \Yy \times \Zz)$ such that
	\eq{
		(P_{\Xx,\Yy})_\sharp \sigma = \pi
		\qandq
		(P_{\Yy,\Zz})_\sharp \sigma = \xi
	}
	where we denoted the projector $P_{\Xx,\Yy}(x,y,z)=(x,y)$ and $P_{\Yy,\Zz}(x,y,z)=(y,z)$.
\end{lem}
\begin{proof}
	The proof of this fundamental result is involved since it requires using the disintegration of measure (which corresponds to conditional probabilities).
\index{conditional probability}
\index{disintegration}
	The disintegration of measures is applicable because the spaces are Polish.
	We disintegrate $\pi$ and $\xi$ against $\be$ to obtain two families $(\pi_y)_{y \in \Yy}$ and $(\xi_y)_{y \in \Yy}$ of probability distributions on $\Xx$ and $\Zz$. These families are defined by the fact that
\index{disintegration}
	\eq{
		\foralls h \in \Cc(\Xx \times \Yy), \quad
		\int_\Yy \Big( \int_\Xx h(x,y) \d \pi_y(x) \Big) \d \be(y) = \int h(x,y) \d\pi(x,y).
	}
	and similarly for $\xi$.
	When $\be=\sum_j \b_j \de_{y_j}$ and $\pi=\sum_{i,j} \P_{i,j} \de_{(x_i,y_j)}$, then this conditional distribution is defined on the support of $\be$ as
	$\pi_{y_j} = \sum_i \frac{\P_{i,j}}{\b_j} \de_{x_i}$ (and similarly for $\xi$).
		The glued measure is then defined by the conditional-product formula
	\eq{
		\foralls g \in \Cc(\Xx \times \Yy \times \Zz), \quad
		\int g(x,y,z) \d\sigma(x,y,z) = \int g(x,y,z) \d \pi_y(x) \d \xi_y(z) \d \be(y).
	}
	For discrete measures, this matches the definition~\eqref{eq-glued-discr}, since $\sigma=\sum_{i,j,k} \S_{i,j,k} \de_{x_i,y_j,z_k}$ where
	\eq{
		\S_{i,j,k} = \frac{\P_{i,j}}{\b_j} \frac{\Q_{j,k}}{\b_j} \b_j.
	}
\end{proof}

Using this gluing lemma, we can now construct the Wasserstein distance in the general setting of arbitrary distributions on a Polish space.
\index{Wasserstein!distance}
\index{gluing lemma}
\index{Polish space}

\begin{defn}[Wasserstein distance]\label{def-wasserstein-distance}
\index{Wasserstein!distance}
	Let $(\X,\dist)$ be a metric space and $p\geq1$. For $\al,\be\in\Pp_p(\X)$, the $p$-Wasserstein distance is
	\eql{\label{eq-defn-wass-dist}
		\Wass_p(\al,\be) \eqdef \MK_{\dist^p}(\al,\be)^{1/p}
		=
		\left(\inf_{\pi\in\Couplings(\al,\be)}
		\int_{\X\times\X}\dist(x,y)^p\d\pi(x,y)\right)^{1/p}.
	}
	It depends on the ground distance $\dist$.
\end{defn}

\begin{prop}[Metric property of the Wasserstein distance]\label{prop-metric-measure}
\index{Wasserstein!distance}
	Definition~\ref{def-wasserstein-distance} defines a distance: $\Wass_p$ is symmetric, positive, $\Wass_p(\al,\be)=0$ if and only if $\al = \be$, and it satisfies the triangle inequality
\index{triangle inequality}
\index{Wasserstein!distance}
\eq{
	\foralls (\al,\be,\ga) \in  \Pp_p(\X)^3, \quad \Wass_p(\al,\ga) \leq \Wass_p(\al,\be) + \Wass_p(\be,\ga).
}
\end{prop}

\begin{proof}
	The symmetry follows from the fact that since $\dist$ is symmetric, if $\pi(x,y)$ is optimal for $\MK_{\dist^p}(\al,\be)$, then
	$\pi(y,x) \in \Couplings(\be,\al)$ is optimal for $\MK_{\dist^p}(\be,\al)$.
	If $\MK_{\dist^p}(\al,\be)=0$, then necessarily an optimal coupling $\pi$ is supported on the diagonal $\De \eqdef \{(x,x)\}_x \subset \Xx^2$.
\index{optimal coupling}
	We denote $\la(x)$ the corresponding measure on the diagonal, i.e. such that $\int h(x,y) \d\pi(x,y) = \int h(x,x) \d \la(x)$.
	Then since $\pi \in \Couplings(\al,\be)$ necessarily $\la=\al$ and $\la=\be$ so that $\al=\be$.

	For the triangle inequality, we consider optimal couplings $\pi \in \Couplings(\al,\be)$ and $\xi \in \Couplings(\be,\ga)$
\index{triangle inequality}
\index{optimal coupling}
	and we glue them according to the Lemma~\ref{lem-gluing-general}.
		We define the composition of the two couplings $(\pi,\xi)$ as $\rho \eqdef (P_{\Xx,\Zz})_\sharp \si$.
	Note that if $\pi$ and $\xi$ are couplings induced by two Monge maps $T_\Xx(x)$ and $T_\Yy(y)$, then $\rho$ is itself induced by the Monge map $T_\Yy \circ T_\Xx$, so that this notion of composition of coupling generalizes the composition of maps.
\index{Monge!problem}
	The triangular inequality follows from
	$$\begin{aligned}
		\Wass_p(\al,\ga) & \leq \Big( \int_{\Xx \times \Zz} \dist(x,z)^p \d\rho(x,z)\Big)^{1/p}
		= \Big( \int_{\Xx \times \Yy \times \Zz} \dist(x,z)^p \d\si(x,y,z)\Big)^{1/p}		\\
		 &\leq \Big( \int_{\Xx \times \Yy \times \Zz} (\dist(x,y)+\dist(y,z))^p \d\si(x,y,z)\Big)^{1/p} \\
		 &\leq \Big( \int_{\Xx \times \Yy \times \Zz} \dist(x,y)^p \d\si(x,y,z)\Big)^{1/p}
		    +  \Big( \int_{\Xx \times \Yy \times \Zz} \dist(y,z)^p \d\si(x,y,z)\Big)^{1/p} \\
	  &= \Big( \int_{\Xx \times \Yy} \dist(x,y)^p \d\pi(x,y)\Big)^{1/p}
		    +  \Big( \int_{\Yy \times \Zz} \dist(y,z)^p \d\xi(y,z)\Big)^{1/p}
		    = \Wass_p(\al,\be)  + \Wass_p(\be,\ga) .
	\end{aligned}$$
\end{proof}

\paragraph{Interpolation induced by an optimal plan.}
\index{plan!interpolation}
\label{sec-kantorovich-plan-interpolation}

The quadratic Wasserstein distance does not only compare two endpoint measures. An optimal plan also says how to move mass between them: each active pair $(x,y)$ travels along the segment joining $x$ to $y$. This turns an optimal coupling into a curve of measures.
\index{optimal coupling}
\index{Wasserstein!distance}

\begin{defn}[$\Wass_2$ geodesic induced by an optimal plan]\label{def-w2-geodesic-induced-by-plan}
\index{Wasserstein!geodesic}
\index{plan!optimal}
\index{McCann interpolation}
	Let $\al_0,\al_1\in\Pp_2(\RR^d)$, and let $\pi^\star\in\Couplings(\al_0,\al_1)$ be optimal for $\Wass_2^2(\al_0,\al_1)$. For $t\in[0,1]$, define
	\[
		e_t(x,y)\eqdef(1-t)x+t y,
		\qquad
		\al_t\eqdef(e_t)_\sharp\pi^\star .
	\]
	The curve $(\al_t)_{t\in[0,1]}$ is the displacement, or McCann, $\Wass_2$ geodesic induced by $\pi^\star$.
\end{defn}
In the discrete case, each mass $\P_{ij}$ moves from $x_i$ to $y_j$ along its own segment. When the optimal plan is not induced by a map, one source atom can split into several moving atoms. If the optimal plan is not unique, different optimal plans may also induce different $\Wass_2$ geodesics.
\index{discrete!measure}
\index{plan!transport}

\begin{alg}[Displacement interpolation from a transport plan]\label{alg:plan-displacement-interpolation}
\index{plan!interpolation}
\index{McCann interpolation}
\textbf{Input:} Measures $\alpha,\beta$ on $\RR^d$, time $t\in[0,1]$.

\textbf{Output:} Displacement interpolant $\alpha_t$.

\textbf{Let} $\pi^\star$ be any minimizer of the quadratic Kantorovich problem.

\textbf{Set} interpolation map:
\(e_t(x,y)=(1-t)x+t y.\)

\textbf{Push forward:}
\(\al_t=(e_t)_\sharp\pi^\star.\)

\textbf{If} $\pi^\star=\sum_{i,j}P^\star_{ij}\delta_{(x_i,y_j)}$ \textbf{then}:
\begin{algblock}

\textbf{Compute}
\(\al_t= \sum_{i,j}P^\star_{ij} \delta_{(1-t)x_i+t y_j}.\)
\end{algblock}
\algreturnskip
\textbf{Return} $\alpha_t$.
\end{alg}

\begin{prop}[Optimal-plan interpolation is a $\Wass_2$ geodesic]\label{prop-plan-interpolation-w2-geodesic}
\index{Wasserstein!geodesic}
\index{constant-speed geodesic}
	Let $(\al_t)_{t\in[0,1]}$ be defined by Definition~\ref{def-w2-geodesic-induced-by-plan}. Then, for every $0\leq s\leq t\leq1$,
	\[
		\Wass_2(\al_s,\al_t)
		=
		(t-s)\Wass_2(\al_0,\al_1).
	\]
	Thus $t\mapsto\al_t$ is a constant-speed geodesic for the metric $\Wass_2$.
\end{prop}
\begin{proof}
	Push the optimal plan $\pi^\star$ forward by $(e_s,e_t)$. This gives a coupling $\gamma_{s,t}\in\Couplings(\al_s,\al_t)$, and
	\[
		\int \norm{z-z'}^2\d\gamma_{s,t}(z,z')
		=
		\int \norm{e_t(x,y)-e_s(x,y)}^2\d\pi^\star(x,y)
		=
		(t-s)^2\Wass_2^2(\al_0,\al_1).
	\]
	Hence $\Wass_2(\al_s,\al_t)\leq(t-s)\Wass_2(\al_0,\al_1)$. Applying this upper bound to the three pairs $(0,s)$, $(s,t)$ and $(t,1)$, and using the triangle inequality of Proposition~\ref{prop-metric-measure}, gives
	\[
		\Wass_2(\al_0,\al_1)
		\leq
		\Wass_2(\al_0,\al_s)+\Wass_2(\al_s,\al_t)+\Wass_2(\al_t,\al_1)
		\leq
		\Wass_2(\al_0,\al_1).
	\]
	All inequalities are therefore equalities, in particular the middle segment has the claimed length.
\end{proof}

\begin{figure}[H]
\centering
\begin{tabular}{@{}ccccc@{}}
\small $t=0$ & \small $t=1/4$ & \small $t=1/2$ & \small $t=3/4$ & \small $t=1$ \\[-.15em]
\includegraphics[width=.17\linewidth]{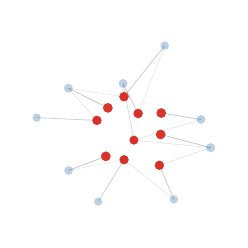} &
\includegraphics[width=.17\linewidth]{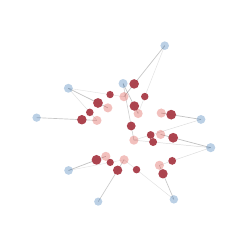} &
\includegraphics[width=.17\linewidth]{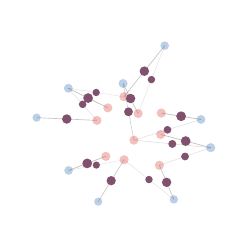} &
\includegraphics[width=.17\linewidth]{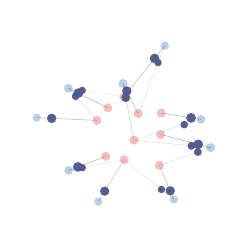} &
\includegraphics[width=.17\linewidth]{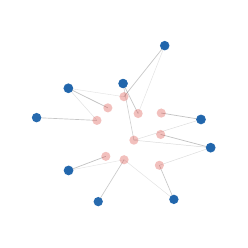}
\end{tabular}
\caption{McCann interpolation induced by a non-deterministic optimal transport plan. In every panel, the red and blue endpoint measures are shown with low opacity, thin gray segments display the support $\P_{ij}>\mathrm{tol}$ of the coupling, and the moving atoms are colored from red to blue along the interpolation.}
\index{McCann interpolation}
\index{plan!transport}
\label{fig:kantorovich-plan-interpolation}
\end{figure}

\paragraph{General geodesic spaces.}
\index{geodesic!space}
For Dirac masses in Euclidean space, the $\Wass_2$ geodesic from $\de_x$ to $\de_y$ is $t\mapsto\de_{(1-t)x+t y}$. The same idea extends to any geodesic metric space $(\X,\dist)$, meaning that each pair of points can be joined by a constant-speed metric geodesic. For each pair $(x,y)$, one replaces the Euclidean segment by a curve $\gamma^{x,y}:[0,1]\to\X$ such that $\gamma^{x,y}_0=x$, $\gamma^{x,y}_1=y$, and
\[
	\dist(\gamma^{x,y}_s,\gamma^{x,y}_t)=|t-s|\dist(x,y).
\]
If this geodesic is unique and depends measurably on $(x,y)$, one defines $e_t(x,y)=\gamma^{x,y}_t$ and sets $\al_t=(e_t)_\sharp\pi^\star$ for an optimal coupling $\pi^\star$. When geodesics are not unique, there is no canonical interpolation of a pair of Diracs unless a choice is made: one may select a particular geodesic between $x$ and $y$, or randomize among several such geodesics. The intrinsic formulation is to choose a probability measure $\eta$ on the path space of constant-speed geodesics, called a dynamical optimal plan, such that $(e_0,e_1)_\sharp\eta$ is an optimal coupling, and to set $\al_t=(e_t)_\sharp\eta$. Different measurable choices, or different conditional distributions over geodesics with the same endpoints, can give different $\Wass_2$ geodesics; the constant-speed identity remains the same. This path-space viewpoint is standard in the general theory of Wasserstein spaces~\cite{ambrosio2006gradient,Villani09,SantambrogioBook}.
\index{Dirac mass}
\index{path space}
\index{dynamical optimal plan}

\paragraph{Comparison with Monge.}
\index{Monge!problem}

	This distance $\Wass_p$ defined through the Kantorovich problem~\eqref{eq-defn-wass-dist} should be contrasted with the directed distance $\tilde\Wass$ obtained using Monge's problem~\eqref{eq-monge-distance}. The Kantorovich feasible set is never empty, since it contains the product coupling, although the $p$-cost may still be infinite without moment assumptions on non-compact spaces. By contrast, Monge's constraint set $\enscond{T}{T_\sharp \al=\be}$ can be empty. When an optimal Monge map exists, Kantorovich gives the same value by choosing the graph coupling $(\Id,T)_\sharp\alpha$; in this sense the Kantorovich problem is the convex relaxation of Monge's problem, with much better stability properties.
\index{Kantorovich!problem}
\index{product!coupling}
\index{Monge!distance}
\index{Monge!problem}

\section{Metric Properties: Topology and Applications}
\index{Wasserstein!distance}

This section shifts from metric axioms to topology and uses. Wasserstein distances metrize weak convergence under moment control, sit between weak and strong topologies, and provide quantitative estimates in probability and robust optimization.
\index{Wasserstein!distance}
\index{weak!convergence}

\paragraph{Convergence in law topology.}
\index{convergence!law}

On a bounded metric space, all $\Wass_p$ distances define the same topology, although they are not equivalent as distances.

\begin{prop}[Equivalence of Wasserstein distances on compact spaces]\label{prop-comp-wass-p}
\index{Wasserstein!distance}
	One has for $p \leq q$
	\eq{
		\Wass_p( \al,\be ) \leq \Wass_q(\al,\be) \leq \text{\upshape diam}(\Xx)^{\frac{q-p}{q}} \Wass_p(\al,\be)^{\frac{p}{q}}
	}
	where $\text{\upshape diam}(\Xx) \eqdef \usup{x,y} d(x,y)$.
\end{prop}
\begin{proof}
		The left inequality follows from Jensen inequality, $\phi(\int c(x,y) \d\pi(x,y)) \leq \int \phi(c(x,y)) \d\pi(x,y)$, applied to any probability distribution $\pi$ and to the convex function $\phi(r)=r^{q/p}$ with $c(x,y)=d(x,y)^p$, so that one gets
\index{Jensen inequality}
\index{convex!function}
		\eq{
			\pa{\int d(x,y)^{p} \d\pi(x,y)}^{\frac{q}{p}} \leq \int d(x,y)^{q} \d\pi(x,y).
		}
		The right inequality follows from
		\eq{
			d(x,y)^q \leq \text{diam}(\Xx)^{q-p} d(x,y)^p.
		}
	\end{proof}

The Wasserstein distance $\Wass_p$ is a weak distance: it compares singular distributions, such as discrete measures, and quantifies spatial shifts between supports. Its topology is studied in detail in~\cite{Villani09,SantambrogioBook,gigli2011user}, while empirical rates are quantified in~\cite{dudley1969speed,fournier2015rate,weed2017sharp,boissard2015distribution,bolley2007quantitative}.
\index{Wasserstein!distance}

\Needspace{5\baselineskip}
\begin{defn}[Weak$^*$ topology]\label{dfn-weak-conv}
\index{topology!weak}
	$(\al_k)_k$ converges weakly$^*$ to $\al$ in $\Mm_+^1(\Xx)$ (denoted $\al_k \rightharpoonup \al$) if and only if for any bounded continuous function $f \in \Cc_b(\Xx)$, $\int_\Xx f \d\al_k \rightarrow \int_\Xx f \d\al$. On compact spaces, $\Cc_b(\Xx)=\Cc(\Xx)$, which is why the boundedness is often left implicit there.
\index{continuous!bounded function}
\end{defn}

\begin{rem}[A Riemann-sum weak limit]\label{rem-riemann-weak-limit}
\index{Riemann sum}
\index{weak!limit}
	On $\Xx=\RR$, the empirical measures on a regular grid satisfy
\index{empirical!measure}
	\[
		\frac{1}{n} \sum_{k=1}^n \de_{k/n} \rightharpoonup \Uu_{[0,1]}.
	\]
	Indeed, for every continuous bounded function $f$,
	\[
		\frac{1}{n} \sum_{k=1}^n f(k/n) \longrightarrow \int_0^1 f(x) \d x,
	\]
	which is precisely the convergence of Riemann sums. This convergence is weak but not strong: for every $n$, the discrete measure and the uniform density are mutually singular, hence their total variation distance is equal to $2$.
\index{Riemann sum}
\index{total variation}
\end{rem}

\begin{rem}[Weak convergence for discrete measures]\label{rem-weak-conv-disc}
\index{weak!convergence}
\index{discrete!measure}
	In the special case of a single Dirac, $\de_{x^{(n)}} \rightharpoonup \de_x$ is equivalent to $\int f \d\de_{x^{(n)}} = f(x^{(n)}) \rightarrow \int f \d\de_{x} = f(x)$ for any continuous $f$. This in turn is equivalent to $x^{(n)} \rightarrow x$.
	For a fixed number of atoms, if $\al_n=\sum_{i=1}^N a_i^{(n)}\de_{x_i^{(n)}}$ and, after extracting a subsequence and relabeling, $a_i^{(n)}\to a_i$ and $x_i^{(n)}\to x_i$, then $\al_n$ converges weakly to $\sum_i a_i\de_{x_i}$, with atoms at identical limits merged. Without a uniform bound on the number of atoms, weak limits of discrete measures can be non-discrete; empirical measures are the standard example.
\index{weak!limit}
\index{empirical!measure}
\end{rem}

In terms of random vectors, if $X_n \sim \al_n$ and $X \sim \al$ (not necessarily defined on the same probability space), weak convergence corresponds to convergence in law of $X_n$ toward $X$.
\index{convergence!in law}
\index{weak!convergence}

\begin{rem}[Modes of convergence for random variables]\label{rem-random-variable-convergences}
\index{random variable}
\index{convergence!mode}
	Convergence of laws should be distinguished from stronger notions of convergence for random variables. If $X_n$ and $X$ are defined on a common probability space, then $X_n\to X$ almost surely means pointwise convergence outside a null set, while convergence in probability means
\index{convergence!in probability}
	\[
		\foralls \epsilon>0,\qquad
		\PP(\norm{X_n-X}>\epsilon)\to0.
	\]
	Almost-sure convergence implies convergence in probability, and convergence in probability implies convergence in law. Convergence in law is exactly weak$^*$ convergence of the probability measures $(X_n)_\sharp\PP\rightharpoonup X_\sharp\PP$, and does not require all variables to live on the same probability space. Strong convergence of measures, for instance convergence in total variation, is different and usually much stronger: it controls the mass assigned to all measurable sets, not only averages against continuous test functions. In particular, total variation convergence implies weak convergence, but the converse fails for empirical approximations of continuous laws.
\index{convergence!in law}
\index{convergence!in probability}
\index{probability measure}
\index{weak!convergence}
\index{total variation}
\end{rem}

\begin{rem}[Central limit theorem]\label{rem-clt}
\index{central!limit theorem}
		The central limit theorem states that if $(X_1,\ldots,X_n)$ are i.i.d. random vectors with finite second moments, $\EE(X_i)=0$, and $\EE(X_i X_i^\top)=\Id$, then the rescaled average $Z_n \eqdef \frac{1}{\sqrt{n}} \sum_{i=1}^n X_i$ converges in law toward a Gaussian $\Gaussian(0,\Id)$. This means that the measure $\al_n$ representing the law of $Z_n$ converges weakly toward the measure $\al$ of the centered normalized Gaussian.
\end{rem}

The total variation norm was introduced in Definition~\ref{defn-total-variation} and Proposition~\ref{prop-tv-dual-measure}. Its induced topology is often called the ``strong'' topology on measures. In the present section we only use the recall that, for a signed difference $\al-\be$,
\index{total variation}
\[
	\norm{\al-\be}_{\TV}=|\al-\be|(\Xx),
\]
so densities give an $L^1$ norm and discrete signed measures give an $\ell^1$ norm of the signed weights.
\index{signed!measure}

The following proposition shows that the TV norm can be seen as a Wasserstein distance, but for a ``degenerate'' 0/1 metric.
\index{Wasserstein!distance}

\begin{prop}[Total variation as Wasserstein for the discrete metric]\label{prop-rel-wass-tv}
\index{total variation}
	Denoting $d$ the 0/1 distance such that $d(x,x)=0$ and $d(x,y)=1$ if $x \neq y$, then
	\eq{
		\Wass_p(\al,\be)^p = \frac{1}{2}\norm{\al-\be}_{\TV}.
	}
\end{prop}
\begin{proof}
	For the sake of simplicity, we do the proof for discrete measures with weights $(\a,\b)$ and without loss of generality assume they have the same support $(x_i)_i$ and we denote $\D \eqdef (d(x_i,x_j))_{i,j}$ which is 0 on the diagonal and one outside.
	Also since $d^p=d$ we consider $p=1$.
	We denote $\text{c}_i = \min(\a_i,\b_i)$.
		By conservation of mass, for every $\P \in \CouplingsD(\a,\b)$, $\P_{i,i} \leq \text{c}_i$, thus
		\eq{
			\dotp{\P}{\D} = \sum_{i \neq j} \P_{i,j}
			= 1 - \sum_i \P_{i,i}
			\geq 1-\sum_i \text{c}_i.
		}
	We need to show that this bound is tight, namely to construct $\hat P \in \CouplingsD(\a,\b)$ such that
	$\diag(\hat P)=\text{c}$. Let
	\eq{
		\bar\a \eqdef \a-\text{c} = (\a-\b)_+ \geq 0
		\qandq
		\bar\b \eqdef \b-\text{c} = (\b-\a)_+ \geq 0
	}
	If $\bar\a=\bar\b=0$, then $\a=\b$ and the diagonal coupling is optimal. Otherwise, one has
		\eq{
			\frac{ \bar\a \otimes \bar\b }{ \dotp{\bar\a}{\ones} } \in \CouplingsD(\bar\a,\bar\b)
		}
		and we remark that $\dotp{\bar\a}{\ones} = \dotp{\bar\b}{\ones} = 1-\dotp{\text{c}}{\ones}$.
		Thus denoting
		\eq{
			\hat \P \eqdef \diag(\text{c}) + \frac{ \bar\a \otimes \bar\b }{ \dotp{\bar\a}{\ones} } \geq 0
		}
	satisfies
	\eq{
			\hat\P \ones = \text{c} + \bar\a = \a
			\qandq
			\hat\P^\top \ones = \text{c} + \bar\b = \b
		}
		so that $\hat P \in \CouplingsD(\a,\b)$  is a coupling so that $\diag(\hat P) = \diag(\text{c})$ since $\diag( \bar\a \otimes \bar\b )=0$. We thus conclude that
		\eq{
			\WassD_1(\a,\b) = \dotp{\D}{\hat \P}
			= \sum_{i,j} \frac{\bar\a_i\bar\b_j}{\dotp{\bar\a}{\ones}}
		= \sum_i \bar\a_i = \sum_i \bar\b_i
		= \frac{1}{2} \sum_i (\bar\a_i + \bar\b_i)
		= \frac{1}{2} \norm{\a-\b}_{\TV}.
	}
\end{proof}

As explained in Remark~\ref{rem-weak-conv-disc}, in the special case of Diracs, $\de_{x_n} \rightharpoonup \de_x$  is equivalent to $x_n \rightarrow x$. One can then contrast the strong topology with the Wasserstein distance if $x_n \neq x$,
\index{topology!strong}
\index{Wasserstein!distance}
\eq{
	\norm{\de_{x_n}-\de_x}_{\TV}=2
	\qandq
	\Wass_p(\de_{x_n},\de_x) = d(x_n,x).
}
This shows that for the strong topology, Diracs never converge, while they do converge for the Wasserstein distance. It is a powerful property of the Wasserstein distance: on compact spaces, it metrizes weak convergence.
\index{Wasserstein!distance}
\index{weak!convergence}

\begin{prop}[Wasserstein metrizes weak convergence on compact spaces]\label{prop-wass-metrizes-weak-compact}
\index{Wasserstein topology}
	If $\Xx$ is compact, $\al_k \rightharpoonup \al$ if and only if $\Wass_p(\al_k,\al) \rightarrow 0$.
\end{prop}
\begin{proof}
	For $p=1$, this is the Kantorovich--Rubinstein metrization theorem: by duality, $\Wass_1$ is the supremum over $1$-Lipschitz test functions, and on a compact metric space this class is compact modulo constants by Arzel\`a--Ascoli. Thus convergence in $\Wass_1$ is equivalent to weak convergence. Proposition~\ref{prop-comp-wass-p} then shows that all Wasserstein distances $\Wass_p$ induce the same convergent sequences on compact spaces. Hence weak convergence is equivalent to convergence in $\Wass_p$ for every $p\geq1$.
\index{Wasserstein!distance}
\index{weak!convergence}
\end{proof}

On non-compact spaces, one needs also to impose convergence of the $p$-th moments. More precisely, on a Polish metric space, $\Wass_p(\alpha_k,\alpha)\to0$ if and only if $\alpha_k\rightharpoonup\alpha$ and, for some reference point $x_0$,
\[
	\int d(x,x_0)^p\d\alpha_k(x)\longrightarrow
	\int d(x,x_0)^p\d\alpha(x).
\]

On a discrete space, the strong and weak topologies coincide, and the following proposition relates the TV and Wasserstein distances.
\index{Wasserstein!distance}
\index{topology!weak}

\begin{prop}[Comparison with total variation on discrete spaces]
\index{total variation}
\index{discrete!space}
	One has
	\eq{
		\frac{d_{\min}}{2}\norm{\al-\be}_{\TV} \leq \Wass_1(\al,\be) \leq \frac{d_{\max}}{2}  \norm{\al-\be}_{\TV}
		\qwhereq
		\choice{
			d_{\min} \eqdef \uinf{x \neq y} d(x,y) \\
			d_{\max} \eqdef \usup{x,y} d(x,y)
		}
	}
\end{prop}

\begin{proof}
	We denote $d_0(x,y)$ the distance such that $d_0(x,x)=0$ and $d_0(x,y)=1$ for $x \neq y$. One has
	\eq{
		d_{\min} d_0(x,y) \leq d(x,y) \leq d_{\max} d_0(x,y)
	}
	so that integrating this against any $\pi \in \Couplings(\al,\be)$ and taking the minimum among those $\pi$ gives the result using Proposition~\ref{prop-rel-wass-tv}.
\end{proof}

This bound is sharp, as this can be observed by taking $\al=\de_{x}$ and $\be=\de_y$, in which case the bound simply reads if $x \neq y$
\eq{
	d_{\min} \leq d(x,y) \leq d_{\max}.
}
This shows that the ratio between the two distances can blow up as $d_{\max}/d_{\min}$ increases. On non-discrete spaces, if $d_{\min}=0$, then the two distances are not equivalent, in line with the fact that the strong and weak topologies do not coincide.
\index{topology!weak}

\section{Wasserstein over Wasserstein}
\index{Wasserstein!over Wasserstein}
\label{sec-wasserstein-over-wasserstein}

The construction can be iterated. Once $(\X,\dist)$ is a metric space, the set of probability measures on $\X$ becomes a metric space through $\Wass_p$. It can therefore serve as a new ground space. This is useful whenever the objects to compare are themselves random probability measures, or mixtures whose components are meaningful objects rather than only a collapsed density.
\index{probability measure}
\index{random probability measure}

The standard setting is that of Polish spaces, introduced in Definition~\ref{def-polish-metric-space}. These assumptions rule out many measure-theoretic pathologies: probability laws can be approximated by countable objects, tightness gives compactness criteria, regular conditional probabilities exist on the associated Borel spaces, and weak convergence is stable. The next proposition records that Wasserstein spaces preserve this structure, so the construction can be iterated without leaving the same well-behaved category.
\index{conditional probability}
\index{tightness}
\index{Wasserstein!space}
\index{weak!convergence}
\index{Polish space}

\begin{prop}[Wasserstein spaces as ground spaces]\label{prop-wasserstein-space-polish}
\index{Wasserstein!space}
	If $(\X,\dist)$ is a Polish metric space, then $\Pp_p(\X)$ endowed with $\Wass_p$ is Polish. If $\X$ is compact, then $\Pp(\X)$ is compact for the Wasserstein topology, and the construction can be iterated to form $\Pp(\Pp(\X))$, $\Pp(\Pp(\Pp(\X)))$, and so on.
\index{Wasserstein topology}
\end{prop}
\begin{proof}
	This is a standard structural theorem for Wasserstein spaces~\cite{Villani09,SantambrogioBook,ambrosio2006gradient}. Completeness follows by representing a $\Wass_p$-Cauchy sequence through almost optimally glued couplings, which gives a Cauchy random sequence whose law is the desired limit; separability follows by approximating measures with finitely supported measures on a countable dense subset and rational weights. If $\X$ is compact, Prokhorov compactness gives compactness of $\Pp(\X)$ for weak convergence, and Proposition~\ref{prop-wass-metrizes-weak-compact} identifies this topology with any Wasserstein topology.
\index{countable dense subset}
\index{Wasserstein!space}
\index{rational weights}
\index{weak!convergence}
\end{proof}

We denote elements of $\Pp_2(\X)$ by $\alpha,\beta,\ldots$. Elements of
$\Pp_2(\Pp_2(\X))$ are denoted by fraktur letters, for instance
$\mathfrak A,\mathfrak B$; they are probability laws over probability measures, or random probability measures. A basic parametric example is obtained from a family $(\alpha_\zeta)_{\zeta\in Z}$ and a probability law $\gamma$ on the parameter space:
\index{probability measure}
\index{random probability measure}
\begin{equation}\label{eq-wow-parametric-law}
	\mathfrak A=(\zeta\mapsto\alpha_\zeta)_\sharp\gamma.
\end{equation}
If $\gamma=\sum_{i=1}^K a_i\delta_{\zeta_i}$, then
\[
	\mathfrak A=\sum_{i=1}^K a_i\delta_{\alpha_{\zeta_i}}.
\]
\begin{defn}[Collapsed, or barycentric, mixture]\label{def-collapsed-barycentric-mixture}
\index{collapsed mixture}
\index{barycentric!mixture}
	For $\mathfrak A\in\Pp(\Pp_2(\X))$, the collapsed, or barycentric, mixture associated with $\mathfrak A$ is the measure $\bar\alpha_{\mathfrak A}$ defined by
	\begin{equation}\label{eq-wow-barycentric-mixture}
		\int_\X f(x)\d\bar\alpha_{\mathfrak A}(x)
		=
		\int_{\Pp_2(\X)}
		\left(\int_\X f(x)\d\alpha(x)\right)
		\d\mathfrak A(\alpha),
	\end{equation}
	for bounded continuous $f$.
\end{defn}
In the finite case, $\bar\alpha_{\mathfrak A}=\sum_i a_i\alpha_{\zeta_i}$.

The Wasserstein distance on the Wasserstein space is
\index{Wasserstein!distance}
\index{Wasserstein!space}
\begin{equation}\label{eq-wow-distance}
	\mathbb W_2^2(\mathfrak A,\mathfrak B)
	\eqdef
	\inf_{\Pi\in\Couplings(\mathfrak A,\mathfrak B)}
	\int_{\Pp_2(\X)\times\Pp_2(\X)}
	\Wass_2^2(\alpha,\beta)\d\Pi(\alpha,\beta).
\end{equation}
For Gaussian mixtures, this separates two levels of geometry. A mixture
\index{Gaussian mixture}
$\sum_i a_i\Gaussian(m_i,\Sigma_i)$ can either be viewed as the collapsed measure on $\X$, or as the component law
\[
	\mathfrak A=\sum_i a_i\delta_{\Gaussian(m_i,\Sigma_i)}
\]
on the Bures-Wasserstein space of Gaussian components. Given two component laws
\[
	\mathfrak A=\sum_i a_i\delta_{\Gaussian(m_i,\Sigma_i)},
	\qquad
	\mathfrak B=\sum_j b_j\delta_{\Gaussian(n_j,\Lambda_j)},
\]
the discrete problem induced by~\eqref{eq-wow-distance} uses the component cost
\[
	C_{ij}=\norm{m_i-n_j}^2+\Bb(\Sigma_i,\Lambda_j)^2.
\]
Let $\Pi^\star$ be an optimal coupling between the weights $a$ and $b$. If $A_{ij}$ denotes the Brenier linear part from $\Sigma_i$ to $\Lambda_j$, then each active component pair is interpolated by
\[
	m_{ij,t}=(1-t)m_i+t n_j,
	\qquad
	\Sigma_{ij,t}
	=
	\big((1-t)\Id+tA_{ij}\big)\Sigma_i
	\big((1-t)\Id+tA_{ij}\big),
\]
and collapsing these component geodesics gives
\[
	\bar\alpha_t=
	\sum_{i,j}\Pi^\star_{ij}\Gaussian(m_{ij,t},\Sigma_{ij,t}).
\]
This component-level interpolation transports Gaussian components as atoms of the Wasserstein space before returning to measures on~$\X$. It is generally not the same as the true $\Wass_2$ interpolation between the collapsed mixture densities, which can split and recombine mass inside and across components.
\index{Bures-Wasserstein geometry}
\index{collapsed mixture}
\index{Wasserstein!space}

\begin{figure}[htbp]
\centering
\begin{tabular}{cc}
\includegraphics[width=.46\linewidth]{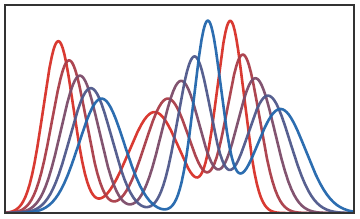} &
\includegraphics[width=.46\linewidth]{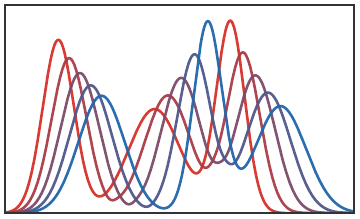} \\[-.1em]
\small component law on Gaussian atoms &
\small collapsed mixture law
\index{collapsed mixture}
\end{tabular}
\caption{Two interpolations between the same three-component one-dimensional Gaussian mixtures. On the left, each mixture is represented as a discrete law over Gaussian components, and the components are matched using the Bures-Wasserstein distance between Gaussians. On the right, the mixtures are first collapsed into ordinary one-dimensional densities and then interpolated by the true quantile formula for $\Wass_2$. The two constructions encode different geometries.}
\index{Gaussian mixture}
\index{Wasserstein!distance}
\index{quantile!formula}
\label{fig:kantorovich-wow-mixtures}
\end{figure}

\begin{prop}[Collapsing is non-expansive]\label{prop-wow-collapsed-bound}
	Let $\mathfrak A,\mathfrak B\in\Pp_2(\Pp_2(\X))$, and let $\bar\alpha_{\mathfrak A}$ and $\bar\beta_{\mathfrak B}$ be the collapsed mixtures defined by~\eqref{eq-wow-barycentric-mixture}. Then
\index{barycentric!mixture}
\index{collapsed mixture}
	\[
		\Wass_2(\bar\alpha_{\mathfrak A},\bar\beta_{\mathfrak B})
		\leq
		\mathbb W_2(\mathfrak A,\mathfrak B).
	\]
\end{prop}
\begin{proof}
	Fix $\Pi\in\Couplings(\mathfrak A,\mathfrak B)$. For every $(\alpha,\beta)$ choose, by a standard measurable selection argument and up to an arbitrarily small error, a coupling $\pi_{\alpha,\beta}\in\Couplings(\alpha,\beta)$ whose quadratic cost is $\Wass_2^2(\alpha,\beta)$. Integrating this Markov kernel against $\Pi$ gives a coupling $\bar\pi$ between $\bar\alpha_{\mathfrak A}$ and $\bar\beta_{\mathfrak B}$. Its cost satisfies
\index{measurable selection}
\index{Markov kernel}
\index{cost!quadratic}
	\[
		\int_{\X\times\X}d(x,y)^2\d\bar\pi(x,y)
		\leq
		\int_{\Pp_2(\X)^2}\Wass_2^2(\alpha,\beta)\d\Pi(\alpha,\beta)
	\]
	up to the arbitrary selection error. Taking first the infimum over $\bar\pi$ and then over $\Pi$ proves the claim.
\end{proof}

The following remark records a useful way in which this iterated construction reappears later for Gromov--Wasserstein lower bounds.

\begin{rem}[Local profiles as Wasserstein-over-Wasserstein laws]
	Given a metric-measure space $\XX=(\X,\dist_\X,\mu_\X)$, each point defines a local distance distribution
\index{Wasserstein!distance}
\index{Gromov-Wasserstein}
\index{Gromov-Wasserstein!distance}
\index{local!distance distribution}
\index{metric-measure space}
\[
	\alpha_x=(\dist_\X(x,\cdot))_\sharp\mu_\X\in\Pp(\RR_+),
	\qquad
	\mathfrak D_\X=(x\mapsto\alpha_x)_\sharp\mu_\X\in\Pp(\Pp(\RR_+)).
\]
The M\'emoli profile lower bound in Proposition~\ref{prop-memoli-gw-profile-lower-bound} is precisely a Wasserstein-over-Wasserstein comparison of these laws of local profiles. It replaces the full pairwise distortion by an ordinary OT problem whose ground cost is itself a one-dimensional Wasserstein distance.
\index{profile lower bound}
\index{local profile}
\index{distortion}
\index{ground cost}
\index{Memoli profile}
\index{Wasserstein!over Wasserstein}
\index{Wasserstein!distance}
	Note that there exist alternative distances which also metricize weak convergence. The simplest ones are Hilbertian kernel norms, which are detailed in Section~\ref{sec-dual-norms}.
\index{kernel!norm}
\index{weak!convergence}
\index{dual!norm}
\end{rem}

\section{Distributional Robustness and \texorpdfstring{$\Wass_\infty$}{W-infinity}}
\index{robustness!distributional}
\label{sec-dro-wasserstein-infinity}

\paragraph{DRO ambiguity sets.}
\index{robustness!distributionally robust optimization}

Wasserstein distances are also used to define ambiguity sets around an empirical law. Given samples $z_i$ and $\hat\alpha_n=\frac1n\sum_i\delta_{z_i}$, a distributionally robust optimization (DRO) problem replaces the empirical risk $\frac1n\sum_i \ell_\theta(z_i)$ by
\index{robustness!distributionally robust optimization}
\index{Wasserstein!distance}
\index{empirical!law}
\[
	\sup_{\beta:\,\Wass_p(\beta,\hat\alpha_n)\leq \rho}
	\int \ell_\theta(z)\d\beta(z),
\]
or, in Lagrangian penalized form, by $\sup_\beta \int \ell_\theta\d\beta-\lambda\Wass_p(\beta,\hat\alpha_n)^p$. The constrained and penalized formulations are linked by the choice of multiplier $\lambda$, but are not the same problem for an arbitrary fixed $\lambda$. Both ask for performance against distributions that can be reached by transporting the empirical mass within a budget. The radius $\rho$ is expressed in the geometry of the data space, so it can encode feature perturbations, domain shift or model misspecification~\cite{esfahani2018data,BlanchetMurthy2019,GaoKleywegt2016}.

The basic computational reason for the popularity of Wasserstein DRO is a dual reformulation. Under the usual upper-semicontinuity and growth assumptions on the loss, one has
\begin{equation}\label{eq-dro-dual-envelope}
	\sup_{\beta:\,\Wass_p(\beta,\hat\alpha_n)^p\leq \rho^p}
	\int \ell_\theta\d\beta
	=
	\inf_{\lambda\geq0}
	\lambda\rho^p
	+
	\frac1n\sum_{i=1}^n
	\sup_z\bigl\{\ell_\theta(z)-\lambda d(z,z_i)^p\bigr\}.
\end{equation}
Thus the robust risk is an empirical risk in which each sample is replaced by its worst penalized perturbation. For $p=1$ and an $L_\theta$-Lipschitz loss, the Kantorovich--Rubinstein dual gives the transparent upper bound
\[
	\sup_{\beta:\,\Wass_1(\beta,\hat\alpha_n)\leq \rho}
	\int \ell_\theta\d\beta
	\leq
	\frac1n\sum_i\ell_\theta(z_i)+\rho L_\theta,
\]
which exhibits Wasserstein robustness as a Lipschitz regularizer. The bound is sharp for worst-case optimization over a full Lipschitz ball of losses, while a fixed loss may not realize all extremal Lipschitz directions.

In machine learning, the inner supremum in~\eqref{eq-dro-dual-envelope} has the form of an adversarial perturbation problem: the adversary moves each datum away from $z_i$ but pays a transport penalty. This connects Wasserstein DRO to adversarial training and certified robustness~\cite{sinha2018certifying,NIPS2015_5745}. In sequential decision problems, the same idea places Wasserstein ambiguity sets around transition kernels or rewards, producing robust Bellman operators and robust reinforcement-learning models~\cite{xu2012distributionallyrobustmdp,yang2017convex}.
\index{adversarial!perturbation}

\begin{prop}[Convexity of transport costs]\label{prop-wasserstein-cost-convex}
\index{cost!transport}
	For any nonnegative lower-semicontinuous cost $c$, the value
	\[
		(\alpha,\beta)\mapsto\MK_c(\alpha,\beta)
	\]
	is jointly convex. In particular, for a ground metric $d$ and $p\geq1$, the map $(\alpha,\beta)\mapsto\Wass_p(\alpha,\beta)^p$ is jointly convex. The distance $\Wass_1$ is jointly convex, but $\Wass_p$ itself need not be convex for $p>1$.
\index{metric!learning}
\end{prop}
\begin{proof}
	Let $\pi_0\in\Couplings(\al_0,\beta_0)$ and $\pi_1\in\Couplings(\al_1,\beta_1)$ be $\eta$-optimal. Then $(1-t)\pi_0+t\pi_1$ is a coupling between $(1-t)\al_0+t\al_1$ and $(1-t)\beta_0+t\beta_1$, and its cost is the corresponding convex combination of the two costs. Letting $\eta\to0$ proves joint convexity of $\MK_c$. Taking $c=d^p$ gives convexity of $\Wass_p^p$. For $p=1$, this is convexity of $\Wass_1$ itself. For $p>1$, the root can destroy convexity: on the real line, $F(t)\eqdef\Wass_p((1-t)\delta_0+t\delta_1,\delta_0)=t^{1/p}$ satisfies $F(1/2)> (F(0)+F(1))/2$.
\end{proof}

This convexity is useful when distributions themselves are decision variables. In the usual DRO problem, however, the ambiguity set is fixed once the data are fixed. Therefore, if $z\mapsto\ell_\theta(z)$ is measurable and $\theta\mapsto\ell_\theta(z)$ is convex for every $z$, then
\[
	\theta\mapsto
	\sup_{\beta:\,\Wass_p(\beta,\hat\alpha_n)\leq\rho}
	\int \ell_\theta\,\d\beta
\]
is convex as a supremum of convex functions of $\theta$. This explains why Wasserstein DRO can preserve convex learning formulations while still modeling adversarial distributional shifts.
\index{convex!function}

\paragraph{\texorpdfstring{$\Wass_\infty$}{W-infinity} robustness.}
\index{Wasserstein!infinity distance}
\index{robustness!Wasserstein infinity}

The limiting distance
\begin{equation}\label{eq-wass-infty}
	\Wass_\infty(\alpha,\beta)
	\eqdef
	\inf_{\pi\in\Couplings(\alpha,\beta)}
	\operatorname*{ess\,sup}_{(x,y)\sim\pi} d(x,y)
\end{equation}
is the limit of $\Wass_p(\alpha,\beta)$ as $p\to\infty$ on bounded spaces, not the limit of the convex programs defining $\Wass_p^p$. It minimizes the worst displacement rather than an average displacement, and the resulting optimization is no longer a linear convex program because of the essential-supremum objective. This makes $\Wass_\infty$ less convenient algorithmically, but very natural in robust formulations where one wants to guarantee that every transported sample stays within a prescribed perturbation radius.

\begin{prop}[$\Wass_\infty$ robust envelope around an empirical law]\label{prop-wasserstein-infty-dro}
\index{empirical!law}
\index{robust!envelope}
	Let $(\Zz,d)$ be a Polish metric space. Let $\hat\alpha=\sum_{i=1}^n a_i\delta_{z_i}$ with $a_i>0$ and $\sum_i a_i=1$, and assume that the closed balls $\overline B(z_i,\rho)$ are compact. For any real-valued upper-semicontinuous loss $\ell$,
	\[
		\sup_{\beta:\,\Wass_\infty(\beta,\hat\alpha)\leq\rho}
		\int \ell(z)\d\beta(z)
		=
		\sum_{i=1}^n a_i\sup_{z\in\overline B(z_i,\rho)}\ell(z).
	\]
\end{prop}
\begin{proof}
	If $\Wass_\infty(\beta,\hat\alpha)\leq\rho$, then, by symmetry, there are couplings $\pi_m\in\Couplings(\hat\alpha,\beta)$ whose essential displacements are at most $\rho+1/m$. Since the two marginals are fixed probability measures on a Polish space, $\Couplings(\hat\alpha,\beta)$ is tight and closed, hence weakly compact by Prokhorov's theorem. After extraction, $\pi_m\rightharpoonup\pi\in\Couplings(\hat\alpha,\beta)$.
\index{probability measure}
\index{Prokhorov theorem}
\index{Polish space}
	For every $\eta>0$, the closed set $F_\eta=\{(x,z):d(x,z)\leq\rho+\eta\}$ has $\pi_m(F_\eta)=1$ for all sufficiently large $m$. Portmanteau's theorem gives $\pi(F_\eta)=1$. Letting $\eta\downarrow0$ along a countable sequence gives $\pi(\{(x,z):d(x,z)\leq\rho\})=1$.
	Disintegrating $\pi$ with respect to the first marginal gives $\pi=\sum_i a_i\delta_{z_i}\otimes\nu_i$, where each $\nu_i$ is supported in the closed ball $\overline B(z_i,\rho)$ and $\beta=\sum_i a_i\nu_i$. Hence
	\[
		\int\ell\,\d\beta
		=
		\sum_i a_i\int\ell\,\d\nu_i
		\leq
		\sum_i a_i\sup_{\overline B(z_i,\rho)}\ell.
	\]
	The reverse inequality follows by choosing, for each $i$, a maximizer $z_i^\star\in\overline B(z_i,\rho)$ and setting $\beta=\sum_i a_i\delta_{z_i^\star}$. The coupling $\sum_i a_i\delta_{(z_i,z_i^\star)}$ has essential displacement at most $\rho$, so this $\beta$ is feasible and attains the displayed value.
\end{proof}

\section{Quantitative Central Limit Theorems}
\index{central!limit theorem}
\label{sec-quantitative-clt}

The weak topology only says whether laws converge; Wasserstein distances also quantify how fast they converge. The next result is a representative theoretical application: the central limit theorem becomes a rate estimate in $\Wass_1$. This is useful because $\Wass_1$ is exactly the supremum over 1-Lipschitz observables, so the bound controls the error of all stable numerical or statistical measurements of the normalized sum at once. Figure~\ref{fig:matching-quantitative-clt} illustrates the elementary Bernoulli case.
\index{central!limit theorem}
\index{Wasserstein!distance}
\index{topology!weak}

\begin{figure}[H]
\centering
\begin{tabular}{@{}ccccc@{}}
\small $n=1$ & \small $n=2$ & \small $n=4$ & \small $n=16$ & \small $n=64$ \\[-.15em]
\includegraphics[width=.17\linewidth]{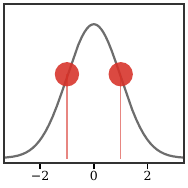} &
\includegraphics[width=.17\linewidth]{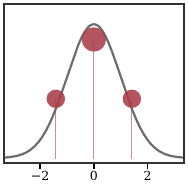} &
\includegraphics[width=.17\linewidth]{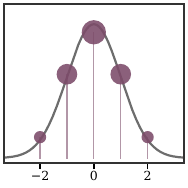} &
\includegraphics[width=.17\linewidth]{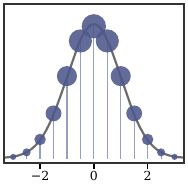} &
\includegraphics[width=.17\linewidth]{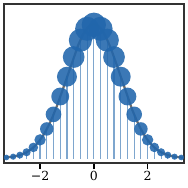}
\end{tabular}
\caption{Central-limit theorem for normalized Bernoulli sums. Starting from $\al_0=\frac12(\delta_{-1}+\delta_1)$, the law of $Z_n=n^{-1/2}\sum_{i=1}^n X_i$ remains discrete, but its rescaled atom heights approach the standard Gaussian density shown in gray. The Wasserstein Berry--Esseen bound below quantifies this weak convergence by a $\Wass_1$ distance.}
\index{central!limit theorem}
\index{weak!convergence}
\label{fig:matching-quantitative-clt}
\end{figure}

The following estimate is the Wasserstein form of the classical Berry--Esseen theorem~\cite{berry1941accuracy,esseen1942liapunoff}. The proof sketch uses Stein's method, whose modern normal-approximation formulation is developed in~\cite{chen2011normal}.
\index{Berry-Esseen!theorem}
\index{normal approximation}
\index{Stein method}

\begin{prop}[Berry--Esseen bound in $\Wass_1$]\label{prop-berry-esseen-w1}
\index{Berry-Esseen!bound}
	Let $(X_i)_{i=1}^n$ be i.i.d. real random variables such that $\EE X_i=0$, $\EE X_i^2=1$ and $\EE |X_i|^3<+\infty$. If $\alpha_n$ is the law of $n^{-1/2}\sum_i X_i$ and $\gamma$ is the standard Gaussian law, then
\index{random variable}
	\[
		\Wass_1(\alpha_n,\gamma)
		\leq
		\frac{C\,\EE |X_1|^3}{\sqrt n},
	\]
	where $C$ is a universal constant.
\end{prop}

\begin{proof}
	We give the standard Stein-method sketch. By Kantorovich--Rubinstein duality,
\index{Kantorovich-Rubinstein!duality}
\index{Stein method}
	\[
		\Wass_1(\alpha_n,\gamma)
		=
		\sup_{\Lip(h)\leq1}
		\left|\EE h(S_n)-\EE h(G)\right|,
		\qquad
		S_n=n^{-1/2}\sum_i X_i,\ G\sim\gamma.
	\]
	For each such $h$, solve Stein's equation
	\[
		f_h'(x)-x f_h(x)=h(x)-\EE h(G).
	\]
	The solution satisfies uniform derivative bounds depending only on the Lipschitz constant of $h$. Hence
	\[
		\EE h(S_n)-\EE h(G)=\EE[f_h'(S_n)-S_nf_h(S_n)].
	\]
	Expanding $f_h'(S_n)$ and $f_h(S_n)$ by replacing the summands one at a time, the first- and second-order terms cancel because $\EE X_i=0$ and $\EE X_i^2=1$. The Taylor remainder is bounded by $C\sum_i \EE |X_i/\sqrt n|^3$, which gives the displayed $n^{-1/2}$ rate. Sharper constants and higher-order transport-distance refinements are studied in~\cite{bobkov2018berry,rio2011asymptotic}.
\end{proof}


\chapter{Dual Problem}
\index{dual!problem}
\label{sec-dual}

Duality turns the transport problem into a search for potentials rather than couplings. This chapter explains why potentials certify optimality, how $c$-transforms regularize them, and why the quadratic case reveals convex analysis behind Brenier maps. Linear-programming duality gives the discrete picture~\cite{bertsimas1997introduction}, while the continuous form is one of the central theorems of OT~\cite{Villani03,SantambrogioBook}.
\index{Brenier!map}
\index{c-transform}

\section{Discrete dual}

The discrete dual gives finite-dimensional certificates of optimality. Its complementary slackness conditions identify where an optimal coupling can put mass.
\index{optimal coupling}
\index{optimality!complementary slackness}

The Kantorovich problem~\eqref{eq-kanto-discr} is a linear program so that one can equivalently compute its value by solving a dual linear program.
\index{Kantorovich!problem}


\begin{defn}[Admissible potentials]\label{def-admissible-potentials}
\index{admissible!potential}
\index{dual!potential}
	For a discrete problem with marginal sizes $n,m$ and cost matrix $\C\in\RR^{n\times m}$, a pair $(\fD,\gD)\in\RR^n\times\RR^m$ is admissible if it lies below the cost:
	\eql{\label{eq-feasible-potential}
		\PotentialsD(\a,\b) \eqdef \enscond{
			(\fD,\gD) \in \RR^n \times \RR^m
		}{ \foralls (i,j) \in \range{n} \times \range{m}, \fD_i+\gD_j \leq \C_{i,j} }.
	}
	Equivalently, $\fD\oplus\gD\leq\C$ entrywise. The notation suppresses the dependence on $\C$, which is fixed in the surrounding problem.
\end{defn}
The two vectors play the role of source and target prices; admissibility means that no transported pair is priced above its travel cost.

\begin{prop}[Discrete Kantorovich duality]\label{prop-duality-discr}
\index{Kantorovich!duality}
One has
\eql{\label{eq-dual}
	\MKD_\C(\a,\b) =
	\umax{(\fD,\gD) \in \PotentialsD(\a,\b)} \dotp{\fD}{\a} + \dotp{\gD}{\b}
}
\end{prop}

\begin{proof}
For the sake of completeness, let us derive this dual problem using Lagrangian duality. The Lagrangian associated to~\eqref{eq-kanto-discr} reads
\index{dual!problem}
\eql{\label{eq-mk-lagr}
	\umin{\P \geq 0} \umax{ (\fD,\gD) \in \RR^n \times \RR^m }
		\dotp{\C}{\P} + \dotp{\a - \P\ones_m}{\fD} + \dotp{\b - \P^\top \ones_n}{\gD}.
}
For a linear program, if the primal constraint set is non-empty, one can always exchange the min and the max and get the same value. We thus consider
\eq{
	\umax{ (\fD,\gD) \in \RR^n \times \RR^m }
	\dotp{\a}{\fD} + \dotp{\b}{\gD}
	+ \umin{\P \geq 0}
		\dotp{\C - \fD\ones_m^\top - \ones_n \gD^\top}{\P}.
}
We conclude by remarking that
\eq{
	\umin{\P \geq 0} \dotp{\Q}{\P} =
	\choice{
		0 \qifq \Q \geq 0\\
		-\infty \quad \text{otherwise}
	}
}
so that the constraint reads $\C - \fD\ones_m^\top - \ones_n \gD^\top = \C-\fD\oplus \gD \geq 0$.
\end{proof}

The primal-dual optimality relation for the Lagrangian~\eqref{eq-mk-lagr} allows locating the support of the optimal transport plan
\index{support}
\index{plan!transport}
\eql{\label{eq-mk-pd-rel}
	\Supp(\P) \subset \enscond{(i,j) \in \range{n} \times \range{m}}{ \fD_i+\gD_j=\C_{i,j} }.
}

Figure~\ref{fig:dual-kantorovich-discrete-potentials} shows these finite-dimensional certificates on a one-dimensional quadratic problem.  The potentials are not transport maps themselves; rather, their contact set with the cost matrix is where an optimal coupling is allowed to put mass.
\index{cost matrix}
\index{optimal coupling}
\index{transport map}

\begin{figure}[ht]
\centering
\begin{tabular}{@{}ccc@{}}
\includegraphics[width=.30\linewidth]{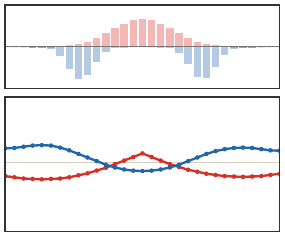} &
\includegraphics[width=.30\linewidth]{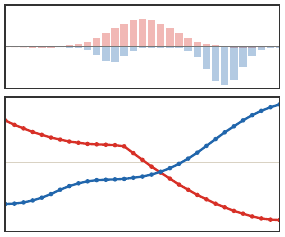} &
\includegraphics[width=.30\linewidth]{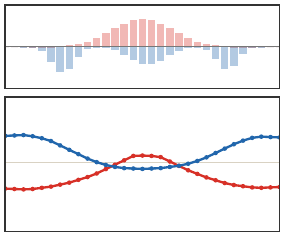} \\[-.1em]
\small balanced target &
\small strongly shifted target &
\small three-mode target
\end{tabular}
\caption{Discrete Kantorovich dual potentials for the quadratic cost $C_{i,j}=|x_i-y_j|^2$.  In each panel the upper strip shows the fixed source histogram in red and the target histogram in blue: a balanced bimodal target, a more strongly shifted target, and a three-component mixture.  The lower strip shows an optimal pair of dual vectors $(\fD,\gD)$, with gauge chosen so that $\dotp{\fD}{\a}=0$.  Complementary slackness states that mass can be transported only through entries where $\fD_i+\gD_j=C_{i,j}$.}
\index{optimality!complementary slackness}
\index{dual!potential}
\index{cost!quadratic}
\label{fig:dual-kantorovich-discrete-potentials}
\end{figure}

The formulation~\eqref{eq-dual} shows that $(\a,\b) \mapsto \MKD_\C(\a,\b)$ is a convex function (as a supremum of linear functions). From the primal problem~\eqref{eq-kanto-discr}, one also sees that $\C \mapsto \MKD_\C(\a,\b)$ is concave.
\index{convex!function}

\section{Auction Algorithm and Dual Prices}
\index{auction algorithm}
\index{dual!price}
\label{sec-auction-dual-ascent}

The assignment algorithms mentioned in Chapter~\ref{sec-matching} become more transparent once one has dual variables. The auction algorithm is a dual price method: it updates target prices and maintains an approximate complementary-slackness certificate. The small tolerance $\epsilon$ removes ties, stabilizes the price updates and gives a quantitative optimality certificate~\cite{bertsekas1981new,bertsekas1992auction,merigot2020optimaltransportalgorithms}.
\index{optimality!certificate}
\index{optimality!complementary slackness}
\index{auction algorithm}

Consider the square assignment problem with costs $C_{i,j}$ and rewrite it as the profit maximization problem with $a_{i,j}=-C_{i,j}$. The auction algorithm keeps prices $p_j$ on the target points and a partial assignment. For an unassigned source $i$, define the best and second-best reduced profits
\index{assignment problem}
\[
	v_i=\max_j (a_{i,j}-p_j),\qquad
	j_i\in\operatorname*{argmax}_j(a_{i,j}-p_j),\qquad
	w_i=\max_{j\neq j_i}(a_{i,j}-p_j).
\]
Source $i$ bids for $j_i$ and increases its price by the gap to the second-best target, plus a margin:
\[
	p_{j_i}\leftarrow p_{j_i}+v_i-w_i+\epsilon.
\]
The target $j_i$ is then assigned to $i$, and its previous owner, if any, becomes unassigned. The iteration stops when all sources are assigned. Algorithm~\ref{alg:auction-bidding} records the bidding loop. Figure~\ref{fig:dual-auction-progression} displays actual auction iterates using the same assignment-state convention as Figure~\ref{fig:matching-hungarian-progression}: flat rows denote unassigned bidders, and one-hot rows denote currently owned targets.
\index{Hungarian primal-dual method}

\begin{alg}[Auction bidding with target prices]\label{alg:auction-bidding}
\textbf{Input:} Profit matrix $A=(a_{ij})$, bid increment $\epsilon>0$.

\textbf{Output:} Assignment map $\sigma$.

\textbf{Initialize:} Set prices $p_j=0$, ownership map $o(j)=\emptyset$, and $U=\{1,\ldots,n\}$.

\textbf{While} the unassigned set $U$ is nonempty \textbf{do}:
\begin{algblock}
\textbf{Set} $i=\min U$.

\textbf{Set} $j_i=\min\argmax_j(a_{ij}-p_j)$.

\textbf{Set} \(v_i=a_{ij_i}-p_{j_i}\) and \(w_i=\max_{j\neq j_i}(a_{ij}-p_j)\).

\textbf{Update price:}
\(p_{j_i}\leftarrow p_{j_i}+v_i-w_i+\epsilon.\)

\textbf{If} $o(j_i)=i'\neq\emptyset$ \textbf{then}:
\begin{algblock}

\textbf{Set} $U\leftarrow U\cup\{i'\}$.

\end{algblock}
\textbf{Set} $o(j_i)=i$ and $U\leftarrow U\setminus\{i\}$.

\end{algblock}
\algreturnskip
\textbf{Return} $\sigma(i)=j$ iff $o(j)=i$.
\end{alg}

\begin{figure}[ht]
\centering
\setlength{\tabcolsep}{2pt}
\begin{tabular}{ccccc}
\includegraphics[width=.18\linewidth]{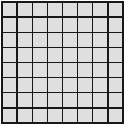} &
\includegraphics[width=.18\linewidth]{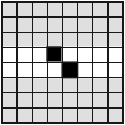} &
\includegraphics[width=.18\linewidth]{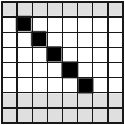} &
\includegraphics[width=.18\linewidth]{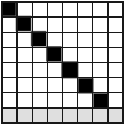} &
\includegraphics[width=.18\linewidth]{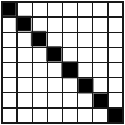} \\[-.1em]
\small initial &
\small 2 bids &
\small 5 bids &
\small 7 bids &
\small final
\end{tabular}
\caption{Matrix view of actual auction iterates on the same diagonally dominant one-dimensional squared-distance assignment as Figure~\ref{fig:matching-hungarian-progression}. Each panel records the current ownership state: unassigned bidders are shown as flat rows, while assigned bidders are shown as one-hot rows at their currently held target. The snapshots show initialization, intermediate price updates, and the final identity assignment satisfying complementary slackness.}
\index{optimality!complementary slackness}
\label{fig:dual-auction-progression}
\end{figure}

For fixed prices $p$, eliminating the bidder utilities $u_i$ in the dual minimization gives the convex objective
\[
	D(p)=\sum_j p_j+\sum_i \max_j(a_{i,j}-p_j),
\]
which comes from the dual constraints $u_i+p_j\geq a_{i,j}$. The auction update should thus be viewed as a price-adjustment method on this nonsmooth dual landscape, rather than as a generic gradient step: the proof of correctness is through approximate complementary slackness.
\index{optimality!complementary slackness}

\begin{defn}[$\epsilon$-complementary slackness]
	An assignment $\sigma$ and prices $p$ satisfy $\epsilon$-complementary slackness if, for every source $i$,
	\[
		a_{i,\sigma(i)}-p_{\sigma(i)}
		\geq
		\max_j(a_{i,j}-p_j)-\epsilon.
	\]
\end{defn}

\begin{prop}[Auction optimality certificate]\label{prop-auction-eps-cs}
\index{optimality!certificate}
	If a complete assignment $\sigma$ satisfies $\epsilon$-complementary slackness, then it is $n\epsilon$-optimal for the profit maximization problem, or equivalently $n\epsilon$-optimal for the original cost minimization problem. If all costs are integers and $\epsilon<1/n$, then $\sigma$ is optimal.
\index{optimality!complementary slackness}
\end{prop}

\begin{proof}
	Let $\tau$ be any assignment. By $\epsilon$-complementary slackness,
\index{optimality!complementary slackness}
	\[
		a_{i,\tau(i)}-p_{\tau(i)}
		\leq
		\max_j(a_{i,j}-p_j)
		\leq
		a_{i,\sigma(i)}-p_{\sigma(i)}+\epsilon.
	\]
	Summing over $i$ cancels prices, because both $\sigma$ and $\tau$ are permutations:
	\[
		\sum_i a_{i,\tau(i)}
		\leq
		\sum_i a_{i,\sigma(i)}+n\epsilon.
	\]
	Thus no assignment has profit more than $n\epsilon$ above that of $\sigma$. Since $a=-C$, the same statement says that the cost of $\sigma$ is at most $n\epsilon$ above the minimum cost. If the costs are integers, all assignment costs are integers; a gap strictly smaller than one therefore forces the gap to be zero.
\index{cost!assignment}
\end{proof}

During the bidding process the last bid made by a source makes its chosen target better, up to the margin $\epsilon$, than all alternatives. Subsequent price increases can only make targets less attractive, and one checks by induction that currently assigned pairs satisfy $\epsilon$-complementary slackness. The standard finite-termination proof normalizes prices by subtracting their minimum, observes that each bid increases one target price by at least $\epsilon$, and bounds the normalized price spread in terms of the range of the profits; see~\cite{bertsekas1992auction,merigot2020optimaltransportalgorithms} for the full bound and implementation details.
\index{optimality!complementary slackness}

\begin{rem}[$\epsilon$-scaling and relation with Sinkhorn]
\index{matrix!scaling}
\index{Sinkhorn!algorithm}
	In practice one starts with a coarse $\epsilon$ and repeatedly decreases it, warm-starting the prices and assignment. This $\epsilon$-scaling strategy is a homotopy method: large $\epsilon$ regularizes the combinatorial problem by enforcing a visible margin between the best and second-best reduced profits, while small $\epsilon$ recovers the exact dual certificate. If one wants a continuous-optimization analogy, the margin is closer to an exact-penalty or proximal continuation parameter than to a literal quadratic penalty.
\index{dual!certificate}

	Sinkhorn scaling plays a parallel role for entropic OT. There, the hard minimum in the dual $c$-transform is replaced by a soft minimum, or log-sum-exp, with temperature $\epsilon$; in the auction algorithm, the hard maximum is kept but the complementary slackness condition is relaxed by $\epsilon$. Both methods therefore use an $\epsilon$-controlled dual continuation, and both recover the unregularized transport certificate as $\epsilon\to0$ under the usual assumptions. The outputs are different: Sinkhorn produces dense entropic couplings, whereas auction keeps a sparse assignment throughout.
\index{Sinkhorn!scaling}
\index{entropic!OT}
\index{log-sum-exp}
\index{optimality!complementary slackness}
\index{auction algorithm}
\index{soft!minimum}
\index{c-transform}
\end{rem}

\section{General formulation}
\label{sec-dual-general}

The continuous dual is the analytic counterpart of the discrete linear program. It uses continuous potentials because measures are probed through integration.

To extend this primal-dual construction to arbitrary measures, it is important to realize that measures are naturally paired in duality with continuous functions, using the pairing $\dotp{\f}{\al} \eqdef \int f \d \al$.


\begin{prop}[Kantorovich duality]\label{prop-kantorovich-duality-general}
\index{Kantorovich!duality}
	Assume that $\Xx$ and $\Yy$ are compact metric spaces and that $c\in\Cc(\Xx\times\Yy)$. Then
	\eql{\label{eq-dual-generic}
		\MK_\c(\al,\be) =
		\umax{(\f,\g) \in \Potentials(\c)}
			\int_\X \f(x) \d\al(x) + \int_\Y \g(y) \d\be(y),
	}
	where the set of admissible dual potentials is
\index{dual!potential}
	\eql{\label{eq-dfn-pot-dual}
		\Potentials(\c) \eqdef \enscond{
			(\f,\g) \in \Cc(\X) \times \Cc(\Y)
		}{
			\forall (x,y), \f(x)+\g(y) \leq \c(x,y)
		}.
	}
	Here, $(\f,\g)$ is a pair of continuous functions, often called ``Kantorovich potentials''.
\index{Kantorovich!potential}
	The same formula extends under the usual lower-semicontinuity and integrability assumptions, replacing maxima by suprema when dual optimizers need not exist.
\end{prop}
\begin{proof}
	Weak duality is immediate: if $f(x)+g(y)\leq c(x,y)$ and $\pi\in\Couplings(\alpha,\beta)$, then
\index{duality!weak}
	\[
		\int f\d\alpha+\int g\d\beta
		=
		\int (f(x)+g(y))\d\pi(x,y)
		\leq
		\int c\d\pi.
	\]
	Taking the supremum over admissible potentials and the infimum over couplings gives ``$\leq$''.

	For the reverse inequality, view the primal problem as a linear program over the locally convex space of Radon measures, paired with $\Cc(\X\times\Y)$. The affine map $\pi\mapsto(\pi_1,\pi_2)$ is continuous for the weak topology, the feasible set is non-empty because it contains $\alpha\otimes\beta$, and the cost is continuous and bounded on compact sets. Since the set of probability measures on the compact product is weakly compact, the set of attainable cost-marginal triples is closed after adding the epigraph variable below. The separating-hyperplane theorem applied to the convex set of attainable triples
\index{affine map}
\index{probability measure}
\index{Radon!measure}
\index{topology!weak}
	\[
		\Big\{\big(\pi_1,\pi_2,\int c\d\pi+r\big): \pi\geq0,\ r\geq0\Big\}
	\]
	gives a continuous affine separator, hence functions $(f,g)$ and a scalar multiplier which can be normalized so that $f\oplus g\leq c$. The separating inequality then states that the supremum over such potentials is at least the primal value. This proves equality. The same argument is the infinite-dimensional analogue of the finite linear-programming proof in Proposition~\ref{prop-duality-discr}.
	\end{proof}

\begin{rem}[Dual attainment from $c$-transforms]\label{rem-kantorovich-dual-attainment}
\index{dual!attainment}
\index{c-transform}
	Under the compactness and continuity assumptions of Proposition~\ref{prop-kantorovich-duality-general}, the maximum in~\eqref{eq-dual-generic} is attained. If $c$ is Lipschitz, Proposition~\ref{prop-c-transform-lipschitz} shows that one may replace an admissible pair by its $c$-transforms without decreasing the dual objective; after fixing one additive gauge, the transformed potentials are uniformly bounded and equi-Lipschitz. Arzel\`a--Ascoli then gives a converging maximizing subsequence, and the closed constraint $f\oplus g\leq c$ passes to the limit.
\index{Kantorovich!duality}
\index{c-transform}
\end{rem}

The discrete case~\eqref{eq-dual} corresponds to the dual vectors being samples of the continuous potentials, \emph{i.e.} $(\fD_i,\gD_j)=(\f(x_i),\g(y_j))$. The primal-dual optimality conditions allow for tracking the support of the optimal plan, and~\eqref{eq-mk-pd-rel} is generalized as
\index{support}
\index{optimal plan}
\eql{\label{eq-mk-pd-rel-cont}
	\Supp(\pi) \subset \enscond{(x,y) \in \X \times \Y}{ \f(x)+\g(y)=\c(x,y) }.
}
For the one-dimensional quadratic cost, the continuous potentials can be read from the monotone map $T=F_\be^{-1}\circ F_\al$: on the active graph, $f'(x)=2(x-T(x))$ and $g=f^c$.
\index{cost!quadratic}

\begin{figure}[ht]
\centering
\begin{tabular}{@{}ccc@{}}
\includegraphics[width=.30\linewidth]{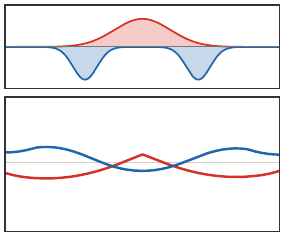} &
\includegraphics[width=.30\linewidth]{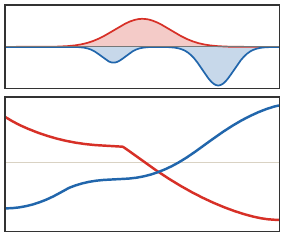} &
\includegraphics[width=.30\linewidth]{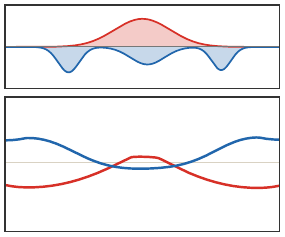} \\[-.1em]
\small balanced target &
\small strongly shifted target &
\small three-mode target
\end{tabular}
\caption{Continuous Kantorovich potentials for the same source and target families as Figure~\ref{fig:dual-kantorovich-discrete-potentials}.  The upper strips show the source density $\al$ in red and the target density $\be$ in blue, including the strongly shifted and three-mode targets.  The lower strips show potentials $f$ and $g=f^c$ for the quadratic cost $c(x,y)=|x-y|^2$, with the same gauge convention as in the discrete figure.  The equality set $f(x)+g(y)=c(x,y)$ contains the monotone transport graph.}
\index{Kantorovich!potential}
\index{cost!quadratic}
\label{fig:dual-kantorovich-continuous-potentials}
\end{figure}

Note that in contrast to the primal problem~\eqref{eq-mk-generic}, showing the existence of solutions to~\eqref{eq-dual-generic} is non-trivial, because the constraint set $\Potentials(\c)$ is not compact and the objective is not coercive. Using the machinery of $c$-transforms detailed in Section~\ref{sec-c-transfo}, one can show that optimal $(\f,\g)$ are necessarily Lipschitz regular, which enables the replacement of the constraint by a compact one.
\index{c-transform}

\section[c-transforms]{$c$-transforms}
\index{c-transform}
\label{sec-c-transfo}

The $c$-transform is the operation that improves potentials without changing feasibility. It is both a proof device for dual attainment and the route from duality to Brenier's convex potentials.
\index{dual!attainment}
\index{convex!potential}

\paragraph{Best-response potentials and the $c$-transform.}
\index{c-transform}

Keeping a dual potential $\f$ fixed, one can maximize in closed form over the second potential in the dual problem~\eqref{eq-dual-generic}, which leads one to consider
\index{dual!potential}
\index{dual!problem}
\eq{
	\usup{\g \in \Cc(\Yy)} \enscond{
		\int g \d \be }{
			\foralls (x,y), \g(y) \leq c(x,y) - \f(x)
		}.
}
The constraint can be replaced by
\eq{
	\foralls y \in \Yy, \quad \g(y) \leq \f^\c(y)
}.

\begin{defn}[$c$-transform]\label{def-c-transform}
\index{c-transform}
	For a function $f:\Xx\to\RR\cup\{-\infty\}$, its $c$-transform is
	\eql{\label{eq-c-transform}
		\foralls y \in \Y, \quad
		\f^\c(y) \eqdef \uinf{x \in \X} \c(x,y) - \f(x).
	}
	For a function $g:\Yy\to\RR\cup\{-\infty\}$, the $\bar c$-transform associated with $\bar c(y,x)=c(x,y)$ is
	\[
		\foralls x \in \X, \quad
		\g^{\bar\c}(x) \eqdef \uinf{y \in \Y} \c(x,y) - \g(y).
	\]
\end{defn}
Since $\be$ is positive, the maximization of $\int \g \d \be$ is thus achieved at those functions such that $\g=\f^\c$ on the support of $\be$, which means $\be$-almost everywhere.

\begin{figure}[H]
\centering
\begin{tabular}{@{}ccc@{}}
\small $p=1$ &
\small $p=2$ &
\small $p=4$ \\[-.15em]
\includegraphics[width=.30\linewidth]{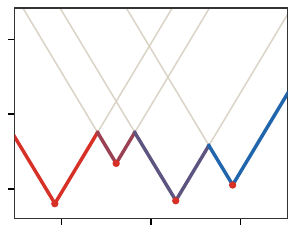} &
\index{c-transform}
\includegraphics[width=.30\linewidth]{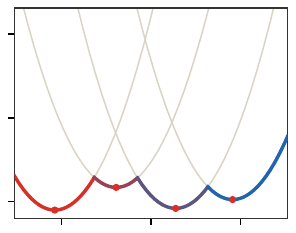} &
\includegraphics[width=.30\linewidth]{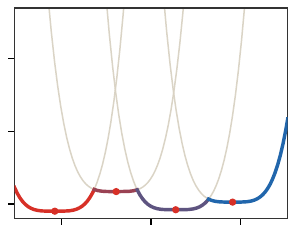}
\end{tabular}
\caption{Discrete $c$-transform as a lower envelope for costs $c_p(x,y)=|x-y|^p$.  The red circles are four source atoms $x_i$ with potential values $f_i$; the gray curves are the translated functions $y\mapsto c_p(x_i,y)-f_i$; the colored curve is their lower envelope $f^c(y)=\min_i c_p(x_i,y)-f_i$. This is the semi-discrete situation where $\X$ is finite, equivalently the source measure $\al$ is discrete; Chapter~\ref{sec-semidiscr-w1} studies the complementary case where the eliminated potential is supported on finitely many target atoms.}
\index{semi-discrete!OT}
\index{c-transform}
\label{fig:dual-c-transform-envelope}
\end{figure}

\begin{prop}[$c$-transforms solve the semi-relaxed problems]\label{prop-c-transform-semi-relaxed}
\index{c-transform}
\index{semi-relaxed!problem}
	For fixed $f$, the maximizers of the dual objective over all $g$ such that $f\oplus g\leq c$ are exactly the functions satisfying $g=f^c$ $\be$-almost everywhere. Equivalently, $f^c$ gives the value of the one-marginal primal problem
	\[
		\inf_{\pi:\,\pi_2=\be}
		\int c(x,y)\d\pi(x,y)-\int f(x)\d\pi_1(x)
		=
		\int f^c(y)\d\be(y).
	\]
	Symmetrically, for fixed $g$, the maximizers over $f$ are the functions satisfying $f=g^{\bar c}$ $\al$-almost everywhere.
\end{prop}
\begin{proof}
	The constraint $f(x)+g(y)\leq c(x,y)$ for all $x$ is equivalent, for each fixed $y$, to
	\[
		g(y)\leq \inf_x c(x,y)-f(x)=f^c(y).
	\]
	Since $\be$ is nonnegative, the largest possible value of $\int g\d\be$ is obtained by saturating this pointwise upper bound on the support of $\be$. The proof for $f=g^{\bar c}$ is identical after exchanging the two marginals.

	For the primal formula, disintegrate any feasible $\pi$ as $\pi(\d x,\d y)=\pi_y(\d x)\be(\d y)$. Then
\index{disintegration}
	\[
		\int c\d\pi-\int f\d\pi_1
		=
		\int\left(\int (c(x,y)-f(x))\d\pi_y(x)\right)\d\be(y)
		\geq
		\int f^c(y)\d\be(y).
	\]
	If minimizers admit a measurable selection, equality is obtained by choosing $\pi_y$ supported on minimizers of $x\mapsto c(x,y)-f(x)$. Otherwise one uses approximate measurable selections and lets the approximation error vanish.
\index{measurable selection}
\end{proof}

The map $(f,g)\mapsto(g^{\bar c},f^c)$ replaces dual potentials by better ones, in the sense that it preserves feasibility and improves the dual objective. Functions of the form $f^c$ and $g^{\bar c}$ are called $c$-concave and $\bar c$-concave functions. These partial minimizations define maximizers on the supports of $\al$ and $\be$, while Definition~\ref{def-c-transform} defines functions on the whole spaces $\X$ and $\Y$. This gives a canonical extension of dual solutions beyond the active supports.
\index{dual!potential}
\index{c-transform}

\begin{prop}[Lipschitz stability of $c$-transforms]\label{prop-c-transform-lipschitz}
\index{Lipschitz stability}
\index{Lipschitz!stability}
\index{c-transform}
	If $c$ is $L$-Lipschitz with respect to its second variable, uniformly in the first one and for the metric $d_\Y$ on $\Y$, then $f^c$ is $L$-Lipschitz.
\end{prop}
\begin{proof}
	For each $x$, set $F_x(y)=c(x,y)-f(x)$ and $F(y)=f^c(y)=\inf_x F_x(y)$. Since all the functions $F_x$ are $L$-Lipschitz,
	\[
		|F(y)-F(y')|
		=
		\left|\inf_x F_x(y)-\inf_x F_x(y')\right|
		\leq
		\sup_x |F_x(y)-F_x(y')|
		\leq
		L d_\Y(y,y').
	\]
\end{proof}

This stability is crucial for dual attainment. When $c$ is Lipschitz on compact spaces, one can replace arbitrary admissible potentials by $c$-transformed ones with a uniform Lipschitz bound; after fixing the harmless additive gauge, compactness follows from the Arzela--Ascoli theorem.
\index{dual!attainment}

\paragraph{Euclidean case.}

The Euclidean quadratic cost is the model case where $c$-transforms become ordinary convex conjugates after removing the quadratic terms. This is the algebraic bridge between Kantorovich duality and Brenier maps.
\index{Kantorovich!duality}
\index{convex!conjugate}
\index{cost!quadratic}
\index{Brenier!map}
\index{c-transform}

The special cost $c(x,y)=-\dotp{x}{y}$ on $\X=\Y=\RR^d$ is central because it reduces the quadratic Wasserstein problem to convex duality. Indeed, for any $\pi\in\Couplings(\al,\be)$,
\[
	\int \norm{x-y}^2\d\pi(x,y)
	=
	\int\norm{x}^2\d\al(x)+\int\norm{y}^2\d\be(y)
	-2\int \dotp{x}{y}\d\pi(x,y).
\]
For $c(x,y)=-\dotp{x}{y}$, one has
\[
	f^c(y)
	=
	\inf_x -\dotp{x}{y}-f(x)
	=
	-(-f)^*(y),
	\qquad
	h^*(y)\eqdef\sup_x \dotp{x}{y}-h(x).
\]
Thus $c$-concave functions are negatives of convex functions. In the one-dimensional bilinear model case, the hard double $c$-transform is therefore an operation of taking concave envelopes.
\index{convex!function}
\index{c-transform}

\begin{rem}[Proof of Brenier's theorem]\label{rem-proof-brenier}
\index{Brenier!theorem}
	For $c(x,y)=\norm{x-y}^2$, subtracting the harmless quadratic terms reduces the geometry to the bilinear cost $-\dotp{x}{y}$. The primal-dual relationship, together with the fact that one can replace $(f,g)$ by $(f^{c\bar c},f^c)$, shows that an optimal plan satisfies
\index{optimal plan}
	\[
		\supp(\pi)
		\subset
		\enscond{(x,y)}{\phi(x)+\phi^*(y)=\dotp{x}{y}},
	\]
	where $\phi=-f^{c\bar c}$ is convex and $-g=\phi^*$. By the Fenchel inequality, equality holds exactly when $y\in\partial\phi(x)$. If $\al$ has a density, convex functions are differentiable Lebesgue-almost everywhere, hence $\al$-almost everywhere, so $\partial\phi(x)$ is a singleton for $\al$-almost every $x$. This yields the Brenier map $T=\nabla\phi$ and explains why the optimal coupling is concentrated on a graph.
\index{optimal coupling}
\index{convex!function}
\index{Brenier!map}
\end{rem}

\paragraph{The failure of alternate optimization.}
\index{alternate optimization}
\index{alternating optimization}

A crucial property of the Legendre transform is that $f^{***}=f^*$, and that $f^{**}$ is the convex envelope of $f$ (the largest convex function below $f$). These properties carry over for the more general setting of $c$-transforms. The required convex-analytic background is standard in Rockafellar's theory~\cite{rockafellar2015convex}.
\index{convex envelope}
\index{convex!function}
\index{Legendre transform}
\index{c-transform}

\begin{prop}[Algebra of $c$-transforms]
\index{c-transform}
The following identities, in which the inequality sign between vectors should be understood elementwise, hold, denoting $\f^{c\bar c} \eqdef (\f^{c})^{\bar c}$:
	\[
		\textnormal{(i) } f \leq f' \Rightarrow f^{c}\geq f'^{c},
		\qquad
		\textnormal{(ii) } f^{\c\bar{\c}} \geq f,
		\qquad
		\textnormal{(iii) } g^{\bar{\c}\c} \geq g,
		\qquad
		\textnormal{(iv) } \f^{\c\bar{\c}\c}=\f^{\c}.
	\]
\end{prop}

\begin{proof} The first inequality (i) follows from the definition of $\c$-transforms (because of the $-$ sign). To prove (ii), expanding the definition of $\f^{\c\bar{\c}}$ we have
\eq{
	\left(\f^{\c\bar{\c}}\right)(x)= \min_{y} \c(x,y)-\f^{\c}(y)
		= \min_{y} \c(x,y) - \min_{x'}( \c(x',y) - \f(x') ).}
Now, since $-\min_{x'}\big(\c(x',y)-\f(x')\big) \geq -(\c(x,y)-\f(x))$, we recover
\eq{
	(\f^{\c\bar{\c}})(x) \geq \min_{y} \c(x,y) - \c(x,y)+\f(x) = \f(x).
}
The relation $\g^{\bar{\c}\c} \geq \g$ is obtained in the same way.
Now, to prove (iv), we first apply (ii) and then (i) with $f'=f^{c\bar c}$ to have $f^c \geq f^{c \bar c c}$.
Then we apply (iii) to $g=f^c$ to obtain $f^c \leq f^{c\bar c c}$.
\end{proof}

This invariance property shows that one can ``improve'' only once the dual potential this way. Indeed, starting from any pair $(f,g)$, one obtains the following iterates by alternating maximization
\index{dual!potential}
\eql{\label{eq-iter-c-trans}
	(f,g) \mapsto (f,f^c) \mapsto (f^{c\bar c},f^c) \mapsto (f^{c\bar c},f^{c\bar c c}) =  (f^{c\bar c},f^c) \ldots
}
so that one reaches a stationary point.
\index{stationarity}

\begin{alg}[Hard alternating $c$-transform closure]\label{alg:hard-c-transform-closure}
\index{c-transform}
\textbf{Input:} Source potential $f$ on $\X$, cost $c$.

\textbf{Output:} Closed $c$-concave pair $(\tilde f,\tilde g)$.

\textbf{Set} target best response:
\(g=f^c, \qquad g(y)=\inf_{x\in\X}\ c(x,y)-f(x).\)

\textbf{Set} source closure:
\(\tilde f=g^{\bar c}=f^{c\bar c}.\)

\textbf{Set} closed target potential:
\(\tilde g=\tilde f^c=f^{c\bar c c}=f^c.\)
\textbf{Return} $(\tilde f,\tilde g)=(f^{c\bar c},f^c)$.
\end{alg}

This failure is the classical behavior of alternating maximization on a non-smooth problem, where the non-smooth part of the functional (here the constraint) mixes the two variables.
The workaround is to introduce smoothing, which is the classical method of augmented Lagrangian, and that we will develop here using entropic regularization, which corresponds to Sinkhorn's algorithm.
\index{Sinkhorn!algorithm}
\index{entropic!regularization}

For the bilinear cost $c(x,y)=-xy$ on a compact interval, the $c$-concave functions are ordinary concave functions and $f^{c\bar c}$ is the smallest concave majorant of $f$. In that model case, a hard transform removes non-concave oscillations in one closure step rather than producing a gradual ascent.

\begin{figure}[ht]
\centering
\begin{tabular}{@{}cc@{}}
\includegraphics[width=.46\linewidth]{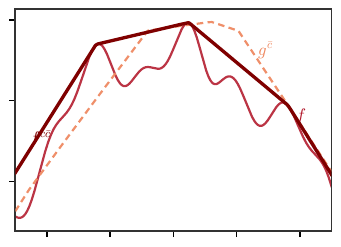} &
\index{c-transform}
\includegraphics[width=.46\linewidth]{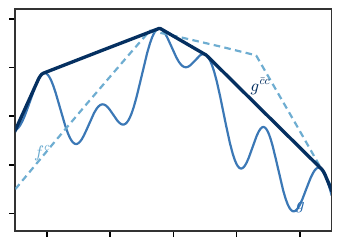} \\[-.1em]
\small source-side closure & \small target-side closure
\end{tabular}
\caption{Hard $c$-transforms for the bilinear cost $c(x,y)=-xy$. The source-side panel uses a reddish palette and starts from a sharply oscillatory potential $f$; the target-side panel uses a blueish palette and starts from a visually different potential $g$. Dark curves are the double-transform closures $f^{c\bar c}$ and $g^{\bar c c}$, which are concave majorants, while dashed lighter curves are the one-sided best responses $g^{\bar c}$ and $f^c$ after a harmless vertical gauge shift. The figure illustrates why exact alternating best responses are useful for dual certificates but do not give the smooth iterative dynamics later provided by entropic regularization.}
\index{c-transform}
\index{entropic!regularization}
\index{dual!certificate}
\label{fig:dual-alternating-c-transform-failure}
\end{figure}


\chapter[Semi-discrete and W1]{Semi-discrete and $\Wass_1$}
\index{semi-discrete!OT}
\label{sec-semidiscr-w1}

This chapter focuses on two computationally useful degeneracies of the dual problem. Semi-discrete OT turns a continuous-to-discrete map into finite-dimensional geometry, while $\Wass_1$ replaces convex potentials by Lipschitz functions and flow fields. The material connects computational geometry~\cite{AurenhammerHA98,Merigot11,merigot2013comparison} with the Kantorovich--Rubinstein and Beckmann formulations~\cite{kantorovich1958space,Beckmann52}.
\index{Beckmann!formulation}
\index{Lipschitz!function}
\index{convex!potential}
\index{semi-discrete!OT}
\index{dual!problem}

\section{Semi-dual}
\index{semi-dual}

The semi-dual eliminates one potential by an exact $c$-transform. It keeps concavity while removing explicit inequality constraints.
\index{c-transform}

Write the dual problem~\eqref{eq-dual-generic} as
\index{dual!problem}
\eq{
	\usup{f,g \in \Cc(\Xx) \times \Cc(\Yy)} \Ee(f,g)
}
where $\Ee(f,g)$ is the dual objective, with value $-\infty$ when the feasibility constraint fails. One can optimize out $g$ exactly and obtain the following semi-dual problem
\index{dual!problem}
\index{semi-dual}
\eql{\label{eq-semi-dual}
	\usup{f \in \Cc(\Xx)} \tilde\Ee(f) \eqdef \Ee(f,f^c) = \usup{g} \Ee(f,g) =
	\int_\Xx f \d \al + \int_\Yy f^c \d \be.
}
Partial maximization of a concave problem preserves concavity, so $\tilde \Ee$ is still concave. The major advantage of the semi-dual is that it removes the explicit inequality constraint, which allows the use of simpler optimization algorithms.
\index{semi-dual}

\section{Semi-discrete}
\index{semi-discrete!OT}

The semi-discrete case is the setting where dual potentials become weights of Laguerre cells. This gives both geometry and algorithms for quantization and density fitting.
\index{dual!potential}
\index{Laguerre cell}

\paragraph{Discrete target and Laguerre cells.}
\index{discrete!target}
\index{Laguerre cell}

A case of particular interest is when $\be = \sum_j \b_j \de_{y_j}$ is discrete (of course the same construction applies if $\al$ is discrete by exchanging the role of $\al,\be$).
One can adapt the definition of the $\bar c$ transform~\eqref{eq-c-transform} to this setting by restricting the minimization to the support $(y_j)_j$ of $\be$,
\index{c-transform}
\eql{\label{eq-disc-c-transfo}
	\foralls \gD \in \RR^m, \;
	\foralls x \in \Xx, \quad
	\gD^{\bar \c}(x) \eqdef \umin{j \in \range{m}} \c(x,y_j) - \gD_j.
}
This transform maps a vector $\gD$ to a continuous function $\gD^{\bar \c} \in \Cc(\Xx)$ under the same regularity assumptions on $c$ as in the continuous setting.
Note that this definition coincides with~\eqref{eq-c-transform} when the target space $\Y$ is restricted to the support of $\be$.
\index{c-transform}

Crucially, using the discrete $\bar c$-transform, when $\be$ is a discrete measure, yields a finite-dimensional optimization,
\eql{\label{eq-semi-dual-discr}
\index{semi-dual}
	\MK_\c(\al,\be) =
		\umax{\gD \in \RR^m}
			\Ee(\gD) \eqdef
			\int_\X \gD^{\bar \c}(x) \d\al(x) + \sum_j \gD_j \b_j.
}
The geometric object encoded by the dual weights is a weighted nearest-neighbor diagram: each point of the source space is assigned to the target atom that realizes the discrete $\bar c$-transform.

\begin{defn}[Laguerre cells and power diagrams]\label{def-laguerre-power-cells}
\index{Laguerre cell}
\index{power diagram}
\index{dual!weight}
	For sites $(y_j)_{j=1}^m$ and weights $\gD\in\RR^m$, the Laguerre cell associated with $y_j$ is
	\eql{\label{eq-laguerre-cells}
		\Laguerre_{j}(\gD) \eqdef \enscond{x \in \Xx}{ \foralls j' \neq j, \c(x,y_j) - \gD_j \leq \c(x,y_{j'}) - \gD_{j'} }.
	}
	The cells cover $\Xx$; after arbitrary tie-breaking on common boundaries, they induce a disjoint partition. When $c(x,y)=\norm{x-y}^2$, this Laguerre decomposition is also called a power diagram. If $\gD$ is constant, it reduces to the ordinary Voronoi diagram.
\end{defn}
For quadratic costs, varying the dual weights moves the walls between adjacent cells while keeping them parallel; this is the geometric mechanism by which the cell masses are adjusted.
\index{cell mass}
\index{cost!quadratic}
\index{dual!weight}
\begin{figure}[H]
\centering
\begin{tabular}{@{}ccc@{}}
\small zero weights &
\small intermediate weights &
\small balanced cells \\[-.15em]
\includegraphics[width=.30\linewidth]{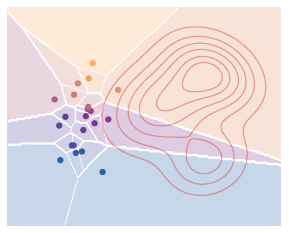} &
\index{Laguerre cell}
\includegraphics[width=.30\linewidth]{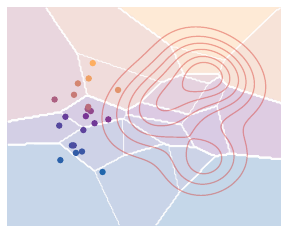} &
\includegraphics[width=.30\linewidth]{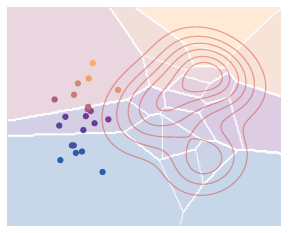}
\end{tabular}
\caption{Laguerre cells for semi-discrete quadratic transport.  The red contours show a continuous source density $\al$ given by a three-component Gaussian mixture on the right.  The twenty-one colored circular sites are the atoms of the discrete target $\be$ sampled from a compact Gaussian cloud on the left; each site color matches its Laguerre cell.  Starting from ordinary Voronoi cells, semi-dual weight updates deform the cells so that the $\al$-mass captured by each cell approaches the prescribed target mass.}
\index{Gaussian mixture}
\index{Laguerre cell}
\index{semi-discrete!OT}
\index{Voronoi cell}
\index{dual!weight}
\index{semi-dual}
\label{fig:semidiscrete-laguerre-cells}
\end{figure}

\paragraph{Mass balance.}
\index{mass!balance}

This allows one to conveniently rewrite the semi-dual energy as
\index{semi-dual}
\eql{\label{eq-semi-disc-energy}
	\Ee(\gD) = \sum_{j=1}^m \int_{\Laguerre_{j}(\gD)} \pa{ c(x,y_j) - \gD_j } \d\al(x) + \dotp{\gD}{\b}.
}
The following proposition provides a formula for the gradient of this concave function.

\begin{prop}[Gradient of the semi-discrete dual]
\index{semi-discrete!dual}
If $\al$ gives zero mass to the Laguerre cell boundaries, then $\Ee$ is differentiable at $\gD$ and
\index{zero mass}
\index{Laguerre cell}
\eq{
	\foralls j \in \range{m}, \quad
	\nabla\Ee(\gD)_j = \b_j - \int_{\Laguerre_{j}(\gD)} \d\al.
}
\end{prop}
\begin{proof}
	For $\al$-almost every $x$, the minimizing index in $\min_j c(x,y_j)-\gD_j$ is unique. If this index is $j(x)$, then the directional derivative in a direction $h\in\RR^m$ is
	\[
		\left.\frac{\d}{\d t}\right|_{t=0}
		\min_j \bigl(c(x,y_j)-(\gD_j+t h_j)\bigr)
		=
		-h_{j(x)}.
	\]
	Dominated convergence gives
	\[
		\d\Ee(\gD)[h]
		=
		-\sum_j h_j\int_{\Laguerre_j(\gD)}\d\al
		+\sum_j h_j\b_j,
	\]
	which is the announced gradient formula.
\end{proof}

The first-order optimality condition shows that solving the dual semi-discrete problem amounts to choosing the weights $\gD$ so that $\int_{\Laguerre_{j}(\gD)} \d\al = \b_j$, i.e. each cell captures the prescribed amount of mass. In this case, the optimal transport $T$ with $T_\sharp \al=\be$ is piecewise constant and maps $x \in \Laguerre_{j}(\gD)$ to $y_j$; for the quadratic cost, uniqueness follows from Brenier's theorem when $\al$ has a density.
\index{optimality!first-order}
\index{Brenier!theorem}
\index{cost!quadratic}
\index{semi-discrete!OT}

\begin{alg}[Semi-discrete Laguerre descent]\label{alg:semidiscrete-laguerre-descent}
\textbf{Input:} Source measure $\alpha$, target atoms $(y_j,\b_j)$, cost $c$, steps $\tau_k$.

\textbf{Output:} Semi-discrete dual weights $\gD$ and Laguerre cells.

\textbf{Initialize:} Set $\gD^{(0)}=0$.

\textbf{For} $k=0,1,\ldots$ \textbf{do}:
\begin{algblock}

\textbf{Compute cells:}
\(\Laguerre_j(\gD^{(k)}) = \enscond{x}{c(x,y_j)-\gD^{(k)}_j\leq c(x,y_\ell)-\gD^{(k)}_\ell\quad\forall \ell}.\)

\textbf{Compute masses:}
\(m_j^{(k)}=\int_{\Laguerre_j(\gD^{(k)})}\d\al .\)

\textbf{Update}
\(\gD^{(k+1)} = \gD^{(k)}+\tau_k\bigl(\b-m^{(k)}\bigr).\)

\textbf{If} $\max_j\abs{m_j^{(k)}-\b_j}\leq\mathrm{tol}$ \textbf{then}:
\begin{algblock}
\textbf{Return} $\gD^{(k+1)}$ and the cells.
\end{algblock}
\end{algblock}
\end{alg}

The quadratic power diagrams of Definition~\ref{def-laguerre-power-cells} have polyhedral cells and can be computed efficiently using computational geometry algorithms~\cite{aurenhammer1987power,AurenhammerHA98,Merigot11}. One classical construction lifts the sites to points $(y_j,\norm{y_j}^2-\gD_j)\in\RR^{d+1}$ and obtains the power diagram by projecting the lower envelope of their convex hull. In dimensions two and three, Chan's output-sensitive convex-hull algorithm~\cite{chan1996optimal} has complexity $O(m\log Q)$ for $m$ sites and $Q$ hull vertices.
\index{power diagram}

\paragraph{Stochastic optimization.}
\index{stochastic!optimization}

The semi-discrete formulation~\eqref{eq-semi-disc-energy} is useful because the objective is an expectation with respect to $\al$,
\index{semi-discrete!OT}
\eql{\label{eq-semi-disc-energy-entropy}
	\Ee(\gD)
	=
	\int_\X E(\gD,x)\d\al(x)
	=
	\EE_X(E(\gD,X)),
	\qquad
	E(\gD,x)\eqdef \gD^{\bar c}(x)+\dotp{\gD}{\b}.
}
Here $X\sim\al$. Away from cell boundaries, the stochastic gradient of the integrand is
\index{stochastic!gradient}
\[
	\nabla_{\gD}E(\gD,x)
	=
	\left(\b_j-\ones_{\Laguerre_j(\gD)}(x)\right)_{j=1}^m,
\]
which is an unbiased estimator of $\nabla\Ee(\gD)$ when cell boundaries have $\al$-measure zero. One can therefore maximize~\eqref{eq-semi-disc-energy} without first discretizing $\al$: the measure is used as a black box from which independent samples are drawn, a natural setup in high-dimensional statistics and machine learning.

Starting from $\gD^{(0)}=0$, stochastic gradient ascent draws $x_\ell\sim\al$ and performs
\index{stochastic!gradient}
\eql{\label{eq-sgd}
	\gD^{(\ell+1)}
	\eqdef
	\gD^{(\ell)}+\tau_\ell\nabla_{\gD}E(\gD^{(\ell)},x_\ell).
}
Equivalently, if
\[
	j_\ell\in\argmin_j\bigl(c(x_\ell,y_j)-\gD_j^{(\ell)}\bigr),
\]
then the coordinate update is
\[
	\gD_j^{(\ell+1)}
	=
	\gD_j^{(\ell)}
	+
	\tau_\ell\bigl(\b_j-\ones_{\{j=j_\ell\}}\bigr).
\]
The step size must decay so that the sampling noise averages out. A typical schedule is
\eql{\label{eq-step-size-sgd}
	\tau_\ell\eqdef \frac{\tau_0}{1+\ell/\ell_0},
}
where $\ell_0$ is a warmup scale. Under standard stochastic-approximation assumptions, one obtains the usual sublinear rate
\index{rate!sublinear}
\[
	\Ee(\gD^\star)-\EE\big(\Ee(\gD^{(\ell)})\big)
	=
	O(\ell^{-1/2}),
\]
where $\gD^\star$ is a maximizer and the expectation is over the i.i.d. samples. This stochastic viewpoint is one of the main algorithmic advantages of the semi-discrete formulation~\cite{Merigot11,genevay2016stochastic}.
\index{semi-discrete!OT}

\begin{alg}[Stochastic semi-discrete ascent]\label{alg:semidiscrete-stochastic-ascent}
\textbf{Input:} Source sampler $x\sim\alpha$, target atoms $(y_j,\b_j)$, steps $\tau_\ell$.

\textbf{Output:} Stochastic semi-discrete dual weights $\gD$.

\textbf{Initialize:} Set $\gD^{(0)}=0$.

\textbf{For} $\ell=0,1,\ldots$ \textbf{do}:
\begin{algblock}

\textbf{Draw} $x_\ell\sim\alpha$.

\textbf{Set} \(j_\ell=\min\argmin_j\bigl(c(x_\ell,y_j)-\gD_j^{(\ell)}\bigr)\).

\textbf{For} $j=1,\ldots,m$ \textbf{do}
\begin{algblock}
\(\gD_j^{(\ell+1)} = \gD_j^{(\ell)} + \tau_\ell\bigl(\b_j-\ones_{\{j=j_\ell\}}\bigr).\)
\end{algblock}
\end{algblock}
\algreturnskip
\textbf{Return} $\gD^{(\ell)}$ or its running average.
\end{alg}

\section{Optimal Quantization}
\index{quantization!optimal}
\label{sec-optimal-quantization}

Optimal quantization asks for the best discrete approximation of a measure by $m$ codepoints. It is the geometric core of vector quantization, compression and $k$-means clustering.

The optimal quantization problem for a measure $\al$ is
\index{quantization!optimal}
\eql{\label{eq-optimal-quantization}
	\Qq_m(\al)
	\eqdef
	\umin{Y=(y_j)_{j=1}^m,\ \b\in\simplex_m}
	\Wass_p\left(\al,\sum_{j=1}^m \b_j\de_{y_j}\right).
}
This problem is classical in approximation theory and information theory~\cite{graf2000foundationsquantization,Lloyd82}. The OT formulation emphasizes that one optimizes both the support locations $Y$ and, unless prescribed, the masses $\b$.

\begin{prop}[Quantization rate and curse of dimensionality]\label{prop-quantization-rate}
\index{quantization!rate}
\index{curse of dimensionality}
	Let $\Omega\subset\RR^d$ be a bounded Lipschitz domain and assume $\al=\rho\,\d x$ on $\Omega$, with $0<\rho_-\leq\rho\leq\rho_+<+\infty$. Then, for fixed $p\geq1$, there exist constants $0<c\leq C<+\infty$ such that
	\[
		c\,m^{-1/d}
		\leq
		\Qq_m(\al)
		\leq
		C\,m^{-1/d}.
	\]
\end{prop}
\begin{proof}
	For the upper bound, partition $\Omega$ into $m$ cells of diameter at most $Cm^{-1/d}$, up to boundary effects, and put one codepoint in each non-empty cell. Sending each point to the codepoint in its cell gives a transport distance bounded by $Cm^{-1/d}$.

	For the lower bound, fix any set $Y$ of $m$ codepoints and write $d_Y(x)=\min_j\norm{x-y_j}$. Since the density is bounded above, the mass of the $t$-neighborhood of $Y$ is at most $Cmt^d$. Choosing $t_0\simeq m^{-1/d}$ small enough gives $\al(\{d_Y>t\})\geq c$ for $0<t<t_0$. Hence
	\[
		\int d_Y(x)^p\d\al(x)
		=
		\int_0^{+\infty} p t^{p-1}\al(\{d_Y>t\})\d t
		\geq
		c t_0^p
		\simeq
		c m^{-p/d}.
	\]
	Taking the $p$-th root and minimizing over $Y$ proves the lower bound.
\end{proof}

This deterministic rate mirrors the empirical OT sample-complexity rate: both are governed by the spacing $m^{-1/d}$ of points in dimension $d$. Quantization is best-case and deterministic, while empirical OT is random, but both display the same curse of dimensionality.
\index{empirical!OT}
\index{curse of dimensionality}
\index{sample complexity}
For fixed codepoints $Y$, the problem is convex with respect to the weights $\b$. The dependence on $Y$ is non-convex and is generally computationally hard. The one-dimensional case is substantially simpler: monotonicity fixes the ordering of the cells, reducing the problem to interval endpoints and centroids; for the uniform law with the quadratic cost, the optimal centroids are equally spaced.
\index{cost!quadratic}

\begin{prop}[Free masses give Voronoi cells]\label{prop-free-masses-voronoi}
\index{mass!free}
\index{Voronoi cell}
	For the cost $c(x,y)=d(x,y)^p$, fix distinct codepoints $Y=(y_j)_{j=1}^m$. Duplicate codepoints can be merged beforehand. Minimizing over the weights $\b\in\simplex_m$ gives
	\[
		\min_{\b\in\simplex_m}
		\Wass_p^p\left(\al,\sum_j \b_j\de_{y_j}\right)
		=
		\int_\X \min_{1\leq j\leq m} c(x,y_j)\d\al(x).
	\]
	An optimal coupling is induced by sending each $x$ to a nearest codepoint. The corresponding cells are the Voronoi cells
\index{optimal coupling}
\index{Voronoi cell}
	\[
		\VV_j(Y)
		\eqdef
		\enscond{x}{\foralls j',\ c(x,y_j)\leq c(x,y_{j'})},
	\]
	up to arbitrary tie-breaking on common boundaries.
\end{prop}
\begin{proof}
	For any coupling between $\al$ and a measure supported on $Y$, the conditional destination of a point $x$ belongs to $Y$, hence its conditional cost is at least $\min_j c(x,y_j)$. Integrating gives the lower bound. Conversely, choose a measurable nearest-codepoint map $T_Y(x)\in\argmin_j c(x,y_j)$, breaking ties measurably, and set $\b_j=\al(T_Y^{-1}(y_j))$. Then $(T_Y)_\sharp\al=\sum_j \b_j\de_{y_j}$ and the induced transport reaches the displayed lower bound.
\end{proof}

Consequently, the quantization energy can be written in the nearest-centroid form
\index{quantization!energy}
\[
	\Qq_m(\al)^p
	=
	\min_Y \Ff(Y),
	\qquad
	\Ff(Y)
	\eqdef
	\int_\X \min_{1\leq j\leq m} c(x,y_j)\d\al(x).
\]
At a differentiability point of this energy, each local minimizer satisfies the centroid condition
\[
	y_j
	\in
	\uargmin{y}
	\int_{\VV_j(Y)} c(x,y)\d\al(x).
\]
For the squared Euclidean cost, this becomes the fixed-point equation
\[
	y_j
	=
	\frac{\int_{\VV_j(Y)} x\d\al(x)}
	{\int_{\VV_j(Y)} \d\al}.
\]
Lloyd's algorithm, also known as the $k$-means algorithm, iterates this fixed point: assign points to nearest sites, then replace each site by the centroid of its cell~\cite{Lloyd82}. With standard tie-breaking, the objective decreases at each step. Since the problem is non-convex in $Y$, the iterates generally converge only to a local minimum. Good seeding matters; for the squared Euclidean objective, $k$-means++ gives a logarithmic approximation guarantee in expectation~\cite{ArthurVassilvitskii2007}.
\index{Lloyd algorithm}

\begin{alg}[Lloyd quantization]\label{alg:lloyd-quantization}
\textbf{Input:} Source measure $\alpha$, initial codepoints $Y^{(0)}=(y_j^{(0)})_{j=1}^m$, squared Euclidean cost, tolerance $\mathrm{tol}$.

\textbf{Output:} Codepoints $Y=(y_j)_{j=1}^m$.

\textbf{Initialize:} Set \(d_0=+\infty\) and \(k=0\).

\textbf{While} \(d_k>\mathrm{tol}\) \textbf{do}:
\begin{algblock}

\textbf{Set} \(k\leftarrow k+1\).

\textbf{Compute Voronoi cells:}
\(\VV_j(Y^{(k-1)}) = \enscond{x}{c(x,y_j^{(k-1)})\leq c(x,y_\ell^{(k-1)})\quad\forall \ell}.\)

\textbf{For} each nonempty cell $\VV_j$ \textbf{do}
\begin{algblock}
\(y_j^{(k)} = \frac{\int_{\VV_j(Y^{(k-1)})}x\,\d\al(x)} {\int_{\VV_j(Y^{(k-1)})}\d\al(x)}.\)
\end{algblock}
\textbf{For} each empty cell $\VV_j$ \textbf{do}:
\begin{algblock}

\textbf{Set} $y_j^{(k)}=y_j^{(k-1)}$.

\end{algblock}

\textbf{Set} \(d_k=\max_j\norm{y_j^{(k)}-y_j^{(k-1)}}\).
\end{algblock}
\algreturnskip
\textbf{Return} $Y^{(k)}$.
\end{alg}

\begin{figure}[H]
\centering
\begin{tabular}{@{}ccc@{}}
\small initialization &
\small after $6$ iterations &
\small after $36$ iterations \\[-.15em]
\includegraphics[width=.30\linewidth]{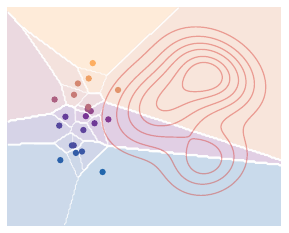} &
\includegraphics[width=.30\linewidth]{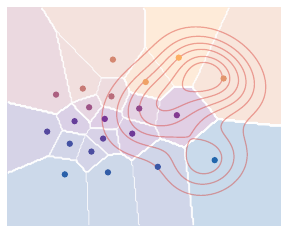} &
\includegraphics[width=.30\linewidth]{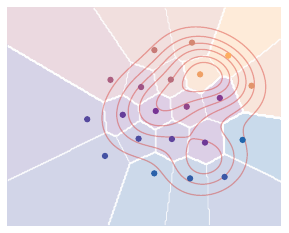}
\end{tabular}
\caption{Lloyd quantization for the same continuous density and twenty-one initial sites as Figure~\ref{fig:semidiscrete-laguerre-cells}.  The red contours show the density $\al$, while the colored disks are the current codepoints and have the same colors as their Voronoi cells.  The iterations move the initially left-located sites toward the high-density region and reshape the cells according to centroidal Voronoi geometry.}
\index{centroidal Voronoi}
\index{Laguerre cell}
\index{Voronoi cell}
\label{fig:semidiscrete-lloyd-quantization}
\end{figure}

\section[W1]{Wasserstein-1 norm}
\index{norm!Wasserstein-1}
\label{sec-W1}

The $\Wass_1$ distance has an especially transparent dual: the admissible potentials are exactly $1$-Lipschitz test functions. This makes $\Wass_1$ the meeting point between transport, PDE formulations and weak norms on signed measures.
\index{signed!measure}

\paragraph{$c$-transform for $\Wass_1$.}
\index{c-transform}

Assume that $d$ is a distance on $\X=\Y$ and take the ground cost $c(x,y)=d(x,y)$. We denote the Lipschitz constant of $f\in\Cc(\X)$ by
\index{ground cost}
\begin{defn}[Lipschitz constant]\label{def-lipschitz-constant}
\index{Lipschitz!constant}
	For a function $f:\X\to\RR$ on a metric space $(\X,d)$, its Lipschitz constant is
	\eql{\label{eq-lip-constant}
		\Lip(f)
		\eqdef
		\sup\enscond{\frac{|f(x)-f(y)|}{d(x,y)}}{x\neq y}.
	}
	The function is $1$-Lipschitz when $\Lip(f)\leq1$.
\end{defn}

\begin{prop}[$c$-transforms and $1$-Lipschitz functions]\label{prop-w1-c-transform-lipschitz}
\index{c-transform}
\index{Lipschitz!function}
	Suppose $\X=\Y$ and $c(x,y)=d(x,y)$. Then there exists $g$ such that $f=g^c$ if and only if $\Lip(f)\leq1$. Furthermore, if $\Lip(f)\leq1$, then $f^c=-f$.
\end{prop}
\begin{proof}
	First suppose $f=g^c$ for some $g$. For $x,y\in\X$,
	\[
		|f(x)-f(y)|
		=
		\left|\inf_z[d(x,z)-g(z)]-\inf_z[d(y,z)-g(z)]\right|
		\leq
		\sup_z |d(x,z)-d(y,z)|
		\leq d(x,y),
	\]
	where the last inequality is the reverse triangle inequality. Thus $\Lip(f)\leq1$.
\index{triangle inequality}

	If $\Lip(f)\leq1$, then $f(x)\leq f(y)+d(x,y)$, so $d(x,y)-f(x)\geq -f(y)$ for all $x$ and hence $f^c(y)\geq -f(y)$. Taking $x=y$ gives $f^c(y)\leq -f(y)$. Therefore $f^c=-f$. Applying the same property to $-f$ gives $(-f)^c=f$, so every $1$-Lipschitz function is $c$-concave.
\index{Lipschitz!function}
\end{proof}

Using the alternating $c$-transform scheme~\eqref{eq-iter-c-trans}, one can replace the dual pair by $(f,-f)$ with $\Lip(f)\leq1$. The Kantorovich dual therefore becomes the Kantorovich--Rubinstein formula
\index{Kantorovich-Rubinstein!formula}
\index{c-transform}
\eql{\label{eq-w1-metric}
	\Wass_1(\al,\be)
	=
	\umax{f}
	\enscond{\int_\X f\d(\al-\be)}{\Lip(f)\leq1}.
}
This expression depends only on the signed measure $\al-\be$. It therefore extends to finite signed measures of total mass zero and defines the Kantorovich--Rubinstein norm on that space~\cite{kantorovich1958space}.
\index{signed!measure}

For a discrete signed measure $\al-\be=\sum_k r_k\de_{z_k}$ with $\sum_k r_k=0$,~\eqref{eq-w1-metric} becomes the finite-dimensional linear program
\index{linear programming!finite-dimensional}
\eql{\label{eq-w1-discr}
	\Wass_1(\al,\be)
	=
	\umax{(f_k)_k}
	\enscond{\sum_k f_k r_k}{\foralls k,\ell,\ |f_k-f_\ell|\leq d(z_k,z_\ell)}.
}
This linear program can be solved by generic interior-point or first-order methods; structured graph versions admit the flow formulations described below.
When $d(x,y)=|x-y|$ on $\RR$, ordering the support points $z_1\leq z_2\leq\cdots$ reduces the constraints to neighboring pairs,
\[
	\Wass_1(\al,\be)
	=
	\umax{(f_k)_k}
	\enscond{\sum_k f_k r_k}{\foralls k,\ |f_{k+1}-f_k|\leq z_{k+1}-z_k}.
\]
In one dimension this is equivalent to the closed-form cumulative formula introduced earlier.

\paragraph{$\Wass_1$ on Euclidean spaces.}

In the special case of Euclidean spaces $\X=\Y=\RR^\dim$, using $\c(x,y) = \norm{x-y}$, the global Lipschitz constraint in the Kantorovich--Rubinstein formula can be made local as a uniform bound on the gradient of $f$,
\index{Kantorovich-Rubinstein!formula}
\eql{\label{eq-w1-cont}
	\Wass_1(\al,\be) =
	\usup{f} \enscond{ \int_{\RR^\dim} \f (\d\al-\d\be) }{ \norm{\nabla \f}_\infty \leq 1 }.
}
Here the constraint $\norm{\nabla \f}_\infty \leq 1$ signifies that the norm of the gradient of $\f$ at any point $x$ is upper bounded by $1$, $\norm{\nabla \f(x)}_2 \leq 1$ for any $x$.

Considering the dual problem to~\eqref{eq-w1-cont}, denoting $\xi \eqdef \al-\be$, and using the equivalent form
\index{dual!problem}
\eq{
	-\iota_{\norm{\cdot}_{\RR^d} \leq 1}(u) = \uinf{v} \dotp{u}{v} + \norm{v}_{\RR^d},
}
one has a maximization on flow vector fields $\flow :  \RR^d \rightarrow \RR^d$
\begin{align*}
	\Wass_1(\al,\be) &=
	\usup{f} \uinf{\flow(x) \in \RR^d} \int_{\RR^\dim} \f \d\xi + \int \dotp{\nabla f(x)}{\flow(x)} \d x + \int \norm{\flow(x)}_{\RR^d} \d x \\
	&=
	\uinf{\flow(x) \in \RR^d} \int \norm{\flow(x)} \d x
		+ \usup{f} \int f(x) ( \d \xi - \diverg(\flow) \d x)
\end{align*}
one obtains an optimization problem under a fixed divergence constraint
\eql{\label{eq-w1-cont-div}
	\Wass_1(\al,\be) =
	\uinf{\flow} \enscond{ \int_{\RR^\dim} \norm{\flow(x)}_{\RR^d} \d x }{  \diverg(\flow)=\al-\be },
}
which is often called the Beckmann formulation~\cite{Beckmann52}.
\index{Beckmann!formulation}
Here the vectorial function $\flow(x) \in \RR^\dim$ can be interpreted as a flow field, describing locally the movement of mass. Outside the support of the two input measures, $\diverg(\flow)=0$, which is the conservation of mass constraint.
\index{support}
Once properly discretized using finite elements, Problems~\eqref{eq-w1-cont} and~\eqref{eq-w1-cont-div} become a nonsmooth convex optimization problem.

The previous formulations~\eqref{eq-w1-cont} and~\eqref{eq-w1-cont-div} of $\Wass_1$ can be generalized to the setting where $\X$ is a Riemannian manifold, i.e. $\c(x,y)=\dist(x,y)$ where $\dist$ is the associated geodesic distance (and then for smooth manifolds, the gradient and divergence should be understood as differential operators on manifolds).

\begin{defn}[Graph geodesic distance]\label{def-graph-geodesic-distance}
\index{graph!geodesic distance}
	Let $G=(V,E)$ be a connected finite graph with positive edge lengths $(\ell_e)_{e\in E}$. The graph geodesic distance between two vertices is
	\[
		d_G(i,j)=\min_{\gamma:i\leadsto j}\sum_{e\in\gamma}\ell_e.
	\]
	The minimum is over all paths $\gamma$ joining $i$ to $j$.
\end{defn}
This graph distance turns $\Wass_1$ into a finite-dimensional flow problem.

\begin{prop}[$\Wass_1$ and Beckmann flow on a graph]\label{prop-graph-w1-beckmann}
\index{Beckmann!flow}
	Let $G=(V,E)$ be a connected finite graph with positive edge lengths $(\ell_e)_{e\in E}$ and graph geodesic distance $d_G$.
	For two probability vectors $\a,\b$ on $V$, set $r=\a-\b$ and orient each edge $e=(i,j)$. If
	\[
		(\nabla_G f)_e=f_j-f_i,\qquad
		\operatorname{div}_G=-\nabla_G^*
	\]
	are the finite-difference gradient and its negative adjoint, then a positive flow on the oriented edge $i\to j$ has positive divergence at $i$ and negative divergence at $j$. With this convention,
	\[
		\Wass_{1,G}(\a,\b)
		=
		\max_f \enscond{\sum_{i\in V} f_i r_i}{|f_i-f_j|\leq \ell_e\quad\forall e=(i,j)}
		=
		\min_m \enscond{\sum_{e\in E}\ell_e |m_e|}{\operatorname{div}_G m=r}.
	\]
	The vector $m_e$ is an oriented edge flow, and the constraint $\operatorname{div}_G m=r$ is conservation of mass at each vertex.
\end{prop}

\begin{proof}
	The edge constraint $|f_i-f_j|\leq\ell_e$ implies, by summing along paths, that $|f_i-f_j|\leq d_G(i,j)$ for all vertices. Conversely, any $1$-Lipschitz function for $d_G$ satisfies the edge constraints because each edge is a path of length $\ell_e$. The first equality is therefore the Kantorovich--Rubinstein formula on the metric space $(V,d_G)$.
\index{Kantorovich-Rubinstein!formula}
\index{Lipschitz!function}

	For the second equality, write the graph Beckmann problem and dualize its equality constraint with a potential $f$:
\index{equality constraint}
	\[
		\inf_m \sum_e\ell_e|m_e|
		+\sup_f \sum_i f_i(r_i-(\operatorname{div}_G m)_i).
	\]
	Using $\operatorname{div}_G=-\nabla_G^*$, the coupling term is $\sum_e m_e(\nabla_G f)_e$. The minimization over each scalar flow $m_e$ is finite exactly when $|(\nabla_G f)_e|\leq\ell_e$, and is then equal to zero. The dual problem is therefore precisely the graph Lipschitz dual above. Strong duality holds because this is a finite-dimensional linear program with a non-empty feasible set: connectedness and $\sum_i r_i=0$ allow the signed surplus to be routed along paths. This proves the graph Beckmann formula.
\index{duality!strong}
\index{linear programming!finite-dimensional}
\index{dual!problem}
\end{proof}

\begin{figure}[H]
\centering
\begin{tabular}{@{}cc@{}}
\includegraphics[width=.46\linewidth]{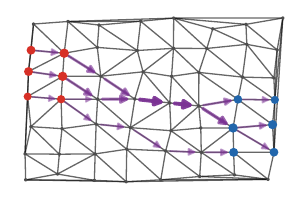}
&
\includegraphics[width=.46\linewidth]{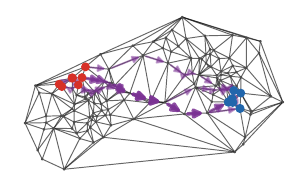}
\\[-.5mm]
{\small quasi-regular Delaunay graph}
&
{\small denser nonuniform Delaunay graph}
\end{tabular}
\index{graph!transport}
\caption{Graph Beckmann formulation of $\Wass_1$ on two Delaunay graphs $G=(V,E)$.  The left panel uses a quasi-regular vertex set, while the right panel uses a denser nonuniform vertex set sampled from a two-component Gaussian mixture.  Red and blue disks encode the positive and negative parts of $r=\alpha-\beta$, localized on six vertices each.  The darker graph edges show the triangulation, while the violet arrows display the optimal signed edge flow $m$: orientation gives the sign, width is proportional to $\sqrt{|m_e|}$, and the flow satisfies $\operatorname{div}_G m=r$.}
\index{signed!positive part}
\index{signed!negative part}
\index{Beckmann!formulation}
\label{fig:w1-graph-transport-flow}
\end{figure}

\begin{rem}[Sparse LP and network simplex]\label{rem-graph-w1-network-simplex}
\index{linear programming!sparse}
\index{simplex network}
	Let $N=|V|$ and $M=|E|$. The graph Beckmann problem is a linear program, but it is much smaller than the dense Kantorovich LP on the same vertex set. Indeed, writing $m=m^+-m^-$ gives
	\[
		\min_{m^+,m^-\geq0}\sum_{e\in E}\ell_e(m^+_e+m^-_e)
		\quad\text{subject to}\quad
		\operatorname{div}_G(m^+-m^-)=r .
	\]
	This formulation has $2M$ nonnegative variables and $N-1$ independent balance constraints, whereas the standard transport LP between two measures on $V$ has $N^2$ coupling variables and $2N-1$ independent marginal constraints. For sparse geometric graphs, such as planar Delaunay graphs where typically $M=O(N)$, the graph formulation is therefore linear-size rather than quadratic-size.
\index{Delaunay!graph}
\index{marginal!constraint}

	The same LP is a minimum-cost transshipment problem. Replace each undirected edge $\{i,j\}$ by the two directed arcs $i\to j$ and $j\to i$, both with cost $\ell_e$, and impose the node balances $\sum_{j}u_{ij}-\sum_j u_{ji}=r_i$. This is exactly the setting of the network simplex method: a basis is a spanning tree, a pivot adds one non-tree arc, creates a unique cycle, sends flow along this cycle, and updates the node potentials and reduced costs~\cite{bertsekas1988dual,Orlin1997}. A basic implementation needs $O(M)$ work to price all arcs and $O(N)$ work to update the tree at each pivot, hence $O(PM)$ arithmetic operations for $P$ pivots on a sparse graph. The pivot count $P$ depends on the rule and can be large in worst-case simplex analyses, but network-simplex variants and general minimum-cost-flow algorithms give polynomial guarantees in $N$ and $M$; in practice, this edge-based formulation is often far cheaper than solving the dense $N^2$-variable transport LP.
\index{flow!minimum-cost}
\index{spanning tree}
\index{cost!reduced}
\end{rem}

\begin{alg}[Graph Beckmann network-simplex pivot]\label{alg:graph-beckmann-network-simplex}
\index{Beckmann!flow}
\index{simplex network}
\textbf{Input:} Graph $G=(V,E)$, edge lengths $\ell_e$, node balances $r_i$ with $\sum_ir_i=0$.

\textbf{Output:} Minimum-cost graph flow $u$.

\textbf{Replace} each undirected edge by two directed arcs.

\textbf{Impose balances:}
\(\sum_j u_{ij}-\sum_j u_{ji}=r_i .\)

\textbf{Initialize:} Add artificial root arcs and compute a feasible tree flow on a spanning tree $T$ with node potentials.

\textbf{While} \(\min_{e\notin T}\bar c_e<0\) \textbf{do}:
\begin{algblock}

\textbf{Set} entering arc \(e\) to the first minimizer of \(\bar c_a\) over \(a\notin T\) in the prescribed arc order.

\textbf{Add} it to the tree.

\textbf{Set} $\mathcal C=$ unique induced cycle, oriented in the direction of the entering arc \(e\).

\textbf{Set} \(\theta=\min\{u_a:\ a\in\mathcal C^-\}\), where \(\mathcal C^-\) are the arcs opposed to the cycle orientation.

\textbf{Update} \(u_a\leftarrow u_a+\theta\) for \(a\in\mathcal C^+\) and \(u_a\leftarrow u_a-\theta\) for \(a\in\mathcal C^-\).

\textbf{Remove} the first arc in \(\mathcal C^-\) attaining the minimum \(\theta\), using the prescribed cycle order.

\textbf{Update} the tree, potentials, and reduced costs.

\end{algblock}
\algreturnskip
\textbf{Return} $u$.
\end{alg}

This graph formulation is the transshipment version of $\Wass_1$. It is the natural discrete analogue of~\eqref{eq-w1-cont-div}: gradients are edge differences, divergences are incidence-matrix balances, and geodesic distance is the shortest-path length. It can be solved by min-cost flow methods on sparse graphs, and entropic or KL-projection variants lead to flow-Sinkhorn algorithms for graph $\Wass_1$~\cite{Beckmann52,peyre2026robust}.
\index{integral probability metric}
\index{Sinkhorn!algorithm}
\index{matrix!incidence}
\index{KL!projection}

\chapter{Divergences and Dual Norms}

\index{dual!norm}
\label{sec-divergences-dual-norms}

This chapter compares OT with divergence-based and adversarial ways of measuring discrepancy. The main stake is topological: $\phi$-divergences are cheap but strong, while dual norms and GAN objectives can be weak enough to compare singular measures. The discussion connects classical information divergences~\cite{ciszar1967information,ali1966general} with modern integral probability metrics and generative modeling~\cite{sriperumbudur2009integral,GAN,WassersteinGAN}.
\index{generative model}
\index{integral probability metric}
\index{phi-divergence}
\index{dual!norm}

\section{Dual norms (Integral Probability Metrics)}
\index{dual!norm}
\index{integral probability metric}
\label{sec-dual-norms}

Dual norms generalize the $\Wass_1$ test-function principle. They are useful in statistics because they compare distributions by restricting the discriminator class.
\index{discriminator class}
\index{dual!norm}

\paragraph{Integral probability metrics.}
\index{integral probability metric}

Formulation~\eqref{eq-w1-cont} is a special case of a dual norm. This viewpoint designs ``weak'' discrepancies by testing signed differences of measures against a controlled class of functions.
\index{measurable function}
\index{dual!norm}

\begin{defn}[Dual norm and integral probability metric]\label{def-dual-norm-ipm}
\index{dual!norm}
\index{integral probability metric}
	For a symmetric convex set $B$ of measurable functions, define on signed measures $\xi$
	\eql{\label{eq-dual-norm-cont}
		\norm{\xi}_B \eqdef
		\usup{\f} \enscond{ \int_{\X} \f(x) \d\xi(x) }{ \f \in B}.
	}
	When this quantity is applied to $\al-\be$ for probability measures, it is often called an integral probability metric.
\end{defn}
The choice of the test-function class $B$ determines both the topology and the statistical behavior of the discrepancy; see~\cite{sriperumbudur2012empirical,sriperumbudur2009integral,sriperumbudur2008injective}.
\index{integral probability metric}
\index{dual!norm}

\begin{example}[Total variation]
\index{total variation}
As recalled in Definition~\ref{defn-total-variation} and Proposition~\ref{prop-tv-dual-measure}, total variation is the dual norm associated with the unit ball of continuous functions
\eq{
	B = \enscond{f \in \Cc(\X)}{\norm{f}_\infty \leq 1}.
}
Total variation is the canonical nontrivial example of a discrepancy that is both a $\phi$-divergence and a dual norm; see~\cite{sriperumbudur2009integral}.
\index{total variation}
\index{phi-divergence}
\index{dual!norm}
\end{example}

\begin{example}[$\Wass_1$ norm]
$\Wass_1$, as defined in~\eqref{eq-w1-cont}, is the dual norm~\eqref{eq-dual-norm-cont} associated with
\index{dual!norm}
\eq{
	B = \enscond{f}{\Lip(f) \leq 1}
}
the set of 1-Lipschitz functions.
\index{Lipschitz!function}
\end{example}

\begin{example}[Flat norm and Dudley metric]
\index{norm!flat}
\index{Dudley metric}
If the set $B$ is bounded and separates measures, then $\norm{\cdot}_B$ is a norm on the whole space $\Mm(\Xx)$ of finite measures.
\index{finite measure}
	This is not the case of $\Wass_1$, which is only finite on signed measures $\xi$ such that $\int_\X \d\xi=0$; otherwise $\norm{\xi}_B=+\infty$ because constants belong to the Lipschitz ball.
This is remedied by imposing a bound on the value of the potential $\f$, which leads for instance to the flat norm,
\index{norm!flat}

\eql{\label{eq-set-flatnorm}
	B=\enscond{f}{\Lip(f) \leq 1 \qandq \norm{\f}_\infty \leq 1}.
}
On compact metric spaces, it metrizes weak convergence on the whole space $\Mm(\X)$ of finite measures.
\index{finite measure}
\index{weak!convergence}
The finite-dimensional version is obtained from the usual $\Wass_1$ dual linear program by adding the box constraints $\abs{\fD_k}\leq1$.
The flat norm is sometimes called the ``Kantorovich--Rubinstein'' norm~\cite{hanin1992kantorovich} and has been used as a fidelity term for inverse problems in imaging~\cite{lellmann2014imaging}.
\index{norm!flat}
The flat norm is similar to the Dudley metric, which uses
\index{Dudley metric}
\eq{\label{eq-set-dudley}
	B=\enscond{f}{\norm{\nabla \f}_\infty + \norm{\f}_\infty \leq 1}.
}
\end{example}

\begin{figure}[ht]
\centering
\setlength{\tabcolsep}{2pt}
\begin{tabular}{@{}ccc@{}}
\small $\Wass_1$ witness & \small MMD witness & \small total variation witness \\[-.15em]
\index{total variation}
\includegraphics[width=.31\linewidth]{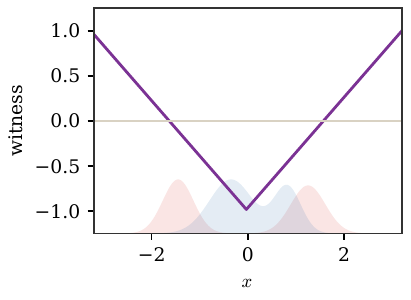} &
\includegraphics[width=.31\linewidth]{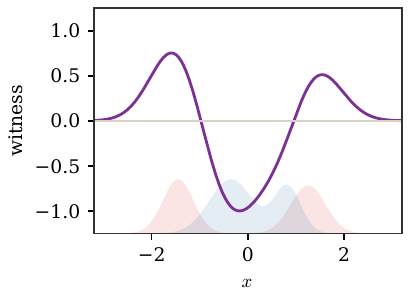} &
\includegraphics[width=.31\linewidth]{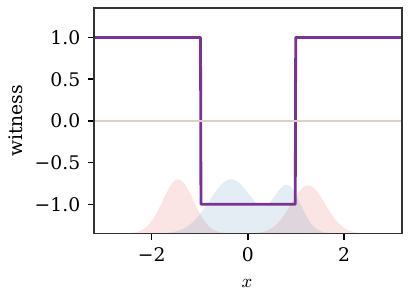}
\end{tabular}
\caption{Dual witnesses for integral probability metrics. The red and blue curves are two one-dimensional probability densities and the violet curve is a normalized optimal dual witness $f^\star_{\alpha,\beta}$ for the IPM variational problem~\eqref{eq-dual-norm-cont}. $\Wass_1$ restricts the slope through Kantorovich--Rubinstein duality~\eqref{eq-w1-cont}, MMD restricts the RKHS norm as in Proposition~\ref{prop-kernel-rkhs-dual}, and total variation can saturate pointwise and therefore reacts sharply to signed density differences.}
\index{Kantorovich-Rubinstein!duality}
\index{integral probability metric}
\index{total variation}
\index{dual!norm}
\index{RKHS}
\label{fig:dualnorms-ipm-witnesses}
\end{figure}

The following proposition gives a useful compact-space criterion. The dual ball should be rich enough to approximate continuous observables, but compact enough for weak convergence to imply uniform convergence over the discriminator class.
\index{discriminator class}
\index{weak!convergence}

\begin{prop}[Metrization by dual norms]\label{prop-dual-norm-metrization}
\index{dual!norm}
	Assume that $\Xx$ is compact, that $B=-B$, and that the measures considered are probability measures.
\index{probability measure}
	\begin{enumerate}
		\item If every function in $\Cc(\Xx)$ can be uniformly approximated by elements of $\Span(B)$, then $\norm{\al_n-\al}_B \rightarrow 0$ implies $\al_n \rightharpoonup \al$.
		\item If $B \subset \Cc(\Xx)$ is compact for $\norm{\cdot}_\infty$, then $\al_n \rightharpoonup \al$ implies $\norm{\al_n-\al}_B \rightarrow 0$.
	\end{enumerate}
\end{prop}
\begin{proof}
	For the first implication, $\norm{\al_n-\al}_B\to0$ and the symmetry of $B$ imply
	\[
		\abs{\dotp{f}{\al_n-\al}}\leq \norm{\al_n-\al}_B
		\qforq f\in B.
	\]
	By linearity, integrals converge for every $h\in\Span(B)$. Let $u\in\Cc(\Xx)$ and choose $h\in\Span(B)$ with $\norm{u-h}_\infty\leq\eta$. Since $\al_n$ and $\al$ are probabilities,
	\[
		\abs{\dotp{u}{\al_n-\al}}
		\leq
		\abs{\dotp{h}{\al_n-\al}}+2\eta.
	\]
	Taking the limsup as $n\to\infty$ and then letting $\eta\to0$ gives $\dotp{u}{\al_n}\to\dotp{u}{\al}$ for all $u\in\Cc(\Xx)$, which is weak convergence.
\index{weak!convergence}

	For the second implication, assume $\al_n\rightharpoonup\al$ and choose a subsequence $(\al_{n_k})_k$ such that $\norm{\al_{n_k}-\al}_B\to\limsup_n\norm{\al_n-\al}_B$. Since $B$ is compact and the map $f\mapsto\dotp{f}{\al_{n_k}-\al}$ is continuous on $B$, the supremum is attained by some $f_{n_k}\in B$. Extracting a further subsequence if needed, $f_{n_k}\to f$ uniformly for some $f\in B$. Then
	\[
		\dotp{f_{n_k}}{\al_{n_k}-\al}
		=
		\dotp{f}{\al_{n_k}-\al}
		+\dotp{f_{n_k}-f}{\al_{n_k}}
		-\dotp{f_{n_k}-f}{\al}.
	\]
	The first term tends to zero by weak convergence and the last two by uniform convergence. Hence every limsup subsequence has limit zero, proving $\norm{\al_n-\al}_B\to0$.
\index{weak!convergence}
\end{proof}

\begin{cor}[Wasserstein metrizes weak convergence]\label{cor-topol-wass}
	On a compact metric space, $\Wass_p$ metrizes weak convergence on probability measures for every $p\geq1$.
\index{probability measure}
\index{weak!convergence}
\end{cor}
\begin{proof}
	For $p=1$, take $B=\enscond{f}{\Lip(f)\leq1}$. The span of $B$ contains all Lipschitz functions, and Lipschitz functions are dense in $\Cc(\Xx)$ on compact metric spaces. This gives the implication $\Wass_1(\al_n,\al)\to0\Rightarrow \al_n\rightharpoonup\al$ by Proposition~\ref{prop-dual-norm-metrization}.
\index{Lipschitz!function}
\index{dual!norm}

	Conversely, constants do not change the pairing with $\al_n-\al$. Fix $x_0\in\Xx$ and normalize potentials by $f(x_0)=0$. The normalized unit Lipschitz ball is uniformly bounded by $\mathrm{diam}(\Xx)$ and equicontinuous, hence compact in $\norm{\cdot}_\infty$ by Arzel\`a--Ascoli. Proposition~\ref{prop-dual-norm-metrization} gives $\Wass_1(\al_n,\al)\to0$. Proposition~\ref{prop-comp-wass-p} shows that all $\Wass_p$ distances induce the same topology on a compact space, so the result follows for every $p\geq1$.
\end{proof}

\section{Dual RKHS Norms and Maximum Mean Discrepancies}
\index{maximum mean discrepancy}
\index{RKHS}
\label{sec-rkhs-mmd}

Kernel methods turn probability measures into mean elements of a reproducing kernel Hilbert space (RKHS). The resulting Hilbertian dual norms are quadratic discrepancies, handled with Euclidean geometry while retaining a weak test-function interpretation. We first recall the positivity assumptions under which the quadratic form on signed measures is nonnegative.
\index{probability measure}
\index{signed!measure}
\index{dual!norm}

\begin{defn}[Positive and conditionally positive kernels]\label{def-positive-kernels}
\index{kernel!positive}
\index{kernel!conditionally positive}
	A symmetric function $\Krkhs:\Xx\times\Xx\to\RR$ is positive definite if for every $n\geq1$, every $x_1,\ldots,x_n\in\Xx$ and every $r\in\RR^n$,
\index{RKHS}
	\eql{\label{eq-dual-kern}
		\sum_{i,j=1}^n r_i r_j \Krkhs(x_i,x_j)\geq0.
	}
	It is conditionally positive definite if the same inequality is required only for zero-sum vectors, $\dotp{r}{\ones_n}=0$.
\end{defn}

The conditional version is the right notion for probability distances, because one applies the quadratic form to signed measures $\xi=\al-\be$ of total mass zero. Adding a term of the form $a(x)+a(y)$ to the kernel does not change $\iint \Krkhs(x,y)\d\xi(x)\d\xi(y)$ on such measures, and many natural distance kernels are only conditionally positive definite.
\index{RKHS}
\index{signed!measure}

\begin{example}[Riesz, energy and Mat\'ern-type kernels]
\index{kernel!Riesz}
\index{kernel!energy}
\index{kernel!Matern}
	On $\RR^d$, translation-invariant kernels are most transparent in Fourier variables. The Riesz family associated with $(-\Delta)^{-s}$ has multiplier $\norm{\om}^{-2s}$ and defines a nonnegative quadratic form on zero-mass measures for which the low-frequency singularity is integrable; this is the kernel counterpart of classical Riesz potentials~\cite{berg84harmonic}. The energy distance corresponds to the conditionally positive kernel $\Krkhs(x,y)=-\norm{x-y}$, whose Fourier multiplier is proportional to $\norm{\om}^{-(d+1)}$; for $\xi=\al-\be$,
\index{kernel!translation-invariant}
\index{Riesz potential}
\index{Fourier multiplier}
\index{kernel!conditionally positive}
\index{RKHS}
\index{energy!distance}
\index{kernel!positive}
	\[
		-\iint \norm{x-y}\d\xi(x)\d\xi(y)
	\]
	is the squared energy distance up to a dimensional constant~\cite{schoenberg38,szekely2004testing}.
\index{energy!distance}

	Shifted kernels replace $(-\Delta)^{-s}$ by $(-\Delta+\lambda I)^{-s}$ with $\lambda>0$. Their Fourier multiplier $(\norm{\om}^2+\lambda)^{-s}$ is bounded at the origin, hence the kernel is positive definite without imposing zero mass. These are Mat\'ern kernels; in closed form they are radial and involve a modified Bessel function~\cite{wendland2005scattered}. The Laplacian kernel $e^{-\norm{x-y}/\sigma}$ is a low-smoothness Mat\'ern example, while the Gaussian kernel $e^{-\norm{x-y}^2/(2\sigma^2)}$ is the infinite-smoothness limit after the usual rescaling of the Mat\'ern smoothness parameter.
\index{zero mass}
\index{Bessel function}
\index{kernel!Laplacian}
\index{Fourier multiplier}
\index{kernel!Gaussian}
\index{kernel!Matern}
\end{example}

\begin{defn}[Kernel norm and MMD]\label{def-kernel-mmd-norm}
\index{kernel!norm}
\index{maximum mean discrepancy}
	Let $\Krkhs$ be positive definite. More generally, let $\Krkhs$ be conditionally positive definite and restrict attention to signed measures of total mass zero. For a signed measure $\xi$ with finite kernel energy, the associated norm is
\index{kernel!energy}
\index{RKHS}
\index{signed!measure}
	\eql{\label{eq-kernel-dual}
		\norm{\xi}^2_{\Krkhs}
		\eqdef
		\iint_{\Xx\times\Xx}\Krkhs(x,y)\d\xi(x)\d\xi(y).
	}
	For two probability measures, the maximum mean discrepancy associated with $\Krkhs$ is
\index{probability measure}
\index{maximum mean discrepancy}
	\[
		\MMD_{\Krkhs}(\al,\be)\eqdef\norm{\al-\be}_{\Krkhs}.
	\]
\end{defn}

These norms are usually called maximum mean discrepancies in statistics and machine learning~\cite{gretton2012kernel,muandet2017kernel}, and kernel norms in shape analysis~\cite{Hofmann2008}. If $X,X'$ are independent with law $\al$, then \(\norm{\al}_{\Krkhs}^2=\EE_{X,X'}(\Krkhs(X,X'))\), whenever this expression is finite.
\index{kernel!norm}
\index{RKHS}
\index{maximum mean discrepancy}

\begin{prop}[Kernel norm as an RKHS dual norm]\label{prop-kernel-rkhs-dual}
\index{dual!norm}
\index{RKHS}
\index{kernel!norm}
\index{norm!RKHS}
	Let $\RKHS$ be the RKHS with reproducing kernel $\Krkhs$, and assume that the kernel mean embedding
\index{kernel!mean embedding}
	\[
		m_\xi \eqdef \int \Krkhs(x,\cdot)\d\xi(x)
	\]
	is well-defined for the signed measure $\xi$. Then
\index{signed!measure}
	\[
		\norm{\xi}_{\Krkhs}
		=
		\sup_{\norm{h}_{\RKHS}\leq1}
		\int h(x)\d\xi(x),
	\]
	so $\norm{\cdot}_{\Krkhs}$ is the dual norm in the sense of~\eqref{eq-dual-norm-cont} associated with the RKHS unit ball.
\index{dual!norm}
\end{prop}
\begin{proof}
	By the reproducing property,
	\[
		\int h(x)\d\xi(x)
		=
		\left\langle h,\int \Krkhs(x,\cdot)\d\xi(x)\right\rangle_{\RKHS}
		=
		\langle h,m_\xi\rangle_{\RKHS}.
	\]
	Cauchy--Schwarz gives
	\[
		\sup_{\norm{h}_{\RKHS}\leq1}\int h\d\xi
		=
		\norm{m_\xi}_{\RKHS}.
	\]
	Finally,
	\[
		\norm{m_\xi}_{\RKHS}^2
		=
		\iint \Krkhs(x,y)\d\xi(x)\d\xi(y),
	\]
	which is exactly~\eqref{eq-kernel-dual}.
\end{proof}

\begin{prop}[Universal kernels metrize weak convergence]\label{prop-mmd-metrization}
\index{maximum mean discrepancy}
\index{kernel!universal}
\index{weak!convergence}
	Assume that $\Xx$ is compact and that the RKHS generated by the continuous kernel $\Krkhs$ is dense in $\Cc(\Xx)$ for the uniform norm. Then
\index{RKHS}
	\[
		\MMD_{\Krkhs}(\al_n,\al)\to0
		\quad\Longleftrightarrow\quad
		\al_n\rightharpoonup\al
	\]
	for probability measures on $\Xx$.
\index{probability measure}
\end{prop}
\begin{proof}
	If $\MMD_{\Krkhs}(\al_n,\al)\to0$, then integrals of all RKHS functions converge. For any $h\in\Cc(\Xx)$ and any $\eta>0$, choose $g\in\RKHS$ with $\norm{h-g}_\infty\leq\eta$. Since $\al_n$ and $\al$ are probabilities,
	\[
		\left|\int h\,\d(\al_n-\al)\right|
		\leq
		2\eta+\left|\int g\,\d(\al_n-\al)\right|,
	\]
	and the last term tends to zero. This proves weak convergence. Conversely, if $\al_n\rightharpoonup\al$, then $\al_n\otimes\al_n$, $\al_n\otimes\al$ and $\al\otimes\al$ converge weakly on the compact product space. Applying this to the continuous bounded function $\Krkhs$ in the identity
\index{weak!convergence}
	\[
		\MMD_{\Krkhs}(\al_n,\al)^2
		=
		\iint \Krkhs\,\d\al_n\d\al_n
		-2\iint \Krkhs\,\d\al_n\d\al
		+\iint \Krkhs\,\d\al\d\al
	\]
	gives convergence to zero.
\end{proof}

We refer to~\cite{berlinet03reproducing,Hofmann2008,scholkopf2002learning} for more details on RKHS functional spaces.
\index{RKHS}

\begin{rem}[Universal kernels]
\index{kernel!universal}
The hypothesis in Proposition~\ref{prop-mmd-metrization} is called universality of the kernel. Equivalently, finite sums of the form $\sum_{i=1}^n a_i \Krkhs(x_i,\cdot)$ are dense in $\Cc(\Xx)$ for the uniform norm. For translation-invariant kernels on $\Xx=\RR^d$, $\Krkhs(x,y)=\Krkhs_0(x-y)$, this is equivalent, in the usual sense on compact sets or with suitable decay assumptions, to the Fourier transform not vanishing on its support~\cite{sriperumbudur2008injective,sriperumbudur2012empirical}.
\index{kernel!translation-invariant}
\index{RKHS}
\index{maximum mean discrepancy}
\end{rem}

In the special case where $\al$ is a discrete measure, one thus has the simple expression
\eq{
	\norm{\al}_{\Krkhs}^2 = \sum_{i=1}^n \sum_{i'=1}^n \a_i\a_{i'} \KrkhsD_{i,i'} = \dotp{\KrkhsD\a}{\a}
\index{RKHS}
	\qwhereq
	\KrkhsD_{i,i'} \eqdef \Krkhs(x_i,x_{i'}).
}
In particular, when $\al=\sum_{i=1}^n \a_i \de_{x_i}$ and $\be=\sum_{i=1}^n \b_i \de_{x_i}$ are supported on the same set of points, $\norm{\al-\be}_{\Krkhs}^2 = \dotp{\KrkhsD(\a-\b)}{\a-\b}$, so that $\norm{\cdot}_{\Krkhs}$ is a Euclidean norm (proper if $\KrkhsD$ is positive definite, degenerate otherwise if $\KrkhsD$ is semidefinite) on the simplex $\simplex_n$.
To compute the discrepancy between two discrete measures, one can use
\eql{\label{eq-mmd-discr}
\index{maximum mean discrepancy}
	\norm{\al-\be}_{\Krkhs}^2 =
		\sum_{i,i'} \a_i \a_{i'} \Krkhs(x_i,x_{i'})	+
		\sum_{j,j'} \b_j \b_{j'} \Krkhs(y_j,y_{j'}) - 2
		\sum_{i,j} \a_i \b_j \Krkhs(x_i,y_j).
}

\section[phi-divergences]{$\phi$-divergences}
\index{phi-divergence}
\label{sec-phi-div}

This section develops divergences based on pointwise density ratios. They are computationally simple and statistically classical, but they do not see small spatial displacements between singular measures.
\index{density!ratio}

\paragraph{Definition by density ratios.}
\index{density!ratio}

We now consider a radically different class of methods to compare distributions, which are simpler to compute ($O(n)$ for discrete distributions) but never metrize weak-* convergence.
\index{weak!convergence}
Note that yet another way is possible, using Bregman divergence, which might metrize weak-* convergence when the associated entropy function is weak-* regular.
\index{Bregman!divergence}
\index{entropy!function}

\begin{defn}[Entropy function]
\label{def_entropy}
A function $\phi : \RR \to \RR \cup \{\infty\}$ is an entropy function if it is lower semicontinuous, convex, $\dom \phi\subset [0,\infty[$, and satisfies the feasibility condition $\dom \phi \cap (0,+\infty) \neq \emptyset$. The speed of growth of $\phi$ at $\infty$ is described by 
\index{lower semicontinuity}
\index{entropy!function}
\eq{
\phi'_\infty = \lim_{x\rightarrow +\infty} \phi(x)/x \in \RR \cup \{\infty\} \, .
}
\end{defn}

If $\phi'_\infty = \infty$, then $\phi$ grows faster than any linear function and $\phi$ is said to be \emph{superlinear}. Any entropy function $\phi$ induces a $\phi$-divergence (also known as Cisz\'ar divergence~\cite{ciszar1967information,ali1966general} or $f$-divergence) as follows.
\index{entropy!function}
\index{phi-divergence}

\begin{defn}[$\phi$-Divergences]
\label{def_divergence}
Let $\phi$ be an entropy function.
\index{entropy!function}
For $\al,\be \in \Mm(\X)$, let $\frac{\d \al}{\d \be} \be + \al^{\perp}$ be the Lebesgue decomposition of $\al$ with respect to $\be$: this means that $\al$ is uniquely decomposed as $\al^{\mathrm{ac}}+\al^\perp$, with $\al^{\mathrm{ac}}\ll\be$, $\al^\perp\perp\be$, and $\al^{\mathrm{ac}}=(\d\al/\d\be)\be$. The divergence $\Divergm_\phi$ is defined by
\index{Lebesgue decomposition}
\eql{\label{eq-phi-div}
	\Divergm_\phi (\al|\be) \eqdef \int_\X \phi\left(\frac{\d \al}{\d \be} \right) \d \be
+ \phi'_\infty \al^{\perp}(\X)
}
if $\al,\be$ are nonnegative and $\infty$ otherwise.
\end{defn}%

The additional term $\phi'_\infty \al^{\perp}(\X)$ in~\eqref{eq-phi-div} is the recession contribution of the perspective functional. It gives the weak-* lower-semicontinuous extension of the density-ratio integral when singular mass appears. This is essential for entropies with linear growth at infinity, such as the absolute value~\eqref{eq-tv-entropy} defining the TV norm. If $\phi$ has superlinear growth, \eg the usual entropy~\eqref{eq-shannon-entropy}, then $\phi'_\infty=+\infty$ so that $\Divergm_\phi (\al|\be) = +\infty$ if $\al$ does not have a density with respect to $\be$.
\index{Shannon!entropy}
\index{density!ratio}

In the discrete setting, assuming
\eql{\label{eq-div-disc-meas}
	\al=\sum_i \a_i \de_{x_i}
	\qandq \be=\sum_i \b_i \de_{x_i}
}
are supported on the same set of $n$ points $(x_i)_{i=1}^n \subset \X$,~\eqref{eq-phi-div} defines a divergence on $\simplex_n$
\eql{\label{eq-discr-diverg}
	\DivergmD_\phi(\a|\b) = \sum_{i \in \Supp(\b)} \phi\pa{ \frac{\a_i}{\b_i} } \b_i + \phi'_\infty \sum_{i \notin \Supp(\b)} \a_i,
}
where $\Supp(\b) \eqdef \enscond{i \in \range{n}}{ \b_i \neq 0 }$.

\begin{proposition}[Basic properties of $\phi$-divergences]
\index{phi-divergence}
If $\phi$ is an entropy function, then $\Divergm_\phi$ is jointly $1$-homogeneous, convex and weak-* lower semicontinuous in $(\al,\be)$.
\index{lower semicontinuity}
\index{entropy!function}
\end{proposition}

\begin{proof}
	One defines the associated perspective function
	\eq{
		\foralls (u,v) \in (\RR_+)^2, \quad
		\psi(u,v) = \choice{
			\phi(u/v) v \qifq v \neq 0, \\
				u \phi'_\infty \qifq v=0
		}
	}
		The joint $1$-homogeneity follows from the definition of this perspective. We prove convexity in the discrete case, where
		\eq{
			\DivergmD_\phi(\a|\b) = \sum_{i} \psi(\a_i,\b_i),
		}
		and it is enough to show that $\psi$ is convex on $(\RR_+)^2$. We first prove this on $\RR_+ \times \RR_+^*$; the extension to $v=0$ follows by lower semicontinuity of the recession value $u\phi'_\infty$.
\index{lower semicontinuity}
		Indeed, for any $\la \in [0,1]$, $\tau=1-\la$, set
		\[
			\theta_1=\frac{\tau v_1}{\tau v_1+\lambda v_2},
			\qquad
			\theta_2=\frac{\lambda v_2}{\tau v_1+\lambda v_2}.
		\]
		Then $\theta_1+\theta_2=1$ and
		\[
			\frac{\tau u_1+\lambda u_2}{\tau v_1+\lambda v_2}
			=
			\theta_1\frac{u_1}{v_1}
			+
			\theta_2\frac{u_2}{v_2}.
		\]
		Convexity of $\phi$ therefore gives
		\[
			\phi\pa{ \frac{\tau u_1+\lambda u_2}{\tau v_1+\lambda v_2} }
			(\tau v_1+\lambda v_2)
			\leq
			\tau v_1\phi\pa{\frac{u_1}{v_1}}
			+
			\lambda v_2\phi\pa{\frac{u_2}{v_2}} .
		\]
		In the general measure case, weak-* lower semicontinuity is the standard lower-semicontinuity theorem for convex integral functionals with recession extension; in the discrete case it is immediate from the lower semicontinuity of $\psi$.
\end{proof}

The following proposition records when $\Divergm_\phi$ is nonnegative.

\begin{proposition}[Non-negativity of $\phi$-divergences]\label{phi-div-positive}
\index{phi-divergence}
Assume that $\phi$ is normalized by $\phi(1)=0$. For probability distributions $(\al,\be) \in \Mm_+^1(\Xx)$, one has $\Divergm_\phi(\al|\be) \geq 0$. If $\phi$ is strictly convex, then one has $\Divergm_\phi(\al|\be)=0$ if and only if $\al=\be$.
This property extends to arbitrary distributions $(\al,\be) \in \Mm_+(\Xx)$ if one furthermore imposes that $\phi \geq 0$.
\end{proposition}

\begin{proof}
		Let $m=\al+\be$ and write $a=\frac{\d\al}{\d m}$ and $b=\frac{\d\be}{\d m}$. Using the perspective function $\psi$ from the previous proof,
		\[
			\Divergm_\phi(\al|\be)=\int \psi(a,b)\d m .
		\]
		Since $\al$ and $\be$ are probabilities, $\int a\d m=\int b\d m=1$. Jensen's inequality and $\psi(1,1)=\phi(1)=0$ give
\index{Jensen inequality}
		\[
			\Divergm_\phi(\al|\be)
			\geq
			\psi\pa{\int a\d m,\int b\d m}=0.
		\]
		If $\phi$ is strictly convex, equality in Jensen forces $a=b$ $m$-almost everywhere, hence $\al=\be$. In the general non-probability case, if $\phi \geq 0$ then the divergence is positive by construction.
\end{proof}

\paragraph{Classical examples and topology.}

The following examples calibrate the strength of $\phi$-divergences. KL is sensitive to absolute continuity, while total variation gives the strong topology and therefore behaves very differently from Wasserstein-type weak metrics.
\index{topology!strong}
\index{absolute continuity}
\index{total variation}
\index{phi-divergence}

\begin{example}[Kullback--Leibler divergence]
\index{Kullback-Leibler divergence}
\label{ex_KLdiv}
The Kullback--Leibler divergence $\KL \eqdef \Divergm_{\phi_{\KL}}$, also known as the relative entropy, was already introduced in~\eqref{eq-defn-rel-entropy} and~\eqref{eq-kl-defn}. It is the divergence associated to the Shannon--Boltzmann entropy function $\phi_{\KL}$, given by
\index{entropy!function}
\index{entropy!relative}
\eql{\label{eq-shannon-entropy}
\index{Shannon!entropy}
	\phi_{\KL}(s)= \begin{cases}
		s\log(s)-s+1 & \textnormal{for } s>0 , \\
		1 & \textnormal{for } s=0 , \\
		+\infty & \textnormal{otherwise.}
		\end{cases}
}
\end{example}

\begin{example}[Total variation]\label{exmp-tv}
\index{total variation}
The total variation distance $\TV \eqdef \Divergm_{\phi_{\TV}}$ is the divergence associated to
\eql{\label{eq-tv-entropy}
	\phi_{\TV}(s)= \begin{cases}
		|s-1| & \textnormal{for } s\geq0 , \\
		+\infty & \textnormal{otherwise.}
		\end{cases}
}
It actually defines a norm on the full space of measures $\Mm(\X)$ where
\eql{\label{eq-defn-tv}
	\TV(\al|\be) = \norm{\al-\be}_{\TV},
	\qwhereq
	\norm{\al}_{\TV} = |\al|(\X) = \int_\X \d|\al|(x).
}
If $\al$ has a density $\density{\al}$ on $\X=\RR^\dim$, then the TV norm is the $L^1$ norm on functions, $\norm{\al}_{\TV} = \int_\X |\density{\al}(x)| \d x = \norm{\density{\al}}_{L^1}$.
If $\al$ is discrete as in~\eqref{eq-div-disc-meas}, then the TV norm is the $\ell^1$ norm of vectors in $\RR^n$, $\norm{\al}_{\TV}=\sum_i |\a_i| = \norm{\a}_{\ell^1}$.
\end{example}

\begin{rem}[Strong vs. weak topology]
\index{topology!strong}
\index{topology!weak}
		The total variation norm~\eqref{eq-defn-tv} defines the so-called ``strong'' topology on the space of measures.
\index{total variation}
	On a compact domain $\X$ of radius $R$, one has
	\eq{
		\Wass_1(\al,\be) \leq R \norm{\al-\be}_{\TV}
	}
	so that this strong notion of convergence implies the weak convergence metrized by Wasserstein distances.
\index{Wasserstein!distance}
\index{weak!convergence}
	The converse is, however, not true, since $\de_x$ does not converge strongly to $\de_y$ if $x \rightarrow y$ (note that
	$\norm{\de_x-\de_y}_{\TV}=2$ if $x \neq y$).
	A chief advantage is that $\Mm_+^1(\Xx)$ (once again on a compact ground space $\X$) is compact for the weak topology so that from any sequence of probability measures $(\al_k)_k$, one can always extract a converging subsequence, which makes it a suitable space for several optimization problems. 
\index{probability measure}
\index{topology!weak}
\end{rem}

\paragraph{Main families of $\phi$-divergences.}
\index{phi-divergence}

Several classical divergences fit in the same template. The power-divergence family
\[
	\phi_\gamma(s)=\frac{s^\gamma-\gamma s+\gamma-1}{\gamma(\gamma-1)}
	\qquad(\gamma\neq0,1)
\]
interpolates between the Pearson $\chi^2$ divergence at $\gamma=2$, the Hellinger-type behavior at $\gamma=1/2$, and, by taking limits, the KL divergence as $\gamma\to1$ and the reverse KL or Burg entropy $\phi_0(s)=-\log s+s-1$ as $\gamma\to0$. The Hellinger divergence is often written separately with $\phi_H(s)=(\sqrt s-1)^2$, giving $\Hellinger(\alpha,\beta)=\norm{\sqrt{\rho_\alpha}-\sqrt{\rho_\beta}}_{L^2}$ when both measures have densities. The Jensen--Shannon divergence is the symmetrized and bounded KL-to-the-mixture divergence
\index{Pearson divergence}
\index{reverse KL divergence}
\index{Burg entropy}
\index{Jensen-Shannon divergence}
\index{Hellinger!divergence}
\[
	\JS(\alpha,\beta)^2
	=
	\frac12\KL\!\left(\alpha\middle|\frac{\alpha+\beta}{2}\right)
	+
	\frac12\KL\!\left(\beta\middle|\frac{\alpha+\beta}{2}\right),
\]
and is generated by a bounded entropy equivalent to $\phi_{\JS}(s)=s\log s-(s+1)\log((s+1)/2)$ up to an irrelevant affine term. Total variation corresponds to the non-smooth entropy $\phi_{\TV}(s)=|s-1|$ and is exceptional because it is both a $\phi$-divergence and an integral probability metric.
\index{integral probability metric}
\index{total variation}
\index{phi-divergence}

\begin{figure}[ht]
\centering
\begin{tabular}{@{}cc@{}}
\small generator functions & \small discrete ratio penalties \\[-.15em]
\includegraphics[width=.43\linewidth]{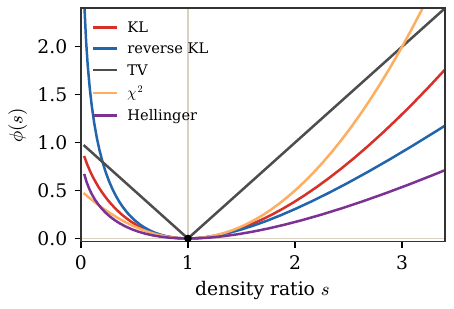} &
\includegraphics[width=.43\linewidth]{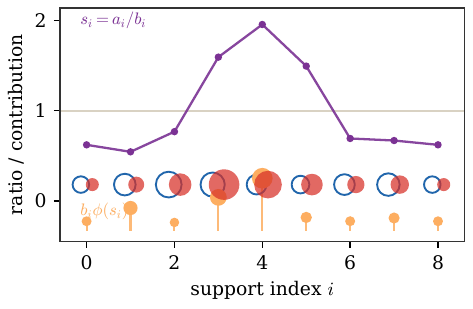}
\end{tabular}
\caption{$\phi$-divergences through density ratios. The left panel shows normalized generators for common divergences as functions of $s=\d\alpha/\d\beta$; all curves vanish at $s=1$ up to affine normalization. The right panel shows the discrete formula $D_\phi(a|b)=\sum_i b_i\phi(a_i/b_i)$: hollow blue circles encode $b_i$, filled red circles encode $a_i$, the violet curve gives the ratios $a_i/b_i$, and orange lollipops show local KL-type contributions.}
\index{phi-divergence}
\index{density!ratio}
\label{fig:dualnorms-phi-generators}
\end{figure}

\begin{rem}[$\phi$-divergences versus Bregman divergences]
\index{phi-divergence}
\index{Bregman!divergence}
	Except for KL-type entropies, $\phi$-divergences should not be confused with Bregman divergences. A $\phi$-divergence compares measures pointwise through the density ratio $\d\alpha/\d\beta$ and is invariant under measurable changes of variables. A Bregman divergence is generated by a convex functional on a linear space and compares two points through first-order Taylor error. KL is special because the integral entropy $\alpha\mapsto\int \rho\log\rho$ produces a Bregman divergence whose density-ratio form is also a $\phi$-divergence.
\index{convex!function}
\index{phi-divergence}
\index{density!ratio}
\end{rem}

\paragraph{Variational dual formula.}
\index{variational dual formula}

The following formula turns a pointwise density-ratio penalty into a dual optimization problem over test functions. It is the analogue, for $\phi$-divergences, of the Kantorovich dual formula for transport costs.
\index{phi-divergence}
\index{density!ratio}
\index{dual!formula}

\begin{proposition}[Dual expression]
	A $\phi$-divergence can be expressed using the Legendre transform
\index{Legendre transform}
	\eq{
		\phi^{*,\geq 0}(s) \eqdef \usup{t \in \RR^+} st - \phi(t)
	}
	(notice that we restrict the function to the positive real)
	of $\phi$ as 
	\eql{\label{eq-dual-div}
		\Divergm_\phi(\al|\be) = \usup{f: \X \rightarrow \RR} \int_\X f(x) \d\al(x) - \int_\X \phi^{*,\geq 0}(f(x)) \d\be(x).
	}
	which equivalently reads that the Legendre transform of $\Divergm_\phi(\cdot|\be)$ reads
\index{Legendre transform}
	\eql{\label{eq-legendre}
		\foralls f \in \Cc(\Xx), \quad
		\Divergm_\phi^*(f|\be) = \int_\X \phi^{*,\geq 0}(f(x)) \d\be(x).
	}
\end{proposition}

\begin{proof}
		We first consider the superlinear case $\phi'_\infty=+\infty$, so that $\Dd_\phi(\al|\be)=+\infty$ if $\al$ does not have a density $\rho \geq 0$ with respect to $\be$, $\d\al=\rho \d\be$. Thus the Legendre-Fenchel transform of $\Dd_\phi(\cdot|\be)$
	reads
	\begin{align*}
		\Dd_\phi^*(f|\be) &= \usup{\rho \geq 0} \int_\Xx f(x) \rho(x)\d\be(x) - \int_\Xx \phi(\rho(x)) \d\be(x)\\
			&= \int_\Xx \usup{\rho(x) \geq 0} \pa{ f(x) \rho(x) -\phi(\rho(x)) }\d\be(x)
		= \int_\Xx \phi^{*,\geq 0}(f(x)) \d\be(x).
	\end{align*}
		Fenchel--Moreau then gives the displayed dual expression. For a general entropy, the same argument is applied to the perspective with its recession term; the singular part is exactly encoded by the effective domain of $\phi^{*,\geq0}$.
	\end{proof}

\section{GANs via Duality}
\index{GAN}

GANs fit naturally into the dual viewpoint: the discriminator is a parameterized potential and the generator moves a reference measure. This section first explains the original divergence-based GAN objective, then contrasts it with integral probability metrics such as MMD and Wasserstein distances.
\index{maximum mean discrepancy}
\index{integral probability metric}
\index{Wasserstein!distance}
\index{reference!measure}

The goal is to fit a generative parametric model $\al_\th = g_{\th,\sharp} \zeta$ to empirical data $\be = \frac{1}{m}\sum_{j} \de_{y_j}$, where $\zeta \in \Mm_+^1(\Zz)$ is a fixed density over the latent space and $g_\th : \Zz \rightarrow \Xx$ is the generator, often a neural network.

\paragraph{Divergence-based adversarial losses.}
\index{adversarial!loss}

Any $\phi$-divergence can be written in adversarial form through the dual formula~\eqref{eq-dual-div}:
\index{phi-divergence}
\index{dual!formula}
\eq{
	\umin{\th} \Divergm_\phi(\al_\th|\be)
	= \umin{\th} \usup{\f} \int_\Xx f(x) \d\al_\th(x) - \Divergm_\phi^*(f|\be)
	= \umin{\th} \usup{\f} \int_\Xx f(g_\th(z)) \d\zeta(z) -
	\frac{1}{m}\sum_j \phi^*(f(y_j)).
}
Replacing the unrestricted potential $f$ by a neural network $f_\xi$ gives a saddle problem
\[
	\min_\theta\max_\xi
	\int_\Zz f_\xi(g_\theta(z))\d\zeta(z)
	-
	\frac1m\sum_j \phi^*(f_\xi(y_j)).
\]
The original vanilla GAN~\cite{GAN} is this construction for the Jensen--Shannon generator discussed above,
\[
	\phi_{\JS}(s)=s\log s-(s+1)\log\frac{s+1}{2},
	\qquad
	\phi_{\JS}^*(u)=-\log(2-e^u),\quad u<\log2,
\]
up to affine normalizations and the usual reparametrization of the potential by a discriminator with values in $(0,1)$. In practice the min--max problem is solved by alternating stochastic gradient descent/ascent on $(\theta,\xi)$. Unlike the convex-concave variational formula, the neural parametrization is non-convex in $\theta$ and non-concave in $\xi$, which explains the instability and mode-collapse pathologies of divergence-based GAN training. These losses estimate density ratios, which is statistically meaningful when the measures overlap but can saturate when the model and data are mutually singular; for example, the Jensen--Shannon divergence is maximal for disjoint supports.
\index{stochastic!gradient}
\index{Jensen-Shannon divergence}
\index{density!ratio}

\paragraph{Dual norms and integral probability metrics.}
\index{dual!norm}
\index{integral probability metric}

Instead of a density-ratio divergence, one can minimize a dual norm~\eqref{eq-dual-norm-cont}, also called an integral probability metric,
\index{density!ratio}
\eq{
	\umin{\th} \norm{\al_\th - \be}_B
	= \umin{\th}
	\usup{\f \in B} \int_{\X} \f(x) \d(\al_\th-\be)(x)
	= \umin{\th}
	\usup{\f \in B} \int_{\Zz} \f( g_\th(z)) \d\zeta - \frac{1}{m} \sum_j f(y_j).
}
MMD-GANs take $B$ to be a unit ball in an RKHS~\cite{MMD-GAN}; Wasserstein GANs take $B$ to be a Lipschitz ball, following Kantorovich--Rubinstein duality~\cite{WassersteinGAN,FrognerNIPS}. The advantage of such choices is topological: for bounded continuous RKHS balls, or for bounded Lipschitz balls on compact spaces, the resulting objective is weak-* continuous. It can therefore compare singular empirical and generated measures through test functions, instead of requiring pointwise density ratios. The price is that the discriminator class must be controlled geometrically, either by a kernel norm, a Lipschitz constraint or a related regularization.
\index{kernel!norm}
\index{RKHS}
\index{maximum mean discrepancy}
\index{Kantorovich-Rubinstein!duality}
\index{discriminator class}
\index{density!ratio}

\begin{rem}[Weight clipping is only a proxy]
	Wasserstein GANs originally used weight clipping, constraining $\norm{\xi}_\infty \leq 1$ as a proxy for enforcing $\f_\xi \in B = \enscond{f}{\Lip(f) \leq 1}$. This parameter set is both smaller than the true Lipschitz ball and non-convex, so clipping should be understood as a practical heuristic rather than a faithful implementation of the Kantorovich--Rubinstein dual constraint.
\end{rem}


\chapter{Entropic Regularization: Sinkhorn Algorithm}
\index{entropic!regularization}
\index{Sinkhorn!algorithm}
\label{sec-sinkhorn}

Entropic regularization makes optimal transport smooth, strictly convex and scalable. This chapter first explains the discrete KL-regularized problem, derives Sinkhorn's alternating matrix scaling algorithm, and then rewrites the same construction as a relative-entropy projection problem. It then records the general continuous formulation, explains the path-space Schr\"odinger problem behind the static coupling formulation, develops the dual soft-transform picture, and presents the main convex regularization variants and the debiased Sinkhorn divergence. The presentation connects the older matrix-scaling literature~\cite{Sinkhorn64,SinkhornKnopp67,Sinkhorn67} with modern entropic OT~\cite{CuturiSinkhorn,peyre2019computational}.
\index{matrix!scaling}
\index{entropic!OT}
\index{path-space!formulation}
\index{Schrodinger!problem}
\index{entropic!regularization}
\index{Sinkhorn!divergence}
\index{scaling!algorithm}
\index{entropy!relative}
\index{matrix!scaling}
\index{soft!transform}

\section{Entropic Regularization for Discrete Measures}
\index{entropic!regularization}
\index{discrete!measure}
\label{sec-entropic-discrete}

Entropy turns a possibly non-unique linear program into a unique smooth problem. The price is bias, but the reward is differentiability and fast scaling algorithms.
\index{scaling!algorithm}

The idea of the entropic regularization of optimal transport is to penalize concentrated couplings by adding the negative of the discrete Shannon--Boltzmann entropy.
\index{entropic!regularization}

\begin{defn}[Discrete Shannon--Boltzmann entropy]\label{def-discrete-shannon-boltzmann-entropy}
\index{Shannon!entropy}
\index{Boltzmann entropy}
	For a nonnegative matrix $\P$, its Shannon--Boltzmann entropy is
	\[
		\HD(\P) \eqdef -\sum_{i,j} \P_{i,j} \log(\P_{i,j}),
	\]
	with the convention $0\log(0)=0$.
\end{defn}
Using this entropy as a regularizing function gives approximate solutions to the original transport problem~\eqref{eq-kanto-discr}
\eql{\label{eq-regularized-discr}
	\MKD_\C^\epsilon(\a,\b) \eqdef
	\umin{\P \in \CouplingsD(\a,\b)}
		\dotp{\P}{\C} - \epsilon \HD(\P).
}

\begin{prop}[Existence and uniqueness of entropic OT]\label{prop-entropic-unique}
\index{entropic!OT}
	Assume that $\a,\b$ are probability histograms and that $\C$ is finite. For every $\epsilon>0$, problem~\eqref{eq-regularized-discr} admits a unique minimizer. If all entries of $\a$ and $\b$ are positive, then this minimizer is positive on every entry.
\index{histogram}
\end{prop}

\begin{proof}
	The transport polytope $\CouplingsD(\a,\b)$ is non-empty and compact, and the objective is continuous on it with the convention $0\log0=0$, so a minimizer exists. On the relative interior of the polytope,
\index{transportation!polytope}
	\[
		-\partial^2 \HD(\P)=\diag(1/\P_{i,j})
	\]
	is positive definite on every non-zero feasible direction. Hence $-\HD$ is strictly convex on the polytope and $\dotp{\P}{\C}-\epsilon\HD(\P)$ is strictly convex, which implies uniqueness.

	If $\a_i,\b_j>0$ and a minimizer had $\P_{i,j}=0$, then for small $t>0$ the perturbation $\P_t=(1-t)\P+t\,\a\otimes\b$ remains feasible. The directional derivative of the entropic part at $t=0$ is $-\infty$ because the derivative of $r\log r$ at $0$ is $-\infty$ along a positive direction. Thus the objective decreases for small $t$, contradicting optimality. Therefore the minimizer is strictly positive.
\end{proof}

\paragraph{Smoothing effect.}
\index{entropic!smoothing}

By Proposition~\ref{prop-entropic-unique}, problem~\eqref{eq-regularized-discr} has a unique optimal solution.
This smoothing, beyond providing uniqueness, actually leads to $\MKD_\C^\epsilon(\a,\b)$ being a smooth function of $\a$, $\b$ and $\C$ whenever these variables stay in the relative interior of their domains. In finite dimension, this follows from strict convexity and the envelope theorem applied to the dual problem.
\index{envelope theorem}
\index{strict!convexity}
\index{dual!problem}
The effect of the entropy is to act as a barrier function for the positivity constraint. As we will show later, this forces the solution $\P$ to be strictly positive on the support of $\a \otimes \b$. We will also show that as $\epsilon\to+\infty$, the solution satisfies $\P \to \a \otimes \b$.

\begin{figure}[H]
\centering
\setlength{\tabcolsep}{2pt}
\begin{tabular}{@{}cccc@{}}
\includegraphics[width=.21\linewidth]{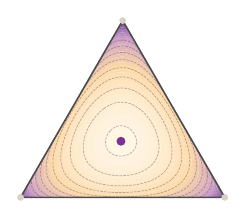} &
\includegraphics[width=.21\linewidth]{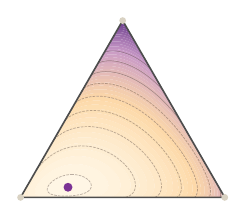} &
\includegraphics[width=.21\linewidth]{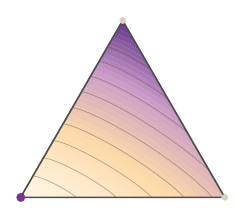} &
\includegraphics[width=.21\linewidth]{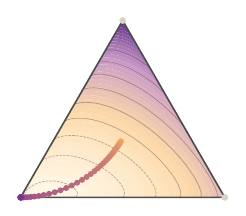} \\[-.1em]
\small large $\epsilon$ &
\small medium $\epsilon$ &
\small small $\epsilon$ &
\small entropic path \\[.35em]
\index{entropic!path}
\includegraphics[width=.21\linewidth]{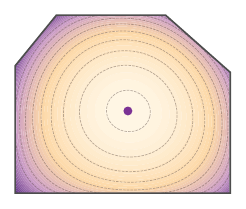} &
\includegraphics[width=.21\linewidth]{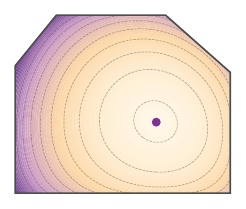} &
\includegraphics[width=.21\linewidth]{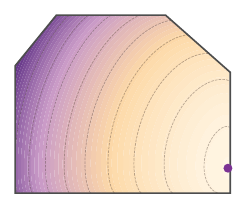} &
\includegraphics[width=.21\linewidth]{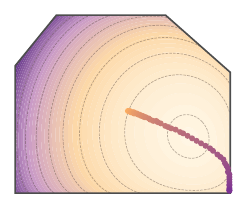} \\[-.1em]
\small LP slack barrier &
\small intermediate &
\small low temperature &
\small barrier path
\end{tabular}
\caption{Entropic regularization and slack barriers. The first row shows the penalized objective $\dotp{c}{p}+\epsilon\sum_i p_i(\log p_i-1)$ on a triangular face of the transport polytope; color and level sets represent the regularized functional itself, not only the linear part. The second row shows the analogous entropy-on-slacks objective $\ell^\top z+\epsilon H(b-Az)$ on a two-dimensional polyhedron $Az\leq b$, with $H(s)=\sum_i s_i(\log s_i-1)$. Large $\epsilon$ selects an interior reference point, while small $\epsilon$ moves the minimizer toward a low-cost face.}
\index{entropic!regularization}
\index{transportation!polytope}
\label{fig:sinkhorn-entropy-lp-geometry}
\end{figure}

\paragraph{Entropy barriers versus generic LP barriers.}
\index{barrier!entropy}
\index{barrier!LP}

For a generic linear program $\min_z \ell^\top z$ subject to $Az\leq b$, one can introduce positive slacks $s=b-Az$ and use an entropy-on-slacks penalty $H(s)=\sum_i s_i(\log s_i-1)$ as a smooth interior regularization. This is a useful analogy for Figure~\ref{fig:sinkhorn-entropy-lp-geometry}, but it is not the standard interior-point barrier for linear programming. The canonical barrier on the positive orthant is the Burg, or reverse-KL, logarithmic barrier $-\sum_i\log s_i$; it is self-concordant and therefore fits the Newton theory of interior-point methods~\cite{nesterov1994interior}. The price is that a generic Newton step solves a dense linear system, leading to cubic per-iteration scaling in the relevant number of variables or constraints. Optimal transport is special: the entropy is placed on the entries of $\P$, while the constraints are only the row and column marginals. This separable structure turns the associated Bregman projections into diagonal rescalings, hence into the Sinkhorn matrix-vector iterations developed next.
\index{Newton step}
\index{self-concordance}
\index{dense linear system}
\index{matrix-vector iteration}
\index{linear programming}
\index{interior-point method}
\index{barrier!logarithmic}
\index{Bregman!projection}


\section{Sinkhorn's Algorithm}
\index{Sinkhorn!algorithm}

Sinkhorn's algorithm is alternating normalization of rows and columns. This section derives the scaling form of the optimizer and explains why each iteration only needs matrix-vector products.
\index{scaling!form}

The underlying matrix-scaling iteration has a long history, including iterative proportional fitting and the work of Sinkhorn and Knopp~\cite{Sinkhorn64,SinkhornKnopp67,Sinkhorn67}. Its modern role in OT was transformed by Cuturi's entropic formulation~\cite{CuturiSinkhorn}: the algorithm became a practical large-scale tool for machine learning, and also changed the way OT is viewed in ML, from a mostly geometric distance to a differentiable computational primitive.
\index{scaling!iterative proportional fitting}
\index{matrix!scaling}

The following proposition shows that the solution of~\eqref{eq-regularized-discr} has a specific form, which can be parameterized using $n+m$ variables. That parameterization is therefore essentially dual, in the sense that a coupling $\P$ in $\CouplingsD(\a,\b)$ has $nm$ variables but $n+m$ constraints.

\begin{prop}[Scaling form of entropic OT]\label{prop-regularized-primal}
\index{scaling!form}
\index{entropic!OT}
$\P$ is the unique solution to~\eqref{eq-regularized-discr} if and only if there exists $(\uD,\vD) \in \RR_+^n \times \RR_+^m$ such that
\eql{\label{eq-scaling-form}
\index{scaling!form}
	\foralls (i,j) \in \range{n} \times \range{m}, \quad \P_{i,j} = \uD_i \K_{i,j} \vD_j
		\qwhereq \K_{i,j} \eqdef e^{-\frac{\C_{i,j}}{\epsilon}},
}
and $\P \in \Couplings(\a,\b)$.
\end{prop}

\begin{proof}
	Without loss of generality, we assume $\a_i,\b_j>0$ (otherwise, rows or columns with zero mass are fixed to zero and can be removed). By Proposition~\ref{prop-entropic-unique}, the minimizer is strictly positive on the remaining support.
\index{zero mass}

	We can thus ignore the positivity constraint when introducing two dual variables $\fD\in\RR^n,\gD\in\RR^m$ for each marginal constraint so that the Lagrangian of~\eqref{eq-regularized-discr} reads
\index{marginal!constraint}
	\eq{\label{eq-sinkhorn-lagrangian}
		\Lag(\P,\fD,\gD)=
		\dotp{\P}{\C}
		+\epsilon\sum_{i,j}\P_{i,j}\log(\P_{i,j})
		+ \dotp{\fD}{\a - \P\ones_m}
		+ \dotp{\gD}{\b - \transp{\P}\ones_n}.
	}
Considering first-order conditions (where we ignore the positivity constraint as explained above), we have
$$
	\frac{\partial\Lag(\P,\fD,\gD)}{\partial \P_{i,j}}= \C_{i,j} + \epsilon (\log\pa{ \P_{i,j} }+1) - \fD_i -\gD_j = 0.
$$
which results, in an optimal $\P$ coupling of the regularized problem, in the expression
$\P_{i,j}=e^{\frac{\fD_i+\gD_j - \C_{i,j}}{\epsilon}-1}$
which can be rewritten in the form provided in the proposition using non-negative vectors $\uD_i \eqdef e^{\fD_i/\epsilon-1}$ and $\vD_j \eqdef e^{\gD_j/\epsilon}$.
\end{proof}

The factorization of the optimal solution exhibited in Equation~\eqref{eq-scaling-form} can be conveniently rewritten in matrix form as $\P=\diag(\uD)\K\diag(\vD)$.
\index{scaling!form}
$\uD,\vD$ must therefore satisfy the following nonlinear equations which correspond to the mass conservation constraints inherent to $\CouplingsD(\a,\b)$,
\index{mass!conservation}
\eql{\label{eq-dualsinkhorn-constraints}
	\diag(\uD)\K\diag(\vD)\ones_m=\a,
	\qandq
	\diag(\vD)\K^\top \diag(\uD)\ones_n=\b,
}
These two equations can be further simplified, since $\diag(\vD)\ones_m$ is  $\vD$, and the multiplication of $\diag(\uD)$ times $\K \vD$ is
\eql{\label{eq-dualsinkhorn-constraints2}
	\uD \odot (\K \vD) = \a
	\qandq
	\vD \odot (\transp{\K}\uD) = \b
}
where $\odot$ corresponds to the entry-wise multiplication of vectors. This problem is known in the numerical analysis community as the matrix scaling problem (see~\cite{nemirovski1999complexity} and references therein).
\index{matrix!scaling}

The problem of normalizing a positive matrix $\K$ has a long history, from iterative proportional fitting in statistics and economics~\cite{kruithof,yule1912methods,Galichon-Entropic} to modern matrix balancing algorithms~\cite{ReviewSinkhorn,cohen2017matrix}.
\index{scaling!iterative proportional fitting}
The problem of normalizing a positive matrix $\K$ by diagonal scaling is well known, in particular when $n=m$ and $\a$ and $\b$ are uniform. This corresponds to diagonal scaling toward bistochasticity, which is a very old problem. The previous result shows that there is a unique such scaled matrix $\P$, thanks to the strong convexity of the regularized problem. The remaining question is how to compute this scaled matrix in practice. If some entries of $\K$ vanish (equivalently, if the cost matrix $\C$ can have infinite values), additional support conditions are needed; here we focus on the strictly positive case.
\index{cost matrix}

An intuitive way to try to solve these equations is to solve them iteratively, by modifying first $\uD$ so that it satisfies the left-hand side of Equation~\eqref{eq-dualsinkhorn-constraints2} and then $\vD$ to satisfy its right-hand side. These two updates define Sinkhorn's algorithm
\index{Sinkhorn!algorithm}
\eql{\label{eq-sinkhorn}
	\itt{\uD} \eqdef \frac{\a}{\K \it{\vD}}
	\qandq
	\itt{\vD} \eqdef \frac{\b}{\transp{\K}\itt{\uD}},
}
initialized with an arbitrary positive vector, for instance $\init{\vD} = \ones_m$. The division operator used above between two vectors is to be understood entry-wise. Note that a different initialization will likely lead to a different solution for $\uD,\vD$, since $\uD,\vD$ are only defined up to a multiplicative constant (if $\uD,\vD$ satisfy \eqref{eq-dualsinkhorn-constraints} then so do $\lambda\uD,\vD/\lambda$ for any $\lambda>0$).
The alternating normalization can be read directly on the coupling matrix: a row update enforces the source marginal and generally perturbs the target marginal, while the next column update does the converse.

\begin{alg}[Sinkhorn scaling]\label{alg:sinkhorn-scaling}
\textbf{Input:} Weights $\a,\b$, cost matrix $\C$, regularization $\epsilon>0$, tolerance $\mathrm{tol}$.

\textbf{Output:} Entropic coupling $\P$.

\textbf{Initialize:} Set $\K_{ij}=e^{-\C_{ij}/\epsilon}$, $\vD^{(0)}=\ones_m$, \(r_0=+\infty\), and \(k=0\).

\textbf{While} \(r_k>\mathrm{tol}\) \textbf{do}:
\begin{algblock}

\textbf{Set} \(k\leftarrow k+1\).

\(\uD^{(k)}=\frac{\a}{\K\vD^{(k-1)}}.\)

\(\vD^{(k)}=\frac{\b}{\transp{\K}\uD^{(k)}}.\)

\(\P^{(k)}=\diag(\uD^{(k)})\K\diag(\vD^{(k)}).\)

\textbf{Set} \(r_k=\max\{\norm{\P^{(k)}\ones_m-\a}_1,\norm{(\P^{(k)})^\top\ones_n-\b}_1\}\).
\end{algblock}
\algreturnskip
\textbf{Return} \(\P^{(k)}\).
\end{alg}

\begin{figure}[H]
\centering
\setlength{\tabcolsep}{1.5pt}
\begin{tabular}{@{}ccccc@{}}
\includegraphics[width=.18\linewidth]{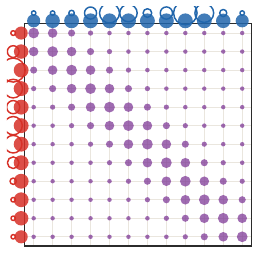} &
\includegraphics[width=.18\linewidth]{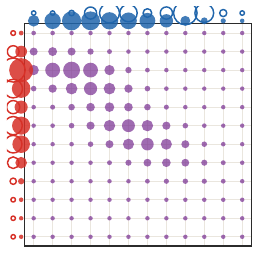} &
\includegraphics[width=.18\linewidth]{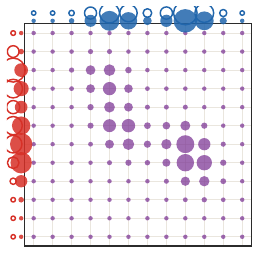} &
\includegraphics[width=.18\linewidth]{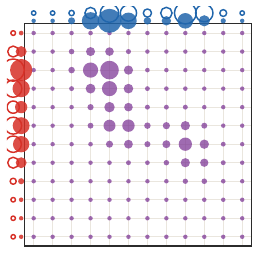} &
\includegraphics[width=.18\linewidth]{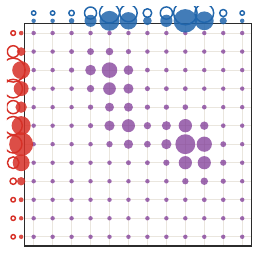} \\[-.1em]
\small initial $K$ &
\small row 1 &
\small column 1 &
\small row 2 &
\small column 2
\end{tabular}
\caption{Marginal constraints during Sinkhorn scaling on a twelve-bin one-dimensional problem. Violet circles represent the current coupling matrix, framed by the thin black box; the red source marginal is displayed on the left and the blue target marginal below the matrix. Hollow side circles show the prescribed marginals, while filled circles show the current marginals. Row normalizations align the red marginal and leave a blue defect; column normalizations align the blue marginal and leave a red defect.}
\index{row normalization}
\index{column normalization}
\index{Sinkhorn!scaling}
\index{marginal!constraint}
\label{fig:sinkhorn-marginal-errors}
\end{figure}

The same alternating projection mechanism is clearer on a dense one-dimensional discretization, where the marginal defects appear as continuous side curves.
\index{alternating!projection}

\begin{figure}[H]
\centering
\begin{tabular}{@{}cccc@{}}
\includegraphics[width=.22\linewidth]{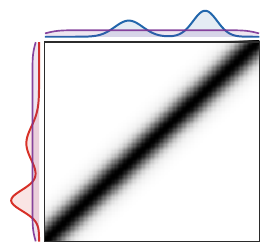} &
\includegraphics[width=.22\linewidth]{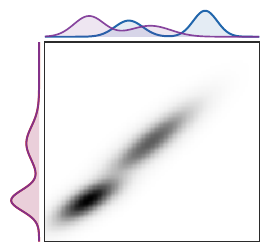} &
\includegraphics[width=.22\linewidth]{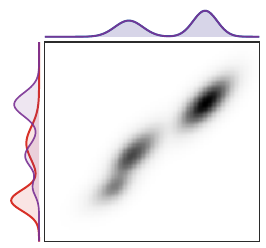} &
\includegraphics[width=.22\linewidth]{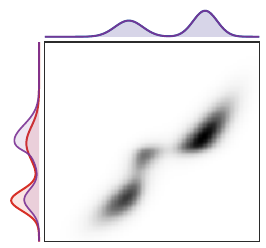} \\[-.1em]
\small initial $K$ &
\small after one row scaling &
\index{scaling!row}
\small after one column scaling &
\index{scaling!column}
\small after 12 cycles
\end{tabular}
\caption{Dense Sinkhorn scaling for one-dimensional Gaussian-mixture marginals. The grayscale matrix is the current coupling, with a thin box delimiting only the matrix and not the side marginal plots. The red and blue side curves are the prescribed source and target marginals, while the violet side curves are the current row and column sums. A row scaling makes the violet curve coincide with the red source marginal, a column scaling makes it coincide with the blue target marginal, and the alternation rapidly stabilizes both sides.}
\index{scaling!row}
\index{scaling!column}
\index{Sinkhorn!scaling}
\index{Gaussian mixture}
\label{fig:sinkhorn-continuous-marginal-scaling}
\end{figure}

After convergence, the regularization strength controls how much of the Gibbs kernel remains visible in the optimal plan.  Small $\epsilon$ produces a concentrated transport band, while larger $\epsilon$ spreads the same marginals into a smoother coupling.
\index{optimal plan}
\index{Gibbs!kernel}

\begin{figure}[H]
\centering
\begin{tabular}{@{}cccc@{}}
\includegraphics[width=.215\linewidth]{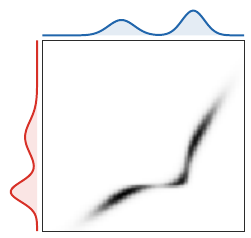} &
\includegraphics[width=.215\linewidth]{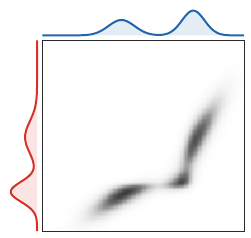} &
\includegraphics[width=.215\linewidth]{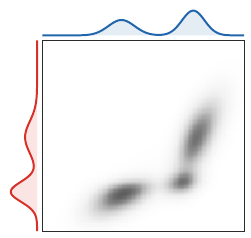} &
\includegraphics[width=.215\linewidth]{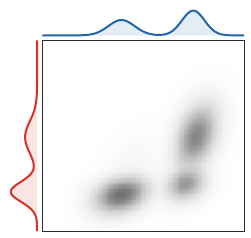} \\[-.1em]
\small $\epsilon=0.010$ &
\small $\epsilon=0.030$ &
\small $\epsilon=0.085$ &
\small $\epsilon=0.240$
\end{tabular}
\caption{Final Sinkhorn couplings for the same one-dimensional Gaussian-mixture marginals and four regularization strengths. Each boxed matrix is the converged solution of the KL-regularized problem and uses a common grayscale normalization; the side curves display the fixed source and target marginals. Decreasing $\epsilon$ sharpens the plan toward an optimal-transport graph, whereas increasing $\epsilon$ keeps more of the diffuse product-measure structure.}
\index{Gaussian mixture}
\index{product!measure}
\label{fig:sinkhorn-coupling-iterations}
\end{figure}

Chapter~\ref{sec-entropic-convergence} gives the formal convergence analysis. Before that, the following figure shows how the dual potentials stabilize along the same scaling iteration.
\index{dual!potential}

\begin{figure}[ht]
\centering
\setlength{\tabcolsep}{2pt}
\begin{tabular}{@{}cccc@{}}
\includegraphics[width=.23\linewidth]{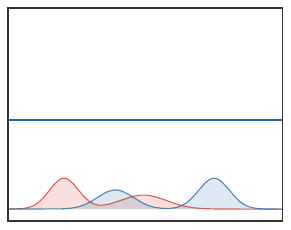} &
\includegraphics[width=.23\linewidth]{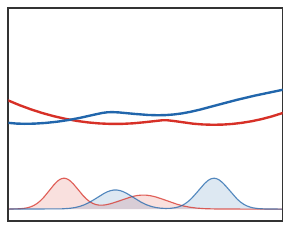} &
\includegraphics[width=.23\linewidth]{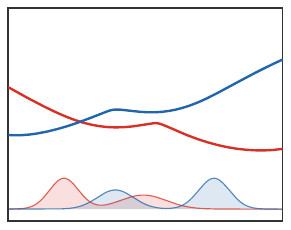} &
\includegraphics[width=.23\linewidth]{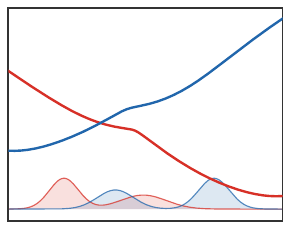} \\[-.1em]
\small $k=0$ &
\small $k=1$ &
\small $k=3$ &
\small $k=12$
\end{tabular}
\caption{KL-normalized Sinkhorn dual potentials along the scaling iteration for the same one-dimensional Gaussian-mixture setting as Figure~\ref{fig:sinkhorn-dual-potentials-epsilon}, with fixed regularization strength $\epsilon=0.045$. The bottom silhouettes show both the source histogram $\a$ in red and the target histogram $\b$ in blue, while the red and blue curves are the logarithmic scaling potentials. All panels share the same axes, making stabilization visible.}
\index{scaling!potential}
\index{Gaussian mixture}
\index{dual!potential}
\label{fig:sinkhorn-potentials-iterations}
\end{figure}

Complexity bounds for Sinkhorn and comparisons with accelerated first-order methods are discussed in~\cite{altschuler2017near,pmlr-v80-dvurechensky18a,knight2008sinkhorn}.
A chief computational advantage of Sinkhorn's algorithm, besides its simplicity, is that the only expensive step is multiplication by the Gibbs kernel. Its complexity therefore scales like $Cnm$, where $C$ is the number of Sinkhorn iterations. Chapter~\ref{sec-convergence-dual} gives a more precise convergence statement: for a fixed regularization strength $\epsilon$, the entropic dual gap has an $O(1/k)$ bound with constants proportional to $1/\epsilon$, while Hilbert-metric arguments give eventual linear convergence when the Gibbs kernel is uniformly positive. To approximate the unregularized OT value to accuracy $\delta$, one must also balance the entropic bias, which is typically $O(\epsilon)$ in finite dimension. Choosing $\epsilon$ proportional to $\delta$ and solving the entropic problem to accuracy $O(\delta)$ leads to the familiar iteration scaling of order $1/\delta^2$ for the unregularized value, up to logarithmic and cost-range factors~\cite{altschuler2017near,pmlr-v80-dvurechensky18a}.
\index{dual!gap}
\index{Sinkhorn!iteration}
\index{Sinkhorn!algorithm}
\index{linear!convergence}
\index{entropic!bias}
\index{Gibbs!kernel}
\index{dual!gap}

\begin{figure}[H]
\centering
\setlength{\tabcolsep}{3pt}
\begin{tabular}{@{}cc@{}}
\includegraphics[width=.58\linewidth]{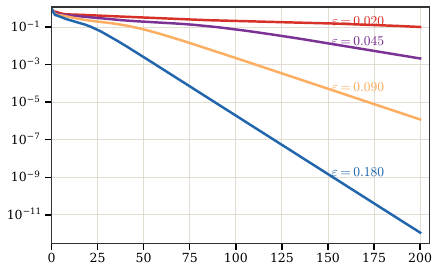} &
\raisebox{.75em}{\includegraphics[width=.31\linewidth]{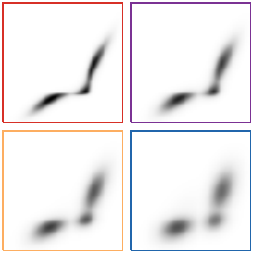}} \\[-.1em]
\small marginal violation &
\small limiting plans $P_\epsilon$
\end{tabular}
\caption{Marginal violation along Sinkhorn half-steps for four values of $\epsilon$ on the same one-dimensional Gaussian-mixture problem, together with the corresponding limiting plans. The plotted error is $\frac12(\|P_k\ones-\a\|_1+\|P_k^\top\ones-\b\|_1)$ on a logarithmic scale. The colored boxes on the right use the same colors as the convergence curves. For fixed positive $\epsilon$, the curves enter a linear regime; smaller $\epsilon$ gives a sharper transport geometry but a more peaked Gibbs kernel and slower scaling.}
\index{Sinkhorn!half-step}
\index{Gaussian mixture}
\index{Gibbs!kernel}
\label{fig:sinkhorn-linear-rate-epsilon}
\end{figure}

In many applications, however, one does not need a highly accurate optimizer for the OT subproblem. Downstream performance is often tied more to the geometric bias of the discrepancy than to exact optimality. In such settings, $C$ is often moderate.
This should be contrasted with generic interior-point methods, which use the logarithmic barrier discussed above and require solving Newton systems. For a dense transportation linear program, these algorithms typically have worst-case complexity of order $O(n^6 \log(1/\delta))$ to reach accuracy $\delta$, up to problem-dependent conditioning factors.
\index{interior-point method}
\index{barrier!logarithmic}

The second crucial aspect of Sinkhorn is that matrix-vector multiplication streams extremely well on GPU. Even better, if one is interested in computing many OT problems with a fixed cost matrix $\C$, one can replace many matrix-vector multiplications with matrix-matrix multiplications, so that the computational gain can be substantial.
\index{cost matrix}

\begin{rem}[Separable Gaussian kernels on grids]\label{rem-sinkhorn-separable-gaussian}
\index{kernel!Gaussian}
	When the samples lie on a Cartesian grid and $c(x,y)=\norm{x-y}^2$, the Gibbs kernel is Gaussian and factorizes along coordinates. If the grid has $q$ points per axis in dimension $d$, so that $N=q^d$ grid points are used, then
\index{Gibbs!kernel}
	\[
		K(x,y)=\exp\!\left(-\frac{\norm{x-y}^2}{\epsilon}\right)
		=
		\prod_{\ell=1}^d
		\exp\!\left(-\frac{(x_\ell-y_\ell)^2}{\epsilon}\right).
	\]
	Multiplication by $K$ can therefore be applied by successively multiplying along each coordinate direction, equivalently by applying one-dimensional Gaussian kernel operators along the axes. On a periodic or sufficiently padded uniform grid these are literal discrete convolutions. A direct dense one-dimensional multiplication costs $O(q^2)$ on each of the $q^{d-1}$ coordinate lines, and this is repeated for $d$ axes. Hence one Sinkhorn half-step costs
\index{Sinkhorn!half-step}
\index{kernel!Gaussian}
	\[
		O(d\,q^{d+1})=O(d\,N^{1+1/d})
	\]
	instead of $O(N^2)$. With FFT-based or truncated Gaussian convolutions, the same separability can be pushed further, but the simple tensor-product estimate already explains why grid-based Sinkhorn can scale much better than a generic dense coupling.
\end{rem}


\section{Reformulation using relative entropy}
\index{entropy!relative}

The KL formulation identifies Sinkhorn as a projection method. It also prepares the continuous and unbalanced settings, where a reference measure is essential.
\index{reference!measure}

A convenient tool to reformulate and ``normalize'' this discrete entropy is relative entropy. It is the finite-dimensional divergence that turns entropy regularization into a projection problem and admits a direct measure-theoretic extension.

\begin{defn}[Discrete relative entropy]\label{def-discrete-relative-entropy}
\index{Kullback-Leibler divergence}
\index{entropy!relative}
	For nonnegative matrices $\P,\Q$ of the same size, the generalized relative entropy, or Kullback--Leibler divergence, is
	\eql{\label{eq-kl-defn}
		\KLD(\P|\Q) \eqdef \sum_{i,j}  \P_{i,j} \log\pa{\frac{\P_{i,j}}{\Q_{i,j}}} - \P_{i,j} + \Q_{i,j}.
	}
	The convention is $0\log(0)=0$, and $\KLD(\P|\Q)=+\infty$ if there exists $(i,j)$ such that $\Q_{i,j}=0$ but $\P_{i,j} \neq 0$.
\end{defn}
For the specific case of comparing probability distributions, where $\P$ and $\Q$ have the same total mass, this further simplifies to
\eq{
	\KLD(\P|\Q) = \sum_{i,j}  \P_{i,j} \log\pa{\frac{\P_{i,j}}{\Q_{i,j}}}.
}
For the reference matrix $\Q=\ones_{n \times m}$, one has
\eq{
	-\KLD(\P|\ones_{n \times m})
	=
	\HD(\P)+\sum_{i,j}\P_{i,j}-n\,m .
}
On fixed-mass couplings the last two terms are constant, so KL regularization with reference $\ones_{n\times m}$ is equivalent to subtracting Shannon--Boltzmann entropy.
$\KLD$ is a particular instance of both a $\phi$-divergence (as defined in Section~\ref{sec-phi-div}) and a Bregman divergence; up to standard affine rescalings, it is the canonical overlap between these two families. This special property is at the heart of the fact that this regularization leads to elegant algorithms and a tractable mathematical analysis.
\index{Bregman!divergence}
\index{phi-divergence}

\begin{prop}[Relative entropy is distance-like]\label{prop-kl-distance-like}
\index{entropy!relative}
	Let $\P,\Q\in\RR_+^{n\times m}$ have the same total mass and assume $\Q_{i,j}>0$ on the support of $\P$. Then $\KLD(\P|\Q)\geq0$, with equality if and only if $\P=\Q$.
\end{prop}
\begin{proof}
	Write $\phi(s)=s\log s-s+1$. Convexity gives $\phi(s)\geq \phi(1)+\phi'(1)(s-1)=0$, and strict convexity gives equality only at $s=1$. Hence
\index{strict!convexity}
	\[
		\KLD(\P|\Q)=\sum_{i,j}\Q_{i,j}\phi(\P_{i,j}/\Q_{i,j})\geq0,
	\]
	with equality only when $\P_{i,j}/\Q_{i,j}=1$ for all entries with $\Q_{i,j}>0$. The support convention rules out positive $\P$ where $\Q=0$, so equality is equivalent to $\P=\Q$.
\end{proof}

Equivalently, when $\P$ and $\Q$ have the same total mass, it reads
\eq{
	\KLD(\P|\Q) = \sum_{i,j}  \phi( \P_{i,j} / \Q_{i,j} ) \Q_{i,j}.
}
where $\phi(s)=s\log(s)$. For any convex $\phi$ such that $\phi(1)=0$, one has indeed by Jensen
\index{Jensen inequality}
\eq{
	\sum_{i,j}  \phi( \P_{i,j} / \Q_{i,j} ) \Q_{i,j} \geq \phi( \sum_{i,j}  \P_{i,j} / \Q_{i,j}  \Q_{i,j}  )
	= \phi( \sum_{i,j} \P_{i,j} ) =\phi(1)= 0.
}

For instance, one can use as reference measure the tensor product $\a \otimes \b = (\a_i \b_j)_{i,j}$ and consider
\index{reference!measure}
\eql{\label{eq-regularized-discr-rescaled}
	\umin{\P \in \CouplingsD(\a,\b)}
		\dotp{\P}{\C} + \epsilon \KLD(\P|\a \otimes \b).
}
This normalization will matter again for unbalanced OT, where changing the reference measure is no longer merely a harmless notational choice.
\index{reference!measure}
\index{unbalanced!OT}

For the balanced problem with fixed positive marginals, however, the choice of tensor-product reference does not affect the selected coupling: it only adds a constant to the objective, as shown in the following proposition. In particular,~\eqref{eq-regularized-discr-rescaled} and~\eqref{eq-regularized-discr} have the same unique solution.

\begin{prop}[Reference measure shift for KL]\label{prop-kl-shift}
\index{reference!measure}
\index{Kullback-Leibler divergence}
	After removing zero-mass rows and columns, assume that $\a,\a'\in\simplex_n$ and $\b,\b'\in\simplex_m$ have positive entries. For every $\P \in \CouplingsD(\a,\b)$, one has
\eq{
	\KLD(\P|\a \otimes \b) =
	\KLD(\P|\a' \otimes \b') - \KLD(\a|\a') - \KLD(\b|\b').
}
Consequently, for fixed positive marginals, changing the positive tensor-product reference measure only adds a constant on the transport polytope. In particular,~\eqref{eq-regularized-discr-rescaled} and~\eqref{eq-regularized-discr} have the same unique minimizer.
\index{transportation!polytope}
\index{reference!measure}
\end{prop}

\begin{proof}
Expanding the logarithm and using the marginal constraints gives
\index{marginal!constraint}
\begin{align*}
	\KLD(\P|\a\otimes\b)
	&=
	\KLD(\P|\a'\otimes\b')
	+
	\sum_i \a_i\log\frac{\a_i'}{\a_i}
	+
	\sum_j \b_j\log\frac{\b_j'}{\b_j} \\[.25em]
	&=
	\KLD(\P|\a'\otimes\b')
	-\KLD(\a|\a')-\KLD(\b|\b').
\end{align*}
\end{proof}

The tensor-product reference is nevertheless useful when supports vary, because it makes explicit which entries are allowed to vanish. It is also the normalization that passes cleanly to the continuous formulation below, where the reference measure is $\alpha\otimes\beta$ rather than an ambient Lebesgue measure.
\index{Lebesgue measure}
\index{reference!measure}

\begin{figure}[H]
\centering
\begin{tabular}{@{}ccc@{}}
\includegraphics[width=.30\linewidth]{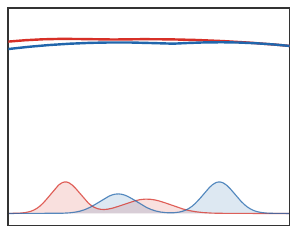} &
\index{dual!potential}
\includegraphics[width=.30\linewidth]{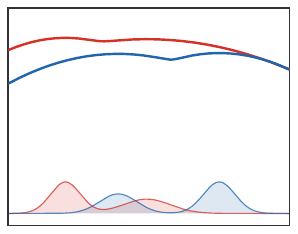} &
\includegraphics[width=.30\linewidth]{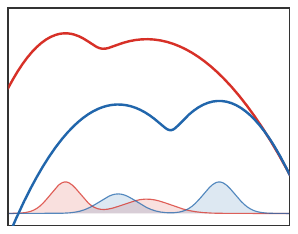} \\[-.1em]
\small $\epsilon=0.010$ &
\small $\epsilon=0.045$ &
\small $\epsilon=0.20$
\end{tabular}
\caption{KL-normalized Sinkhorn dual potentials for the same one-dimensional Gaussian-mixture histograms. The bottom silhouettes show both the source histogram $\a$ in red and the target histogram $\b$ in blue. The red and blue curves are the logarithmic scalings $\fD_i^\epsilon=\epsilon\log u_i^\epsilon$ and $\gD_j^\epsilon=\epsilon\log v_j^\epsilon$, with gauge $\dotp{\fD^\epsilon}{\a}=0$, computed from $\P_{i,j}=u_i\,\a_i\b_j e^{-\C_{i,j}/\epsilon}\,v_j$. The squared Euclidean cost is normalized by its median positive entry, and all panels use the same axes; increasing $\epsilon$ turns the hard $c$-transform geometry into smoother log-sum-exp potentials.}
\index{Gaussian mixture}
\index{log-sum-exp}
\index{dual!potential}
\index{c-transform}
\label{fig:sinkhorn-dual-potentials-epsilon}
\end{figure}

The KL-normalized formulation also makes the two limiting regimes of the regularization parameter transparent.

\begin{prop}[Convergence with $\epsilon$]\label{prop-convergence-eps}
	Assume, after removing zero-mass rows and columns, that $\a$ and $\b$ are positive and that $\C$ is finite.
	The unique solution $\P_\epsilon$ of~\eqref{eq-regularized-discr} converges to the optimal solution with maximal entropy within the set of all optimal solutions of the Kantorovich problem, namely
\index{Kantorovich!problem}
\eql{\label{eq-entropy-conv-1}
	\P_\epsilon \overset{\epsilon \rightarrow 0}{\longrightarrow}
	\uargmin{\P} \enscond{ -\HD(\P) }{
		\P \in \CouplingsD(\a,\b), \dotp{\P}{\C} = \MKD_\C(\a,\b)
	}
}
so that in particular
\eq{
	\MKD_\C^\epsilon(\a,\b) \overset{\epsilon \rightarrow 0}{\longrightarrow} \MKD_\C(\a,\b).
}
Moreover,
\eql{\label{eq-entropy-conv-2}
	\P_\epsilon \overset{\epsilon \rightarrow \infty}{\longrightarrow}
	\a \otimes \b.
}
\end{prop}

\begin{proof}
	\textbf{Case $\epsilon \rightarrow 0$.}
	 We consider a sequence $(\epsilon_\ell)_\ell$ such that $\epsilon_\ell \rightarrow 0$ and $\epsilon_\ell > 0$.
	We denote $\P_\ell$ the solution of~\eqref{eq-regularized-discr} for $\epsilon=\epsilon_\ell$.
	Since $\CouplingsD(\a,\b)$ is bounded, we can extract a sequence (that we do not relabel for the sake of simplicity) such that $\P_\ell \rightarrow \P^\star$. Since $\CouplingsD(\a,\b)$ is closed, $\P^\star \in \CouplingsD(\a,\b)$. We consider any $\P$ such that $\dotp{\C}{\P} = \MKD_\C(\a,\b)$. Using the equivalent KL-normalized formulation~\eqref{eq-regularized-discr-rescaled}, optimality of $\P$ and $\P_\ell$ for their respective optimization problems (for $\epsilon=0$ and $\epsilon=\epsilon_\ell$) gives
	\eql{\label{eq-proof-gamma-conv}
			0 \leq \dotp{\C}{\P_\ell} - \dotp{\C}{\P} \leq \epsilon_\ell ( \KLD(\P|\a \otimes \b)-\KLD(\P_\ell|\a \otimes \b) ).
		}
	Since $\KLD$ is continuous, taking the limit $\ell \rightarrow +\infty$ in this expression shows that
	$\dotp{\C}{\P^\star} = \dotp{\C}{\P}$ so that $\P^\star$ is a feasible point of~\eqref{eq-entropy-conv-1}. Furthermore, dividing by $\epsilon_\ell$ in~\eqref{eq-proof-gamma-conv} and taking the limit shows that
		$\KLD(\P^\star|\a \otimes \b) \leq \KLD(\P|\a \otimes \b)$, which shows that $\P^\star$ is a solution of~\eqref{eq-entropy-conv-1}. Since the solution $\P_0^\star$ to this program is unique by strict convexity of $\KLD(\cdot|\a \otimes \b)$ on the optimal face, one has $\P^\star = \P_0^\star$, and the whole sequence is converging.
\index{strict!convexity}

\textbf{Case $\epsilon \rightarrow +\infty$.}
Subtracting $\min_{i,j}\C_{i,j}$ from the cost changes every feasible objective by the same constant, so it does not change the minimizer. We can therefore assume $\C\geq0$. Evaluating the energy at $\a \otimes \b$ (which belongs to the constraint set $\CouplingsD(\a,\b)$), one has
	\eq{
		\dotp{\C}{\P_\epsilon} + \epsilon \KLD(\P_\epsilon|\a \otimes \b) \leq \dotp{\C}{\a \otimes \b} + \epsilon \times 0
	}
	and since $\dotp{\C}{\P_\epsilon} \geq 0$, this leads to
	\eq{
		\KLD(\P_\epsilon|\a \otimes \b) \leq \epsilon^{-1} \dotp{\C}{\a \otimes \b} \leq \frac{\norm{\C}_\infty}{\epsilon}
	}
	so that $\KLD(\P_\epsilon|\a \otimes \b) \rightarrow 0$ and thus $\P_\epsilon \rightarrow \a \otimes \b$ since $\KLD$ is a valid divergence.
\end{proof}

\begin{figure}[H]
\centering
\begin{tabular}{@{}cccc@{}}
\includegraphics[width=.225\linewidth]{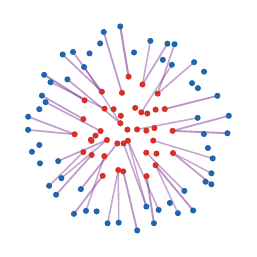} &
\includegraphics[width=.225\linewidth]{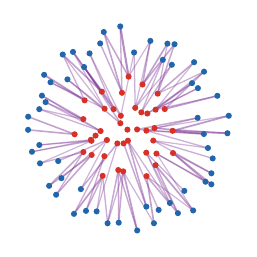} &
\includegraphics[width=.225\linewidth]{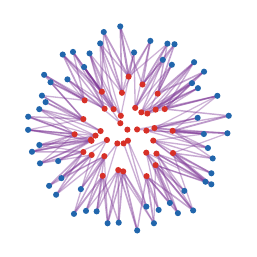} &
\includegraphics[width=.225\linewidth]{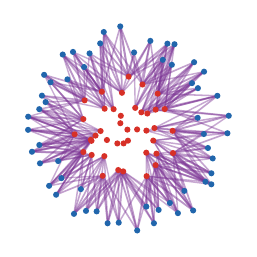} \\[-.1em]
\small $\epsilon=0.018$ &
\small $\epsilon=0.045$ &
\small $\epsilon=0.120$ &
\small $\epsilon=0.320$
\end{tabular}
\caption{Entropically regularized couplings between the canonical red disk and blue annulus point clouds for four fixed regularization strengths. The squared Euclidean cost is normalized by its median and each plan is computed by log-domain Sinkhorn until the marginal residual is below the notebook tolerance. Violet segments display the largest visible entries of the computed plan, with thickness and opacity proportional to transported mass. The plans are strictly positive for every $\epsilon>0$, but the visible mass pattern evolves from nearly radial and sparse to diffuse as $\epsilon$ increases.}
\index{log-domain Sinkhorn}
\index{marginal!residual}
\label{fig:sinkhorn-plan-epsilon}
\end{figure}


\section{General Formulation}

The continuous formulation replaces matrices by measures and discrete KL by relative entropy. It is the static endpoint problem solved by Sinkhorn; the next section explains how it is obtained from an optimization problem on stochastic paths.
\index{entropy!relative}

One can consider arbitrary measures by replacing the discrete entropy with the relative entropy with respect to the product measure $\d\al\otimes\d\be(x,y) \eqdef \d\al(x)\d\be(y)$, and propose a regularized counterpart to~\eqref{eq-mk-generic} using
\index{product!measure}
\eql{\label{eq-entropic-generic}
	\MK_\c^\epsilon(\al,\be) \eqdef
	\umin{\pi \in \Couplings(\al,\be)}
		\int_{\X \times \Y} c(x,y) \d\pi(x,y) + \epsilon \KL(\pi|\al\otimes\be)
}
The measure-theoretic definition is the exact analogue of Definition~\ref{def-discrete-relative-entropy}, with absolute continuity replacing the entrywise support condition.

\begin{defn}[Relative entropy of measures]\label{def-measure-relative-entropy}
\index{Kullback-Leibler divergence}
\index{entropy!relative}
	For nonnegative measures $\pi$ and $\xi$ on $\X\times\Y$, the relative entropy is
	\eql{\label{eq-defn-rel-entropy}
		\KL(\pi|\xi) \eqdef \int_{\X \times \Y} \log\Big( \frac{\d \pi}{\d\xi}(x,y) \Big) \d\pi(x,y)
		  + \int_{\X \times \Y} (\d\xi(x,y)-\d\pi(x,y)).
	}
		By convention, $\KL(\pi|\xi)=+\infty$ if $\pi$ is not absolutely continuous with respect to $\xi$.
\end{defn}
For fixed balanced marginals, the specific product reference in the entropy only matters up to additive constants, exactly as in Proposition~\ref{prop-kl-shift}, provided the alternative reference marginals are mutually absolutely continuous with $\al$ and $\be$. Its support and absolute-continuity structure are nevertheless essential: they determine which couplings have finite entropy. This distinction becomes substantive in the unbalanced setting, where the marginal constraints no longer freeze the additive terms.
\index{marginal!constraint}
\index{absolute continuity}
The path-space meaning of this static problem is developed next. The main point is that a noisy reference dynamics first defines a probability law on trajectories; after optimizing out the conditional law of the path given its endpoints, only the endpoint coupling remains.
\index{conditional law}
\index{path-space!formulation}
\index{endpoint coupling}
\index{path!space}


\section{Path-Space Schr\"odinger Problem}
\index{path-space!formulation}
\index{Schrodinger!problem}
\label{sec-path-space-schrodinger}

Schr\"odinger's reciprocal problem is naturally posed on paths rather than on endpoint pairs. The Sinkhorn problem appears after the path law is reduced to its two endpoint marginals.

\paragraph{Unregularized path-space transport.}
\index{path-space!formulation}
\index{path-space!transport}

Throughout this section, both endpoint measures live on the same state space $\X$; this is the setting relevant to dynamic transport and Brownian bridges. The static coupling formulation above also makes sense between two different spaces $\X$ and $\Y$.
\index{bridge!Brownian}
Let $\Om=C([0,1];\X)$ be a path space and denote by
\index{path!space}
\[
	e_t:\Om\to\X,\qquad e_t(\omega)=\omega_t
\]
the evaluation maps. A probability $M\in\Pp(\Om)$ is a law of random trajectories. Imposing the endpoint constraints
\[
	(e_0)_\sharp M=\al,
	\qquad
	(e_1)_\sharp M=\be
\]
means that the random path starts with law $\al$ and ends with law $\be$. Given a path action $\Aa:\Om\to[0,+\infty]$, the unregularized path-space problem is
\index{path-space!formulation}
\index{path!space}
\begin{equation}\label{eq-path-space-ot}
	\inf_{M\in\Pp(\Om)}
	\enscond{
		\int_\Om \Aa(\omega)\,\d M(\omega)
	}{
		(e_0)_\sharp M=\al,\ (e_1)_\sharp M=\be
	}.
\end{equation}
For the quadratic Wasserstein geometry on $\RR^d$, one takes
\[
	\Aa(\omega)=
	\begin{cases}
	\displaystyle \int_0^1 \norm{\dot\omega_t}^2\,\d t,
		& \text{if }\omega\text{ is absolutely continuous},\\[.4em]
	+\infty, & \text{otherwise}.
	\end{cases}
\]
This is the Lagrangian version of the Benamou--Brenier formulation recalled in Remark~\ref{rem-bb-path-space}.
\index{Benamou-Brenier}
\index{Benamou-Brenier!formulation}
\index{path!space}

The endpoint cost induced by the action is
\begin{equation}\label{eq-path-action-endpoint-cost}
	c_\Aa(x,y)
	\eqdef
	\inf_{\omega\in\Om}
	\enscond{\Aa(\omega)}{e_0(\omega)=x,\ e_1(\omega)=y}.
\end{equation}
For the quadratic action above, the minimizing path is the straight segment
$\omega_t=(1-t)x+ty$, and $c_\Aa(x,y)=\norm{x-y}^2$.

\begin{prop}[Endpoint reduction of path-space transport]\label{prop-path-space-ot-endpoint-reduction}
\index{path-space!formulation}
\index{path-space!transport}
	Assume that minimizing paths in~\eqref{eq-path-action-endpoint-cost} can be selected measurably, or more generally that the infimum can be approximated by measurable selections. Then~\eqref{eq-path-space-ot} has the same value as the Kantorovich problem
\index{measurable selection}
\index{Kantorovich!problem}
\index{path!space}
	\[
		\inf_{\pi\in\Couplings(\al,\be)}
		\int_{\X\times\X} c_\Aa(x,y)\,\d\pi(x,y).
	\]
	Moreover, if $\pi^\star$ is an optimal endpoint coupling and $\omega^{x,y}$ is an optimal path from $x$ to $y$, then
\index{endpoint coupling}
	\[
		M^\star
		=
		\int_{\X\times\X}\delta_{\omega^{x,y}}\,\d\pi^\star(x,y)
	\]
	is an optimal path law.
\end{prop}

\begin{proof}
	Let $M$ be any feasible path law and set $\pi=(e_0,e_1)_\sharp M$. Then $\pi\in\Couplings(\al,\be)$ and, by the definition of $c_\Aa$,
	\[
		\int_\Om \Aa(\omega)\,\d M(\omega)
		\geq
		\int_{\X\times\X} c_\Aa(x,y)\,\d\pi(x,y).
	\]
	This proves that the path-space value is at least the Kantorovich value. Conversely, given a coupling $\pi$ and a measurable selection $(x,y)\mapsto\omega^{x,y}$ with action arbitrarily close to $c_\Aa(x,y)$, the mixture of Dirac path laws
\index{measurable selection}
\index{path-space!formulation}
\index{path!space}
	\[
		M=\int\delta_{\omega^{x,y}}\,\d\pi(x,y)
	\]
	has endpoints $\al$ and $\be$ and action equal, up to the selected approximation error, to $\int c_\Aa\,\d\pi$. Optimizing over $\pi$ and letting the approximation error vanish proves equality. If exact minimizing paths are selected for an optimal $\pi^\star$, the displayed $M^\star$ is optimal.
\end{proof}

Thus the classical coupling problem can be read as a path problem where the endpoints are chosen first and the connecting path is then selected with minimal action. The Schr\"odinger problem changes exactly this last step: between two endpoints, it keeps the random fluctuations of a reference dynamics instead of collapsing onto a deterministic least-action path.
\index{Schrodinger!problem}

\paragraph{Entropic path-space problem.}
\index{path-space!formulation}
\index{path-space!problem}

Let $\Rr^\epsilon\in\Pp(\Om)$ be a reference path law, for instance a Brownian or Langevin dynamics at noise level $\epsilon$. Schr\"odinger's dynamic problem is the entropy projection
\index{Langevin dynamics}
\begin{equation}\label{eq-schrodinger-path-space}
\index{path-space!formulation}
	\mathrm{SB}_\epsilon(\al,\be)
	\eqdef
	\inf_{M\in\Pp(\Om)}
	\enscond{
		\epsilon\KL(M|\Rr^\epsilon)
	}{
		(e_0)_\sharp M=\al,\ (e_1)_\sharp M=\be
	}.
\end{equation}
This is the dynamic Schr\"odinger bridge problem. It asks for the most likely path law, relative to the prior dynamics $\Rr^\epsilon$, among all path laws matching the observed endpoint marginals. This viewpoint goes back to Schr\"odinger's reciprocal problem~\cite{Schroedinger31} and is surveyed in modern OT language in~\cite{leonard2012schrodinger,LeonardSchroedinger}; stochastic-control formulations are developed in~\cite{chen2016relation}.
\index{Schrodinger!bridge}

\paragraph{Viscous Benamou--Brenier formulations.}
\index{viscous Benamou-Brenier formulation}
The Schr\"odinger interpolation also admits dynamic optimal-control descriptions, which are viscous analogues of the Benamou--Brenier formula. Write $\sigma$ for the diffusion normalization in this paragraph; it is independent of the entropic temperature $\epsilon$ used elsewhere. Formally, for smooth positive densities and with the convention
\index{Schrodinger!interpolation}
\index{Benamou-Brenier}
\[
	\partial_t\rho_t+\diverg(\rho_t v_t)
	=
	\frac{\sigma}{2}\Delta\rho_t,
\]
one minimizes, among curves joining the prescribed endpoint densities, the kinetic action
\index{kinetic action}
\[
	\int_0^1\!\int
	\frac12\norm{v_t(x)}^2\rho_t(x)\d x\d t .
\]
Equivalently, one can absorb the diffusion into the velocity by writing
\[
	u_t=v_t-\frac{\sigma}{2}\nabla\log\rho_t,
	\qquad
	\partial_t\rho_t+\diverg(\rho_t u_t)=0.
\]
Expanding $v_t=u_t+\frac{\sigma}{2}\nabla\log\rho_t$ and using the continuity equation for $(\rho_t,u_t)$ gives
\index{continuity equation}
	\begin{align*}
		\int_0^1\!\int
		\frac12\norm{v_t}^2\rho_t\,\d x\,\d t
		&=
		\int_0^1\!\int
		\left(
			\frac12\norm{u_t}^2
		+
		\frac{\sigma^2}{8}\norm{\nabla\log\rho_t}^2
	\right)\rho_t\,\d x\,\d t
		\\
		&\quad+
		\frac{\sigma}{2}
		\left[
			\int\rho_1\log\rho_1\,\d x
		-
			\int\rho_0\log\rho_0\,\d x
		\right].
	\end{align*}
Since the entropy term depends only on the prescribed endpoints, the same minimizers are obtained from the modified Benamou--Brenier action
\index{Benamou-Brenier}
\[
	\int_0^1\!\int
	\left(
		\frac12\norm{u_t(x)}^2
		+
		\frac{\sigma^2}{8}\norm{\nabla\log\rho_t(x)}^2
	\right)
	\rho_t(x)\d x\d t .
\]
Thus the Schr\"odinger bridge is a least-action interpolation with both transport kinetic energy and a Fisher-information penalty. If one instead writes the viscous equation with diffusion coefficient $\sigma\Delta\rho_t$, the same formula is obtained after replacing $\sigma$ above by $2\sigma$, so the Fisher coefficient becomes $\sigma^2/2$.
\index{Fisher information}
\index{Schrodinger!bridge}

The reduction to endpoint couplings follows from disintegration. We write
\index{disintegration}
\index{endpoint coupling}
\[
	\Rr^\epsilon(\d\omega)
	=
	\int \Rr^{\epsilon,x,y}(\d\omega)\,
		\Rr^\epsilon_{01}(\d x,\d y),
	\qquad
	\Rr^\epsilon_{01}\eqdef(e_0,e_1)_\sharp\Rr^\epsilon,
\]
where $\Rr^{\epsilon,x,y}$ is the reference bridge conditioned on endpoints $(x,y)$. Similarly, for any feasible $M$, write
\[
	M(\d\omega)
	=
	\int M^{x,y}(\d\omega)\,\pi(\d x,\d y),
	\qquad
	\pi\eqdef(e_0,e_1)_\sharp M.
\]

\begin{prop}[Endpoint reduction of the Schr\"odinger problem]\label{prop-schrodinger-endpoint-reduction}
\index{Schrodinger!problem}
	Assume that the regular conditional laws above exist and that the relative-entropy chain rule applies, with value $+\infty$ when absolute continuity fails. Then
\index{conditional law}
\index{absolute continuity}
\index{entropy!relative}
	\begin{equation}\label{eq-schrodinger-static-endpoint}
		\mathrm{SB}_\epsilon(\al,\be)
		=
		\inf_{\pi\in\Couplings(\al,\be)}
		\epsilon\KL(\pi|\Rr^\epsilon_{01}).
	\end{equation}
	For a fixed endpoint coupling $\pi$ with finite $\KL(\pi|\Rr^\epsilon_{01})$, the minimizing path law is the mixture of reference bridges
\index{endpoint coupling}
	\begin{equation}\label{eq-schrodinger-bridge-mixture}
		M^\pi
		=
		\int \Rr^{\epsilon,x,y}\,\d\pi(x,y).
	\end{equation}
	Consequently, if $\pi^\star$ solves~\eqref{eq-schrodinger-static-endpoint}, then $M^{\pi^\star}$ solves the path-space problem~\eqref{eq-schrodinger-path-space}.
\index{path-space!formulation}
\index{path!space}
\end{prop}

\begin{proof}
	If $M$ has finite relative entropy with respect to $\Rr^\epsilon$, then $\pi=(e_0,e_1)_\sharp M$ is necessarily absolutely continuous with respect to $\Rr^\epsilon_{01}$. The chain rule for relative entropy gives
\index{entropy!relative}
	\begin{equation}\label{eq-kl-chain-rule-path}
		\KL(M|\Rr^\epsilon)
		=
		\KL(\pi|\Rr^\epsilon_{01})
		+
		\int_{\X\times\X}
			\KL(M^{x,y}|\Rr^{\epsilon,x,y})
			\,\d\pi(x,y).
	\end{equation}
	The second term is nonnegative and vanishes exactly when $M^{x,y}=\Rr^{\epsilon,x,y}$ for $\pi$-almost every $(x,y)$. Thus, once an endpoint coupling $\pi$ with finite $\KL(\pi|\Rr^\epsilon_{01})$ is fixed, the best path law is~\eqref{eq-schrodinger-bridge-mixture}, and the remaining minimization is precisely~\eqref{eq-schrodinger-static-endpoint}. If no such endpoint coupling exists, both sides are $+\infty$.
\index{Schrodinger!bridge}
\index{endpoint coupling}
\end{proof}

This proposition makes the connection between the dynamic and static views precise. The static Schr\"odinger problem stores only the endpoints of the optimal random trajectories; the full bridge is recovered by filling each transported endpoint pair with the corresponding reference bridge. In the zero-noise limit, these bridges concentrate on least-action paths and one recovers the unregularized path-space transport problem, hence the Monge--Kantorovich problem~\cite{leonard2012schrodinger}.
\index{Kantorovich!problem}
\index{path-space!formulation}
\index{path-space!transport}
\index{Schrodinger!problem}

\paragraph{Brownian bridges and Sinkhorn couplings.}
\index{bridge!Brownian}
\index{Sinkhorn!coupling}

For $\X=\RR^d$, take $\Rr^\epsilon$ to be a Brownian reference dynamics, up to the conventional scaling of $\epsilon$. Its endpoint law has a heat-kernel density of the form
\[
	p_\epsilon(x,y)\propto \exp\!\left(-\frac{\norm{x-y}^2}{\epsilon}\right)
\]
after absorbing harmless constants into $\epsilon$. More generally, suppose that the endpoint prior can be written, up to a normalization constant independent of $\pi$, as
\[
	\Rr^\epsilon_{01}(\d x,\d y)
	\propto
	\exp\!\left(-\frac{c(x,y)}{\epsilon}\right)
	\al(\d x)\be(\d y).
\]
This includes the usual heat-kernel reference after rewriting it with respect to $\al\otimes\be$, whenever the one-time endpoint densities are fixed and mutually absolutely continuous. The additional one-body density factors only add constants under the marginal constraints.
\index{marginal!constraint}
Then, for every $\pi\in\Couplings(\al,\be)$,
\[
	\epsilon\KL(\pi|\Rr^\epsilon_{01})
	=
	\int c(x,y)\,\d\pi(x,y)
	+
	\epsilon\KL(\pi|\al\otimes\be)
	+
	\mathrm{constant},
\]
where the constant does not depend on $\pi$. Hence~\eqref{eq-schrodinger-static-endpoint} is exactly the continuous Sinkhorn problem~\eqref{eq-entropic-generic}, up to an additive constant in the value. The optimal static coupling $\pi_\epsilon^\star$ is the endpoint law of the most likely controlled noisy dynamics, while the associated path law is
\[
	M_\epsilon^\star
	=
	\int \Rr^{\epsilon,x,y}\,\d\pi_\epsilon^\star(x,y).
\]
In words, Sinkhorn computes which endpoints should be paired; the path-space Schr\"odinger bridge then connects each paired endpoint by a Brownian bridge rather than by a deterministic straight line.
\index{path-space!formulation}
\index{Schrodinger!bridge}
\index{bridge!Brownian}
\index{path!space}

Figure~\ref{fig:sinkhorn-path-space-bridges} illustrates this endpoint-to-path lifting on a small discrete example. The red atoms form a compact source cloud and the blue atoms surround them. The first panel uses the unregularized OT coupling and zero bridge noise; the next panels increase $\epsilon$, which simultaneously softens the endpoint coupling and amplifies the Brownian fluctuations between paired endpoints.
\index{path-space!formulation}
\index{endpoint coupling}

\begin{alg}[Endpoint-to-path Schr\"odinger lift]\label{alg:schrodinger-endpoint-path-lift}
\index{Schrodinger!bridge}
\index{bridge!Brownian}
\textbf{Input:} Endpoint laws $\alpha,\beta$, cost $c$, regularization $\epsilon>0$, reference bridges $\Rr^{\epsilon,x,y}$.

\textbf{Output:} Schr\"odinger path law $M_\epsilon^\star$.

\textbf{Let} $\pi_\epsilon^\star$ be a minimizer of the static entropic endpoint problem:
\(\pi_\epsilon^\star \in \argmin_{\pi\in\Couplings(\al,\be)} \int c\,\d\pi+\epsilon\KL(\pi|\al\otimes\be).\)

\textbf{For} each endpoint pair $(x,y)$ sampled from $\pi_\epsilon^\star$ \textbf{do}:
\begin{algblock}

\textbf{Draw} bridge path:
\(\omega\sim\Rr^{\epsilon,x,y}.\)
\end{algblock}
\algreturnskip
\textbf{Return} \(M_\epsilon^\star= \int \Rr^{\epsilon,x,y}\,\d\pi_\epsilon^\star(x,y).\)
\end{alg}

\begin{figure}[H]
\centering
\setlength{\tabcolsep}{2pt}
\begin{tabular}{@{}cccc@{}}
\includegraphics[width=.235\linewidth]{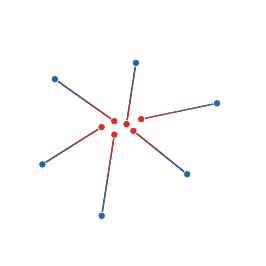} &
\index{path!space}
\includegraphics[width=.235\linewidth]{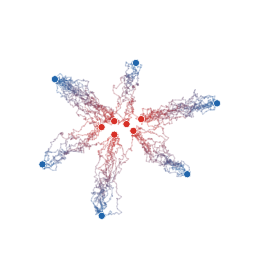} &
\includegraphics[width=.235\linewidth]{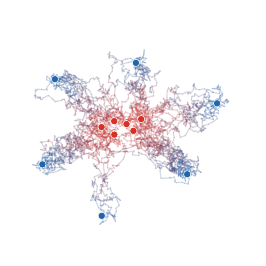} &
\includegraphics[width=.235\linewidth]{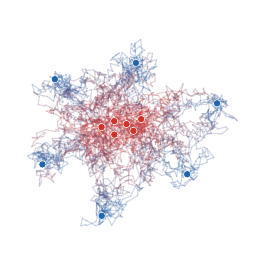} \\[-.1em]
\small $\epsilon=0$ &
\small $\epsilon=0.035$ &
\small $\epsilon=0.10$ &
\small $\epsilon=0.28$
\end{tabular}
\caption{Endpoint couplings lifted to Brownian bridges. Each panel transports six equally weighted red atoms concentrated near the center to six blue atoms distributed around them. For the displayed value of $\epsilon$, the endpoint coupling is either the exact quadratic OT plan ($\epsilon=0$) or the entropic Sinkhorn plan ($\epsilon>0$). A total of 60 paths is sampled by a multinomial allocation with probabilities $\pi_{ij}$, and each selected pair $(x_i,y_j)$ is filled by a Brownian bridge with covariance proportional to $\epsilon t(1-t)$. Colors encode time from red to blue.}
\index{endpoint coupling}
\index{bridge!Brownian}
\label{fig:sinkhorn-path-space-bridges}
\end{figure}

\begin{defn}[Mutual information]\label{def-mutual-information}
\index{mutual information}
	If $(X,Y)\sim\pi$ have marginals $X\sim\al$ and $Y\sim\be$, the mutual information of the pair is
	\[
		\Ii(X,Y)\eqdef \KL(\pi|\al\otimes\be).
	\]
	It is nonnegative and vanishes if and only if $X$ and $Y$ are independent.
\end{defn}
With this terminology, the entropic problem~\eqref{eq-entropic-generic} is equivalent to
\index{mutual information}
\[
	\inf_{X\sim\al,\;Y\sim\be}
	\EE\bigl(c(X,Y)\bigr)+\epsilon\Ii(X,Y).
\]
Large $\epsilon$ therefore favors nearly independent endpoints, while small $\epsilon$ suppresses endpoint randomness and recovers an optimal Monge--Kantorovich coupling in the limit. When the unregularized quadratic problem has a Brenier map, this limiting coupling is deterministic.
\index{Brenier!map}


\section{Dual of Sinkhorn}
\index{Sinkhorn!dual}

The dual point of view replaces couplings by potentials and soft $c$-transforms. It is the right formulation for stabilized implementations and differentiation.
\index{soft!c-transform}

\paragraph{Discrete dual.}

The following proposition details the dual problem associated with the KL-normalized formulation~\eqref{eq-regularized-discr-rescaled}. This formulation has the same minimizer as~\eqref{eq-regularized-discr}; its optimal value is shifted by the constant $\epsilon\HD(\a)+\epsilon\HD(\b)$.
\index{dual!problem}

\begin{prop}[Dual of entropic OT]
\index{entropic!OT}
The optimal value of~\eqref{eq-regularized-discr-rescaled} is
\eql{\label{eq-dual-formulation}
	\umin{\P \in \CouplingsD(\a,\b)}
		\dotp{\P}{\C}+\epsilon\KLD(\P|\a\otimes\b)
	=
	\umax{\fD \in \RR^n,\gD \in \RR^m}
		 \dotp{\fD}{\a} + \dotp{\gD}{\b}
	- \epsilon \sum_{i,j} \exp\pa{
		\frac{\fD_i+\gD_j-\C_{i,j}}{\epsilon}
	} \a_i \b_j + \epsilon.
}
The optimal $(\fD,\gD)$ are linked to scalings $(\uD,\vD)$ appearing in~\eqref{eq-scaling-form} through
\index{scaling!form}
\eql{\label{eq-entropy-pd}
		\uD_i=\a_i e^{\fD_i/\epsilon}
		\qandq
		\vD_j=\b_j e^{\gD_j/\epsilon}.
}
\end{prop}

\begin{proof}
We introduce Lagrange multipliers and consider
\eq{
	\umin{\P \geq 0} \umax{\fD,\gD} \dotp{\C}{\P} + \epsilon \KLD(\P|\a \otimes \b) + \dotp{\a-\P\ones}{\fD} + \dotp{\b-\P^\top\ones}{\gD}.
}
	Finite-dimensional convex duality allows us to exchange the minimum over $\P$ with the maximum over $(\fD,\gD)$, giving
	\eq{
		 \umax{\fD,\gD}
		 \dotp{\fD}{\a} + \dotp{\gD}{\b}
		 +
		 \epsilon \umin{\P \geq 0}
		 \pa{
			\KLD(\P|\a \otimes \b)
			-
			\dotp{\frac{\fD\oplus\gD-\C}{\epsilon}}{\P}
		 }
			=
			\dotp{\fD}{\a} + \dotp{\gD}{\b} - \epsilon \KLD^*\pa{ \frac{\fD\oplus\gD-\C}{\epsilon}|\a \otimes \b }.
	}
	One concludes by using~\eqref{eq-legendre} for $\phi(r)=r \log(r)-r+1$
	\eq{
		\KLD^*(H|\a \otimes \b) = \sum_{i,j} \phi^*(H_{i,j}) \a_i \b_j.
	}
	Indeed, the scalar maximization
	\[
		\phi^*(s)=\sup_{r\geq0}\{rs-r\log r+r-1\}
	\]
	has first-order condition $s-\log r=0$, hence $r=e^s$ and $\phi^*(s)=e^s-1$.
\end{proof}

\paragraph{Discrete soft $c$-transforms.}
\index{soft!c-transform}

Since the dual problem~\eqref{eq-dual-formulation} is smooth and concave, one can perform alternating block maximization. For a fixed $\gD$, maximizing with respect to $\fD$ leads to the following equation after zeroing the derivative with respect to $\fD$:
\index{dual!problem}
\eq{
	\a_i - e^{\frac{\fD_i}{\epsilon}} \a_i \sum_j \exp\pa{
		\frac{\gD_j-\C_{i,j}}{\epsilon}
	} \b_j = 0
}
which leads to the explicit solution
\eq{
	\fD_i = -\epsilon \log \sum_j \exp\pa{ \frac{\gD_j-\C_{i,j}}{\epsilon} } \b_j.
}
The log-sum-exp form is best read as a smoothed minimum. This gives the entropic analogue of the hard $c$-transform.

\begin{defn}[Soft-min and discrete soft $c$-transform]\label{def-discrete-soft-c-transform}
\index{soft!minimum}
\index{soft!c-transform}
\index{c-transform}
		For $h\in\RR^m$ and weights $\b\in\simplex_m$, the weighted soft-min at temperature $\epsilon>0$ is
		\[
			{\min}_{\b}^\epsilon(h) \eqdef -\epsilon \log \sum_j e^{-h_j/\epsilon} \b_j.
		\]
		It converges to $\min_j h_j$ as $\epsilon\to0$. Given a cost matrix $\C$, the discrete soft $c$-transforms are
	\eql{\label{eq-soft-c-1}
		\fD_i = {\min}_{\b}^\epsilon( \C_{i,\cdot} - \gD )
	}
	and
	\eql{\label{eq-soft-c-2}
		\gD_j = {\min}_{\a}^\epsilon( \C_{\cdot,j} - \fD ).
	}
\end{defn}
Exponentiating these iterations recovers exactly the Sinkhorn algorithm. These iterations, however, become unstable for small $\epsilon$. To apply the algorithm in this regime, one needs to stabilize it using the celebrated log-sum-exp trick. This follows from noticing that, similarly to the minimum operator, one has
\index{log-sum-exp}
\index{Sinkhorn!algorithm}
\eq{
	{\min}_{\b}^\epsilon(h-\text{cst})={\min}_{\b}^\epsilon(h)-\text{cst}
}
and to replace the computation of ${\min}_{\b}^\epsilon(h)$ by its stabilized version (equal when using infinite precision computation) ${\min}_{\b}^\epsilon(h-\min(h)) + \min(h)$.

\begin{alg}[Log-domain Sinkhorn by soft transforms]\label{alg:log-domain-sinkhorn}
\textbf{Input:} Weights $\a,\b$, cost matrix $\C$, regularization $\epsilon>0$, tolerance $\mathrm{tol}$.

\textbf{Output:} Entropic coupling $\P$ computed from stabilized potentials.

\textbf{Initialize:} Set $\gD^{(0)}=0$, \(\eta_0=+\infty\), and \(k=0\).

\textbf{While} \(\eta_k>\mathrm{tol}\) \textbf{do}:
\begin{algblock}

\textbf{Set} \(k\leftarrow k+1\).

\textbf{Compute} stabilized soft transform:
\(\fD_i^{(k)} = -\epsilon\log\sum_j \exp\!\left(\frac{\gD_j^{(k-1)}-\C_{ij}}{\epsilon}\right)\b_j.\)

\textbf{Compute} stabilized reverse soft transform:
\(\gD_j^{(k)} = -\epsilon\log\sum_i \exp\!\left(\frac{\fD_i^{(k)}-\C_{ij}}{\epsilon}\right)\a_i.\)

\textbf{Set} \(\eta_k=\max\{\norm{\fD^{(k)}-\fD^{(k-1)}}_\infty,\norm{\gD^{(k)}-\gD^{(k-1)}}_\infty\}\), with the first term ignored for \(k=1\).
\end{algblock}
\algreturnskip
\textbf{Return} \(\P_{ij} = \a_i\b_j \exp\!\left(\frac{\fD_i^{(k)}+\gD_j^{(k)}-\C_{ij}}{\epsilon}\right).\)
\end{alg}

\begin{figure}[ht]
\centering
\setlength{\tabcolsep}{2pt}
\begin{tabular}{@{}cccc@{}}
\small hard transform & \small $\epsilon=.55$ & \small $\epsilon=.14$ & \small $\epsilon=.035$ \\[-.15em]
\includegraphics[width=.235\linewidth]{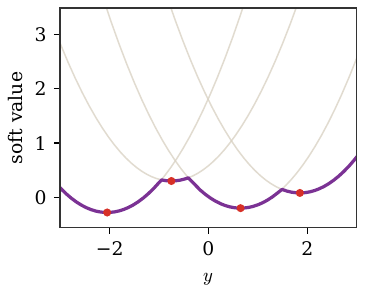} &
\index{soft!c-transform}
\includegraphics[width=.235\linewidth]{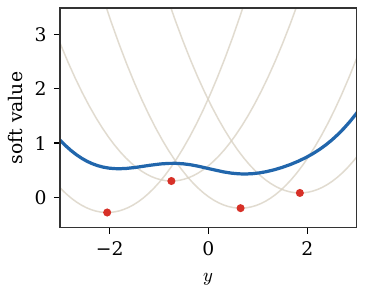} &
\index{c-transform}
\includegraphics[width=.235\linewidth]{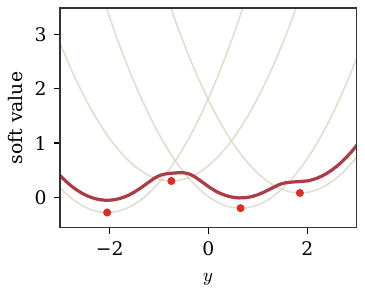} &
\includegraphics[width=.235\linewidth]{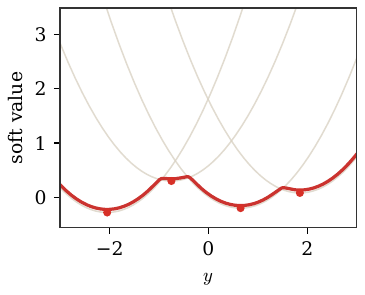}
\end{tabular}
\caption{Soft $c$-transforms for decreasing temperatures. The hard transform is the lower envelope of shifted cost functions, while a positive $\epsilon$ replaces the pointwise minimum by a log-sum-exp soft minimum. As $\epsilon$ decreases, the soft transform sharpens and approaches the non-smooth hard envelope.}
\index{soft!c-transform}
\index{c-transform}
\index{log-sum-exp}
\index{soft!transform}
\index{soft!minimum}
\label{fig:sinkhorn-soft-c-transform-epsilon}
\end{figure}

\paragraph{Continuous dual and soft-transforms.}
\index{continuous!dual}
\index{soft!transform}

For general, not necessarily discrete, measures $(\al,\be)$, the KL-regularized problem~\eqref{eq-entropic-generic} has the concave dual
\eql{\label{eq-dual-sinkh-cont}
	\MK_\c^\epsilon(\al,\be)
	=
	\usup{f\in\Cc(\Xx),\,g\in\Cc(\Yy)}
	\mathcal D_\epsilon(f,g),
}
where
\eql{\label{eq-dual-sinkhorn-objective}
	\mathcal D_\epsilon(f,g)
	\eqdef
	\int_{\Xx} f(x)\d\al(x)
	+
	\int_{\Yy} g(y)\d\be(y)
	-
	\epsilon
	\int_{\Xx\times\Yy}
	\pa{
		e^{\frac{f(x)+g(y)-c(x,y)}{\epsilon}}-1
	}
	\d\al(x)\d\be(y).
}
This is the smooth counterpart of the hard feasibility constraint $f\oplus g\leq c$ from the Kantorovich dual: violations are penalized exponentially and disappear in the limit $\epsilon\to0$.

The corresponding soft $c$-transforms are the exact block maximizers of this dual objective. They are the continuous log-integral counterparts of Definition~\ref{def-discrete-soft-c-transform}.

\begin{defn}[Continuous soft $c$-transforms]\label{def-continuous-soft-c-transform}
\index{soft!c-transform}
\index{soft!transform}
	For $f\in\Cc(\Xx)$ and $g\in\Cc(\Yy)$, define
	\begin{align}
		f^{c,\epsilon}(y)
		&\eqdef
		-\epsilon\log\!\left(
			\int_{\Xx}
			e^{\frac{f(x)-c(x,y)}{\epsilon}}\d\al(x)
		\right),
		\qquad y\in\Yy, \label{eq-soft-c-cont-f}\\
		g^{\bar c,\epsilon}(x)
		&\eqdef
		-\epsilon\log\!\left(
			\int_{\Yy}
			e^{\frac{g(y)-c(x,y)}{\epsilon}}\d\be(y)
		\right),
		\qquad x\in\Xx . \label{eq-soft-c-cont-g}
	\end{align}
\end{defn}
In the case of discrete measures, these formulas reduce to~\eqref{eq-soft-c-1} and~\eqref{eq-soft-c-2}. The same calculus also gives the entropic semi-discrete problem: when one measure is discrete, one alternates between a finite-dimensional potential and an integral soft transform, often estimated stochastically.
\index{soft!transform}
\index{semi-discrete!OT}

\begin{prop}[Existence and uniqueness of entropic dual potentials]\label{prop-entropic-dual-potentials}
\index{dual!potential}
	Assume that $\Xx$ and $\Yy$ are compact and that $c$ is continuous. The dual problem~\eqref{eq-dual-sinkh-cont} has solutions, and the set of solutions is of the form
\index{dual!problem}
	\[
		(f^\star+\lambda,g^\star-\lambda),
		\qquad \lambda\in\RR .
	\]
\end{prop}

\begin{proof}
	Normalize potentials by imposing $\int f\d\al=0$. Replacing any pair $(f,g)$ by the corresponding soft $c$-transforms does not decrease the dual objective, because each soft transform is the exact maximizer in one block variable. For transformed potentials, the oscillations are bounded by the oscillation of the cost:
\index{soft!c-transform}
\index{soft!transform}
	\[
		\norm{f}_V+\norm{g}_V
		\leq
		2(\sup c-\inf c).
	\]
	Moreover the modulus of continuity of the soft transforms is controlled by that of $c$; for instance
\index{soft!transform}
	\[
		\abs{g^{\bar c,\epsilon}(x)-g^{\bar c,\epsilon}(x')}
		\leq
		\sup_y\abs{c(x,y)-c(x',y)}.
	\]
	After normalization, maximizing sequences are therefore uniformly bounded and equicontinuous. Arzel\`a--Ascoli gives a uniformly converging subsequence, and continuity of~\eqref{eq-dual-sinkhorn-objective} gives a maximizer.

	For uniqueness, use strict convexity of
\index{strict!convexity}
	\[
		H\mapsto \int e^{H/\epsilon}\d(\al\otimes\be)
	\]
	on the image of $(f,g)\mapsto H=f\oplus g-c$, modulo constants. If two optimal pairs exist, their midpoint is also optimal; strict convexity forces the two functions $f\oplus g$ to agree $\al\otimes\be$-almost everywhere. Since the potentials are continuous on compact supports, this implies that $f-f'$ is constant and $g-g'$ is the opposite constant.
\index{strict!convexity}
\end{proof}

\begin{rem}[Convexity properties of soft transforms]\label{rem-soft-transform-convexity}
\index{soft!transform}
	The log-sum-exp part behaves like a smoothed maximum and preserves convexity. Since the soft transform takes the negative of this quantity after inserting the cost, it preserves the usual $c$-concavity structure. In particular, for the bilinear cost $c(x,y)=-\dotp{x}{y}$, the transform $f^{c,\epsilon}$ is concave for any $f$. Therefore, for the quadratic cost $c(x,y)=\norm{x-y}^2/2$, the optimal potentials have the form $f^\star(x)=\norm{x}^2/2-\phi^\star(x)$ and $g^\star(y)=\norm{y}^2/2-\psi^\star(y)$, where $\phi^\star$ and $\psi^\star$ are convex.
\index{log-sum-exp}
\index{cost!quadratic}
\index{c-concavity}
\end{rem}

\begin{rem}[Gaussian marginals]\label{rem-sinkhorn-gaussian-marginals}
\index{Gaussian!marginals}
	For $c(x,y)=\norm{x-y}^2$ and Gaussian marginals, the soft transforms preserve quadratic functions, because products and convolutions of Gaussian functions remain Gaussian. Hence optimal entropic potentials are quadratic and the optimal entropic coupling is Gaussian. Section~\ref{sec-gaussian-sinkhorn} makes this finite-dimensional closure explicit.
\index{entropic!potential}
\index{soft!transform}
\end{rem}


\section{Other Convex Regularizers}
\index{convex!regularizer}
\label{sec-sinkhorn-other-regularizers}

KL regularization is the case that leads to multiplicative Sinkhorn scalings. Replacing the KL divergence by another density-ratio penalty keeps the same transport constraints but changes the scalar law linking the optimal density to the dual potentials.
\index{Sinkhorn!scaling}
\index{dual!potential}
\index{density!ratio}

Let $\phi$ be an entropy function in the sense of Definition~\ref{def_entropy}, and recall the $\phi$-divergence $\Divergm_\phi$ from Definition~\ref{def_divergence}. For $\epsilon>0$, define the $\phi$-regularized transport value
\index{entropy!function}
\index{phi-divergence}
\eql{\label{eq-phi-regularized-ot}
\index{phi-divergence regularized OT}
	\MK_{\c,\phi}^{\epsilon}(\al,\be)
	\eqdef
	\umin{\pi\in\Couplings(\al,\be)}
	\int_{\Xx\times\Yy} c(x,y)\d\pi(x,y)
	+
	\epsilon\Divergm_\phi(\pi|\al\otimes\be).
}

\begin{prop}[Dual and density law for $\phi$-regularized OT]\label{prop-phi-regularized-ot-dual}
\index{phi-divergence regularized OT}
	Under the usual Fenchel--Rockafellar qualification assumptions, for instance compact spaces, continuous $c$, and finite value in~\eqref{eq-phi-regularized-ot}, one has
\index{duality!Fenchel-Rockafellar}
\index{Fenchel duality}
	\eql{\label{eq-phi-regularized-ot-dual}
		\MK_{\c,\phi}^{\epsilon}(\al,\be)
		=
		\usup{f\in\Cc(\Xx),\,g\in\Cc(\Yy)}
		\int f\d\al+\int g\d\be
		-
		\epsilon
		\int
		\phi^{*,\geq 0}\!\left(\frac{f(x)+g(y)-c(x,y)}{\epsilon}\right)
		\d\al(x)\d\be(y).
	}
	If an optimal plan has density $r^\star=\frac{\d\pi^\star}{\d(\al\otimes\be)}$ and optimal potentials $(f^\star,g^\star)$, then
\index{optimal plan}
	\[
		\frac{f^\star(x)+g^\star(y)-c(x,y)}{\epsilon}
		\in
		\partial\phi(r^\star(x,y))
		\qquad
		\al\otimes\be\text{-a.e.}
	\]
	In the smooth interior this reads
	\[
		r^\star(x,y)
		=
		(\phi')^{-1}
		\!\left(\frac{f^\star(x)+g^\star(y)-c(x,y)}{\epsilon}\right).
	\]
\end{prop}

\begin{proof}
	Introduce dual variables $(f,g)$ for the two marginal constraints. For fixed $(f,g)$, the minimization over $\pi$ gives
\index{marginal!constraint}
	\[
		\int f\d\al+\int g\d\be
		+
		\inf_{\pi\in\Mm_+(\Xx\times\Yy)}
		\left\{
		\int\big(c-(f\oplus g)\big)\d\pi
		+
		\epsilon\Divergm_\phi(\pi|\al\otimes\be)
		\right\}.
	\]
	Using the Legendre formula~\eqref{eq-legendre} for the convex functional $\Divergm_\phi(\cdot|\al\otimes\be)$, the infimum equals
\index{convex!function}
	\[
		-\epsilon
		\int
		\phi^{*,\geq0}\!\left(\frac{f(x)+g(y)-c(x,y)}{\epsilon}\right)
		\d\al(x)\d\be(y),
	\]
	which gives~\eqref{eq-phi-regularized-ot-dual}. Equality in the Fenchel inequality is equivalent to the subgradient inclusion
\index{phi-divergence regularized OT}
	\[
		\frac{f^\star\oplus g^\star-c}{\epsilon}
		\in
		\partial\phi(r^\star),
	\]
	and inversion of $\phi'$ gives the density law when $\phi$ is differentiable and the optimizer is in the interior of its domain.
\end{proof}

For the KL entropy $\phi(r)=r\log r-r+1$, one has $\phi^{*,\geq0}(s)=e^s-1$. Taking this parameter to be the Sinkhorn temperature $\epsilon$ in~\eqref{eq-phi-regularized-ot-dual} recovers exactly the continuous Sinkhorn dual~\eqref{eq-dual-sinkhorn-objective}. Other choices replace the exponential law by another scalar transfer function:
\index{Sinkhorn!dual}
\index{continuous!Sinkhorn dual}
\index{phi-divergence regularized OT}
\[
\begin{array}{lll}
	\phi(r)=r\log r-r+1 &\Rightarrow& r^\star=e^s,\\[.25em]
	\phi(r)=r-\log r-1 &\Rightarrow& r^\star=(1-s)^{-1}\quad (s<1),\\[.25em]
	\phi(r)=\frac12(r-1)^2 &\Rightarrow& r^\star=(1+s)_+,
\end{array}
\qquad
s\eqdef\frac{f^\star\oplus g^\star-c}{\epsilon}.
\]

\begin{figure}[ht]
\centering
\setlength{\tabcolsep}{2pt}
\begin{tabular}{@{}ccc@{}}
\includegraphics[width=.31\linewidth]{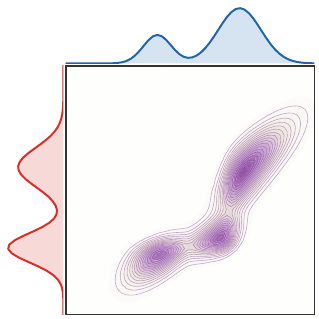} &
\includegraphics[width=.31\linewidth]{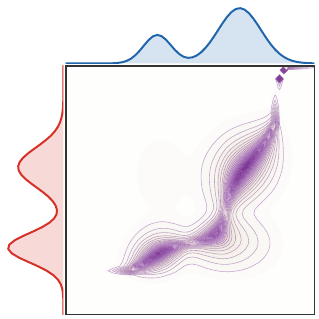} &
\includegraphics[width=.31\linewidth]{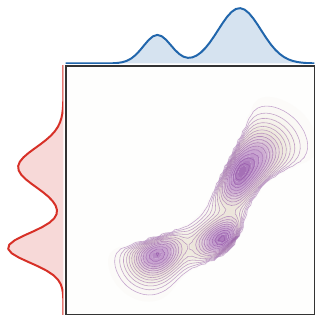} \\[-.1em]
\small $\phi(r)=r\log r-r+1$ &
\small $\phi(r)=r-\log r-1$ &
\small $\phi(r)=\frac12(r-1)^2$
\end{tabular}
\caption{Density-ratio regularizers and coupling support. The red and blue side curves are fixed one-dimensional Gaussian-mixture marginals. Dense violet level sets display the square root of the coupling density on a common scale, using the smaller regularization strength $\epsilon=.06$. KL regularization gives the usual diffuse positive plan, the Burg barrier keeps a positive but differently tailed support, and the quadratic density penalty can set entries exactly to zero through its positive-part law.}
\index{Gaussian mixture}
\index{density!ratio}
\label{fig:sinkhorn-entropic-versus-quadratic-regularization}
\end{figure}

\paragraph{Bregman vs. $\phi$-divergence regularization.}
\index{Bregman!regularization}
\index{phi-divergence}

The previous construction regularizes OT by a density-ratio divergence. This differs from using a Bregman divergence generated by a convex functional on the space of measures.
\index{convex!function}
\index{Bregman!divergence}
\index{reference!measure}
\index{density!ratio}

\Needspace{7\baselineskip}
\begin{defn}[Measure Bregman divergence]\label{def-measure-bregman-divergence}
\index{Bregman!divergence}
\index{first variation}
	If $\Phi$ is a differentiable convex functional on a convex class of nonnegative measures and $\xi$ is a reference measure in its domain, the measure Bregman divergence generated by $\Phi$ is
	\begin{equation}\label{eq-measure-bregman-divergence}
		B_\Phi(\pi|\xi)
		\eqdef
		\Phi(\pi)-\Phi(\xi)
		-
		\int \delta\Phi(\xi)\d(\pi-\xi),
	\end{equation}
	where $\delta\Phi(\xi)$ is the first variation and the formula is understood whenever the right-hand side is well-defined.
\end{defn}
In finite dimension this reduces to Definition~\ref{def-bregman-divergence}. The KL divergence is special because it is simultaneously a density-ratio divergence and a Bregman divergence.
\index{first variation}
\index{Wasserstein!gradient flow}
\index{Bregman!divergence}
\index{density!ratio}

\begin{prop}[Dual comparison: Bregman vs. density-ratio penalties]\label{prop-bregman-phi-dual-comparison}
	Fix the marginals $\alpha,\beta$ and set $\xi\eqdef\alpha\otimes\beta$. Let $\Phi$ be a convex Gateaux-differentiable functional on nonnegative measures on $\Xx\times\Yy$. Its convex conjugate is, for a continuous test function $u$,
\index{nonnegative measure}
\index{convex!conjugate}
	\[
		\Phi^*(u)
		\eqdef
		\sup_{\pi\geq0}
		\left\{
			\int u\,\d\pi-\Phi(\pi)
		\right\}.
	\]
	Define the Bregman-regularized value, using the same product reference as the density-ratio penalty,
\index{density!ratio}
	\[
		\MK_{\c,\Phi}^{\epsilon}(\alpha,\beta)
		\eqdef
		\inf_{\pi\in\Couplings(\alpha,\beta)}
		\int c\,\d\pi+\epsilon B_\Phi(\pi|\xi).
	\]
	Assume that Fenchel duality is exact for this constrained problem, as happens in finite-dimensional discretizations and, more generally, under standard compactness and lower-semicontinuity hypotheses. Then
\index{Fenchel duality}
	\eql{\label{eq-bregman-regularized-ot-dual}
		\MK_{\c,\Phi}^{\epsilon}(\alpha,\beta)
		=
		\sup_{f,g}
		\int f\d\alpha+\int g\d\beta
		-
		\epsilon
		\left[
			\Phi^*\!\left(
				\delta\Phi(\xi)+\frac{f\oplus g-c}{\epsilon}
			\right)
			-
			\Phi^*(\delta\Phi(\xi))
		\right].
	}
	If $(f^\star,g^\star)$ and $\pi^\star$ are optimal and the solution is interior, then
	\[
		\delta\Phi(\pi^\star)
		=
		\delta\Phi(\xi)+\frac{f^\star\oplus g^\star-c}{\epsilon}.
	\]
	By contrast, the density-ratio formulation~\eqref{eq-phi-regularized-ot} has the scalar-integral dual~\eqref{eq-phi-regularized-ot-dual} and the pointwise density law
\index{phi-divergence regularized OT}
\index{density!ratio}
	\[
		\frac{f^\star\oplus g^\star-c}{\epsilon}
		\in
		\partial\phi\!\left(
			\frac{\d\pi^\star}{\d(\alpha\otimes\beta)}
		\right).
	\]
\end{prop}

\begin{proof}
	Using~\eqref{eq-measure-bregman-divergence}, the Bregman-regularized objective can be written, up to a constant independent of $\pi$, as
\index{Bregman!divergence}
	\[
		\epsilon\Phi(\pi)
		+
		\int\big(c-\epsilon\delta\Phi(\xi)\big)\d\pi
		-\epsilon\Phi(\xi)+\epsilon\int\delta\Phi(\xi)\d\xi .
	\]
	Introduce dual potentials $(f,g)$ for the two marginal constraints. The inner minimization over nonnegative measures $\pi$ gives
\index{nonnegative measure}
\index{marginal!constraint}
\index{dual!potential}
	\[
		\inf_{\pi\geq0}
		\left\{
			\epsilon\Phi(\pi)
			+
			\int
			\big(c-f\oplus g-\epsilon\delta\Phi(\xi)\big)
			\d\pi
		\right\}
		=
		-\epsilon
		\Phi^*\!\left(
			\delta\Phi(\xi)+\frac{f\oplus g-c}{\epsilon}
		\right).
	\]
	Fenchel equality at $\xi$ gives
	\[
		-\Phi(\xi)+\int\delta\Phi(\xi)\d\xi
		=
		\Phi^*(\delta\Phi(\xi)),
	\]
	which yields~\eqref{eq-bregman-regularized-ot-dual}. Equality in Fenchel's inequality gives the optimality condition for $\pi^\star$. The density-ratio dual and density law are exactly those of Proposition~\ref{prop-phi-regularized-ot-dual}. Placing the two formulas side by side shows the structural difference: Bregman regularization translates the reference measure in the dual coordinate $\delta\Phi$, whereas $\phi$-regularization applies a scalar nonlinearity to the density with respect to the moving product reference $\alpha\otimes\beta$.
\index{phi-divergence regularized OT}
\index{Bregman!regularization}
\index{reference!measure}
\index{density!ratio}
\end{proof}

When $\Phi$ is separable with respect to a fixed dominating measure $\xi_0$, say $\Phi(\pi)=\int h(\d\pi/\d\xi_0)\d\xi_0$ with $\xi\ll\xi_0$, the Bregman optimality condition becomes
\[
	h'\!\left(\frac{\d\pi^\star}{\d\xi_0}\right)
	=
	h'\!\left(\frac{\d\xi}{\d\xi_0}\right)
	+
	\frac{f^\star\oplus g^\star-c}{\epsilon}.
\]
This is an additive update in entropy coordinates. The density-ratio formulation instead uses the scalar law associated with $\phi$ relative to $\alpha\otimes\beta$. These two laws coincide for the KL entropy, where $h'(r)=\log r$ turns additive dual shifts into multiplicative scalings. The next proposition shows that, under natural smoothness assumptions, this is the only overlap.
\index{density!ratio}

\begin{prop}[KL is the common Bregman and $\phi$ case]\label{prop-kl-only-bregman-phi}
	Let $\omega$ be a finite reference measure with a nontrivial measurable subset. Work on probability measures $\alpha=p\omega$ and $\beta=q\omega$ whose densities are bounded above and below away from $0$. Assume that $\Phi$ is twice Gateaux differentiable along bounded zero-mass density perturbations, and that $\phi\in C^2(0,+\infty)$ is convex with $\phi(1)=0$. If
\index{probability measure}
\index{reference!measure}
	\[
		B_\Phi(\alpha|\beta)=\Divergm_\phi(\alpha|\beta)
	\]
	for all such $\alpha,\beta$, then there exist $c\geq0$ and $a\in\RR$ such that
	\[
		\phi(t)=c\,t\log t+a(t-1).
	\]
	Hence the common divergence is $c\KL(\alpha|\beta)$, and $\Phi(p\omega)$ differs from $c\int p\log p\,\d\omega$ by an affine functional on the positive probability simplex.
\index{probability simplex}
\end{prop}

\begin{proof}
	Fix $\beta=q\omega$ and perturb it by $\alpha_t=(q+t h)\omega$, where $h$ is bounded, $\int h\d\omega=0$, and $t$ is small enough that $q+t h>0$. Differentiating twice at $t=0$ gives
	\[
		D^2\Phi(q)[h,h]
		=
		\phi''(1)\int \frac{h^2}{q}\d\omega .
	\]
	Indeed the left-hand side is the second variation of the Bregman error, while
	\[
		\Divergm_\phi(\alpha_t|\beta)
		=
		\int q\,\phi\left(1+t h/q\right)\d\omega
	\]
	has second derivative $\phi''(1)\int h^2/q\,\d\omega$. Setting $c=\phi''(1)\geq0$, the functional $\Psi(p\omega)=\Phi(p\omega)-c\int p\log p\d\omega$ has zero second variation along every zero-mass line segment in the positive simplex. It is therefore affine there, and $B_\Psi=0$. Thus $B_\Phi=c\KL$.

	It remains to identify the generator. Since $\Divergm_\phi=c\KL$, the function $g(t)=\phi(t)-c\,t\log t$ generates the zero $\phi$-divergence. Testing on densities for which the ratio $p/q$ takes two values $x<1<y$, with weights chosen so that the mean ratio is $1$, gives $(y-1)g(x)+(1-x)g(y)=0$. Hence $g(t)/(t-1)$ is constant on $(0,+\infty)\setminus\{1\}$, so $g(t)=a(t-1)$. This proves the claim.
\index{phi-divergence}
\end{proof}

Thus the two generalizations lead to different duals and different algorithms. Bregman regularization by $B_\Phi(\pi|\xi)$ keeps the projection geometry of Section~\ref{sec-convergence-init}: linear costs tilt the reference in dual coordinates and alternating marginal updates are Bregman projections. A density-ratio penalty $\Divergm_\phi(\pi|\alpha\otimes\beta)$ instead gives the Fenchel dual~\eqref{eq-phi-regularized-ot-dual} and, for interior solutions, the pointwise law $r^\star=(\phi')^{-1}((f\oplus g-c)/\epsilon)$. Proposition~\ref{prop-bregman-phi-dual-comparison} makes the distinction explicit at the dual level. The Bregman dual contains the global conjugate $\Phi^*$ and the product reference $\alpha\otimes\beta$, whereas the $\phi$-dual integrates a scalar conjugate against the moving product measure $\alpha\otimes\beta$. Only for KL do these two viewpoints coincide and reduce to multiplicative Sinkhorn scalings.
\index{Sinkhorn!scaling}
\index{phi-divergence regularized OT}
\index{Bregman!regularization}
\index{Bregman!projection}
\index{product!measure}
\index{density!ratio}


\section{Sinkhorn Divergences}
\index{Sinkhorn!divergence}
\label{sec-sinkhorn-div}

Sinkhorn divergences remove the entropic self-bias while retaining smoothness. They interpolate between OT-like geometry and kernel-like norms, which explains their statistical behavior.
\index{Sinkhorn!divergence}

\paragraph{Entropic bias.}
\index{entropic!bias}

A major issue with the value of the Sinkhorn problem~\eqref{eq-entropic-generic} is that $\MK_\c^\epsilon(\al,\be)>0$. In particular,
\eq{
	\al_\epsilon = \uargmin{\be} \MK_\c^\epsilon(\al,\be)
}
does not satisfy $\al_\epsilon=\al$ unless $\epsilon=0$. The following proposition shows that the bias induced by this entropic regularization dominates the large $\epsilon$ limit.
\index{entropic!regularization}

\begin{prop}[Large-temperature entropic bias]
\index{entropic!bias}
		Assume that $c$ is bounded and continuous. Then $\MK_\c^\epsilon(\al,\be) \rightarrow \iint c(x,y)\d\al(x)\d\be(y)$ as $\epsilon \rightarrow +\infty$.
\end{prop}
\begin{proof}
	Let $(f_\epsilon,g_\epsilon)$ be optimal dual potentials, normalized by $\int g_\epsilon\d\beta=0$. The soft $c$-transform equation gives
\index{soft!c-transform}
\index{dual!potential}
	\[
		f_\epsilon(x)
		=
		-\epsilon\log\int
		\exp\!\left(\frac{g_\epsilon(y)-c(x,y)}{\epsilon}\right)\d\beta(y).
	\]
	For bounded $c$, the oscillations of normalized entropic potentials are bounded uniformly in $\epsilon$ by the oscillation of $c$. Hence the log-sum-exp expansion is uniform:
\index{log-sum-exp}
\index{entropic!potential}
	\[
		f_\epsilon(x)
		=
		-\int (g_\epsilon(y)-c(x,y))\d\beta(y)+O(\epsilon^{-1})
		=
		\int c(x,y)\d\beta(y)+O(\epsilon^{-1}).
	\]
	At optimality the exponential penalty in the dual integrates to zero, so
	\[
		\MK_\c^\epsilon(\alpha,\beta)
		=
		\int f_\epsilon\d\alpha+\int g_\epsilon\d\beta
		=
		\iint c(x,y)\d\alpha(x)\d\beta(y)+O(\epsilon^{-1}).
	\]
	This proves the limit.
\end{proof}

So in the large $\epsilon$ limit, $\MK_\c^\epsilon$ behaves like an inner product and not like a norm. The following special case makes the resulting attraction explicit.

\begin{example}[Large-temperature collapse for quadratic costs]
	The limiting functional minimized by $\al_\epsilon$ is linear in the second argument:
	\[
		\be\mapsto \int V_\alpha(y)\d\beta(y),
		\qquad
		V_\alpha(y)\eqdef\int c(x,y)\d\alpha(x).
	\]
	Thus any limiting minimizer is supported on $\argmin V_\alpha$. When this minimizer is unique,
	\[
		\al_\epsilon \rightharpoonup \delta_{y^\star(\alpha)},
		\qquad
		y^\star(\alpha)=\uargmin{y} V_\alpha(y).
	\]
	For the quadratic cost $c(x,y)=\norm{x-y}^2$ on $\RR^d$, assuming $\alpha$ has finite second moment, one has $V_\alpha(y)=\norm{y-\int x\d\alpha(x)}^2+\mathrm{const}$, so the collapse is toward the Dirac mass at the mean of $\alpha$.
\end{example}

\paragraph{Sinkhorn divergences.}
\index{Sinkhorn!divergence}

The raw entropic OT value has a large-temperature attraction toward the product coupling. The standard debiasing subtracts the two self-interaction energies, in the same spirit as passing from a positive kernel value to the associated squared distance.
\index{entropic!OT}
\index{energy!interaction}
\index{product!coupling}
\index{kernel!positive}

\begin{defn}[Sinkhorn divergence]\label{def-sinkhorn-divergence}
\index{Sinkhorn!divergence}
	For $\epsilon>0$, the debiased Sinkhorn divergence associated with the entropic OT value $\MK_\c^\epsilon$ is
	\eql{\label{eq-sinkhorn-divergence}
		\bar\MK_\c^\epsilon(\al,\be) \eqdef
		\MK_\c^\epsilon(\al,\be) - \frac{1}{2} \MK_\c^\epsilon(\al,\al) - \frac{1}{2}\MK_\c^\epsilon(\be,\be).
	}
\end{defn}
Although this formula is a debiasing by self-costs, its non-negativity is not automatic from the definition; Proposition~\ref{prop-sinkhorn-positive} proves it below by a kernel Cauchy--Schwarz argument.

\begin{figure}[ht]
\centering
\setlength{\tabcolsep}{1.5pt}
\begin{tabular}{@{}cccc@{}}
\includegraphics[width=.235\linewidth]{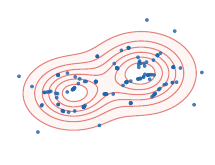} &
\index{Sinkhorn!divergence}
\includegraphics[width=.235\linewidth]{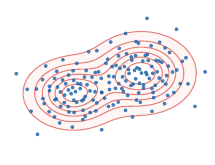} &
\includegraphics[width=.235\linewidth]{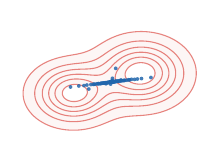} &
\includegraphics[width=.235\linewidth]{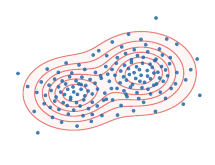} \\[-.1em]
\small $\MK_\c^\epsilon$, small $\epsilon$ &
\small $\bar\MK_\c^\epsilon$, small $\epsilon$ &
\small $\MK_\c^\epsilon$, large $\epsilon$ &
\small $\bar\MK_\c^\epsilon$, large $\epsilon$
\end{tabular}
\caption{Visualization of the debiasing effect by point optimization. The red level sets show a fixed two-Gaussian target density $\beta$, while the blue atoms are an optimized empirical measure $\alpha_n$ initialized in the same way in all panels. For small $\epsilon$, the entropic cost $\MK_\c^\epsilon$ and the debiased Sinkhorn divergence $\bar\MK_\c^\epsilon$ both keep the two overlapping modes. For large $\epsilon$, minimizing $\MK_\c^\epsilon$ collapses the atoms toward the barycenter predicted by the large-temperature bias, whereas the self-cost subtraction in~\eqref{eq-sinkhorn-divergence} keeps a bimodal cloud.}
\index{Sinkhorn!divergence}
\index{empirical!measure}
\label{fig:sinkhorn-divergence-debiasing}
\end{figure}

We first record a fundamental lemma: at an optimal dual pair, the exponential regularization term integrates to zero, so the entropic cost can be read directly from the potentials. The same cancellation holds along Sinkhorn iterations after each exact block update.
\index{Sinkhorn!iteration}

\begin{lem}[Entropic dual cost at optimum]
	Let $(f_{\al,\be},g_{\al,\be})$ be optimal dual potentials, normalized arbitrarily. Then
\index{dual!potential}
	\eql{\label{eq-formula-cost-dual}
		\MK_\c^\epsilon(\al,\be) = \dotp{f_{\al,\be}}{\al} + \dotp{g_{\al,\be}}{\be}.
	}
\end{lem}
\begin{proof}
	We first notice that at optimality, the relation
	\eq{
		f_{\al,\be} = -\epsilon\log\int_\Yy e^{\frac{g_{\al,\be}(y)-c(x,y)}{\epsilon}} \d\be(y)
	}
	after taking the exponential, equivalently reads
	\eq{
		1 = \int_\Yy e^{\frac{f_{\al,\be}(x) + g_{\al,\be}(y)-c(x,y)}{\epsilon}} \d\be(y)
		\qarrq
			\int_{\Xx\times\Yy} \pa{ e^{\frac{f_{\al,\be} \oplus g_{\al,\be}-c}{\epsilon}} -1 } \d(\al \otimes \be) = 0.
	}
	Substituting this identity in~\eqref{eq-dual-sinkh-cont} gives the result.
\end{proof}

The next proposition records the two limiting regimes of this debiased quantity.

\begin{prop}[Asymptotics of Sinkhorn divergences]\label{prop-sinkhorn-divergence-asymptotics}
\index{Sinkhorn!divergence}
	Assume that the two measures are supported on the same space and that $c$ is bounded, continuous, nonnegative and satisfies $c(x,x)=0$.
	Then $\bar\MK_\c^\epsilon(\al,\be) \rightarrow \MK_\c(\al,\be)$ when $\epsilon\rightarrow 0$ and
	\eq{
		\bar\MK_\c^\epsilon(\al,\be) \rightarrow
		\frac{1}{2}\int -c \d(\al-\be) \otimes \d(\al-\be)
		\qwhenq
		\epsilon \rightarrow +\infty.
	}
\end{prop}
\begin{proof}
	The discrete convergence result above already gives the correct intuition; we now use the standard continuous argument.
	\textbf{Case $\epsilon \rightarrow 0$.}
	The first limit follows from the standard $\Gamma$-convergence argument for entropic optimal transport: the entropy term is lower semicontinuous along weakly converging couplings, while any finite-cost coupling can be approximated by couplings with finite entropy. Since $c(x,x)=0$, the two self-costs in the debiased expression converge to zero, and the cross term converges to $\MK_\c(\al,\be)$.
\index{lower semicontinuity}

	\textbf{Case $\epsilon \rightarrow +\infty$.} We denote by $(f_\epsilon,g_\epsilon)$ optimal dual potentials. After normalizing them and using boundedness of $c$, their oscillations stay uniformly bounded, so the following expansion is uniform. The optimality condition on $f_\epsilon$ (equivalently the Sinkhorn fixed point on $f_\epsilon$) reads
\index{dual!potential}
		\begin{align*}
			f_\epsilon &= -\epsilon \log \int \exp\pa{ \frac{g_\epsilon(y) - c(\cdot,y)}{\epsilon} } \d \be(y)
			=	-\epsilon \log \int \pa{ 1 + \frac{g_\epsilon(y) - c(\cdot,y)}{\epsilon} + o(1/\epsilon) } \d \be(y) \\
			&=	-\epsilon \log\pa{1+\frac{1}{\epsilon}\int (g_\epsilon(y)-c(\cdot,y))\d\be(y)+o(1/\epsilon)}
			= - \int g_\epsilon \d\be + \int c(\cdot,y) \d \be(y) + o(1).
		\end{align*}
	Plugging this relation in the dual expression~\eqref{eq-formula-cost-dual}
	\eq{
		\MK_\c^\epsilon(\al,\be) = \int f_\epsilon \d\al + \int g_\epsilon \d \be =
		\iint c(x,y) \d\al(x)\d\be(y) + o(1).
	}
	Applying this expansion to $(\alpha,\beta)$, $(\alpha,\alpha)$ and $(\beta,\beta)$ gives
	\[
		\bar\MK_\c^\epsilon(\alpha,\beta)
		\to
		\int c\,\d\alpha\otimes\d\beta
		-\frac12\int c\,\d\alpha\otimes\d\alpha
		-\frac12\int c\,\d\beta\otimes\d\beta
		=
		-\frac12\int c\,\d(\alpha-\beta)\otimes\d(\alpha-\beta).
	\]
\end{proof}

\begin{rem}[Large-temperature Hilbertian limit]
\index{Hilbertian!limit}
	If $-c$ defines a conditionally positive definite kernel, the large-temperature limit in Proposition~\ref{prop-sinkhorn-divergence-asymptotics} is the square of a Hilbertian kernel norm. A typical example is $c(x,y)=\norm{x-y}^p$ for $0 < p < 2$, which corresponds to the energy-distance kernel. This kernel norm is the dual of a homogeneous Sobolev norm.
\index{homogeneous Sobolev norm}
\index{Sobolev norm}
\index{kernel!norm}
\index{kernel!positive definite}
\index{large-temperature limit}
\index{Sinkhorn!divergence}
\index{energy!distance}
\end{rem}

We now show that this debiased Sinkhorn divergence is positive.

\begin{prop}[Non-negativity of Sinkhorn divergences]\label{prop-sinkhorn-positive}
\index{Sinkhorn!divergence}
	If $k(x,y)=e^{-c(x,y)/\epsilon}$ is positive definite, then $\bar\MK_\c^\epsilon(\al,\be) \geq 0$.
\end{prop}

\begin{proof}
	In the following, we denote by $(f_{\al,\be},g_{\al,\be})$ optimal dual potentials for the
\index{dual!potential}
	dual Schr\"odinger problem between $\al$ and $\be$.
\index{Schrodinger!problem}
	We denote by $f_{\al,\al}=g_{\al,\al}$ (one can assume they are equal by symmetry) the solution for the problem between $\al$ and itself.
	Using the suboptimal function $(f_{\al,\al},g_{\be,\be})$ in the dual maximization problem, and using relation~\eqref{eq-formula-cost-dual} for the simplified expression of the dual cost, one obtains
	\eq{
		\MK_\c^\epsilon(\al,\be) \geq \dotp{f_{\al,\al}}{\al} + \dotp{g_{\be,\be}}{\be}
			 - \epsilon \dotp{ e^{\frac{f_{\al,\al} \oplus g_{\be,\be}-c}{\epsilon}} -1 }{\al \otimes \be}
	}
	Moreover $\dotp{f_{\al,\al}}{\al} = \frac{1}{2}\MK_\c^\epsilon(\al,\al)$, and similarly for $\be$, so the previous inequality equivalently reads
	\eq{
		\frac{1}{\epsilon} \bar \MK_\c^\epsilon(\al,\be)
			\geq 	1 - \dotp{ e^{\frac{f_{\al,\al} \oplus g_{\be,\be}-c}{\epsilon}} }{\al \otimes \be}
			= 1 - \dotp{\tilde\al}{\tilde\be}_{k}
	}
		where $\tilde\al=e^{f_{\al,\al}/\epsilon} \al$, $\tilde\be=e^{f_{\be,\be}/\epsilon} \be$
		and we introduced the inner product, valid because $k$ is positive definite,
	$\dotp{\tilde\al}{\tilde\be}_{k} \eqdef \int k(x,y) \d\tilde\al(x) \d\tilde\be(y)$.
	The self Sinkhorn fixed point equation, once exponentiated, reads pointwise
\index{self Sinkhorn}
	\[
		e^{f_{\al,\al}(x)/\epsilon}\int k(x,y)\d\tilde\al(y)=1
		\qquad \text{for }\al\text{-a.e. }x,
	\]
	and hence
	\eq{
		\norm{\tilde \al}_k^2 = \dotp{k(\tilde\al)}{\tilde\al}
		= \int e^{f_{\al,\al}(x)/\epsilon} k(\tilde\al)(x)\d\al(x) = 1
	}
		and similarly $\norm{\tilde \be}_k^2=1$. Therefore, by Cauchy--Schwarz, one has
	$1 - \dotp{\tilde\al}{\tilde\be}_{k} \geq 0$.
\end{proof}

\begin{rem}[Strict positivity]
\index{strict!positivity}
	Under additional assumptions on the kernel, one can furthermore show that $\bar\MK_\c^\epsilon(\al,\be)=0$ implies $\al=\be$, and that this debiased divergence metrizes convergence in law.
\index{convergence!in law}
\end{rem}


\chapter{Entropic Regularization: Convergence}
\index{entropic!regularization}
\index{Sinkhorn!convergence}
\label{sec-sinkhorn-advanced}
\label{sec-entropic-convergence}
\label{sec-convergence-dual}

Convergence for entropic optimal transport has two complementary meanings. At fixed marginals and fixed temperature, one studies how Sinkhorn iterates approach the regularized optimizer; when the marginals are empirical, one also studies how the regularized value and potentials behave as the number of samples grows toward a mean-field limit. This chapter keeps these two scales together: first the algorithmic convergence of matrix scaling and soft transforms, then the Gaussian closed forms and sample-complexity consequences that explain the statistical role of the regularization.
\index{mean-field!limit}
\index{sample complexity}
\index{matrix!scaling}
\index{soft!transform}

The algorithmic part of the chapter revisits Sinkhorn convergence through three complementary lenses. Bregman projections explain the alternating-projection geometry, Fortet's order argument gives qualitative fixed-point convergence, robust Bregman estimates give a non-asymptotic $O(1/k)$ dual-gap bound, and Hilbert's metric gives a clean linear contraction when the kernel is uniformly positive. The first and last viewpoints explain convergence mechanisms, while the robust estimate is often the most useful for explicit complexity guarantees.
\index{Sinkhorn!convergence}
\index{Bregman!projection}
\index{dual!gap}

\section{Sinkhorn Convergence: Bregman View}
\index{Bregman!projection}
\index{Sinkhorn!convergence}
\label{sec-convergence-init}

This section explains Sinkhorn as alternating Bregman projections. The main message is geometric: each row or column rescaling is the KL projection onto one affine marginal constraint, so convergence follows from the Pythagorean identity for Bregman divergences.
\index{KL!projection}
\index{marginal!constraint}
\index{affine!marginal constraint}
\index{Pythagorean identity}
\index{Bregman!divergence}
\index{Bregman!projection}
\index{KL!projection}

For simplicity, this section is written for discrete measures, but the same ideas carry over to general measures. The robust-rate section later revisits the alternating-projection mechanism through convex duality, with constants expressed through the oscillation of the potentials and of the cost range.

\paragraph{Alternating $\KL$ projections.}
\index{alternating!projection}

The projection viewpoint explains Sinkhorn as repeated enforcement of one marginal constraint at a time. It is not specific to entropy, although the KL case is the one where the projections reduce to elementary row and column scalings.
\index{scaling!row}
\index{scaling!column}
\index{marginal!constraint}

\begin{defn}[Bregman divergence]\label{def-bregman-divergence}
\index{Bregman!divergence}
	Let $\Phi$ be a differentiable strictly convex function on a convex domain $\Omega$. The Bregman divergence generated by $\Phi$ is
\index{convex!function}
	\[
		B_\Phi(P|Q)\eqdef \Phi(P)-\Phi(Q)-\dotp{\nabla\Phi(Q)}{P-Q}.
	\]
	For the negative entropy $\Phi(P)=\sum_{i,j}P_{i,j}\log P_{i,j}$ on the positive orthant, one obtains $B_\Phi(P|Q)=\KLD(P|Q)$ up to the harmless convention at the boundary.
\index{Shannon!negative entropy}
\end{defn}

Bregman divergences are useful because their geometry can encode constraints. A Legendre-type generator $\Phi$ blows up, or has an infinite derivative, at the boundary of its domain. For negative entropy, positivity is therefore built into the divergence, so one projects onto affine marginal constraints without separately handling non-negativity.
\index{Shannon!negative entropy}
\index{affine!marginal constraint}
\index{marginal!constraint}
\index{Bregman!divergence}

\paragraph{Linear tilts and Gibbs references.}
\index{linear!tilt}
\index{Gibbs!reference}

The next proposition explains why adding a linear cost to a Bregman penalty merely shifts the reference point in dual coordinates. The usual Gibbs--KL reformulation is the entropy specialization.

\begin{prop}[Linear tilts of Bregman penalties]\label{prop-bregman-linear-tilt}
\index{linear!tilt}
\index{Bregman!penalty}
	Let $\Phi$ be differentiable and strictly convex, and let $B_\Phi$ be its Bregman divergence. Fix a reference point $\Q$ in the interior of the domain. Assume that there exists $\Q^\C$ such that
\index{Bregman!divergence}
	\[
		\nabla\Phi(\Q^\C)=\nabla\Phi(\Q)-\C/\epsilon .
	\]
	Then, for all $\P$ in the domain,
	\[
		\dotp{\P}{\C}+\epsilon B_\Phi(\P|\Q)
		=
		\epsilon B_\Phi(\P|\Q^\C)+\text{\upshape cst},
	\]
	where the constant does not depend on $\P$.
\end{prop}
\begin{proof}
	Subtract the two Bregman divergences:
\index{Bregman!divergence}
	\[
		B_\Phi(\P|\Q^\C)-B_\Phi(\P|\Q)
		=
		\dotp{\nabla\Phi(\Q)-\nabla\Phi(\Q^\C)}{\P}
		+\text{cst}.
	\]
	Using $\nabla\Phi(\Q)-\nabla\Phi(\Q^\C)=\C/\epsilon$ and multiplying by $\epsilon$ gives the claim.
\end{proof}

For the negative entropy $\Phi(\P)=\sum_{i,j}\P_{i,j}\log\P_{i,j}$, one has $B_\Phi=\KLD$. Taking $\Q=\a\otimes\b$ gives the tilted reference
\index{Shannon!negative entropy}
\[
	\K_{\a,\b}^\epsilon
	\eqdef
	(\a\otimes\b)\odot e^{-\C/\epsilon}.
\]
Thus
\[
	\dotp{\P}{\C}+\epsilon\KLD(\P|\a\otimes\b)
	=
	\epsilon\KLD(\P|\K_{\a,\b}^\epsilon)+\text{\upshape cst}.
\]
On the transport polytope, scaling $\K_{\a,\b}^\epsilon$ is equivalent to scaling the Gibbs kernel $\K=e^{-\C/\epsilon}$ because the factors $\a_i$ and $\b_j$ can be absorbed into the Sinkhorn scalings.
\index{Sinkhorn!scaling}
\index{transportation!polytope}
\index{Gibbs!kernel}

Thus the unique solution $\P_\epsilon$ of~\eqref{eq-regularized-discr} is the KL projection of the tilted Gibbs reference onto $\CouplingsD(\a,\b)$:
\index{Gibbs!reference}
\index{KL!projection}
\eql{\label{eq-kl-proj}
	\P_\epsilon = \Proj_{\CouplingsD(\a,\b)}^\KLD(\K_{\a,\b}^\epsilon) \eqdef \uargmin{\P \in \CouplingsD(\a,\b)} \KLD(\P|\K_{\a,\b}^\epsilon).
}

\paragraph{Cyclic projection convergence.}
\index{cyclic!projection}

The convergence mechanism is the classical one of Bregman~\cite{bregman1967relaxation}.

\begin{prop}[Cyclic Bregman projections on affine constraints]\label{prop-cyclic-kl-affine}
\index{Bregman!projection}
\index{affine!constraint}
	Let $\Phi$ be a Legendre strictly convex generator on a finite-dimensional convex domain, and let $\Cc_1,\Cc_2$ be affine constraint sets whose intersection meets the domain. Define
	\[
		P_{k+1}
		=
		\Proj_{\Cc_2}^{B_\Phi}
		\Proj_{\Cc_1}^{B_\Phi}(P_k),
	\]
	starting from an interior point $P_0$. Assume the projections are well-defined and that the iterates remain in a compact subset of the domain. Then $P_k$ converges to the Bregman projection of $P_0$ onto $\Cc_1\cap\Cc_2$. In particular, the KL case converges for positive affine marginal constraints on a bounded transportation polytope.
\index{affine!marginal constraint}
\index{transportation!polytope}
\index{marginal!constraint}
\index{Bregman!projection}
\end{prop}
\begin{proof}
	We first prove the Pythagorean identity used by the projection argument. For three interior points one has the Bregman three-point formula
\index{Bregman!three-point formula}
\index{Pythagorean identity}
	\[
		B_\Phi(Q|P)
		=
		B_\Phi(Q|P^+)
		+
		B_\Phi(P^+|P)
		+
		\dotp{\nabla\Phi(P^+)-\nabla\Phi(P)}{Q-P^+}.
	\]
	If $P^+=\Proj_\Cc^{B_\Phi}(P)$ and $\Cc$ is affine, the first-order optimality condition for minimizing $R\mapsto B_\Phi(R|P)$ over $R\in\Cc$ is
\index{optimality!first-order}
	\[
		\dotp{\nabla\Phi(P^+)-\nabla\Phi(P)}{R-P^+}=0
		\qquad \forall R\in\Cc,
	\]
	because $R-P^+$ ranges over the tangent linear space of $\Cc$. Taking $R=Q\in\Cc$ cancels the last term and gives
	\[
		B_\Phi(Q|P)=B_\Phi(Q|P^+)+B_\Phi(P^+|P)
		\qquad\forall Q\in\Cc.
	\]

	Let $(Z_\ell)_\ell$ be the half-step sequence obtained by alternating projections onto $\Cc_1$ and $\Cc_2$, so that $Z_{2k}=P_k$ and $Z_{2k+2}=P_{k+1}$. Fix $Q\in\Cc_1\cap\Cc_2$. Applying the identity at each half-step gives
\index{Sinkhorn!half-step}
\index{alternating!projection}
	\[
		B_\Phi(Q|Z_\ell)-B_\Phi(Q|Z_{\ell+1})
		=
		B_\Phi(Z_{\ell+1}|Z_\ell)\geq0.
	\]
	Thus $B_\Phi(Q|Z_\ell)$ decreases and the series $\sum_\ell B_\Phi(Z_{\ell+1}|Z_\ell)$ is finite. The compactness assumption gives cluster points. Since the projection drops tend to zero and $\Phi$ is strictly convex on compact subsets of the domain, $\norm{Z_{\ell+1}-Z_\ell}\to0$. Every cluster point of the even subsequence is therefore also a cluster point of the adjacent odd subsequence. Because these two subsequences lie alternately in the closed affine sets $\Cc_1$ and $\Cc_2$, every cluster point belongs to $\Cc_1\cap\Cc_2$.

	Let $\bar P$ be such a cluster point. For each half-step, the dual displacement
\index{Sinkhorn!half-step}
	\[
		\nabla\Phi(Z_{\ell+1})-\nabla\Phi(Z_\ell)
	\]
	belongs to the normal space of the affine set onto which one projects. Telescoping and using the convergence of $Z_\ell$ gives
	\[
		\nabla\Phi(\bar P)-\nabla\Phi(P_0)
		\in
		N_{\Cc_1}+N_{\Cc_2}
		=
		N_{\Cc_1\cap\Cc_2},
	\]
	where the last equality uses that the sets are affine. This is precisely the first-order optimality condition for minimizing $R\mapsto B_\Phi(R|P_0)$ over $R\in\Cc_1\cap\Cc_2$. Thus $\bar P$ is the Bregman projection of $P_0$ onto the intersection. Strict convexity gives uniqueness of this minimizer, so all cluster points coincide and the whole sequence converges. The KL statement follows by choosing the negative entropy generator.
\index{Shannon!negative entropy}
\index{optimality!first-order}
\index{Bregman!projection}
\index{strict!convexity}
\end{proof}

\begin{alg}[Cyclic Bregman projections]\label{alg:cyclic-bregman-projections}
\index{cyclic!Bregman projection}
\textbf{Input:} Constraint sets $\Cc_1,\Cc_2$, Bregman divergence $B_\Phi$, interior point $P_0$, constraint defects \(\mathrm{def}_{\Cc_1},\mathrm{def}_{\Cc_2}\), tolerance $\mathrm{tol}$.

\textbf{Output:} Point in $\Cc_1\cap\Cc_2$ when the intersection is feasible.

\textbf{Specialize} to entropic OT, if needed:
\(\Cc_1=\Cc^1_\a, \qquad \Cc_2=\Cc^2_\b, \qquad B_\Phi=\KL.\)

\textbf{Initialize:} Set \(r_0=+\infty\) and \(k=0\).

\textbf{While} \(r_k>\mathrm{tol}\) \textbf{do}:
\begin{algblock}

\textbf{Set} \(k\leftarrow k+1\).

\(P_{k-1/2}=\Proj_{\Cc_1}^{B_\Phi}(P_{k-1}).\)

\(P_k=\Proj_{\Cc_2}^{B_\Phi}(P_{k-1/2}).\)

\textbf{Set} \(r_k=\max\{\mathrm{def}_{\Cc_1}(P_k),\mathrm{def}_{\Cc_2}(P_k)\}\).
\end{algblock}
\algreturnskip
\textbf{Return} $P_k$.
\end{alg}

Denoting
\eq{
	\Cc^1_\a \eqdef \enscond{\P}{\P\ones_m=\a}
	\qandq
	\Cc^2_\b \eqdef \enscond{\P}{\transp{\P}\ones_n=\b}
}
the rows and columns constraints, one has $\CouplingsD(\a,\b) = \Cc^1_\a \cap \Cc^2_\b$. One can use KL, or more generally Bregman, iterative projections~\cite{bregman1967relaxation,Ruschendorf95,RuschendorfThomsen}
\eql{\label{eq-kl-sinkh-proj}
	\itt{\P} \eqdef \Proj_{\Cc^1_\a}^{\KLD}(\it{\P})
	\qandq
	\ittt{\P} \eqdef \Proj_{\Cc^2_\b}^{\KLD}(\itt{\P}).
}
Since the sets $\Cc^1_\a$ and $\Cc^2_\b$ are affine, Proposition~\ref{prop-cyclic-kl-affine} applies with $P_0=\K_{\a,\b}^\epsilon$ and shows convergence to the solution of~\eqref{eq-kl-proj}.

\paragraph{Row and column scalings.}
\index{scaling!row}
\index{scaling!column}

The two projectors are simple to compute since they correspond to scaling respectively the rows and the columns, as explained in this proposition.

\begin{prop}[KL projections are scalings]
\index{KL!projection}
One has
\eq{
	 \Proj_{\Cc^1_\a}^{\KLD}(\P) = \diag\pa{\frac{\a}{\P \ones_m}} \P
	 \qandq
	 \Proj_{\Cc^2_\b}^{\KLD}(\P) =  \P \diag\pa{\frac{\b}{\P^\top \ones_n}}.
}
\end{prop}
\begin{proof}
	Consider the problem along each row or column vector to impose a fixed sum $s \in \RR_+$
	\eq{
		\umin{p} \enscond{ \KL(p|q) }{ \dotp{p}{\ones}=s }.
	}
	The Lagrange multiplier equation for this problem reads
	\eq{
		\log(p/q)+\la \ones=0 \qarrq
		p = u q \qwhereq u = e^{-\la}>0.
	}
	The constraint $\dotp{p}{\ones} = s$ is equivalent to $\dotp{uq}{\ones}=s$, i.e. $u=s/\sum_i q_i$, which gives the desired scaling formula
\index{scaling!form}
	$p=s q/\sum_i q_i$.
\end{proof}

These iterations are equivalent to Sinkhorn iterations~\eqref{eq-sinkhorn} since defining
\index{Sinkhorn!iteration}
\eq{\label{eq-sink-matrix}\P^{(2\ell)} \eqdef \diag(\it{\uD}) \K \diag(\it{\vD}),}
one has
\begin{align*}
	\P^{(2\ell+1)} &\eqdef \diag(\itt{\uD}) \K \diag(\it{\vD}) \\
	\qandq
	\P^{(2\ell+2)} &\eqdef \diag(\itt{\uD}) \K \diag(\itt{\vD})
\end{align*}
In practice, however, one should prefer using~\eqref{eq-sinkhorn}, which only requires manipulating scaling vectors and multiplying by a Gibbs kernel, and can often be accelerated when the kernel has separable, sparse, low-rank or geometric structure.
\index{scaling!vectors}
\index{Gibbs!kernel}

Such a convergence analysis using Bregman projection is of limited interest because it only works directly in finite dimension. For instance, the linear convergence speed one can obtain from strong convexity degrades with the dimension and with $\epsilon$. The robust dual analysis below gives a dimension-free qualitative message: the constants are expressed through the oscillation of the potentials and the cost range, and the resulting $O(1/k)$ estimate explains what can be guaranteed before any asymptotic linear regime becomes visible.
\index{Bregman!projection}
\index{linear!convergence}
It is also possible to anneal $\epsilon$ during the iterations and to rely on multiscale strategies in low dimensions.

\section{Sinkhorn Convergence: Monotone Point of View}
\index{monotone!convergence}
\index{Sinkhorn!convergence}

There is another, older way to understand convergence, going back to Fortet's proof of the Schr\"odinger system~\cite{fortet1940schrodinger,essid2019fortet,leonard2019fortet}. It does not primarily estimate a contraction factor. Instead, it uses the order structure of the soft transforms.
\index{Schrodinger!system}
\index{soft!transform}

\begin{prop}[Monotone fixed-point route to Sinkhorn convergence]\label{prop-fortet-monotone}
\index{monotone!fixed point}
\index{Sinkhorn!convergence}
	Let $c$ be bounded and continuous on compact spaces, and define the double Sinkhorn map, normalized by subtracting its $\alpha$-mean,
	\[
		\mathcal A(f) \eqdef
		(f^{c,\epsilon})^{\bar c,\epsilon}
		-\int (f^{c,\epsilon})^{\bar c,\epsilon}\d\alpha.
	\]
	The map $\mathcal A$ is order preserving on the quotient by additive constants. If a representative of the initial class is chosen so that $f_0\leq \mathcal A(f_0)$, then representatives of the iterates $f_{k+1}=\mathcal A(f_k)$ can be chosen to increase pointwise, remain uniformly bounded in oscillation, and converge to a fixed point. Since constants are free, any bounded initialization can be shifted downward to satisfy the subsolution inequality. The fixed point is the entropic potential, hence the associated Sinkhorn scalings converge.
\index{Sinkhorn!scaling}
\index{order-preserving map}
\index{entropic!potential}
\end{prop}

\begin{proof}
	The soft $c$-transform is order reversing: if $g\leq g'$, then
\index{soft!c-transform}
	\[
		-\epsilon\log\int e^{(g-c)/\epsilon}\d\beta
		\geq
		-\epsilon\log\int e^{(g'-c)/\epsilon}\d\beta.
	\]
	The composition of two order-reversing transforms is therefore order preserving. The transform also commutes with additive constants in the projective sense, which is why the argument is naturally stated for equivalence classes modulo constants. Starting from a subsolution representative gives $f_0\leq f_1$, and order preservation gives representatives satisfying $f_k\leq f_{k+1}$ for all $k$. Soft-transform oscillation bounds, controlled by $\sup c-\inf c$, prevent escape to infinity after normalization. Monotone pointwise convergence, compactness of equicontinuous soft transforms, and continuity of $\mathcal A$ then give a fixed point. Uniqueness of entropic potentials up to constants, Proposition~\ref{prop-entropic-unique} and the dual uniqueness statement above identify this fixed point with the Sinkhorn solution.
\index{order-preserving map}
\index{entropic!potential}
\index{soft!transform}
\end{proof}

This proof is qualitative rather than quantitative, but it is conceptually useful: Sinkhorn is not only alternating projection or projective contraction; it is also a monotone fixed-point iteration on potential classes once constants are quotiented out.
\index{alternating!projection}

\begin{defn}[Variation seminorm]\label{def-variation-seminorm}
\index{variation seminorm}
	For a bounded real-valued function $h$, the variation seminorm is
	\[
		\norm{h}_V\eqdef \sup h-\inf h .
	\]
	It vanishes exactly on constant functions, hence becomes a norm after quotienting by additive constants.
\end{defn}
This is the natural size for Sinkhorn potentials because adding constants changes their gauge but not the coupling.

\begin{prop}[Topical maps are variation-nonexpansive]\label{prop-topical-variation-nonexpansive}
\index{topical map}
	Let $E$ be a vector space of real-valued bounded functions, ordered pointwise, and write
	$\norm{\cdot}_V$ for the variation seminorm of Definition~\ref{def-variation-seminorm}.
	Let $\mathcal T:E\to E$ be monotone and additively homogeneous,
\index{additively homogeneous map}
	\[
		f\leq g \Rightarrow \mathcal T(f)\leq \mathcal T(g),
		\qquad
		\mathcal T(f+\lambda)=\mathcal T(f)+\lambda
		\quad \forall \lambda\in\RR .
	\]
	Then
	\[
		\norm{\mathcal T(f)-\mathcal T(g)}_V
		\leq
		\norm{f-g}_V .
	\]
	The same conclusion holds for order-reversing maps satisfying $\mathcal T(f+\lambda)=\mathcal T(f)-\lambda$.
\end{prop}

\begin{proof}
	Set $a=\inf(f-g)$ and $b=\sup(f-g)$. Then
	\[
		g+a\leq f\leq g+b .
	\]
	If $\mathcal T$ is order preserving and additively homogeneous, applying $\mathcal T$ gives
\index{order-preserving map}
\index{additively homogeneous map}
	\[
		\mathcal T(g)+a\leq \mathcal T(f)\leq \mathcal T(g)+b .
	\]
	Hence every value of $\mathcal T(f)-\mathcal T(g)$ lies in $[a,b]$, so its oscillation is at most $b-a=\norm{f-g}_V$. For an order-reversing, additively anti-homogeneous map, the same inequalities give
	\[
		\mathcal T(g)-b\leq \mathcal T(f)\leq \mathcal T(g)-a,
	\]
	and the oscillation bound is identical.
\end{proof}

\begin{cor}[Soft transforms are nonexpansive]\label{cor-soft-transform-nonexpansive}
\index{soft!transform}
	For every $\epsilon>0$, the soft $c$-transforms~\eqref{eq-soft-c-cont-f}--\eqref{eq-soft-c-cont-g} are $1$-Lipschitz for the variation seminorm. Consequently, the double Sinkhorn map used in Proposition~\ref{prop-fortet-monotone}, after quotienting constants, is also $1$-Lipschitz for $\norm{\cdot}_V$.
\index{soft!c-transform}
\index{variation seminorm}
\end{cor}

\begin{proof}
	The soft transform is order reversing and satisfies $(g+\lambda)^{\bar c,\epsilon}=g^{\bar c,\epsilon}-\lambda$, and similarly for the other block. Proposition~\ref{prop-topical-variation-nonexpansive} applies to each block, and the composition of two $1$-Lipschitz maps is $1$-Lipschitz.
\index{soft!transform}
\end{proof}

\begin{rem}[Topical maps and projective geometry]\label{rem-topical-maps}
\index{topical map}
\index{projective!geometry}
	Order-preserving additively homogeneous maps are called topical maps in nonlinear Perron--Frobenius theory~\cite{lemmens2012nonlinear}. Proposition~\ref{prop-topical-variation-nonexpansive} is the basic nonexpansiveness mechanism behind Fortet's monotone argument. The Hilbert-metric analysis in Section~\ref{sec-sinkhorn-hilbert} is stronger: under strict positivity assumptions on the kernel it upgrades nonexpansiveness to a genuine projective contraction.
\index{Perron-Frobenius theorem}
\index{additively homogeneous map}
\index{strict!positivity}
\end{rem}


\section{Sinkhorn Convergence: Sublinear Robust Rate}
\index{rate!sublinear}
\index{Sinkhorn!convergence}

Sinkhorn is cyclic coordinate ascent on the smooth dual objective $\mathcal D_\epsilon$ defined in~\eqref{eq-dual-sinkhorn-objective}; equivalently, it alternates KL projections on the two marginal constraint sets. The following statement records the rate most useful for complexity estimates: before any possible linear regime becomes visible, the dual objective gap decreases at order $1/k$. We state it for balanced entropic OT, which is the specialization needed here. Related robust Bregman-projection rates are developed in~\cite{peyre2026robust,altschuler2017near,pmlr-v80-dvurechensky18a}, and statistical consequences of entropic smoothing are analyzed in~\cite{genevay2018sample,bigot2017central}.
\index{dual!objective gap}
\index{KL!projection}
\index{coordinate ascent}
\index{entropic!OT}
\index{entropic!smoothing}
\index{marginal!constraint}
\index{Bregman!projection}
\index{KL!projection}

\begin{prop}[Pinsker inequality]\label{prop-pinsker}
\index{Pinsker inequality}
	If $p,q\in\simplex_n$, then
	\[
		\norm{p-q}_1^2 \leq 2\KL(p|q).
	\]
\end{prop}
\begin{proof}
	Let $A=\{i:p_i\geq q_i\}$ and set $a=\sum_{i\in A}p_i$, $b=\sum_{i\in A}q_i$. Then $a-b=\frac12\norm{p-q}_1$. Applying data processing for relative entropy to the partition $(A,A^c)$ gives
\index{data processing inequality}
\index{entropy!relative}
	\[
		\KL(p|q)
		\geq
		a\log\frac{a}{b}+(1-a)\log\frac{1-a}{1-b}.
	\]
	For fixed $b\in(0,1)$, the function
	\[
		h(a)=a\log\frac{a}{b}+(1-a)\log\frac{1-a}{1-b}-2(a-b)^2
	\]
	satisfies $h(b)=h'(b)=0$ and
	\[
		h''(a)=\frac1a+\frac1{1-a}-4\geq0,
	\]
	because $a(1-a)\leq1/4$. Hence the binary relative entropy is at least $2(a-b)^2=\frac12\norm{p-q}_1^2$. The boundary cases follow by approximation.
\index{entropy!relative}
\end{proof}

\begin{prop}[A compact $O(1/k)$ dual rate]\label{prop-sinkhorn-dual-rate}
\index{dual!rate}
	Assume that $\Xx$ and $\Yy$ are compact, that $c$ is bounded, and write $R=\sup c-\inf c$. Let $(f_k,g_k)$ be Sinkhorn dual iterates normalized by $\int f_k\d\alpha=0$, and let
\index{Sinkhorn!dual}
	\[
		\Delta_k
		\eqdef
		\mathcal D_\epsilon(f^\star,g^\star)-\mathcal D_\epsilon(f_k,g_k)
	\]
	be the dual suboptimality gap for the entropic dual objective~\eqref{eq-dual-sinkhorn-objective}. Then there exists a numerical constant $C$ such that
	\[
		\Delta_k \leq \frac{C R^2}{\epsilon(k+1)}.
	\]
\end{prop}

\begin{proof}
	The proof uses three elementary ingredients. First, the soft $c$-transform bounds the oscillation of every normalized iterate and every normalized optimum:
\index{soft!c-transform}
	\[
		\norm{f_k}_V+\norm{g_k}_V+\norm{f^\star}_V+\norm{g^\star}_V \leq C R.
	\]
	Here $\norm{h}_V\eqdef\sup h-\inf h$ is the variation seminorm, the same projective norm used in Hilbert's metric in Section~\ref{sec-sinkhorn-hilbert}. It is natural because dual potentials can be shifted by constants without changing the coupling.
\index{variation seminorm}
\index{dual!potential}

	Second, each Sinkhorn half-step is an exact KL projection. The Pythagorean identity for KL projections gives the ascent identity
\index{Sinkhorn!half-step}
\index{Pythagorean identity}
\index{KL!projection}
	\[
		\mathcal D_\epsilon(f_{k+1},g_{k+1})-
		\mathcal D_\epsilon(f_k,g_k)
		=
		\epsilon\big[\KL(\pi^\star|\pi_k)-\KL(\pi^\star|\pi_{k+1})\big],
	\]
	where $\pi_k=e^{(f_k\oplus g_k-c)/\epsilon}\alpha\otimes\beta$ and $\pi^\star$ is the optimal entropic coupling. The KL drop controls the marginal residuals through Pinsker's inequality, Proposition~\ref{prop-pinsker}. Third, convexity of the exponential dual objective gives a one-step estimate of the form
\index{marginal!residual}
\index{Pinsker inequality}
	\[
		\Delta_k^2
		\leq
		\frac{C R^2}{\epsilon}
		\big(\mathcal D_\epsilon(f_{k+1},g_{k+1})-
		\mathcal D_\epsilon(f_k,g_k)\big).
	\]
	This is the usual Bregman-projection estimate: the dual gap is controlled by the product of a bounded dual radius and the marginal residual corrected by the next projection, while the residual squared is controlled by the KL drop. Summing the reciprocal inequality obtained from the last display yields
\index{dual!radius}
\index{marginal!residual}
\index{Bregman!projection}
\index{dual!gap}
	\[
		\frac{1}{\Delta_{k+1}}-\frac{1}{\Delta_k}
		\geq \frac{\epsilon}{C R^2},
	\]
	and therefore $\Delta_k\leq C R^2/(\epsilon(k+1))$.
\end{proof}

\begin{cor}[Approximating unregularized OT by regularized dual costs]\label{cor-sinkhorn-dual-complexity}
	Consider discrete histograms $\a\in\simplex_n$, $\b\in\simplex_m$ and a finite cost matrix $\C$. Let $\mathcal D_{\epsilon,k}$ be the KL-normalized entropic dual value after $k$ Sinkhorn cycles, and let $\Delta_k$ be its dual gap. Define the entropy-corrected lower bound
\index{histogram}
\index{cost matrix}
\index{dual!gap}
	\[
		L_{\epsilon,k}
		\eqdef
		\mathcal D_{\epsilon,k}
		-\epsilon\HD(\a)-\epsilon\HD(\b),
		\qquad
		\HD(\a)=-\sum_i a_i\log a_i .
	\]
	Then
	\[
		0\leq
		\MKD_\C(\a,\b)-L_{\epsilon,k}
		\leq
		\epsilon\log(nm)+\Delta_k .
	\]
	Consequently, choosing $\epsilon\leq\delta/(2\log(nm))$ and running Sinkhorn until $\Delta_k\leq\delta/2$ gives a $\delta$-accurate lower bound on the unregularized OT value. Under Proposition~\ref{prop-sinkhorn-dual-rate}, the intermediate condition is $k+1 \geq 2 C R^2/(\epsilon\delta)$. With the above choice of $\epsilon$, it is sufficient to take
	\[
		k+1 \geq \frac{4 C R^2\log(nm)}{\delta^2},
	\]
	hence $k=O(R^2\log(nm)/\delta^2)$, up to constants and logarithmic stabilization factors.
\end{cor}
\begin{proof}
	The KL-normalized objective differs from the entropy convention~\eqref{eq-regularized-discr} by the constant $\epsilon\HD(\a)+\epsilon\HD(\b)$ on the transport polytope, because
\index{transportation!polytope}
	\[
		\KLD(\P|\a\otimes\b)
		=
		-\HD(\P)+\HD(\a)+\HD(\b).
	\]
	Let $E_\epsilon$ be the optimum of the entropy-regularized objective $\dotp{\P}{\C}-\epsilon\HD(\P)$. Since $0\leq\HD(\P)\leq\log(nm)$ for any coupling matrix,
	\[
		\MKD_\C(\a,\b)-\epsilon\log(nm)
		\leq
		E_\epsilon
		\leq
		\MKD_\C(\a,\b).
	\]
	The corrected iterate satisfies $L_{\epsilon,k}=E_\epsilon-\Delta_k$, which gives the displayed value bound. The final iteration estimate follows by combining $\Delta_k\leq C R^2/(\epsilon(k+1))$ with the target $\Delta_k\leq\delta/2$.
\end{proof}

The same identity also yields computable stopping diagnostics. The KL drops are exactly marginal defects: after a row update the row marginal is correct and the remaining drop is measured by the column marginal, and conversely after a column update. Thus marginal violations monitor both feasibility and the remaining dual gap, up to the bounded-radius constant above.
\index{dual!gap}

\begin{alg}[Certified entropic approximation of discrete OT]\label{alg:certified-entropic-ot-accuracy}
\index{entropic!regularization}
\index{dual!gap}
\textbf{Input:} Cost matrix $C\in\RR^{n\times m}$, weights $\a,\b$, target accuracy $\delta>0$.

\textbf{Output:} Certified lower bound $L_{\epsilon,k}$ for exact OT.

\textbf{Set}
\(\epsilon=\frac{\delta}{2\log(nm)}.\)

\textbf{Initialize} stabilized Sinkhorn potentials and set \(k=0\), \(\widehat\Delta_0=+\infty\).

\textbf{While} \(\widehat\Delta_k>\delta/2\) \textbf{do}:
\begin{algblock}

\textbf{Set} \(k\leftarrow k+1\).

\textbf{Compute} one row soft-transform update and one column soft-transform update in stabilized log-domain variables.

\textbf{Compute} the entropic dual value \(\mathcal D_{\epsilon,k}\) and the certified gap upper bound \(\widehat\Delta_k\).
\end{algblock}
\algreturnskip
\textbf{Return} \(L_{\epsilon,k} = \mathcal D_{\epsilon,k} - \epsilon\HD(\a)-\epsilon\HD(\b), \qquad 0\leq\MKD_{\C}(\a,\b)-L_{\epsilon,k}\leq\delta.\)
\end{alg}


\section{Sinkhorn Convergence: Linear Hilbert Metric Rate}
\index{Sinkhorn!convergence}
\index{linear!convergence}
\index{Hilbert!metric}
\label{sec-sinkhorn-hilbert}

Hilbert's projective metric gives a complementary convergence mechanism. Instead of following objective values, it measures distances between positive scaling vectors modulo global multiplication. Positive kernels are contractions in this geometry, yielding a global linear convergence statement.
\index{Hilbert!projective metric}
\index{linear!convergence}
\index{kernel!positive}
\index{scaling!vectors}

As initially explained by~\cite{franklin1989scaling}, the global convergence analysis of Sinkhorn is greatly simplified using Hilbert's projective metric on positive vectors.

\begin{defn}[Hilbert metric]\label{def-hilbert-metric}
\index{Hilbert!metric}
\index{Hilbert!projective metric}
	On $\RR_{+,*}^n$, Hilbert's projective metric is
	\eql{\label{eq-hilbert-metric}
		\foralls (\uD,\uD') \in (\RR_{+,*}^n)^2, \quad
		\Hilbert(\uD,\uD') \eqdef
			\norm{\log(\uD)-\log(\uD')}_V .
	}
	where, for vectors, $\norm{z}_V=\max_i z_i-\min_i z_i$.
\end{defn}
Multiplying both vectors by arbitrary positive constants does not change this quantity, so it is a distance only after passing to projective classes.

\begin{prop}[Hilbert metric on the projective cone]
\index{Hilbert!metric}
\index{projective!cone}
	The function $\Hilbert$ defines a complete distance on the projective cone $\RR_{+,*}^n/\sim$, where $\uD \sim \uD'$ means that $\uD=s\uD'$ for some $s>0$.
\index{projective!cone}
\end{prop}
\begin{proof}
	The map $\uD \mapsto \log \uD$ identifies $\RR_{+,*}^n/\sim$ with the quotient vector space $\RR^n/\Span(\ones_n)$, because multiplying $\uD$ by $s>0$ adds the constant vector $\log(s)\ones_n$. The variation seminorm $\norm{z}_V=\max_i z_i-\min_i z_i$ vanishes exactly on constant vectors, so it induces a norm on this quotient. Symmetry, the triangle inequality and separation therefore follow from the corresponding norm properties. Completeness follows because $\RR^n/\Span(\ones_n)$ is finite-dimensional and all finite-dimensional normed spaces are complete.
\index{triangle inequality}
\end{proof}
It was introduced independently by~\cite{birkhoff1957extensions} and~\cite{samelson1957perron} to provide quantitative proofs of the Perron-Frobenius theorem (convergence of iterations of positive matrices). Sinkhorn should be viewed as a nonlinear generalization of Perron-Frobenius.
\index{Perron-Frobenius theorem}

\begin{thm}[Birkhoff contraction theorem]\label{thm-birkhoff}
\index{Birkhoff contraction theorem}
	Let $\K \in \RR_{+,*}^{n \times m}$, then for $(\vD,\vD') \in (\RR_{+,*}^m)^2$
	\eq{
		\Hilbert(\K \vD,\K \vD') \leq \la(\K) \Hilbert(\vD,\vD')
		\text{ where }
		\choice{
			\la(\K) \eqdef \frac{ \sqrt{\eta(\K)}-1 }{ \sqrt{\eta(\K)}+1 } < 1 \\
			\eta(\K) \eqdef \umax{i,j,k,\ell} \frac{ \K_{i,k} \K_{j,\ell} }{ \K_{j,k} \K_{i,\ell} }.
		}
	}
\end{thm}
\begin{proof}
	We recall the finite-dimensional Birkhoff--Hopf estimate. For a positive linear map $A$ on a cone, define its projective diameter
\index{projective!diameter}
	\[
		\Delta(A) \eqdef \sup_{u,v>0} \Hilbert(Au,Av).
	\]
	Then
	\[
		\Hilbert(Au,Av) \leq \tanh(\Delta(A)/4)\Hilbert(u,v).
	\]
	Indeed, after quotienting by positive scalings, write $r_k=u_k/v_k$ and normalize so that $e^{-h/2}\leq r_k\leq e^{h/2}$, where $h=\Hilbert(u,v)$. The ratio between two coordinates of $Au/Av$ is a quotient of two weighted averages of the numbers $r_k$. A two-point extremal argument shows that the largest possible contraction is obtained when the mass of the two weights is placed on the two endpoints $e^{-h/2}$ and $e^{h/2}$; the cross-ratio bound defining $\Delta(A)$ then gives
	\[
		\Hilbert(Au,Av)
		\leq
		2\log\frac{e^{\Delta(A)/4}e^{h/2}+e^{-\Delta(A)/4}e^{-h/2}}
		{e^{\Delta(A)/4}e^{-h/2}+e^{-\Delta(A)/4}e^{h/2}}
		\leq \tanh(\Delta(A)/4)h.
	\]
	For the matrix $\K$, its projective diameter is
\index{projective!diameter}
	\[
		\Delta(\K)=\log \eta(\K),
		\qquad
		\eta(\K)=\max_{i,j,k,\ell}
		\frac{\K_{i,k}\K_{j,\ell}}{\K_{j,k}\K_{i,\ell}}.
	\]
	Therefore $\tanh(\Delta(\K)/4)=(\sqrt{\eta(\K)}-1)/(\sqrt{\eta(\K)}+1)$, which is the claimed contraction factor.
\end{proof}

This result extends to arbitrary convex cones and affine mappings from the cone to its interior.
\index{affine map}

The following theorem of~\cite{franklin1989scaling} uses Theorem~\ref{thm-birkhoff} to show the linear convergence of Sinkhorn's iterations.
\index{linear!convergence}
\index{Sinkhorn!algorithm}

\begin{thm}[Linear convergence of Sinkhorn]
\index{linear!convergence}
	One has $(\it{\uD},\it{\vD}) \rightarrow (\uD^\star,\vD^\star)$ and
	\eql{\label{eq-convlin-sinkh}
		\Hilbert(\it{\uD}, \uD^\star) = O(\la(\K)^{2\ell}), \quad
		\Hilbert(\it{\vD}, \vD^\star) = O(\la(\K)^{2\ell}).
	}
	One also has
	\eql{\label{eq-convsinkh-control}
		\Hilbert(\it{\uD}, \uD^\star) \leq \frac{\Hilbert( \it{\P}\ones_m,\a )}{1-\la(\K)}
		\qandq
		\Hilbert(\it{\vD}, \vD^\star) \leq \frac{\Hilbert( \P^{(\ell),\top} \ones_n,\b )}{1-\la(\K)},
	}
	where we denoted $\it{\P} \eqdef \diag(\it{\uD}) \K \diag(\it{\vD})$. Lastly, one has
	\eql{\label{eq-convlin-sinkh-prim}
		\norm{\log(\it{\P}) - \log(\P^\star)}_\infty \leq \Hilbert(\it{\uD}, \uD^\star) + \Hilbert(\it{\vD}, \vD^\star)
	}
	where $\P^\star$ is the unique solution of~\eqref{eq-regularized-discr}.
\end{thm}

\begin{proof}
	Notice that for any $(\vD,\vD') \in (\RR_{+,*}^m)^2$, one has
	\eq{
		\Hilbert(\vD,\vD') = \Hilbert(\vD/\vD',\ones_m) = \Hilbert(\ones_m/\vD,\ones_m/\vD'),
	}
	since indeed $\Hilbert(\a/\vD,\a/\vD') = \Hilbert(\vD,\vD')$.
	This shows that
	\begin{align*}
		\Hilbert(\itt{\uD},\uD^\star) &= \Hilbert\pa{ \frac{\a}{\K \it{\vD}}, \frac{\a}{\K \vD^\star} }
		= \Hilbert( \K \it{\vD}, \K \vD^\star ) \leq \la(\K) \Hilbert( \it{\vD}, \vD^\star ).
	\end{align*}
	where we used Theorem~\ref{thm-birkhoff}. This shows~\eqref{eq-convlin-sinkh}. One also has, using the triangular inequality,
\index{triangle inequality}
	\begin{align*}
		\Hilbert(\it{\uD},\uD^\star) &\leq \Hilbert(\itt{\uD},\it{\uD}) + \Hilbert(\itt{\uD},\uD^\star)
		\leq \Hilbert\pa{ \frac{\a}{\K \it{\vD}},\it{\uD} } + \la(\K) \Hilbert(\it{\uD},\uD^\star) \\
		&= \Hilbert\pa{ \a,\it{\uD} \odot  ( \K \it{\vD} ) } + \la(\K) \Hilbert(\it{\uD},\uD^\star),
	\end{align*}
	which gives the first part of~\eqref{eq-convsinkh-control} since
	$\it{\uD} \odot  ( \K \it{\vD} ) = \it{\P}\ones_m$ (the second one being similar).
	The proof of~\eqref{eq-convlin-sinkh-prim} follows from~\cite[Lemma 3]{franklin1989scaling}.
\end{proof}

\paragraph{Dual-potential form of the contraction.}
\index{dual!potential}
\index{Birkhoff contraction theorem}

The Hilbert-metric contraction above can also be read directly on the dual potentials $(f,g)$ through their variation norms. For bounded cost $c$ (e.g. on compact spaces),
\index{dual!potential}
\eq{
	\norm{f_k-f^\star}_V = O(\la^k)
	\qandq
		\norm{g_k-g^\star}_V = O(\la^k)
}
\eq{
	\norm{ \log \frac{\d\pi_k}{\d\pi^\star} }_\infty
	=
	\norm{ (f_k-f^\star) \oplus (g_k-g^\star) }_\infty
	\leq
	\norm{f_k-f^\star}_V + \norm{g_k-g^\star}_V
}
where the contraction ratio is the Birkhoff factor of the Gibbs kernel $\K_\epsilon=e^{-c/\epsilon}$. Namely, with $\eta=\eta(\K_\epsilon)$ as in Theorem~\ref{thm-birkhoff} and $R\eqdef\sup c-\inf c$,
\index{contraction ratio}
\index{Gibbs!kernel}
\[
	\la=\frac{\sqrt{\eta}-1}{\sqrt{\eta}+1}
	\leq
	\tanh(R/(2\epsilon))<1 .
\]
One also has the following bounds
\eq{
	\norm{f_k-f^\star}_V \leq \frac{ \norm{\log\frac{\d \pi_{k,1}}{\d\al}}_\infty }{1-\la}
}
which can be used to provide a posterior estimate of the rate of convergence and serve as a stopping criterion.


The bound~\eqref{eq-convsinkh-control} shows that some error measures on the marginal constraints violation, for instance, $\norm{\it{\P} \ones_m - \a}_1$ and $\norm{\transp{\it{\P}} \ones_n - \b}_1$, are useful stopping criteria to monitor the convergence.
\index{marginal!constraint}
This theorem shows that the Sinkhorn algorithm converges linearly, but the worst-case rate becomes exponentially bad as $\epsilon \rightarrow 0$, since the global contraction factor is controlled by the cost range divided by $\epsilon$. In practice, one often observes a much better local linear regime after enough iterations.
\index{Sinkhorn!algorithm}
The same Hilbert-metric mechanism extends beyond finite matrices to positive integral operators under suitable compactness and positivity assumptions.
An important limitation of this analysis is that it requires a uniformly bounded cost and a kernel bounded away from degeneracy; Gaussian distributions with quadratic cost therefore require a different approach.
\index{cost!quadratic}

\section{Entropic Optimal Transport between Gaussians}
\index{Gaussian!Sinkhorn}
\label{sec-gaussian-sinkhorn}

Gaussian marginals provide an explicit finite-dimensional model of Sinkhorn's behavior. The soft $c$-transform preserves quadratic potentials, the optimal entropic coupling is Gaussian, and the value can be written with matrix square roots~\cite{janati2020gaussian}. This is the entropic counterpart of the Gaussian $\Wass_2$ formula in Proposition~\ref{prop-gaussian-w2-bures}.
\index{matrix square root}
\index{soft!c-transform}
\index{quadratic!potential}
\index{Gaussian!marginals}

\begin{prop}[Quadratic closure of Sinkhorn iterates]\label{prop-gaussian-sinkhorn-closure}
\index{quadratic!closure}
\index{Sinkhorn!iteration}
	Let $\be=\Gaussian(\mean_\be,\cov_\be)$ on $\RR^d$ and take $c(x,y)=\norm{x-y}^2$. If $g(y)$ is a quadratic polynomial such that the Gaussian integral below is finite, then the soft transform
\index{soft!transform}
	\[
		f(x)
		=
		-\epsilon\log
		\int
		\exp\!\left(\frac{g(y)-\norm{x-y}^2}{\epsilon}\right)
		\d\be(y)
	\]
	is a quadratic polynomial in $x$. In particular, starting Sinkhorn from $g_0=0$ gives
	\[
		f_1(x)
		=
		\frac{\epsilon}{2}
		\log\det\!\left(\Id+\frac{2\cov_\be}{\epsilon}\right)
		+
		\epsilon
		\dotp{x-\mean_\be}{(\epsilon\Id+2\cov_\be)^{-1}(x-\mean_\be)}.
	\]
\end{prop}

\begin{proof}
	The exponent is the sum of a quadratic polynomial in $y$ and the logarithm of the Gaussian density of $\be$. Completing the square in $y$ evaluates the integral as a positive constant times the exponential of a quadratic polynomial in $x$. Taking $-\epsilon\log$ therefore gives a quadratic polynomial.

	For $g_0=0$, let $Y\sim\be$. The Gaussian identity
	\[
		\EE\exp\!\left(-\frac{\norm{x-Y}^2}{\epsilon}\right)
		=
		\det\!\left(\Id+\frac{2\cov_\be}{\epsilon}\right)^{-1/2}
		\exp\!\left(
			-\dotp{x-\mean_\be}{(\epsilon\Id+2\cov_\be)^{-1}(x-\mean_\be)}
		\right)
	\]
	gives the displayed expression.
\end{proof}

\begin{prop}[Balanced entropic OT between Gaussians]\label{prop-gaussian-sinkhorn-closed-form}
\index{entropic!OT}
\index{Gaussian!Sinkhorn}
	Let $\al=\Gaussian(\mean_\al,\cov_\al)$ and $\be=\Gaussian(\mean_\be,\cov_\be)$ with positive-definite covariances, and let
	\[
		\cov_\al^{1/2}\cov_\be^{1/2}
		=
		U\diag(\sigma_i)V^\top
	\]
	be a singular-value decomposition. For the balanced objective
	\[
		\min_{\pi\in\Couplings(\al,\be)}
		\int\norm{x-y}^2\d\pi(x,y)
		+
		\epsilon\KL(\pi|\al\otimes\be),
	\]
	the optimizer is Gaussian with cross-covariance
	\[
		K_\epsilon
		=
		\cov_\al^{1/2}
		U\diag(s_i)V^\top
		\cov_\be^{1/2},
		\qquad
		s_i
		=
		\frac{\sqrt{\epsilon^2+16\sigma_i^2}-\epsilon}{4\sigma_i}.
	\]
	The optimal value is
	\[
		\norm{\mean_\al-\mean_\be}^2
		+
		\tr(\cov_\al)+\tr(\cov_\be)
		+
		\sum_i
		\left(
			-2\sigma_i s_i
			-\frac{\epsilon}{2}\log(1-s_i^2)
		\right).
	\]
	As $\epsilon\downarrow0$, $s_i\to1$ and the full covariance contribution, including the two trace terms, converges to $\Bb(\cov_\al,\cov_\be)^2$.
\end{prop}

\begin{proof}
	Let $(X,Y)$ be any coupling with finite second moments and cross-covariance
	\[
		K=\EE\big[(X-\mean_\al)(Y-\mean_\be)^\top\big].
	\]
	Replacing $(X,Y)$ by the Gaussian vector with the same mean and covariance leaves the quadratic cost unchanged. Since the marginals are fixed,
\index{cost!quadratic}
	\[
		\KL(\pi|\al\otimes\be)
		=
		-h(X,Y)+h(\al)+h(\be),
	\]
	where $h$ denotes differential entropy when it is finite. Among laws with a fixed covariance, the Gaussian maximizes entropy; if the entropy is not finite, the relative entropy is already $+\infty$. Thus the Gaussian replacement cannot increase the objective, and it is enough to optimize over Gaussian couplings.
\index{entropy!relative}

	Any such coupling has covariance
	\[
		\begin{pmatrix}
			\cov_\al & K\\
			K^\top & \cov_\be
		\end{pmatrix}.
	\]
	Write $K=\cov_\al^{1/2}S\cov_\be^{1/2}$. The block covariance constraint is equivalent to the singular values of $S$ being at most one, and finite entropy forces them to be strictly smaller than one. The cost depends on $K$ through
	\[
		\norm{\mean_\al-\mean_\be}^2+\tr(\cov_\al)+\tr(\cov_\be)-2\tr(K),
	\]
	while
	\[
		\KL(\pi|\al\otimes\be)
		=
		-\frac12\log\det(\Id-SS^\top).
	\]
	By von Neumann's trace inequality, the minimizer aligns $S$ with the singular vectors of $\cov_\al^{1/2}\cov_\be^{1/2}$, so $S=U\diag(s_i)V^\top$. The problem separates into scalar minimizations
	\[
		\min_{0\leq s<1}
		-2\sigma_i s-\frac{\epsilon}{2}\log(1-s^2).
	\]
	The first-order condition is $2\sigma_i=\epsilon s/(1-s^2)$, whose positive solution is the displayed $s_i$. Substitution gives the value formula. Since $s_i\to1$ and $\epsilon\log(1-s_i^2)\to0$ as $\epsilon\downarrow0$, the spectral sum converges to $-2\sum_i\sigma_i$. The identity
	\[
		\sum_i\sigma_i
		=
		\tr\!\left((\cov_\al^{1/2}\cov_\be\cov_\al^{1/2})^{1/2}\right)
	\]
	then gives the Bures--Wasserstein covariance contribution.
\index{Bures-Wasserstein geometry}
\end{proof}

\begin{alg}[Closed-form Gaussian Sinkhorn coupling]\label{alg:gaussian-sinkhorn-closed-form}
\index{Gaussian!Sinkhorn}
\index{entropic!OT}
\textbf{Input:} Gaussian marginals $\al=\Gaussian(\mean_\al,\cov_\al)$, $\be=\Gaussian(\mean_\be,\cov_\be)$, scale $\epsilon>0$.

\textbf{Output:} Gaussian entropic coupling covariance.

\textbf{Compute singular value decomposition}
\(\cov_\al^{1/2}\cov_\be^{1/2} = U\diag(\sigma_i)V^\top .\)

\textbf{For} each $\sigma_i>0$ \textbf{do}
\begin{algblock}
\(s_i=\frac{\sqrt{\epsilon^2+16\sigma_i^2}-\epsilon}{4\sigma_i}.\)
\end{algblock}
\textbf{Set} cross-covariance:
\(K_\epsilon = \cov_\al^{1/2}U\diag(s_i)V^\top\cov_\be^{1/2}.\)
\textbf{Return} Gaussian coupling with means $(\mean_\al,\mean_\be)$, marginal covariances $(\cov_\al,\cov_\be)$, and cross-covariance $K_\epsilon$.
\end{alg}

\begin{cor}[Gaussian Sinkhorn divergence and smoothed Bures term]\label{cor-gaussian-sinkhorn-divergence}
\index{Bures!smoothed term}
\index{Sinkhorn!divergence}
\index{Gaussian!Sinkhorn divergence}
\index{Gaussian!Sinkhorn}
	Let $\al=\Gaussian(\mean_\al,\cov_\al)$ and $\be=\Gaussian(\mean_\be,\cov_\be)$ have positive-definite covariances. For $r>0$, define
	\[
		\tau_\epsilon(r)
		\eqdef
		\frac{\sqrt{\epsilon^2+16r^2}-\epsilon}{4r},
		\qquad
		\psi_\epsilon(r)
		\eqdef
		-2r\,\tau_\epsilon(r)
		-\frac{\epsilon}{2}\log\bigl(1-\tau_\epsilon(r)^2\bigr).
	\]
	If $\sigma_i(\Sigma,\Lambda)$ denotes the singular values of $\Sigma^{1/2}\Lambda^{1/2}$ and $\lambda_i(\Sigma)$ the eigenvalues of $\Sigma$, then the debiased Sinkhorn divergence~\eqref{eq-sinkhorn-divergence} is
\index{Sinkhorn!divergence}
	\[
		\bar\MK_{\norm{\cdot-\cdot}^2}^{\epsilon}(\al,\be)
		=
		\norm{\mean_\al-\mean_\be}^2
		+
		\Bb_\epsilon(\cov_\al,\cov_\be)^2,
	\]
	where the Gaussian covariance contribution is the debiased smoothed Bures term
\index{Gaussian!covariance}
\index{Bures!smoothed term}
	\[
		\Bb_\epsilon(\Sigma,\Lambda)^2
		\eqdef
		\sum_i \psi_\epsilon\bigl(\sigma_i(\Sigma,\Lambda)\bigr)
		-\frac12\sum_i\psi_\epsilon\bigl(\lambda_i(\Sigma)\bigr)
		-\frac12\sum_i\psi_\epsilon\bigl(\lambda_i(\Lambda)\bigr).
	\]
	Moreover $\Bb_\epsilon(\Sigma,\Lambda)^2\to\Bb(\Sigma,\Lambda)^2$ as $\epsilon\downarrow0$, where $\Bb$ is the Bures--Wasserstein metric of Proposition~\ref{prop-gaussian-w2-bures}.
\index{Bures-Wasserstein geometry}
\end{cor}

\begin{proof}
	Proposition~\ref{prop-gaussian-sinkhorn-closed-form} writes the raw entropic value as the squared mean displacement plus trace terms and a spectral sum. With the notation above, the spectral part is exactly $\sum_i\psi_\epsilon(\sigma_i(\cov_\al,\cov_\be))$. Applying the same formula to the self-costs $(\al,\al)$ and $(\be,\be)$ replaces these singular values by the eigenvalues of $\cov_\al$ and $\cov_\be$. In the polarization formula~\eqref{eq-sinkhorn-divergence}, the trace terms cancel:
\index{Sinkhorn!divergence}
	\[
		\tr\cov_\al+\tr\cov_\be
		-\frac12(2\tr\cov_\al)
		-\frac12(2\tr\cov_\be)=0,
	\]
	while the polarization of the squared mean terms leaves exactly $\norm{\mean_\al-\mean_\be}^2$. This gives the displayed formula. Since $\tau_\epsilon(r)\to1$ and $\epsilon\log(1-\tau_\epsilon(r)^2)\to0$, one has $\psi_\epsilon(r)\to-2r$. The limit is therefore
	\[
		\tr\Sigma+\tr\Lambda
		-2\sum_i\sigma_i(\Sigma,\Lambda)
		=
		\Bb(\Sigma,\Lambda)^2,
	\]
	which is the Bures formula~\eqref{eq-bure-defn}.
\end{proof}

\begin{prop}[One-dimensional Gaussian Sinkhorn rate]\label{prop-gaussian-sinkhorn-1d-rate}
\index{Gaussian!Sinkhorn}
	Consider $\al=\be=\Gaussian(0,1)$ on $\RR$ with $c(x,y)=(x-y)^2$. If a dual potential has the form $g_q(y)=q y^2+\mathrm{cst}$, then one soft transform has quadratic coefficient
\index{dual!potential}
\index{soft!transform}
	\[
		T_\epsilon(q)
		=
		1-\frac{1}{1-q+\epsilon/2},
		\qquad q<1+\epsilon/2,
	\]
	and one full Sinkhorn cycle acts as $q\mapsto T_\epsilon(T_\epsilon(q))$. The fixed point $q_\star=T_\epsilon(q_\star)$ is determined by
	\[
		A_\star^2-\frac{\epsilon}{2}A_\star-1=0,
		\qquad
		A_\star\eqdef 1-q_\star+\frac{\epsilon}{2}
		=
		\frac{\epsilon+\sqrt{\epsilon^2+16}}{4}.
	\]
	Consequently the local asymptotic contraction factor of one full Sinkhorn cycle on the quadratic coefficient is
	\[
		\rho_\epsilon=A_\star^{-4}
		=
		\left(\frac{4}{\epsilon+\sqrt{\epsilon^2+16}}\right)^4 .
	\]
\end{prop}

\begin{proof}
	Completing the square in
	\[
		\int
		\exp\!\left(\frac{q y^2-(x-y)^2}{\epsilon}\right)
		\d\Gaussian(0,1)(y)
	\]
	gives the coefficient $T_\epsilon(q)$. The fixed-point equation $q_\star=1-1/A_\star$, together with $q_\star=1+\epsilon/2-A_\star$, gives
	\[
		A_\star^2-\frac{\epsilon}{2}A_\star-1=0.
	\]
	The positive solution is the displayed $A_\star$. Since
	\[
		T_\epsilon'(q)=-\frac{1}{(1-q+\epsilon/2)^2},
	\]
	the derivative of the full-cycle map at the fixed point is $T_\epsilon'(q_\star)^2=A_\star^{-4}$. This gives the local asymptotic rate.
\end{proof}

This scalar calculation illustrates the general Gaussian convergence picture of Chizat, Delalande and Va\v{s}kevi\v{c}ius~\cite{chizat2024sharper}: the rate improves when $\epsilon$ is large or the covariance scales overlap well, and deteriorates in the small-temperature limit where the entropic coupling approaches a deterministic Brenier map.
\index{Brenier!map}


\section{Sample Complexity}
\index{sample complexity}

This section separates two statistical regimes. Exact OT resolves geometry at all spatial scales and pays dimension-dependent empirical rates; fixed-temperature Sinkhorn divergences smooth the dual potentials and recover parametric fluctuations, at the price of regularization bias.
\index{exact OT}
\index{Sinkhorn!divergence}
\index{dual!potential}

The sample complexity of unregularized OT suffers from the curse of dimensionality. Entropic regularization changes this picture: for a fixed $\epsilon>0$, Sinkhorn divergences have parametric $n^{-1/2}$ statistical rates, although the constant deteriorates when $\epsilon\to0$~\cite{genevay2018sample,bigot2017central}. Related two-sample-testing viewpoints are developed in~\cite{ramdas2017wasserstein}, and the large-$\epsilon$ kernel limit connects to classical MMD tests~\cite{gretton2012kernel}. This improvement should be balanced against the regularization bias, which vanishes only when $\epsilon$ is sent to zero.
\index{maximum mean discrepancy}
\index{curse of dimensionality}
\index{entropic!regularization}
\index{sample complexity}

\begin{figure}[ht]
\centering
\setlength{\tabcolsep}{2pt}
\begin{tabular}{@{}cc@{}}
\includegraphics[width=.45\linewidth]{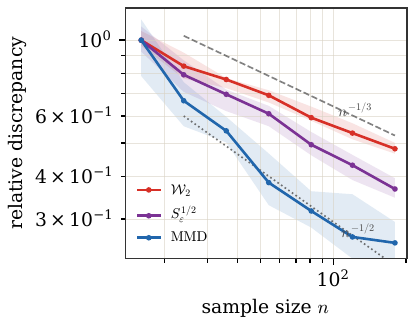} &
\includegraphics[width=.45\linewidth]{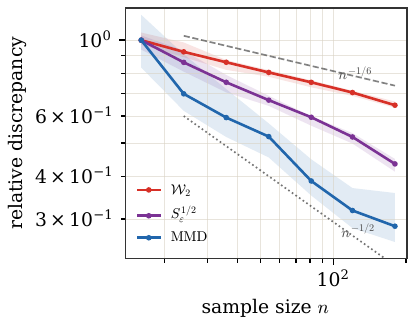} \\[-.1em]
\small $d=3$ & \small $d=6$
\end{tabular}
\caption{Empirical fluctuations in dimensions three and six. For each sample size $n$, two independent empirical measures $\al_n$ and $\al'_n$ are drawn from the same standard Gaussian law in $\RR^d$. The curves show the median of $D(\al_n,\al'_n)$, normalized by its value at the smallest displayed $n$, with interquartile bands. Exact OT follows a slower dimension-dependent scale, while MMD and the fixed-$\epsilon$ Sinkhorn divergence behave closer to the parametric $n^{-1/2}$ guide. This is a statistical illustration, not a solver benchmark.}
\index{Sinkhorn!divergence}
\index{empirical!measure}
\label{fig:sinkhorn-bias-variance-tradeoff}
\end{figure}

\begin{prop}[Empirical OT has dimension-dependent value rates]\label{prop-empirical-ot-rate}
\index{empirical!OT}
\index{empirical!OT rate}
	Let $\alpha$ and $\beta$ be probability distributions with densities bounded above and below on $[0,1]^d$, and let $\hat\alpha_n$ and $\hat\beta_m$ be independent empirical measures. For $d>2p$, the expected empirical error for estimating the two-sample distance obeys
\index{empirical!measure}
	\[
		\EE\abs{\Wass_p(\hat\alpha_n,\hat\beta_m)-\Wass_p(\alpha,\beta)}
		\lesssim
		n^{-1/d}+m^{-1/d}.
	\]
	The exponent changes in low dimension, but the important message is that exact OT deteriorates with the ambient dimension.
\index{exact OT}
\end{prop}

\begin{proof}
	By the triangle inequality,
\index{triangle inequality}
	\[
		\abs{\Wass_p(\hat\alpha_n,\hat\beta_m)-\Wass_p(\alpha,\beta)}
		\leq
		\Wass_p(\hat\alpha_n,\alpha)+\Wass_p(\hat\beta_m,\beta).
	\]
	We then recall the standard multiscale argument for each one-sample term, suppressing constants~\cite{dudley1969speed,fournier2015rate,weed2017sharp}. For the upper bound, partition $[0,1]^d$ into dyadic cubes. At scale $2^{-j}$, the empirical mass fluctuation over the cells is of order $n^{-1/2}2^{jd/2}$, while moving this excess mass inside cells costs $2^{-j}$. Summing the multiscale contributions up to the scale where the expected number of samples per cell is of order one gives $2^{-J}$ with $2^{Jd}\simeq n$, hence $n^{-1/d}$. The same estimate with $m$ samples gives the second term. Matching lower bounds for empirical OT follow from packing arguments; they show that this dimension dependence is intrinsic for exact OT.
\index{dyadic cube}
\index{packing argument}
\index{exact OT}
\index{empirical!OT}
\index{dimension dependence}
\end{proof}

\begin{prop}[MMD has a parametric value rate]\label{prop-mmd-sample-rate}
\index{maximum mean discrepancy}
\index{rate!parametric}
	Let $k$ be a bounded positive definite kernel with RKHS $\Hh_k$, and define
\index{RKHS}
\index{kernel!positive definite}
	\[
		\MMD_k(\alpha,\beta)
		=
		\norm{\int k(x,\cdot)\d(\alpha-\beta)(x)}_{\Hh_k}.
	\]
	If $\hat\alpha_n$ and $\hat\beta_m$ are independent empirical measures, then
\index{empirical!measure}
	\[
		\EE\abs{\MMD_k(\hat\alpha_n,\hat\beta_m)-\MMD_k(\alpha,\beta)}
\index{maximum mean discrepancy}
		\leq
		\kappa\left(\frac1{\sqrt n}+\frac1{\sqrt m}\right)
	\]
	when $k(x,x)\leq\kappa^2$.
\end{prop}

\begin{proof}
	Let $\Phi(x)=k(x,\cdot)$ be the feature map and $m_\alpha=\EE\Phi(X)$. The reverse triangle inequality for the RKHS norm gives
\index{triangle inequality}
\index{RKHS}
	\[
		\abs{\MMD_k(\hat\alpha_n,\hat\beta_m)-\MMD_k(\alpha,\beta)}
\index{maximum mean discrepancy}
		\leq
		\MMD_k(\hat\alpha_n,\alpha)+\MMD_k(\hat\beta_m,\beta).
	\]
	Independence cancels the cross terms after taking the squared norm and expectation, giving
	\[
		\EE\MMD_k(\hat\alpha_n,\alpha)^2
		=\frac1n\EE\norm{\Phi(X)-m_\alpha}_{\Hh_k}^2
		=\frac{1}{n}\Big(\EE k(X,X)-\EE k(X,X')\Big).
	\]
	The same estimate applies to $\hat\beta_m$, and Jensen's inequality together with $k(x,x)\leq\kappa^2$ gives the displayed bound.
\index{Jensen inequality}
\end{proof}

\begin{prop}[Sinkhorn divergences interpolate the rates]\label{prop-sinkhorn-sample-rate}
\index{Sinkhorn!divergence}
	Assume that $\alpha$ and $\beta$ are supported in a compact subset of $\RR^d$ and that the cost is smooth. For fixed $\epsilon>0$, debiased Sinkhorn divergences satisfy representative empirical bounds of the form
\index{Sinkhorn!divergence}
	\[
		\EE\abs{\bar\MK_\c^\epsilon(\hat\alpha_n,\hat\beta_m)-\bar\MK_\c^\epsilon(\alpha,\beta)}
		\leq
		C_{c,d}\,\epsilon^{-d/2}\left(\frac1{\sqrt n}+\frac1{\sqrt m}\right),
	\]
	up to constants and exponents depending on the precise smoothness class and support diameter. Thus regularization removes the $n^{-1/d}$ curse for fixed $\epsilon$, while the prefactor deteriorates as $\epsilon\to0$.
\end{prop}

\begin{proof}
	The proof follows the empirical-process argument of~\cite{genevay2018sample}. By the envelope theorem, the fluctuation of $\MK_\c^\epsilon$ with respect to its first marginal is controlled by the class of entropic dual potentials. The soft $c$-transform smooths these potentials at spatial scale $\sqrt\epsilon$ for a quadratic-type cost. Covering a bounded $d$-dimensional domain at this scale gives an effective complexity of order $\epsilon^{-d/2}$. Standard Rademacher or Dudley entropy bounds then give an empirical-process fluctuation of order $\epsilon^{-d/2}/\sqrt n$ for each marginal. Applying the same estimate to the three terms defining the debiased divergence gives the stated bound.
\index{Rademacher bound}
\index{Dudley entropy}
\index{soft!c-transform}
\index{empirical!process}
\index{envelope theorem}
\index{dual!potential}
\end{proof}

\begin{rem}[No free lunch when approximating exact OT]\label{rem-sinkhorn-no-free-lunch}
\index{exact OT}
	The parametric rate in Proposition~\ref{prop-sinkhorn-sample-rate} holds for fixed $\epsilon$. If the goal is to approximate the unregularized OT value, one must also account for the regularization bias. In a typical bounded-cost finite-dimensional regime,
\index{rate!parametric}
	\[
		\abs{\bar\MK_\c^\epsilon(\alpha,\beta)-\MK_\c(\alpha,\beta)}
		\leq C\epsilon,
		\qquad
		\EE\abs{\bar\MK_\c^\epsilon(\hat\alpha_n,\hat\beta_n)-\bar\MK_\c^\epsilon(\alpha,\beta)}
		\leq C_{c,d}\epsilon^{-d/2}n^{-1/2}.
	\]
	Balancing the two terms gives $\epsilon\simeq n^{-1/(d+2)}$ and total error of order $n^{-1/(d+2)}$. Equivalently, target accuracy $\eta$ requires choosing $\epsilon\simeq\eta$ and $n\simeq\eta^{-(d+2)}$ samples under this bound. Thus entropic smoothing improves the statistical behavior at fixed scale, but approximating exact OT still forces a bias-variance tradeoff whose exponent deteriorates with dimension.
\index{exact OT}
\index{entropic!smoothing}
\end{rem}



\chapter{Generalized Wasserstein Distances}
\index{generalized!Wasserstein}
\index{Wasserstein!distance}
\label{sec-extensions}
\label{sec-generalized-wasserstein-distances}

The first family of extensions keeps the idea of a distance between measures, but changes the geometry used to compare them. The variants below relax mass conservation, reduce high-dimensional transport to one-dimensional projections, or replace the trace quadratic cost by spectral gauges and robust projected viewpoints. They are useful when standard $\Wass_p$ is too rigid or too expensive, while still preserving a metric interpretation.
\index{mass!conservation}
\index{one-dimensional!projection}
\index{cost!quadratic}
\index{spectral!gauge}

\section{Unbalanced OT}
\index{unbalanced!OT}
\label{sec-unbalanced}

Unbalanced OT allows mass creation and destruction by penalizing marginal mismatch. It is essential when histograms are not normalized, when observations contain outliers, or when only part of the source should match the target~\cite{LieroMielkeSavareLong,2015-chizat-unbalanced,2017-chizat-focm}.
\index{histogram}
\index{mass!creation}

\paragraph{Relaxed formulation.}
\index{Kantorovich!relaxation}

For nonnegative measures $(\al,\be) \in \Mm_+(\Xx) \times \Mm_+(\Yy)$, a generic relaxed formulation is
\index{nonnegative measure}
\eql{\label{eq-unbalanced-primal}
	\UW_c(\al,\be)= \uinf{\pi \in \Mm_+(\Xx \times \Yy)} \int_{\Xx \times \Yy} c(x,y) \d\pi(x,y)
		+ \Dd_{\psi_1}(\pi_1|\al) + \Dd_{\psi_2}(\pi_2|\be),
}
where $\psi_1,\psi_2$ are convex entropy functions. Exact conservation $(\pi_1,\pi_2)=(\al,\be)$ is replaced by a cost for changing the marginals. Writing $\psi_s=\tau\bar\psi_s$ exposes the relaxation scale:
\index{entropy!function}
\[
	\UW_{c,\tau}(\alpha,\beta)
	=
	\inf_{\pi\geq0}
	\int c\d\pi
	+
	\tau\Dd_{\bar\psi_1}(\pi_1|\alpha)
	+
	\tau\Dd_{\bar\psi_2}(\pi_2|\beta).
\]
Large $\tau$ makes marginal mismatch expensive and approaches balanced OT when the total masses are compatible. Small $\tau$ makes creation and destruction cheap; after rescaling by $\tau$, the zero-transport part reveals the pure divergence geometry.

\begin{prop}[Small-transport-scale limit for marginal penalties]\label{prop-unbalanced-small-scale-limit}
\index{marginal!penalty}
	Assume that $\alpha,\beta$ are finite measures on a compact metric space $\Xx$, that $c$ is continuous, $c\geq0$, and $c(x,y)=0$ if and only if $x=y$. Assume also that the marginal divergences are nonnegative, weak-* lower semicontinuous, and have weak-* compact sublevel sets on $\Mm_+(\Xx)$. Then
\index{compact sublevel set}
\index{finite measure}
\index{lower semicontinuity}
\index{marginal!divergence}
	\[
		\lim_{\tau\downarrow0}\frac1\tau\UW_{c,\tau}(\alpha,\beta)
		=
		\inf_{\rho\in\Mm_+(\Xx)}
		\Dd_{\bar\psi_1}(\rho|\alpha)
		+
		\Dd_{\bar\psi_2}(\rho|\beta).
	\]
	The right-hand side is the infimal gluing divergence obtained by matching the two measures through a common zero-transport marginal $\rho$. In the dominated case, if $\alpha=a\lambda$, $\beta=b\lambda$, and the minimizing common marginal is absolutely continuous, $\rho=r\lambda$, this divergence decouples pointwise as
	\[
		\int \mathfrak m_{\bar\psi_1,\bar\psi_2}(a(x),b(x))\d\lambda(x),
		\qquad
		\mathfrak m_{\bar\psi_1,\bar\psi_2}(a,b)
		\eqdef
		\inf_{r\geq0} a\,\bar\psi_1(r/a)+b\,\bar\psi_2(r/b),
	\]
	with the usual recession conventions when $a=0$ or $b=0$. For superlinear entropies, and in particular for KL, finite energy forces this dominated form. Thus, when $\bar\psi_1=\bar\psi_2$ is the KL entropy,
\index{recession convention}
	\[
		\inf_{\rho\in\Mm_+(\Xx)} \KL(\rho|\alpha)+\KL(\rho|\beta)
		=
		\int (\sqrt{a}-\sqrt{b})^2\d\lambda .
	\]
	Thus the KL marginal relaxation contains the squared Hellinger distance as its local mass-variation limit.
\index{marginal!relaxation}
\end{prop}

\begin{proof}
	For the upper bound, restrict to diagonal plans $\pi=(\Id,\Id)_\sharp\rho$, whose transport cost is zero and whose two marginals are both $\rho$. This gives the desired upper bound after optimizing over $\rho$.

	For the lower bound, let $\tau_n\downarrow0$ and let $\pi_n$ be almost minimizing plans with bounded scaled values $\tau_n^{-1}\UW_{c,\tau_n}(\alpha,\beta)$. Since the divergences are nonnegative, $\int c\d\pi_n=O(\tau_n)$, hence $\int c\d\pi_n\to0$. The bounded scaled values also put the two marginals in compact divergence sublevel sets. Since a coupling has the same total mass as each marginal, the couplings are tight on $\Xx\times\Xx$. Up to subsequences, $\pi_n\rightharpoonup\pi_0$. Lower semicontinuity of the transport cost yields $\int c\d\pi_0=0$, so $\pi_0$ is concentrated on the diagonal. Its two marginals are therefore equal to a common measure $\rho$. Lower semicontinuity of the marginal divergences gives
\index{lower semicontinuity}
\index{marginal!divergence}
	\[
		\liminf_n\frac1{\tau_n}\UW_{c,\tau_n}(\alpha,\beta)
		\geq
		\Dd_{\bar\psi_1}(\rho|\alpha)+\Dd_{\bar\psi_2}(\rho|\beta),
	\]
	and optimizing over $\rho$ gives the lower bound.

	In the dominated case, writing $\rho=r\lambda$ gives
	\[
		\Dd_{\bar\psi_1}(\rho|\alpha)+\Dd_{\bar\psi_2}(\rho|\beta)
		=
		\int
		a\,\bar\psi_1(r/a)+b\,\bar\psi_2(r/b)
		\d\lambda,
	\]
	so the minimization over $\rho$ decouples into the scalar envelope $\mathfrak m_{\bar\psi_1,\bar\psi_2}$. For KL, no singular part is admissible when $\alpha$ and $\beta$ are dominated by $\lambda$. The pointwise objective is $r\log(r/a)-r+a+r\log(r/b)-r+b$. Its optimality condition is $\log(r/a)+\log(r/b)=0$, hence $r=\sqrt{ab}$, and the minimum is $a+b-2\sqrt{ab}=(\sqrt a-\sqrt b)^2$.
\end{proof}

\begin{prop}[Dual of unbalanced optimal transport]\label{prop-dual-unbalanced-ot}
\index{unbalanced!OT dual}
\index{unbalanced!OT}
	Under the usual Fenchel--Rockafellar qualification assumptions, one has equality between~\eqref{eq-unbalanced-primal} and
\index{duality!Fenchel-Rockafellar}
\index{Fenchel duality}
	\[
		\UW_c(\al,\be)
		=
		\usup{f\oplus g\leq c}
		-
		\Dd_{\psi_1}^*(-f|\al)
		-
		\Dd_{\psi_2}^*(-g|\be).
	\]
\end{prop}

\begin{proof}
	Use the variational formula~\eqref{eq-legendre} for the dual of a divergence and introduce the marginal variables through continuous potentials:
	\[
	\inf_{\pi\geq0}\sup_{f,g}
	\int c\d\pi + \int -f\d\pi_1 + \int -g\d\pi_2
	-
	\Dd_{\psi_1}^*(-f|\al)
	-
	\Dd_{\psi_2}^*(-g|\be).
	\]
	Exchanging the infimum and the supremum gives
	\[
	\sup_{f,g}
	-
	\Dd_{\psi_1}^*(-f|\al)
	-
	\Dd_{\psi_2}^*(-g|\be)
	+
	\inf_{\pi\geq0}\int\big(c-(f\oplus g)\big)\d\pi.
	\]
	The last infimum is $0$ when $f\oplus g\leq c$ and $-\infty$ otherwise, which gives the displayed dual.
\end{proof}

\paragraph{Reverse and homogeneous formulations.}
\index{homogeneous formulation}

The Liero--Mielke--Savar\'e formulation rewrites marginal penalties as a local transport cost and then homogenizes it. Assuming first that the reference measures and transported marginals have mutually absolutely continuous parts, one can factor the objective as
\index{reference!measure}
\index{marginal!penalty}
\begin{align*}
&\int c(x,y)\d\pi(x,y)
+\Dd_{\psi_1}(\pi_1|\al)+\Dd_{\psi_2}(\pi_2|\be) \\
&\qquad =
\int
\left(
	c(x,y)
	+
	\psi_1\!\left(\frac{\d\pi_1}{\d\al}(x)\right)\frac{\d\al}{\d\pi_1}(x)
	+
	\psi_2\!\left(\frac{\d\pi_2}{\d\be}(y)\right)\frac{\d\be}{\d\pi_2}(y)
\right)
\d\pi(x,y).
\end{align*}
This motivates the local reverse cost
\eql{\label{eq-unbalanced-reverse-local-cost}
	L_c(r,s) \eqdef c + r\psi_1(1/r)+s\psi_2(1/s),
}
with the usual recession convention at $r=0$ or $s=0$. If $\al=F\pi_1+\al^\perp$ and $\be=G\pi_2+\be^\perp$ are the Lebesgue decompositions of the reference marginals with respect to the transported marginals, then the reverse formulation reads
\index{recession convention}
\index{Lebesgue decomposition}
\index{reverse formulation}
\[
	\UW_c(\al,\be)
	=
	\inf_{\pi\geq0}
	\int L_{c(x,y)}(F(x),G(y))\d\pi(x,y)
	+
	\psi_1(0)\al^\perp(\Xx)
	+
	\psi_2(0)\be^\perp(\Yy).
\]

The homogeneous formulation is obtained by taking the perspective transform of $L_c$,
\index{homogeneous formulation}
\index{perspective transform}
\eql{\label{eq-unbalanced-homogeneous-local-cost}
	H_c(r,s) \eqdef \inf_{\theta>0}\theta L_c(r/\theta,s/\theta),
}
which is positively $1$-homogeneous. It defines
\eql{\label{eq-homogeneous}
	\HW_c(\al,\be)
	=
	\inf_{\pi\geq0}
	\int H_{c(x,y)}(F(x),G(y))\d\pi(x,y)
	+
	\psi_1(0)\al^\perp(\Xx)
	+
	\psi_2(0)\be^\perp(\Yy).
}

\begin{prop}[Homogenization does not change the unbalanced cost]\label{prop-homogeneous-unbalanced}
\index{homogenization}
\index{cost!unbalanced}
	One has $\HW_c(\al,\be)=\UW_c(\al,\be)$.
\end{prop}

\begin{proof}
	The inequality $\HW\leq\UW$ follows from $H_c\leq L_c$ by taking $\theta=1$. Conversely, take a feasible measure $\pi$ in the homogeneous formulation. By definition of the perspective transform, for every $(x,y)$ and every $\eta>0$ there exists a scale $\theta(x,y)>0$ such that
\index{homogeneous formulation}
\index{perspective transform}
	\[
		H_{c(x,y)}(F(x),G(y))+\eta
		\geq
		\theta(x,y)L_{c(x,y)}\big(F(x)/\theta(x,y),G(y)/\theta(x,y)\big).
	\]
	Replacing $\pi$ by the rescaled measure $\tilde\pi=\theta\pi$ and the densities by $F/\theta$ and $G/\theta$ gives an admissible competitor for the reverse formulation with cost no larger than the homogeneous cost plus $\eta\pi(\Xx\times\Yy)$. Letting $\eta\to0$ yields $\UW\leq\HW$. The singular terms are unchanged because the same rescaling is performed before taking the Lebesgue decomposition of the marginals.
\index{Lebesgue decomposition}
\index{reverse formulation}
\end{proof}

\paragraph{Conic lifting.}
\index{conic lifting}
\index{cone!lifting}

Assume now that $\Xx=\Yy$ and $\psi_1=\psi_2=\psi$. The last formulation lifts the problem to the cone space $\mathfrak C[\Xx]\eqdef(\Xx\times\RR_+)/\sim$, where all points $(x,0)$ are identified at the apex. For an exponent $p\geq1$, define
\[
	\De((x,r),(y,s))\eqdef H_{c(x,y)}(r^p,s^p)^{1/p}.
\]
Several classical unbalanced geometries are obtained by choosing $\psi$, $c$ and $p$ so that $\De$ is a distance on the cone:
\begin{itemize}
	\item $\Dd_\psi=\KL$, $p=2$, and $c(x,y)=-\log\cos^2(d(x,y)\wedge\pi/2)$ give the Hellinger--Kantorovich or Wasserstein--Fisher--Rao cone metric
\index{cone!metric}
	\[
		\De((x,r),(y,s))^2=r^2+s^2-2rs\cos(d(x,y)\wedge\pi/2).
	\]
	\item $\Dd_\psi=\KL$, $p=2$, and $c(x,y)=d(x,y)^2$ give the Gaussian Hellinger cone metric
\index{cone!metric}
	\[
		\De((x,r),(y,s))^2=r^2+s^2-2rs e^{-d(x,y)^2/2}.
	\]
	\item $\Dd_\psi=\TV$, $p=1$, and $c(x,y)=d(x,y)$ give the partial-transport cone cost
	\[
		\De((x,r),(y,s))=r+s-(r\wedge s)(2-d(x,y))_+.
	\]
\end{itemize}
The corresponding cone value is
\[
	\CW(\al,\be)
	=
	\inf_{\gamma\in\Mm_+(\mathfrak C[\Xx]^2)}
	\left\{
	\int \De((x,r),(y,s))^p\d\gamma
	\; ; \;
	\int r^p\d\gamma_1(\cdot,r)=\al,
	\int s^p\d\gamma_2(\cdot,s)=\be
	\right\}.
\]

\begin{thm}[Cone formulation of unbalanced OT]\label{thm-cone-unbalanced-ot}
\index{cone!formulation}
\index{unbalanced!OT}
	One has $\UW=\HW=\CW$. If $\De$ is a distance, then $\CW^{1/p}$ is a distance between nonnegative measures.
\index{nonnegative measure}
\end{thm}

\begin{proof}
	The equality $\UW=\HW$ is Proposition~\ref{prop-homogeneous-unbalanced}. To prove $\HW=\CW$, disintegrate an admissible cone coupling $\gamma$ with respect to its spatial variables $(x,y)$ and radii $(r,s)$. The cone marginal constraints say precisely that the spatial marginals are recovered after weighting by $r^p$ and $s^p$. Since $\De((x,r),(y,s))^p=H_{c(x,y)}(r^p,s^p)$, integrating the cone cost gives the homogeneous objective. Conversely, any homogeneous competitor can be lifted to the cone by placing, over each $(x,y)$, radii whose $p$th powers are the two density factors appearing in $H_c$.
\index{disintegration}
\index{marginal!constraint}

	If $\De$ is a distance on the cone, then $\CW^{1/p}$ is the usual $p$-Wasserstein distance between lifted measures under the linear cone-marginal constraints. Symmetry and the triangle inequality follow from the corresponding Wasserstein properties and the gluing lemma on the cone. If the distance is zero, an optimal cone coupling is concentrated on the diagonal of the cone, so the weighted projections agree and therefore $\alpha=\beta$.
\index{triangle inequality}
\index{Wasserstein!distance}
\index{gluing lemma}
\end{proof}

\paragraph{Entropic KL relaxation.}
\index{KL!relaxation}

A generic entropic regularization of unbalanced OT reads
\index{entropic!regularization}
\index{unbalanced!OT}
\[
	\inf_{\pi\in\Mm_+(\Xx\times\Yy)}
	\int c\d\pi
	+
	\Dd_{\psi_1}(\pi_1|\al)
	+
	\Dd_{\psi_2}(\pi_2|\be)
	+
	\epsilon\Dd_\phi(\pi|\al\otimes\be).
\]
Its dual is
\[
	\sup_{f,g}
	-
	\Dd_{\psi_1}^*(-f|\al)
	-
	\Dd_{\psi_2}^*(-g|\be)
	-
	\epsilon\Dd_\phi^*\left(\frac{f\oplus g-c}{\epsilon}\middle|\al\otimes\be\right).
\]
For $\Dd_\phi=\KL$, the primal-dual relation is $\d\pi=e^{(f\oplus g-c)/\epsilon}\d\al\d\be$. If in addition $\Dd_{\psi_1}=\Dd_{\psi_2}=\tau\KL$, the dual specializes to
\[
	\sup_{f,g}
	-\tau\int(e^{-f/\tau}-1)\d\al
	-\tau\int(e^{-g/\tau}-1)\d\be
	-\epsilon\iint
	\left(e^{(f\oplus g-c)/\epsilon}-1\right)
	\d\al\d\be,
\]
and coordinate maximization gives the damped soft transforms
\index{soft!transform}
\begin{align*}
	f(x)
	&=
	-
	\frac{\tau\epsilon}{\tau+\epsilon}
	\log\int_{\Yy}
	\exp\left(\frac{g(y)-c(x,y)}{\epsilon}\right)\d\be(y),\\
	g(y)
	&=
	-
	\frac{\tau\epsilon}{\tau+\epsilon}
	\log\int_{\Xx}
	\exp\left(\frac{f(x)-c(x,y)}{\epsilon}\right)\d\al(x).
\end{align*}
In the discrete case, with $K_{i,j}=e^{-\C_{i,j}/\epsilon}\a_i\b_j$ and $\rho=\tau/(\tau+\epsilon)$, this gives the generalized Sinkhorn scaling
\index{Sinkhorn!scaling}
\[
	u_i\leftarrow \left(\frac{\a_i}{(K v)_i}\right)^\rho,
	\qquad
	v_j\leftarrow \left(\frac{\b_j}{(\transp{K}u)_j}\right)^\rho,
	\qquad
	\P=\diag(u)K\diag(v).
\]
The exponent $\rho<1$ is the visible difference with balanced Sinkhorn: marginal corrections are damped because violating the marginals is allowed.

\begin{alg}[Unbalanced Sinkhorn scaling]\label{alg:unbalanced-sinkhorn}
\textbf{Input:} Weights $\a,\b$, cost matrix $\C$, entropic scale $\epsilon>0$, KL strength $\tau>0$, tolerance $\mathrm{tol}$.

\textbf{Output:} Unbalanced entropic coupling $\P$.

\textbf{Initialize:} Set
\(K_{ij}=e^{-\C_{ij}/\epsilon}\a_i\b_j, \quad \rho=\frac{\tau}{\tau+\epsilon}, \quad u^{(0)}=\ones_n, \quad v^{(0)}=\ones_m, \quad \eta_0=+\infty, \quad k=0.\)

\textbf{While} \(\eta_k>\mathrm{tol}\) \textbf{do}:
\begin{algblock}

\textbf{Set} \(k\leftarrow k+1\).

\(u^{(k)} = \left(\frac{\a}{K v^{(k-1)}}\right)^\rho, \qquad v^{(k)} = \left(\frac{\b}{\transp{K}u^{(k)}}\right)^\rho.\)

\textbf{Set} \(\eta_k=\max\{\norm{u^{(k)}-u^{(k-1)}}_\infty,\norm{v^{(k)}-v^{(k-1)}}_\infty\}\).
\end{algblock}
\algreturnskip
\textbf{Return} \(\P^{(k)}=\diag(u^{(k)})K\diag(v^{(k)})\).
\end{alg}

\begin{figure}[H]
\centering
\setlength{\tabcolsep}{3pt}
\begin{tabular}{@{}ccc@{}}
\small small $\tau$ & \small medium $\tau$ & \small large $\tau$ \\[-.15em]
\includegraphics[width=.30\linewidth]{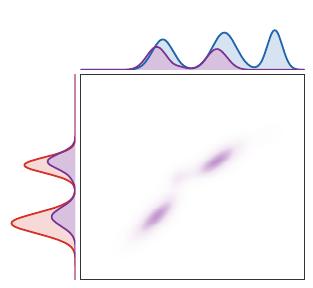} &
\includegraphics[width=.30\linewidth]{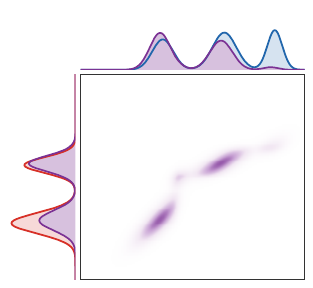} &
\includegraphics[width=.30\linewidth]{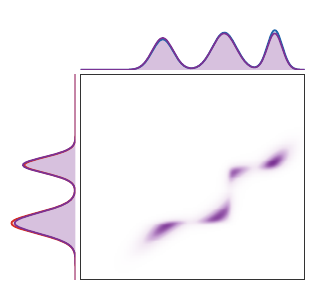}
\end{tabular}
\caption{KL unbalanced OT on one-dimensional Gaussian-mixture densities. The central matrix is the transported coupling. On the left, the red curve is the prescribed source marginal and the violet curve is the transported source marginal; the red gap is destroyed mass. On the top, the blue curve is the prescribed target marginal and the violet curve is the transported target marginal; the blue gap is created mass. Increasing $\tau$ makes marginal mismatch more expensive, so more mass is transported, including toward the far right target mode.}
\index{created mass}
\index{destroyed mass}
\index{Gaussian mixture}
\index{unbalanced!OT}
\label{fig:unbalanced-mass-relaxation}
\end{figure}

The entropy used in the marginal relaxation also changes the qualitative behavior. A KL penalty leads to smooth multiplicative rescaling. The reverse-KL, or Burg, penalty blows up when a transported marginal vanishes where the prescribed marginal is positive, so it discourages complete deletion of small modes. Total variation has a linear kink and behaves closer to partial transport: mass is either kept active or created and destroyed at nearly constant marginal price.
\index{partial transport}
\index{marginal!relaxation}
\index{total variation}

\begin{figure}[H]
\centering
\setlength{\tabcolsep}{3pt}
\begin{tabular}{@{}ccc@{}}
\small KL & \small Burg & \small total variation \\[-.15em]
\index{total variation}
\includegraphics[width=.30\linewidth]{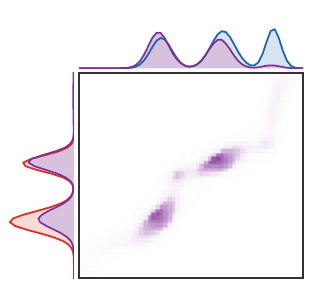} &
\includegraphics[width=.30\linewidth]{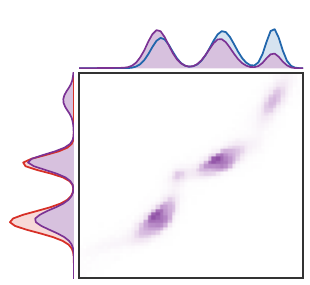} &
\includegraphics[width=.30\linewidth]{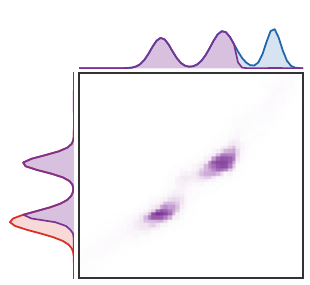}
\end{tabular}
\caption{Effect of the marginal divergence in unbalanced entropic OT. The geometric cost, entropic plan regularization $\epsilon$, and relaxation strength $\tau$ are fixed; only the marginal penalty changes. KL allows smooth mass variation, Burg keeps transported marginals from vanishing on prescribed modes, and total variation gives a sharper active-mass selection. The side plots use the same convention as Figure~\ref{fig:unbalanced-mass-relaxation}.}
\index{mass!variation}
\index{marginal!divergence}
\index{marginal!penalty}
\index{total variation}
\label{fig:unbalanced-divergence-choice}
\end{figure}

\section{Sliced Wasserstein Distances}
\index{Wasserstein!distance}
\index{sliced Wasserstein!distance}
\label{sec-sliced-wasserstein}

Sliced Wasserstein trades exact high-dimensional geometry for many one-dimensional projections. It is cheap, differentiable after sorting, and often effective in imaging and learning. For measures on $\RR^d$ and $\theta\in\Sphere^{d-1}$, let $P_\theta(x)=\dotp{\theta}{x}$ be the projection on direction $\theta$.
\index{one-dimensional!projection}

\begin{defn}[Sliced Wasserstein distance]\label{def-sliced-wasserstein}
\index{sliced Wasserstein!distance}
	Let $\sigma$ be the uniform probability measure on the sphere $\Sphere^{d-1}$. The sliced $p$-Wasserstein distance is
	\eql{\label{eq-sliced-wasserstein}
		\SW_p(\al,\be)^p
		\eqdef
		\int_{\Sphere^{d-1}}
		\Wass_p\big((P_\theta)_\sharp\al,(P_\theta)_\sharp\be\big)^p
		\d\sigma(\theta).
	}
\end{defn}
This construction is closely related to the Radon transform and is much cheaper to approximate numerically than high-dimensional OT, since each projected problem can be solved by sorting or quantiles~\cite{rabin-ssvm-11,2013-Bonneel-barycenter,kolouri2016sliced}. It metrizes the same weak-plus-moment topology as $\Wass_p$, but its geometry is not bi-Lipschitz equivalent to $\Wass_p$ in high dimension~\cite{nadjahi2019asymptotic}.
\index{probability measure}
\index{Radon!transform}

\begin{alg}[Monte Carlo sliced Wasserstein]\label{alg:monte-carlo-sliced-wasserstein}
\index{sliced Wasserstein!distance}
\textbf{Input:} Equal-weight point clouds $(x_i)_{i=1}^n$, $(y_i)_{i=1}^n$, exponent $p$, number of directions $L$.

\textbf{Output:} Monte Carlo estimate $\widehat{\SW}_p^p(\alpha,\beta)$.

\textbf{For} $\ell=1,\ldots,L$ \textbf{do}:
\begin{algblock}

\textbf{Sample} $\theta_\ell\sim\sigma$ on $\Sphere^{d-1}$.

\textbf{Set} $s_i^\ell=\dotp{\theta_\ell}{x_i}$ and $t_i^\ell=\dotp{\theta_\ell}{y_i}$.

\textbf{Let} $\sigma_\ell,\tau_\ell$ be stable sorting permutations:
\(s_{\sigma_\ell(1)}^\ell\leq\cdots\leq s_{\sigma_\ell(n)}^\ell, \quad t_{\tau_\ell(1)}^\ell\leq\cdots\leq t_{\tau_\ell(n)}^\ell.\)

\textbf{Compute}
\(E_\ell=\frac1n\sum_{i=1}^n\abs{s_{\sigma_\ell(i)}^\ell-t_{\tau_\ell(i)}^\ell}^p.\)
\end{algblock}
\algreturnskip
\textbf{Return} \(\widehat{\SW}_p^p(\alpha,\beta)=\frac1L\sum_{\ell=1}^L E_\ell.\)
\end{alg}

\begin{figure}[H]
\centering
\setlength{\tabcolsep}{3pt}
\begin{tabular}{@{}cccc@{}}
\includegraphics[width=.23\linewidth]{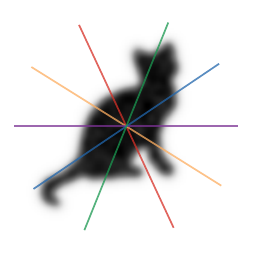} &
\includegraphics[width=.17\linewidth]{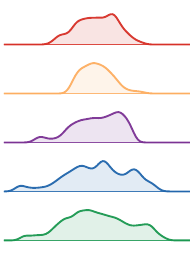} &
\includegraphics[width=.17\linewidth]{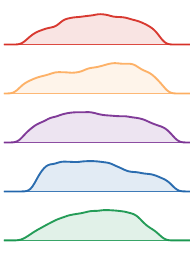} &
\includegraphics[width=.23\linewidth]{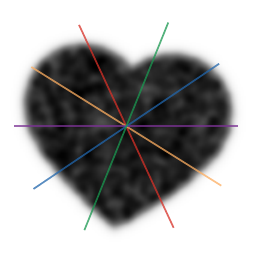}
\\[-.1em]
\small source density & \small source slices & \small target slices & \small target density
\end{tabular}
\caption{Sliced Wasserstein projections between two planar densities. The source and target are rendered as smoothed black-and-white density images obtained from dense farthest-point samples of two silhouettes. Five fixed directions are drawn on both densities. For each direction, the middle panels show smoothed one-dimensional density estimates of the projected measures $(P_\theta)_\sharp\alpha$ and $(P_\theta)_\sharp\beta$. Sliced OT averages one-dimensional Wasserstein discrepancies over many such directions, replacing a difficult planar comparison by a collection of sorted one-dimensional comparisons.}
\label{fig:sliced-wasserstein-projections}
\end{figure}

\begin{prop}[Metric properties of sliced Wasserstein]\label{prop-sliced-wasserstein-metric}
\index{sliced Wasserstein!distance}
	For $p\geq1$, $\SW_p$ is a distance on $\Pp_p(\RR^d)$. Moreover, $\SW_p$ metrizes weak convergence together with convergence of the $p$th moment. Finally,
\index{weak!convergence}
	\[
		\SW_p(\alpha,\beta)\leq \Wass_p(\alpha,\beta),
	\]
	and, for $p=2$ with the uniform probability measure on the sphere,
\index{probability measure}
	\[
		\SW_2(\alpha,\beta)^2\leq \frac1d\Wass_2(\alpha,\beta)^2.
	\]
\end{prop}

\begin{proof}
	Non-negativity and symmetry follow from the one-dimensional Wasserstein distance. For the triangle inequality, apply the triangle inequality of $\Wass_p$ for each direction $\theta$ and then Minkowski's inequality in $L^p(\Sphere^{d-1})$.
\index{Minkowski inequality}
\index{triangle inequality}
\index{Wasserstein!distance}

	If $\SW_p(\alpha,\beta)=0$, then $(P_\theta)_\sharp\alpha=(P_\theta)_\sharp\beta$ for almost every direction. By continuity of characteristic functions this holds for all directions, and the Cram\'er--Wold theorem implies $\alpha=\beta$. This proves separation.

	The bound $\SW_p\leq\Wass_p$ follows because $P_\theta$ is $1$-Lipschitz, so
	\[
		\Wass_p((P_\theta)_\sharp\alpha,(P_\theta)_\sharp\beta)
		\leq
		\Wass_p(\alpha,\beta)
	\]
	for every $\theta$. For $p=2$, using any coupling $\pi$ between $\alpha$ and $\beta$,
	\[
		\int_{\Sphere^{d-1}}\int |\dotp{x-y}{\theta}|^2\d\pi(x,y)\d\sigma(\theta)
		=
		\frac1d\int\norm{x-y}^2\d\pi(x,y).
	\]
	Optimizing over $\pi$ gives the sharper inequality. The weak-convergence statement follows from the same Cram\'er--Wold mechanism plus the moment condition: convergence in $\SW_p$ gives convergence of almost all one-dimensional projections and tightness of the $p$th moments; conversely, weak convergence with $p$th-moment convergence implies convergence of projected $\Wass_p$ distances and dominated convergence on the sphere.
\index{tightness}
\index{one-dimensional!projection}
\index{weak!convergence}
\end{proof}

\begin{rem}[Hilbert embedding for $\SW_2$]\label{rem-sliced-hilbert-embedding}
\index{Hilbert!embedding}
	In one dimension, $\Wass_2$ is the $L^2(0,1)$ distance between quantile functions. Hence
\index{quantile!function}
	\[
		\SW_2(\alpha,\beta)^2
		=
		\int_{\Sphere^{d-1}}\int_0^1
		\abs{F_{\theta,\alpha}^{-1}(u)-F_{\theta,\beta}^{-1}(u)}^2
		\d u\d\sigma(\theta),
	\]
	where $F_{\theta,\alpha}^{-1}$ is the quantile of $(P_\theta)_\sharp\alpha$. Thus $\SW_2$ is a Hilbertian distance after embedding each measure into its field of projected quantiles. Consequently, $\exp(-\gamma\SW_2^2)$ is a positive definite kernel on probability measures for every $\gamma>0$.
\index{probability measure}
\index{kernel!positive definite}
	Conversely, on compact sets, $\Wass_p$ can be bounded by a dimension-dependent power of $\SW_p$; such inequalities are weaker than the direct bound of Proposition~\ref{prop-sliced-wasserstein-metric} and explain why sliced distances metrize the same topology without being bi-Lipschitz equivalent to $\Wass_p$ in high dimension~\cite{bonnotte2013unidimensional,nadjahi2019asymptotic}.
\end{rem}

\begin{defn}[Max-sliced Wasserstein]\label{def-sliced-variants}
\index{sliced Wasserstein!distance}
\index{sliced Wasserstein!max}
	The max-sliced distance replaces the average over directions by the most discriminating one:
	\[
		\MaxSW_p(\alpha,\beta)
		\eqdef
		\sup_{\theta\in\Sphere^{d-1}}
		\Wass_p((P_\theta)_\sharp\alpha,(P_\theta)_\sharp\beta).
	\]
	It is useful when only a small set of projections carries most of the discrepancy, for instance in generative modeling~\cite{deshpande2019maxsliced}.
\index{generative model}
\end{defn}

\paragraph{Subspace-sliced variants.}
\index{subspace!sliced Wasserstein}

One-dimensional slices are extremely cheap, but they may discard too much geometry in high dimension. A natural compromise is to project onto $k$-dimensional subspaces: the projected OT problems remain lower dimensional, while each projection retains correlations inside a small block of coordinates. Varying $k$ therefore interpolates between ordinary slicing and full OT.

\begin{defn}[Subspace-sliced Wasserstein]\label{def-subspace-sliced-wasserstein}
\index{sliced Wasserstein!distance}
\index{subspace!sliced Wasserstein}
	Fix $1\leq k\leq d$. Subspace-sliced variants replace one-dimensional lines by $k$-dimensional orthogonal projections. If $U\in\RR^{d\times k}$ satisfies $\transp{U}U=\Id_k$, then
\index{subspace!sliced Wasserstein}
	\[
		\SW_{p,k}(\alpha,\beta)^p
		\eqdef
		\int \Wass_p((\transp{U})_\sharp\alpha,(\transp{U})_\sharp\beta)^p\d U,
	\]
	where $\d U$ denotes the normalized invariant measure on the Stiefel manifold
\index{Stiefel manifold}
	$\mathrm{St}(d,k)=\{U\in\RR^{d\times k}\;:\;\transp{U}U=\Id_k\}$, and
	\[
		\MaxSW_{p,k}(\alpha,\beta)
		\eqdef
		\sup_{\transp{U}U=\Id_k}
		\Wass_p((\transp{U})_\sharp\alpha,(\transp{U})_\sharp\beta).
	\]
	The case $k=1$ recovers ordinary sliced and max-sliced Wasserstein, while $k=d$ recovers the original Wasserstein distance.
\index{sliced Wasserstein!max}
\index{Wasserstein!distance}
\end{defn}

\begin{prop}[Basic bounds for sliced variants]\label{prop-sliced-variant-bounds}
\index{sliced Wasserstein!distance}
	Let $p\geq1$ and let $\alpha,\beta\in\Pp_p(\RR^d)$. With normalized spherical and Stiefel measures,
	\[
		\SW_p(\alpha,\beta)
		\leq
		\MaxSW_p(\alpha,\beta)
		\leq
		\Wass_p(\alpha,\beta).
	\]
	For $k$-dimensional subspace projections,
\index{subspace!projection}
	\[
		\SW_{p,k}(\alpha,\beta)
		\leq
		\MaxSW_{p,k}(\alpha,\beta)
		\leq
		\Wass_p(\alpha,\beta).
	\]
\end{prop}

\begin{proof}
	The first inequality in each line follows because an $L^p$ average over a probability space is bounded by the corresponding supremum. The second inequality follows because orthogonal projections are $1$-Lipschitz: pushing any admissible coupling between $\alpha$ and $\beta$ through a projection gives an admissible coupling for the projected measures with no larger transport cost. Optimizing over couplings and then averaging or maximizing over the projection gives the result.
\end{proof}

\paragraph{Min-sliced lifted transport plans.}
\index{sliced Wasserstein!min transport}
\index{plan!transport}

The preceding constructions define distances between projected measures. A different use of slicing is to use a projection only as a device for building a feasible high-dimensional transport plan. For equal-weight empirical measures $\alpha=n^{-1}\sum_i\delta_{x_i}$ and $\beta=n^{-1}\sum_i\delta_{y_i}$, sort the projected samples $\dotp{x_i}{\theta}$ and $\dotp{y_j}{\theta}$, and let $\sigma_\theta$ be the monotone matching induced by this sorting. The lifted plan
\index{empirical!measure}
\index{monotone!matching}
\index{plan!lifted}
\[
	\pi_\theta
	=
	\frac1n\sum_{i=1}^n\delta_{(x_i,y_{\sigma_\theta(i)})}
\]
is a genuine coupling between $\alpha$ and $\beta$ in the original space. Min-SWGG-type methods then choose the projection whose lifted plan has the smallest full-dimensional quadratic cost,
\index{cost!quadratic}
\index{plan!lifted}
\[
	\operatorname{MSWGG}_2(\alpha,\beta)^2
	\eqdef
	\min_{\theta\in\Sphere^{d-1}}
	\int\norm{x-y}^2\d\pi_\theta(x,y).
\]
This quantity is not a projected distance; it is a cheap feasible-plan construction. Consequently it gives an upper bound on $\Wass_2^2(\alpha,\beta)$, and the interest is algorithmic: the plan is obtained by sorting rather than by solving a high-dimensional linear program.
Indeed, each $\pi_\theta$ is an admissible coupling between the original measures, so
\[
	\Wass_2^2(\alpha,\beta)
	\leq
	\int\norm{x-y}^2\d\pi_\theta(x,y),
	\qquad
	\Wass_2^2(\alpha,\beta)
	\leq
	\operatorname{MSWGG}_2(\alpha,\beta)^2.
\]

\begin{alg}[Lifted min-sliced matching]\label{alg:lifted-min-sliced-matching}
\index{sliced Wasserstein!min transport}
\textbf{Input:} Equal-weight point clouds $(x_i)_{i=1}^n$, $(y_i)_{i=1}^n$, finite direction set $\Theta\subset\Sphere^{d-1}$.

\textbf{Output:} Feasible coupling $\pi_{\theta^\star}$ induced by the selected projection direction.

\textbf{For} each $\theta\in\Theta$ \textbf{do}:
\begin{algblock}

\textbf{Let} $\sigma_\theta,\tau_\theta$ be stable sorting permutations of $\dotp{\theta}{x_i}$ and $\dotp{\theta}{y_j}$.

\textbf{Match} $x_{\sigma_\theta(k)}$ to $y_{\tau_\theta(k)}$ for $k=1,\ldots,n$.

\textbf{Store} rank-matching permutation $\rho_\theta=\tau_\theta\circ\sigma_\theta^{-1}$.

\textbf{Evaluate}
\(E(\theta)=\frac1n\sum_{i=1}^n\norm{x_i-y_{\rho_\theta(i)}}^2.\)
\end{algblock}
\textbf{Set} $\theta^\star=\min\argmin_{\theta\in\Theta}E(\theta)$ for the fixed order on $\Theta$.
\textbf{Return} \(\pi_{\theta^\star}=\frac1n\sum_i\delta_{(x_i,y_{\rho_{\theta^\star}(i)})}.\)
\end{alg}

\begin{figure}[H]
\centering
\setlength{\tabcolsep}{3pt}
\begin{tabular}{@{}ccc@{}}
\small selected projection & \small lifted sliced plan & \small quadratic $W_2$ plan \\[-.15em]
\includegraphics[width=.30\linewidth]{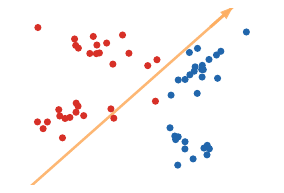} &
\index{sliced Wasserstein!min transport}
\includegraphics[width=.30\linewidth]{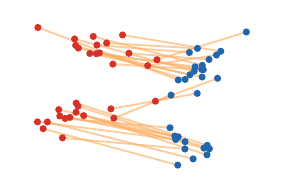} &
\index{plan!lifted}
\includegraphics[width=.30\linewidth]{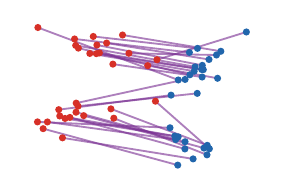}
\end{tabular}
\caption{Lifted min-sliced plan. A one-dimensional direction is selected by a small deterministic sweep, then red and blue atoms are sorted after projection and matched in that order. The middle panel lifts this one-dimensional matching back to the plane; it is a feasible coupling but not the same object as the quadratic $W_2$ matching shown on the right. This illustrates why sliced constructions are computationally light and interpretable, while losing some of the geometry of the full transport problem.}
\label{fig:min-sliced-transport-plan}
\end{figure}

\section{Vector Quantiles and Linear Optimal Transport}
\index{vector quantile}
\index{linear!OT}
\label{sec-linear-ot}
\label{sec-vector-quantiles-linearized-transport}

Linear OT starts from the multivariate analogue of quantile coordinates. The one-dimensional quantile function represents a probability measure by the monotone map sending a fixed reference law to it; in dimension $d>1$, Brenier's theorem gives the corresponding construction after choosing an absolutely continuous reference probability $\rho$, typically the uniform law on a convex body or a standard Gaussian.
\index{probability measure}
\index{Brenier!theorem}
\index{quantile!function}
\index{linear!OT}

\paragraph{Vector quantiles.}
\index{vector quantile}

Assume that $\rho$ is absolutely continuous. For a target law $\mu$ with finite second moment, its vector quantile relative to $\rho$ is the Brenier map
\index{Brenier!map}
\[
	\T_\mu=\nabla\phi_\mu,
	\qquad
	(\T_\mu)_\sharp\rho=\mu,
\]
or equivalently the solution of
\[
	\min_{\T_\sharp\rho=\mu}
	\int \norm{x-\T(x)}^2\d\rho(x).
\]
This construction is canonical only after fixing $\rho$: changing the reference law changes the coordinates used to represent $\mu$. Vector quantile regression uses the same idea conditionally, replacing scalar conditional quantiles by conditional Brenier maps and thereby encoding multivariate ranks and depths~\cite{carlier2016vector}.
\index{conditional!quantile}
\index{vector quantile}
\index{Brenier!map}

\paragraph{Linearized Wasserstein coordinates.}
\index{Wasserstein!coordinate}

Linear OT replaces a nonlinear transport distance by a Hilbert norm between reference maps. It is useful when one reference measure is fixed and many nearby distributions must be compared cheaply. Let $T_\alpha$ be the Brenier map pushing $\rho$ to $\alpha$, understood as an element of $L^2(\rho;\RR^d)$ and hence defined only $\rho$-almost everywhere. The linear OT embedding is
\index{reference!measure}
\index{Brenier!map}
\index{linear!OT}
\begin{equation}\label{eq-lot-embedding}
	\alpha \mapsto T_\alpha-\Id\in L^2(\rho;\RR^d),
	\qquad
	\LOT_\rho(\alpha,\beta)=\norm{T_\alpha-T_\beta}_{L^2(\rho)}.
\end{equation}
If one of the two targets equals the reference, the linearized distance is exact: for instance, $\LOT_\rho(\rho,\alpha)=\norm{T_\alpha-\Id}_{L^2(\rho)}=\Wass_2(\rho,\alpha)$. For two arbitrary targets, the coupling $(T_\alpha,T_\beta)_\sharp\rho$ is admissible but not generally optimal, so $\LOT_\rho$ is a tangent-space approximation of the Wasserstein geometry~\cite{wang2013linear}.
For a family $(\alpha_s)_s$ with weights $(\lambda_s)_s$, the linearized barycenter is obtained by averaging maps,
\[
	\bar T=\sum_s\lambda_s T_{\alpha_s},
	\qquad
	\bar\alpha_{\LOT}=\bar T_\sharp\rho.
\]
This is exact in one dimension, where quantile functions linearize $\Wass_2$, and it is especially useful when many barycenters with changing weights must be evaluated quickly.
\index{quantile!function}

\begin{rem}[Three Hilbertian embeddings of measures]\label{rem-three-hilbertian-measure-embeddings}
\index{Hilbertian!embedding}
	Several constructions in this text embed measures into Hilbert spaces, but they encode different geometries. Kernel mean embeddings send $\alpha$ to $\int k(x,\cdot)\d\alpha(x)$ in an RKHS and lead to MMD distances; see Section~\ref{sec-rkhs-mmd}. Quadratic sliced Wasserstein sends a measure to the collection of one-dimensional quantile functions of its projections, viewed in $L^2(\Sphere^{d-1}\times[0,1])$; see Section~\ref{sec-sliced-wasserstein}. Linear OT sends $\alpha$ to the displacement field $T_\alpha-\Id$ from a fixed reference $\rho$ in $L^2(\rho;\RR^d)$. The first construction is linear in the measure and depends on the kernel, the second is nonlinear but reduces OT to projected one-dimensional quantiles, and the third is a tangent approximation to the full Wasserstein geometry around a chosen reference.
\index{sliced Wasserstein!distance}
\index{RKHS}
\index{maximum mean discrepancy}
\index{kernel!mean embedding}
\index{linear!OT}
\end{rem}

\begin{figure}[htbp]
\centering
\setlength{\tabcolsep}{2pt}
\begin{tabular}{@{}ccc@{}}
\includegraphics[width=.45\linewidth]{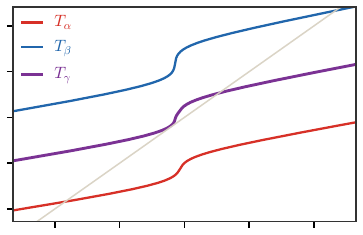} &
\index{linear!OT}
\includegraphics[width=.45\linewidth]{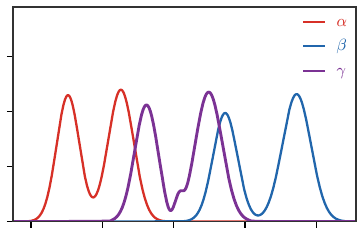} \\[-.1em]
\small 1D maps $T_\alpha,T_\beta,T_\gamma$ & \small measures $\alpha,\beta,\gamma$
\end{tabular}

\vspace{.25em}
\begin{tabular}{@{}ccc@{}}
\includegraphics[width=.30\linewidth]{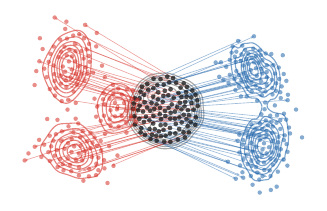} &
\index{linear!OT}
\includegraphics[width=.30\linewidth]{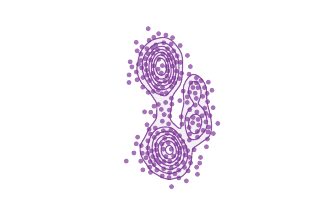} &
\includegraphics[width=.30\linewidth]{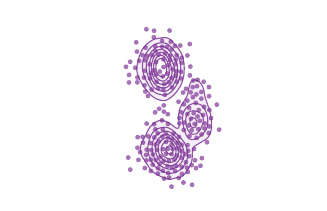} \\[-.1em]
\small maps from $\rho$ to $\alpha,\beta$ &
\small linearized barycenter $\bar T_\sharp\rho$ &
\small McCann midpoint
\end{tabular}
\caption{Linear OT coordinates. Fixing a reference measure $\rho$ turns each target into a map $T_\alpha$ from $\rho$ to $\alpha$, or equivalently into the displacement field $T_\alpha-\Id$. In one dimension this is exactly the quantile parametrization of $\Wass_2$, so averaging the maps toward a two-component $\alpha$ and a two-component $\beta$ gives the true Wasserstein barycenter. In two dimensions, the first panel shows the reference-to-target maps, computed on dense clouds and displayed on farthest-point subsets, with $\beta$ represented by two Gaussian components. The middle purple panel shows the linearized barycenter obtained by averaging the two maps from $\rho$. The right purple panel shows the genuine McCann midpoint between $\alpha$ and $\beta$, obtained by solving a direct OT problem between the two target clouds and interpolating at $t=1/2$.}
\index{Wasserstein!barycenter}
\index{reference!measure}
\index{linear!OT}
\label{fig:dualnorms-linear-ot-embedding}
\end{figure}

\begin{prop}[Local stability of linear OT]\label{prop-linear-ot-stability}
\index{linear!OT}
	Assume that the measures are supported on a fixed convex compact set, with densities bounded above and below, and that the Brenier maps from $\rho$ are regular. Then, for $\alpha,\beta$ in a sufficiently small regular neighborhood of $\rho$,
\index{Brenier!map}
	\[
		\Wass_2(\alpha,\beta)\leq \LOT_\rho(\alpha,\beta)
		\quad\text{and}\quad
		\LOT_\rho(\alpha,\beta)\leq C\Wass_2(\alpha,\beta)^\eta
	\]
	for constants $C>0$ and $\eta\in(0,1]$ depending on regularity.
\end{prop}
\begin{proof}
	The first inequality is immediate: $(T_\alpha,T_\beta)_\sharp\rho$ is a feasible coupling between $\alpha$ and $\beta$. The reverse local estimate is a standard stability statement for the Monge--Amp\`ere equation under the stated regularity assumptions: changes in the target measure control changes in the Brenier potential in H\"older norms, hence control $T_\alpha-T_\beta$ in $L^2(\rho)$. In simple one-dimensional settings, quantile functions make this exact with $\eta=1$.
\index{feasible coupling}
\index{Monge-Ampere equation}
\index{quantile!function}
	In several dimensions one should not read the statement as a global Lipschitz estimate in $\Wass_2$. Quantitative stability results for semi-discrete and Monge--Amp\`ere maps give H\"older exponents depending on the dimension, density bounds, support geometry and regularity; see for instance the estimates of M\'erigot, Delalande and Chazal~\cite{merigot2020stability}. Under stronger smooth perturbations of uniformly convex smooth densities, elliptic regularity can give Lipschitz dependence in stronger function norms, but converting those controls to Wasserstein perturbations generally loses powers.
\index{semi-discrete!OT}
\index{elliptic regularity}
\end{proof}

\section{Spectral and Robust Wasserstein Distances}
\index{Wasserstein!distance}
\index{spectral!Wasserstein}
\index{robust!Wasserstein}
\index{robust!Wasserstein distance}
\label{sec-spectral-subspace-wasserstein}

Spectral OT changes the scalar quadratic cost by measuring the whole displacement covariance through a matrix gauge. The same object admits a robust projected formulation: instead of fixing one projection, one maximizes over the polar set of the gauge. Subspace robust OT is the important non-convex rank-constrained version of this idea~\cite{paty2019subspace}; spectral gauges provide its convex minimax counterpart and connect to recent spectral-gradient viewpoints such as Muon dynamics~\cite{peyre2026muon}.
\index{rank constraint}
\index{minimax}
\index{polar set}
\index{projected formulation}
\index{Muon algorithm}
\index{displacement!covariance}
\index{cost!quadratic}
\index{spectral!gauge}

\begin{defn}[Monotone spectral gauge]\label{def-monotone-spectral-gauge}
\index{spectral!gauge}
\index{monotone!spectral gauge}
	A monotone spectral gauge on positive semidefinite matrices is a convex, positively $1$-homogeneous map $\gamma:\mathbb S_+^d\to\RR_+$ such that $\gamma(M)=0$ only for $M=0$, $\gamma(QM\transp{Q})=\gamma(M)$ for every orthogonal matrix $Q$, and
\index{positive!semidefinite matrix}
	\[
		0\preceq M\preceq N
		\quad\Longrightarrow\quad
		\gamma(M)\leq\gamma(N).
	\]
\end{defn}
The monotonicity condition means that increasing the displacement covariance in Loewner order cannot decrease the transport penalty.

\begin{defn}[Spectral Wasserstein distance]\label{def-spectral-wasserstein}
\index{Wasserstein!distance}
\index{spectral!Wasserstein}
\index{spectral!Wasserstein distance}
	Let $\gamma$ be a monotone spectral gauge. For a coupling $\pi\in\Couplings(\alpha,\beta)$, define its displacement covariance
\index{displacement!covariance}
	\[
		M_\pi\eqdef\int (x-y)(x-y)^\top\d\pi(x,y).
	\]
	The spectral Wasserstein distance associated with $\gamma$ is
\index{spectral!Wasserstein distance}
	\eql{\label{eq-spectral-wasserstein}
		\Wass_\gamma(\alpha,\beta)^2
		\eqdef
		\inf_{\pi\in\Couplings(\alpha,\beta)}\gamma(M_\pi).
	}
\end{defn}

The special case $\gamma(M)=\tr(M)$ gives the usual quadratic Wasserstein distance $\Wass_2$. The spectral gauge $\gamma(M)=\lambda_{\max}(M)$ instead measures the worst transported variance direction. For $A\succeq0$, define the quadratic projected transport cost
\index{Wasserstein!distance}
\index{spectral!gauge}
\eql{\label{eq-quadratic-projected-cost}
	\Wass_{2,A}(\alpha,\beta)^2
	\eqdef
	\inf_{\pi\in\Couplings(\alpha,\beta)}
	\int (x-y)^\top A(x-y)\d\pi(x,y)
	=
	\Wass_2((A^{1/2})_\sharp\alpha,(A^{1/2})_\sharp\beta)^2.
}
The polar set of the gauge is
\index{polar set}
\eql{\label{eq-spectral-polar-set}
	\mathcal B_\gamma
	\eqdef
	\enscond{A\succeq0}{\tr(AM)\leq\gamma(M)\quad\text{for all }M\succeq0},
}
so that, for a closed gauge, $\gamma(M)=\sup_{A\in\mathcal B_\gamma}\tr(AM)$.

\begin{prop}[Robust representation and metric equivalence]\label{prop-spectral-wasserstein-robust}
\index{robust!representation}
\index{metric!equivalence}
	Assume, for simplicity, that the measures are compactly supported and that $\gamma$ is closed and finite on the positive semidefinite cone. Then
	\[
		\Wass_\gamma(\alpha,\beta)^2
		=
		\sup_{A\in\mathcal B_\gamma}
		\Wass_{2,A}(\alpha,\beta)^2.
	\]
	If there exist constants $0<a\leq b<+\infty$ such that $a\Id\in\mathcal B_\gamma$ and $\mathcal B_\gamma\subset\enscond{A}{0\preceq A\preceq b\Id}$, equivalently
	\[
		a\tr(M)\leq\gamma(M)\leq b\tr(M)
		\qquad (M\succeq0),
	\]
	then
	\[
		\sqrt a\,\Wass_2(\alpha,\beta)
		\leq
		\Wass_\gamma(\alpha,\beta)
		\leq
		\sqrt b\,\Wass_2(\alpha,\beta).
	\]
	In particular, $\Wass_\gamma$ is a distance. When $\gamma$ is the restriction of a norm to the positive semidefinite cone, these bounds hold automatically in finite dimension, so $\Wass_\gamma$ is equivalent to $\Wass_2$ on measures with finite second moments.
\end{prop}

\begin{proof}
	Using the polar representation of $\gamma$,
	\[
		\Wass_\gamma(\alpha,\beta)^2
		=
		\inf_{\pi\in\Couplings(\alpha,\beta)}
		\sup_{A\in\mathcal B_\gamma}\tr(AM_\pi).
	\]
	The coupling set is convex and compact for weak convergence under compact support. Since $\gamma$ is a finite gauge on a finite-dimensional cone and vanishes only at the origin, it is equivalent to the trace norm on the slice $\tr(M)=1$, so $\mathcal B_\gamma$ is convex and compact. The map $(\pi,A)\mapsto\tr(AM_\pi)$ is affine in each variable and continuous. Sion's minimax theorem gives
\index{minimax}
\index{trace norm}
\index{weak!convergence}
	\[
		\inf_\pi\sup_{A\in\mathcal B_\gamma}\tr(AM_\pi)
		=
		\sup_{A\in\mathcal B_\gamma}\inf_\pi\tr(AM_\pi)
		=
		\sup_{A\in\mathcal B_\gamma}\Wass_{2,A}(\alpha,\beta)^2.
	\]
	For fixed $A\succeq0$, $\Wass_{2,A}$ is the Wasserstein pseudodistance associated with the seminorm $x\mapsto\norm{A^{1/2}x}$. Since all terms are nonnegative, the robust identity also gives
	\[
		\Wass_\gamma(\alpha,\beta)
		=
		\sup_{A\in\mathcal B_\gamma}\Wass_{2,A}(\alpha,\beta).
	\]
	A supremum of pseudodistances is symmetric and satisfies the triangle inequality.
\index{triangle inequality}

	If $a\Id\in\mathcal B_\gamma$ and $A\preceq b\Id$ for all $A\in\mathcal B_\gamma$, then
	\[
		a\Wass_2(\alpha,\beta)^2
		=
		\Wass_{2,a\Id}(\alpha,\beta)^2
		\leq
		\Wass_\gamma(\alpha,\beta)^2
		\leq
		b\Wass_2(\alpha,\beta)^2,
	\]
	which proves definiteness and equivalence with $\Wass_2$. The equivalence between these operator bounds and $a\tr(M)\leq\gamma(M)\leq b\tr(M)$ follows directly from the polar formula. In finite dimension, any norm restricted to the positive semidefinite cone is equivalent to the trace norm on that cone. The finite-second-moment case follows by truncation when these norm-equivalence bounds hold.
\index{polar formula}
\index{operator bound}
\index{trace norm}
\end{proof}

\begin{defn}[Subspace robust Wasserstein]\label{def-subspace-robust-wasserstein}
\index{robust!Wasserstein}
\index{subspace!robust Wasserstein}
	For $1\leq k\leq d$, the Paty--Cuturi subspace robust Wasserstein distance is
\index{Wasserstein!distance}
\index{robust!Wasserstein distance}
	\[
		\SRW_{2,k}(\alpha,\beta)
		\eqdef
		\sup_{\transp{U}U=\Id_k}
		\Wass_2((\transp{U})_\sharp\alpha,(\transp{U})_\sharp\beta)
		=
		\sup_{P^2=P=\transp{P},\ \tr(P)=k}\Wass_{2,P}(\alpha,\beta).
	\]
\end{defn}

For the Ky Fan gauge
\[
	\gamma_k(M)=\sum_{\ell=1}^k\lambda_\ell(M),
\]
where the eigenvalues are sorted in decreasing order, the polar set is
\index{polar set}
\[
	\mathcal B_{\gamma_k}=\enscond{A}{0\preceq A\preceq\Id,\ \tr(A)\leq k}.
\]
Thus $k=d$ gives $\gamma_d(M)=\tr(M)$ and recovers $\Wass_2$. The convex hull of rank-$k$ projectors is
\[
	\enscond{A}{0\preceq A\preceq\Id,\ \tr(A)=k},
\]
and, since $M\succeq0$, the associated support function is the same Ky Fan gauge. Thus $\Wass_{\gamma_k}$ is the convexified spectral counterpart of $\SRW_{2,k}$, while $\SRW_{2,k}$ keeps the original non-convex rank constraint. For $k=1$, $\gamma_1(M)=\lambda_{\max}(M)$ and $\mathcal B_{\gamma_1}=\enscond{A\succeq0}{\tr(A)\leq1}$.
This top-eigenvalue spectral Wasserstein geometry is the case connected to Muon-type spectral dynamics in~\cite{peyre2026muon}.
\index{spectral dynamics}
\index{Muon algorithm}
\index{spectral!Wasserstein}

\begin{figure}[H]
\centering
\setlength{\tabcolsep}{3pt}
\begin{tabular}{@{}cc@{}}
\includegraphics[width=.34\linewidth]{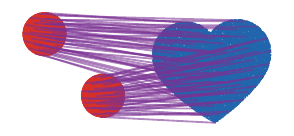} &
\includegraphics[width=.34\linewidth]{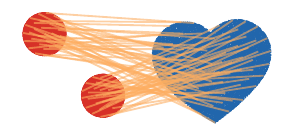}
\\[-.1em]
\small trace-gauge coupling & \small $\lambda_{\max}$-gauge coupling
\index{trace!gauge}
\end{tabular}
\\[.35em]
\begin{tabular}{@{}ccccc@{}}
\includegraphics[width=.17\linewidth]{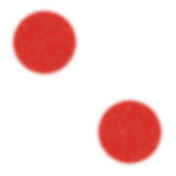} &
\includegraphics[width=.17\linewidth]{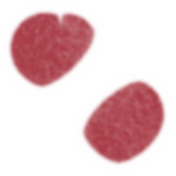} &
\includegraphics[width=.17\linewidth]{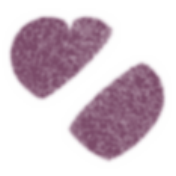} &
\includegraphics[width=.17\linewidth]{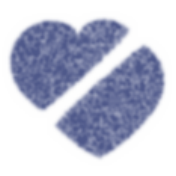} &
\includegraphics[width=.17\linewidth]{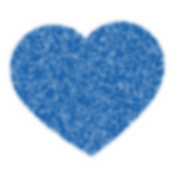} \\[-.1em]
\multicolumn{5}{c}{\small trace-gauge interpolation} \\[.2em]
\index{trace!gauge}
\includegraphics[width=.17\linewidth]{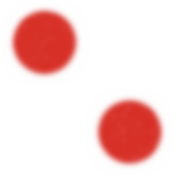} &
\includegraphics[width=.17\linewidth]{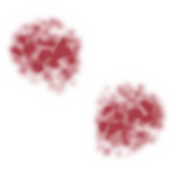} &
\includegraphics[width=.17\linewidth]{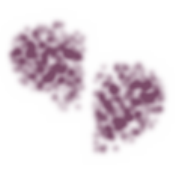} &
\includegraphics[width=.17\linewidth]{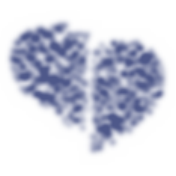} &
\includegraphics[width=.17\linewidth]{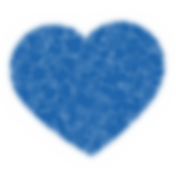} \\[-.1em]
\multicolumn{5}{c}{\small $\lambda_{\max}$-gauge interpolation} \\[-.1em]
\small $t=0$ & \small $t=1/4$ & \small $t=1/2$ & \small $t=3/4$ & \small $t=1$
\end{tabular}
\caption{Trace and spectral gauges for displacement covariances. The trace gauge minimizes the average squared displacement and gives the usual quadratic transport plan from a shifted two-disk source silhouette to a shifted heart-shaped target. The $\lambda_{\max}$ gauge penalizes the worst projected displacement variance; the displayed plan is obtained by approximating the robust formulation with finitely many directions. The last two rows render the corresponding displacement interpolations from very dense farthest-point silhouette samples and kernel-smoothed lifted plans, with the same convention as Figure~\ref{fig:monge-shape-mccann-interpolation}: white means zero density, and high density saturates in the red-to-blue interpolation color of the corresponding time. The spatial separation makes the different displacement geometries easier to read.}
\index{displacement!interpolation}
\index{displacement!covariance}
\index{McCann interpolation}
\index{spectral!gauge}
\index{plan!transport}
\index{plan!lifted}
\index{trace!gauge}
\label{fig:spectral-wasserstein-gauge}
\end{figure}



\chapter{Generalized OT Problems}
\label{sec-generalized-ot-problems}

The second family changes the optimization problem rather than only the ground distance. Barycenters average several measures, multi-marginal OT couples many measures at once, inverse OT learns the cost from observed transport, and weak OT allows randomized conditional responses. These formulations remain close to Kantorovich linear programming, but the object being optimized is richer than a single two-marginal coupling.
\index{linear programming}
\index{multi-marginal}

\section{OT Barycenters}
\index{Wasserstein!barycenter}
\label{sec-barycenters}

Barycenters ask how to average probability measures rather than points. This section explains the variational definition, the special closed forms in one dimension and for Gaussians, and the entropic algorithms used in practice.
\index{probability measure}

\paragraph{Fr\'echet means.}
\index{Frechet mean}

For discrete input histograms $\{\b_s\}_{s=1}^S$, with $\b_s \in \simplex_{n_s}$, and weights $\la \in \simplex_S$, a Wasserstein barycenter can be computed by minimizing
\index{histogram}
\index{Wasserstein!barycenter}
\eql{\label{eq-wass-discr}
	\umin{\a \in \simplex_n} \sum_{s=1}^S \la_s \MKD_{\C_s}(\a,\b_s),
}
where the cost matrices $\C_s \in \RR^{n \times n_s}$ are prescribed.
\index{cost matrix}

This barycenter problem was originally introduced by~\cite{Carlier_wasserstein_barycenter} following earlier ideas of~\cite{carlierekelandmatching}. For the quadratic cost on $\X=\RR^d$, their theory gives existence and uniqueness when at least one input measure is absolutely continuous, and more generally under hypotheses ensuring that the relevant optimal maps are well defined. Discrete existence, consistency and fixed-point constructions are further studied in~\cite{anderes2016discrete,alvarez2016fixed,leGouic2016existence}.
\index{Wasserstein!barycenter}
\index{cost!quadratic}

Given a set of input measures $(\be_s)_s$ defined on some space $\X$, the barycenter problem becomes
\eql{\label{eq-barycenter-generic}
	\umin{\al \in \Mm_+^1(\X)} \sum_{s=1}^S \la_s \MK_{\c}(\al,\be_s).
}
For $\X=\RR^d$ and $c(x,y)=\norm{x-y}^2$, if one input measure has a density, then the barycenter is unique~\cite{Carlier_wasserstein_barycenter}.
\index{Wasserstein!barycenter}

\begin{example}[Two measures recover a Wasserstein geodesic]
\index{Wasserstein!geodesic}
	For $S=2$, $c(x,y)=\norm{x-y}^2$ and weights $(1-t,t)$, the barycenter is the point at time $t$ on the Wasserstein geodesic between $\beta_0$ and $\beta_1$. If $T$ is the Brenier map from $\beta_0$ to $\beta_1$, this barycenter is $((1-t)\Id+tT)_\sharp\beta_0$, the McCann interpolation detailed in Section~\ref{sec-geodesic-convexity}. If no Monge map is available, the same construction uses an optimal coupling $\pi$ and the interpolation map $(x,y)\mapsto(1-t)x+ty$, giving $((1-t)x+ty)_\sharp\pi$.
\index{Monge!problem}
\index{optimal coupling}
\index{McCann interpolation}
\index{convexity!geodesic}
\index{interpolation map}
\index{Brenier!map}
\end{example}
\begin{example}[Dirac inputs recover Fr\'echet means]
	Problem~\eqref{eq-barycenter-generic} generalizes the computation of barycenters of points $(x_s)_{s=1}^S \in \X^S$ to arbitrary measures. Indeed, if $\be_s=\de_{x_s}$ is a single Dirac mass, then a solution to~\eqref{eq-barycenter-generic} is $\de_{x^\star}$ where $x^\star$ is a Fr\'echet mean of the points $(x_s)_s$.
\index{Frechet mean}
\index{Dirac mass}
\end{example}

\begin{figure}[ht]
\centering
\begin{tabular}{@{}cc@{}}
\includegraphics[width=.42\linewidth]{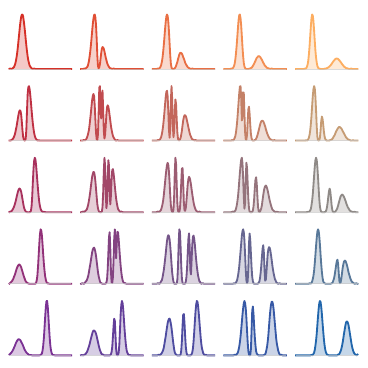} &
\includegraphics[width=.42\linewidth]{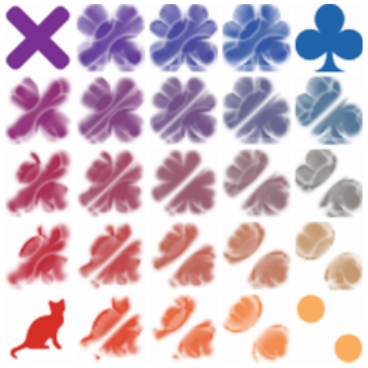}
\\[-.1em]
\small one-dimensional quantile barycenters &
\index{quantile!barycenter}
\small two-dimensional entropic barycenters
\index{entropic!barycenter}
\end{tabular}
\caption{Wasserstein barycenter grids for four corner measures. The left panel uses the one-dimensional formula $Q_{u,v}=\sum_{i,j}\lambda_{ij}(u,v)Q_{ij}$ for one Gaussian law and three asymmetric two-Gaussian mixtures, and displays densities reconstructed from the averaged quantiles. The right panel computes entropic Wasserstein barycenters on a common pixel grid for the cat, two-disk, cross and clover silhouettes, using the normalized squared ground cost, $\epsilon=4\cdot10^{-4}$ and a Sinkhorn tolerance of $5\cdot10^{-8}$. The barycenters are rendered as density images with values clamped at their $95\%$ quantile rather than by threshold contours. Colors interpolate between the four corners and encode the same bilinear weights in both panels.}
\index{Gaussian mixture}
\index{Wasserstein!barycenter}
\label{fig:barycenters-four-shapes}
\end{figure}

\begin{rem}[Mean of a quadratic barycenter]
\index{quadratic!barycenter}
\index{barycenter mean}
	For $c(x,y)=\norm{x-y}^2$, the mean of the barycenter $\al^\star$ is necessarily the barycenter of the means,
	\eq{
			\int_\Xx x \d\al^\star(x) = \sum_s \la_s \int_\Xx x \d\be_s(x).
	}
	Indeed, the squared Wasserstein distance decomposes into a squared distance between means plus a centered Wasserstein term. Minimizing the resulting quadratic function of the barycenter mean gives the displayed identity. If the input measures have compact support, the usual multi-marginal barycentric construction also gives a barycenter supported in the convex hull of the input supports.
\index{Wasserstein!distance}
\index{multi-marginal}
\end{rem}

The next elementary proposition explains why~\eqref{eq-barycenter-generic} is a convex optimization problem over measures. The difficulty is not convexity, but the fact that the unknown is itself a measure whose support is not known in advance.

\begin{prop}[Convexity of the OT cost]\label{prop-barycenter-ot-cost-convexity}
\index{cost!transport}
	The map $(\al,\be)\mapsto\MK_\c(\al,\be)$ is convex.
\end{prop}

\begin{proof}
	Let $(\al_0,\be_0)$ and $(\al_1,\be_1)$ be two pairs of probability measures and let $t\in[0,1]$. For $\eta>0$, choose couplings $\pi_i\in\Couplings(\al_i,\be_i)$ such that
\index{probability measure}
	\[
		\int c\d\pi_i\leq\MK_\c(\al_i,\be_i)+\eta
		\qquad (i=0,1).
	\]
	Then $\pi_t=(1-t)\pi_0+t\pi_1$ is a coupling between $(1-t)\al_0+t\al_1$ and $(1-t)\be_0+t\be_1$. Hence
	\[
		\MK_\c((1-t)\al_0+t\al_1,(1-t)\be_0+t\be_1)
		\leq
		(1-t)\MK_\c(\al_0,\be_0)+t\MK_\c(\al_1,\be_1)+\eta.
	\]
	Letting $\eta\to0$ gives the claim.
\end{proof}

Even when all input measures are discrete, the support of a barycenter is not known a priori. The multi-marginal formulation of Section~\ref{sec-multimarginal-ot} shows that a discrete barycenter can be supported on all weighted averages of one support point from each input. This gives at most $\prod_s n_s$ candidate points if the $s$-th input has $n_s$ atoms, which is prohibitive when the number of inputs is large. A common numerical compromise is therefore to prescribe a smaller support for the barycenter and solve a fixed-support problem.
\index{multi-marginal!OT}
\index{multi-marginal}

\paragraph{One-dimensional case.}

On the line, barycenters become linear after the quantile change of variables. This gives the rare case where the barycenter is explicit rather than the solution of a high-dimensional optimization problem.
\index{change-of-variables}

\begin{prop}[Quantile barycenters on the line]\label{prop-quantile-barycenters}
\index{quantile!barycenter}
	For $\X=\RR$ and $c(x,y)=|x-y|^2$, the quantile function of a Wasserstein barycenter is the weighted average of the input quantile functions:
\index{Wasserstein!barycenter}
\index{quantile!function}
	\[
		\cumul{\al^\star}^{-1}(r)
		=
		\sum_{s=1}^S\la_s\cumul{\be_s}^{-1}(r),
		\qquad r\in[0,1].
	\]
\end{prop}

\begin{proof}
	The one-dimensional formula~\eqref{eq-wass-cumul} gives
	\[
		\sum_s\la_s\Wass_2^2(\al,\be_s)
		=
		\int_0^1
		\sum_s\la_s
		\abs{\cumul{\al}^{-1}(r)-\cumul{\be_s}^{-1}(r)}^2
		\d r.
	\]
	The minimization decouples pointwise in $r$. For each fixed $r$, the minimizer of $z\mapsto\sum_s\la_s|z-\cumul{\be_s}^{-1}(r)|^2$ is the weighted average $\sum_s\la_s\cumul{\be_s}^{-1}(r)$. This function is nondecreasing because it is a positive weighted sum of nondecreasing quantile functions, hence it is a valid quantile function.
\index{quantile!function}
\end{proof}

\paragraph{Gaussian case.}

Gaussian barycenters show that the same separation as in the Gaussian Wasserstein formula~\eqref{eq-dist-gauss} persists: means average linearly, while covariances average according to the Bures--Wasserstein geometry.
\index{Bures-Wasserstein geometry}
\index{Gaussian!barycenter}

\begin{example}[Gaussian inputs remain Gaussian]
	The barycenter of Gaussian measures is Gaussian. In one dimension, it is obtained by averaging the means and the standard deviations, so the barycenter variance is the square of this averaged standard deviation. In higher dimensions, the covariance $\cov$ minimizes the Bures objective
\index{Gaussian!measure}
	\[
		\cov \mapsto \sum_s \la_s \Bb(\cov,\cov_s)^2,
	\]
	and equivalently solves the fixed-point equation
	\[
		\cov =
		\sum_s \la_s
		\pa{\cov^{1/2}\cov_s\cov^{1/2}}^{1/2}.
	\]
	This is the covariance analogue of the usual Euclidean barycenter equation: the mean part averages linearly, while the covariance part averages through the Bures--Wasserstein geometry~\cite{alvarez2016fixed,bhatia2018bures}.
\index{Bures-Wasserstein geometry}
\end{example}

\begin{alg}[Gaussian barycenter fixed point]\label{alg:gaussian-barycenter-fixed-point}
\index{Gaussian!barycenter}
\index{fixed point!iteration}
\textbf{Input:} Gaussian measures $\Gaussian(\mean_s,\cov_s)$, weights $\lambda_s$, tolerance $\mathrm{tol}$.

\textbf{Output:} Gaussian barycenter $\Gaussian(\mean,\cov)$.

\textbf{Set}
\(\mean=\sum_s\lambda_s\mean_s .\)

\textbf{Initialize:} Set $\cov^{(0)}=\sum_s\lambda_s\cov_s$.

\textbf{For} $k=0,1,\ldots$ \textbf{do}:
\begin{algblock}
\(S^{(k)} = \sum_s\lambda_s \left((\cov^{(k)})^{1/2}\cov_s(\cov^{(k)})^{1/2}\right)^{1/2}.\)

\(\cov^{(k+1)} = (\cov^{(k)})^{-1/2} \left(S^{(k)}\right)^2 (\cov^{(k)})^{-1/2}.\)

\textbf{If} $\norm{\cov^{(k+1)}-\cov^{(k)}}\leq\mathrm{tol}$ \textbf{then}:
\begin{algblock}

\textbf{Set} $\cov=\cov^{(k+1)}$.
\textbf{Return} $\Gaussian(\mean,\cov)$.
\end{algblock}
\end{algblock}
\end{alg}

\begin{figure}[ht]
\centering
\begin{tabular}{@{}cc@{}}
\includegraphics[width=.35\linewidth]{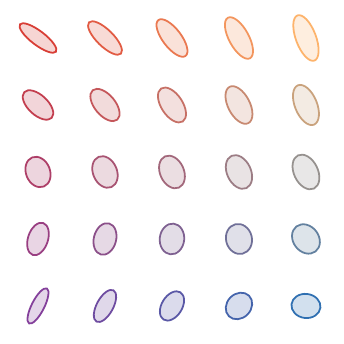} &
\includegraphics[width=.35\linewidth]{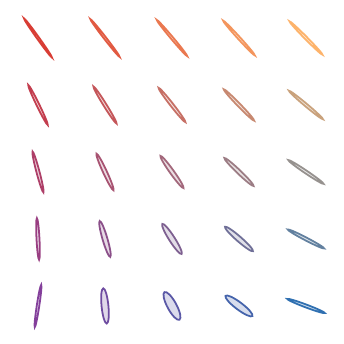}
\\[-.1em]
\small moderate anisotropy &
\small strong anisotropy
\end{tabular}
\caption{Bures--Wasserstein barycenters of centered Gaussian covariance matrices. Each panel shows a $5\times5$ grid of barycenter ellipses for four corner covariances, without separate input panels: the corner ellipses are the four input covariances themselves. The right grid uses more anisotropic inputs, making the nonlinear rotation and scaling of covariance barycenters more visible.}
\index{Wasserstein!barycenter}
\label{fig:barycenters-gaussian-covariances}
\end{figure}

\paragraph{Sinkhorn for barycenters.}
\index{Sinkhorn!barycenter}

A key difference with the regularized two-marginal OT problem is that there is no canonical reference measure $\al\otimes\be$, because the barycenter $\al$ is unknown. To reduce complexity, one usually fixes a candidate support for the barycenter and solves the discrete problem~\eqref{eq-wass-discr}; this introduces a discretization error but keeps the number of unknowns manageable.
\index{discretization error}
\index{reference!measure}

One can then use entropic smoothing and approximate~\eqref{eq-wass-discr} by
\index{entropic!smoothing}
\eql{\label{eq-entropic-bary}
	\umin{\a\in\simplex_n}
	\sum_{s=1}^S\la_s\MKD_{\C_s}^\epsilon(\a,\b_s)
}
for some $\epsilon>0$. This is a smooth convex minimization problem, which can be tackled using gradient descent~\cite{CuturiBarycenter}. An alternative is to use a descent method, typically quasi-Newton, on the semi-dual~\cite{2016-Cuturi-siims}; this is useful when adding extra regularization on the barycenter, for instance to impose smoothness.
\index{semi-dual}

A simple but effective approach, as remarked in~\cite{2015-benamou-cisc}, rewrites~\eqref{eq-entropic-bary} as a weighted KL projection problem
\index{KL!projection}
\eql{\label{eq-bary-entropy-couplings}
	\umin{(\P_s)_s}
	\enscond{
		\epsilon\sum_s\la_s\KLD(\P_s|\K_s)
	}{
		\foralls s,\ \transp{\P_s}\ones_n=\b_s,
		\quad
		\P_1\ones_{n_1}=\ldots=\P_S\ones_{n_S}
	},
}
where $\K_s\eqdef e^{-\C_s/\epsilon}$. The barycenter $\a$ is implicitly encoded in the common row marginal
\[
	\a=\P_1\ones_{n_1}=\ldots=\P_S\ones_{n_S}.
\]
The optimal couplings solving~\eqref{eq-bary-entropy-couplings} have scaling form
\index{optimal coupling}
\index{scaling!form}
\eql{\label{eq-bary-opt}
	\P_s=\diag(\uD_s)\K_s\diag(\vD_s),
}
and the generalized Sinkhorn iterations are
\index{Sinkhorn!iteration}
\begin{align}
	\foralls s\in\range{1,S},\qquad
	\itt{\vD}_s
	&\eqdef
	\frac{\b_s}{\transp{\K_s}\it{\uD}_s},
	\label{eq-sinkhorn-bary}
	\\
	\foralls s\in\range{1,S},\qquad
	\itt{\uD}_s
	&\eqdef
	\frac{\itt{\a}}{\K_s\itt{\vD}_s},
	\label{eq-sinkhorn-bary-2}
	\\
	\qwhereq
	\itt{\a}
	&\eqdef
	\prod_s(\K_s\itt{\vD}_s)^{\la_s}.
	\label{eq-sinkhorn-bary-3}
\end{align}
The geometric mean in~\eqref{eq-sinkhorn-bary-3} enforces the fact that all couplings share the same barycenter marginal.

\begin{alg}[Entropic barycenter Sinkhorn]\label{alg:entropic-barycenter-sinkhorn}
\textbf{Input:} Costs $\C_s$, target weights $\b_s$, barycenter weights $\lambda_s$, regularization $\epsilon>0$, tolerance $\mathrm{tol}$.

\textbf{Output:} Barycenter weights $\a$ and couplings $\P_s$.

\textbf{Initialize:} Set $\K_s=e^{-\C_s/\epsilon}$, $\uD_s^{(0)}=\ones_n$ for all $s$, \(r_0=+\infty\), and \(k=0\).

\textbf{While} \(r_k>\mathrm{tol}\) \textbf{do}:
\begin{algblock}

\textbf{Set} \(k\leftarrow k+1\).

\textbf{For} each marginal $s$ \textbf{do}
\begin{algblock}
\(\vD_s^{(k)} = \frac{\b_s}{\transp{\K_s}\uD_s^{(k-1)}}.\)
\end{algblock}
\textbf{Compute} barycenter marginal:
\(\a^{(k)} = \prod_s \bigl(\K_s\vD_s^{(k)}\bigr)^{\lambda_s}.\)

\textbf{For} each marginal $s$ \textbf{do}
\begin{algblock}
\(\uD_s^{(k)} = \frac{\a^{(k)}}{\K_s\vD_s^{(k)}}.\)
\end{algblock}

\textbf{Set} \(\P_s^{(k)}=\diag(\uD_s^{(k)})\K_s\diag(\vD_s^{(k)})\) for all \(s\).

\textbf{Set} \(r_k=\max_s \max\{\norm{\P_s^{(k)}\ones-\a^{(k)}}_1,\norm{(\P_s^{(k)})^\top\ones-\b_s}_1\}\).
\end{algblock}
\algreturnskip
\textbf{Return} \(\a^{(k)}\) and \(\P_s^{(k)}\).
\end{alg}

\begin{prop}[Dual of entropic barycenters]\label{prop-dual-entropic-barycenters}
\index{entropic!barycenter}
	The optimal scalings in~\eqref{eq-bary-opt} can be written as $(\uD_s,\vD_s)=(e^{\fD_s/\epsilon},e^{\gD_s/\epsilon})$, where $(\fD_s,\gD_s)_s$ solve the dual problem
\index{dual!problem}
	\eql{\label{eq-dual-bary-entropy}
		\umax{(\fD_s,\gD_s)_s}
		\enscond{
			\sum_s\la_s
			\left(
				\dotp{\gD_s}{\b_s}
				-\epsilon\dotp{\K_s e^{\gD_s/\epsilon}}{e^{\fD_s/\epsilon}}
			\right)
		}{
			\sum_s\la_s\fD_s=0
		}.
	}
\end{prop}

\begin{proof}
	Introduce Lagrange multipliers in~\eqref{eq-bary-entropy-couplings}:
	\begin{align*}
		\umin{(\P_s)_s,\a}
		\umax{(\fD_s,\gD_s)_s}
		\sum_s\la_s\Big(
			\epsilon\KLD(\P_s|\K_s)
			+\dotp{\a-\P_s\ones_{n_s}}{\fD_s}
			+\dotp{\b_s-\transp{\P_s}\ones_n}{\gD_s}
		\Big).
	\end{align*}
	Strong duality holds, so one can exchange the minimum and maximum. The minimization with respect to $\a$ gives the constraint $\sum_s\la_s\fD_s=0$, and the minimization with respect to $\P_s$ gives the Legendre transform of $\KLD(\cdot|\K_s)$:
\index{duality!strong}
\index{Legendre transform}
	\[
		\umax{(\fD_s,\gD_s)_s}
		\sum_s\la_s
		\left[
			\dotp{\gD_s}{\b_s}
			-\epsilon
			\KLD^*\left(\frac{\fD_s\oplus\gD_s}{\epsilon}\middle|\K_s\right)
		\right],
		\qquad
		\sum_s\la_s\fD_s=0.
	\]
	The separable conjugate is
	\eql{\label{eq-legendre-kl-bary}
		\KLD^*(\VectMode{U}|\K)
		=
		\sum_{i,j}\K_{i,j}\big(e^{\VectMode{U}_{i,j}}-1\big),
	}
	because for $k>0$,
	\[
		\sup_{r\geq0}
		ur-\big(r\log(r/k)-r+k\big)
		=
		k(e^u-1),
	\]
	and the case $k=0$ follows by lower semicontinuity. Dropping constants independent of $(\fD_s,\gD_s)_s$ gives~\eqref{eq-dual-bary-entropy}. The coordinate maximization in $\gD_s$ gives~\eqref{eq-sinkhorn-bary}; the block maximization in all $(\fD_s)_s$ gives the common marginal~\eqref{eq-sinkhorn-bary-3} and then~\eqref{eq-sinkhorn-bary-2}.
\index{lower semicontinuity}
\end{proof}

Classical applications include two-dimensional image interpolation, three-dimensional shape interpolation, and barycenters on surfaces where the ground cost is the square of the geodesic distance; see~\cite{2015-solomon-siggraph} for applications to computer graphics and imaging.
\index{ground cost}

\section{Multimarginal OT}
\index{multi-marginal!OT}
\label{sec-multimarginal-ot}
\label{subsec:multi-marginal}

Multi-marginal OT couples more than two measures at once. It is the natural language for barycenters, matching with teams and several-body costs, but its tensor dimension is the main computational obstacle.
\index{multi-marginal}

\paragraph{Definition and basic structure.}

The multi-marginal formulation replaces a coupling between two measures by a joint distribution with several prescribed marginals. Given measures $(\al_s)_{s=1}^S$ on spaces $(\X_s)_{s=1}^S$ and a cost $c:\X_1\times\cdots\times\X_S\to\RR$, the problem reads
\index{multi-marginal}
\[
	\inf_{\pi\in\Couplings(\al_1,\ldots,\al_S)}
	\int_{\X_1\times\cdots\times\X_S}
	c(x_1,\ldots,x_S)\d\pi(x_1,\ldots,x_S),
\]
where $\Couplings(\al_1,\ldots,\al_S)$ is the set of probability measures whose $s$-th marginal is $\al_s$. This is still a linear program in the discrete setting, but its ambient tensor has size $\prod_s n_s$.
\index{probability measure}

\paragraph{Multi-marginal formulation of barycenters.}
\index{multi-marginal}
\index{multi-marginal!formulation}

Wasserstein barycenters are the central example. For the squared Euclidean cost, one can introduce a latent barycenter point and eliminate it explicitly, leading to the multi-marginal cost
\index{multi-marginal}
\index{multi-marginal!cost}
\index{Wasserstein!barycenter}
\[
	c_{\mathrm{bar}}(x_1,\ldots,x_S)
	=
	\min_{x\in\RR^d}
	\sum_{s=1}^S\la_s\norm{x-x_s}^2.
\]

\begin{prop}[Multi-marginal formula for quadratic barycenters]\label{prop-multimarginal-barycenter}
\index{multi-marginal}
\index{multi-marginal!OT}
\index{quadratic!barycenter}
	Let $\be_1,\ldots,\be_S\in\Pp_2(\RR^d)$ and $\la\in\simplex_S$. Define
	\[
		B(x_1,\ldots,x_S)=\sum_{s=1}^S\la_s x_s,
		\qquad
		c_{\mathrm{bar}}(x_1,\ldots,x_S)=\min_x\sum_s\la_s\norm{x-x_s}^2.
	\]
	If $\pi^\star$ solves the multi-marginal OT problem with marginals $(\be_s)_s$ and cost $c_{\mathrm{bar}}$, then $\al^\star=B_\sharp\pi^\star$ is a Wasserstein barycenter. Conversely, every barycenter is obtained this way from an optimal multi-marginal plan.
\index{Wasserstein!barycenter}
\index{multi-marginal}
\end{prop}

\begin{proof}
	For any candidate barycenter $\al$ and couplings $\pi_s\in\Couplings(\al,\be_s)$, glue the couplings along their common $\al$ marginal to obtain a joint law of $(X,Y_1,\ldots,Y_S)$. Conditioning on $(Y_s)_s$ and minimizing over $X$ gives
\index{joint!law}
	\[
		\sum_s\la_s\EE\norm{X-Y_s}^2
		\geq
		\EE\min_x\sum_s\la_s\norm{x-Y_s}^2
		=
		\EE c_{\mathrm{bar}}(Y_1,\ldots,Y_S).
	\]
	Taking the infimum over the couplings gives that the barycenter value is at least the multi-marginal value. Conversely, from an optimal multi-marginal plan $\pi^\star$, set $X=B(Y_1,\ldots,Y_S)$. The couplings between $X$ and each $Y_s$ are feasible for the barycenter problem and attain exactly the multi-marginal cost, proving equality and the formula.
\index{multi-marginal!cost}
	If $\al^\star$ is any barycenter, choose optimal couplings between $\al^\star$ and each $\be_s$ and glue them along the common $\al^\star$ marginal. Since the barycenter and multi-marginal values are equal, the conditional minimization inequality above must be an equality. Thus $X=B(Y_1,\ldots,Y_S)$ almost surely for the induced optimal multi-marginal plan, and $\al^\star=B_\sharp\pi^\star$.
\index{optimal coupling}
\index{multi-marginal}
\end{proof}

\begin{cor}[Gaussian and discrete barycenters]\label{cor-gaussian-discrete-barycenters}
\index{Gaussian!barycenter}
\index{discrete!barycenter}
	Quadratic Wasserstein barycenters of Gaussian measures are Gaussian. If the input measures are discrete, then there exists a barycenter supported on the set of weighted averages $\sum_s\la_s x_{s,i_s}$ of one support point from each input; in particular, if the $s$-th input has $n_s$ atoms, a barycenter exists with at most $\prod_s n_s$ atoms.
\index{Wasserstein!barycenter}
\index{Gaussian!measure}
\end{cor}

\begin{proof}
	Let the input Gaussians have means $\mean_s$ and covariances $\cov_s$. For any candidate barycenter $\al$ with mean $\mean$ and covariance $\cov$, Gelbrich's inequality~\cite{gelbrich1990formula}, proved later in Theorem~\ref{thm-gelbrich-projection}, gives
	\[
		\Wass_2^2(\al,\be_s)
		\geq
		\norm{\mean-\mean_s}^2+\Bb(\cov,\cov_s)^2,
	\]
	with equality for the Gaussian law with mean $\mean$ and covariance $\cov$. Therefore the barycenter objective is bounded below by a function depending only on $(\mean,\cov)$, and this lower bound is attained by the Gaussian measure with the minimizing mean and covariance. Hence at least one barycenter is Gaussian, and uniqueness in the usual nondegenerate setting gives the Gaussian barycenter mentioned above. For discrete inputs, any multi-marginal optimizer is supported on the finite product of the input supports, and $B$ maps this product to at most $\prod_s n_s$ points.
\index{Gaussian!barycenter}
\index{Gaussian!measure}
\index{multi-marginal}
\end{proof}

\paragraph{Entropic regularization of multi-marginal OT.}
\index{entropic!regularization}
\index{multi-marginal!OT}

As in the two-marginal case, adding an entropic penalty with respect to the product measure $\al_1\otimes\cdots\otimes\al_S$ leads to scaling algorithms:
\index{scaling!algorithm}
\index{product!measure}
\[
	\inf_{\pi\in\Couplings(\al_1,\ldots,\al_S)}
	\int c\d\pi+\epsilon\KL(\pi|\al_1\otimes\cdots\otimes\al_S).
\]
The optimizer has the generalized Gibbs form
\[
	\d\pi^\star(x_1,\ldots,x_S)
	=
	\exp\!\left(\frac{\sum_s f_s(x_s)-c(x_1,\ldots,x_S)}{\epsilon}\right)
	\prod_s\d\al_s(x_s),
\]
and generalized Sinkhorn iterations alternately update one potential $f_s$ so that the $s$-th marginal is correct. The bottleneck is the tensor size $\prod_s n_s$ in the discrete case. Practical barycenter solvers therefore exploit separability of the cost, low-rank structure, convolutional kernels, or a fixed barycenter support.
\index{Sinkhorn!iteration}

In finite dimension, the direct generalized scaling scheme is the tensor version of Sinkhorn.

\begin{alg}[Multi-marginal Sinkhorn]\label{alg:multimarginal-sinkhorn}
\textbf{Input:} Marginals $\a_s\in\simplex_{n_s}$, tensor cost $C$, regularization $\epsilon>0$, tolerance $\mathrm{tol}$.

\textbf{Output:} Multi-marginal entropic coupling tensor $P$.

\textbf{Build}
\(K_{i_1,\ldots,i_S} = \exp\!\left(-\frac{C_{i_1,\ldots,i_S}}{\epsilon}\right) \prod_{s=1}^S(\a_s)_{i_s}.\)

\textbf{Initialize:} Set \(u_s=\ones_{n_s}\) for all $s$ and residual \(r=+\infty\).

\textbf{While} \(r>\mathrm{tol}\) \textbf{do}:
\begin{algblock}

\textbf{For} $s=1,\ldots,S$ \textbf{do}:
\begin{algblock}
\((u_s)_i \leftarrow \frac{(\a_s)_i} { \sum_{i_1,\ldots,i_{s-1},i_{s+1},\ldots,i_S} K_{i_1,\ldots,i_{s-1},i,i_{s+1},\ldots,i_S} \prod_{r\neq s}(u_r)_{i_r}}.\)
\end{algblock}

\textbf{Set} \(P_{i_1,\ldots,i_S}=K_{i_1,\ldots,i_S}\prod_s (u_s)_{i_s}\).

\textbf{Set} \(r=\max_s\norm{(\mathrm{proj}_s)_\sharp P-\a_s}_1\).
\end{algblock}
\algreturnskip
\textbf{Return} \(P\).
\end{alg}

\section{Metric learning and inverse OT}
\index{metric!learning}
\index{inverse OT}

This final section points to inverse problems where the ground cost itself is learned. Such problems are typically bilevel and non-convex, but OT provides useful gradients with respect to the cost.
\index{ground cost}

\paragraph{Metric learning and derivatives of OT}
\index{metric!learning}
\index{cost!derivative}

OT is convex with respect to the measure and concave with respect to the cost. Ground-metric learning was explicitly studied in~\cite{CuturiGroundMetric2014}, and it connects to the broader metric-learning literature~\cite{MAL-019,bellet2015metric}.
\index{metric!learning}

\begin{prop}[Derivative with respect to the cost]
\index{cost!derivative}
	In the discrete setting, assume that the optimal coupling for $\MKD_\C(\a,\b)$ is unique and denote it by $\P^\star(\C)$. Then $\C \mapsto \MKD_\C(\a,\b)$ is differentiable at $\C$ and
\index{optimal coupling}
	\[
		\nabla_\C \MKD_\C(\a,\b)=\P^\star(\C).
	\]
\end{prop}
\begin{proof}
	The value is the minimum of affine functions of $\C$,
	\[
		\MKD_\C(\a,\b)=\min_{\P\in\CouplingsD(\a,\b)}\dotp{\C}{\P}.
	\]
	The envelope theorem, or equivalently Danskin's theorem, states that the subdifferential with respect to $\C$ is the convex hull of the optimal couplings. If the optimizer is unique, this subdifferential is the singleton $\{\P^\star(\C)\}$, hence the value is differentiable with the displayed gradient.
\index{Danskin theorem}
\index{subdifferential}
\index{envelope theorem}
\end{proof}

Thus, if the cost is parameterized as $\C_\theta$, gradients of losses involving OT values are obtained by backpropagating through the inner product $\dotp{\P^\star(\C_\theta)}{\partial_\theta \C_\theta}$. The difficulty is not differentiating a solved OT problem, but learning a cost for which the resulting matching has the desired semantic behavior; this is a bilevel and usually non-convex optimization problem.

\begin{figure}[ht]
\centering
\begin{tabular}{@{}ccc@{}}
\includegraphics[width=.30\linewidth]{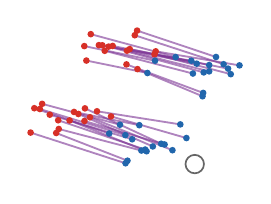} &
\index{metric!learning}
\includegraphics[width=.30\linewidth]{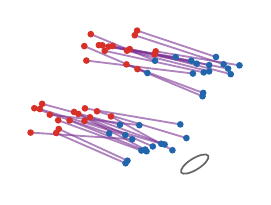} &
\includegraphics[width=.30\linewidth]{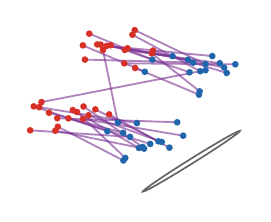}
\\[-.1em]
\small Euclidean cost & \small moderate anisotropy & \small strong anisotropy
\end{tabular}
\caption{Changing the ground metric changes the optimal coupling. The same red and blue empirical measures are matched with $c_A(x,y)=(x-y)^\top A(x-y)$ for the Euclidean metric and two increasingly anisotropic Mahalanobis metrics. The small gray ellipse shows the unit ball of the metric: directions in which the ellipse is elongated are cheaper, and this deforms the transport segments selected by the OT plan.}
\index{empirical!measure}
\label{fig:metric-learning-cost-deformation}
\end{figure}

\paragraph{Inverse Optimal Transport}
\index{inverse OT}

Inverse OT asks for a ground cost that explains observed matchings or flows as optimal transport plans. In its most direct form, one observes a plan $\widehat\pi$ with marginals $(\alpha,\beta)$ and seeks a cost $c$ such that $\widehat\pi$ is optimal for
\index{ground cost}
\index{plan!transport}
\[
	\inf_{\pi\in\Couplings(\alpha,\beta)}\int c(x,y)\d\pi(x,y).
\]
This is ill-posed without structure: adding potentials $u(x)+v(y)$ to a cost does not change the set of optimal couplings, and many costs can rationalize the same sparse plan.
\index{optimal coupling}

A useful statistical methodology is to measure the suboptimality of the observed plan through a Fenchel--Young loss. Write the score as $s=-c$ and define the convex regularized prediction value
\index{Fenchel-Young loss}
\[
	G_\epsilon(s)=
	\sup_{\pi\in\Couplings(\alpha,\beta)}
	\int s\d\pi-\epsilon\KL(\pi|\alpha\otimes\beta).
\]
The Fenchel--Young loss
\index{Fenchel-Young loss}
\[
	\mathcal L_\epsilon(c;\widehat\pi)
	=
	G_\epsilon(-c)+G_\epsilon^*(\widehat\pi)+\int c\d\widehat\pi
\]
is nonnegative by Fenchel's inequality and vanishes exactly when $\widehat\pi\in\partial G_\epsilon(-c)$, i.e. when $\widehat\pi$ satisfies the regularized optimality conditions for $c$. Sharpened Fenchel--Young losses for inverse problems over measures and inverse entropic/unbalanced OT are developed in~\cite{andrade2025sharpened}; curvature and identifiability of inverse OT with respect to the cost are studied in~\cite{peyre2026curvature}. Entropic regularization is important here because it makes the forward map smoother and provides gradients with respect to $c$, at the price of a bias that must be controlled statistically.
\index{Fenchel-Young loss}
\index{entropic!regularization}
\index{unbalanced!OT}

In the discrete unregularized case, this loss reduces to the optimality gap of the observed coupling. For $\widehat P\in\CouplingsD(\a,\b)$ and a cost matrix $C$, denote it by
\[
	\mathcal L_{\mathrm{iOT}}(C;\widehat P)
	=
	\dotp{C}{\widehat P}
	-
	\umin{P\in\CouplingsD(\a,\b)} \dotp{C}{P}.
\]
This inverse-OT gap loss is nonnegative and vanishes exactly when $\widehat P$ is optimal for $C$.
\index{inverse OT!gap loss}

In practice, one restricts the cost to a finite-dimensional model class, often affine:
\[
	C_\theta=\sum_{r=1}^R \theta_r C^{(r)},
	\qquad \theta\in\Theta,
\]
where $\Theta$ is convex and the matrices $C^{(r)}$ encode features, graph distances or a Mahalanobis parameterization. This viewpoint appears in low-rank and sparse inverse OT models~\cite{dupuy2016estimating,andrade2024sparsistency} and in convex formulations for learning OT costs from observed plans~\cite{ma2020learning,peyre2026curvature}.

\index{inverse OT}
A minimal finite-dimensional model is obtained by learning a bilinear cost on $\RR^d$,
\[
	c_A(x,y)=\dotp{Ax}{y},
	\qquad A\in\RR^{d\times d}.
\]
For empirical measures $\al=\frac1n\sum_i\de_{x_i}$ and $\be=\frac1n\sum_j\de_{y_j}$, this gives the cost matrix
\[
	C(A)_{i,j}=\dotp{Ax_i}{y_j},
\]
so both maps $A\mapsto C(A)$ and $A\mapsto c_A$ are linear. Inverse OT within this model asks which matrix $A$ makes an observed matching or coupling look optimal; learning the cost is thus reduced to estimating a linear parameter.
\index{bilinear!cost}

For a fixed matrix $A$, the forward prediction is the optimal face
\[
	\mathcal P_A\eqdef
	\uargmin{P\in\CouplingsD(\ones_n/n,\ones_n/n)}
	\dotp{C(A)}{P}.
\]
When this face is a singleton, write its element as $P_A$; otherwise $P_A$ denotes a deterministic tie-broken selection. Although $A\mapsto C(A)$ is linear, the solution correspondence $A\mapsto\mathcal P_A$ is polyhedral: changing $A$ changes the direction in which the transport polytope is probed, and a tie-broken selection is constant on normal-cone cells. Figure~\ref{fig:inverse-ot-forward-logo} illustrates this correspondence on the OT4ML point clouds. The construction follows the visual idea of the Python Optimal Transport logo~\cite{flamary2021pot}: red source atoms, blue target atoms and straight segments show the selected optimal bijection. The rank-one matrices $A=-e_1e_1^\top$ and $A=-e_2e_2^\top$ only score horizontal or vertical correlations. The matrix $A=-I$ gives the usual quadratic $\Wass_2$ assignment, up to the marginal-only terms discussed below, while $A=+I$ reverses the correlation and produces an anti-$\Wass_2$ matching.

\begin{figure}[ht]
\centering
\begin{tabular}{@{}cc@{}}
\includegraphics[width=.48\linewidth]{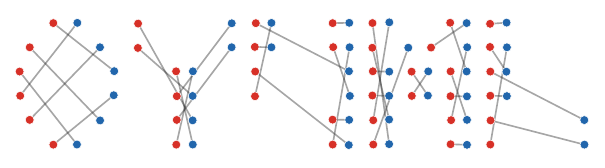} &
\includegraphics[width=.48\linewidth]{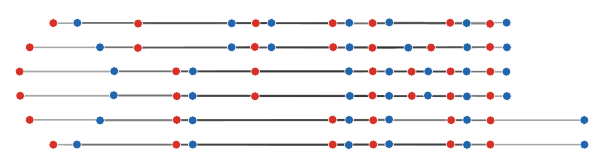}
\\[-.1em]
\small horizontal rank one, $A=-e_1e_1^\top$ &
\small vertical rank one, $A=-e_2e_2^\top$
\\[.35em]
\includegraphics[width=.48\linewidth]{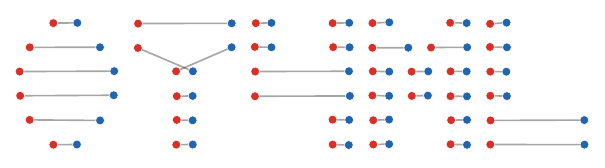} &
\includegraphics[width=.48\linewidth]{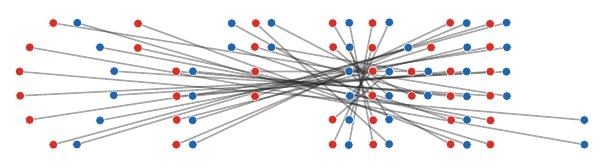}
\\[-.1em]
\small quadratic $\Wass_2$, $A=-I$ &
\small anti-$\Wass_2$, $A=+I$
\end{tabular}
\caption{Forward solutions of the bilinear cost $c_A(x,y)=\dotp{Ax}{y}$ on the OT4ML logo point clouds. Each panel solves the equal-weight assignment problem with a different matrix $A$; the source atoms are red, the target atoms are blue, and the gray segments give one deterministic optimal bijection.}
\label{fig:inverse-ot-forward-logo}
\end{figure}

This elementary model already contains the quadratic Wasserstein assignment. Adding to a cost matrix a term depending only on $x_i$ or only on $y_j$ shifts all feasible couplings by the same constant, and therefore does not change the optimizer. Since
\[
	\norm{x-y}^2=\norm{x}^2+\norm{y}^2-2\dotp{x}{y},
\]
the usual quadratic Wasserstein assignment has the same optimizer as the bilinear cost with $A_\star=-I$, up to these marginal-only terms and an irrelevant positive factor. The inverse problem goes in the opposite direction: after observing a coupling, one asks which matrices $A$ could have generated it. Figure~\ref{fig:inverse-ot-gap-loss} generates an observed coupling $\widehat P$ from this cost on two empirical mixtures of Gaussians, and then evaluates $\mathcal L_{\mathrm{iOT}}(C(A_t);\widehat P)$ along the anisotropic path
\[
	A_t=-\diag(1+t,1-t),
	\qquad -1\leq t\leq 1,
\]
so that $t=0$ recovers the matrix that generated the observed coupling. Equivalently, with equal weights, $\widehat P\in\CouplingsD(\ones_n/n,\ones_n/n)=\mathcal B_n/n$ and the plotted loss is the Kantorovich gap
\[
	\mathcal L_{\mathrm{iOT}}(C(A_t);\widehat P)
	=
	\dotp{C(A_t)}{\widehat P}
	-
	\umin{P\in\CouplingsD(\ones_n/n,\ones_n/n)}
	\dotp{C(A_t)}{P},
	\qquad
	C(A_t)_{i,j}=\dotp{A_t x_i}{y_j}.
\]
Because $t\mapsto C(A_t)$ is affine and the Kantorovich value is a minimum of affine functions over the fixed polytope $\CouplingsD(\ones_n/n,\ones_n/n)$, this one-dimensional gap is convex and piecewise affine. Its zero set can contain an interval for a small sample, reflecting the fact that the same observed coupling remains optimal for a cone of nearby costs.

\begin{figure}[ht]
\centering
\begin{tabular}{@{}ccc@{}}
\includegraphics[width=.35\linewidth]{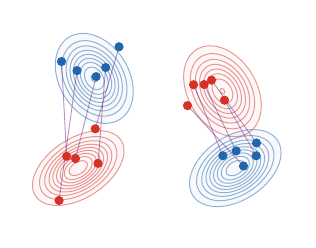} &
\includegraphics[width=.29\linewidth]{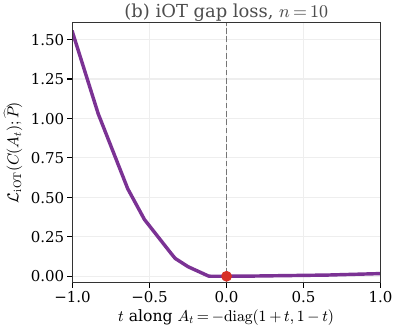} &
\includegraphics[width=.29\linewidth]{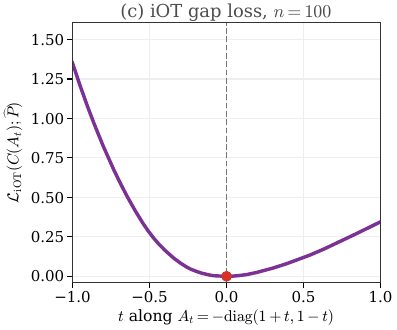}
\\[-.1em]
\small observed map, $n=10$ &
\small loss curve, $n=10$ &
\small loss curve, $n=100$
\end{tabular}
\caption{Inverse-OT gap loss for a bilinear cost. Panel (a): two empirical mixtures of two Gaussians are matched with the cost $c_{A_\star}(x,y)=\dotp{A_\star x}{y}$ for $A_\star=-I$, which gives the same optimizer as the quadratic $\Wass_2$ cost; red and blue level sets display the two sampling densities. Panels (b,c): the unregularized Fenchel--Young Kantorovich gap $\mathcal L_{\mathrm{iOT}}(C(A_t);\widehat P)$ along $A_t=-\diag(1+t,1-t)$ for $n=10$ and $n=100$, using the same vertical scale. The red dot marks the generating parameter $t=0$; the curves are convex and piecewise affine.}
\label{fig:inverse-ot-gap-loss}
\end{figure}

The comparison between $n=10$ and $n=100$ illustrates an important statistical effect: as the number of sampled points grows, the flat region of the empirical gap typically shrinks and the loss develops more visible curvature around the generating parameter. This anticipates the population theory of Peyr\'e, Poon and Tron~\cite{peyre2026curvature}: in the limit $n\to+\infty$, when the limiting Monge map itself has nondegenerate curvature as the cost parameter varies, the iOT loss identifies the cost robustly (up to the usual marginal-only gauge freedoms). In that regime, minimizing the gap is not only a certificate of optimality of the observed transport, but also a stable way to recover the underlying cost.

\begin{prop}[Convex dual-gap formulation of inverse OT]\label{prop-inverse-ot-convex}
\index{dual!gap}
\index{inverse OT}
	Let $\widehat P\in\CouplingsD(\a,\b)$ be an observed coupling and let $C_\theta$ depend affinely on $\theta\in\Theta$, where $\Theta$ is convex. The condition that $\widehat P$ is optimal for the cost $C_\theta$ is equivalent to the existence of dual potentials $(f,g)$ such that
\index{dual!potential}
	\[
		f_i+g_j\leq (C_\theta)_{i,j}
		\qandq
		\sum_{i,j}\widehat P_{i,j}\big((C_\theta)_{i,j}-f_i-g_j\big)=0.
	\]
	Consequently, for a convex regularizer $R$, the noisy inverse problem can be relaxed as the convex program
\index{convex!regularizer}
	\eql{\label{eq-inverse-ot-convex}
		\umin{\theta\in\Theta,f,g}
		R(\theta)+\lambda
		\sum_{i,j}\widehat P_{i,j}\big((C_\theta)_{i,j}-f_i-g_j\big)
		\quad\text{subject to}\quad
		f_i+g_j\leq (C_\theta)_{i,j}\quad\forall i,j.
	}
\end{prop}

\begin{proof}
	For a fixed cost $C_\theta$, Kantorovich duality gives
\index{Kantorovich!duality}
	\[
		\min_{P\in\CouplingsD(\a,\b)}\dotp{C_\theta}{P}
		=
		\max_{f_i+g_j\leq (C_\theta)_{i,j}}
		\dotp{f}{\a}+\dotp{g}{\b}.
	\]
	Since $\widehat P$ has marginals $(\a,\b)$, every dual feasible pair satisfies
	\[
		\dotp{C_\theta}{\widehat P}-\dotp{f}{\a}-\dotp{g}{\b}
		=
		\sum_{i,j}\widehat P_{i,j}\big((C_\theta)_{i,j}-f_i-g_j\big)
		\geq0.
	\]
	This nonnegative quantity is exactly the primal-dual gap of $\widehat P$. It vanishes if and only if $\widehat P$ reaches the dual value and is therefore optimal. If $C_\theta$ is affine and $\Theta$ and $R$ are convex, the constraints and objective in~\eqref{eq-inverse-ot-convex} are convex, proving the relaxation claim.
\index{dual!gap}
\end{proof}

\begin{alg}[Inverse OT by dual-gap fitting]\label{alg:inverse-ot-dual-gap-learning}
\index{inverse OT}
\index{dual!gap}
\textbf{Input:} Observed plan $\widehat P\in\CouplingsD(\a,\b)$, features $C^{(r)}$, feasible set $\Theta$, regularizer $R$.

\textbf{Output:} Identified cost $C_{\theta^\star}$ and potentials $(f^\star,g^\star)$.

\textbf{Set} parametric cost:
\(C_\theta=\sum_r\theta_r C^{(r)}.\)

\textbf{Let} $(\theta^\star,f^\star,g^\star)$ be a minimizer of
\(\min_{\theta\in\Theta,f,g} R(\theta)+\lambda \sum_{i,j}\widehat P_{i,j}\big((C_\theta)_{i,j}-f_i-g_j\big)\)

\textbf{Subject to}
\(f_i+g_j\leq(C_\theta)_{i,j} \qquad\text{for all }(i,j).\)
\textbf{Return} $\theta^\star$, $C_{\theta^\star}$, and $(f^\star,g^\star)$.
\end{alg}

The formulation~\eqref{eq-inverse-ot-convex} is useful because it avoids differentiating through a forward OT solver: it learns a cost by making the observed plan satisfy complementary slackness. In statistical settings, $\widehat P$ is only partially observed or noisy, so one adds sparsity, low-rank, smoothness or metric constraints to select a meaningful cost~\cite{dupuy2016estimating,andrade2024sparsistency}. For entropic OT, the optimality condition becomes smoother:
\index{entropic!OT}
\index{optimality!complementary slackness}
\[
	\widehat P_{i,j}\approx a_i b_j
	\exp\pa{\frac{f_i+g_j-(C_\theta)_{i,j}}{\epsilon}},
\]
which leads to likelihood-based or KL-based convex objectives when $C_\theta$ is affine, and connects inverse OT with generalized Sinkhorn iterations and transport-regularized inverse problems~\cite{karlsson2016generalized,ma2020learning}. Neural parameterizations of $C_\theta$ are more flexible but reintroduce non-convexity; the convex formulation above is the clean mathematical baseline.
\index{Sinkhorn!iteration}


\section{Weak Optimal Transport}
\index{weak!optimal transport}
\label{sec-weak-ot}

Weak OT relaxes the cost so that it depends on the conditional distribution of destinations rather than only on pointwise pairs. It is useful when a source point is allowed to choose a randomized response and the model only penalizes an aggregate of that response, such as its conditional mean.
\index{conditional law}

\paragraph{Barycentric projection of a coupling.}
\index{barycentric!projection}
The first object to isolate is therefore the map obtained by collapsing each conditional law to its barycenter.

\begin{defn}[Barycentric projection of a coupling]\label{def-barycentric-projection}
\index{barycentric!projection}
	Let $\alpha,\beta\in\Pp_2(\RR^d)$ and let $\pi\in\Couplings(\alpha,\beta)$. Disintegrate $\pi$ with respect to its first marginal as
	$\pi(\d x,\d y)=\pi_x(\d y)\alpha(\d x)$. The barycentric projection of $\pi$ is the map
	\begin{equation}\label{eq-barycentric-projection}
		\bar T_\pi(x)
		\eqdef
		\int_{\RR^d}y\d\pi_x(y),
		\qquad
		\bar\beta_\pi
		\eqdef
		(\bar T_\pi)_\sharp\alpha.
	\end{equation}
\end{defn}
The projected target $\bar\beta_\pi$ records the distribution of conditional means, not the full second marginal. Thus it is generally different from $\beta$; if $\pi=(\Id,T)_\sharp\alpha$ is induced by a map, then $\bar T_\pi=T$ and $\bar\beta_\pi=\beta$. This projection is not an optimal map for an arbitrary coupling: a deterministic rotation of a radially symmetric source, for example, projects to the rotation itself, whereas the optimal map from the source to itself is the identity. The useful positive statement is attached to quadratic optimal plans, as in the tangent-space viewpoint on $\Wass_2$ developed by Ambrosio, Gigli and Savar{\'e}~\cite[Chap.~7]{ambrosio2006gradient}.
\index{optimal plan}

\begin{prop}[Barycentric projection of a quadratic optimal plan]\label{prop-barycentric-projection-optimal}
\index{optimal plan}
\index{barycentric!projection}
	Let $\pi\in\Couplings(\alpha,\beta)$ be optimal for the quadratic cost $\norm{x-y}^2$ between $\alpha,\beta\in\Pp_2(\RR^d)$, and define $\bar T_\pi$ and $\bar\beta_\pi$ by~\eqref{eq-barycentric-projection}. Then $(\Id,\bar T_\pi)_\sharp\alpha$ is an optimal coupling between $\alpha$ and $\bar\beta_\pi$. Equivalently, $\bar T_\pi$ is a quadratic optimal transport map from $\alpha$ to the projected target $\bar\beta_\pi$.
\index{optimal coupling}
\index{transport map}
\index{barycentric!projection}
\index{cost!quadratic}
\end{prop}

\begin{proof}
	By Theorem~\ref{thm:opt_ccm}, $\pi$ is concentrated on a $c$-cyclically monotone set $\Gamma$ for $c(x,y)=\norm{x-y}^2$. For the quadratic cost, and since it is enough to check cyclic permutations, this means that every finite cycle $(x_i,y_i)_{i=1}^m\subset\Gamma$ satisfies
\index{cost!quadratic}
	\[
		\sum_{i=1}^m \dotp{x_i}{y_i}
		\geq
		\sum_{i=1}^m \dotp{x_i}{y_{i+1}},
		\qquad y_{m+1}=y_1.
	\]
	After changing the disintegration on an $\alpha$-negligible set, $\pi_x$ is supported on the section
\index{disintegration}
	\[
		\Gamma_x=\enscond{y}{(x,y)\in\Gamma}
	\]
	for $\alpha$-a.e. $x$.
	Choose $x_1,\ldots,x_m$ in this full-measure set and independently sample $Y_i\sim\pi_{x_i}$. Applying the cyclic inequality to $(x_i,Y_i)$ and taking expectations gives
	\[
		\sum_{i=1}^m \dotp{x_i}{\bar T_\pi(x_i)}
		\geq
		\sum_{i=1}^m \dotp{x_i}{\bar T_\pi(x_{i+1})}.
	\]
	Thus $(\Id,\bar T_\pi)_\sharp\alpha$ is concentrated on a cyclically monotone graph. By the cyclic-monotonicity characterization of quadratic optimality, this plan is optimal between its two marginals, namely $\alpha$ and $\bar\beta_\pi$.
\index{cyclic!monotonicity}
\end{proof}

Weak transport costs use the same disintegration but allow the objective to depend on the whole conditional law, or on summaries such as the barycentric projection~\eqref{eq-barycentric-projection}. The framework was introduced through general transport costs and weak transport inequalities in~\cite{gozlan2017kantorovich}; existence, duality and optimality conditions on Polish spaces are developed in~\cite{backhoff2019weak}. For a weak cost $C:\X\times\Pp(\Y)\to\RR\cup\{+\infty\}$, the weak OT value is
\index{weak!cost}
\index{disintegration}
\index{conditional law}
\index{barycentric!projection}
\index{weak!transport}
\index{Polish space}
\begin{equation}\label{eq-weak-ot}
	\WOT_C(\alpha,\beta)
	\eqdef
	\inf_{\pi\in\Couplings(\alpha,\beta)}
	\int C(x,\pi_x)\d\alpha(x).
\end{equation}
The classical Kantorovich problem is recovered when $C(x,\nu)=\int c(x,y)\d\nu(y)$, because the objective then becomes $\int c(x,y)\d\pi(x,y)$. The genuinely weak behavior starts when $C$ is nonlinear in $\nu$.
\index{Kantorovich!problem}

\begin{prop}[Weak Kantorovich duality]\label{prop-weak-ot-duality}
\index{Kantorovich!duality}
\index{weak!Kantorovich duality}
	Assume that $\X,\Y$ are compact metric spaces and that $C(x,\nu)$ is lower semicontinuous, bounded from below and convex in $\nu$, with the standard qualification assumptions ensuring Fenchel--Rockafellar duality. For $g\in\Cc(\Y)$ define the weak $C$-transform
\index{duality!Fenchel-Rockafellar}
\index{lower semicontinuity}
\index{Fenchel duality}
	\[
		g^C(x)
		\eqdef
		\inf_{\nu\in\Pp(\Y)}
		\left\{
			C(x,\nu)-\int g(y)\d\nu(y)
		\right\}.
	\]
	Then
	\[
		\WOT_C(\alpha,\beta)
		=
		\sup_{g\in\Cc(\Y)}
		\left\{
			\int g^C(x)\d\alpha(x)+\int g(y)\d\beta(y)
		\right\}.
	\]
	When $C(x,\nu)=\int c(x,y)\d\nu(y)$, this reduces to the usual Kantorovich dual with $g^C(x)=\inf_y(c(x,y)-g(y))$.
\end{prop}

\begin{proof}
	For any coupling $\pi$ and any $g\in\Cc(\Y)$, the definition of $g^C$ gives
	\[
		C(x,\pi_x)\geq g^C(x)+\int g(y)\d\pi_x(y).
	\]
	After integration with respect to $\alpha$, the second term becomes $\int g\d\beta$ because the second marginal of $\pi$ is $\beta$. This proves weak duality.
\index{duality!weak}

	For the reverse inequality, consider the convex minimization over probability kernels $x\mapsto\pi_x$ with the affine constraint $\int \pi_x\d\alpha(x)=\beta$. Fenchel--Rockafellar duality gives a continuous Lagrange multiplier $g$ for this marginal constraint. Minimizing the Lagrangian over each conditional law gives exactly the pointwise term $g^C(x)$, while the multiplier contributes $\int g\d\beta$. The compactness, lower semicontinuity, convexity and qualification assumptions ensure no duality gap. This is the weak-cost analogue of Proposition~\ref{prop-duality-discr}.
\index{duality!Fenchel-Rockafellar}
\index{lower semicontinuity}
\index{Fenchel duality}
\index{conditional law}
\index{marginal!constraint}
\index{affine!constraint}
\end{proof}

\begin{prop}[Barycentric weak transport is weaker than $\Wass_2$]\label{prop-barycentric-weak-ot}
\index{weak!transport}
	Let $\alpha,\beta\in\Pp_2(\RR^d)$ and define
	\[
		C_{\mathrm{bar}}(x,\nu)
		=
		\norm{x-\int y\d\nu(y)}^2.
	\]
	Equivalently, for a coupling $\pi$, the integrand is $\norm{x-\bar T_\pi(x)}^2$. Then
	\[
		\mathcal{W}_{C_{\mathrm{bar}}}(\alpha,\beta)
		\leq
		\Wass_2^2(\alpha,\beta).
	\]
\end{prop}

\begin{proof}
	Let $\pi$ be any coupling and disintegrate it as $\pi_x\alpha$. By Jensen's inequality,
\index{Jensen inequality}
\index{disintegration}
	\[
		\norm{x-\bar T_\pi(x)}^2
		\leq
		\int\norm{x-y}^2\d\pi_x(y).
	\]
	Integrating in $x$ gives $\int C_{\mathrm{bar}}(x,\pi_x)\d\alpha(x)\leq\int\norm{x-y}^2\d\pi(x,y)$. Taking the infimum over $\pi$ proves the claim.
\end{proof}

\begin{figure}[ht]
\centering
\begin{tabular}{@{}cc@{}}
\includegraphics[width=.43\linewidth]{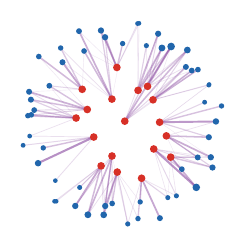} &
\index{barycentric!projection}
\includegraphics[width=.43\linewidth]{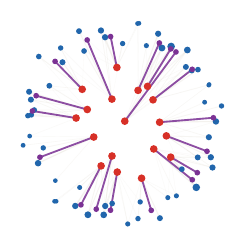}
\\[-.1em]
\small conditionals $\pi_x$ & \small barycentric projection $\bar T_\pi$
\end{tabular}
\caption{Weak barycentric transport on a small disk-to-annulus coupling. The left panel shows the full conditional laws: each red source atom splits its mass among several blue target atoms, with segment thickness proportional to transported mass. The right panel collapses each conditional law $\pi_x$ to its barycenter $\bar T_\pi(x)=\int y\d\pi_x(y)$, shown in violet. The barycentric weak cost only sees the red-to-violet displacement, and therefore ignores the conditional spread around each barycenter.}
\index{weak!cost}
\index{barycentric!transport}
\label{fig:weak-ot-barycentric-projection}
\end{figure}

The barycentric cost is the canonical example to keep in mind: admissibility still constrains the full conditional laws to have second marginal $\beta$, but the objective only charges the displacement from $x$ to $\bar T_\pi(x)$ and ignores the conditional variance around this barycenter. This connects weak OT with martingale transport, Strassen-type convex-order constraints, barycentric projections and learning problems where conditional averages are meaningful objects.
\index{conditional law}
\index{barycentric!projection}


\chapter{Beyond Comparing Measures}
\label{sec-beyond-comparing-measures}

The last group leaves the setting of scalar measures on a common ambient space. Vector- and matrix-valued OT transports mass with internal degrees of freedom, Gromov--Wasserstein compares metric-measure spaces without a prescribed correspondence, and quantum OT replaces scalar couplings by positive operators. In each case, the transport plan also has to encode structure carried by the support, the fibers or the non-commutative state space.
\index{positive operator}
\index{plan!transport}
\index{correspondence}
\index{quantum!OT}
\index{metric-measure space}
\index{Gromov-Wasserstein}

\section{Vector and Matrix-Valued Measures}
\index{vector-valued measure}
\index{matrix-valued measure}
\label{sec-vector-matrix-valued-measures}

Scalar OT transports a nonnegative density. In imaging, color processing, spectral analysis, diffusion tensor imaging and quantum-inspired models, the object attached to a point can instead have several nonnegative components or a positive semidefinite matrix. The first step beyond scalar OT is the positive vector-valued case, where the fiber remains linear and commutative but the channels may interact.
\index{positive!semidefinite matrix}

\paragraph{Positive vector-valued measures.}
\index{positive!vector-valued measure}

\begin{defn}[Positive vector-valued measure]\label{def-positive-vector-valued-measure}
\index{vector-valued measure}
\index{positive!vector-valued measure}
	A positive $\RR_+^m$-valued measure on $\X$ is a tuple
	\[
		\mu=(\mu^1,\ldots,\mu^m)\in\Mm_+(\X;\RR_+^m),
	\]
	where each component $\mu^k$ is a nonnegative finite measure.
\end{defn}
This models multi-channel densities such as color histograms, spectral bins or several species transported on the same domain. In a conservative model the mass of each channel is preserved, so one assumes $\mu_0^k(\X)=\mu_1^k(\X)$ for every $k$. The natural vector-valued extension of OT therefore starts from the positive cone $\RR_+^m$.
\index{finite measure}
\index{histogram}
\index{positive!cone}

To keep the notation readable, first assume that the endpoints and the curve have densities. The direct analogue of Benamou--Brenier fixes a vector density $u_t(x)\in\RR_+^m$ and a spatial flux
\index{flux}
\index{Benamou-Brenier}
$V_t(x)=(V_{t,1},\ldots,V_{t,d})\in(\RR^m)^d$, where $V_{t,\ell}^k$ is the momentum of channel $k$ in spatial direction $\ell$. The conservative vector transport cost associated with an action density $\Phi$ is
\index{momentum}
\begin{equation}\label{eq-vector-valued-bb}
	\mathcal W_{\Phi}^2(\mu_0,\mu_1)
	\eqdef
	\inf_{u,V}
	\int_0^1\!\int_\X
	\Phi(u_t(x),V_t(x))\,\d x\,\d t
\end{equation}
subject to the endpoint constraints $u_0\d x=\mu_0$, $u_1\d x=\mu_1$ and the componentwise continuity equation
\index{continuity equation}
\begin{equation}\label{eq-vector-valued-continuity}
	\partial_t u_t+\nabla_x\cdot V_t=0,
	\qquad
	(\nabla_x\cdot V_t)^k=\sum_{\ell=1}^d \partial_{x_\ell} V_{t,\ell}^k .
\end{equation}
Thus each component satisfies its own continuity equation, but the cost may still couple the components. Singular curves are handled as in scalar dynamic OT by replacing densities and fluxes by measures and using the lower semicontinuous perspective recession convention.
\index{perspective recession}
\index{recession convention}
\index{flux}
\index{lower semicontinuity}
\index{continuity equation}

The following elementary family separates independent channel motion from genuinely coupled vector transport.

\begin{example}[Diagonal and coupled positive mobilities]
Choose a mobility matrix $\mathsf M(u)\in\mathbb S_+^m$, where $\mathbb S_+^m$ denotes the cone of real symmetric positive semidefinite matrices, and set
\index{positive!semidefinite matrix}
\[
	\Phi_{\mathsf M}(u,V)
	=
	\sum_{\ell=1}^d V_{\ell}^{\top}\mathsf M(u)^\dagger V_{\ell},
\]
with the usual convention that the value is finite only when each $V_\ell$ belongs to the range of $\mathsf M(u)$. One chooses $\mathsf M$ so that this matrix perspective is convex and one-homogeneous in $(u,V)$; this holds for the linear positive mobilities below. For $m=1$ and $\mathsf M(u)=u$, one recovers exactly the scalar Benamou--Brenier action. For
\index{Benamou-Brenier}
\[
	\mathsf M_{\mathrm{diag}}(u)=\diag(u_1,\ldots,u_m),
\]
the channels move independently. Non-diagonal mobilities are the simplest way to couple the coordinates while keeping the same componentwise conservation law. For instance, with $q=m^{-1/2}(1,\ldots,1)$ and $\kappa\geq0$,
\[
	\mathsf M_\kappa(u)=\diag(u)+\kappa\Big(\sum_{k=1}^m u_k\Big) q q^\top
\]
increases the mobility in the common channel direction $q$ while leaving transverse directions controlled by the diagonal part. The local cost of moving one component can therefore depend on the densities and momenta of the other components, even though each component mass remains conserved.
\index{momentum}
\end{example}

\begin{prop}[Diagonal positive vector Benamou--Brenier]\label{prop-diagonal-positive-vector-bb}
\index{Benamou-Brenier}
Assume that $\mu_0^k,\mu_1^k\in\Mm_+(\X)$ have the same mass $m_k$ for every $k$. For the diagonal mobility $\mathsf M_{\mathrm{diag}}$, the value of~\eqref{eq-vector-valued-bb} is
\[
	\mathcal W_{\mathrm{diag}}^2(\mu_0,\mu_1)
	=
	\sum_{k:m_k>0} m_k\,\Wass_2^2\!\left(\frac{\mu_0^k}{m_k},\frac{\mu_1^k}{m_k}\right),
\]
with the convention that zero-mass channels contribute zero.
\end{prop}
\begin{proof}
For $\mathsf M_{\mathrm{diag}}(u)=\diag(u_1,\ldots,u_m)$, the action separates as
\[
	\sum_{\ell=1}^d V_\ell^\top \mathsf M_{\mathrm{diag}}(u)^\dagger V_\ell
	=
	\sum_{k=1}^m \frac{|V^k|^2}{u^k},
\]
where $V^k=(V_1^k,\ldots,V_d^k)$ is the spatial momentum of channel $k$, and the scalar perspective convention is used. The constraint~\eqref{eq-vector-valued-continuity} also separates into $\partial_t u^k+\nabla\cdot V^k=0$. The minimization therefore splits into $m$ independent scalar Benamou--Brenier problems. If $m_k=0$, nonnegativity and conservation force the whole channel to vanish. If $m_k>0$, normalizing $\rho_t^k=u_t^k/m_k$ and $p_t^k=V_t^k/m_k$ factors the channel action as $m_k\int |p_t^k|^2/\rho_t^k$, hence the scalar value is $m_k\Wass_2^2(\mu_0^k/m_k,\mu_1^k/m_k)$. Summing over the channels proves the claim.
\index{momentum}
\index{Benamou-Brenier}
\end{proof}

The conservative positive-cone model above is the basic extension of Benamou--Brenier. Adding a source term $\partial_t u+\nabla\cdot V=S$ and a convex perspective penalty in $S$ gives unbalanced or reaction--transport variants. Such generalized transport models with dissipation and density modulation were developed by Maas, Rumpf, Sch{\"o}nlieb and Simon~\cite{maas2015generalized,maas2016generalized}; related nonlinear mobility distances and gradient structures appear in~\cite{dolbeault2009new,MielkeCVPDE}. Figure~\ref{fig:vector-valued-measure-geodesics} contrasts the exact diagonal case $\kappa=0$, where each positive channel is transported by its quantile map, with a large-$\kappa$ illustrative common-mode interpolation in which the channels move more coherently. The endpoints are two-mode mixtures: at each spatial mode the two channels have Gaussian profiles with the same center but different amplitudes.
\index{gradient!structure}
\index{vector-valued measure}
\index{positive!cone}
\index{quantile!map}

\begin{figure}[H]
\centering
\begin{tabular}{@{}cc@{}}
\includegraphics[width=.45\linewidth]{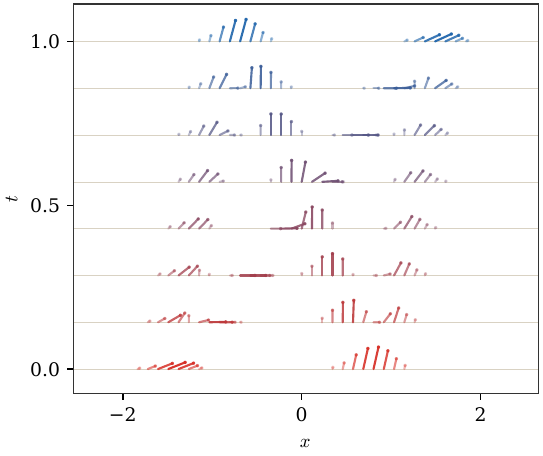} &
\index{vector-valued measure}
\includegraphics[width=.45\linewidth]{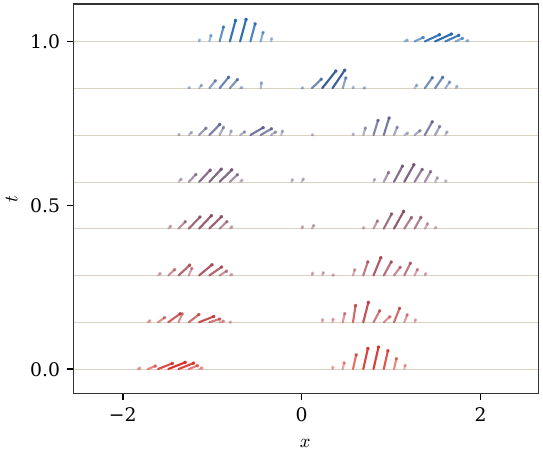}
\\[-.1em]
\small independent channels $(\kappa=0)$ &
\small strongly coupled channels $(\kappa\gg1)$
\end{tabular}
\caption{One-dimensional positive $\RR_+^2$-valued transport displayed by arrow glyphs at eight time levels. Each endpoint is a mixture of two localized Gaussian modes, and, inside each mode, both channel profiles have the same center. Each arrow is proportional to the local fiber value $(u_t^1(x),u_t^2(x))$, and time runs vertically from the red source to the blue target. Left: for $\kappa=0$, the diagonal mobility of Proposition~\ref{prop-diagonal-positive-vector-bb} gives two independent scalar quantile geodesics. Right: a large-$\kappa$ common-mode interpolation bends the display toward the direction $q=2^{-1/2}(1,1)$, illustrating the qualitative effect of a mobility that favors coherent channel motion while keeping the same componentwise continuity equation.}
\index{continuity equation}
\label{fig:vector-valued-measure-geodesics}
\end{figure}

\paragraph{Positive matrix-valued measures.}
\index{matrix-valued measure}
\index{positive!matrix-valued measure}
The next simplest fiber is the positive matrix cone. This is the simplest tensor-valued model beyond vectors: the diagonal entries behave like positive channels, while the eigenvectors encode local orientations.

\begin{defn}[Positive matrix-valued measure]\label{def-positive-matrix-valued-measure}
\index{matrix!cone}
\index{positive!matrix-valued measure}
	Write $\mathbb S^m$ for real symmetric matrices and $\mathbb S_+^m$ for the positive semidefinite cone. A positive $\mathbb S_+^m$-valued measure is an element
	\[
		\mathcal A\in\Mm_+(\X;\mathbb S_+^m).
	\]
\end{defn}
If $\mathcal A$ has density $A(x)\succeq0$, then $\tr A(x)$ is the scalar amount of mass at $x$, while, wherever $\tr A(x)>0$, the normalized matrix $A(x)/\tr A(x)$ records an internal covariance or orientation. This is the matrix analogue of the positive vector case: diagonal matrices encode nonnegative vector components, and non-diagonal matrices add a local eigenbasis.

The conservative Benamou--Brenier model fixes a matrix density $A_t(x)\in\mathbb S_+^m$ and symmetric matrix fluxes $P_t(x)=(P_{t,1},\ldots,P_{t,d})\in(\mathbb S^m)^d$. With no flux through the boundary of $\X$, the full matrix mass $\int_\X A_t(x)\d x$ is conserved, so the endpoints must have the same total matrix. The model minimizes the matrix-perspective action
\index{flux}
\index{Benamou-Brenier}
\begin{equation}\label{eq-matrix-valued-bb}
\begin{aligned}
\mathcal W_{\mathrm{mat}}^2(\mathcal A_0,\mathcal A_1)
\eqdef
\inf_{A,P}\int_0^1\!\int_\X
\sum_{\ell=1}^d \tr\big(P_{t,\ell}^{\top} A_t^\dagger P_{t,\ell}\big)\d x\d t
\end{aligned}
\end{equation}
subject to $A_0\d x=\mathcal A_0$, $A_1\d x=\mathcal A_1$ and to the matrix-valued continuity equation
\index{matrix-valued measure}
\index{continuity equation}
\begin{equation}\label{eq-matrix-valued-continuity}
\partial_t A_t+\nabla_x\cdot P_t=0,
\qquad
\nabla_x\cdot P_t=\sum_{\ell=1}^d \partial_{x_\ell}P_{t,\ell}.
\end{equation}
Here $A^\dagger$ denotes the Moore--Penrose inverse, with the usual perspective convention: the action is finite only when the columns of each $P_{t,\ell}$ belong to the range of $A_t$. The map $(A,P)\mapsto\tr(P^{\top}A^\dagger P)$ is the matrix fractional function; it is jointly convex on $A\succeq0$. This gives the simplest non-trivial matrix-valued transport model: spatial motion is conservative, but the fiber carries orientation through the eigenvectors of $A_t(x)$.

\begin{prop}[Diagonal matrix subproblem]\label{prop-matrix-diagonal-reduction}
\index{matrix!subproblem}
Assume that the endpoints are diagonal in a fixed orthonormal basis,
\[
	\mathcal A_i=\diag(\mu_i^1,\ldots,\mu_i^m),
	\qquad i=0,1,
\]
and that $\mu_0^k(\X)=\mu_1^k(\X)=m_k$ for every $k$. If one restricts the admissible curves in~\eqref{eq-matrix-valued-bb} to remain diagonal in that basis,
\[
	A_t=\diag(u_t^1,\ldots,u_t^m),
	\qquad
	P_{t,\ell}=\diag(V_{t,\ell}^1,\ldots,V_{t,\ell}^m),
\]
then the value of this restricted matrix problem is
\[
	\sum_{k:m_k>0} m_k\,\Wass_2^2\!\left(\frac{\mu_0^k}{m_k},\frac{\mu_1^k}{m_k}\right),
\]
with zero contribution from zero-mass channels. Thus the commuting matrix submodel is exactly the diagonal positive vector-valued Benamou--Brenier model of Proposition~\ref{prop-diagonal-positive-vector-bb}.
\index{Benamou-Brenier}
\end{prop}
\begin{proof}
The continuity equation~\eqref{eq-matrix-valued-continuity} is diagonal entry by diagonal entry and gives $\partial_t u^k+\nabla\cdot V^k=0$. Moreover,
\index{continuity equation}
\[
	\sum_{\ell=1}^d \tr\big(P_{t,\ell}^{\top}A_t^\dagger P_{t,\ell}\big)
	=
	\sum_{k=1}^m \frac{|V_t^k|^2}{u_t^k}
\]
with the same scalar perspective convention as before. The admissible set and the action are therefore exactly those of the diagonal vector model.
\end{proof}

The restriction to a fixed diagonal basis gives eigenvalue transport; it should be read as a commuting submodel, not as a claim that non-diagonal excursions can never change the unrestricted value. The genuinely matrix-valued case starts when the eigenspaces vary with $x$ or along the interpolation, so that the transported object carries both mass and orientation. Static matrix-valued Monge--Kantorovich problems and dual test-function metrics were developed in~\cite{Ning2014metrics,JiangSpectral,ning2015matrix}; dynamic versions and related non-commutative geometries appear in~\cite{Chen2016,ChenGangbo17,Carlen2014,2016-peyre-qot}. Figure~\ref{fig:matrix-valued-measure-geodesic} shows the analogous independent/coupled contrast for positive $2\times2$ matrix fibers, using two localized matrix modes whose eigenvalue profiles share a common center at each mode.
\index{Kantorovich!problem}
\index{matrix-valued measure}

\begin{figure}[H]
\centering
\begin{tabular}{@{}cc@{}}
\includegraphics[width=.46\linewidth]{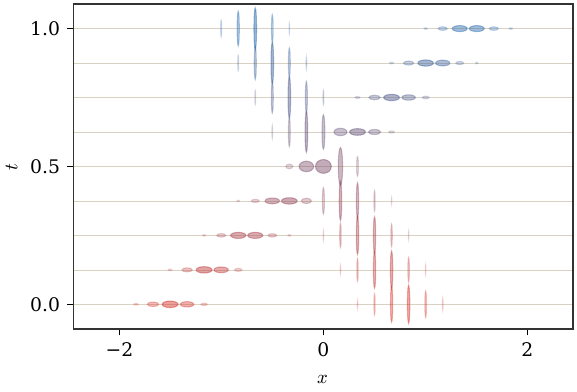} &
\index{matrix-valued measure}
\includegraphics[width=.46\linewidth]{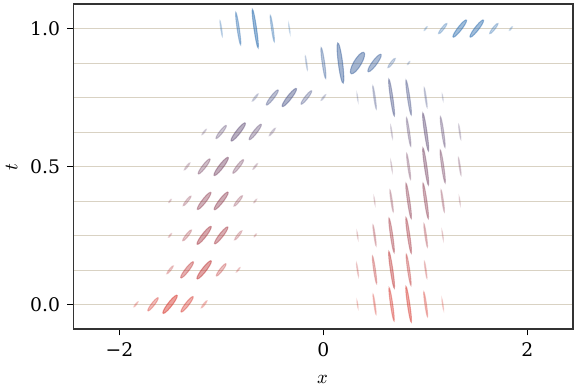}
\\[-.1em]
\small commuting tensor channels &
\small strongly coupled tensor fibers
\end{tabular}
\caption{Positive $2\times2$ matrix-valued transport on a one-dimensional base. Each endpoint is a mixture of two localized matrix modes; within one mode, both eigenvalue profiles are Gaussian bumps with the same center. Each ellipse is the glyph of a positive semidefinite matrix $A_t(x)$, with axes given by eigenvectors and eigenvalues. Left: the matrices are diagonal in a fixed basis, giving the commuting tensor analogue of independent vector channels. Right: a coupled illustrative interpolation bends packet motion toward the trace-density transport and uses non-commuting eigendirections; the superposition remains positive semidefinite and produces spatially varying orientations.}
\index{positive!semidefinite matrix}
\label{fig:matrix-valued-measure-geodesic}
\end{figure}

\section{Gromov--Wasserstein}
\index{Gromov-Wasserstein}
\label{sec-gromov-wasserstein}

Gromov--Wasserstein compares spaces through their internal distance structures rather than through a fixed ambient ground cost. This is the right extension for graphs, shapes and point clouds whose points are not pre-aligned.
\index{ground cost}

\paragraph{Discrete formulation.}

Optimal transport needs a ground cost $\C$ to compare histograms $(\a,\b)$, and thus cannot be used directly if the histograms are not defined on the same underlying space, or if one cannot pre-register these spaces to define a ground cost.
\index{histogram}
\index{ground cost}
To address this issue, one can instead use a weaker requirement: two matrices $\distD \in \RR^{n \times n}$ and $\distD' \in \RR^{m \times m}$ are available and represent relationships between the points on which the histograms are defined. A typical scenario is when these matrices are powers of distance matrices.
The Gromov--Wasserstein problem reads
\index{Gromov-Wasserstein}
\eql{\label{eq-gw-def}
	\GWD( (\a,\distD), (\b,\distD') )^p \eqdef
	\umin{ \P \in \CouplingsD(\a,\b) }
		\Ee_{\distD,\distD'}(\P) \eqdef
		\sum_{i,j,i',j'} \De(\distD_{i,i'},\distD'_{j,j'})^p \P_{i,j}\P_{i',j'},
}
where $p \geq 1$ and $\De$ is a distance on $\RR$, typically $\De(u,v)=|u-v|$.
This is a non-convex quadratic problem over the transport polytope. In the uniform case with the additional hard-assignment constraint $m=n$ and $\P$ a permutation matrix, it becomes a Quadratic Assignment Problem (QAP)~\cite{loiola-2007}; this restricted graph-matching form is already NP-hard in full generality.
\index{assignment problem}
\index{quadratic!assignment problem}
\index{matrix!permutation}
\index{transportation!polytope}
The relaxed coupling formulation used in GW can therefore be read as a soft graph-matching model~\cite{lyzinski-2015}.

\begin{figure}[ht]
\centering
\begin{tabular}{@{}ccc@{}}
\includegraphics[width=.30\linewidth]{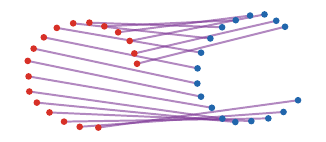} &
\includegraphics[width=.30\linewidth]{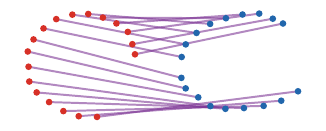} &
\includegraphics[width=.30\linewidth]{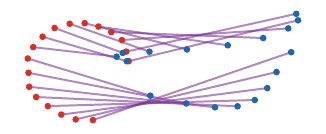}
\\[-.1em]
\small isometric copy &
\small mild deformation &
\small strong deformation
\end{tabular}
\caption{Gromov--Wasserstein correspondences under increasing deformation. The red and blue point clouds are not compared through an ambient Euclidean cross-cost; instead, the GW coupling compares their internal pairwise distances. A perfectly isometric copy admits a clean structural match, while mild and deliberately stronger deformations progressively bend the correspondence.}
\index{pairwise distance}
\index{Gromov-Wasserstein}
\index{correspondence}
\label{fig:gromov-isometry-matching}
\end{figure}

When the matrices $\distD,\distD'$ are genuine distance matrices, the general construction below shows that $\GWD$ satisfies the triangle inequality and defines a distance between metric spaces equipped with a probability distribution, up to measure-preserving isometries.
\index{triangle inequality}
\index{measure-preserving!isometry}
This distance was introduced and studied in detail by Memoli in~\cite{memoli-2011}. An in-depth mathematical exposition (in particular, its geodesic structure and gradient flows) is given in~\cite{SturmGW}. See also~\cite{schmitzer2013modelling} for applications in computer vision.
\index{gradient!flow}
Its relation to Hausdorff and Gromov--Hausdorff distances is discussed at the end of this section.
\index{Hausdorff distance}
\index{Gromov-Hausdorff distance}

\paragraph{General setting.}

\begin{defn}[Metric-measure space]\label{def-metric-measure-space}
\index{metric-measure space}
	A metric-measure space is a triple
	\[
		\XX=(\X,\dist_\X,\al),
	\]
	where $(\X,\dist_\X)$ is a metric space and $\al$ is a probability measure on $\X$.
\end{defn}
The general setting corresponds to computing couplings between metric-measure spaces $\XX=(\X,\dist_\X,\al)$ and $\YY=(\Y,\dist_\Y,\be)$, where the distance and the measure are both part of the data.
The natural setting is that of Polish metric spaces; compactness is often assumed in this section to avoid existence and integrability issues.
One defines
	\begin{align}
		\label{eq-gw-generic}
		\GW( \XX,\YY )^p \eqdef
		\umin{ \pi \in \Couplings(\al,\be) }
		\int_{\X^2 \times \Y^2}
		 \De(\dist_\X(x,x'),\dist_\Y(y,y'))^p
		\d\pi(x,y)\d\pi(x',y').
	\end{align}

\begin{prop}[Euclidean GW is controlled by Wasserstein]\label{prop-gw-controlled-by-wasserstein}
	Let $\alpha,\beta$ be probability measures on $\RR^d$, equipped with the Euclidean distance, and take $\De(u,v)=|u-v|$ in~\eqref{eq-gw-generic}. Then
\index{probability measure}
	\[
		\GW((\RR^d,\norm{\cdot},\alpha),(\RR^d,\norm{\cdot},\beta))
		\leq
		2\Wass_p(\alpha,\beta).
	\]
\end{prop}
\begin{proof}
	Let $\pi$ be any coupling between $\alpha$ and $\beta$. For two independent pairs $(X,Y),(X',Y')\sim\pi$, the reverse triangle inequality gives
\index{triangle inequality}
	\[
		\big|\norm{X-X'}-\norm{Y-Y'}\big|
		\leq
		\norm{X-Y}+\norm{X'-Y'}.
	\]
	Taking the $L^p$ norm and using Minkowski gives a bound by $2(\int\norm{x-y}^p\d\pi)^{1/p}$. Optimizing over $\pi$ proves the claim.
\index{Minkowski inequality}
\end{proof}

To turn GW from a distortion score into a metric statement, one must quotient out the relabelings that preserve both distances and mass.

\begin{defn}[Isometric metric-measure spaces]\label{def-isometric-mm-spaces}
\index{metric-measure space}
\index{measure-preserving!isometry}
	Two metric-measure spaces $\XX=(\X,\dist_\X,\al)$ and $\YY=(\Y,\dist_\Y,\be)$ are isometric if there exists a measurable map $\phi:\supp(\al)\to\supp(\be)$ such that $\phi_\sharp\al=\be$, $\phi(\supp(\al))=\supp(\be)$, and
	\[
		\dist_\Y(\phi(x),\phi(x'))=\dist_\X(x,x')
	\]
	for all $x,x'\in\supp(\al)$.
\end{defn}

The next theorem explains why the averaged distortion above is not merely a matching score. Once one quotients out measure-preserving isometries, it defines a genuine distance between metric-measure spaces.
\index{distortion}
\index{measure-preserving!isometry}
\index{metric-measure space}

\begin{thm}[Gromov--Wasserstein metric modulo isometries]\label{thm-gw-metric}
\index{Gromov-Wasserstein}
\index{Gromov-Wasserstein!metric}
\index{modulo isometry}
	For compact metric-measure spaces, $p\geq1$ and $\De(u,v)=|u-v|$, $\GW$ defines a distance up to measure-preserving isometries.
\index{measure-preserving!isometry}
\end{thm}

\begin{proof}
	If $\GW(\XX,\YY)=0$ and $\pi$ is an optimal plan, then $\dist_\X(x,x')=\dist_\Y(y,y')$ holds $\pi\otimes\pi$-almost everywhere. By continuity, this equality holds on $\supp(\pi)^2$. We show that $\XX$ and $\YY$ are isometric by showing that both are isometric to the support space $(\supp(\pi),\dist_\pi,\pi)$, where
\index{optimal plan}
	\[
		\dist_\pi((x,y),(x',y'))
		\eqdef
		\frac12\dist_\X(x,x')+\frac12\dist_\Y(y,y').
	\]
	The first projection $\psi:(x,y)\mapsto x$ is measure-preserving. For $((x,y),(x',y'))\in\supp(\pi)^2$,
	\[
		\dist_\X(\psi(x,y),\psi(x',y'))
		=
		\dist_\X(x,x')
		=
		\dist_\Y(y,y')
		=
		\dist_\pi((x,y),(x',y')),
	\]
	so $\psi$ is an isometry and therefore injective. To see surjectivity onto $\supp(\al)$, take $x\in\supp(\al)$. Since $\psi_\sharp\pi=\al$, there is a sequence $(x_k,y_k)\in\supp(\pi)$ with $x_k\to x$. The equality of distances on $\supp(\pi)$ makes $(y_k)_k$ Cauchy, and compactness gives a convergent subsequence with limit $(x,y)\in\supp(\pi)$. The same argument for the second projection shows that the support space is also isometric to $\YY$.

	For the triangle inequality, let $\pi$ be an optimal coupling between $\XX$ and $\YY$, and $\xi$ an optimal coupling between $\YY$ and $\ZZ=(\Z,\dist_\Z,\ga)$. By the gluing lemma, take $\sigma$ on $\X\times\Y\times\Z$ whose $(\X,\Y)$ and $(\Y,\Z)$ marginals are $\pi$ and $\xi$. Let $\rho=(P_{\X,\Z})_\sharp\sigma$, and write $\bar\sigma=\sigma\otimes\sigma$ for the product law of two independent triples $(x,y,z)$ and $(x',y',z')$. Then $\rho$ is feasible between $\XX$ and $\ZZ$, and
\index{triangle inequality}
\index{optimal coupling}
\index{gluing lemma}
	\begin{align*}
		\GW(\XX,\ZZ)
		&\leq
		\left(
		\int
		\abs{\dist_\X(x,x')-\dist_\Z(z,z')}^p
		\d\bar\sigma
		\right)^{1/p} \\
		&\leq
		\left(
		\int
		\abs{\dist_\X(x,x')-\dist_\Y(y,y')}^p
		\d\bar\sigma
		\right)^{1/p}
		+
		\left(
		\int
		\abs{\dist_\Y(y,y')-\dist_\Z(z,z')}^p
		\d\bar\sigma
		\right)^{1/p} \\
		&=
		\GW(\XX,\YY)+\GW(\YY,\ZZ),
	\end{align*}
		where the second inequality uses the pointwise triangle inequality followed by Minkowski's inequality. Symmetry and non-negativity are immediate.
\index{Minkowski inequality}
\end{proof}

The metric structure also gives geodesics. Sturm's construction is useful conceptually because it allows one to speak about interpolation, barycenters and gradient flows directly on the space of metric-measure spaces, even though the intermediate space lives on a product support and is therefore expensive numerically~\cite{SturmGW}.
\index{metric-measure space}
\index{gradient!flow}

\begin{prop}[Gromov--Wasserstein geodesics]\label{prop-gw-geodesics}
\index{Wasserstein!geodesic}
\index{Gromov-Wasserstein}
\index{Gromov-Wasserstein!geodesic}
	Let $\XX_0=(\X_0,\dist_{\X_0},\al_0)$ and $\XX_1=(\X_1,\dist_{\X_1},\al_1)$ be compact metric-measure spaces, take $\De(u,v)=|u-v|$ in~\eqref{eq-gw-generic}, and let $\pi^\star$ be an optimal coupling. Define, on $\Z=\X_0\times\X_1$,
\index{optimal coupling}
	\[
		\dist_t((x_0,x_1),(x'_0,x'_1))
		\eqdef
		(1-t)\dist_{\X_0}(x_0,x'_0)+t\dist_{\X_1}(x_1,x'_1),
		\qquad
		\XX_t=(\Z,\dist_t,\pi^\star).
	\]
	At $t=0$ and $t=1$, and possibly in degenerate cases, one quotients $\Z$ by the zero-distance relation associated with $\dist_t$. Then $t\mapsto\XX_t$ is a constant-speed geodesic:
\index{constant-speed geodesic}
	\[
		\GW(\XX_s,\XX_t)=|t-s|\GW(\XX_0,\XX_1)
		\qquad
		\forall s,t\in[0,1].
	\]
\end{prop}

\begin{proof}
	Write $D=\GW(\XX_0,\XX_1)$. For $s<t$, couple $\XX_s$ and $\XX_t$ by the diagonal coupling induced by the identity on $\Z$ and the measure $\pi^\star$. For two independent points $z=(x_0,x_1)$ and $z'=(x'_0,x'_1)$ sampled from $\pi^\star$,
	\[
		\dist_t(z,z')-\dist_s(z,z')
		=
		(t-s)\big(\dist_{\X_1}(x_1,x'_1)-\dist_{\X_0}(x_0,x'_0)\big).
	\]
	Using this feasible coupling gives $\GW(\XX_s,\XX_t)\leq(t-s)D$. The same construction with the projections from $\Z$ to $\X_0$ and $\X_1$ gives $\GW(\XX_0,\XX_t)\leq tD$ and $\GW(\XX_t,\XX_1)\leq(1-t)D$. The triangle inequality for $\GW$ then yields, for $0\leq s\leq t\leq1$,
\index{triangle inequality}
\index{feasible coupling}
	\[
		D
		\leq
		\GW(\XX_0,\XX_s)+\GW(\XX_s,\XX_t)+\GW(\XX_t,\XX_1)
		\leq
		sD+\GW(\XX_s,\XX_t)+(1-t)D,
	\]
	so $\GW(\XX_s,\XX_t)\geq(t-s)D$. This proves equality.
\end{proof}

\begin{figure}[ht]
\centering
\begin{tabular}{@{}cc@{}}
\includegraphics[width=.43\linewidth]{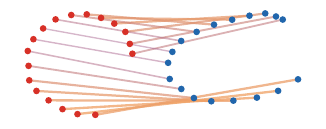} &
\includegraphics[width=.23\linewidth]{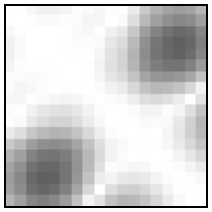}
\\[-.1em]
\small hard correspondence $\sigma$ &
\small residual $|d_\X-d_\Y\circ\sigma|$
\end{tabular}
\caption{Local distortion in a mildly non-isometric GW match. The left panel colors transport segments by the average residual induced by the displayed hard correspondence. The right panel shows the pairwise-distance residual matrix $|d_\X(x_i,x_{i'})-d_\Y(y_{\sigma(i)},y_{\sigma(i')})|$ in white-to-black scale, with darker entries marking larger local distortion. This matrix is the local contribution minimized by the discrete GW objective for the displayed correspondence.}
\index{distance residual}
\index{distortion}
\index{correspondence}
\label{fig:gromov-nonisometric-distortion}
\end{figure}

We now record the profile lower bound used above as a useful initialization principle for the non-convex solver.
\index{profile lower bound}

\begin{prop}[M\'emoli profile lower bound]\label{prop-memoli-gw-profile-lower-bound}
\index{Memoli profile}
	Let $\XX=(\X,\dist_\X,\alpha)$ and $\YY=(\Y,\dist_\Y,\beta)$ be compact metric-measure spaces and take $\De(u,v)=|u-v|$ in~\eqref{eq-gw-generic}, with the same exponent $p\geq1$. For each $x\in\X$ and $y\in\Y$, define the distance-profile measures on $\RR_+$ by
\index{metric-measure space}
	\[
		\alpha_x\eqdef(\dist_\X(x,\cdot))_\sharp\alpha,
		\qquad
		\beta_y\eqdef(\dist_\Y(y,\cdot))_\sharp\beta.
	\]
	Let $\mathsf E_\XX=(x\mapsto\alpha_x)_\sharp\alpha$ and $\mathsf E_\YY=(y\mapsto\beta_y)_\sharp\beta$, which are probability measures on $\Pp(\RR_+)$. Then
\index{probability measure}
	\[
		\Wass_p\bigl(\mathsf E_\XX,\mathsf E_\YY\bigr)
		\leq
		\GW(\XX,\YY).
	\]
	Here the left-hand distance is taken on the space $\Pp(\RR_+)$ of profile measures.
	Its ground cost is the one-dimensional Wasserstein distance $\Wass_p$.
\index{ground cost}
\index{Wasserstein!distance}
\end{prop}
\begin{proof}
	Fix any $\pi\in\Couplings(\alpha,\beta)$. It induces a coupling $(x,y)\mapsto(\alpha_x,\beta_y)$ between $\mathsf E_\XX$ and $\mathsf E_\YY$, hence
	\[
		\Wass_p\bigl(\mathsf E_\XX,\mathsf E_\YY\bigr)^p
		\leq
		\int_{\X\times\Y}\Wass_p(\alpha_x,\beta_y)^p\,\d\pi(x,y).
	\]
	For fixed $(x,y)$, the map $(x',y')\mapsto(\dist_\X(x,x'),\dist_\Y(y,y'))$ pushes the same coupling $\pi$ to a coupling between $\alpha_x$ and $\beta_y$. Therefore
	\[
		\Wass_p(\alpha_x,\beta_y)^p
		\leq
		\int_{\X\times\Y}
		\abs{\dist_\X(x,x')-\dist_\Y(y,y')}^p
		\d\pi(x',y').
	\]
	Integrating in $(x,y)$ gives
	\[
		\Wass_p\bigl(\mathsf E_\XX,\mathsf E_\YY\bigr)^p
		\leq
		\int_{\X^2\times\Y^2}
		\abs{\dist_\X(x,x')-\dist_\Y(y,y')}^p
		\d\pi(x,y)\d\pi(x',y').
	\]
	Taking the infimum over $\pi$ and then the $p$-th root proves the claim.
\end{proof}

This lower bound is useful computationally because the profile cost matrix $C_{ij}=\Wass_p(\alpha_{x_i},\beta_{y_j})^p$ is an ordinary OT cost between points. Solving this easier OT problem gives a geometry-aware initialization for the non-convex GW iterations below, before the full pairwise-distance distortion is optimized.
\index{distortion}
\index{cost matrix}

\paragraph{Entropic regularization and iterative solver.}
\index{entropic!regularization}

For the common squared distortion $\De(u,v)^2=(u-v)^2$, one often computes a stationary point of the entropic relaxation
\index{distortion}
\index{stationarity}
\eql{\label{eq-gw-entropy}
	\umin{\P\in\CouplingsD(\a,\b)}
	\Ee_{\distD,\distD'}(\P)-\epsilon\HD(\P).
}
Although the objective is non-convex, successive linearizations lead to a practical mirror-descent scheme~\cite{peyre2016gromov}. Up to an irrelevant global factor in the gradient, one alternates
\index{linear!linearization}
\eql{\label{eq-gw-sinkh}
	\itt{\P}
	\eqdef
	\umin{\P\in\CouplingsD(\a,\b)}
	\dotp{\P}{\it{\C}}-\epsilon\HD(\P),
	\qquad
	\it{\C}
	\eqdef
	\distD^{\odot2}\a\,\ones_m^\top
	+
	\ones_n(\distD'^{\odot2}\b)^\top
	-
	2\distD\,\it{\P}\,\transp{\distD'}.
}
Each update is an ordinary entropic OT problem and can therefore be solved with Sinkhorn iterations. This is the standard entropic GW solver used to compute soft maps between domains; it improves scalability and smooths the landscape, but it does not remove the non-convexity of the GW objective.
\index{Sinkhorn!iteration}
\index{entropic!OT}

\begin{alg}[Entropic Gromov--Wasserstein linearization]\label{alg:entropic-gromov-wasserstein}
\textbf{Input:} Metric matrices $\distD,\distD'$, weights $\a,\b$, regularization $\epsilon>0$, tolerance $\mathrm{tol}$.

\textbf{Output:} Approximate entropic GW coupling $\P\in\CouplingsD(\a,\b)$.

\textbf{Initialize:} Set $\P^{(0)}=\a\otimes\b$.

\textbf{For} $k=0,1,\ldots$ \textbf{do}:
\begin{algblock}
\(\C^{(k)} = \distD^{\odot2}\a\,\ones_m^\top + \ones_n(\distD'^{\odot2}\b)^\top - 2\distD\,\P^{(k)}\,\transp{\distD'}.\)

\textbf{Solve} entropic OT subproblem:
\(\P^{(k+1)} = \uargmin{\P\in\CouplingsD(\a,\b)} \dotp{\P}{\C^{(k)}}-\epsilon\HD(\P).\)

\textbf{If} $\norm{\P^{(k+1)}-\P^{(k)}}\leq\mathrm{tol}$ \textbf{then}:
\begin{algblock}
\textbf{Return} $\P^{(k+1)}$.
\end{algblock}
\end{algblock}
\end{alg}

\paragraph{Adding features.}
\index{cost!feature}
Fused Gromov--Wasserstein augments the structural term with a feature transport cost~\cite{vayer2019optimaltransportstructured}. In the discrete case, given a cross-feature cost $M\in\RR^{n\times m}$ and a parameter $\lambda\in[0,1]$, one minimizes
\index{Gromov-Wasserstein}
\index{Gromov-Wasserstein!fused}
\index{cost!feature}
\[
	\operatorname{FGW}_{\lambda,p}((\a,\distD),(\b,\distD'))^p
	\eqdef
	\umin{\P\in\CouplingsD(\a,\b)}
	(1-\lambda)\sum_{i,j}M_{ij}\P_{ij}
	+\lambda
	\sum_{i,j,i',j'}\De(\distD_{ii'},\distD'_{jj'})^p\P_{ij}\P_{i'j'}.
\]
The first term compares node attributes in the usual OT sense, while the second compares the intrinsic geometry. The endpoints $\lambda=0$ and $\lambda=1$ recover feature-only OT and pure GW respectively; intermediate values trade attribute matching against structural matching. This is useful when the two spaces have both distances and features, and these two sources of information may disagree.

\begin{figure}[ht]
\centering
\begin{tabular}{@{}ccc@{}}
\small feature-only OT &
\small fused GW &
\small pure GW \\[-.15em]
\includegraphics[width=.30\linewidth]{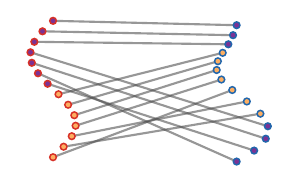} &
\includegraphics[width=.30\linewidth]{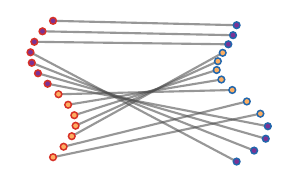} &
\includegraphics[width=.30\linewidth]{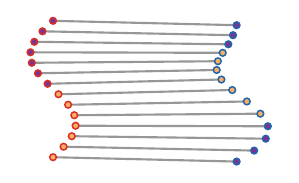}
\end{tabular}
\caption{Feature information and intrinsic geometry in fused Gromov--Wasserstein. Small inner disks encode binary node features. Feature-only OT follows the attributes even when this crosses the shape structure, pure GW follows the intrinsic ordering, and fused GW balances the feature term with the pairwise-distance distortion.}
\index{feature term}
\index{node feature}
\index{distortion}
\index{Gromov-Wasserstein!fused}
\label{fig:fused-gromov-feature-geometry}
\end{figure}

\paragraph{Hausdorff and Gromov--Hausdorff viewpoints.}
\index{Hausdorff distance}
\index{Gromov-Hausdorff distance}

If $A,B$ are compact subsets of a common metric space $(\Z,\dist_\Z)$, their Hausdorff distance is
\index{Hausdorff distance}
\[
	d_{\mathrm H}^{\Z}(A,B)
	=
	\max\left\{
		\sup_{a\in A}\inf_{b\in B}\dist_\Z(a,b),
		\sup_{b\in B}\inf_{a\in A}\dist_\Z(a,b)
	\right\}.
\]
The Gromov--Hausdorff distance removes the common ambient space by minimizing this quantity over all isometric embeddings into a third space:
\index{isometric embedding}
\index{Hausdorff distance}
\index{Gromov-Hausdorff distance}
\[
	d_{\mathrm{GH}}(\X,\Y)
	=
	\inf_{\Z,\phi,\psi}
	d_{\mathrm H}^{\Z}(\phi(\X),\psi(\Y)).
\]
Equivalently, it is half the minimal distortion of a correspondence between $\X$ and $\Y$~\cite{gromov-2001,memoli-2007}. This is a worst-case set distance: every point must be matched with small distortion. Gromov--Wasserstein replaces correspondences by probability couplings and worst-case distortion by averaged distortion. It is therefore better adapted to noisy sampled shapes and weighted graphs, but it can ignore small sets of mass that would dominate the Hausdorff distance.
\index{distortion}
\index{correspondence}
\index{Gromov-Wasserstein}
\index{Hausdorff distance}

\section{Quantum Optimal Transport}
\index{quantum!optimal transport}
\label{sec-quantum-ot}

Quantum optimal transport replaces probability vectors by density matrices and scalar couplings by positive operators on a tensor product space. This is the right language when the transported objects are matrix-valued signals, covariance-like descriptors or quantum states, and it exposes a precise bridge between OT, non-commutative entropy and operator scaling. The finite-dimensional formulation below follows the semidefinite viewpoint developed in matrix-valued and quantum OT~\cite{Ning2014metrics,Chen2016,ChenGangbo17,2016-peyre-qot,caglioti2019quantum,chakrabarti2019quantum}.
\index{positive operator}
\index{quantum!OT}
\index{operator scaling}

\paragraph{Finite-dimensional states and couplings.}
\index{density!matrix}
\index{quantum!coupling}

\begin{defn}[Hermitian and density matrices]\label{def-hermitian-density-matrices}
\index{Hermitian matrix}
\index{density!matrix}
	Let $\mathbb H_n$ be the real vector space of $n\times n$ Hermitian matrices,
	\[
		\mathbb H_n^+
		=
		\enscond{A\in\mathbb H_n}{A\succeq0},
		\qquad
		\mathbb H_n^{+,1}
		=
		\enscond{A\in\mathbb H_n^+}{\tr(A)=1}.
	\]
	Elements of $\mathbb H_n^{+,1}$ are density matrices.
\end{defn}
A joint quantum state between $\CC^n$ and $\CC^m$ is a matrix $T\in\mathbb H_{nm}^+$ acting on $\CC^n\otimes\CC^m$. Its marginals are the partial traces, defined by duality through
\index{density!matrix}
\index{trace!partial}
\begin{equation}\label{eq-qot-partial-traces}
	\tr(F\,\operatorname{Tr}_B T)=\tr((F\otimes\Id_m)T),
	\qquad
	\tr(G\,\operatorname{Tr}_A T)=\tr((\Id_n\otimes G)T),
\end{equation}
for all $F\in\mathbb H_n$ and $G\in\mathbb H_m$. Thus $\operatorname{Tr}_B(T)\in\mathbb H_n^+$ and $\operatorname{Tr}_A(T)\in\mathbb H_m^+$ play exactly the role of the two marginals of a classical coupling. The feasible set is never empty, since $A\otimes B$ has marginals $A$ and $B$.

\begin{defn}[Finite-dimensional quantum OT]\label{def-finite-dimensional-qot}
\index{quantum!OT}
	Let $A\in\mathbb H_n^{+,1}$, $B\in\mathbb H_m^{+,1}$ and let $C\in\mathbb H_{nm}$ be a Hermitian cost observable. The quantum OT value is the semidefinite program
\index{semidefinite program}
\index{Hermitian matrix}
	\begin{equation}\label{eq-qot-primal}
		\mathrm{QOT}_C(A,B)
		\eqdef
		\min_{T\in\mathbb H_{nm}^+}
		\enscond{\tr(CT)}{\operatorname{Tr}_B(T)=A,\ \operatorname{Tr}_A(T)=B}.
	\end{equation}
\end{defn}

\begin{example}[Classical diagonal case]\label{rem-qot-classical-diagonal-case}
	If $A$, $B$, $C$ and $T$ are all diagonal in fixed bases, then $A$ and $B$ are probability vectors, $T$ is a nonnegative matrix and the partial-trace constraints reduce to the usual row and column sum constraints. Hence classical Kantorovich OT is the diagonal, commutative subcase of~\eqref{eq-qot-primal}. The genuinely quantum feature is that $T$ may contain off-diagonal coherences and entanglement.
\index{trace!partial constraint}
\end{example}

\begin{prop}[Quantum Kantorovich duality]\label{prop-qot-duality}
\index{Kantorovich!duality}
\index{quantum!Kantorovich duality}
	For $A\in\mathbb H_n^{+,1}$ and $B\in\mathbb H_m^{+,1}$, the dual of~\eqref{eq-qot-primal} is
	\begin{equation}\label{eq-qot-dual}
		\mathrm{QOT}_C(A,B)
		=
		\max_{F\in\mathbb H_n,\ G\in\mathbb H_m}
		\left\{
			\tr(FA)+\tr(GB)
			:\,
			F\otimes\Id_m+\Id_n\otimes G\preceq C
		\right\}.
	\end{equation}
	If $A$ and $B$ are positive definite, strong duality follows directly from Slater's condition; the semidefinite case follows by restriction to the supports of $A$ and $B$ or by approximation.
\index{duality!strong}
\index{Slater condition}
\end{prop}

\begin{proof}
	Introduce Hermitian Lagrange multipliers $F$ and $G$ for the two marginal constraints. The Lagrangian is
\index{Hermitian matrix}
\index{marginal!constraint}
	\[
		\tr(CT)+\tr\!\left(F(A-\operatorname{Tr}_B T)\right)
		+\tr\!\left(G(B-\operatorname{Tr}_A T)\right)
		=
		\tr(FA)+\tr(GB)
		+\tr\!\left((C-F\otimes\Id_m-\Id_n\otimes G)T\right),
	\]
	where~\eqref{eq-qot-partial-traces} was used in the last equality. Minimizing over $T\succeq0$ gives a finite lower bound if and only if $C-F\otimes\Id_m-\Id_n\otimes G\succeq0$, in which case the infimum in $T$ is $0$. This gives the dual program. When $A,B\succ0$, the coupling $A\otimes B$ is strictly feasible, so Slater's theorem gives equality of primal and dual values and dual attainment. The general finite-dimensional semidefinite case is obtained by approximation $A_\delta=(1-\delta)A+\delta\Id_n/n$, $B_\delta=(1-\delta)B+\delta\Id_m/m$ and by compactness, or equivalently by reducing to the supports of $A$ and $B$.
\index{Slater condition}
\index{dual!attainment}
\index{trace!partial}
\end{proof}

The dual potentials have the usual scalar gauge freedom: replacing $(F,G)$ by $(F+t\Id_n,G-t\Id_m)$ leaves both the constraint and the value unchanged because $\tr(A)=\tr(B)=1$.
\index{dual!potential}

\paragraph{Entropic regularization and Bregman iterations.}
\index{entropic!regularization}
\index{Bregman!iteration}
As in scalar OT, one obtains a smoother problem by adding the convex quantum entropy functional associated with the matrix logarithm.

\begin{defn}[von Neumann quantum entropy]\label{def-von-neumann-quantum-entropy}
\index{von Neumann entropy}
	For a density matrix or positive semidefinite matrix $T$, the shifted von Neumann entropy functional used here is
	\[
		H(T)=\tr\!\left(T(\log T-\Id)\right),
		\qquad
		\nabla H(T)=\log T,
	\]
	with the convention $0\log0=0$ on eigenvalues. This is the convex negative quantum entropy; on trace-one states it differs from the physical entropy $-\tr(T\log T)$ by a sign and an additive constant.
\end{defn}
For $\epsilon>0$ define
\begin{equation}\label{eq-qot-entropic-primal}
	\mathrm{QOT}_C^\epsilon(A,B)
	=
	\min_{T\succeq0}
	\left\{
		\tr(CT)+\epsilon H(T)
		:\,
		\operatorname{Tr}_B(T)=A,
		\operatorname{Tr}_A(T)=B
	\right\}.
\end{equation}
This is the non-commutative analogue of entropic OT~\cite{2016-peyre-qot,chakrabarti2019quantum}: the Shannon entropy of a coupling is replaced by the trace entropy of a density matrix.
\index{trace entropy}
\index{entropic!OT}
\index{Shannon!entropy}
\index{density!matrix}

\begin{prop}[Entropic quantum OT duality]\label{prop-qot-entropic-duality}
\index{quantum!OT}
\index{quantum!OT duality}
	Assume $A\succ0$, $B\succ0$ and $\epsilon>0$. Then~\eqref{eq-qot-entropic-primal} has a unique positive minimizer. Its dual is
	\begin{equation}\label{eq-qot-entropic-dual}
		\mathrm{QOT}_C^\epsilon(A,B)
		=
		\max_{F\in\mathbb H_n,\ G\in\mathbb H_m}
		\left\{
			\tr(FA)+\tr(GB)
			-\epsilon\,
			\tr\exp\!\left(\frac{F\otimes\Id_m+\Id_n\otimes G-C}{\epsilon}\right)
		\right\}.
	\end{equation}
	At optimality, primal and dual variables are linked by the Gibbs formula
\index{Gibbs!formula}
	\begin{equation}\label{eq-qot-gibbs-coupling}
		T_e(F,G)
		=
		\exp\!\left(\frac{F\otimes\Id_m+\Id_n\otimes G-C}{\epsilon}\right),
	\end{equation}
	with $\operatorname{Tr}_B(T_e)=A$ and $\operatorname{Tr}_A(T_e)=B$.
\end{prop}

\begin{proof}
	The feasible set is compact and nonempty, and it contains the positive definite point $A\otimes B$. The trace entropy is strictly convex on positive semidefinite matrices, hence the regularized primal has a unique minimizer. Slater's condition justifies the Lagrange dual computation. The Fenchel identity
\index{trace entropy}
\index{Slater condition}
\index{positive!semidefinite matrix}
	\[
		\sup_{T\succeq0}\ \tr(YT)-\epsilon H(T)
		=
		\epsilon\,\tr\exp(Y/\epsilon)
	\]
	is the matrix analogue of the scalar exponential conjugacy. Applying it to the Lagrangian of~\eqref{eq-qot-entropic-primal}, with
	$Y=F\otimes\Id_m+\Id_n\otimes G-C$, gives~\eqref{eq-qot-entropic-dual}. The stationarity condition of this Fenchel identity gives~\eqref{eq-qot-gibbs-coupling}; differentiating the dual objective with respect to $F$ and $G$ yields the two marginal equations.
\index{stationarity}
\index{Gibbs!coupling}
\end{proof}

Writing $K=\exp(-C/\epsilon)$, the objective in~\eqref{eq-qot-entropic-primal} differs by a constant from $\epsilon$ times the quantum KL divergence
\[
	D_H(T|K)
	=
	\tr\!\left(T(\log T-\log K)-T+K\right).
\]
The exact quantum analogue of Sinkhorn is an implicit alternating Bregman projection scheme onto the affine marginal sets
\index{Bregman!projection}
\[
	\mathcal M_A=\{T\succeq0:\operatorname{Tr}_B(T)=A\},
	\qquad
	\mathcal M_B=\{T\succeq0:\operatorname{Tr}_A(T)=B\}.
\]

\begin{prop}[Exact Bregman projections]\label{prop-qot-bregman-projections}
\index{Bregman!projection}
	Assume $A,B\succ0$ and let $K=\exp(-C/\epsilon)$. The minimizer of~\eqref{eq-qot-entropic-primal} is equivalently the minimizer of $D_H(T|K)$ over $\mathcal M_A\cap\mathcal M_B$. Moreover, if a current positive definite matrix has Gibbs form $T_e(F,G)$, then its Bregman projection onto $\mathcal M_A$ has the form $T_e(F^+,G)$, where $F^+$ is chosen so that $\operatorname{Tr}_B T_e(F^+,G)=A$. The projection onto $\mathcal M_B$ is analogous. Thus, when each one-block marginal equation is solved exactly, alternating Bregman projections are equivalent to alternating block maximization of the dual~\eqref{eq-qot-entropic-dual}.
\end{prop}

\begin{proof}
	Since $\log K=-C/\epsilon$, the identity
	\[
		\tr(CT)+\epsilon H(T)
		=
		\epsilon D_H(T|K)-\epsilon\tr(K)
	\]
	holds, so the primal minimizer is the constrained Bregman projection of $K$ up to an additive constant. For the projection of a positive definite matrix $S$ onto $\mathcal M_A$, the affine set contains the positive definite point $A\otimes\Id_m/m$; the entropy derivative $\log T-\log S$ is singular at the boundary, so the projection lies in the interior of the positive cone. Its Lagrangian first variation is
\index{first variation}
\index{Bregman!projection}
\index{positive!cone}
	\[
		\log T-\log S-\Lambda\otimes\Id_m=0
	\]
	for a Hermitian multiplier $\Lambda$. Hence $T=\exp(\log S+\Lambda\otimes\Id_m)$. If $S=T_e(F,G)$, this is again of the form $T_e(F+\epsilon\Lambda,G)$. The multiplier is fixed by the marginal equation $\operatorname{Tr}_B(T)=A$. The same argument applies to $\mathcal M_B$. Finally, the first-order optimality condition for maximizing~\eqref{eq-qot-entropic-dual} over one block is exactly the corresponding marginal equation, so the Bregman and block-dual views coincide.
\index{Hermitian matrix}
\index{optimality!first-order}
\end{proof}

In the diagonal case this proposition gives the usual multiplicative Sinkhorn updates. In the non-commutative case, however, the exact block equations
\index{Sinkhorn!update}
\[
	\operatorname{Tr}_B T_e(F,G)=A,
	\qquad
	\operatorname{Tr}_A T_e(F,G)=B
\]
do not admit scalar division formulas, because the exponential of $F\otimes\Id_m+\Id_n\otimes G-C$ cannot be separated unless the local potential $F\otimes\Id_m+\Id_n\otimes G$ commutes with the cost $C$.
When all matrices are diagonal in the same basis, commutativity restores the scalar form $T_e=\diag(u)K\diag(v)$ and the marginal equations reduce to the usual Sinkhorn divisions.

Algorithm~\ref{alg:quantum-exact-bregman} records this exact but implicit non-commutative analogue of Sinkhorn.

\begin{alg}[Exact quantum Bregman projections]\label{alg:quantum-exact-bregman}
\index{quantum!Bregman projection}
\textbf{Input:} Density matrices $A,B$, cost $C$, regularization $\epsilon>0$, tolerance $\mathrm{tol}$.

\textbf{Output:} Quantum entropic coupling $T$ with partial traces $A$ and $B$.

\textbf{Initialize:} Set Hermitian potentials $F^{(0)}=0$ and $G^{(0)}=0$.

\textbf{For} $k=0,1,\ldots$ \textbf{do}:
\begin{algblock}
\(T^{(k)}= T_e(F^{(k)},G^{(k)}) = \exp\!\left( \frac{F^{(k)}\otimes\Id_m+\Id_n\otimes G^{(k)}-C}{\epsilon} \right).\)

\textbf{Solve} $A$-projection equation:
\(\operatorname{Tr}_B T_e(F^+,G^{(k)})=A,\)

\textbf{Set} $F^{(k+1)}=F^+$.

\textbf{Solve} $B$-projection equation:
\(\operatorname{Tr}_A T_e(F^{(k+1)},G^+)=B,\)

\textbf{Set} $G^{(k+1)}=G^+$.

\textbf{If} both partial-trace residuals are at most $\mathrm{tol}$ \textbf{then}:
\begin{algblock}
\textbf{Return} $T_e(F^{(k+1)},G^{(k+1)})$.
\end{algblock}
\end{algblock}
\end{alg}

\paragraph{Gurvits scaling and quantum Sinkhorn.}
\index{Gurvits scaling}
\index{quantum!Sinkhorn}
The algorithm often called quantum Sinkhorn comes from the operator-scaling literature of Gurvits and subsequent developments~\cite{gurvits2003classical,gurvits2004classical,georgiou2015positive,garg2018recent}. It replaces the true Gibbs coupling~\eqref{eq-qot-gibbs-coupling} by the symmetric factorization
\index{operator scaling}
\index{Gibbs!coupling}
\begin{equation}\label{eq-qot-symmetric-scaling}
	T_s(F,G)
	=
	\exp\!\left(\frac{Z}{2\epsilon}\right)\exp(-C/\epsilon)\exp\!\left(\frac{Z}{2\epsilon}\right)
	=
	(U\otimes V)K(U\otimes V),
	\qquad
	Z=F\otimes\Id_m+\Id_n\otimes G,
\end{equation}
where $U=\exp(F/(2\epsilon))$, $V=\exp(G/(2\epsilon))$ and $K=\exp(-C/\epsilon)$. If $[Z,C]=0$, then $T_s(F,G)=T_e(F,G)$; otherwise this is a Strang-type symmetric surrogate.
\index{Strang splitting}

Fix a Choi convention and let $\mathcal K:\mathbb H_m\to\mathbb H_n$ be the completely positive map represented by the positive Choi matrix $K$; let $\mathcal K^\star$ be its Hilbert--Schmidt adjoint. Up to the transpose dictated by the chosen Choi convention, the marginal equations for the symmetric coupling take the operator-scaling form
\index{Choi matrix}
\index{operator scaling}
\index{scaling!form}
\[
	U\,\mathcal K(V^2)\,U=A,
	\qquad
	V\,\mathcal K^\star(U^2)\,V=B.
\]
They can be enforced by the explicit congruence normalizations
\index{congruence normalization}
\begin{equation}\label{eq-qot-gurvits-updates}
\begin{aligned}
	R_V&=\mathcal K(V^2),
	&
	U&\leftarrow
	R_V^{-1/2}\bigl(R_V^{1/2} A R_V^{1/2}\bigr)^{1/2}R_V^{-1/2},
	\\
	S_U&=\mathcal K^\star(U^2),
	&
	V&\leftarrow
	S_U^{-1/2}\bigl(S_U^{1/2} B S_U^{1/2}\bigr)^{1/2}S_U^{-1/2}.
\end{aligned}
\end{equation}
\begin{alg}[Gurvits/operator scaling for quantum Sinkhorn]\label{alg:quantum-gurvits-scaling}
\textbf{Input:} Positive marginals $A,B$, positive kernel operator $K$, maps $\mathcal K,\mathcal K^\star$, tolerance $\mathrm{tol}$.

\textbf{Output:} Symmetrically scaled coupling $T_s$.

\textbf{Initialize:} Set $U=\Id_n$ and $V=\Id_m$.

\textbf{Set} residual \(r=+\infty\).

\textbf{While} \(r>\mathrm{tol}\) \textbf{do}:
\begin{algblock}
\(R_V=\mathcal K(V^2), \qquad U\leftarrow R_V^{-1/2}\bigl(R_V^{1/2} A R_V^{1/2}\bigr)^{1/2}R_V^{-1/2}.\)

\(S_U=\mathcal K^\star(U^2), \qquad V\leftarrow S_U^{-1/2}\bigl(S_U^{1/2} B S_U^{1/2}\bigr)^{1/2}S_U^{-1/2}.\)

\textbf{Set} \(T_s=(U\otimes V)K(U\otimes V)\) and \(r\) to the maximum of its two operator-marginal residuals against $A$ and $B$.
\end{algblock}
\algreturnskip
\textbf{Return} \(T_s\).
\end{alg}
These inverse square roots are well-defined when $K\succ0$ and $U,V,A,B\succ0$. This is Gurvits/operator scaling with prescribed targets; when all matrices are diagonal it reduces to classical Sinkhorn scaling, and when the targets are proportional to identities it matches the usual bistochastic operator-scaling normalization, up to the conventional trace normalization.
\index{inverse square root}
\index{Sinkhorn!scaling}
\index{trace norm}
\index{operator scaling}

\begin{rem}[Gurvits scaling is not the exact Bregman scheme]
\index{Gurvits scaling}
	It is important not to identify~\eqref{eq-qot-gurvits-updates} with the exact Bregman scheme for~\eqref{eq-qot-entropic-primal}. The exact Bregman step would enforce the marginals of $T_e(F,G)=\exp((Z-C)/\epsilon)$ and would be a block maximization of the true concave dual~\eqref{eq-qot-entropic-dual}. Gurvits scaling instead enforces the marginals of the surrogate
\[
	T_s=
	\exp\!\left(\frac{Z}{2\epsilon}\right)
	\exp(-C/\epsilon)
	\exp\!\left(\frac{Z}{2\epsilon}\right).
\]
	The two coincide in the commuting/diagonal regime, but in general the Baker--Campbell--Hausdorff commutator terms do not vanish. The Gurvits iteration should therefore be understood as a tractable symmetric operator-scaling approximation to entropic Q--OT, not as the literal alternating KL projection algorithm.
\index{commutator}
\index{Baker-Campbell-Hausdorff formula}
\index{operator scaling}
\index{KL!projection}
\end{rem}

\begin{rem}[Operator-valued couplings]
\index{operator-valued coupling}
	The same definitions extend formally from matrices to separable Hilbert spaces by replacing density matrices with positive trace-class operators of trace one, observables with bounded self-adjoint operators and~\eqref{eq-qot-partial-traces} with partial traces defined by duality against local bounded observables. If $\Pi(A,B)$ denotes positive trace-class operators with partial traces $A$ and $B$, a bounded cost observable $C$ gives the problem $\inf_{T\in\Pi(A,B)}\Tr(CT)$. For unbounded positive costs one must define the energy through the quadratic form or spectral truncations, and in the entropic case one must ensure that the Gibbs operator $\exp(-C/\epsilon)$ is trace class and that the partial traces of the candidate coupling are well-defined. The matrix formulas above are therefore the clean finite-dimensional core; the operator version adds domain and compactness assumptions rather than a different algebraic structure.
\index{density!matrix}
\index{trace-class operator}
\index{trace!partial}
\end{rem}


\chapter{Dynamic Optimal Transport}
\index{dynamic!optimal transport}
\label{sec-dynamic-optimal-transport}
\label{sec-wasserstein-flows}

Optimal transport becomes especially powerful once distances between measures are seen as actions of moving mass. This chapter first develops the dynamic language: continuity equations describe admissible measure evolutions, while the Benamou--Brenier formula identifies $\Wass_2$ with a least-action principle. These ideas prepare the gradient-flow and generative-model chapters that follow.
\index{continuity equation}
\index{generative model}
\index{Benamou-Brenier}
\index{gradient!flow}

\section{Evolutions over the Space of Measures}
\index{measure!evolution}

We start with the continuity equation because it is the common language for particles, densities and weak measure evolutions. It also makes precise which velocity fields actually move mass.
\index{velocity field}
\index{continuity equation}

\paragraph{Lagrangian and Eulerian descriptions.}
\index{Lagrangian!description}
\index{Eulerian!description}

We consider the evolution $t \mapsto \alpha_t \in \Pp(\RR^d)$. Such an evolution can be described in a ``Lagrangian'' way as the advection of particles along a (time-dependent) vector field $v_t(x)$ in $\RR^d$. At the particle level, this advection is governed by
\index{time-dependent vector field}
\begin{equation}
    \frac{\d x(t)}{\d t} = v_t(x(t)), \label{eq:lagrangian-advection}
\end{equation}
and we write $T_t$ for the associated flow map, so that $T_t(x(0))=x(t)$. The advected measure is then $\alpha_t = (T_t)_\sharp \alpha_0$. For discrete measures, $\alpha_t = \frac{1}{n} \sum_{i=1}^n \delta_{x_i(t)}$, meaning each $x_i(t)$ solves \eqref{eq:lagrangian-advection}.
\index{flow!map}

In the Eulerian description, the same motion is written directly on the evolving measure: the particle ODE becomes the PDE
\index{particle ODE}
\index{Eulerian!description}
\begin{equation}
    \frac{\partial \alpha_t}{\partial t} + \diverg(v_t \alpha_t) = 0. \label{eq:eulerian-advection}
\end{equation}
This PDE is often referred to as the advection equation, the continuity equation, or Liouville's equation when operating over a phase space. It is only a classical PDE when $\alpha_t$ has a smooth density. For general measures, and in particular for empirical measures, it is understood in the weak sense: for any smooth test function $(t,x)\mapsto\varphi(t,x)$ compactly supported in time,
\index{empirical!measure}
\index{continuity equation}
\begin{equation}\label{eq:eulerian-advection-weak}
	\int_0^1\!\int_{\RR^d}
	\left(\partial_t\varphi(t,x)+\dotp{v_t(x)}{\nabla_x\varphi(t,x)}\right)
	\d\alpha_t(x)\d t
	=
	0.
\end{equation}
This equation is obtained from~\eqref{eq:eulerian-advection} by integration by parts. Hence, for smooth positive densities, the classical and weak formulations are equivalent; the weak formulation is useful because it still makes sense for discrete measures whose particles evolve according to~\eqref{eq:lagrangian-advection}.

\begin{prop}[Lagrangian flows solve the continuity equation]\label{prop-lagrangian-flow-continuity}
\index{Lagrangian!flow}
\index{continuity equation}
	Consider a smooth flow $T_t:\RR^d\to\RR^d$ and define $\alpha_t=(T_t)_\sharp\alpha_0$. Define the Eulerian velocity field by
\index{velocity field}
\index{Eulerian!velocity}
	\[
		v_t(T_t(y))=\partial_t T_t(y).
	\]
	Then $(\alpha_t,v_t)$ solves the continuity equation in the weak sense~\eqref{eq:eulerian-advection-weak}.
\index{continuity equation}
	In particular, if $\alpha_0=\frac1n\sum_i\delta_{x_i(0)}$ is empirical, then $\alpha_t=\frac1n\sum_i\delta_{x_i(t)}$ is empirical as well, with particle velocities $\dot x_i(t)=v_t(x_i(t))$.
\end{prop}

\begin{proof}
	Let $\varphi(t,x)$ be a smooth test function vanishing at $t=0$ and $t=1$. Since $\alpha_t=(T_t)_\sharp\alpha_0$,
	\[
		\frac{\d}{\d t}\int \varphi(t,x)\d\alpha_t(x)
		=
		\frac{\d}{\d t}\int \varphi(t,T_t(y))\d\alpha_0(y).
	\]
	The chain rule gives
	\[
		\frac{\d}{\d t}\int \varphi(t,T_t(y))\d\alpha_0(y)
		=
		\int
		\left(
			\partial_t\varphi(t,T_t(y))
			+\dotp{\nabla_x\varphi(t,T_t(y))}{\partial_t T_t(y)}
		\right)
		\d\alpha_0(y).
	\]
	Using the definition of $v_t$ and the push-forward relation, this equals
\index{push-forward}
	\[
		\int
		\left(
			\partial_t\varphi(t,x)+\dotp{\nabla_x\varphi(t,x)}{v_t(x)}
		\right)
		\d\alpha_t(x).
	\]
	Integrating in time and using the boundary values of $\varphi$ gives~\eqref{eq:eulerian-advection-weak}.
\end{proof}

\paragraph{From measure evolutions to vector fields.}
\index{measure!evolution}

For a given evolution $(\alpha_t)_t$, there are typically infinitely many velocity fields $v_t$ satisfying
\index{velocity field}
\begin{equation}\label{eq:inverse-flow}
	\partial_t\alpha_t+\diverg(\alpha_t v_t)=0.
\end{equation}
This non-uniqueness comes from the kernel of the weighted divergence. The linear space of vector fields that leave a measure $\alpha$ invariant is
\[
	\Hh_\alpha=\enscond{v}{\diverg(\alpha v)=0}.
\]
It is usually non-trivial: for instance, if $\alpha$ is an isotropic Gaussian, $\Hh_\alpha$ contains rotational vector fields generated by anti-symmetric matrices.

\paragraph{Dacorogna--Moser inversion.}
\index{Dacorogna-Moser inversion}

Reconstructing particles from an observed density evolution is therefore ill-posed. A simple choice, introduced by Dacorogna and Moser~\cite{DacorognaMoser1990}, is to impose that the flux $\alpha_t v_t$ is a gradient field, leading formally to
\index{flux}
\begin{equation}\label{eq:dacorogna-moser}
	v_t=-\frac{1}{\alpha_t}\nabla\Delta^{-1}(\partial_t\alpha_t),
\end{equation}
with suitable boundary conditions, for instance vanishing at infinity. This formula is useful conceptually but delicate when $\alpha_t$ vanishes, and it does not generally produce a gradient velocity field.
\index{velocity field}
\index{gradient!velocity}

\paragraph{Least-square inversion and gradient structure.}
\index{least-square!inversion}
\index{gradient!structure}

A more robust choice, used implicitly in flow matching, optimal transport and Wasserstein gradient flows, is to select among all admissible velocities the one with smallest kinetic energy:
\index{Wasserstein!gradient}
\index{gradient!flow}
\index{Wasserstein!gradient flow}
\index{flow!matching}
\begin{equation}\label{eq:least-square-field}
	\umin{v}
	\frac12\int_0^1\!\int_{\RR^d}\norm{v_t(x)}^2\,\d\alpha_t(x)\d t
	\quad
	\text{subject to}
	\quad
	\partial_t\alpha_t+\diverg(\alpha_t v_t)=0.
\end{equation}

\begin{prop}[Least-square velocities are gradients]\label{prop-least-square-gradient-velocity}
\index{least-square!velocity}
\index{gradient!velocity}
	Assume that $\alpha_t=\rho_t\,\d x$ is a smooth positive density curve and that boundary terms vanish. The minimizer of~\eqref{eq:least-square-field}, if it exists, is a gradient field
	\[
		v_t=\nabla\phi_t,
	\]
	where $\phi_t$, unique up to an additive constant, solves the weighted Poisson equation
\index{weighted!Poisson equation}
	\begin{equation}\label{eq:least-square-field-explicit}
		-\diverg(\rho_t\nabla\phi_t)=\partial_t\rho_t,
		\qquad
		v_t=-\nabla\Delta_{\alpha_t}^{-1}(\partial_t\alpha_t),
		\quad
		\Delta_{\alpha_t}\phi=\diverg(\alpha_t\nabla\phi).
	\end{equation}
\end{prop}

\begin{proof}
	Introduce a Lagrange multiplier $\phi_t$ for the continuity equation. The constrained problem has the formal saddle formulation
\index{continuity equation}
	\[
		\umin{v}\umax{\phi}
		\int_0^1
		\left[
			\frac12\int_{\RR^d}\norm{v_t(x)}^2\,\d\alpha_t(x)
			+\int_{\RR^d}\phi_t(x)\big(\diverg(\alpha_t v_t)(x)+\partial_t\alpha_t(x)\big)\d x
		\right]\d t.
	\]
	Integrating by parts in the divergence term gives, for each $t$,
	\[
		\int \left(\frac12\norm{v_t}^2-\dotp{\nabla\phi_t}{v_t}\right)\d\alpha_t
		+
		\int \phi_t\,\partial_t\alpha_t .
	\]
	The pointwise minimizer in $v_t$ is therefore $v_t=\nabla\phi_t$. Substituting this into the constraint $\partial_t\rho_t+\diverg(\rho_t v_t)=0$ gives the weighted Poisson equation in~\eqref{eq:least-square-field-explicit}. The inverse notation is just a shorthand for solving this equation on zero-mean right-hand sides, modulo additive constants.
\index{zero mean}
\index{weighted!Poisson equation}
\end{proof}

In general this inversion is still computationally demanding, but special choices of $(\alpha_t)_t$ lead to simpler formulas; this is the mechanism exploited later by flow matching.
\index{flow!matching}

\begin{alg}[Least-square velocity reconstruction]\label{alg:least-square-velocity-reconstruction}
\index{least-square!velocity}
\index{weighted!Poisson equation}
\textbf{Input:} Smooth positive density curve $(\rho_t)_{t\in[0,1]}$ and boundary conditions.

\textbf{Output:} Minimal-energy velocity field $v_t$ realizing the curve.

\textbf{For} each time $t$ \textbf{do}:
\begin{algblock}

\textbf{Compute} $\partial_t\rho_t$.

\textbf{Solve} weighted Poisson equation:
\(-\diverg(\rho_t\nabla\phi_t)=\partial_t\rho_t, \qquad \int\phi_t\rho_t=0.\)

\textbf{Set}
\(v_t=\nabla\phi_t .\)
\end{algblock}
\algreturnskip
\textbf{Return} $(v_t)_{t\in[0,1]}$.
\end{alg}

\section{Benamou--Brenier dynamic formulation of OT}
\index{Benamou-Brenier}
\index{Benamou-Brenier!formulation}
\label{sec-benamou-brenier-dynamic}

The dynamic formulation identifies $\Wass_2$ with the kinetic energy of the cheapest continuity-equation path. It is the point where OT becomes a least-action principle.
\index{continuity equation}
\index{dynamic!formulation}

Instead of assuming that a whole curve $(\alpha_t)_{t\in[0,1]}$ is prescribed, one only fixes its endpoints $\alpha_0$ and $\alpha_1$ and minimizes the least-square energy~\eqref{eq:least-square-field}. The theorem of Benamou and Brenier states that this geodesic energy is exactly the squared Wasserstein distance~\cite{benamou2000computational}.
\index{Wasserstein!distance}

\begin{thm}[Benamou--Brenier]\label{thm-benamou-brenier}
\index{Benamou-Brenier}
	For probability measures $\alpha_0,\alpha_1\in\Pp_2(\RR^d)$,
\index{probability measure}
	\begin{equation}\label{eq:benamou-brenier}
	\Wass_2^2(\alpha_0,\alpha_1)
	=
	\inf_{(\alpha_t,v_t)}
	\int_0^1\!\int_{\RR^d}\norm{v_t(x)}^2\,\d\alpha_t(x)\d t,
	\end{equation}
	where the infimum is over $(\alpha_t,v_t)$ solving $\partial_t\alpha_t+\nabla\!\cdot(\alpha_t v_t)=0$ with $\alpha_{t=0}=\alpha_0$ and $\alpha_{t=1}=\alpha_1$. If $\alpha_0$ has a density and $T$ is the optimal Monge map $T_\sharp\alpha_0=\alpha_1$, the minimizer is
\index{Monge!problem}
	\begin{equation}\label{eq:static-to-dynamic}
		\alpha_t=((1-t)\Id+tT)_\sharp\alpha_0,
		\qquad
		v_t\big((1-t)x+tT(x)\big)=T(x)-x.
	\end{equation}
\end{thm}

\begin{proof}
	For the inequality ``dynamic $\leq$ static'', assume first that a Monge map $T$ exists and define $(\alpha_t,v_t)$ by~\eqref{eq:static-to-dynamic}. Since the Lagrangian velocity $T(x)-x$ is independent of $t$,
\index{Monge!problem}
	\[
		\int_0^1\!\int\norm{v_t}^2\,\d\alpha_t\d t
		=
		\int\norm{T(x)-x}^2\,\d\alpha_0(x),
	\]
	so the dynamic cost is no larger than the static Monge cost. Without a Monge map, the same construction is made with an optimal coupling $\pi$: sample $(X,Y)\sim\pi$ and move along the straight path $\gamma_{X,Y}(t)=(1-t)X+tY$. This path measure has action $\int\norm{x-y}^2\d\pi(x,y)$; projecting path velocities onto their conditional mean at time $t$ gives an admissible Eulerian velocity with no larger action, so the dynamic value is no larger than the Kantorovich value.
\index{dynamic value}
\index{optimal coupling}
\index{Eulerian!velocity}
\index{path!measure}

	Conversely, for a smooth deterministic path, take the flow $T_t$ defined by $\dot T_t=v_t\circ T_t$ and $T_0=\Id$. Then $\alpha_t=(T_t)_\sharp\alpha_0$ and $(T_1)_\sharp\alpha_0=\alpha_1$. Jensen's inequality gives
\index{Jensen inequality}
	\[
		\norm{T_1(x)-x}^2
		\leq
		\int_0^1\norm{v_t(T_t(x))}^2\d t.
	\]
	After integration with respect to $\alpha_0$, the Monge cost is bounded above by the dynamic action. For general finite-energy solutions of the continuity equation, the superposition principle lifts the curve to a probability measure on absolutely continuous paths; applying Jensen's inequality pathwise gives a coupling of the endpoints whose quadratic cost is no larger than the action. Thus the Kantorovich value is bounded above by the dynamic value.
\index{dynamic value}
\index{absolutely continuous path}
\index{dynamic action}
\index{superposition principle}
\index{Jensen inequality}
\index{probability measure}
\index{continuity equation}
\index{cost!quadratic}
\end{proof}

Although~\eqref{eq:benamou-brenier} is not jointly convex in $(\alpha_t,v_t)$, it becomes convex after replacing velocities by the momentum measure $m_t=v_t\alpha_t$ and using the perspective action. In the absolutely continuous case $\alpha_t=\rho_t\,\d x$ and $m_t(x)=\rho_t(x)v_t(x)$, this reads
\index{momentum}
\index{Benamou-Brenier}
\begin{equation}\label{eq:benamou-brenier-convex}
	\Wass_2^2(\alpha_0,\alpha_1)
	=
	\inf_{\substack{\partial_t\rho_t+\diverg m_t=0\\
	\rho_{t=0}\d x=\alpha_0,\ \rho_{t=1}\d x=\alpha_1}}
	\int_0^1\!\int_{\RR^d}
	\frac{\norm{m_t(x)}^2}{\rho_t(x)}\,\d x\,\d t,
	\end{equation}
with the usual convention that the integrand is $0$ when $(\rho_t,m_t)=(0,0)$ and $+\infty$ when $\rho_t=0$ but $m_t\neq0$. For singular endpoints or curves, the same statement is interpreted with vector-valued momentum measures and the corresponding recession convention. This convex reformulation enables geodesic interpolation by convex optimization once the domain is discretized.
\index{recession convention}
\index{momentum}

\paragraph{Proximal splitting.}
\index{proximal!splitting}
\index{Douglas-Rachford!splitting}

The convex momentum formulation also explains the original Benamou--Brenier solver. Papadakis, Peyr\'e and Oudet~\cite{FPapPeyOud13} showed that, after discretization, the ALG2 scheme is a Douglas--Rachford splitting, equivalently ADMM on the Fenchel--Rockafellar dual. Suppressing discretization indices, write $U=(\rho,m)$ and introduce the perspective integrand
\[
	J(\rho,m)=
	\begin{cases}
		\norm{m}^2/\rho, & \rho>0,\\
		0, & (\rho,m)=(0,0),\\
		+\infty, & \text{otherwise}.
	\end{cases}
\]
Let
\[
	\mathcal F(U)=\int_0^1\!\int_{\RR^d}J(\rho_t(x),m_t(x))\,\d x\,\d t,
	\qquad
	\mathcal C=\enscond{(\rho,m)}{\partial_t\rho+\diverg m=0,\ \rho_0\d x=\alpha_0,\ \rho_1\d x=\alpha_1},
\]
and let $\mathcal G=\iota_{\mathcal C}$ be the indicator of this affine continuity constraint. The convex Benamou--Brenier problem is therefore
\[
	\min_U\; \mathcal F(U)+\mathcal G(U).
\]
On the resulting Hilbert space of unknowns, or formally for an $L^2$-type product structure, the two proximal operators are
\[
	\prox_{\tau\mathcal F}(\bar U)
	=
	\argmin_U
	\frac12\norm{U-\bar U}_{\mathsf H}^2+\tau\mathcal F(U),
	\qquad
	\prox_{\tau\mathcal G}(\bar U)
	=
	\argmin_{U\in\mathcal C}
	\frac12\norm{U-\bar U}_{\mathsf H}^2 .
\]
For the standard product metric, the first map is local in $(t,x)$: it is the proximal operator of the convex perspective $J$. The second map is the orthogonal projection onto the affine set defined by the divergence equation and endpoint constraints. Douglas--Rachford then alternates these two simple operations:
\[
	\begin{aligned}
		U^{k+1} &= \prox_{\tau\mathcal F}(Z^k),\\
		\widetilde U^{k+1} &= \prox_{\tau\mathcal G}(2U^{k+1}-Z^k),\\
		Z^{k+1} &= Z^k+\widetilde U^{k+1}-U^{k+1}.
	\end{aligned}
\]
Equivalently one may swap the roles of $\mathcal F$ and $\mathcal G$. At convergence, the two shadow sequences $U^k$ and $\widetilde U^k$ agree and give a minimizer of the convex dynamic problem. This viewpoint is useful because it separates the nonlinear but pointwise perspective proximal step from the global but linear projection onto the continuity equation~\cite{FPapPeyOud13,Eckstein1992}.
\index{perspective function}
\index{continuity equation}
\index{Fenchel-Rockafellar duality}
\index{ADMM}

\begin{alg}[Douglas--Rachford for dynamic Benamou--Brenier]\label{alg:benamou-brenier-douglas-rachford}
\index{Douglas-Rachford!splitting}
\textbf{Input:} Functionals $\mathcal F,\mathcal G=\iota_{\mathcal C}$, proximal parameter $\tau>0$, initial field $Z^0$.

\textbf{Output:} Discrete density-momentum field $U^\star$.

\textbf{For} $k=0,1,\ldots$ \textbf{do}:
\begin{algblock}
\(U^{k+1}=\prox_{\tau\mathcal F}(Z^k).\)

\textbf{Project} reflected point:
\(\widetilde U^{k+1} = \prox_{\tau\mathcal G}(2U^{k+1}-Z^k) = \Proj_{\mathcal C}(2U^{k+1}-Z^k).\)

\textbf{Update}
\(Z^{k+1}=Z^k+\widetilde U^{k+1}-U^{k+1}.\)

\textbf{If} $\norm{U^{k+1}-\widetilde U^{k+1}}\leq\mathrm{tol}$ \textbf{then}:
\begin{algblock}
\textbf{Return} $U^{k+1}$.
\end{algblock}
\end{algblock}
\end{alg}

\begin{figure}[ht]
\centering
\newcommand{\bbgeodesicpanelheight}{.155\linewidth}
\begin{tabular}{@{}cc@{}}
\includegraphics[height=\bbgeodesicpanelheight]{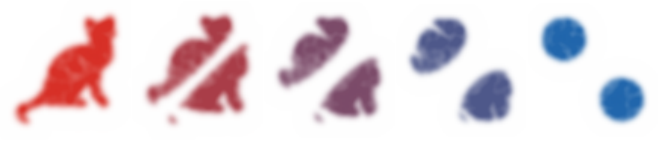} &
\index{Benamou-Brenier}
\includegraphics[height=\bbgeodesicpanelheight]{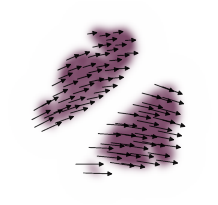}
\\[-.1em]
\small density sequence &
\small midpoint velocities
\end{tabular}
\caption{Benamou--Brenier geodesic between two sampled silhouettes. A discrete quadratic OT plan between finely subsampled cat and two-disks point clouds induces the McCann interpolation $Z_t=(1-t)X+tY$, which is the Lagrangian realization of the least-action solution. The left panel renders local color images of the smaller-bandwidth kernel-smoothed densities with enough padding to include the full silhouettes. The right panel overlays shortened velocity arrows centered at evenly subsampled midpoint particles $Z_{1/2}$; each displayed arrow runs in data coordinates from a source-side tail to a target-side head along the matched characteristic direction $Y-X$, but is not drawn as the full endpoint segment from $X$ to $Y$.}
\index{McCann interpolation}
\index{Benamou-Brenier}
\label{fig:dynamic-benamou-brenier-geodesic}
\end{figure}

\begin{rem}[Path-space formulation]\label{rem-bb-path-space}
\index{path-space!formulation}
	Let $\Ss=C([0,1];\RR^d)$ be the space of continuous paths endowed with the uniform topology. For $t\in[0,1]$ define the evaluation map
	\[
		P_t:\Ss\to\RR^d,
		\qquad
		P_t(\gamma)=\gamma(t).
	\]
	The Benamou--Brenier cost admits the equivalent formulation
\index{Benamou-Brenier}
	\[
		\Wass_2^2(\alpha_0,\alpha_1)
		=
		\inf_{M\in\Pp(\Ss)}
		\enscond{
			\int_{\Ss}\!\int_0^1\norm{\dot\gamma(t)}^2\d t\,\d M(\gamma)
		}{
			(P_0)_\sharp M=\alpha_0,\ (P_1)_\sharp M=\alpha_1
		}.
	\]
	If $\alpha_0$ has a density, the minimizer $M^*$ is unique. Its time marginals reproduce the optimal curve: $\alpha_t=(P_t)_\sharp M^*$ for all $t$. Furthermore, for a.e. $t$, denoting $Q_t(\gamma)=\dot\gamma(t)$ on absolutely continuous paths, the conditional law of the velocity is deterministic:
\index{absolutely continuous path}
\index{conditional law}
	\[
		(P_t,Q_t)_\sharp M^*(\d x,\d q)
		=
		\alpha_t(\d x)\delta_{v_t^*(x)}(\d q),
	\]
	where $v_t^*$ is the optimal velocity field in the Benamou--Brenier formulation. Hence $M^*$ concentrates on straight-line geodesics and, for a.e. $t$, assigns exactly one direction at $\alpha_t$-a.e. spatial point.
\index{Benamou-Brenier}
\index{velocity field}
\index{Benamou-Brenier!formulation}
\end{rem}

\paragraph{Extensions of the dynamic formulation.}
\index{dynamic!formulation}
\label{sec-bb-extensions}

The same variational grammar extends beyond the quadratic Wasserstein distance. One changes either the kinetic exponent, the mobility or the balance equation, while keeping a continuity-type constraint and a convex perspective action.
\index{Wasserstein!distance}

\begin{rem}[Generalized Benamou--Brenier distances]\label{rem-generalized-bb}
\index{Benamou-Brenier}
\index{Benamou-Brenier!distance}
	The dynamic formulation is not specific to $\Wass_2$. For measures with finite $p$-th moments and $p>1$, one has the analogous action formula
\index{dynamic!formulation}
	\[
		\Wass_p^p(\alpha_0,\alpha_1)
		=
		\inf_{\substack{\partial_t\alpha_t+\nabla\!\cdot(\alpha_t v_t)=0\\
		\alpha_{t=0}=\alpha_0,\ \alpha_{t=1}=\alpha_1}}
		\int_0^1\!\int_{\RR^d} |v_t(x)|^p\,\d\alpha_t(x)\,\d t.
	\]
	When $\alpha_t=\rho_t\,\d x$ and $m_t=\rho_t v_t$, this becomes the convex perspective action
	\[
		\int_0^1\!\int_{\RR^d}
		\frac{|m_t(x)|^p}{\rho_t(x)^{p-1}}\,\d x\,\d t,
		\qquad
		\partial_t\rho_t+\nabla\!\cdot m_t=0,
	\]
	with the usual convention that the integrand is $0$ if $(\rho,m)=(0,0)$ and $+\infty$ if $\rho=0$ but $m\neq0$.

	A second class of variants changes the mobility of the medium: the quadratic action $|m|^2/\rho$ is replaced by $|m|^2/\theta(\rho)$ for a suitable concave mobility $\theta$. Under appropriate structural assumptions, this produces transport metrics adapted to nonlinear diffusions and finite-volume discretizations~\cite{dolbeault2009new}. On finite graphs and Markov chains, the analogous action uses an edge mobility, often the logarithmic mean of the endpoint densities, and leads to discrete Wasserstein geometries~\cite{Maas2011,MielkeCVPDE}. These extensions keep the same variational grammar as Benamou--Brenier: a continuity-type constraint, an action density, and geodesics obtained by minimizing an integrated kinetic cost.
\index{Benamou-Brenier}
\end{rem}

\paragraph{Dynamic unbalanced OT.}
\index{unbalanced!OT}
\index{dynamic!unbalanced OT}

Unbalanced dynamic transport is obtained by allowing mass to be created and destroyed along the path. The continuity equation is replaced by a balance equation, and the action penalizes both spatial motion and growth. This dynamic formulation underlies the Hellinger--Kantorovich and Wasserstein--Fisher--Rao metrics~\cite{LieroMielkeSavareShort,2017-chizat-focm}; its equivalence with static entropy-transport and cone formulations is developed in~\cite{LieroMielkeSavareLong,2015-chizat-unbalanced}. A representative quadratic action is
\index{continuity equation}
	\[
		\partial_t\rho_t+\nabla\!\cdot m_t=s_t,
		\qquad
		\int_0^1\!\int
		\left(\frac{|m_t|^2}{\rho_t}+\kappa^2\frac{s_t^2}{\rho_t}\right)\d x\,\d t,
	\]
	with the same perspective convention as above. Equivalently, writing $m_t=\rho_t v_t$ and $s_t=\rho_t g_t$, one minimizes $\int_0^1\int(|v_t|^2+\kappa^2 g_t^2)\d\rho_t\d t$ under $\partial_t\rho_t+\nabla\!\cdot(\rho_t v_t)=g_t\rho_t$. The parameter $\kappa$ fixes the relative cost of reaction and transport; changing it rescales the radial/angular balance in the associated cone metric. For measure-valued triples, the action is understood in the lower-semicontinuous perspective sense
	\[
		\mathcal A_\kappa(\rho,m,s)
		\eqdef
		\int
		\left(
			\frac{\norm{\dot m}^2}{\dot\rho}
			+
			\kappa^2\frac{\dot s^2}{\dot\rho}
		\right)\d\lambda,
		\qquad
		(\dot\rho,\dot m,\dot s)
		=
		\left(\frac{\d\rho}{\d\lambda},\frac{\d m}{\d\lambda},\frac{\d s}{\d\lambda}\right),
	\]
	where $\lambda$ dominates $\rho$ and the total variations of $m$ and $s$, and the value is independent of this choice. The convention is $0/0=0$ and $a/0=+\infty$ for $a>0$, so finite action forces both the flux and source to be absolutely continuous with respect to the transported mass.
\index{balance equation}
\index{Hellinger-Kantorovich}
\index{Wasserstein-Fisher-Rao}

\begin{prop}[Static/dynamic equivalence for unbalanced OT]\label{prop-static-dynamic-unbalanced}
\index{static-dynamic equivalence}
\index{unbalanced!OT}
		Fix the action above and let $\CW_\kappa$ be the cone value of Theorem~\ref{thm-cone-unbalanced-ot} with the cone metric normalized to the same growth scale $\kappa$. For nonnegative finite measures $\alpha_0,\alpha_1$ on $\RR^d$, the dynamic value
	\begin{equation}\label{eq-dynamic-unbalanced-ot}
		\WFR_\kappa^2(\alpha_0,\alpha_1)
		\eqdef
		\inf_{\substack{\partial_t\rho_t+\nabla\cdot m_t=s_t\\
		\rho_0=\alpha_0,\ \rho_1=\alpha_1}}
		\int_0^1
		\mathcal A_\kappa(\rho_t,m_t,s_t)\,\d t
	\end{equation}
		equals the static cone formulation $\CW_\kappa(\alpha_0,\alpha_1)$. Hence the static unbalanced problem and the balance-equation least-action problem define the same geodesic distance.
\end{prop}

\begin{proof}
	The cone construction turns variation of mass into radial motion and spatial transport into angular motion on $\mathfrak C[\RR^d]$. Applying the Benamou--Brenier theorem on the cone to the lifted endpoint measures gives a dynamic least-action problem on $\mathfrak C[\RR^d]$ whose static value is the cone value $\CW_\kappa$ of Theorem~\ref{thm-cone-unbalanced-ot}. This is the standard static/dynamic identification for the Hellinger--Kantorovich and Wasserstein--Fisher--Rao metrics~\cite{LieroMielkeSavareShort,LieroMielkeSavareLong,2017-chizat-focm,2015-chizat-unbalanced}.
\index{cone!lifting}
\index{Benamou-Brenier}

	Projecting a cone curve back to the base space with the weight $r^2$ produces a measure curve $\rho_t$, a spatial flux $m_t$ and a source term $s_t$ satisfying the balance equation. With the matching normalization of the cone metric, the cone kinetic energy decomposes exactly into the perspective action $\mathcal A_\kappa$ in~\eqref{eq-dynamic-unbalanced-ot}. Conversely, any finite-action triple $(\rho_t,m_t,s_t)$ can be lifted to a cone curve whose radial velocity realizes the growth term and whose spatial velocity realizes the transport term, with the same action after relaxation. The two infima are therefore equal; lower semicontinuity gives the general finite-measure statement from the smooth positive case.
\index{lower semicontinuity}
\end{proof}

\begin{figure}[H]
\centering
\setlength{\tabcolsep}{1pt}
\begin{tabular}{@{}rccccc@{}}
& \small $t=0$ & \small $t=.25$ & \small $t=.5$ & \small $t=.75$ & \small $t=1$ \\
\small balanced &
\includegraphics[width=.165\linewidth]{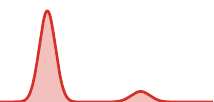} &
\includegraphics[width=.165\linewidth]{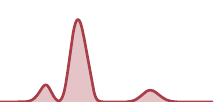} &
\includegraphics[width=.165\linewidth]{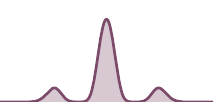} &
\includegraphics[width=.165\linewidth]{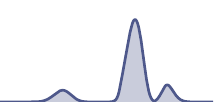} &
\includegraphics[width=.165\linewidth]{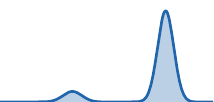} \\
\small unbalanced &
\includegraphics[width=.165\linewidth]{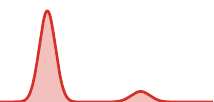} &
\includegraphics[width=.165\linewidth]{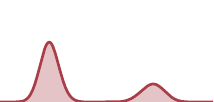} &
\includegraphics[width=.165\linewidth]{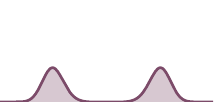} &
\includegraphics[width=.165\linewidth]{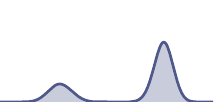} &
\includegraphics[width=.165\linewidth]{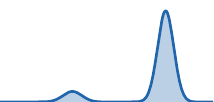}
\end{tabular}
\caption{Balanced and unbalanced Sinkhorn-barycenter interpolations between two one-dimensional Gaussian mixtures with swapped modal masses. The balanced row conserves total mass, so excess mass from the dominant left mode must move along the line toward the dominant right target mode, producing transient mass in the middle. The unbalanced row uses KL-relaxed marginal constraints; mass can be attenuated near overrepresented modes and recreated near underrepresented modes, giving a reaction--transport interpolation closer to the Wasserstein--Fisher--Rao intuition.}
\index{Sinkhorn!barycenter}
\index{unbalanced!barycenter}
\label{fig:dynamic-unbalanced-geodesic}
\end{figure}


\chapter{Wasserstein Gradient Flows}
\index{Wasserstein!gradient}
\index{gradient!flow}
\index{Wasserstein!gradient flow}
\label{sec-wasserstein-gradient-flows}

Once $\Wass_2$ is a dynamic metric, one can run gradient descent directly on the space of measures. This chapter derives the formal Wasserstein gradient, explains the JKO minimizing-movement scheme, records the role of geodesic convexity in convergence, and then applies the same calculus to mean-field neural-network training.
\index{mean-field!neural network}
\index{minimizing movement scheme}
\index{Wasserstein!gradient}
\index{convexity!geodesic}

\section{Minimizing Movements and Wasserstein Gradients}
\index{minimizing movement}
\index{Wasserstein!gradient}

This first section explains how a variational implicit-Euler step on measures gives rise, in the small-step limit, to a continuity equation driven by the Wasserstein gradient of the energy.
\index{Wasserstein!gradient}
\index{continuity equation}

We now consider a function $f(\alpha)$ and seek a minimizing evolution $(\alpha_t)_t$. The general strategy of minimizing movement over a metric space is to construct a discrete-time evolution using an implicit Euler scheme:
\index{implicit Euler scheme}
\index{minimizing movement}
\begin{equation}
    \alpha_{t+\tau} := \arg\min_\alpha \frac{1}{2 \tau} \Wass_2(\alpha_t, \alpha)^2 + f(\alpha). \label{eq:jko-discr}
\end{equation}

\begin{alg}[JKO minimizing movement]\label{alg:jko-minimizing-movement}
\index{JKO scheme}
\textbf{Input:} Energy $f$, initial measure $\alpha^0$, time step $\tau>0$, number of steps $K$.

\textbf{Output:} Discrete gradient-flow trajectory $(\alpha^k)_{k=0}^K$.

\textbf{For} $k=0,\ldots,K-1$ \textbf{do}:
\begin{algblock}
\(\alpha^{k+1} \in \uargmin{\alpha\in\Pp_2(\RR^d)} \frac{1}{2\tau}\Wass_2^2(\alpha^k,\alpha)+f(\alpha).\)

\textbf{Set} $\alpha_t^\tau=\alpha^k$ for $t\in[k\tau,(k+1)\tau)$.

\end{algblock}
\algreturnskip
\textbf{Return} $(\alpha^k)_{k=0}^K$ and $\alpha_t^\tau$.
\end{alg}

\paragraph{Euclidean gradient flows.}
\index{gradient!flow}

If we restrict \eqref{eq:jko-discr} to finite dimensions and assume $\alpha_t = \delta_{x(t)}$ and $\alpha = \delta_x$ (single Dirac measures), this matches the implicit Euler scheme:
\index{Dirac mass}
\index{implicit Euler scheme}
\begin{equation*}
    x(t+\tau) := \arg\min_x \frac{1}{2 \tau} \norm{x - x(t)}^2 + h(x),
\end{equation*}
where $h(x) = f(\delta_x)$. Its solution is formally given by the implicit Euler formula:
\begin{equation*}
    x(t+\tau) = (\Id + \tau \nabla h)^{-1}(x(t)).
\end{equation*}
In contrast, the explicit Euler scheme is:
\begin{equation*}
    x(t+\tau) = (\Id - \tau \nabla h)(x(t)) = x(t) - \tau \nabla h(x(t)).
\end{equation*}

Both schemes converge as $\tau \to 0$ to:
\begin{equation}
    \dot{x}(t) = -\nabla h(x(t)). \label{eq:grad-flow-classical}
\end{equation}

\paragraph{Wasserstein gradient formula.}
\index{Wasserstein!gradient}

The implicit Euler scheme has the advantage that it does not require $h$ or $f$ to be smooth. For $f$, this is crucial to handle evolutions over measures that may have densities, atoms or other singular parts.
\index{implicit Euler scheme}

As $\tau \to 0$, under certain conditions on $f$, \eqref{eq:jko-discr} defines a continuous evolution $t \mapsto \alpha_t$. As discussed earlier, this evolution can be described as a Lagrangian evolution \eqref{eq:lagrangian-advection}. We use the following first-variation convention: for any $\beta\in\Pp(\RR^d)$ and the signed zero-mass perturbation $\rho=\beta-\alpha$,
\index{first variation}
\index{signed!measure}
\begin{equation*}
    f((1-\tau)\alpha + \tau \beta) = f(\alpha + \tau \rho) = f(\alpha) + \tau \int [\delta f(\alpha)](x) \d \rho(x) + o(\tau).
\end{equation*}
The key infinitesimal object is the vector field that represents this differential in the Wasserstein metric.

\Needspace{7\baselineskip}
\begin{defn}[Wasserstein gradient]\label{def-wasserstein-gradient}
\index{Wasserstein!gradient}
	Assume that $f$ admits a smooth first variation $\delta f(\alpha)$. In the smooth formal calculus on $\Pp_2(\RR^d)$, the Wasserstein gradient of $f$ at $\alpha$ is the gradient vector field
	\begin{equation*}
	    \Wgrad f(\alpha) = \nabla_x \delta f(\alpha).
	\end{equation*}
\end{defn}
The associated formal gradient flow is the continuity equation
\begin{equation}
    \frac{\partial \alpha_t}{\partial t} + \diverg(-\Wgrad f(\alpha_t) \alpha_t) = 0. \label{eq:wassflow-pde}
\end{equation}
The following proposition explains why this vector field is the Riemannian gradient for the $L^2(\alpha)$ metric on velocities.
\index{signed!measure}
\begin{proposition}[Formal Wasserstein gradient]\label{prop-formal-wass-gradient}
\index{Wasserstein!gradient}
	Assume that $f$ admits a smooth first variation $\delta f(\alpha)$ and that $\alpha$ has a smooth positive density. For infinitesimal perturbations generated by a velocity field $v$ through $(\Id+\tau v)_\sharp\alpha$, the differential of $f$ is
\index{first variation}
\index{velocity field}
	\[
		\frac{\d}{\d\tau}_{|\tau=0}
		f\big((\Id+\tau v)_\sharp\alpha\big)
		=
		\int \dotp{\nabla \delta f(\alpha)(x)}{v(x)}\d\alpha(x).
	\]
	Hence, for the Riemannian metric $\norm{v}_{L^2(\alpha)}^2=\int \norm{v}^2\d\alpha$, the Wasserstein gradient is the vector field
\index{Wasserstein!gradient}
	\[
		\Wgrad f(\alpha)=\nabla \delta f(\alpha).
	\]
\end{proposition}
\begin{proof}
	The push-forward expansion gives, in the sense of distributions,
\index{push-forward}
	\[
		(\Id+\tau v)_\sharp\alpha
		=
		\alpha-\tau\diverg(\alpha v)+o(\tau).
	\]
	Using the definition of the first variation,
\index{first variation}
	\[
		f((\Id+\tau v)_\sharp\alpha)
		=
		f(\alpha)-\tau\int \delta f(\alpha)\diverg(\alpha v)\d x+o(\tau).
	\]
	An integration by parts, with either compact support or vanishing boundary flux, gives
\index{flux}
	\[
		-\int \delta f(\alpha)\diverg(\alpha v)\d x
		=
		\int \dotp{\nabla\delta f(\alpha)}{v}\d\alpha.
	\]
	By definition of the Riesz representative for the $L^2(\alpha)$ metric, this representative is $\nabla\delta f(\alpha)$.
\end{proof}

The Wasserstein gradient-flow viewpoint already appears in John D. Lafferty's PhD work, published as ``The Density Manifold and Configuration Space Quantization'', under the name ``density manifold''. It was then systematically developed by Otto, who exposed the formal Riemannian structure of this space~\cite{otto2001geometry}. Rigorous metric-space treatments and numerical JKO schemes can be found in~\cite{ambrosio2006gradient,JDB-JKO,2015-Peyre-siims,gallouet2017jko}.
\index{Wasserstein!gradient}
\index{Wasserstein!gradient flow}
\index{JKO scheme}

\paragraph{From the JKO step to the velocity field.}
\index{velocity field}

A first-order expansion of the JKO step explains why~\eqref{eq:wassflow-pde} uses the vector field $\Wgrad f(\alpha)$. Write \eqref{eq:jko-discr} as a minimization over displacement fields $v$ such that $\alpha = (\Id + \tau v)_\sharp \alpha_t$:
	\[
		\min_{v} \frac{1}{2 \tau} \tau^2 \norm{v}_{L^2(\alpha_t)}^2 + f( (\Id + \tau v)_\sharp \alpha_t ).
	\]
Then we perform a first-order Taylor expansion of this formulation using
	\[
		 (\Id + \tau v)_\sharp \alpha_t =  \alpha_t - \tau \diverg( v \alpha_t ) + o(\tau)
	\]
	\[
		f( (\Id + \tau v)_\sharp \alpha_t ) = f(\alpha_t) - \tau \int \delta f(\alpha_t) \diverg( v \alpha_t ) \d x + o(\tau)
	\]
	\[
		 = f(\alpha_t) + \tau \int \langle \nabla_x \delta f(\alpha_t)(x), v(x) \rangle \d \alpha_t(x) + o(\tau)
	\]
to obtain the following first-order expansion in $\tau$ of the problem minimized in \eqref{eq:jko-discr}
	\[
		 \min_{v} f(\alpha_t) +  \tau \int \Big[ \frac{1}{2} \norm{v(x)}^2 + \langle \Wgrad f(\alpha_t)(x), v(x) \rangle \Big]\d \alpha_t(x) + o(\tau).
	\]
The pointwise minimizer is $v=-\Wgrad f(\alpha_t)$, which gives the velocity in the continuity equation.
\index{continuity equation}
We now detail examples of such Wasserstein gradient flows.
\index{Wasserstein!gradient}
\index{gradient!flow}
\index{Wasserstein!gradient flow}

\begin{figure}[ht]
\centering
\begin{tabular}{@{}cc@{}}
\small density iterates & \small quantile motion \\[-.15em]
\includegraphics[width=.43\linewidth]{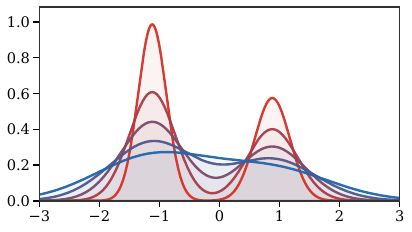} &
\includegraphics[width=.43\linewidth]{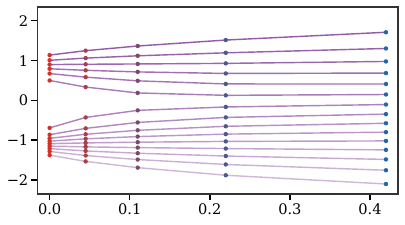}
\end{tabular}
\caption{JKO minimizing movements for the entropy flow in one dimension. The left panel displays successive implicit-Euler minimizers for the heat equation, colored from red to blue. The right panel tracks inverse CDF values $Q_t(s)=F_t^{-1}(s)$ for selected probability levels $s$, giving a Lagrangian view of the proximal movement in Wasserstein space.}
\index{minimizing movement}
\index{Wasserstein!space}
\index{heat equation}
\label{fig:gradflow-jko-entropy-steps}
\end{figure}

\paragraph{Discrete evolutions.}
\index{measure!evolution}

If $f(\alpha)$ can be evaluated on discrete distributions and $\Wgrad$ is continuous in this case, the flow \eqref{eq:wassflow-pde} maintains the number of Dirac masses, $\alpha_t = \frac{1}{n} \sum_i \delta_{x_i(t)}$. The particles $X(t) := (x_i(t))_i$ evolve according to a system of coupled ODEs:
\index{Dirac mass}
\begin{equation}
    \dot{x}_i(t) = -n\nabla_{x_i} F(X(t)), \label{eq:wassflows-particles}
\end{equation}
where $F(X) := f\left(\frac{1}{n} \sum_i \delta_{x_i}\right)$ and the factor $n$ comes from the empirical Wasserstein metric $\frac1n\sum_i \norm{\dot x_i}^2$.

\begin{alg}[Empirical Wasserstein particle descent]\label{alg:empirical-wasserstein-particle-flow}
\index{particle!Wasserstein descent}
\textbf{Input:} Particles $X^0=(x_1^0,\ldots,x_n^0)$, functional $f$, step size $h$, tolerance $\mathrm{tol}$.

\textbf{Output:} Particle trajectory $(X^k)_k$ and empirical measures.

\textbf{Define}
\(F(X)=f\!\left(\frac1n\sum_{i=1}^n\delta_{x_i}\right).\)

\textbf{For} $k=0,1,\ldots$ \textbf{do}:
\begin{algblock}

\textbf{For} $i=1,\ldots,n$ \textbf{do}
\begin{algblock}
\(g_i^k=n\nabla_{x_i}F(X^k), \qquad x_i^{k+1}=x_i^k-h\,g_i^k.\)
\end{algblock}
\textbf{If} \(\max_i\norm{x_i^{k+1}-x_i^k}\leq \mathrm{tol}\) \textbf{then}:
\begin{algblock}
\textbf{Return} $\frac1n\sum_i\delta_{x_i^{k+1}}$.
\end{algblock}
\end{algblock}
\end{alg}

\paragraph{Linear Functionals.}
\index{linear!functional}

The first benchmark is a functional whose first variation does not depend on the current measure.

\begin{example}[Linear potentials generate independent particles]
Take a linear functional
   \begin{equation}\label{eq:linear-func}
       f(\alpha) = \int h(x) \d \alpha(x).
   \end{equation}
   Then $\delta f(\alpha) = h$ is independent of $\alpha$, and the flow \eqref{eq:wassflow-pde} becomes
   \begin{equation*}
       \frac{\partial \alpha_t}{\partial t} + \diverg(-\nabla h \alpha_t) = 0.
   \end{equation*}
   This implies particles move independently according to the usual gradient flow \eqref{eq:grad-flow-classical}.
\index{gradient!flow}
\end{example}

\paragraph{Shannon Neg-Entropy.}
\index{Shannon!negative entropy}

Entropy gives the opposite benchmark: the flow is no longer a deterministic push-forward of particles, but a diffusion of density.

\begin{example}[Entropy generates heat and porous-medium flows]
The canonical density-dependent functional is the Shannon neg-entropy
   \begin{equation}
       f(\alpha) = \int \log\left(\frac{\d \alpha}{\d x}(x)\right) \d \alpha(x). \label{eq:entropy-func}
   \end{equation}
   Here, $\delta f(\alpha) = \log(\d\alpha/\d x)$ up to an additive constant, so $\Wgrad f(\alpha) = \nabla \alpha/\alpha$ (often called the score). The flow \eqref{eq:wassflow-pde} becomes the heat equation
\index{score function}
\index{heat equation}
   \begin{equation*}
       \partial_t \alpha_t = \Delta \alpha_t.
   \end{equation*}
   Other entropy functionals lead to nonlinear diffusion equations; finite-volume and particle discretizations are discussed in~\cite{CarrilloFiniteVolume,GianazzaARMA,Maas2011,ErbarHeatManifold}.
For a generalized entropy
   \begin{equation}\label{eq:gen-entropies}
       f(\alpha) = \int g\left(\frac{\d \alpha}{\d x}\right) \d x,
   \end{equation}
	   with a scalar convex function $g$, one obtains nonlinear diffusions in the smooth-density regime:
\index{convex!function}
   \begin{equation*}
       \frac{\partial \alpha_t}{\partial t} = \Delta(P(\alpha_t)),
   \end{equation*}
   where the pressure $P$ satisfies $P'(s)=s g''(s)$. For example, $g(s) = s \log(s)$ gives $P(s)=s$ and recovers \eqref{eq:entropy-func}, while $g(s) = s^m/(m-1)$, $m > 1$, gives $P(s)=s^m$ up to an additive constant and yields the porous-medium equation.
\index{porous medium equation}
\end{example}
The preceding examples are also governed by a precise geodesic-convexity criterion.
	A celebrated theorem by McCann~\cite{mccann1997convexity} states that an internal energy of the form~\eqref{eq:gen-entropies}, for $g : \RR^+ \to \RR\cup\{+\infty\}$ with $g(0)=0$, is geodesically convex on $\Pp(\RR^d)$ when $g$ is convex and the map $r \mapsto r^d g(r^{-d})$ is convex and nonincreasing on $(0,+\infty)$.
	Examples of such functions are $g(s)=s^q$ for $q>1$ and Shannon entropy $g(s)=s \log(s)$. %
\index{Shannon!entropy}
	By contrast, $g(s)=-\log(s)$, associated with the reverse KL divergence, does not satisfy this displacement-convexity criterion.
\index{reverse KL divergence}
\index{displacement!convexity}

\begin{figure}[ht]
\centering
\begin{tabular}{@{}ccc@{}}
\small heat equation & \small porous medium $m=2$ & \small porous medium $m=6$ \\[-.15em]
\index{heat equation}
\includegraphics[width=.30\linewidth]{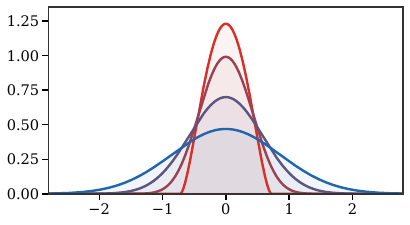} &
\includegraphics[width=.30\linewidth]{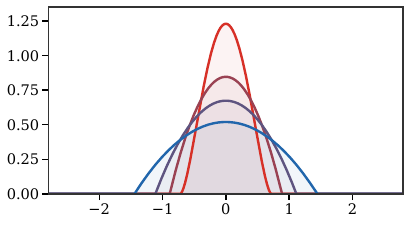} &
\includegraphics[width=.30\linewidth]{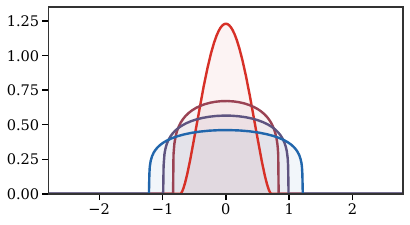}
\end{tabular}
\caption{Entropy-driven Wasserstein gradient flows from the same compact initial density. The heat flow is generated by Shannon entropy $g(\rho)=\rho\log\rho$ and instantly develops Gaussian tails. The porous-medium flows use the power entropy $g(\rho)=\rho^m/(m-1)$, hence $\partial_t\rho=\Delta(\rho^m)$: the middle panel has $m=2$, while the right panel has the stronger nonlinearity $m=6$, i.e. $\partial_t\rho=\Delta(\rho^6)$. Larger powers diffuse mainly where the density is high, producing a flatter core and a sharper compact free boundary.}
\index{Wasserstein!gradient flow}
\index{Shannon!entropy}
\label{fig:gradflow-heat-versus-porous-medium}
\end{figure}

\paragraph{Interaction Energies.}
\index{energy!interaction}

In a similar spirit, to obtain nonlinear evolutions without requiring the measure to have density, one can consider
   \begin{equation}
       f(\alpha) := \iint k(x, y) \d \alpha(x) \d \alpha(y). \label{eq:quadratic-func}
   \end{equation}
   For a symmetric kernel $k$:
   \begin{equation*}
       \delta f(\alpha)(x) = 2 \int k(x, y) \d \alpha(y), \quad \Wgrad f(\alpha)(x) = 2 \int \nabla_x k(x, y) \d \alpha(y).
   \end{equation*}
   For $\alpha_0 = \frac{1}{n} \sum_i \delta_{x_i}$, the flow \eqref{eq:wassflow-pde} implies particles $(x_i(t))_i$ obey:
   \begin{equation*}
       \dot{x}_i(t) = -\frac{2}{n} \sum_j \nabla k(x_i(t), x_j(t)).
   \end{equation*}
	   If $k$ is positive definite, or more generally conditionally positive definite on signed measures of zero total mass as for the energy-distance kernel $k(x,y)=-\norm{x-y}$, and one minimizes the squared kernel discrepancy to a teacher distribution $\beta$, then
\index{teacher distribution}
\index{energy!distance}
\index{signed!measure}
   \[
       \norm{\alpha-\beta}_k^2
       =
       \iint k\d\alpha\d\alpha
       -2\int\!\left(\int k(x,y)\d\beta(y)\right)\d\alpha(x)
       +\mathrm{constant}.
       \]
	       Thus MMD-type training energies are exactly an interaction energy plus a linear potential; the teacher distribution appears through the potential $x\mapsto-2\int k(x,y)\d\beta(y)$. The corresponding empirical Wasserstein gradient flow is
\index{teacher distribution}
\index{Wasserstein!gradient}
\index{gradient!flow}
\index{maximum mean discrepancy}
\index{Wasserstein!gradient flow}
\index{energy!interaction}
       \[
		\dot x_i(t)
		=
		-\frac{2}{n}\sum_j\nabla_x k(x_i(t),x_j(t))
		+
		2\int\nabla_x k(x_i(t),y)\d\beta(y).
       \]

The corresponding simulation loop is Algorithm~\ref{alg:mmd-particle-flow}.

\begin{alg}[MMD particle flow against a teacher law]\label{alg:mmd-particle-flow}
\index{maximum mean discrepancy}
\textbf{Input:} Initial particles $(x_i^0)_{i=1}^n$, teacher law $\beta$ or teacher samples $(y_b)_{b=1}^B$, kernel $k$, step size $h$.

\textbf{Output:} Particle trajectory targeting $\beta$.

\textbf{For} $k=0,1,\ldots$ \textbf{do}:
\begin{algblock}

\textbf{For} $i=1,\ldots,n$ \textbf{do}
\begin{algblock}
\textbf{Set} self-interaction \(r_i^k=-\frac{2}{n}\sum_{j=1}^n\nabla_x k(x_i^k,x_j^k)\).

\textbf{If} $\beta$ is available analytically \textbf{then}:
\begin{algblock}

\textbf{Set} teacher attraction \(a_i^k=2\int\nabla_x k(x_i^k,y)\d\beta(y)\).

\end{algblock}
\textbf{If} only samples $(y_b)_{b=1}^B$ are available \textbf{then}:
\begin{algblock}

\textbf{Set} \(a_i^k=\frac{2}{B}\sum_{b=1}^B\nabla_x k(x_i^k,y_b)\).

\end{algblock}
\textbf{Set} velocity \(v_i^k=r_i^k+a_i^k\).

\textbf{Update}
\(x_i^{k+1}=x_i^k+h\,v_i^k.\)
\end{algblock}
\end{algblock}
\algreturnskip
\textbf{Return} $(x_i^k)_{i,k}$.
\end{alg}

	       The first term is a kernelized self-interaction, while the second is the attraction induced by the continuous teacher kernel mean. At the continuum level, characteristic positive-definite kernels, and the Euclidean energy-distance kernel on probability measures, have $\beta$ as the unique minimizer of $\norm{\alpha-\beta}_k^2$. For a finite number of particles, however, the flow can only form a kernelized quadrature of $\beta$, and small particle systems may cover the target modes poorly. Figure~\ref{fig:gradflow-mmd-particle-count} illustrates this finite-particle effect.
\index{kernelized self-interaction}
\index{teacher kernel mean}
\index{characteristic kernel}
\index{kernel quadrature}
\index{probability measure}
\index{kernel!positive definite}
\index{energy!distance}

\begin{figure}[ht]
\centering
\begin{tabular}{@{}ccc@{}}
\small $n=10$ particles & \small $n=50$ particles & \small $n=300$ particles \\[-.15em]
\includegraphics[width=.30\linewidth]{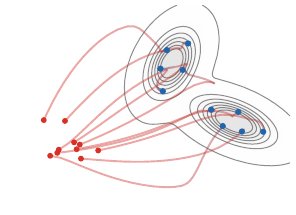} &
\includegraphics[width=.30\linewidth]{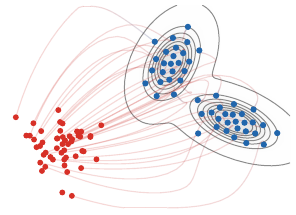} &
\includegraphics[width=.30\linewidth]{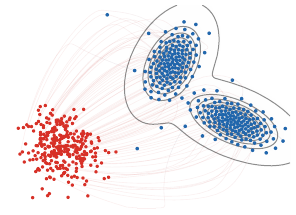}
\end{tabular}
\caption{Particle count in the deterministic Wasserstein gradient flow of the squared MMD-type discrepancy to a smooth two-Gaussian teacher distribution, using here the energy-distance kernel $k(x,y)=-\norm{x-y}$. The teacher itself is shown only through true density contours, while red dots are a compact shifted Gaussian initialization placed away from the target, red-to-blue curves show a thinned subset of particle trajectories, and blue dots show the stabilized long-time particles. With too few particles, the empirical measure forms a sparse kernelized quadrature and may under-cover the target modes; increasing $n$ makes the particle cloud approximate the continuous target geometry more faithfully.}
\index{teacher distribution}
\index{kernel quadrature}
\index{Wasserstein!gradient flow}
\index{empirical!measure}
\index{energy!distance}
\label{fig:gradflow-mmd-particle-count}
\end{figure}

\begin{figure}[ht]
\centering
\begin{tabular}{@{}ccc@{}}
\small repulsive kernel & \small attractive kernel & \small attraction--repulsion \\[-.15em]
\includegraphics[width=.30\linewidth]{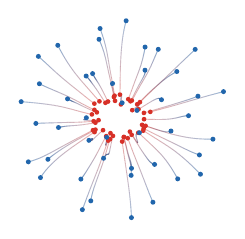} &
\includegraphics[width=.30\linewidth]{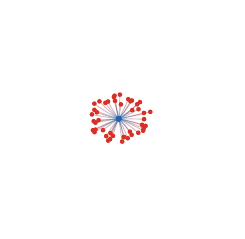} &
\includegraphics[width=.30\linewidth]{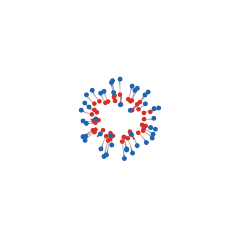}
\end{tabular}
\caption{Interaction-energy particle flows for three choices of $k$. A positive Gaussian kernel $k(x,y)=\exp(-\norm{x-y}^2/(2\sigma^2))$ produces short-range repulsion under Wasserstein descent; changing its sign produces attraction and collapse; adding a quadratic long-range attraction to the repulsive kernel yields a balanced attraction--repulsion dynamics. The curves use arclength-based red-to-blue coloring along a longer integration of the coupled particle ODE~\eqref{eq:wassflows-particles}.}
\index{particle ODE}
\index{energy!interaction}
\label{fig:gradflow-interaction-particles}
\end{figure}

\begin{figure}[ht]
\centering
\begin{tabular}{@{}cccc@{}}
\small OT rays & \small MMD force & \small Sinkhorn force & \small drifting field \\[-.15em]
\includegraphics[width=.225\linewidth]{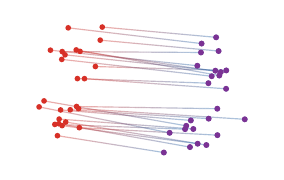} &
\includegraphics[width=.225\linewidth]{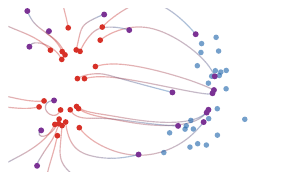} &
\includegraphics[width=.225\linewidth]{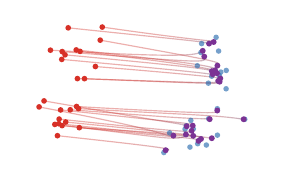} &
\index{Sinkhorn!divergence}
\includegraphics[width=.225\linewidth]{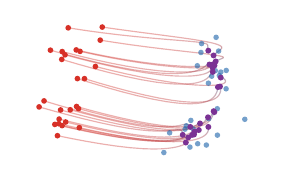}
\end{tabular}
\caption{Particle trajectories induced by different discrepancy geometries. The red particles and blue target cloud are the same in all panels. Straight OT displacement produces rays from an optimal matching; an MMD-type witness field gives smoother nonlocal forces; the Sinkhorn-divergence force is an entropic, debiased transport attraction; and the normalized drifting field combines attraction to data with self-repulsion. The figure is qualitative: it compares geometric behavior, not solver performance.}
\label{fig:gradflow-particle-objective-geometries}
\end{figure}

\paragraph{Stochastic particles and McKean--Vlasov limits.}
\index{stochastic!particle}
\index{McKean-Vlasov limit}
More generally, deterministic particle flows have stochastic counterparts, where Brownian noise at the particle level becomes an entropy term at the level of measures. If the drift does not depend on the empirical measure, each particle evolves independently according to
\index{empirical!measure}
\[
	\d X_t=b(X_t)\d t+\sqrt{2}\sigma\d B_t,
\]
and the one-particle law $\alpha_t=\rho_t\d x$ directly satisfies the linear Fokker--Planck equation
\index{Fokker-Planck equation}
\[
	\partial_t\rho_t=-\diverg(b\rho_t)+\sigma^2\Delta\rho_t.
\]

\begin{example}[Langevin drift as a free-energy flow]
	If $b=-\nabla V$, this linear Fokker--Planck equation is the $\Wass_2$ gradient flow of the free energy
	\[
		\rho\mapsto \int V\rho\,\d x+\sigma^2\int\rho\log\rho\,\d x.
	\]
\end{example}

The mean-field case is different: the drift is recomputed from the current empirical distribution of all particles,
\index{empirical!distribution}
\index{gradient!flow}
\[
	\d X_i^n(t)=b(X_i^n(t),\mu_t^n)\d t+\sqrt{2}\sigma\d B_i(t),
	\qquad
	\mu_t^n=\frac1n\sum_{i=1}^n\delta_{X_i^n(t)} .
\]
For finite $n$, the empirical law $\mu_t^n$ is itself random. Under suitable Lipschitz, growth and chaotic-initialization assumptions, propagation of chaos states that finitely many particles become asymptotically independent as $n\to\infty$, all with the same deterministic law $\rho_t\d x$; equivalently, the empirical measure $\mu_t^n$ converges in probability to this law. The limiting density solves the nonlinear Fokker--Planck, or McKean--Vlasov, equation
\index{Fokker-Planck equation}
\index{McKean-Vlasov limit}
\index{propagation of chaos}
\index{empirical!law}
\[
	\partial_t \rho_t
	=
	-\diverg\big(b(x,\rho_t)\rho_t\big)
	+
	\sigma^2\Delta\rho_t .
\]
When the interaction drift has variational form
\[
	b(x,\rho)=-\nabla \frac{\delta \mathcal E}{\delta \rho}(x),
\]
	this PDE is the Wasserstein gradient flow of the entropy-regularized energy
\index{Wasserstein!gradient}
\index{gradient!flow}
\index{Wasserstein!gradient flow}
	\[
		\mathcal E(\rho)+\sigma^2\int \rho\log\rho\,\d x .
	\]

\begin{figure}[ht]
\centering
\begin{tabular}{@{}c@{}}
\small independent Langevin particles \\[-.15em]
\includegraphics[width=.86\linewidth]{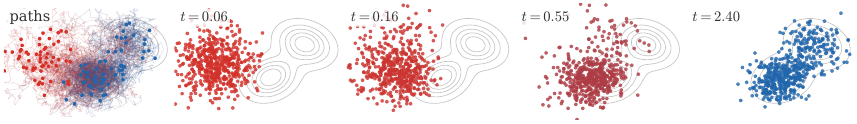} \\[.25em]
\small deterministic KDE-score particles \\[-.15em]
\includegraphics[width=.86\linewidth]{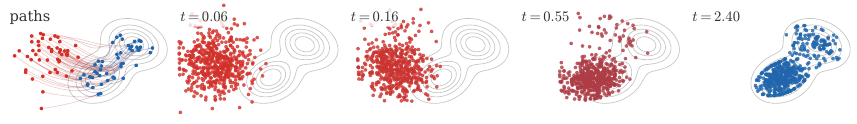} \\[.25em]
\small grid Fokker--Planck density \\[-.15em]
\includegraphics[width=.86\linewidth]{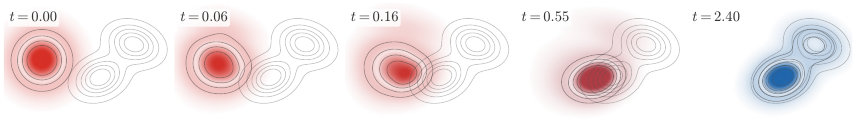}
\end{tabular}
\caption{Three numerical representations of the same entropy-regularized Wasserstein gradient flow of $\KL(\rho|\beta)$, where $\beta$ is a two-Gaussian target shifted to the right of an initially isotropic Gaussian density. The first row simulates independent Langevin particles and displays a thinned set of trajectories in the left panel. The second row evolves many deterministic particles with velocity $\tau(\nabla\log\beta-\nabla\log\rho_t)$, estimating $\nabla\log\rho_t$ by a sharper kernel-density score; only representative trajectories and particle subsets are displayed. The third row solves the corresponding Fokker--Planck equation on a grid, starting from the initial density in the left panel. The remaining columns use front-loaded times, so that the onset of the flow and the later deformation toward a bimodal law are both visible.}
\index{Wasserstein!gradient flow}
\index{Fokker-Planck equation}
\label{fig:gradflow-fokker-planck-three-representations}
\end{figure}

\section{Geodesic Convexity and Convergence}
\index{convexity!geodesic}
\label{sec-geodesic-convexity}

Geodesic convexity is the convexity notion adapted to Wasserstein geometry. It is the condition that turns the formal gradient-flow calculus into a convergence theory.
\index{convexity!geodesic}
\index{gradient!flow}

\paragraph{Geodesics and convexity.}

A constant-speed $\Wass_2$ geodesic between $\alpha_0$ and $\alpha_1$ is obtained, as in Definition~\ref{def-w2-geodesic-induced-by-plan}, from any optimal coupling $\pi^\star\in\Couplings(\alpha_0,\alpha_1)$ by the McCann interpolation
\index{optimal coupling}
\index{constant-speed geodesic}
\index{McCann interpolation}
\[
	\alpha_t=((1-t)P_0+tP_1)_\sharp\pi^\star,
	\qquad t\in[0,1],
\]
where $P_0(x,y)=x$ and $P_1(x,y)=y$. If the optimal plan is induced by a Brenier map $T$, this reduces to $((1-t)\Id+tT)_\sharp\alpha_0$. The coupling formula is important because geodesics exist even when no Monge map exists, for instance when a Dirac mass must split.
\index{Monge!problem}
\index{optimal plan}
\index{Brenier!map}
\index{Dirac mass}

\begin{defn}[Geodesic convexity]\label{def-geodesic-convexity}
\index{convexity!geodesic}
	A functional $f$ on $\Pp_2(\RR^d)$ is geodesically convex if for every $\Wass_2$ geodesic $(\alpha_t)_t$,
	\[
		f(\alpha_t)\leq (1-t)f(\alpha_0)+t f(\alpha_1).
	\]
	It is $\lambda$-geodesically convex if the right-hand side is improved by $-\frac{\lambda}{2}t(1-t)\Wass_2^2(\alpha_0,\alpha_1)$.
\end{defn}

\begin{prop}[Basic geodesically convex energies]\label{prop-basic-geodesic-convexity}
\index{convexity!geodesic}
	The following formal statements hold on $\Pp_2(\RR^d)$.
	\begin{enumerate}
		\item If $h$ is convex, then $\alpha\mapsto\int h\d\alpha$ is geodesically convex; if $h$ is $\lambda$-strongly convex, it is $\lambda$-geodesically convex.
		\item If $W(x-y)$ is convex as a function of the displacement, then $\alpha\mapsto\frac12\iint W(x-y)\d\alpha(x)\d\alpha(y)$ is geodesically convex.
		\item Shannon entropy $\alpha\mapsto\int\rho\log\rho\,\d x$ is geodesically convex.
\index{Shannon!entropy}
		\item The relative entropy $\KL(\alpha|\gamma)$ with $\d\gamma=e^{-V}\d x/Z$ is $\lambda$-geodesically convex when $V$ is $\lambda$-strongly convex.
\index{entropy!relative}
	\end{enumerate}
\end{prop}
\begin{proof}
	Along a Monge geodesic $X_t=(1-t)X_0+tX_1$, convexity of $h$ gives $h(X_t)\leq(1-t)h(X_0)+t h(X_1)$, and strong convexity gives the additional quadratic term; integrating proves the first claim. The interaction claim follows similarly by applying convexity of $W$ to pairwise differences $X_t-X_t'=(1-t)(X_0-X_0')+t(X_1-X_1')$ and integrating over two independent copies. The entropy claim is McCann's displacement convexity theorem; at the density level it follows from the concavity of the Jacobian determinant under the interpolation of optimal maps. Finally, $\KL(\alpha|\gamma)=\int\rho\log\rho\,\d x+\int V\d\alpha+\mathrm{constant}$, so it is the sum of displacement-convex entropy and a $\lambda$-geodesically convex linear potential.
\index{displacement!convexity}
\index{Jacobian determinant}
\index{convexity!geodesic}
\index{Monge!geodesic}
\end{proof}

\paragraph{Convergence of the flow.}

In general, analyzing \eqref{eq:wassflow-pde} is delicate. The cleanest case is when $f$ is geodesically convex in the sense above. This condition is the Wasserstein analogue of convexity in Euclidean gradient descent.

\begin{prop}[Energy decay for convex Wasserstein flows]\label{prop-convex-wass-flow-rate}
\index{energy!decay}
\index{Wasserstein!flow}
	Assume formally that $f$ is geodesically convex, admits a smooth first variation, and has a minimizer $\alpha^\star$. Let $(\alpha_t)_t$ be a smooth solution of the Wasserstein gradient flow
\index{first variation}
\index{Wasserstein!gradient}
\index{gradient!flow}
\index{Wasserstein!gradient flow}
	\[
		\partial_t\alpha_t+\diverg(\alpha_t v_t)=0,
		\qquad
		v_t=-\Wgrad f(\alpha_t).
	\]
	Then
	\[
		\frac{\d}{\d t} f(\alpha_t)
		=
		-\int \norm{\Wgrad f(\alpha_t)(x)}^2\,\d\alpha_t(x)
		\leq 0.
	\]
	If $T_t$ is the optimal map from $\alpha_t$ to $\alpha^\star$, then
	\[
		f(\alpha_t)-f(\alpha^\star)
		\leq
		-\frac{\d}{\d t}\frac12 \Wass_2^2(\alpha_t,\alpha^\star),
	\]
	and consequently
	\[
		f(\alpha_t)-f(\alpha^\star)
		\leq
		\frac{\Wass_2^2(\alpha_0,\alpha^\star)}{2t}.
	\]
	If $f$ is $\lambda$-geodesically convex with $\lambda>0$, then
	\[
		f(\alpha_t)-f(\alpha^\star)
		\leq
		e^{-2\lambda t}\bigl(f(\alpha_0)-f(\alpha^\star)\bigr).
	\]
\end{prop}

\begin{proof}
	The chain rule and Proposition~\ref{prop-formal-wass-gradient} give
	\[
		\frac{\d}{\d t}f(\alpha_t)
		=
		\int \dotp{\Wgrad f(\alpha_t)(x)}{v_t(x)}\,\d\alpha_t(x)
		=
		-\int \norm{\Wgrad f(\alpha_t)(x)}^2\,\d\alpha_t(x).
	\]
	Geodesic convexity along the geodesic
\index{convexity!geodesic}
	$((1-s)\Id+sT_t)_\sharp\alpha_t$ gives
	\[
		f(\alpha^\star)-f(\alpha_t)
		\geq
		\int \dotp{\Wgrad f(\alpha_t)(x)}{T_t(x)-x}\,\d\alpha_t(x).
	\]
	Since $v_t=-\Wgrad f(\alpha_t)$, this reads
	\[
		f(\alpha_t)-f(\alpha^\star)
		\leq
		\int \dotp{v_t(x)}{T_t(x)-x}\,\d\alpha_t(x).
	\]
	The standard first-variation formula for the squared Wasserstein distance gives
\index{Wasserstein!distance}
	\[
		\frac{\d}{\d t}\frac12\Wass_2^2(\alpha_t,\alpha^\star)
		=
		\int \dotp{x-T_t(x)}{v_t(x)}\,\d\alpha_t(x),
	\]
	which proves the differential inequality. Integrating it from $0$ to $t$ and using the monotonicity of $s\mapsto f(\alpha_s)$ gives
	\[
		t\bigl(f(\alpha_t)-f(\alpha^\star)\bigr)
		\leq
		\int_0^t \bigl(f(\alpha_s)-f(\alpha^\star)\bigr)\,\d s
		\leq
		\frac12\Wass_2^2(\alpha_0,\alpha^\star).
	\]
	If $f$ is $\lambda$-geodesically convex, the Wasserstein analogue of strong convexity gives the slope inequality
	\[
		\int\norm{\Wgrad f(\alpha_t)}^2\d\alpha_t
		\geq
		2\lambda\bigl(f(\alpha_t)-f(\alpha^\star)\bigr).
	\]
	Combining it with the energy dissipation identity yields
\index{energy!dissipation}
	\[
		\frac{\d}{\d t}\bigl(f(\alpha_t)-f(\alpha^\star)\bigr)
		\leq
		-2\lambda\bigl(f(\alpha_t)-f(\alpha^\star)\bigr),
	\]
	and Gronwall's lemma gives the exponential rate.
\end{proof}

\begin{prop}[Convex examples covered by the theory]\label{prop-convex-flow-examples}
	The hypotheses of Proposition~\ref{prop-convex-wass-flow-rate} are satisfied in the following standard cases, at least at the formal smooth level used in this section.
	\begin{enumerate}
		\item For the linear energy $f(\alpha)=\int h\d\alpha$, geodesic convexity holds when $h$ is convex. If $h$ is $\lambda$-strongly convex, then $f$ is $\lambda$-geodesically convex and the flow enjoys the exponential rate of Proposition~\ref{prop-convex-wass-flow-rate}.
\index{convexity!geodesic}
		\item For the interaction energy $f(\alpha)=\frac12\iint W(x-y)\d\alpha(x)\d\alpha(y)$, geodesic convexity holds when $W$ is convex and even. This covers repulsive or attractive pairwise models whose displacement cost has no non-convex wells.
\index{energy!interaction}
		\item The Shannon entropy $f(\alpha)=\int\rho\log\rho\,\d x$ and, more generally, McCann displacement-convex internal energies generate diffusion-type Wasserstein gradient flows.
\index{Wasserstein!gradient}
\index{gradient!flow}
\index{Wasserstein!gradient flow}
\index{Shannon!entropy}
		\item If $\gamma=Z^{-1}e^{-V}\d x$ and $V$ is $\lambda$-strongly convex, then the relative entropy $\KL(\alpha|\gamma)$ is $\lambda$-geodesically convex. Its flow is the Fokker--Planck equation with invariant law $\gamma$.
\index{Fokker-Planck equation}
\index{entropy!relative}
	\end{enumerate}
\end{prop}
\begin{proof}
	Let $(\alpha_t)_t$ be the McCann interpolation between $\alpha_0$ and $\alpha_1$, written with an optimal coupling as $X_t=(1-t)X_0+tX_1$. For a linear energy, Jensen's inequality gives
\index{Jensen inequality}
\index{optimal coupling}
\index{McCann interpolation}
	\[
		h(X_t)\leq(1-t)h(X_0)+t h(X_1),
	\]
	and the strong convexity version gives the additional term $-\frac{\lambda}{2}t(1-t)\norm{X_0-X_1}^2$. Integrating over the optimal coupling proves geodesic convexity and $\lambda$-geodesic convexity.
\index{convexity!geodesic}

	For interaction energies, use two independent copies of the optimal coupling. The pairwise displacement evolves as
\index{optimal coupling}
\index{energy!interaction}
	\[
		X_t-X_t'=(1-t)(X_0-X_0')+t(X_1-X_1').
	\]
	Convexity of $W$ gives the convexity inequality after integration over the product coupling. Evenness of $W$ ensures that the interaction is symmetric in the two particles and matches the usual factor $1/2$ in~\eqref{eq:quadratic-func}.
\index{product!coupling}

	The entropy claim is McCann's displacement-convexity theorem. For smooth positive densities and Brenier maps, it follows from the change-of-variables formula and the concavity of the determinant along positive matrices; the general statement is obtained by approximation. Finally,
\index{change-of-variables}
\index{change-of-variables formula}
\index{displacement!convexity}
\index{Brenier!map}
	\[
		\KL(\alpha|\gamma)=\int\rho\log\rho\,\d x+\int V\d\alpha+\log Z,
	\]
	so it is the sum of the displacement-convex entropy and the $\lambda$-geodesically convex linear potential generated by $V$. Proposition~\ref{prop-convex-wass-flow-rate} then applies to all four cases.
\end{proof}

\paragraph{Convexity and curvature.}

The same language is not restricted to subsets of $\RR^d$. If $(\X,\dist,\mathfrak m)$ is a geodesic metric-measure space, $\Wass_2$ geodesics can be defined by transporting each pair of endpoints along metric geodesics, or more intrinsically by dynamical optimal plans on path space, as discussed in Section~\ref{sec-kantorovich-plan-interpolation}. Given a reference measure $\mathfrak m$, the entropy relative to $\mathfrak m$ is
\index{metric-measure space}
\index{geodesic!space}
\index{entropy!relative}
\index{dynamical optimal plan}
\[
	\mathrm{Ent}_{\mathfrak m}(\alpha)
	\eqdef
	\begin{cases}
		\displaystyle \int_\X \rho\log\rho\,\d\mathfrak m,
		& \text{if } \alpha=\rho\,\mathfrak m,\\
		+\infty,
		& \text{otherwise.}
	\end{cases}
\]
On a smooth Riemannian manifold $(M,g)$, the Ricci curvature tensor $\mathrm{Ric}_g$ is the trace of the Riemann curvature tensor; the lower bound $\mathrm{Ric}_g\geq\lambda g$ means that $\mathrm{Ric}_g(v,v)\geq\lambda |v|_g^2$ for every tangent vector $v$. The fundamental link between curvature and optimal transport is that this tensor lower bound is exactly encoded by geodesic convexity of entropy.
\index{Ricci curvature}
\index{Riemannian manifold}
\index{convexity!geodesic}

\begin{thm}[Ricci curvature and entropy convexity]\label{thm-ricci-entropy-convexity}
\index{Ricci curvature}
\index{entropy!relative}
\index{convexity!geodesic}
	Let $(M,g)$ be a smooth compact connected Riemannian manifold without boundary, and let $\mathfrak m=\mathrm{vol}_g$. For $\lambda\in\RR$, the lower Ricci bound $\mathrm{Ric}_g\geq\lambda g$ holds if and only if $\mathrm{Ent}_{\mathfrak m}$ is $\lambda$-geodesically convex on $(\Pp_2(M),\Wass_2)$.
\end{thm}

This equivalence was developed in the smooth Riemannian setting by Cordero-Erausquin, McCann and Schmuckenschl{\"a}ger and by von Renesse and Sturm~\cite{cordero2001riemannian,vonrenesse2005transport}; it is a central theme of the optimal-transport approach to curvature in Villani's monograph~\cite{Villani09}. Lott--Villani and Sturm then used the same entropy-convexity principle to define synthetic lower Ricci curvature bounds on metric-measure spaces~\cite{lott2009ricci,sturm2006geometry1,sturm2006geometry2}. Outside this convex, curvature-controlled regime, such as in the mean-field neural-network example below, the flow may still be informative but its convergence analysis requires problem-specific arguments.
\index{metric-measure space}
\index{synthetic Ricci curvature}
\index{mean-field!neural network}

\section{Training Two-Layer MLPs as Wasserstein Flows}

Mean-field limits recast the training of wide neural networks as transport of a distribution of neurons. This section shows how the particle ODE of gradient descent becomes a Wasserstein flow in parameter space.
\index{particle ODE}
\index{neuron}
\index{mean-field!limit}

We use $z\in\RR^d$ for the input data and $y\in\RR^{d'}$ for the label. A neuron is a particle
\[
	x=(u,v)\in\RR^d\times\RR^{d'},
\]
where $u$ is the inner weight and $v$ is the outer vector weight. For a scalar nonlinearity $\sigma$, define the vector-valued feature
\[
	\psi(x,z)=v\,\sigma(\dotp{u}{z})\in\RR^{d'}.
\]
The width-$n$ network and its mean-field version are
\[
	G_X(z)=\frac1n\sum_{i=1}^n\psi(x_i,z),
	\qquad
	G_\alpha(z)=\int\psi(x,z)\d\alpha(x),
	\qquad
	\alpha=\frac1n\sum_i\delta_{x_i}.
\]
This formulation removes the artificial ordering of neurons and allows $\alpha$ to be a continuous distribution of infinitely many neurons.
\index{neuron}

Let $\rho$ be a probability distribution on data-label pairs $(z,y)\in\RR^d\times\RR^{d'}$. The population risk is
\[
	f(\alpha)=\int \ell(G_\alpha(z),y)\d\rho(z,y),
\]
and the empirical risk is the special case $\rho=\rho_N\eqdef N^{-1}\sum_{k=1}^N\delta_{(z_k,y_k)}$. Since $\alpha\mapsto G_\alpha$ is linear, $f$ is convex as a function of $\alpha$ whenever $\ell(\cdot,y)$ is convex. For the empirical neuron law $\alpha_X=n^{-1}\sum_i\delta_{x_i}$, the Wasserstein metric induces on particles the rescaled metric $n^{-1}\sum_i\norm{\dot x_i}^2$. The corresponding particle flow is
\index{neuron law}
\[
	\dot x_i=-n\nabla_{x_i}F(X),
	\qquad
	F(X)=f\!\left(\frac1n\sum_i\delta_{x_i}\right),
\]
which is the gradient flow of $F(X)=f(\alpha_X)$ for this Wasserstein particle metric, equivalently Euclidean gradient descent with the time scale multiplied by $n$. It gives a particle discretization of~\eqref{eq:wassflow-pde}.
\index{gradient!flow}

Assume that $\ell$ is differentiable in its first variable. The first variation is
\index{first variation}
\begin{equation}\label{eq-mlp-first-variation-general}
	\delta f(\alpha)(x)
	=
	\int
	\dotp{\nabla_1\ell(G_\alpha(z),y)}{\psi(x,z)}
	\d\rho(z,y),
\end{equation}
and the Wasserstein gradient in parameter space is
\index{Wasserstein!gradient}
\[
	\Wgrad f(\alpha)(x)
	=
	\nabla_x\delta f(\alpha)(x)
	=
	\int
	[D_x\psi(x,z)]^\top\nabla_1\ell(G_\alpha(z),y)
	\d\rho(z,y).
\]
For the squared Euclidean loss $\ell(s,y)=\frac12\norm{s-y}^2$, the energy is the sum of a quadratic interaction and a linear potential:
\begin{equation}\label{eq-mlp-square-loss-quadratic-linear}
	f(\alpha)
	=
	\frac12\iint k(x,x')\d\alpha(x)\d\alpha(x')
	+
	\int g(x)\d\alpha(x)
	+\frac12\int\norm{y}^2\d\rho(z,y),
\end{equation}
with
\begin{equation}\label{eq-mlp-kernel-linear-potential}
	k(x,x')=\int\dotp{\psi(x,z)}{\psi(x',z)}\d\rho(z,y),
	\qquad
	g(x)=-\int\dotp{y}{\psi(x,z)}\d\rho(z,y).
\end{equation}
Thus
\[
	\delta f(\alpha)(x)=\int k(x,x')\d\alpha(x')+g(x),
	\qquad
	\Wgrad f(\alpha)(x)=\int\nabla_x k(x,x')\d\alpha(x')+\nabla_xg(x).
\]
These kernels are generally not convex in the particle variable, so the geodesic-convex convergence theory above does not apply directly.

\begin{figure}[ht]
\centering
\begin{tabular}{@{}cc@{}}
\small neuron trajectories & \small directional concentration \\[-.15em]
\index{neuron}
\includegraphics[width=.34\linewidth]{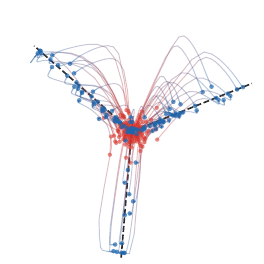} &
\includegraphics[width=.34\linewidth]{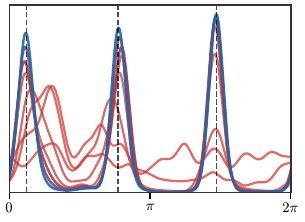}
\end{tabular}
\caption{Mean-field training of a homogeneous two-layer model as transport in neuron space. The left panel shows the Wasserstein particle gradient flow in the reduced homogeneous coordinates $(|u|v_1,|u|v_2)$, with black dashed rays marking the teacher directions. The right panel shows the weighted angular density along a front-loaded sequence of times, colored from red to blue, so that the early concentration of neuron directions is visible. The display follows the rendering of the auxiliary MLP experiment but keeps only the $W_2$ flow, not the spectral-flow comparison.}
\index{neuron}
\index{two-layer neural network}
\index{gradient!flow}
\label{fig:gradflow-mlp-homogeneous-relu}
\end{figure}

\paragraph{Classical convexity and stationarity.}
\index{convexity!classical}
\index{stationarity}
Before using the specific homogeneity mechanism of Chizat and Bach, it is useful to isolate a simpler convex-analytic principle behind many mean-field arguments. Consider an energy of the form
\[
	F(\alpha)
	=
	\frac12\iint k(x,x')\d\alpha(x)\d\alpha(x')
	+
	\int V(x)\d\alpha(x)
	+C,
\]
on probability measures over a parameter domain. Assume that the quadratic part is convex in the classical sense, namely convex for the affine structure of measures:
\index{probability measure}
\[
	Q((1-s)\alpha+s\beta)\leq (1-s)Q(\alpha)+sQ(\beta),
	\qquad
	Q(\alpha)=\frac12\iint k\d\alpha\d\alpha.
\]
This is ordinary convexity of the functional on the convex set of measures, not displacement convexity along $\Wass_2$ geodesics.
\index{convexity!along geodesics}
\index{displacement!convexity}

\begin{prop}[Affine convexity and stationary densities]\label{prop-classical-convex-stationary}
\index{convexity!classical}
\index{stationary density}
	Let $F=Q+\int V\,\d\alpha+C$ be as above, and assume that $Q$ is classically convex. Suppose that a Wasserstein gradient flow for $F$ converges to a measure $\alpha_\infty=\rho_\infty\d x$. Assume also the standard regularity needed to pass to the limit in the first variation, and assume that the support and positivity of $\rho_\infty$ allow the stationary condition to be tested against all admissible zero-mass density perturbations. In the form needed here, assume that this stationarity yields, for every competitor $\beta$, the variational inequality
\index{stationary condition}
\index{first variation}
\index{stationarity}
\index{Wasserstein!gradient}
\index{gradient!flow}
\index{Wasserstein!gradient flow}
	\[
		\int \delta F(\alpha_\infty)(x)\d(\beta-\alpha_\infty)(x)\geq0.
	\]
	Then $\alpha_\infty$ is a global minimizer of $F$.
\index{global!minimizer}
\end{prop}
\begin{proof}[Proof sketch]
	The dissipation identity for the gradient flow gives stationarity of the limit: formally, after passing to the limit,
\index{gradient!flow}
	\[
		\int \norm{\nabla \delta F(\alpha_\infty)}^2\d\alpha_\infty=0.
	\]
	Without such a support and positivity assumption, this identity only controls the first variation on the region explored by the limit. The density hypothesis allows one to test against sufficiently many signed density perturbations of total mass zero. By approximation and the assumed regularity, this yields the displayed first-order variational inequality for arbitrary competitors $\beta$. Classical convexity of $F$ in the affine variable $\alpha$ then gives the usual subgradient inequality
\index{subgradient inequality}
\index{first variation}
\index{convexity!classical}
	\[
		F(\beta)\geq F(\alpha_\infty)
		+
		\int \delta F(\alpha_\infty)\d(\beta-\alpha_\infty)
		\geq F(\alpha_\infty).
	\]
	Thus no competitor has smaller energy. For square-loss two-layer mean-field models, \eqref{eq-mlp-square-loss-quadratic-linear} is exactly of this quadratic-plus-linear form, and positive semidefiniteness of the induced kernel $k$ is the classical convexity assumption.
\index{linear form}
\index{two-layer neural network}
\index{convexity!classical}
\end{proof}

The mean-field description of two-layer training was developed in several works, including~\cite{ChizatBach2018OverparameterizedOT,MeiMontanariNguyen2018MeanFieldNN}. The distinctive contribution of Chizat and Bach is a global-convergence analysis for positively homogeneous networks without adding an explicit regularizer or relying on noisy SGD to create a Laplacian term. The following formal statement isolates the core mechanism and ignores the technical issues due to ReLU non-smoothness, support propagation and compactness.
\index{two-layer neural network}

\begin{prop}[Formal global optimality for two-homogeneous mean-field flows]\label{prop-formal-chizat-bach}
\index{global!optimality}
	Assume that the feature is positively two-homogeneous in the neuron variable,
\index{neuron}
	\[
		\psi(\lambda x,z)=\lambda^2\psi(x,z)
		\qquad(\lambda>0),
	\]
	and that $f(\alpha)=J(G_\alpha)$ with $J$ convex and differentiable as a functional of the predictor. Let $\alpha$ be a smooth stationary point of the Wasserstein flow, so that $\nabla_x\delta f(\alpha)(x)=0$ on $\supp(\alpha)$. Assume also full directional support: for every nonzero direction $\omega$, the support of $\alpha$ intersects the ray $\{\lambda\omega:\lambda>0\}$. Then $\alpha$ is a global minimizer of $f$ over the mean-field model class.
\index{stationarity}
\index{global!minimizer}
\end{prop}
\begin{proof}
	Write
	\[
		h_\alpha(x)=\delta f(\alpha)(x)
		=
		\left\langle \nabla J(G_\alpha),\psi(x,\cdot)\right\rangle_\rho .
	\]
	By two-homogeneity of $\psi$, one has $h_\alpha(\lambda x)=\lambda^2h_\alpha(x)$. Normalize a nonzero direction $\omega$ and choose $r_\omega>0$ with $r_\omega\omega\in\supp(\alpha)$. Stationarity gives a zero radial derivative at this point:
	\[
		0=\frac{\d}{\d r}h_\alpha(r\omega)\bigg|_{r=r_\omega}
		=2r_\omega h_\alpha(\omega).
	\]
	Hence $h_\alpha(\omega)=0$ for every direction $\omega$, and by homogeneity $h_\alpha(x)=0$ for every $x$.

	For any competitor $\beta$, convexity of $J$ gives
	\[
		f(\beta)-f(\alpha)
		\geq
		\int h_\alpha(x)\d(\beta-\alpha)(x)=0.
	\]
	Thus no competitor has smaller risk. The rigorous theorem replaces the full directional support assumption by propagation and overparameterization hypotheses ensuring that a negative descent direction would be present in the support and would contradict stationarity.
\index{stationarity}
\end{proof}


\chapter{Generative Models via Transportation}
\index{generative model}
\label{sec-generative-models-transportation}

The preceding gradient-flow calculus is variational. Modern machine-learning models often use the same transportation language more broadly: one may prescribe an interpolation and regress its velocity, fit a one-step generator to a descent field, or view network depth as a continuous transport of token measures. The examples below separate what is genuinely a Wasserstein gradient flow from what is a transportation dynamics with a useful geometric interpretation.
\index{Wasserstein!gradient}
\index{gradient!flow}
\index{Wasserstein!gradient flow}
\index{token measure}

\section{Generative Models via Flow Matching}
\index{flow!matching}
\index{generative model}

Flow matching constructs a generative map by learning the velocity field of an interpolation. The key computational insight is that a constrained continuity-equation problem can be trained by an unconstrained regression.
\index{velocity field}
\index{continuity equation}
\index{flow!matching}

Generative models aim to build a transportation map $T$ between a reference distribution $\alpha$ (typically an isotropic Gaussian) and the target data distribution $\beta$. Since such reference measures are non-atomic, a measurable map with $T_\sharp\alpha=\beta$ exists on standard Borel spaces, for instance by identifying both probability spaces with the unit interval and using a quantile-type rearrangement. This abstract existence statement is much weaker than having an explicit and numerically stable construction of $T$.
\index{reference!distribution}
\index{reference!measure}
\index{generative model}
Optimal transport is one approach to achieving this, but it is computationally expensive and raises questions about how to estimate it from samples. A different route is to prescribe an interpolation between noise and data, learn its velocity, and obtain $T$ by integrating a time-dependent vector field $v_t$. This point of view sits at the meeting point of two literatures, surveyed from a transport perspective in~\cite{Peyre2026OptimalDiffusionTransports}. The diffusion branch builds on score matching~\cite{Hyvarinen2005ScoreMatching}, denoising score matching~\cite{Vincent2011DenoisingScoreMatching}, nonequilibrium noising chains~\cite{SohlDickstein2015DeepUnsupervised}, denoising diffusion probabilistic models~\cite{Ho2020DDPM}, score-based generative modeling~\cite{Song2019ScoreMatchingGenerative}, and the continuous-time score-SDE/probability-flow formulation~\cite{Song2021ScoreSDE}. The deterministic regression branch was introduced, essentially in parallel, under three closely related names: flow matching~\cite{Lipman2022FlowMatching}, rectified flow~\cite{Liu2023RectifiedFlow}, and stochastic interpolants~\cite{Albergo2025StochasticInterpolants}. In all three cases, the computational object is a velocity field whose regression loss avoids simulating the learned ODE during training.
\index{denoising diffusion probabilistic model}
\index{score-SDE}
\index{probability-flow ODE}
\index{score-based generative modeling}
\index{velocity field}
\index{integral probability metric}
\index{flow!matching}
\index{time-dependent vector field}
\index{denoising score matching}
\index{stochastic!interpolant}
\index{regression loss}
\index{flow!rectified}
\index{flow!matching}
This vector field $v_t$ is obtained by constructing an interpolation $\alpha_t$ and then finding $v_t$ using the least-squares formula~\eqref{eq:least-square-field-explicit}. As we will explain, for a specific class of interpolation (obtained by a parametric push-forward), this $v_t$ can be obtained by avoiding explicitly inverting a Laplacian and instead computing a simple conditional expectation. This conditional expectation can itself be estimated by solving another least-squares problem, but this time unconstrained, making the estimation feasible from finite samples of $\alpha$ and $\beta$.
\index{conditional!expectation}
\index{push-forward}

\paragraph{Stochastic interpolant.}
\index{stochastic!interpolant}

We assume that $\alpha_t$ is defined via a ``projection'' (in a loose sense) of a latent distribution $\pi \in \Pp(\RR^{d'})$, using an operator $P_t : \RR^{d'} \to \RR^d$ where $d' \gg d$, i.e.
\begin{equation}\label{eq:interp-coupling}
	\forall t \in [0,1], \quad \alpha_t := (P_t)_\sharp \pi.
\end{equation}
The basic two-endpoint construction already covers most flow-matching paths used in practice.

\begin{example}[Linear two-endpoint stochastic interpolants]
	Set $d'=2d$, write $(x,y)\in\RR^d\times\RR^d$, and choose $P_0(x,y)=x$ and $P_1(x,y)=y$. If $\pi$ has marginals $(\alpha_0,\alpha_1)$, then $\alpha_t=(P_t)_\sharp\pi$ interpolates between the two endpoint laws. The simplest choices are the independent coupling $\pi=\alpha_0\otimes\alpha_1$ and the straight path
	\[
		P_t(x,y)=(1-t)x+ty.
	\]
	With this linear path and an arbitrary coupling $\pi$, the regression below is the common core of flow matching and rectified flow: Lipman et al. emphasize conditional probability paths and simulation-free training of continuous normalizing flows, while rectified flow emphasizes straight couplings, reflow, and the possibility of reducing transport costs and discretization error~\cite{Lipman2022FlowMatching,Liu2023RectifiedFlow}.

	More complex constructions are possible when sampling from $\pi$ remains simple; stochastic interpolants add latent variables or noise, connecting deterministic flows, probability-flow ODEs and diffusion SDEs~\cite{Albergo2025StochasticInterpolants}.
\end{example}
\index{continuous normalizing flow}
\index{reflow}
\index{discretization error}
\index{latent variable}
\index{conditional probability}
\index{probability-flow ODE}
\index{normalizing flow}
\index{probability path}
\index{integral probability metric}
\index{flow!matching}
\index{flow!rectified}
\index{flow!probability ODE}
\index{probabilistic coupling}
\index{stochastic!interpolant}
\index{probability path}
\index{trivial coupling}
\index{flow!rectified}
\index{flow!matching}

If $\pi = \alpha \otimes \beta$ and $\alpha = \frac{1}{n} \sum_i \delta_{x_i}$, $\beta = \frac{1}{m} \sum_j \delta_{y_j}$, then $\alpha_t$ consists of $n \times m$ Dirac masses
\index{Dirac mass}
\begin{equation*}
    \alpha_t = \frac{1}{nm} \sum_{i,j} \delta_{P_t(x_i,y_j)}.
\end{equation*}
If $\pi = (\Id, T)_\sharp \alpha$ is a Brenier-type coupling, then $\alpha_t = ((1-t)\Id + tT)_\sharp \alpha$ is the so-called McCann OT interpolation.

\begin{figure}[H]
\centering
\begin{tabular}{@{}ccc@{}}
\small product pairing & \small OT pairing & \small curved bridge \\[-.15em]
\includegraphics[width=.30\linewidth]{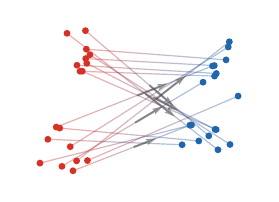} &
\index{flow!matching}
\includegraphics[width=.30\linewidth]{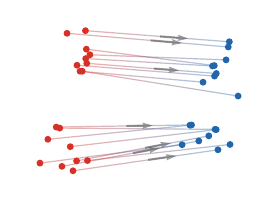} &
\includegraphics[width=.30\linewidth]{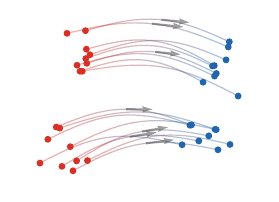}
\end{tabular}
\caption{Flow matching interpolants between the same empirical source and target measures. A product-style random pairing produces crossing paths, an OT pairing gives direct displacement rays, and a curved bridge changes the path geometry while keeping the same endpoints. Gray arrows mark representative midpoint velocities $\partial_tP_t$.}
\index{flow!matching}
\label{fig:generative-flow-matching-interpolants}
\end{figure}

\paragraph{Flow matching formula.}
\index{flow!matching}

This interpolation is not directly useful for sampling from $\beta$, but it can be used to define a flow field $v_t$ so that the Eulerian advection equation~\eqref{eq:eulerian-advection} holds.
This flow field is computed by solving an unconstrained least-squares problem, or equivalently, it is a conditional expectation.
\index{conditional!expectation}

\begin{prop}[Flow matching vector field]\label{prop-flow-matching-vector-field}
\index{flow!matching}
	For each fixed $t$, assume $\partial_tP_t\in L^2(\pi;\RR^d)$. The solution of the flow-matching problem over measurable fields $v_t:\RR^d\to\RR^d$
\index{assignment problem}
\index{flow!matching}
	\begin{equation}
		\min_{v_t} \int_{\RR^{d'}} \norm{v_t(P_t(u)) - [\partial_t P_t](u)}^2 \, \d\pi(u). \label{eq:flow-matching}
	\end{equation}
	Equivalently, the minimizer is characterized $\alpha_t$-almost everywhere by the conditional expectation
\index{conditional!expectation}
	\begin{equation}\label{eq:flow-match-conditional}
		v_t(z) = \EE_{u \sim \pi} \big( [\partial_t P_t](u) \, \big| \, z = P_t(u) \big).
	\end{equation}
	Then the pair $(\alpha_t,v_t)$ satisfies the continuity equation~\eqref{eq:eulerian-advection}.
\index{continuity equation}
\end{prop}

\begin{proof}
	We first recall the two equivalent ways of writing the interpolated measure. Formally, one may write
	\[
		\alpha_t(z)=\int_{\RR^{d'}}\delta(z-P_t(u))\,\d\pi(u),
	\]
	while the rigorous meaning is that, for every smooth test function $\varphi$,
	\begin{equation}\label{eq:flow-matching-pushforward-test}
\index{push-forward}
		\int_{\RR^d}\varphi(z)\,\d\alpha_t(z)
		=
		\int_{\RR^{d'}}\varphi(P_t(u))\,\d\pi(u).
	\end{equation}
	The minimizer in~\eqref{eq:flow-matching} is the orthogonal projection in $L^2(\pi;\RR^d)$ of the latent velocity $\partial_tP_t(u)$ onto the closed subspace of functions that depend on $u$ only through $P_t(u)$. This projection is the conditional expectation~\eqref{eq:flow-match-conditional}. Formally, this can be read as
\index{conditional!expectation}
\index{flow!matching}
	\[
		v_t(z)=\frac{1}{\alpha_t(z)}
		\int_{\RR^{d'}}\delta(z-P_t(u))[\partial_tP_t](u)\,\d\pi(u),
	\]
	and rigorously it means that, for every smooth test vector field $m$,
	\begin{equation}\label{eq:v_t}
		\int \dotp{m(z)}{v_t(z)} \, \d\alpha_t(z)
		=
		\int \dotp{m(P_t(u))}{[\partial_t P_t](u)} \, \d\pi(u).
	\end{equation}

	We now prove that this field transports the curve $(\alpha_t)_t$. The weak form of
	$\partial_t\alpha_t+\diverg(\alpha_t v_t)=0$ is that, for every smooth scalar test function $\varphi$,
	\begin{equation}\label{eq:flow-matching-weak-target}
		\frac{\d}{\d t}\int\varphi(z)\,\d\alpha_t(z)
		-
		\int\dotp{v_t(z)}{\nabla\varphi(z)}\,\d\alpha_t(z)
		=0.
	\end{equation}
	Using~\eqref{eq:flow-matching-pushforward-test} and differentiating under the integral sign gives
\index{push-forward}
\index{flow!matching}
	\begin{equation}\label{eq:flow-matching-test-derivative}
		\frac{\d}{\d t}\int \varphi(z)\d\alpha_t(z)
		=
		\int \dotp{\nabla\varphi(P_t(u))}{[\partial_t P_t](u)}\d\pi(u).
	\end{equation}
	On the other hand, applying~\eqref{eq:v_t} with $m=\nabla\varphi$ gives
	\begin{equation}\label{eq:flow-matching-velocity-test}
		\int\dotp{v_t(z)}{\nabla\varphi(z)}\,\d\alpha_t(z)
		=
		\int \dotp{\nabla\varphi(P_t(u))}{[\partial_t P_t](u)}\d\pi(u).
	\end{equation}
	Comparing~\eqref{eq:flow-matching-test-derivative} and~\eqref{eq:flow-matching-velocity-test} yields~\eqref{eq:flow-matching-weak-target}, which is the desired continuity equation.
\index{continuity equation}
\index{flow!matching}
\end{proof}

The conditional expectation in~\eqref{eq:flow-match-conditional} has a simple measure-theoretic meaning. Let $\alpha_t=(P_t)_\sharp\pi$ and define the vector-valued measure $m_t$ on $\RR^d$ by
\[
	\int_{\RR^d}\dotp{\psi(z)}{\d m_t(z)}
	\eqdef
	\int_{\RR^{d'}}\dotp{\psi(P_t(u))}{[\partial_tP_t](u)}\d\pi(u)
\]
for every bounded continuous vector field $\psi$. Since $\alpha_t(A)=0$ implies $\pi(P_t^{-1}(A))=0$, one has $m_t\ll\alpha_t$. The Radon--Nikodym decomposition of $m_t$ with respect to $\alpha_t$ is therefore
\[
	\d m_t(z)=v_t(z)\d\alpha_t(z),
	\qquad
	v_t=\frac{\d m_t}{\d\alpha_t}.
\]
In the language of Lebesgue decomposition, the flux measure $m_t$ has only an absolutely continuous part with respect to $\alpha_t$ and no singular part; the conditional expectation is precisely this density.
Equivalently, disintegrating $\pi$ with respect to the map $P_t$ gives $\pi(\d u)=\pi_{t,z}(\d u)\alpha_t(\d z)$, where $\pi_{t,z}$ is supported on the fiber $\{u\,:\,P_t(u)=z\}$, and
\[
	v_t(z)=\int_{\{P_t(u)=z\}}[\partial_tP_t](u)\d\pi_{t,z}(u).
\]
Thus the solution of \eqref{eq:flow-matching} is the conditional expectation of the velocities $\partial_t P_t$: intuitively, $v_t(z)$ is the average velocity of all trajectories passing through $z$.
\index{conditional!expectation}
\index{Radon-Nikodym derivative}
\index{disintegration}
Numerically, $(x,t) \to v_t(x)$ can be parameterized by a neural network (e.g., a U-Net for vision tasks) and estimated using stochastic gradient descent on the objective in \eqref{eq:flow-matching}.
\index{stochastic!gradient}
\index{flow!matching}

\begin{alg}[Flow matching regression and sampling]\label{alg:flow-matching-regression}
\index{flow!matching}
\textbf{Input:} Interpolant $P_t(u)$, training source $u\sim\pi$, parametrized field $v_\theta(t,z)$, training steps $N$.

\textbf{Output:} Learned sampler $X_0\mapsto X_1$.

\textbf{Training:}

\textbf{For} $q=1,\ldots,N$ \textbf{do}:
\begin{algblock}
\textbf{Draw} $t_q\sim\mathrm{Unif}(0,1)$ and $u_q\sim\pi$.

\textbf{Set} $z_q=P_{t_q}(u_q)$ and $w_q=\partial_tP_t(u_q)|_{t=t_q}$.

\textbf{Update} $\theta$ by one stochastic-gradient step on \(\norm{v_\theta(t_q,z_q)-w_q}^2.\)

\end{algblock}

\textbf{Sampling:}

\textbf{Draw} $X_0\sim\alpha_0$.

\textbf{Integrate}
\(\dot X_t=v_\theta(t,X_t), \qquad t\in[0,1].\)
\textbf{Return} $X_1$.
\end{alg}

For the exact field $v_t$, integrating the ODE $\dot{x}=v_t(x)$ defines a transport map $T_t$. If $v_t$ is regular enough, or more generally if the continuity equation has a unique solution for this velocity, then $(T_t)_\sharp\alpha_0=\alpha_t$. Thus the same interpolation as~\eqref{eq:interp-coupling} is represented by a deterministic flow rather than by the original coupling.
\index{transport map}
\index{continuity equation}
The sampling procedure consists in first drawing $X_0 \sim \alpha$, and then integrating the ODE $\dot{X}_t = v_t(X_t)$ starting with $X_{t=0} = X_0$.
In the ideal exact-field limit, the resulting $X_{t=1}$ is distributed according to $\alpha_1 = \beta$.

\paragraph{Connection with diffusion models.}
\index{diffusion model}

In the special case where $P_t(x,y)=(1-t)x+ty$ is a linear interpolation and $\pi = \alpha \otimes \beta$, the curve $\alpha_t$ is a convolution of rescaled versions of $\alpha_0$ and $\alpha_1$. The flow-matching problem~\eqref{eq:flow-matching} becomes
\index{assignment problem}
\index{flow!matching}
\[
    \min_{(v_t)_t} \int_{\RR^{d} \times \RR^d} \norm{v_t( (1-t)x+t y ) - (y-x) }^2 \, \d\alpha_0(x) \d\alpha_1(y).
\]
When one endpoint is an isotropic Gaussian, this construction is closely related to the probability-flow formulation of diffusion models, up to the usual change of time parametrization~\cite{Song2021ScoreSDE}. This is why flow matching can be viewed both as a deterministic alternative to diffusion training and as a common language for diffusion paths, OT-inspired paths, and rectified paths~\cite{Lipman2022FlowMatching,Liu2023RectifiedFlow,Albergo2025StochasticInterpolants}. The next two propositions are written in the noising direction, from a data law $\alpha$ to a Gaussian; reversing time gives the corresponding sampling flow. They also give an explicit closed form for $v_t$ and show that it is a gradient field.
\index{probability-flow ODE}
\index{integral probability metric}
\index{diffusion model}
\index{flow!matching}
In this setting, $v_t$ is also the solution of the constrained least-squares problem~\eqref{eq:least-square-field-explicit}. The regression~\eqref{eq:flow-matching} is computationally simpler because the continuity equation has already been enforced by the chosen interpolant.
\index{continuity equation}
To prove this, we rely on Tweedie's formula, which expresses the optimal Gaussian denoiser through the score, i.e. the gradient of the log-density.
\index{Gaussian denoiser}
\index{score function}
\index{Tweedie identity}

\begin{proposition}[Tweedie identity]\label{prop:Tweedie}
Let $W$ be a random vector in $\RR^{d}$ with density $\beta$.
For $\sigma>0$, observe
\[
Z \;=\; W + \sigma\,\varepsilon,
\quad\text{where } \varepsilon \sim \Gaussian(0,I_{d})
\text{ is independent of } W .
\]
Denote by
\[
	\beta_\sigma \;=\; \beta * \Gaussian\bigl(0,\sigma^{2}I_{d}\bigr)
\]
the density of $Z$.
Then
\[
\EE\bigl[\,W \mid Z=z\bigr]
      \;=\; z \;+\;\sigma^{2}\,\nabla \log \beta_\sigma(z)
\qquad\text{for all } z \in \RR^{d}.
\]
\end{proposition}

\begin{proof}
Bayes' rule gives the conditional density
$
p_{W|Z}(w\mid z)
= \dfrac{\beta(w)\,\varphi_\sigma(z-w)}{\beta_\sigma(z)}
$
with $\varphi_\sigma$ the $\Gaussian(0,\sigma^{2}I_{d})$ density.
Hence
\[
\EE[W\mid Z=z]
= \frac{1}{\beta_\sigma(z)}
      \int_{\RR^{d}} w\,
             \beta(w)\,\varphi_\sigma(z-w)\,\d w .
\]
Differentiating the Gaussian convolution under the integral sign and using
$
\nabla_z\varphi_\sigma(z-w)
     = -\sigma^{-2}(z-w)\,\varphi_\sigma(z-w)
$
yields
\[
\nabla_z\beta_\sigma(z)
= \int \beta(w)\,\nabla_z\varphi_\sigma(z-w)\,\d w
= -\sigma^{-2}\Bigl(z-\EE[W\mid Z=z]\Bigr)\,\beta_\sigma(z).
\]
Rearranging finishes the proof.
\end{proof}

\begin{proposition}[Gaussian-endpoint flow-matching field]\label{prop:flow}
\index{flow!matching}
Let $X\sim\alpha$ and $Y\sim\Gaussian(0,I_{d})$ be independent.
For $t\in(0,1)$ set
\[
Z_t \;=\; (1-t)\,X + t\,Y,
\qquad
\alpha_t =\text{Law}(Z_t).
\]
The regression minimizer $v^\star:\RR^d\times(0,1)\to\RR^d$ of
\[
\min_{v}\;\int_{0}^{1}\!
         \iint_{\RR^{d}\times\RR^{d}}
              \bigl|y-x-v\bigl((1-t)x+t y,t\bigr)\bigr|^{2}\,
              \d\alpha(x)\,\d\Gaussian(y)\,\d t
\]
is
\[
v^\star(x,t)
= -\frac{1}{1-t}\,x \;-\; \frac{t}{1-t}\,\nabla\log\alpha_t(x)
\qquad (x\in\RR^{d},\;t\in(0,1)).
\]
In particular, for each $t\in(0,1)$ this field is a gradient field,
\[
	v^\star(\cdot,t)=-\nabla
	\left(
		\frac{\norm{\cdot}^2}{2(1-t)}
		+\frac{t}{1-t}\log\alpha_t
	\right).
\]
\end{proposition}

\begin{proof}
Fix $t\in(0,1)$ and write $W=(1-t)X$, $\sigma=t$, so that
$Z_t = W + \sigma\,Y$ matches the setting of Proposition~\ref{prop:Tweedie}.
\index{Tweedie identity}
Conditional expectations satisfy
\index{conditional!expectation}
$
v^\star(z,t)
= \EE[Y-X\mid Z_t=z]
= \frac{1}{t}\,\EE[Z_t-W\mid Z_t=z]
  -\,\frac{1}{1-t}\,\EE[W\mid Z_t=z].
$
Applying Proposition~\ref{prop:Tweedie} to $\EE[W\mid Z_t=z]$ and
\index{Tweedie identity}
noting $\EE[Y\mid Z_t=z]
      = -\,t\,\nabla\log\alpha_t(z)$
gives the claimed formula.
\end{proof}

\begin{figure}[ht]
\centering
\begin{tabular}{@{}cc@{}}
\small forward noising & \small reverse probability flow \\[-.15em]
\includegraphics[width=.43\linewidth]{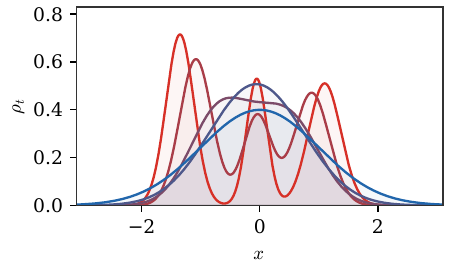} &
\includegraphics[width=.43\linewidth]{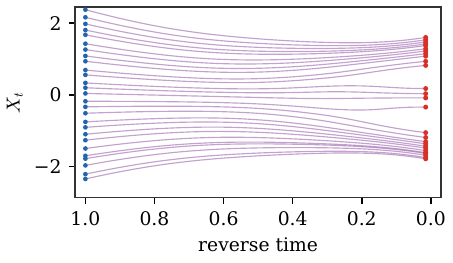}
\end{tabular}
\caption{One-dimensional diffusion bridge for a Gaussian-mixture data law. The forward path $Z_t=(1-t)X+tY$ smooths the red data density toward a blue Gaussian endpoint. Reversing the probability-flow ODE transports a denser set of blue noise samples back toward the data modes, making the splitting of trajectories across mixture components visible.}
\index{Gaussian mixture}
\index{probability-flow ODE}
\index{flow!probability ODE}
\label{fig:generative-diffusion-1d-forward-backward}
\end{figure}

The same probability-flow intuition is visible in two dimensions. For a discrete data law, or more generally for a Gaussian mixture, the noising density is a Gaussian mixture whose score can be evaluated explicitly. This makes it possible to draw backward trajectories without training a neural network. In the plots below, the Gaussian endpoint has covariance $\sigma^2\Id$ to keep the geometry visible at the scale of the three atoms. For a scalar noising schedule $Z_t=a_tX+b_tY$, the intermediate law has component centers $a_t c_j$ and covariance $(b_t\sigma)^2\Id$. For the linear bridge, $p_t(z)=\sum_j w_j\Gaussian((1-t)c_j,(t\sigma)^2\Id)$, with $s_t=\nabla\log p_t$, and the scaled version of Proposition~\ref{prop:flow} gives $v_t(z)=-(z+t\sigma^2s_t(z))/(1-t)$.
\index{noising schedule}
\index{Gaussian mixture}
\index{score function}

\begin{alg}[Exact probability-flow sampling for a Gaussian mixture]\label{alg:gaussian-mixture-probability-flow-sampling}
\index{probability-flow ODE}
\index{Gaussian mixture}
\textbf{Input:} Gaussian-mixture data law, schedule $(a_t,b_t)$, noise level $\sigma$, number of samples $R$.

\textbf{Output:} Backward samples $(Z_0^{(r)})_r$.

\textbf{Define} the noising variable:
\(Z_t=a_tX+b_tY, \qquad Y\sim\Gaussian(0,\sigma^2\Id).\)

\textbf{Compute} closed-form mixture density $p_t$ and score $s_t=\nabla\log p_t$.

\textbf{Set} probability-flow velocity:
\(v_t(z)=\frac{a'_t}{a_t}z+ \left(\frac{a'_tb_t^2}{a_t}-b'_tb_t\right)\sigma^2s_t(z).\)

\textbf{For} $r=1,\ldots,R$ \textbf{do}:
\begin{algblock}

\textbf{Draw} $Z_1^{(r)}$ from the Gaussian endpoint.

\textbf{Integrate} $\dot Z_t^{(r)}=v_t(Z_t^{(r)})$ backward from $t=1$ to $t=0$.

\end{algblock}
\algreturnskip
\textbf{Return} $(Z_0^{(r)})_r$.
\end{alg}

\begin{figure}[ht]
\centering
\begin{tabular}{@{}c@{}}
\small linear bridge, $a_t=1-t$, $b_t=t$ \\[-.15em]
\includegraphics[width=.72\linewidth]{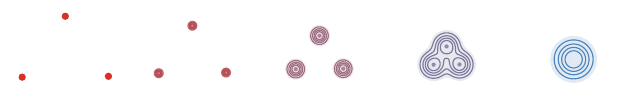} \\[.15em]
\small OU/variance-preserving bridge, $a_t=\cos(\pi t/2)$, $b_t=\sin(\pi t/2)$ \\[-.15em]
\index{bridge!OU}
\index{bridge!variance-preserving}
\includegraphics[width=.72\linewidth]{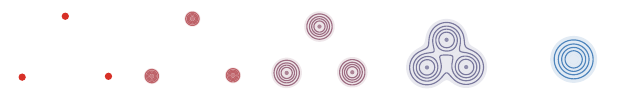}
\end{tabular}
\caption{Two-dimensional noising paths from three Dirac masses to a single Gaussian. The top row shows the linear interpolation $Z_t=(1-t)X+tY$, whose component centers move linearly toward the origin and whose covariance grows like $(t\sigma)^2\Id$. The bottom row uses the variance-preserving Ornstein--Uhlenbeck coefficients $a_\tau=e^{-\tau}$ and $b_\tau=\sqrt{1-e^{-2\tau}}$, reparametrized by $\tau=-\log\cos(\pi t/2)$ so that $a_t=\cos(\pi t/2)$ and $b_t=\sin(\pi t/2)$. It has the same endpoints but a different speed of contraction and noising.}
\index{Ornstein-Uhlenbeck process}
\index{bridge!variance-preserving}
\label{fig:generative-diffusion-2d-forward-backward}
\end{figure}

\paragraph{When is the induced map optimal?}

Integrating the learned velocity gives a deterministic map from $\alpha_0$ to $\alpha_1$, but this map is not automatically the Brenier optimal map. It is optimal only in special cases where the accumulated flow remains the gradient of a convex potential. The Gaussian product-coupling case already shows the precise obstruction: the interpolated covariances are simple, the velocity is affine, but the terminal map can contain a hidden rotational part. This phenomenon, and its extensions to rectified flows and mixtures, is analyzed in depth in~\cite{HertrichChambolleDelon2025RectifiedOT}.
\index{convex!potential}
\index{product!coupling}
\index{flow!rectified}

\begin{prop}[Gaussian flow matching and optimality]\label{prop-gaussian-flow-matching-optimality}
\index{Gaussian!flow matching}
	Let $\Sigma_0,\Sigma_1\succ0$ and let $X_0\sim\Gaussian(0,\Sigma_0)$ and $X_1\sim\Gaussian(0,\Sigma_1)$ be independent. Consider the linear flow-matching interpolation
\index{flow!matching}
	\[
		Z_t=(1-t)X_0+tX_1,
		\qquad
		\alpha_t=\operatorname{Law}(Z_t)=\Gaussian(0,\Sigma_t),
	\]
	where
	\begin{equation}\label{eq-gaussian-product-fm-covariance}
		\Sigma_t=(1-t)^2\Sigma_0+t^2\Sigma_1.
	\end{equation}
	Then the exact flow-matching velocity is affine, $v_t(z)=A_tz$, with
\index{flow!matching}
	\begin{equation}\label{eq-gaussian-product-fm-velocity}
		A_t=\bigl(t\Sigma_1-(1-t)\Sigma_0\bigr)\Sigma_t^{-1}.
	\end{equation}
	The induced flow map $T_t^{\rm FM}$ from $\alpha_0$ to $\alpha_t$ is
\index{flow!map}
	\begin{equation}\label{eq-gaussian-product-fm-map}
		T_t^{\rm FM}
		=
		\Sigma_0^{1/2}
		\Bigl((1-t)^2\Id+t^2\Sigma_0^{-1/2}\Sigma_1\Sigma_0^{-1/2}\Bigr)^{1/2}
		\Sigma_0^{-1/2}.
	\end{equation}
	In particular,
	\begin{equation}\label{eq-gaussian-product-fm-terminal-map}
		T_1^{\rm FM}
		=
		\Sigma_0^{1/2}
		\bigl(\Sigma_0^{-1/2}\Sigma_1\Sigma_0^{-1/2}\bigr)^{1/2}
		\Sigma_0^{-1/2}.
	\end{equation}
	This terminal map coincides with the quadratic optimal transport map
\index{transport map}
	\begin{equation}\label{eq-gaussian-brenier-map-comparison}
		T^{\rm OT}
		=
		\Sigma_0^{-1/2}
		\bigl(\Sigma_0^{1/2}\Sigma_1\Sigma_0^{1/2}\bigr)^{1/2}
		\Sigma_0^{-1/2}
	\end{equation}
	if and only if $\Sigma_0\Sigma_1=\Sigma_1\Sigma_0$.
\end{prop}

\begin{proof}
	The conditional-expectation formula gives
\index{conditional!expectation}
	\[
		v_t(z)=\EE[X_1-X_0\mid Z_t=z].
	\]
	Since all variables are jointly Gaussian, this conditional expectation is linear and
\index{conditional!expectation}
	\[
		v_t(z)
		=
		\operatorname{Cov}(X_1-X_0,Z_t)\operatorname{Cov}(Z_t)^{-1}z
		=
		\bigl(t\Sigma_1-(1-t)\Sigma_0\bigr)\Sigma_t^{-1}z,
	\]
	which proves~\eqref{eq-gaussian-product-fm-velocity}. To solve the characteristic equation, whiten the source by setting
	\[
		C=\Sigma_0^{-1/2}\Sigma_1\Sigma_0^{-1/2},
		\qquad
		\widetilde Z_t=\Sigma_0^{-1/2}Z_t.
	\]
	In these coordinates the source covariance is $\Id$ and
	\[
		\widetilde\Sigma_t=(1-t)^2\Id+t^2C.
	\]
	Because $\Id$ and $C$ commute, the affine flow map in whitened coordinates is simply $\widetilde T_t=\widetilde\Sigma_t^{1/2}$. Indeed,
\index{flow!map}
	\[
		\frac{\d}{\d t}\widetilde\Sigma_t^{1/2}
		=
		\bigl(tC-(1-t)\Id\bigr)\widetilde\Sigma_t^{-1/2},
	\]
	which is exactly the equation $\dot{\widetilde T}_t=\widetilde A_t\widetilde T_t$ with $\widetilde T_0=\Id$. Returning to the original coordinates gives~\eqref{eq-gaussian-product-fm-map}, and $t=1$ gives~\eqref{eq-gaussian-product-fm-terminal-map}.

	Both $T_1^{\rm FM}$ and $T^{\rm OT}$ push $\Gaussian(0,\Sigma_0)$ to $\Gaussian(0,\Sigma_1)$. The Brenier map between nondegenerate Gaussians is the unique symmetric positive definite linear map with this property. Hence $T_1^{\rm FM}=T^{\rm OT}$ if and only if $T_1^{\rm FM}$ is symmetric positive definite. The map $T_1^{\rm FM}$ is similar to $C^{1/2}$, so if it is symmetric then it is automatically positive definite. It remains to characterize symmetry. Since $C^{1/2}$ is symmetric positive definite,
\index{Brenier!map}
	\[
		(T_1^{\rm FM})^\top
		=
		\Sigma_0^{-1/2}C^{1/2}\Sigma_0^{1/2}.
	\]
	Thus symmetry of $T_1^{\rm FM}$ is equivalent to
	$\Sigma_0 C^{1/2}=C^{1/2}\Sigma_0$, hence to $\Sigma_0 C=C\Sigma_0$ by functional calculus. Multiplying this identity on the left and right by $\Sigma_0^{1/2}$ gives $\Sigma_0\Sigma_1=\Sigma_1\Sigma_0$. Conversely, if $\Sigma_0$ and $\Sigma_1$ commute, they are orthogonally co-diagonalizable, and both~\eqref{eq-gaussian-product-fm-terminal-map} and~\eqref{eq-gaussian-brenier-map-comparison} reduce in that basis to the diagonal map with entries $\sqrt{\lambda_{1,k}/\lambda_{0,k}}$. This proves the equivalence.
\index{Brenier!map}
\end{proof}

The proposition gives a compact warning about a common overinterpretation of flow matching.

\begin{rem}[Changing the bridge speed does not restore optimality]
	The same terminal map~\eqref{eq-gaussian-product-fm-terminal-map} is obtained for any scalar schedule $Z_t=a_tX_0+b_tX_1$ with the same endpoints, because after whitening the covariance path remains $a_t^2\Id+b_t^2C$. Thus changing the speed of a scalar Gaussian bridge, for instance by using an OU schedule, cannot repair the non-optimality created by non-commuting covariances.

	Commuting covariances reduce the terminal map to independent one-dimensional scalings, whereas non-commuting covariances create a non-symmetric affine map, hence a transport with a rotational or shearing component. More generally, mixture-like paths can create the same obstruction even when every instantaneous velocity looks natural. This distinction is closely related to counterexamples showing that flow maps associated with Fokker--Planck or diffusion-type evolutions do not in general provide optimal transport maps~\cite{LavenantSantambrogio2022FlowMap}. In particular, starting from an isotropic Gaussian does not by itself guarantee optimality once the target distribution is non-Gaussian; additional structural assumptions on the path or on the coupling are needed.
\end{rem}
\index{affine map}
\index{flow!map}
\index{rotational component}
\index{shearing component}
\index{Lavenant criterion}
\index{transport map}
\index{Fokker-Planck equation}
\index{flow!matching}
\index{Gaussian!flow matching}
\index{covariance!commuting}
\index{flow!map}

\paragraph{Variations on the interpolant.}

The geometry of the generated trajectories depends on the chosen interpolant, not only on the two endpoint laws. There is first a harmless ambiguity: a monotone reparametrization $Z_t=(1-\lambda(t))X+\lambda(t)Y$ of the linear bridge only changes the speed of the flow,
\[
	v_t(z)=\lambda'(t)\,v^{\rm lin}_{\lambda(t)}(z),
	\qquad
	v^{\rm lin}_{r}(z)=\EE[Y-X\mid (1-r)X+rY=z].
\]
It therefore leaves the spatial integral curves unchanged. Diffusion models use a genuinely different family of noising paths. If
\index{diffusion model}
\[
	Z_t=a_tX+b_tY,\qquad Y\sim\Gaussian(0,\sigma^2\Id),
\]
then both the mixture centers and the component variances are changed. Writing $p_t$ for the density of $Z_t$ and $s_t=\nabla\log p_t$, Tweedie's formula gives, away from times where $a_t=0$,
\index{Tweedie identity}
\[
	v_t(z)=a'_t\,\EE[X\mid Z_t=z]+b'_t\,\EE[Y\mid Z_t=z]
	=\frac{a'_t}{a_t}z+
	\left(\frac{a'_tb_t^2}{a_t}-b'_tb_t\right)\sigma^2s_t(z).
\]
For the linear bridge, $a_t=1-t$ and $b_t=t$, this recovers the formula above. For the variance-preserving Ornstein--Uhlenbeck noising used in diffusion models,
\index{Ornstein-Uhlenbeck process}
\index{bridge!variance-preserving}
\index{diffusion model}
\[
	a_\tau=e^{-\tau},\qquad b_\tau=\sqrt{1-e^{-2\tau}},
\]
one obtains the forward probability-flow velocity $v_\tau(z)=-z-\sigma^2\nabla\log p_\tau(z)$. Sampling follows the reverse field $z+\sigma^2\nabla\log p_\tau(z)$ as $\tau$ decreases. This is the noising law used in the left panel of Figure~\ref{fig:generative-diffusion-versus-ot-2d}; the trajectories are more curved than for the linear bridge because the centers and variances evolve according to the OU/Fokker--Planck scaling rather than by affine interpolation. Numerically, the integration is stopped at a small positive time before the Dirac endpoint, where the score becomes singular.
\index{Fokker-Planck equation}

The finite-time coefficients $a_t=\cos(\pi t/2)$ and $b_t=\sin(\pi t/2)$ are not a new spatial interpolant: they are exactly the OU coefficients after the time change $\tau=-\log\cos(\pi t/2)$. Figure~\ref{fig:generative-diffusion-schedule-comparison} therefore compares OU with a genuinely different scalar bridge,
\[
	a_t=(1-t)(1-2t),
	\qquad
	b_t=t,
\]
whose data coefficient changes sign before vanishing. This overshooting bridge is mainly a diagnostic example: it keeps the same endpoints, but its intermediate mixture reflects through the origin and produces visibly different reverse trajectories.
\index{bridge!overshooting}

\begin{figure}[H]
\centering
\begin{tabular}{@{}cc@{}}
\small diffusion-like trajectories & \small OT displacement rays \\[-.15em]
\includegraphics[width=.39\linewidth]{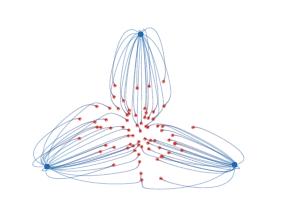} &
\includegraphics[width=.39\linewidth]{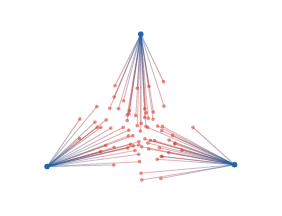}
\end{tabular}
\caption{Diffusion-style sampling trajectories compared with OT rays in the three-Dirac setting of Figure~\ref{fig:generative-diffusion-2d-forward-backward}. Red particles are sampled from the centered Gaussian endpoint and transported toward the three blue atoms. The left panel integrates the reverse probability-flow ODE for the variance-preserving OU noising $a_\tau=e^{-\tau}$, $b_\tau=\sqrt{1-e^{-2\tau}}$, using the closed-form Gaussian-mixture score and stopping just before the singular Dirac endpoint. The right panel uses the straight displacement rays selected by a quadratic OT matching to the same atoms.}
\index{Gaussian mixture}
\index{bridge!variance-preserving}
\index{probability-flow ODE}
\index{score function}
\index{flow!probability ODE}
\label{fig:generative-diffusion-versus-ot-2d}
\end{figure}

\begin{figure}[H]
\centering
\begin{tabular}{@{}ccc@{}}
\small linear bridge & \small OU VP bridge & \small overshooting bridge \\[-.15em]
\index{bridge!overshooting}
\index{bridge!OU}
\includegraphics[width=.30\linewidth]{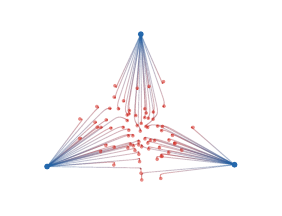} &
\includegraphics[width=.30\linewidth]{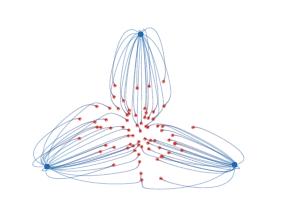} &
\includegraphics[width=.30\linewidth]{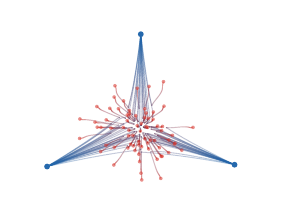}
\end{tabular}
\caption{Effect of the interpolant on the exact reverse flow for the same three-Dirac target and the same Gaussian endpoint. The linear bridge $a_t=1-t$, $b_t=t$ produces almost radial curves. The variance-preserving OU bridge $a_\tau=e^{-\tau}$, $b_\tau=\sqrt{1-e^{-2\tau}}$ changes the relative speed of contraction and noising. The overshooting bridge $a_t=(1-t)(1-2t)$, $b_t=t$ is not a time reparameterization of either one and produces a more pronounced bending of the reverse trajectories.}
\index{reverse flow}
\index{bridge!overshooting}
\index{bridge!variance-preserving}
\index{bridge!OU}
\label{fig:generative-diffusion-schedule-comparison}
\end{figure}


\section{One-Step Generative Models}
\index{one-step!generative model}

One-step generative models try to keep the geometric training principle of flows while removing the expensive multi-step integration at sampling time. The idea is to evolve the model distribution during training, but to store the final evolution in a single generator evaluation.
\index{generative model}

\paragraph{Training a one-step flow.}
\index{one-step!flow}

Let $\zeta$ be a simple latent distribution and let $\alpha_\theta=(G_\theta)_\sharp\zeta$ be the model distribution. Assume that the target data distribution is $\beta$. A Wasserstein-flow construction chooses a discrepancy
\[
	\mathcal E_\beta(\alpha),
\]
for instance a smoothed $\KL(\alpha|\beta)$, an MMD/IPM loss, or the debiased Sinkhorn divergence $\bar\MK_\c^\epsilon(\alpha,\beta)$ introduced in Section~\ref{sec-sinkhorn-div}. The associated formal descent is
\index{maximum mean discrepancy}
\index{integral probability metric}
\index{Sinkhorn!divergence}
\begin{equation}\label{eq-one-step-wgf}
	\partial_t\mu_t+\diverg(\mu_t w_t)=0,
	\qquad
	w_t(x)=-\nabla\delta_\alpha \mathcal E_\beta(\mu_t)(x).
\end{equation}
Instead of integrating~\eqref{eq-one-step-wgf} at inference time, one fits a parametric residual field $U_\eta$ along the current model distribution:
\begin{equation}\label{eq-one-step-l2-fit}
	\min_\eta \int_0^1\!\int
		\norm{U_\eta(t,x)-w_t(x)}^2
		\,\d\mu_t(x)\,\d t.
\end{equation}
In a particle or generator implementation, the learned residual is then used to update the current generator by
\[
	\alpha_{\theta}^{+}
	=
	(\Id+\tau U_\eta)_\sharp \alpha_\theta,
	\qquad\text{or equivalently}\qquad
	G_\theta^{+}(z)=G_\theta(z)+\tau U_\eta(G_\theta(z)).
\]

\begin{alg}[One-step Wasserstein-flow generator update]\label{alg:one-step-wgf-generator-update}
\index{one-step!generative model}
\textbf{Input:} Generator $G_{\theta_k}$, latent law $\zeta$, data law $\beta$, numerical descent-field oracle $W_\beta$, step size $\tau$, batch size $B$.

\textbf{Output:} Updated generator $G_{\theta_{k+1}}$.

\textbf{Draw} $z_b\sim\zeta$ for $b=1,\ldots,B$.

\textbf{Set} $x_b=G_{\theta_k}(z_b)$.

\textbf{Set} \(w_k(x)=W_\beta[\alpha_{\theta_k}](x)\), where \(W_\beta[\alpha]=-\nabla\delta_\alpha\mathcal E_\beta(\alpha)\).

\textbf{Set} $\eta_k$ by minimizing the empirical least-squares loss:
\(\frac1B\sum_{b=1}^B \norm{U_{\eta}(x_b)-w_k(x_b)}^2.\)

\textbf{Update by composition:}
\(G_{\theta_{k+1}}(z) = G_{\theta_k}(z)+\tau U_{\eta_k}(G_{\theta_k}(z)).\)
\textbf{Return} $G_{\theta_{k+1}}$.
\end{alg}

After many training updates, the accumulated generator is evaluated once at test time. This is the organizing principle behind recent one-step methods based on Wasserstein gradient flows: W-Flow uses such a construction with the Sinkhorn divergence as a tractable global discrepancy~\cite{Han2026WFlow}, while drifting methods evolve the generated distribution during training through a fitted vector field and also admit one-step inference~\cite{Deng2026Drifting}. The gradient-flow interpretation of drifting models, and its relation to KL, MMD, sliced-Wasserstein and Sinkhorn-type discrepancies, is analyzed in~\cite{Gretton2026DriftingWGF,He2026SinkhornDrifting}. These ideas are also connected to the Sinkhorn-type normalization dynamics used to model attention in Sinkformers~\cite{Sander2022Sinkformers}.
\index{Wasserstein!gradient}
\index{gradient!flow}
\index{Wasserstein!gradient flow}
\index{Sinkhorn!divergence}
\index{drifting!model}

\paragraph{Self-corrected drifting fields.}
\index{drifting!field}

Drifting methods need not start from an exact Wasserstein gradient. They often prescribe an attraction-minus-repulsion field and then regress this field in $L^2(\mu_t)$. A simple continuous version uses a positive kernel $K_\epsilon(x,y)$ and defines, for any measure $\nu$,
\index{Wasserstein!gradient}
\index{kernel!positive}
\begin{equation}\label{eq-normalized-kernel-drift}
	B_\epsilon[\nu](x)
	\eqdef
	\frac{\int (y-x)K_\epsilon(x,y)\,\d\nu(y)}
	     {\int K_\epsilon(x,y)\,\d\nu(y)}.
\end{equation}
For the Gaussian kernel
\index{kernel!Gaussian}
$K_\epsilon(x,y)=\exp(-\norm{x-y}^2/(2\epsilon))$, this normalized field is a score of a smoothed density:
\index{score function}
\begin{equation}\label{eq-normalized-kernel-score}
	B_\epsilon[\nu](x)
	=
	\epsilon\nabla\log\!\left(\int K_\epsilon(x,y)\,\d\nu(y)\right).
\end{equation}
The drifting velocity is then
\begin{equation}\label{eq-cross-minus-self-drift}
	u_t(x)=B_\epsilon[\beta](x)-B_\epsilon[\mu_t](x)
	=
	\epsilon\nabla\log
	\frac{\int K_\epsilon(x,y)\,\d\beta(y)}
	     {\int K_\epsilon(x,y)\,\d\mu_t(y)}.
\end{equation}
The first term pulls samples toward data, while the second term corrects self-attraction and prevents all particles from collapsing onto the same high-density region. Sinkhorn drifting replaces these one-sided kernel normalizations by two-sided entropic OT couplings, so that the cross and self terms are normalized by Sinkhorn scaling rather than by a single denominator~\cite{He2026SinkhornDrifting}.
\index{Sinkhorn!scaling}
\index{kernel!norm}
\index{entropic!OT}

\begin{alg}[Self-corrected drifting particle update]\label{alg:self-corrected-drifting-particles}
\index{drifting!model}
\textbf{Input:} Particles $x_i^k$ for $\mu_k$, data samples $(y_b)_{b=1}^B$ from $\beta$, kernel scale $\epsilon$, step $h$.

\textbf{Output:} Updated particles $x_i^{k+1}$.

\textbf{For} each particle $i$ \textbf{do}:
\begin{algblock}

\textbf{Set} \(Z_{\beta,i}=\sum_{b=1}^B K_\epsilon(x_i^k,y_b)\) and \(b_i^k=Z_{\beta,i}^{-1}\sum_{b=1}^B (y_b-x_i^k)K_\epsilon(x_i^k,y_b)\).

\textbf{Set} \(Z_{\mu,i}=\sum_{j=1}^n K_\epsilon(x_i^k,x_j^k)\) and \(m_i^k=Z_{\mu,i}^{-1}\sum_{j=1}^n (x_j^k-x_i^k)K_\epsilon(x_i^k,x_j^k)\).

\textbf{Set}
\(u_i^k=b_i^k-m_i^k.\)

\textbf{Update}
\(x_i^{k+1}=x_i^k+h\,u_i^k.\)
\end{algblock}
\algreturnskip
\textbf{Return} $(x_i^{k+1})_i$.
\end{alg}

\begin{figure}[ht]
\centering
\begin{tabular}{@{}cc@{}}
\includegraphics[width=.38\linewidth]{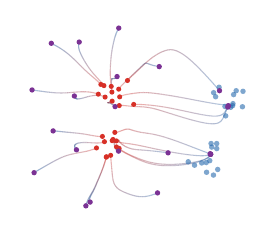} &
\index{drifting!model}
\includegraphics[width=.38\linewidth]{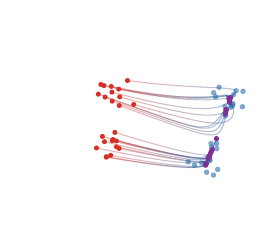}
\\[-.1em]
\small raw kernel drift & \small self-corrected drift
\index{self-corrected field}
\end{tabular}
\caption{Drifting trajectories for a small particle generator. The raw Laplacian-kernel drift has weak long-range attraction and can leave particles away from the data modes. The self-corrected field uses the difference $B_\epsilon[\beta]-B_\epsilon[\mu_t]$, so a longer integration brings particles to the blue modes while repelling them from their own current concentration.}
\label{fig:generative-drifting-model-trajectories}
\end{figure}

\begin{prop}[Drifting as a time-dependent Wasserstein gradient]\label{prop-drifting-semi-relaxed-gradient}
\index{drifting!model}
\index{Wasserstein!gradient}
	Let $\mu_t$ be a smooth curve of positive densities and let $u_t=\nabla\phi_t$ be a smooth time-dependent gradient field. Define the semi-relaxed functional
	\begin{equation}\label{eq-semi-relaxed-drift-functional}
		\mathcal R_t(\alpha|\mu_t)
		\eqdef
		-\int \phi_t(x)\,\d\alpha(x)
		+\int \phi_t(x)\,\d\mu_t(x).
	\end{equation}
	Here $\mu_t$ and $\phi_t$ are frozen when taking the first variation with respect to the first argument $\alpha$. Then the continuity equation
\index{first variation}
\index{continuity equation}
	\[
		\partial_t\mu_t+\diverg(\mu_t u_t)=0
	\]
	is the formal Wasserstein gradient descent of the time-dependent functional $\alpha\mapsto\mathcal R_t(\alpha|\mu_t)$.
\index{Wasserstein!gradient}
\end{prop}

\begin{proof}
	Since $\mu_t$ and $\phi_t$ are fixed in the variation with respect to $\alpha$, the first variation is
\index{first variation}
	\[
		\delta_\alpha \mathcal R_t(\alpha|\mu_t)(x)=-\phi_t(x).
	\]
	By Proposition~\ref{prop-formal-wass-gradient},
	\[
		\Wgrad \mathcal R_t(\alpha|\mu_t)
		=
		\nabla\delta_\alpha \mathcal R_t(\alpha|\mu_t)
		=
		-\nabla\phi_t
		=
		-u_t.
	\]
	The Wasserstein gradient-descent velocity is the negative of this gradient, namely $u_t$. Substituting this velocity in the continuity equation gives the claimed flow.
\index{Wasserstein!gradient}
\index{continuity equation}
\end{proof}

\begin{example}[Kernel drifting as a semi-relaxed divergence]
\index{Gaussian!kernel drifting}
\index{kernel!drifting}
\index{semi-relaxed!divergence}
	For the Gaussian-kernel drift~\eqref{eq-cross-minus-self-drift}, set
	\[
		\phi_t(x)=
		\epsilon\log
		\frac{\int K_\epsilon(x,y)\,\d\beta(y)}
		     {\int K_\epsilon(x,y)\,\d\mu_t(y)}.
	\]
	Then $u_t=\nabla\phi_t$, so Proposition~\ref{prop-drifting-semi-relaxed-gradient} shows that kernel drifting is the Wasserstein gradient descent of
\index{Wasserstein!gradient}
\index{kernel!drifting}
	\[
		\mathcal R_t^{\mathrm{drift}}(\alpha|\mu_t)
		=
		\epsilon
		\int
		\log
		\frac{\int K_\epsilon(x,y)\,\d\mu_t(y)}
		     {\int K_\epsilon(x,y)\,\d\beta(y)}
		\,\d\alpha(x)
		+\mathrm{constant}.
	\]
	It is ``semi-relaxed'' because the current model $\mu_t$ is used to build the potential, but it is not varied inside the denominator when computing the first variation in $\alpha$.
\index{first variation}
\end{example}

\begin{rem}[General fields and projection onto gradients]
\index{gradient!projection}
	A general regressed field $b_t$ is not necessarily a Wasserstein gradient, since Wasserstein tangent vectors are represented by gradient fields modulo $L^2(\mu_t)$-null directions. The gradient component is obtained by the weighted projection
\index{Wasserstein!gradient}
	\[
		\nabla\phi_t
		=
		\uargmin{\nabla\phi}
		\int \norm{\nabla\phi(x)-b_t(x)}^2\,\d\mu_t(x).
	\]
	One may first normalize $b_t$ pointwise, for instance by $b_t/(\norm{b_t}+\eta)$, or globally by $\norm{b_t}_{L^2(\mu_t)}$, before this projection. Proposition~\ref{prop-drifting-semi-relaxed-gradient} then applies to the projected field. Non-gradient components can still be useful in a parametric model, but they are not descent directions of a scalar functional for the $\Wass_2$ Riemannian metric.
\end{rem}

\section{Evolution in Depth of Transformers}
\index{transformer}
\index{depth evolution}

Deep residual architectures can be read as time discretizations of ODEs or PDEs. For transformers, the transported objects are token measures and the velocity is induced by attention.
\index{token measure}

Transformers were introduced as sequence-to-sequence architectures driven by self-attention~\cite{Vaswani2017Attention} and have since become a central architecture for language and vision models~\cite{Brown2020LanguageModels,Dosovitskiy2021Image}. Their distinctive feature is that each token is updated by a data-dependent average of all other tokens. This makes an attention layer permutation-equivariant before positional encoding, context dependent after conditioning on the input sequence, and naturally compatible with a measure viewpoint in which a prompt is regarded as an empirical distribution of tokens.
\index{transformer}
\index{empirical!distribution}
\index{attention!self-attention}

The mathematical limit used below concerns depth rather than model scale: one lets the number of residual attention layers grow while each layer makes a small update, as in continuous-depth neural networks~\cite{Chen2018NeuralODE}. For attention, the resulting velocity is nonlinear in the current token law because it is normalized by the whole context. This measure-theoretic view appears in the analysis of attention as a Lipschitz or interacting-particle operator~\cite{Vuckovic2020MathematicalAttention,Geshkovski2023MathematicalPerspective}, in the Sinkhorn-normalized dynamics of Sinkformers~\cite{Sander2022Sinkformers}, and in recent well-posedness and mean-field-limit results for several attention mechanisms~\cite{Castin2025DynamicsTransformers}. It also separates the infinite-depth limit studied here from the token-limit question, where one controls how a finite empirical context approximates its limiting attention operator~\cite{Bohbot2025TokenSampleComplexity}.
\index{empirical context}
\index{attention!operator}
\index{infinite-depth limit}
\index{token limit}
\index{transformer}
\index{mean-field!limit}
\index{attention!mechanism}
\index{continuous!depth}

We now consider very deep transformers, focusing on a single-head attention mechanism for simplicity while ignoring MLP layers, layer normalization, causality, and masking. This stripped-down framework is best suited to modeling encoders and vision transformers; the references above indicate which parts of this simplified picture extend to richer attention mechanisms.
\index{attention!single-head}
\index{layer normalization}
\index{causality}
\index{masking}
\index{transformer}

\paragraph{Attention as a context-dependent velocity.}
\index{attention}
\index{velocity field}

After tokenization, embedding, and positional encoding, each input (from a set of tokens) is represented as a point cloud $(x_i)_{i=1}^n$ of $n$ points in the space of vectorized tokens. An attention layer with skip connection and rescaling by $1/T$ (where $T$ is the depth) defines a transformation of the tokens:
\begin{equation*}
    x_i \mapsto x_i + \frac{1}{T} \sum_j \frac{e^{\langle Q x_i, K x_j \rangle} V x_j}{\sum_{\ell} e^{\langle Q x_i, K x_\ell \rangle}},
\end{equation*}
where $\theta = (K, Q, V)$ are the parameters of the attention layer, represented by three matrices.

\paragraph{Token measure evolution.}
\index{token measure}

To handle an arbitrary number of tokens, we define $\alpha = \frac{1}{n} \sum_i \delta_{x_i}$ as the empirical measure of tokens and rewrite the transformer mapping as:
\index{empirical!measure}
\index{transformer}
\begin{equation*}
    x_i \mapsto x_i + \frac{1}{T} \Gamma_\theta[\alpha](x_i),
\end{equation*}
where
\begin{equation*}
    \Gamma_\theta[\alpha](x) :=
    \frac{\int e^{\langle Q x, K y \rangle} V y \, \d \alpha(y)}
    {\int e^{\langle Q x, K z \rangle} \, \d \alpha(z)}.
\end{equation*}
At the level of the token distribution, the layer pushes $\alpha$ forward by the ``in-context'' mapping $\Gamma_{\theta_t}[\alpha]$, which depends on the context $\alpha$, the tokens, and the depth-dependent parameters $\theta_t$. Denoting $t \in [0, 1]$ as the depth and $\tau = 1/T$ as the step size, this gives:
\begin{equation*}
    \alpha_{t+\tau} = (\Id + \tau \Gamma_{\theta_t}[\alpha_t])_\sharp \alpha_t.
\end{equation*}
As $\tau \to 0$, this converges formally to the conservation equation
\begin{equation*}
    \partial_t \alpha_t + \diverg(\alpha_t \Gamma_{\theta_t}[\alpha_t]) = 0.
\end{equation*}

\begin{alg}[Residual attention depth evolution]\label{alg:residual-attention-depth-evolution}
\index{attention!self-attention}
\index{transformer}
\textbf{Input:} Tokens $(x_i^0)_{i=1}^n$, depth $T$, layer parameters $(Q_k,K_k,V_k)$.

\textbf{Output:} Final token measure $\alpha_T$.

\textbf{Initialize:}
\(\alpha_0=\frac1n\sum_{i=1}^n\delta_{x_i^0}, \qquad \tau=1/T.\)

\textbf{For} $k=0,\ldots,T-1$ \textbf{do}:
\begin{algblock}

\textbf{For} $i=1,\ldots,n$ \textbf{do}
\begin{algblock}
\(\Gamma_{\theta_k}[\alpha_k](x_i^k) = \frac{\sum_j \exp(\dotp{Q_kx_i^k}{K_kx_j^k})\,V_kx_j^k} {\sum_j \exp(\dotp{Q_kx_i^k}{K_kx_j^k})}.\)

\(x_i^{k+1}=x_i^k+\tau\,\Gamma_{\theta_k}[\alpha_k](x_i^k).\)
\end{algblock}
\textbf{Set}
\(\alpha_{k+1}=(\Id+\tau\Gamma_{\theta_k}[\alpha_k])_\sharp\alpha_k.\)
\end{algblock}
\algreturnskip
\textbf{Return} $\alpha_T$.
\end{alg}

\paragraph{Gradient structure and limitations.}
\index{gradient!structure}

When the token space has dimension $d$ and the query/key space has dimension $r$, take $Q,K\in\RR^{r\times d}$ and $V\in\RR^{d\times d}$. If $V=Q^\top K$, the field $\Gamma_\theta[\alpha]$ is a gradient vector field in the token variable. Indeed, define the log-partition potential
\[
	\Phi_\alpha(x)
	=
	\int \exp(\dotp{Qx}{Ky})\d\alpha(y),
	\qquad
	U_\alpha(x)=\log\Phi_\alpha(x).
\]
Then
\[
	\nabla_x U_\alpha(x)
	=
	\frac{\int Q^\top K y\,\exp(\dotp{Qx}{Ky})\d\alpha(y)}
	     {\int \exp(\dotp{Qx}{Kz})\d\alpha(z)}
	=
	\Gamma_\theta[\alpha](x).
\]
This is an instantaneous gradient in $x$. It is not, however, the gradient of the first variation of a fixed functional of $\alpha$, because the potential $U_\alpha$ itself depends on the current measure through the same attention normalization. Thus the PDE is generally a transportation dynamics, not a Wasserstein gradient flow. Special variants recover additional structure: Sinkhorn attention can be interpreted through doubly stochastic normalization and Wasserstein-type gradient flows~\cite{Sander2022Sinkformers,Castin2025DynamicsTransformers}, while layer normalization leads naturally to dynamics on the sphere and to modified metrics. The key open difficulty for the present viewpoint is training: after the architecture has been rewritten as a controlled transport equation, learning corresponds to optimizing the time-dependent parameters $(\theta_t)_t$ rather than merely analyzing the forward PDE for fixed parameters.
\index{doubly stochastic normalization}
\index{transport equation}
\index{controlled transport}
\index{first variation}
\index{layer normalization}
\index{Wasserstein!gradient}
\index{gradient!flow}
\index{transformer}
\index{Wasserstein!gradient flow}

\section{Flows over the Gaussian Manifold}
\index{Gaussian!manifold}

Gaussian measures provide a useful testing ground for the preceding dynamics. They are not invariant under a general Wasserstein gradient flow: a nonlinear velocity usually creates non-Gaussian densities immediately. The useful substitute is to either identify affine velocities, which exactly preserve Gaussianity, or to project the dynamics onto the Gaussian manifold. In both cases the measure PDE reduces to matrix ODEs for the mean and covariance. This viewpoint is emphasized in the survey~\cite{Peyre2026OptimalDiffusionTransports} and is useful for comparing diffusion paths, Wasserstein gradient flows, drifting fields and transformer-type dynamics.
\index{affine!closure}
\index{drifting!field}
\index{Wasserstein!gradient}
\index{gradient!flow}
\index{transformer}
\index{Wasserstein!gradient flow}
\index{Gaussian!manifold}
\index{Gaussian!measure}

For constrained gradient flows on this family, the covariance equation is the finite-dimensional Bures--Wasserstein gradient flow on positive definite matrices. Thus Gaussian closure is not just a computational shortcut: it is the restriction of Wasserstein geometry to the Gaussian submanifold, where affine gradient fields encode tangent vectors.
\index{Gaussian!submanifold}
\index{affine!gradient}
\index{Bures-Wasserstein geometry}
\index{Wasserstein!gradient}
\index{Gaussian!closure}
\index{gradient!flow}

\begin{figure}[ht]
\centering
\begin{tabular}{@{}ccc@{}}
\small $W_2$ geodesic & \small Sinkhorn closure & \small drifting closure \\[-.15em]
\includegraphics[width=.30\linewidth]{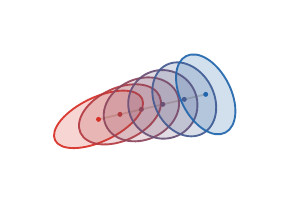} &
\index{Gaussian!closure}
\includegraphics[width=.30\linewidth]{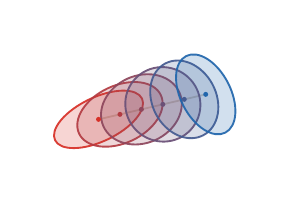} &
\includegraphics[width=.30\linewidth]{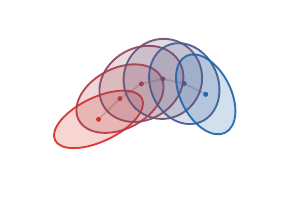}
\end{tabular}
\caption{Gaussian closures of transport dynamics between two overlapping anisotropic Gaussians. The left panel is the exact $W_2$ Gaussian geodesic. The middle panel shows a regularized Sinkhorn-style closure, where the same mean path is accompanied by inflated intermediate covariances. The right panel shows a drifting-style closure with a curved mean path and moment-matched covariance ellipses. These finite-dimensional pictures keep only means and covariances, and therefore discard higher-order shape information that would be created by a genuinely nonlinear velocity field.}
\index{moment matching}
\index{Gaussian!closure}
\index{velocity field}
\label{fig:gradflow-gaussian-closure}
\end{figure}

\paragraph{Gaussianity preservation.}
\index{Gaussian!preservation}

The first question is invariance: one wants a simple criterion ensuring that the continuity equation does not leave the finite-dimensional Gaussian family.
\index{continuity equation}

\begin{prop}[Affine velocities preserve Gaussianity]\label{prop-gaussian-affine-closure}
\index{Gaussian!preservation}
\index{affine!closure}
	Let $\alpha_t=\Gaussian(\mean_t,\cov_t)$, with $\cov_t$ positive definite, solve the continuity equation with an affine velocity
	\[
		v_t(x)=b_t+A_t(x-\mean_t).
	\]
	Then $\alpha_t$ remains Gaussian and its moments solve
	\[
		\dot \mean_t=b_t,
		\qquad
		\dot\cov_t=A_t\cov_t+\cov_t A_t^\top .
	\]
	Conversely, any smooth Gaussian curve with positive definite covariance can be generated by such an affine velocity. If one wants the velocity to be a Wasserstein tangent gradient, one chooses the unique symmetric solution of the Lyapunov equation
\index{affine!closure}
\index{Lyapunov equation}
	\[
		A_t\cov_t+\cov_t A_t=\dot\cov_t.
	\]
\end{prop}

\begin{proof}
	Let $X_t$ follow the characteristic ODE $\dot X_t=b_t+A_t(X_t-\mean_t)$. This linear ODE maps Gaussian random variables to Gaussian random variables. Taking expectation gives $\dot\mean_t=b_t$. Writing $\tilde X_t=X_t-\mean_t$, one has $\dot{\tilde X}_t=A_t\tilde X_t$, hence
\index{random variable}
	\[
		\dot\cov_t
		=
		\frac{\d}{\d t}\EE(\tilde X_t\tilde X_t^\top)
		=
		A_t\cov_t+\cov_t A_t^\top .
	\]
	For the converse, set $b_t=\dot\mean_t$ and choose any matrix $A_t$ satisfying the covariance equation. Since $\cov_t$ is positive definite, the Lyapunov map $A\mapsto A\cov_t+\cov_t A$ is invertible on symmetric matrices, which gives the unique symmetric choice when a gradient velocity is required. In that case $v_t$ is the gradient of the quadratic potential $x\mapsto \dotp{b_t}{x}+\dotp{A_t(x-\mean_t)}{x-\mean_t}/2$.
\index{quadratic!potential}
\index{gradient!velocity}
\end{proof}

\paragraph{Constrained evolution on the Gaussian manifold.}
\index{Gaussian!manifold}

For non-affine velocities, the finite-dimensional substitute is to project the Wasserstein dynamics onto the Gaussian manifold.
\index{affine!closure}
\index{Gaussian!manifold}

Let
\[
	\mathcal G=\{\Gaussian(\mean,\cov):\mean\in\RR^d,\ \cov\succ0\}
\]
be the Gaussian submanifold of $\Pp_2(\RR^d)$. The Wasserstein gradient of a functional constrained to a smooth submanifold $\mathcal M\subset\Pp_2$ is defined as the Riesz representative of the differential restricted to tangent velocities of $\mathcal M$. Equivalently, it is the small-step limit of the constrained JKO scheme
\index{Gaussian!submanifold}
\index{Wasserstein!gradient}
\index{JKO scheme}
\[
	\alpha^{k+1}\in
	\uargmin{\alpha\in\mathcal M}
	\frac{1}{2\tau}\Wass_2^2(\alpha,\alpha^k)+f(\alpha).
\]
For $\mathcal M=\mathcal G$, tangent velocities are affine gradient fields $v(x)=b+A(x-\mean)$ with $A=A^\top$. The constrained gradient is therefore the $L^2(\Gaussian(\mean,\cov))$ projection of the ambient Wasserstein gradient onto this finite-dimensional affine space, whenever the ambient gradient exists.
\index{affine!gradient}
\index{Wasserstein!gradient}

\begin{prop}[Gaussian-constrained Wasserstein gradients]\label{prop-gaussian-gradient-bullet-list}
\index{Wasserstein!gradient}
	Let $f$ be a smooth functional and assume that its restriction to nondegenerate Gaussian measures can be written as
\index{Gaussian!measure}
	\[
		f(\Gaussian(\mean,\cov))=F(\mean,\cov).
	\]
	Then the Wasserstein gradient constrained to the Gaussian family is the affine vector field
\index{Wasserstein!gradient}
	\[
		v_F(x)
		=
		\nabla_\mean F(\mean,\cov)
		+
		2\nabla_\cov F(\mean,\cov)(x-\mean),
	\]
	where $\nabla_\cov F$ denotes the symmetric matrix derivative. Equivalently, $v_F$ is the $L^2(\Gaussian(\mean,\cov))$ projection of the ambient Wasserstein gradient onto affine gradient fields, whenever the ambient gradient exists. Hence the gradient descent flow constrained to Gaussian measures satisfies
\index{affine!gradient}
\index{Wasserstein!gradient}
\index{Gaussian!measure}
	\begin{equation}\label{eq-gaussian-wgf-closure}
		\dot\mean_t=-\nabla_\mean F(\mean_t,\cov_t),
		\qquad
		\dot\cov_t=-2\bigl(\cov_t\nabla_\cov F(\mean_t,\cov_t)+\nabla_\cov F(\mean_t,\cov_t)\cov_t\bigr),
	\end{equation}
	and the descent velocity is affine.
\end{prop}

\begin{proof}
	Test the functional along a Gaussian tangent vector, represented by an affine gradient field
\index{affine!gradient}
	\[
		v(x)=b+A(x-\mean)
	\]
	with $A$ symmetric. The induced first-order variations are $\dot\mean=b$ and $\dot\cov=A\cov+\cov A$. Therefore
	\[
		\d F(\mean,\cov)[b,A\cov+\cov A]
		=
		\dotp{\nabla_\mean F}{b}
		+
		\mathrm{tr}\!\left(\nabla_\cov F(A\cov+\cov A)\right).
	\]
	Since $A$, $\cov$ and $\nabla_\cov F$ are symmetric, the second term equals
	\[
		2\,\mathrm{tr}(\nabla_\cov F\,A\cov)
		=
		\int \dotp{2\nabla_\cov F(x-\mean)}{A(x-\mean)}\d\Gaussian(\mean,\cov)(x).
	\]
	Together with the mean term, this gives
	\[
		\d F(\mean,\cov)[\dot\mean,\dot\cov]
		=
		\int \dotp{v_F(x)}{v(x)}\d\Gaussian(\mean,\cov)(x)
	\]
	for all affine gradient fields $v$. This identifies the constrained Wasserstein gradient in the induced $L^2(\alpha)$ metric, or equivalently the projection of the ambient gradient when it exists. Substituting the descent velocity $-v_F$ in Proposition~\ref{prop-gaussian-affine-closure} gives~\eqref{eq-gaussian-wgf-closure}.
\index{Wasserstein!gradient}
\index{affine!closure}
\end{proof}

This proposition should be read as the organizing rule for Gaussian closures: once the scalar energy has been reduced to a function of $(\mean,\cov)$, its constrained Wasserstein gradient is automatically affine and the covariance follows the Bures-type ODE~\eqref{eq-gaussian-wgf-closure}. When the first variation of $f$ is quadratic, this constrained gradient coincides with the full Wasserstein gradient.
\index{first variation}
\index{Gaussian!closure}

\paragraph{Gaussian-preserving gradient flows.}
\index{Gaussian!preservation}
\index{gradient!flow}

The next examples show that many familiar energies already have affine Wasserstein gradients on Gaussian inputs, so their full flow remains inside the Gaussian family.
\index{Wasserstein!gradient}

\begin{example}[Gaussian energies and affine gradients]\label{ex-gaussian-affine-gradients}
\index{Gaussian!energy}
\index{affine!gradient}
	Proposition~\ref{prop-gaussian-gradient-bullet-list} turns many standard energies into explicit affine fields:
	\begin{itemize}
		\item \emph{Quadratic potential energy.} If
\index{quadratic!potential}
		\[
			f(\alpha)=\int \Bigl(\frac12 x^\top Hx+\dotp{\ell}{x}\Bigr)\d\alpha(x),
			\qquad H=H^\top,
		\]
		then
		\[
			\Wgrad f(\alpha)(x)=Hx+\ell=(H\mean+\ell)+H(x-\mean).
		\]
		This is the Gaussian form of transport under a quadratic confinement.

		\item \emph{Quadratic interaction energy.} If
\index{energy!interaction}
		\[
			f(\alpha)=\frac14\iint (x-y)^\top G(x-y)\d\alpha(x)\d\alpha(y),
			\qquad G=G^\top,
		\]
		then $F(\mean,\cov)=\frac12\mathrm{tr}(G\cov)$ and
		\[
			\Wgrad f(\alpha)(x)=G(x-\mean).
		\]
		The mean is unchanged and the covariance contracts or expands according to the signs of $G$.

		\item \emph{Relative entropy to a Gaussian.} For $\bar\alpha=\Gaussian(\bar\mean,\bar\cov)$,
\index{entropy!relative}
		\[
			f(\alpha)=\KL(\alpha|\bar\alpha)
		\]
		has
		\[
			\Wgrad f(\alpha)(x)
			=
			\bar\cov^{-1}(\mean-\bar\mean)
			+
			(\bar\cov^{-1}-\cov^{-1})(x-\mean).
		\]
		The descent equations are the Ornstein--Uhlenbeck moment equations
\index{Ornstein-Uhlenbeck process}
		\[
			\dot\mean_t=-\bar\cov^{-1}(\mean_t-\bar\mean),
			\qquad
			\dot\cov_t=2\Id-\bar\cov^{-1}\cov_t-\cov_t\bar\cov^{-1}.
		\]

		\item \emph{Squared Wasserstein distance to a Gaussian.} For
\index{Wasserstein!distance}
		\[
			f(\alpha)=\frac12\Wass_2^2(\alpha,\bar\alpha),
			\qquad
			\bar\alpha=\Gaussian(\bar\mean,\bar\cov),
		\]
		the Gaussian Brenier map $T_{\alpha\to\bar\alpha}$ is affine,
\index{Brenier!map}
		\[
			T_{\alpha\to\bar\alpha}(x)=\bar\mean+M(x-\mean),
			\qquad
			M=\cov^{-1/2}(\cov^{1/2}\bar\cov\cov^{1/2})^{1/2}\cov^{-1/2}.
		\]
		Hence
		\[
			\Wgrad f(\alpha)(x)=x-T_{\alpha\to\bar\alpha}(x)
			=
			(\mean-\bar\mean)+(\Id-M)(x-\mean),
		\]
		and descent moves each Gaussian infinitesimally along the Bures--Wasserstein geodesic toward $\bar\alpha$.
\index{Wasserstein!geodesic}
\index{Bures-Wasserstein geometry}

		\item \emph{Gaussian-only losses.} Sliced $\SW_2^2$ losses to a Gaussian, Gaussian Sinkhorn divergences, and any smooth closed formula depending only on $(\mean,\cov)$ fit the same constrained-gradient template:
\index{Gaussian!Sinkhorn}
\index{Gaussian!Sinkhorn divergence}
\index{Sinkhorn!divergence}
			\[
				v_F(x)=\nabla_\mean F+2\nabla_\cov F(x-\mean).
			\]
		For Gaussian Sinkhorn divergences this finite-dimensional flow is studied in~\cite{HardionLacombe2026GaussianSinkhornFlow}.
\index{Gaussian!Sinkhorn divergence}
\index{Sinkhorn!divergence}
	\end{itemize}

	Not every PDE preserves Gaussianity exactly. For instance, Wasserstein flows of relative Fisher information, related to quantum-drift or higher-order diffusion equations, typically require a Gaussian projection to close on $(\mean,\cov)$. Such projected closures are still useful: they expose the finite-dimensional dynamics predicted by a variational model and make it easy to compare variational flows with non-variational affine dynamics such as drifting fields or the Gaussian transformer closure below.
\index{drifting!field}
\index{transformer}
\index{Gaussian!projection}
\index{Fisher information}
\end{example}

\begin{prop}[Centered Gaussian covariance catalogue]\label{prop-centered-gaussian-covariance-catalogue}
\index{Gaussian!covariance}
	Let $\gamma=\Gaussian(0,\Id)$ and let $\mu_t=\Gaussian(0,C_t)$ with $C_t\succ0$. For the normalizations displayed below, the Wasserstein descent constrained to the centered Gaussian manifold satisfies $\dot C_t=h(C_t)$, with
\index{Gaussian!manifold}
	\[
	\begin{array}{rcl}
		\KL(\mu|\gamma) &:& h(C)=2(\Id-C), \\[.35em]
		\frac12\,\mathcal I(\mu|\gamma) &:& h(C)=2(C^{-1}-C), \\[.35em]
		\Wass_2^2(\mu,\gamma) &:& h(C)=4(C^{1/2}-C), \\[.35em]
		\MMD_k^2(\mu,\gamma),\quad k(x,y)=\dotp{x}{y}^2
\index{maximum mean discrepancy}
			&:& h(C)=8(C-C^2), \\[.35em]
		S_\epsilon(\mu,\gamma)
			&:& h(C)=
			4\left(C+\frac{\epsilon^2}{16}\Id\right)^{1/2}
			-2\left(C^2+\frac{\epsilon^2}{16}\Id\right)^{1/2}
			-2C-\frac{\epsilon}{2}\Id, \\[.75em]
		\SW_2^2(\mu,\gamma)
			&:& h(C)=V(C)C+CV(C),
	\end{array}
	\]
	where $S_\epsilon$ is the debiased Sinkhorn divergence for the quadratic cost $\norm{x-y}^2$ and KL regularization strength $\epsilon$, and
\index{Sinkhorn!divergence}
\index{cost!quadratic}
	\[
		V(C)=
		2\int_{\Sphere^{d-1}}
		\left(\frac{1}{\sqrt{\theta^\top C\theta}}-1\right)
		\theta\theta^\top \,\d\sigma(\theta)
	\]
	for the normalized spherical measure $\sigma$. Here
	\[
		\mathcal I(\mu|\gamma)=\int \left|\nabla\log\rho(x)+x\right|^2\rho(x)\,\d x
		\qquad(\mu=\rho\,\d x).
	\]
	Thus the unhalved Fisher divergence has right-hand side $4(C^{-1}-C)$. Multiplying any of these energies by a constant simply rescales the corresponding right-hand side.
\end{prop}

\begin{proof}
	Each row is obtained by identifying the affine descent velocity $v(x)=M_Cx$ generated by the corresponding Gaussian-constrained calculation and then applying Proposition~\ref{prop-gaussian-affine-closure}, which gives $\dot C=M_CC+CM_C^\top$. For $\KL(\cdot|\gamma)$, the Fokker--Planck velocity is $(C^{-1}-\Id)x$, hence $\dot C=2(\Id-C)$. For the Fisher row, the restriction of $\frac12\mathcal I$ to centered Gaussians is
\index{Fokker-Planck equation}
\index{affine!closure}
	\[
		\frac12\left(\tr(C)+\tr(C^{-1})-2d\right).
	\]
	Using Proposition~\ref{prop-gaussian-gradient-bullet-list} gives the descent velocity $(C^{-2}-\Id)x$, hence $\dot C=2(C^{-1}-C)$. This row should be read as a Gaussian projected closure of the fourth-order Fisher flow.

	For $\Wass_2^2(\cdot,\gamma)$, the Brenier map from $\Gaussian(0,C)$ to $\gamma$ is $C^{-1/2}x$, so the descent velocity for the unhalved squared distance is $2(C^{-1/2}-\Id)x$, giving $4(C^{1/2}-C)$. For the polynomial MMD row, centered Gaussians satisfy $\MMD_k^2(\mu,\gamma)=\norm{C-\Id}_{\mathrm F}^2$; the first variation is quadratic and its descent velocity is $4(\Id-C)x$, giving $8(C-C^2)$.
\index{first variation}
\index{maximum mean discrepancy}
\index{Brenier!map}

	Gaussian Sinkhorn dual potentials are quadratic, so the velocity is again linear; differentiating the closed Gaussian formula yields the displayed spectral expression. The square roots are spectral functions of $C$, hence commute with $C$, which is why the covariance ODE closes as a matrix function of $C$ alone. For sliced Wasserstein, each one-dimensional projection is a Gaussian transport with velocity $2((\theta^\top C\theta)^{-1/2}-1)\dotp{\theta}{x}\theta$; averaging these velocities over $\Sphere^{d-1}$ gives $v(x)=V(C)x$ and thus $\dot C=V(C)C+CV(C)$.
\index{covariance!ODE}
\index{Sinkhorn!dual}
\index{sliced Wasserstein!distance}
\index{Gaussian!Sinkhorn}
\index{one-dimensional!projection}
\index{dual!potential}
\index{covariance!ODE}
\end{proof}

\begin{example}[Linear mean-field networks as cross-moment flows]
\index{flow!cross-moment}
	In the two-layer model above, take the linear activation $\sigma(s)=s$, so that
\index{two-layer neural network}
	\[
		\psi((u,v),z)=v\,\dotp{u}{z}.
	\]
	The predictor is the linear map
	\[
		G_{\alpha_t}(z)=Q_t z,
		\qquad
		Q_t=\int v u^\top\d\alpha_t(u,v)\in\RR^{d'\times d}.
	\]
	Thus the energy depends on the neuron law only through the cross moment $Q_t$, a subblock of the raw second moment of $x=(u,v)$. For square loss, set
\index{second moment}
\index{neuron law}
	\[
		S=\int zz^\top\d\rho(z,y),
		\qquad
		R=\int y z^\top\d\rho(z,y),
		\qquad
		H_t=Q_tS-R.
	\]
	The first variation is
\index{first variation}
	\[
		\delta f(\alpha_t)(u,v)=\dotp{H_t}{v u^\top}=v^\top H_t u.
	\]
	Hence the particle velocity in parameter space is linear:
	\[
		-\nabla_{(u,v)}\delta f(\alpha_t)(u,v)
		=
		-\begin{pmatrix}
			0 & H_t^\top \\
			H_t & 0
		\end{pmatrix}
		\begin{pmatrix} u \\ v \end{pmatrix}.
	\]
	Therefore a Gaussian law of neurons remains Gaussian. Its mean and covariance follow Proposition~\ref{prop-gaussian-affine-closure}, with a matrix depending only on the current cross moment $Q_t$, a raw second-moment subblock that becomes a covariance subblock when the neuron law is centered. This exact closure is special to the linear activation; for nonlinear activations, Gaussian closures are usually projections rather than invariant families.
\index{covariance!subblock}
\index{neuron law}
\index{neuron}
\index{Gaussian!closure}
\index{affine!closure}
\end{example}

\paragraph{Non-variational Gaussian-preserving flows.}
\index{Gaussian!preservation}
\index{Gaussian!preserving flow}

The last examples are not ordinary gradient flows of a fixed scalar energy on the full Wasserstein space. They preserve Gaussianity because the prescribed velocity field is affine when evaluated on Gaussian measures.
\index{velocity field}
\index{Wasserstein!space}
\index{Gaussian!measure}
\index{gradient!flow}

\begin{example}[Flow matching and diffusion paths between Gaussians]
	Consider a prescribed Gaussian interpolation $\alpha_t=\Gaussian(\mean_t,\cov_t)$. Proposition~\ref{prop-gaussian-affine-closure} shows that an exact flow-matching velocity can be taken affine:
\index{affine!closure}
\index{flow!matching}
	\[
		v_t(x)=\dot\mean_t+A_t(x-\mean_t),
		\qquad
		A_t\cov_t+\cov_t A_t=\dot\cov_t.
	\]
	In the isotropic case $\cov_t=s_t^2\Id$, this reduces to the transparent formula
	\[
		v_t(x)=\dot\mean_t+\frac{\dot s_t}{s_t}(x-\mean_t).
	\]
	For instance, the diffusion noising path
	\[
		X_t=a_tX_0+\sigma_t Z,\qquad Z\sim\Gaussian(0,\Id),
	\]
	has $\mean_t=a_t\mean_0$ and $\cov_t=a_t^2\cov_0+\sigma_t^2\Id$. Thus, in the Gaussian case, diffusion paths and flow-matching paths reduce to the same mean-covariance bookkeeping, although the corresponding training objectives are different.
\index{flow!matching}
\end{example}

\begin{example}[Gaussian kernel drifting]
\index{kernel!Gaussian}
\index{Gaussian!kernel drifting}
	Let the target be $\beta=\Gaussian(\bar\mean,\bar\cov)$ and assume $\mu_t=\Gaussian(\mean_t,\cov_t)$. For the Gaussian kernel
	\[
		K_\epsilon(x,y)=\exp(-\norm{x-y}^2/(2\epsilon)),
	\]
	the normalized field~\eqref{eq-normalized-kernel-drift} satisfies
	\[
		B_\epsilon[\mu_t](x)
		=
		-\epsilon(\cov_t+\epsilon\Id)^{-1}(x-\mean_t).
	\]
	Thus the drifting velocity~\eqref{eq-cross-minus-self-drift} is affine and preserves Gaussianity. With
	\[
		A_t=(\cov_t+\epsilon\Id)^{-1},
		\qquad
		\bar A=(\bar\cov+\epsilon\Id)^{-1},
	\]
	the ODE is
	\[
		\dot\mean_t=\epsilon\bar A(\bar\mean-\mean_t),
		\qquad
		\dot\cov_t=\epsilon\bigl((A_t-\bar A)\cov_t+\cov_t(A_t-\bar A)\bigr).
	\]
	This finite-dimensional model explains the stabilizing role of the self-normalized repulsion term in drifting: without it, the covariance equation loses the $A_t\cov_t+\cov_tA_t$ contribution.
\end{example}

\begin{example}[Gaussian closure of attention dynamics]
\index{Gaussian!closure}
\index{attention}
	For the transformer PDE, assume $\alpha=\Gaussian(\mean,\cov)$. Since exponential tilting preserves Gaussianity,
\index{transformer}
	\[
		\frac{\int e^{\dotp{Qx}{Ky}}\,y\,\d\alpha(y)}
		     {\int e^{\dotp{Qx}{Kz}}\,\d\alpha(z)}
		=
		\mean+\cov K^\top Qx.
	\]
	Therefore
	\[
		\Gamma_\theta[\alpha](x)=V\mean+V\cov K^\top Qx
	\]
	is affine. The Gaussian token law is preserved and satisfies
	\[
		\dot\mean_t=(V_t+V_t\cov_tK_t^\top Q_t)\mean_t,
		\qquad
		\dot\cov_t=B_t\cov_t+\cov_tB_t^\top,
		\qquad
		B_t=V_t\cov_tK_t^\top Q_t.
	\]
	When $V_t=Q_t^\top K_t$, the matrix $B_t=Q_t^\top K_t\cov_tK_t^\top Q_t$ is symmetric positive semidefinite, matching the gradient-field case mentioned above.
	This closure is not a convergence theorem for trained transformers. It is instead a tractable model of how attention can shear, amplify or contract a cloud of tokens through its covariance.
\end{example}

\paragraph{Contractive Gaussian projection.}
\index{Gaussian!projection}
\index{Gaussian!contractive projection}

The preceding examples show when Gaussianity is preserved or imposed by projection. Gelbrich's inequality~\cite{gelbrich1990formula} gives a useful variational explanation: replacing a measure by the Gaussian with the same first two moments cannot increase its Wasserstein distance to another similarly projected measure.
\index{Wasserstein!distance}

\begin{thm}[Gelbrich theorem]\label{thm-gelbrich-projection}
\index{Gelbrich theorem}
	For $\mu\in\Pp_2(\RR^d)$, let
	\[
		\mathcal R\mu\eqdef \Gaussian(\mean_\mu,\cov_\mu)
	\]
	be the Gaussian with the same mean and covariance as $\mu$. Then
	\[
		\Wass_2^2(\mathcal R\mu,\mathcal R\nu)
		=
		\norm{\mean_\mu-\mean_\nu}^2+\Bb^2(\cov_\mu,\cov_\nu)
		\leq
		\Wass_2^2(\mu,\nu).
	\]
\end{thm}

\begin{proof}
	Take any coupling $(X,Y)$ of $\mu$ and $\nu$, center the variables, and write $C=\EE[(X-\mean_\mu)(Y-\mean_\nu)^\top]$. In the positive definite case, positivity of the block covariance matrix implies the factorization $C=\cov_\mu^{1/2}K\cov_\nu^{1/2}$ with $\norm{K}_{\mathrm{op}}\leq1$, and therefore, by operator/nuclear norm duality,
\index{covariance!matrix}
	\[
		\tr C\leq \tr\left((\cov_\mu^{1/2}\cov_\nu\cov_\mu^{1/2})^{1/2}\right).
	\]
	The semidefinite case follows by adding $\eta\Id$ to both covariance matrices and letting $\eta\downarrow0$.
	Expanding $\EE\norm{X-Y}^2$ gives the lower bound
	\[
		\EE\norm{X-Y}^2
		\geq
		\norm{\mean_\mu-\mean_\nu}^2+\Bb^2(\cov_\mu,\cov_\nu).
	\]
	Taking the infimum over couplings proves the inequality, while equality for Gaussian laws is Proposition~\ref{prop-gaussian-w2-bures}.
\end{proof}

The following preservation criterion is a direct consequence of Gelbrich's theorem and was explained to us by Hugo Lavenant. It says that a functional which does not increase under moment-matched Gaussian projection admits Gaussian minimizing movements from Gaussian initial data.
\index{moment matching}
\index{Lavenant criterion}
\index{Gaussian!projection}
\index{minimizing movement}
\index{Gelbrich theorem}

\begin{thm}[Hugo Lavenant Gaussian-preservation criterion]\label{thm-lavenant-gaussian-preserving-jko}
\index{Lavenant criterion}
\index{Gaussian!preservation}
	Let $F:\Pp_2(\RR^d)\to(-\infty,+\infty]$ satisfy
	\[
		F(\mathcal R\mu)\leq F(\mu)
		\qquad\forall\mu\in\Pp_2(\RR^d),
	\]
	with $\mathcal R$ defined in Theorem~\ref{thm-gelbrich-projection}. If $\gamma$ is Gaussian and $\nu$ minimizes the JKO step
	\[
		\eta\mapsto F(\eta)+\frac1{2\tau}\Wass_2^2(\gamma,\eta),
	\]
	then $\mathcal R\nu$ is also a minimizer. If the JKO minimizer is unique, it is Gaussian. Thus any unique Wasserstein gradient flow obtained as the limit of this JKO scheme preserves Gaussian initial data.
\index{Wasserstein!gradient}
\index{gradient!flow}
\index{Wasserstein!gradient flow}
\index{JKO scheme}
\end{thm}

\begin{proof}
	For the JKO claim, $\mathcal R\gamma=\gamma$ because $\gamma$ is Gaussian. Hence, for any competitor $\eta$,
	\[
		F(\mathcal R\eta)+\frac1{2\tau}\Wass_2^2(\gamma,\mathcal R\eta)
		\leq
		F(\eta)+\frac1{2\tau}\Wass_2^2(\gamma,\eta).
	\]
	Applying this to a minimizer $\eta=\nu$ shows that $\mathcal R\nu$ is again a minimizer. Uniqueness forces $\nu=\mathcal R\nu$.
\end{proof}

\begin{rem}[Gaussian barycenters from contraction]
\index{Gaussian!barycenter}
	The same projection argument also explains why quadratic Wasserstein barycenters of Gaussian measures are Gaussian. If $\be_s$ are Gaussian and
\index{Wasserstein!barycenter}
\index{Gaussian!measure}
	\[
		F(\mu)=\sum_s\la_s\Wass_2^2(\mu,\be_s),
	\]
	then $\mathcal R\be_s=\be_s$, and Theorem~\ref{thm-gelbrich-projection} gives $F(\mathcal R\mu)\leq F(\mu)$. Thus the moment-matched Gaussian projection of any barycenter is again a barycenter; when the barycenter is unique, it must itself be Gaussian. This is the contraction viewpoint behind Corollary~\ref{cor-gaussian-discrete-barycenters}.
\index{moment matching}
\index{Gaussian!projection}
\end{rem}

\chapter*{Conclusion}
\addcontentsline{toc}{chapter}{Conclusion}

Optimal transport is useful because it keeps several viewpoints active at once. It is a linear program over couplings, a duality theory for potentials, a geometry on probability measures, a source of PDEs and gradient flows, and a computational toolbox built around linear programming, Sinkhorn scaling and low-dimensional projections. These viewpoints reinforce each other: Brenier maps explain geodesics, geodesic convexity explains convergence of flows, entropic regularization turns transport into scalable differentiable losses, and dual norms or sliced variants reveal what is gained and lost when OT is simplified.
\index{Sinkhorn!scaling}
\index{Brenier!map}
\index{dual!norm}
\index{linear programming}
\index{entropic!regularization}
\index{probability measure}
\index{convexity!geodesic}
\index{gradient!flow}

For machine learning, this interplay is especially stimulating. Modern generative modeling, diffusion and flow matching, inverse problems, robust optimization, and even continuous-depth views of transformers all ask for ways to move, compare or learn distributions in high dimension. These applications do not merely consume existing OT theory; they create difficult mathematical questions because empirical measures are singular, dimensions are large, models are parametrized and non-convex, and computational approximations are inseparable from statistical error. The strength of OT is precisely that it gives a common language for these tensions, while still leaving enough open structure for new mathematics to be needed.
\index{empirical!measure}
\index{flow!matching}
\index{generative model}
\index{transformer}

\phantomsection
\section*{Acknowledgements}
\addcontentsline{toc}{section}{Acknowledgements}

This work was supported by the European Research Council (ERC project WOLF) and by the French government under the management of Agence Nationale de la Recherche as part of the ``France 2030'' program, reference ANR-23-IACL-0008 (PRAIRIE-PSAI).


\bibliographystyle{plain}
\bibliography{all}

\clearpage
\appendix

\chapter{Notation Table}
\label{app-notation-table}

This appendix collects the main notation used throughout the book. The last
column points to the first section, equation, definition, proposition or theorem
where the notation is defined or first used in a mathematically meaningful way.

{\small
\setlength{\tabcolsep}{4pt}
\renewcommand{\arraystretch}{1.16}
\begin{longtable}{>{\raggedright\arraybackslash}p{.24\textwidth}
                  >{\raggedright\arraybackslash}p{.45\textwidth}
                  >{\raggedright\arraybackslash}p{.24\textwidth}}
\hline
\textbf{Notation} & \textbf{Meaning} & \textbf{First reference} \\
\hline
\endfirsthead
\hline
\textbf{Notation} & \textbf{Meaning} & \textbf{First reference} \\
\hline
\endhead
\hline
\endfoot

\multicolumn{3}{l}{\textbf{Ambient spaces, measures and elementary objects}}\\
\hline
$\RR^d$ & Euclidean ambient space. & Section~\ref{sec-measures} \\
$\X,\Y$ & Source and target spaces. & Eq.~\eqref{eq-monge-continuous} \\
$\Mm(\X)$ & Finite signed Radon measures on $\X$. & Section~\ref{sec-measures} \\
$\Mm_+(\X),\Mm_+^1(\X)$ & Positive finite measures and probability measures. & Section~\ref{sec-measures} \\
$\Pp(\X),\Pp_p(\X)$ & Probability measures, with finite $p$-moment for $\Pp_p$. & Section~\ref{sec-kantorovich-continuous} \\
$\simplex_n$ & Probability simplex of histograms of length $n$. & Definition~\ref{def-probability-simplex} \\
$\de_x$ & Dirac mass at $x$. & Definition~\ref{def-discrete-measure} \\
$\al,\be,\ga$ & Source, target and auxiliary probability measures. & Eq.~\eqref{eq-monge-continuous} \\
$\density{\al}$ & Density of $\al$ with respect to a reference measure. & Definition~\ref{def-relative-density} \\
$\d\al,\d x$ & Integration against $\al$ and against Lebesgue measure. & Section~\ref{sec-measures} \\
$\EE$ & Expectation of a random variable. & Section~\ref{sec-measures} \\
$\supp(\pi)$ & Topological support of a measure. & Definition~\ref{def:support} \\
$\Supp(\b)$ & Index support of a histogram. & Eq.~\eqref{eq-discr-diverg} \\
$\Cc(\X)$ & Continuous real-valued functions on $\X$. & Section~\ref{sec-measures} \\
$\norm{\cdot}$ & Euclidean norm or the norm indicated by a subscript. & Chapter~\ref{sec-matching} \\
$\dotp{\cdot}{\cdot}$ & Euclidean/Frobenius pairing or measure-function pairing. & Section~\ref{sec-measures} \\

\hline
\multicolumn{3}{l}{\textbf{Discrete matching and discrete Kantorovich OT}}\\
\hline
$(x_i)_i,(y_j)_j$ & Source and target point clouds. & Eq.~\eqref{eq-optimal-assignment} \\
$\C=(\C_{i,j})$ & Cost matrix between source and target points. & Eq.~\eqref{eq-optimal-assignment} \\
$\sigma\in\Perm(n)$ & Permutation encoding a one-to-one matching. & Eq.~\eqref{eq-optimal-assignment} \\
$P_\sigma,\mathcal P_n^{\mathrm{perm}}$ & Permutation matrix and the set of all such matrices. & Definition~\ref{def-permutation-matrices} \\
$\mathcal B_n$ & Birkhoff polytope of bistochastic matrices. & Definition~\ref{def-birkhoff-polytope} \\
$\a,\b$ & Discrete probability histograms. & Eq.~\eqref{eq-discr-couplings} \\
$\P$ & Discrete transport/coupling matrix. & Eq.~\eqref{eq-discr-couplings} \\
$\CouplingsD(\a,\b)$ & Polytope of discrete couplings with marginals $\a,\b$. & Eq.~\eqref{eq-discr-couplings} \\
$\ones_n,\transp{\P}$ & All-ones vector and transpose of $\P$. & Eq.~\eqref{eq-discr-couplings} \\
$\MKD_\C(\a,\b)$ & Discrete Kantorovich optimal value with cost $\C$. & Eq.~\eqref{eq-kanto-discr} \\
$\distD$ & Ground distance matrix for discrete Wasserstein distances. & Definition~\ref{def-discrete-wasserstein-distance} \\
$\WassD_p(\a,\b)$ & Discrete $p$-Wasserstein distance. & Definition~\ref{def-discrete-wasserstein-distance} \\

\hline
\multicolumn{3}{l}{\textbf{Monge maps, one-dimensional OT and Gaussians}}\\
\hline
$T,\T$ & Transport map. & Eq.~\eqref{eq-monge-continuous} \\
$T_\sharp\al$ & Push-forward of $\al$ by $T$. & Definition~\ref{defn-pushfwd} \\
$T^\sharp g$ & Pullback of a test function, $T^\sharp g=g\circ T$. & Remark~\ref{rem-pullback-pushforward} \\
$\Id$ & Identity map. & Definition~\ref{defn-pushfwd} \\
$\tilde\Wass_p$ & Directed Monge transport distance. & Eq.~\eqref{eq-monge-distance} \\
$\nabla\phi$ & Brenier map for quadratic cost. & Theorem~\ref{thm-brenier} \\
$\cumul{\al}$ & Cumulative distribution function of a 1-D measure. & Eq.~\eqref{eq-cumul-defn} \\
$\cumul{\al}^{-1}$ & Quantile function of a 1-D measure. & Eq.~\eqref{eq-OT-map-1d} \\
$\Gaussian(\mean,\cov)$ & Gaussian law with mean $\mean$ and covariance $\cov$. & Eq.~\eqref{eq-gauss-pf} \\
$\mean_\al,\cov_\al$ & Mean and covariance of a Gaussian measure $\al$. & Eq.~\eqref{eq-dist-gauss} \\
$\Bb(\cov_\al,\cov_\be)$ & Bures covariance distance. & Definition~\ref{def-bures-metric} \\
$\tr(\cov)$ & Trace of a matrix. & Eq.~\eqref{eq-dist-gauss} \\

\hline
\multicolumn{3}{l}{\textbf{Continuous Kantorovich OT and Wasserstein distances}}\\
\hline
$\pi$ & Coupling or transport plan. & Definition~\ref{def-continuous-couplings} \\
$\Couplings(\al,\be)$ & Set of couplings between $\al$ and $\be$. & Eq.~\eqref{eq-coupling-generic} \\
$\MK_\c(\al,\be)$ & Kantorovich optimal value with ground cost $\c$. & Eq.~\eqref{eq-mk-generic} \\
$\dist$ & Ground distance on the underlying metric space. & Eq.~\eqref{eq-defn-wass-dist} \\
$\Wass_p(\al,\be)$ & $p$-Wasserstein distance. & Definition~\ref{def-wasserstein-distance} \\
$\Wass_\infty(\al,\be)$ & Worst-displacement Wasserstein distance. & Eq.~\eqref{eq-wass-infty} \\
$\mathfrak A,\mathfrak B$ & Probability laws over probability measures. & Eq.~\eqref{eq-wow-parametric-law} \\
$\bar\alpha_{\mathfrak A}$ & Collapsed mixture associated with a law over measures. & Definition~\ref{def-collapsed-barycentric-mixture} \\
$\mathbb W_2$ & Wasserstein distance on the Wasserstein space. & Eq.~\eqref{eq-wow-distance} \\
$\Gamma$ & $c$-cyclically monotone subset of $\X\times\Y$. & Definition~\ref{def:ccm} \\
$\rho$ & Glued or composed coupling. & Lemma~\ref{lem-gluing-general} \\
$\rightharpoonup$ & Weak$^*$ convergence of measures. & Definition~\ref{dfn-weak-conv} \\
$\TV,\norm{\cdot}_{\TV}$ & Total variation divergence/norm. & Section~\ref{sec-measures} \\

\hline
\multicolumn{3}{l}{\textbf{Duality, transforms and weak norms}}\\
\hline
$\fD,\gD$ & Discrete dual potentials. & Eq.~\eqref{eq-dual} \\
$\f,\g$ & Continuous dual potentials. & Eq.~\eqref{eq-dual-generic} \\
$\PotentialsD(\a,\b)$ & Feasible set of discrete dual potentials. & Eq.~\eqref{eq-feasible-potential} \\
$\Potentials(\c)$ & Feasible set of continuous dual potentials. & Eq.~\eqref{eq-dfn-pot-dual} \\
$f^c,g^c$ & $c$-transform of a potential. & Definition~\ref{def-c-transform} \\
$\Laguerre_j(\gD)$ & Laguerre/power cell in semi-discrete OT. & Eq.~\eqref{eq-laguerre-cells} \\
$\Qq_m(\al)$ & Optimal $m$-point quantization error. & Eq.~\eqref{eq-optimal-quantization} \\
$\Lip(f)$ & Lipschitz constant of $f$. & Eq.~\eqref{eq-lip-constant} \\
$\Wass_1$ & Kantorovich--Rubinstein distance/norm. & Eq.~\eqref{eq-w1-metric} \\
$\Wass_{1,G}$ & Graph Wasserstein-1/transshipment distance. & Proposition~\ref{prop-graph-w1-beckmann} \\
$d_G,\nabla_G,\operatorname{div}_G$ & Graph geodesic distance, gradient and divergence. & Proposition~\ref{prop-graph-w1-beckmann} \\
$\norm{\cdot}_B$ & Dual norm induced by a discriminator class $B$. & Eq.~\eqref{eq-dual-norm-cont} \\
$\RKHS,\Krkhs$ & Reproducing kernel Hilbert space and its kernel. & Definition~\ref{def-kernel-mmd-norm} \\
$\MMD_k$ & Maximum mean discrepancy/kernel norm for $k$. & Definition~\ref{def-kernel-mmd-norm} \\
$\Divergm_\phi,\DivergmD_\phi$ & Continuous and discrete $\phi$-divergences. & Eq.~\eqref{eq-phi-div} \\
$\phi'_\infty$ & Recession slope of an entropy function. & Definition~\ref{def_entropy} \\
$\phi^\star$ & Legendre transform of $\phi$. & Eq.~\eqref{eq-legendre} \\
$\KL,\KLD$ & Continuous and discrete Kullback--Leibler divergences. & Definitions~\ref{def-discrete-relative-entropy}, \ref{def-measure-relative-entropy} \\
$\Hellinger$ & Hellinger divergence/distance. & Section~\ref{sec-phi-div} \\
$\JS$ & Jensen--Shannon divergence. & Section~\ref{sec-phi-div} \\

\hline
\multicolumn{3}{l}{\textbf{Entropic regularization and Sinkhorn algorithms}}\\
\hline
$\epsilon$ & Entropic regularization strength. & Eq.~\eqref{eq-regularized-discr} \\
$\HD(\P)$ & Shannon--Boltzmann entropy of a matrix. & Definition~\ref{def-discrete-shannon-boltzmann-entropy} \\
$\MKD_\C^\epsilon(\a,\b)$ & Discrete entropic OT value. & Eq.~\eqref{eq-regularized-discr} \\
$\MK_\c^\epsilon(\al,\be)$ & Continuous entropic OT value. & Eq.~\eqref{eq-entropic-generic} \\
$\K$ & Gibbs kernel $e^{-\C/\epsilon}$. & Eq.~\eqref{eq-scaling-form} \\
$\uD,\vD$ & Left and right Sinkhorn scalings. & Eq.~\eqref{eq-scaling-form} \\
$\diag(\uD)\K\diag(\vD)$ & Scaling form of the entropic coupling. & Eq.~\eqref{eq-sink-matrix} \\
$\odot$ & Entrywise product of vectors. & Eq.~\eqref{eq-dualsinkhorn-constraints2} \\
$\it{\uD},\itt{\uD}$ & Current and next Sinkhorn iterates. & Eq.~\eqref{eq-sinkhorn} \\
$\Hilbert$ & Hilbert projective metric on positive vectors. & Definition~\ref{def-hilbert-metric} \\
$\Proj^\KLD$ & KL/Bregman projection. & Eq.~\eqref{eq-kl-proj} \\
$\bar\MK_\c^\epsilon(\al,\be)$ & Debiased Sinkhorn divergence. & Eq.~\eqref{eq-sinkhorn-divergence} \\

\hline
\multicolumn{3}{l}{\textbf{Extensions of OT}}\\
\hline
$\psi_1,\psi_2$ & Entropy functions penalizing marginal mismatch. & Eq.~\eqref{eq-unbalanced-primal} \\
$\UW_c,\UW_{c,\tau}$ & Relaxed unbalanced OT value with marginal penalties. & Eq.~\eqref{eq-unbalanced-primal} \\
$L_c$ & Reverse-formulation local unbalanced cost. & Eq.~\eqref{eq-unbalanced-reverse-local-cost} \\
$H_c$ & Homogeneous perspective of the local cost $L_c$. & Eq.~\eqref{eq-unbalanced-homogeneous-local-cost} \\
$\HW$ & Homogeneous unbalanced formulation. & Eq.~\eqref{eq-homogeneous} \\
$\mathfrak{C}[\X]$ & Cone over the metric space $\X$. & Section~\ref{sec-unbalanced} \\
$\CW,\CW_\kappa$ & Cone formulation of unbalanced OT, with $\CW_\kappa$ using growth scale $\kappa$. & Theorem~\ref{thm-cone-unbalanced-ot}, Eq.~\eqref{eq-dynamic-unbalanced-ot} \\
$\mathcal A_\kappa$ & Dynamic unbalanced perspective action for transport and growth. & Eq.~\eqref{eq-dynamic-unbalanced-ot} \\
$\WFR_\kappa$ & Wasserstein--Fisher--Rao dynamic distance with growth scale $\kappa$. & Eq.~\eqref{eq-dynamic-unbalanced-ot} \\
$\be_s,\la_s$ & Input measures and weights in barycenter problems. & Eq.~\eqref{eq-barycenter-generic} \\
$\al^\star$ & Optimal measure, often a barycenter. & Section~\ref{sec-barycenters} \\
$\SW_p$ & Sliced Wasserstein distance. & Definition~\ref{def-sliced-wasserstein} \\
$\Sphere^{d-1}$ & Unit sphere of projection directions. & Definition~\ref{def-sliced-wasserstein} \\
$P_\theta$ & Projection on direction $\theta$. & Definition~\ref{def-sliced-wasserstein} \\
$\MaxSW_p$ & Max-sliced Wasserstein distance. & Definition~\ref{def-sliced-variants} \\
$\operatorname{MSWGG}_2$ & Min-sliced lifted-plan upper bound on $\Wass_2$. & Section~\ref{sec-sliced-wasserstein} \\
$\SW_{p,k},\MaxSW_{p,k}$ & Average and max Wasserstein distances over $k$-dimensional projections. & Definition~\ref{def-subspace-sliced-wasserstein} \\
$\Wass_\gamma$ & Spectral Wasserstein distance associated with a matrix gauge $\gamma$. & Eq.~\eqref{eq-spectral-wasserstein} \\
$\mathcal B_\gamma$ & Polar set defining the robust projected form of $\Wass_\gamma$. & Eq.~\eqref{eq-spectral-polar-set} \\
$\Wass_{2,A}$ & Quadratic Wasserstein pseudodistance after projection by $A^{1/2}$. & Eq.~\eqref{eq-quadratic-projected-cost} \\
$\SRW_{2,k}$ & Paty--Cuturi subspace robust Wasserstein distance. & Section~\ref{sec-spectral-subspace-wasserstein} \\
$\LOT_\rho$ & Linear OT distance around reference $\rho$. & Eq.~\eqref{eq-lot-embedding} \\
$\bar T_\pi$ & Barycentric projection of a coupling $\pi$. & Eq.~\eqref{eq-barycentric-projection} \\
$\bar\beta_\pi$ & Pushforward of $\alpha$ by the barycentric projection. & Eq.~\eqref{eq-barycentric-projection} \\
$\WOT_C$ & Weak OT value with conditional-law cost $C$. & Eq.~\eqref{eq-weak-ot} \\
$g^C$ & Weak $C$-transform in weak OT duality. & Proposition~\ref{prop-weak-ot-duality} \\
$u_t,V_t$ & Positive vector-valued density and spatial flux. & Eqs.~\eqref{eq-vector-valued-bb}, \eqref{eq-vector-valued-continuity} \\
$\mathcal W_{\Phi}$ & Dynamic vector-valued BB-type cost. & Eq.~\eqref{eq-vector-valued-bb} \\
$\distD,\distD'$ & Intra-domain distance matrices in discrete GW. & Eq.~\eqref{eq-gw-def} \\
$\De$ & Discrepancy between intra-domain distances. & Eq.~\eqref{eq-gw-def} \\
$\GWD$ & Discrete Gromov--Wasserstein cost. & Eq.~\eqref{eq-gw-def} \\
$\XX,\YY$ & Metric-measure spaces. & Definition~\ref{def-metric-measure-space} \\
$\GW$ & Continuous Gromov--Wasserstein distance. & Eq.~\eqref{eq-gw-generic} \\
$d_{\mathrm H},d_{\mathrm{GH}}$ & Hausdorff and Gromov--Hausdorff distances. & Section~\ref{sec-gromov-wasserstein} \\
$\operatorname{FGW}_{\lambda,p}$ & Fused Gromov--Wasserstein distance. & Section~\ref{sec-gromov-wasserstein} \\
$\mathbb S^m,\mathbb S_+^m$ & Real symmetric matrices and their positive semidefinite cone. & Definition~\ref{def-positive-matrix-valued-measure} \\
$A_t,P_t$ & Positive matrix-valued density and spatial matrix flux. & Eqs.~\eqref{eq-matrix-valued-bb}, \eqref{eq-matrix-valued-continuity} \\
$\mathcal W_{\mathrm{mat}}$ & Conservative matrix-valued BB-type cost. & Eq.~\eqref{eq-matrix-valued-bb} \\
$\mathbb H_n,\mathbb H_n^+,\mathbb H_n^{+,1}$ & Hermitian matrices, positive semidefinite Hermitian matrices and density matrices. & Definition~\ref{def-hermitian-density-matrices} \\
$\operatorname{Tr}_A,\operatorname{Tr}_B$ & Partial traces of a bipartite matrix. & Eq.~\eqref{eq-qot-partial-traces} \\
$\mathrm{QOT}_C(A,B)$ & Finite-dimensional quantum OT value with cost observable $C$. & Eq.~\eqref{eq-qot-primal} \\
$\mathrm{QOT}_C^\epsilon(A,B)$ & Entropically regularized quantum OT value. & Eq.~\eqref{eq-qot-entropic-primal} \\
$T_e(F,G),T_s(F,G)$ & Exact Gibbs coupling and symmetric Gurvits-scaling surrogate. & Eqs.~\eqref{eq-qot-gibbs-coupling}, \eqref{eq-qot-symmetric-scaling} \\

\hline
\multicolumn{3}{l}{\textbf{Dynamic OT and Wasserstein gradient flows}}\\
\hline
$\alpha_t$ & Time-dependent curve of probability measures. & Eq.~\eqref{eq:eulerian-advection} \\
$v_t$ & Eulerian velocity field transporting $\alpha_t$. & Eq.~\eqref{eq:eulerian-advection} \\
$T_t$ & Lagrangian particle flow map. & Eq.~\eqref{eq:lagrangian-advection} \\
$P_t$ & Interpolant map in flow matching; later also path evaluation. & Eq.~\eqref{eq:interp-coupling} \\
$\Wass_2^2$ via action & Benamou--Brenier dynamic formulation. & Eq.~\eqref{eq:benamou-brenier} \\
$\Wgrad f(\alpha)$ & Wasserstein gradient of a functional. & Proposition~\ref{prop-formal-wass-gradient} \\
$\delta f(\alpha)$ & First variation of $f$ at $\alpha$. & Proposition~\ref{prop-formal-wass-gradient} \\
$\partial_t\alpha+\diverg(\alpha v)=0$ & Continuity equation. & Eq.~\eqref{eq:eulerian-advection} \\
$\alpha_{t+\tau}$ & One JKO/minimizing-movement step. & Eq.~\eqref{eq:jko-discr} \\
$\Ss=C([0,1];\RR^d)$ & Path space in the superposition formulation. & Chapter~\ref{sec-wasserstein-flows} \\
\end{longtable}
}

\cleardoublepage
\phantomsection
\addcontentsline{toc}{chapter}{Index}
\printindex

\end{document}